\documentclass[onecolumn, 12 pt, doublespace, fullpage, letterpaper]{report}

\usepackage{amssymb}
\usepackage{textcomp}
\usepackage{graphicx}
\usepackage{graphics}
\usepackage{epsfig}
\usepackage{epstopdf}
\usepackage{float}
\usepackage{color}
\usepackage[cmex10]{amsmath}
\usepackage{latexsym,amsfonts}
\usepackage{amsthm}
\usepackage{url}
\usepackage{longtable}
\usepackage[figuresright]{rotating}
\usepackage{listings}
\usepackage{etoolbox}
\usepackage{sectsty}

\chapterfont{\fontsize{14}{15}\selectfont}
\sectionfont{\fontsize{14}{15}\selectfont}
\subsectionfont{\fontsize{14}{15}\selectfont}
\subsubsectionfont{\fontsize{14}{15}\selectfont}

\usepackage{fancyhdr}       
\fancyheadoffset[L]{0.25in}

\makeatletter
\patchcmd{\@makechapterhead}{50\p@}{20pt}{}{}
\patchcmd{\@makeschapterhead}{50\p@}{20pt}{}{}
\makeatother

\fancypagestyle{plain}{
\fancyhf{} 
\fancyhead[C]{\thepage} 

}

\usepackage{setspace}
\usepackage{subfigure}

   \usepackage{ifthen,xkeyval,xfor,amsgen}
   \usepackage[acronym,toc, nogroupskip]{glossaries}
   \newglossary[slg]{symbols}{syi}{sbl}{List of Symbols}
 
   \makeglossaries

\newglossaryentry{symb:Pi}{
name=$\pi$, type=symbols,
description=A mathematical constant whose value is the ratio of any circle's circumference to its diameter,
sort=symbolpi
}

\newglossaryentry{symb:Phi}{
name=$\varphi$, type=symbols,
description=An angle,
sort=symbolphi
}

\newglossaryentry{symb:Lambda}{
name=$\lambda$, type=symbols,
description=Lambda indicates usually an eigenvalue in linear algebra,
sort=symbollambda
}

\newacronym{toc}{ToC}{Table of Contents}
\newacronym{los}{LoS}{List of Symbols}
\newacronym{loa}{LoA}{List of Abbreviations}
\newacronym{phd}{PhD}{Doctoral}
\newacronym{MS}{MS}{Masters}
\newacronym{M$}{MS}{Microsoft}
\newacronym{CD}{CD}{Compact Disc}
\newacronym{kaust}{KAUST}{King Abdullah University of Science and Technology}

\newacronym{AD}{AD}{Active Directory\protect\glsadd{glos:AD}}

\newglossaryentry{glos:AD}{
name=Active Directory,
description={Active Directory is the directory service for Windows based networks, that allows central organization and administration of any network resource. It allows a single-sign-on concept independent from network topologies or network protocols. As a prerequisite you need a Windows Server acting as Domain Controller. This computer stores all necessary data, e.g.~usernames and corresponding passwords}
}

\newglossaryentry{glos:RespF}{name=response file, description={A file 
that allows unattended software installation}}

\usepackage{hyperref} 

\usepackage{amsmath,amssymb,amsthm}
\usepackage{color,mathtools}

\usepackage{multirow}
\usepackage{enumitem}
\usepackage{nicefrac}

\usepackage{tablefootnote}

\usepackage{algorithm}
\usepackage{algorithmic}
\usepackage{nicefrac}
\usepackage{xspace}
\usepackage{subfigure}

\usepackage{bm}

\usepackage{amssymb,amsmath,amsthm,dsfont}

\providecommand{\lin}[1]{\ensuremath{\left\langle #1 \right\rangle}}

  \providecommand{\R}{\mathbb{R}} 
  \providecommand{\N}{\mathbb{N}} 
  
  \renewcommand{\epsilon}{\varepsilon}

  \DeclareMathOperator*{\argmin}{arg\,min}

  \let\lll\ll
  \renewcommand{\ll}{\mathbf{l}}

  \providecommand{\tt}{\mathbf{t}}

  \providecommand{\mA}{\mathbf{A}}
  \providecommand{\mB}{\mathbf{B}}
  
  \providecommand{\mD}{\mathbf{D}}

  \providecommand{\mG}{\mathbf{G}}
  \providecommand{\mH}{\mathbf{H}}

  \providecommand{\mP}{\mathbf{P}}
  \providecommand{\mQ}{\mathbf{Q}}
  \providecommand{\mR}{\mathbf{R}}
  \providecommand{\mS}{\mathbf{S}}

\theoremstyle{plain}
\newtheorem{theorem}{Theorem}[section]

\newtheorem{lemma}[theorem]{Lemma}
\newtheorem{corollary}[theorem]{Corollary}
\theoremstyle{definition}
\newtheorem{definition}[theorem]{Definition}
\newtheorem{assumption}[theorem]{Assumption}
\theoremstyle{remark}
\newtheorem{remark}[theorem]{Remark}
\newtheorem{example}[theorem]{Example}

\DeclarePairedDelimiterX{\inp}[2]{\langle}{\rangle}{#1, #2}
\DeclarePairedDelimiterX{\abs}[1]{\lvert}{\rvert}{#1}
\DeclarePairedDelimiterX{\norm}[1]{\lVert}{\rVert}{#1}
\DeclarePairedDelimiterX{\cbr}[1]{\{}{\}}{#1} 
\DeclarePairedDelimiterX{\rbr}[1]{(}{)}{#1} 
\DeclarePairedDelimiterX{\sbr}[1]{[}{]}{#1} 

\definecolor{mydarkblue}{rgb}{0,0.08,0.45}

\usepackage{xcolor}
\usepackage{color}
\usepackage{graphicx}

\usepackage{verbatim}
\usepackage{multirow}

\newcommand{\cA}{{\cal A}}

\newcommand{\cC}{{\cal C}}
\newcommand{\cD}{{\cal D}}
\newcommand{\cE}{{\cal E}}

\newcommand{\cM}{{\cal M}}

\newcommand{\cO}{{\cal O}}
\newcommand{\cP}{{\cal P}}

\newcommand{\cS}{{\cal S}}

\newcommand{\cU}{{\cal U}}

\newcommand{\cX}{{\cal X}}

\newcommand{\eqdef}{\overset{\text{def}}{=}} 
\newcommand{\dotprod}[2]{\left\langle #1,#2\right\rangle} 

\DeclareMathOperator{\Var}{Var}         
\DeclareMathOperator{\Prob}{Prob}

\DeclareMathOperator{\signum}{sign}     
\DeclareMathOperator{\sign}{sign}     
\DeclareMathOperator{\prox}{prox}       

\newcommand{\Diag}[1]{\mathbf{Diag}\left( #1\right)}

\providecommand{\trace}[1]{{\rm Trace}\left( #1\right)}

\newcommand{\Exp}[1]{{\bf E}\left[#1\right] }    
\newcommand{\E}[1]{{\bf E}\left[#1\right] } 
\newcommand{\EE}[2]{{\bf E}_{#1}\left[#2\right] }

\usepackage{url}

\usepackage{xspace}

\newcommand{\NC}{\cC_{\rm nat}}
\newcommand{\Qint}{\cC_{\rm int}}
\newcommand{\ND}{\cD_{\rm nat}^{p,s}}
\newcommand{\GD}{\cD_{\rm gen}^{\cC ,p,s}}
\newcommand{\SD}{\cD_{\rm sta}^{p,s}}
\newcommand{\SDs}[1]{\cD_{\rm sta}^{p,#1}}

\DeclarePairedDelimiter\ceil{\lceil}{\rceil}
\DeclarePairedDelimiter\floor{\lfloor}{\rfloor}

\newcommand{\sgd}{{\tt SGD}\xspace}
\newcommand{\fedavg}{{\tt FedAvg}\xspace}
\newcommand{\fedmin}{{\tt FedAvgMin}\xspace}
\newcommand{\fedmean}{{\tt FedAvgMean}\xspace}
\newcommand{\fedavgrr}{{\tt FedAvgRR}\xspace}
\newcommand{\fedshuffle}{{\tt FedShuffle}\xspace}
\newcommand{\fedshufflemvr}{{\tt FedShuffleMVR}\xspace}
\newcommand{\fedgen}{{\tt FedShuffleGen}\xspace}
\newcommand{\mvr}{{\tt MVR}\xspace}
\newcommand{\fedmvr}{{\tt FedShuffleMVR}\xspace}
\newcommand{\fednova}{{\tt FedNova}\xspace}
\newcommand{\fedlin}{{\tt FedLin}\xspace}
\newcommand{\fednovarr}{{\tt FedNovaRR}\xspace}
\newcommand{\fedrr}{{\tt FedRR}\xspace}
\newcommand{\locrr}{{\tt LocalRR}\xspace}
\newcommand{\gd}{{\tt GD}\xspace}

\newcommand{\tool}{FjORD\xspace}

\newcommand{\Dp}{\cD_{\cP}}
\newcommand{\F}{\mathbf{F}}
\newcommand{\sm}{\text{softmax}}

\newcommand{\M}{\text{max}}
\newcommand{\pim}{p_{\M}^i}
\newcommand{\pcm}{p_{\M}^c}
\newcommand{\w}{\boldsymbol{w}}
\newcommand{\ti}{\text{tiers}}
\newcommand{\ttt}[1]{\texttt{#1}}

\usepackage{booktabs}

\def\<{\left\langle}
\def\>{\right\rangle}
\def\({\left(}
\def\){\right)}

\newcommand{\stepsize}{\eta}

\newcommand{\onenorm}[1]{\left\| #1 \right\|_1}      
\newcommand{\threenorm}[1]{\left\| #1 \right\|_3}      

\usepackage{pifont}
\newcommand{\cmark}{{\color{green}\ding{51}}}%
\newcommand{\xmark}{{\color{red}\ding{55}}}%

\newcommand{\Z}{\mathbb{Z}}

\newcommand{\B}{\mathbb{B}}
\newcommand{\U}{\mathbb{U}}

\newcommand{\ExpC}[1]{{{\mathbf{E}_C}}\left[#1\right] }
\newcommand{\Expg}[1]{{{\mathbf{E}_{\nabla}}}\left[#1\right] }

\usepackage{bbm}

\newcommand{\lp}{\left(}
\newcommand{\rp}{\right)}

\newcommand{\Sam}{S}

\newcommand{\fs}{f^\star} 
\newcommand{\xs}{x^\star} 

\newcommand{\hf}{\hat{f}}
\newcommand{\hfs}{\hat{f}^\star}
\newcommand{\hx}{\hat{x}^\star}
\newcommand{\tw}{\tilde{w}}
\newcommand{\hw}{\hat{w}}

\usepackage[lmargin=1.5in, rmargin=1in, vmargin=1in, headsep=0.083334in]{geometry}
\newcommand{\mathsym}[1]{{}}
\newcommand{\unicode}[1]{{}}
\renewcommand{\thechapter}{\arabic{chapter}}
\renewcommand\bibname{\centering BIBLIOGRAPHY}

\begin{document}


\vspace{2pt}
\thispagestyle{empty}
\addvspace{10mm}

\begin{center}

{\textbf{{\large Better Methods and Theory for Federated Learning: \\ Compression, Client Selection and Heterogeneity}}}\vfill 
{Dissertation by}\\
{ Samuel Horv\'{a}th}\vfill

{ In Partial Fulfillment of the Requirements}\\[12pt]
{ For the Degree of}\\[12pt]
{Doctor of Philosophy} \vfill
{King Abdullah University of Science and Technology }\\
{Thuwal, Kingdom of Saudi Arabia}
\vfill
{June 2022}

\end{center}

\newpage

\begin{center}

\end{center}

\begin{center}

{ \textbf{{\large EXAMINATION COMMITTEE PAGE}}}\\\vspace{1cm}

\end{center}
\noindent{The dissertation of Samuel Horv\'{a}th is approved by the examination committee}
\addcontentsline{toc}{chapter}{Examination Committee Page}

\vspace{4\baselineskip}

\begin{onehalfspacing}
\noindent{Committee Chairperson: Peter Richt\'{a}rik (KAUST)}\\
Committee Members: Salman Avestimehr (University of Southern California), Marco Canini (KAUST), Marc G.\ Genton (KAUST), Michael Rabbat (Facebook AI Research)\vfill
\end{onehalfspacing}


\newpage
\addcontentsline{toc}{chapter}{Copyright}
\vspace*{\fill}
\begin{center}
{ \copyright June, 2022}\\
{Samuel Horv\'{a}th}\\
{All Rights Reserved}
\end{center}

\begin{center}

\end{center}

\begin{center}
{{\bf\fontsize{14pt}{14.5pt}\selectfont \uppercase{ABSTRACT}}}
\end{center}

\doublespacing
\addcontentsline{toc}{chapter}{Abstract}

\begin{center}
{{\fontsize{14pt}{14.5pt}\selectfont {Better Methods and Theory for Federated Learning: \\ Compression, Client Selection and Heterogeneity\\
Samuel Horv\'{a}th}}}
\end{center}

\singlespacing

Federated learning (FL) is an emerging machine learning paradigm involving multiple clients, e.g., mobile phone devices, with an incentive to collaborate in solving a machine learning problem coordinated by a central server. FL was proposed in 2016 by Kone\v{c}n\'{y} et al.~\cite{konevcny2016afederated, FEDLEARN2016, FEDOPT2016} and McMahan et al.~\cite{mcmahan17fedavg} as a viable privacy-preserving alternative to traditional centralized machine learning since, by construction, the training data points are decentralized and never transferred by the clients to a central server. Therefore, to a certain degree, FL mitigates the privacy risks associated with centralized data collection. 

Unfortunately, optimization for FL faces several specific issues that centralized optimization usually does not need to handle. 
%
In this thesis, we identify several of these challenges and propose new methods and algorithms to address them, with the ultimate goal of enabling practical FL solutions supported with mathematically rigorous guarantees. 
In particular, in the first four chapters after the introduction, we focus on the communication bottleneck, and devise novel compression mechanisms and tools that can provably accelerate the training process. In the sixth chapter, we address another significant challenge of FL: partial participation of clients in each round of the training process. More concretely, we propose the first importance client sampling strategy that is compatible with two core privacy requirements of FL: secure aggregation and statelessness of clients.
The seventh chapter is dedicated to another challenge in the cross-device FL setting---system heterogeneity, i.e., the diversity in clients' processing capabilities and network bandwidth, and the communication overhead caused by slow connections. To tackle this, we introduce the ordered dropout (OD) mechanism. OD promotes an ordered, nested representation of knowledge in neural networks and enables the extraction of lower-footprint sub-models without retraining, which offers fair and accurate learning in this challenging FL setting. 
Lastly, in the eight chapter, we study several key algorithmic ingredients behind some of the most popular methods for cross-device FL aimed to tackle heterogeneity and communication bottleneck. In particular, we propose a general framework for analyzing methods employing all these techniques simultaneously, which helps us better understand their combined effect. Our approach identifies several inconsistencies, and enables better utilization of these components, including the popular practice of running multiple local training steps before aggregation.




\begin{center}

{\bf\fontsize{14pt}{14.5pt}\selectfont \uppercase{Acknowledgements}}\\\vspace{1cm}
\end{center}

\addcontentsline{toc}{chapter}{Acknowledgements}

First and foremost, I would like to thank my PhD supervisor Peter Richt\'{a}rik for his continuous support throughout this journey. I am very grateful to you for giving me a chance to pursue my graduate studies under your mentorship. I owe an enormous debt of gratitude to you for your guidance, leadership, research ideas, our discussions,  extraordinary support and opportunities, career advice and a lot of encouragement that have significantly shaped me.

Furthermore, I would like to thank all the members and visitors of our Optimization for Machine Learning research group at KAUST. It has been a pleasure to work with this group of extraordinarily bright people. In particular, I wish to thank Filip Hanzely, Konstantin Mishchenko, Elnur Gasanov, Dmitry Kovalev, Slavom\'{ı}r Hanzely, Egor Shulgin, Grigory Malinovsky, Konstantin Burlachenko, Lukang Sun, Kai Yi,  Alibek Sailanbayev, Igor Sokolov, Yazeed Basyoni, Aritra Dutta, Mher Safaryan, El Houcine Bergou, Xun Qian, Zhize Li, Alexander Tyurin, Avetik Karagulyan, Laurent Condat, \v{L}udov\'{i}t Horv\'{a}th, Wenlin Chen, Abdurakhmon Sadiev, Rafał Szlendak, Ahmed Khaled, Aleksandr Beznosikov, Hoang Nguyen, S\'{e}lim Chraibi, Sebastian Stich, Nicolas Loizou and Robert Mansel Gower.

Furthermore, I am very grateful to my internship and research visits hosts, namely Jakub Ma\v{c}ina (Exponea, Bratislava, Slovakia), Michael I.\ Jordan (University of California, Berkeley, USA),  C\'{e}dric Archambeau (Amazon, Berlin, Germany), Stefanos Laskaridis (Samsung AI Centre, Cambridge, United Kingdom) and Micheal Rabbat (Facebook AI Research, Montreal, Canada), as well as to other people I had the chance to interact with, including Milo\v{s} Kondela, Mat\'{u}\v{s} Cimerman, Jakub Bahyl, Jakub Kmec, Ondrej Brichta, Marek Bruchat\'{y}, Veronika Urban\v{c}oková (Exponea), Lihua Lei (Berkeley), Aaron Klein, Aaron Mishkin, Louis Tiao, Eric Lee, Matthias Seeger, Valerio Perrone, Beyza Ermis (Amazon), M\'{a}rio Almeida, Ilias Leontiadis, Stylianos Venieris, Nicholas Lane (Samsung AI Centre), Maziar Sanjabi and Lin Xiao (Facebook AI Research). Thank you for the opportunity to work with you!

Needles to say, I very much appreciate the excellent research conditions and all the support I received from KAUST; I feel fortunate for all the opportunities I was given here. 

Last but not least, I am deeply grateful to my family and friends for their love and support.



\begin{onehalfspacing}

\renewcommand{\contentsname}{\centerline{\textbf{{\large TABLE OF CONTENTS}}}}
\tableofcontents
\cleardoublepage



\cleardoublepage
\addcontentsline{toc}{chapter}{\listfigurename} 
\renewcommand*\listfigurename{\centerline{LIST OF FIGURES}} 
\listoffigures

\cleardoublepage
\addcontentsline{toc}{chapter}{\listtablename}
\renewcommand*\listtablename{\centerline{LIST OF TABLES}} 
\listoftables

\end{onehalfspacing}

\singlespacing

\chapter{Introduction}
\label{chapter:introduction}

Over the past few years, advances in machine learning (ML) in general and in deep learning (DL) in particular have revolutionised the way we interact with modern digital devices. For instance, a couple years ago, we would have never imagined deep learning applications to bring us self-driving cars and virtual assistants like Alexa, Siri, and Google Assistant. But today, these creations are part of our everyday life. 
Much of this success relies on the availability of large-scale training infrastructures and the availability of \textit{vast amounts of training data}~\cite{fbdatacenter2018hpca}.
However, ML users and providers of ML solutions and services are becoming aware of the \textit{privacy implications} of this increasingly data-hungry process, which motivated the creation of various privacy-preserving initiatives by service providers~\cite{apple} and government regulators~\cite{gdpr}. Apart from privacy considerations, the \textit{data locality paradigm}, i.e., requiring the data to be processed at the same location where it was first captured and stored, is becoming an essential machine learning component due to energy efficiency and climate change considerations~\cite{qiu2021first}.

Federated Learning (FL) is an emerging subfield of machine learning (ML) that promises to provide a solution to the these issues as it allows the training of models without the data ever leaving the users' devices. Instead, FL enables clients to collaboratively train a shared model by moving the computation to their devices. The concept of shifting computation to distributed edge devices has been around for a long time, but it was largely limited to simple tasks only, e.g., querying in sensor networks~\cite{deshpande2005model} and fog computing~\cite{bonomi2012fog}. 
However, the recent growth in storage capacity, data availability and computational capabilities of edge devices has enabled local training to be done directly on the devices.

FL in the form as we use it nowadays was proposed in 2016 by Kone\v{c}n\'{y} et al.~\cite{konevcny2016afederated, FEDLEARN2016, FEDOPT2016} and McMahan et al.~\cite{mcmahan17fedavg}. 
In its basic form, FL involves distributed training that operates in the following way. In each training round, \textit{participating devices}, which are selected from the pool of all available devices, download the latest model from a \textit{central server} that orchestrates training. Next, each selected device performs local training using its local data, which leads to an updated local model. These locally trained models are sent by the participating devices back to the central server, where updates are (securely) ``\textit{aggregated}'' to form an updated version of the shared model distributed to a new set of participating devices in the next training round.

There are two main settings in FL: \textit{cross-silo} and \textit{cross-device}~\cite{kairouz2019advances}, each with its own specific challenges. In the cross-silo setting, several companies or organizations share the common objective of training a model based on the union of their data. However, they do not wish to share their data directly due to privacy requirements, e.g., patients' medical data for hospitals or bank transaction details. Typically, we assume that the number of organizations is not very large, that each organization has relatively good computational resources, and we can identify each organization, i.e., we are allowed to store ``\textit{states}'' on the clients.  In the cross-device setting, the clients are mobile phones or various IoT devices. There are many such clients (millions or hundreds of millions), and each device has a relatively small amount of data. Their objective is to collaborate in training a prediction model or a recommendation model that is to be deployed on the devices. The devices also do not communicate their data directly as they might reveal some private information that clients do not wish to share. In this scenario, clients do not know about each other and collaboration is orchestrated by a centralized authority. In addition, there are no ``client states'' in the cross-device setup as each client will likely participate a few times only, perhaps just once. Visualization of cross-device FL lifecycle is provided in Figure~\ref{fig:fl_google}.

\begin{figure}[t]
    \centering
    \includegraphics[width=0.9\textwidth]{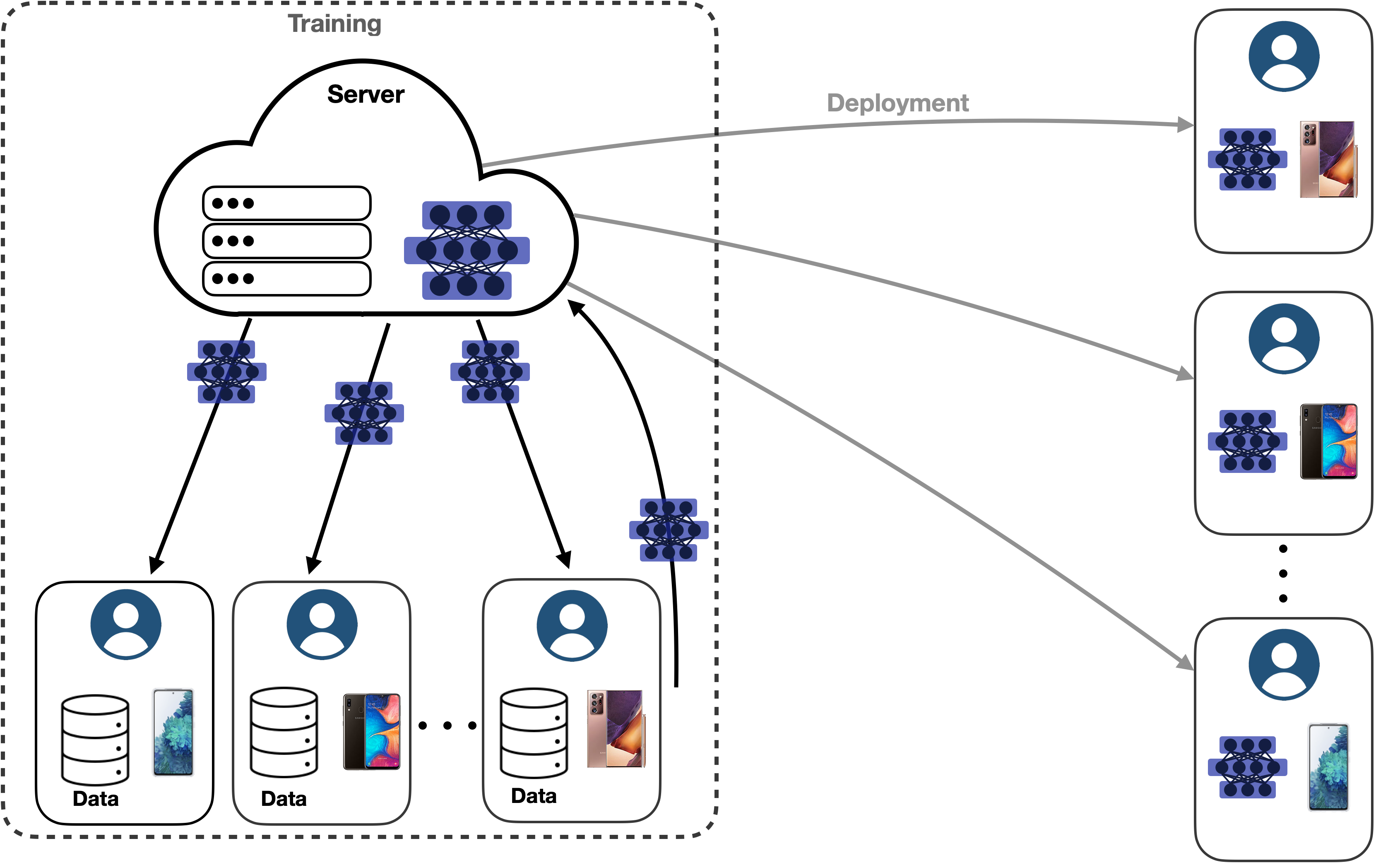}
    \caption{The lifecycle of a FL-trained model in the cross-device setting.}
    \label{fig:fl_google}
\end{figure}

Recently, FL has been subject of increasing attention in academia. The same holds for the major industry providers that already use FL in their deployed systems as a critical component to ensure privacy guarantees for models based on training data distributed at the edge.
For instance, Google uses FL in the Gboard mobile keyboard app for applications including next word prediction~\cite{hard18gboard}, emoji suggestion~\cite{gboard19emoji} or “Hey Google” Assistant~\cite{googleassistant2021}. Apple uses FL for applications like the QuickType keyboard and the vocal classifier for “Hey Siri”~\cite{apple19wwdc}. In the financial sector, FL is used to detect money laundering~\cite{webank2020} and financial fraud~\cite{intel2020}. In medical applications, FL is used for drug discovery~\cite{melloddy2020} and medical images analysis~\cite{owkin2020}. 

In general, FL can be a handy tool for machine learning in applications where a large amount of data is distributed across many clients with an incentive to collaborate, but unwilling to share their local data to a third party for processing. One such application is learning user patterns across a large pool of smartphone devices, with particular examples being face detection, next-word prediction, and voice or phrase recognition. Users may opt not to directly share their data to maintain their privacy or to save their smartphones' limited capabilities, such as network bandwidth or current battery level. Thus, FL can play a crucial role in enabling training and deployment on smartphones without diminishing the user experience while protecting personal data. Another possible application area for FL is the Internet of Things (IoT). Modern IoT networks, such as wearable devices, autonomous vehicles, or smart homes, usually contain several sensors that allow them to collect data in real time in order to operate. For instance, autonomous vehicles require an excellent up-to-date model of the current traffic conditions, such as the motion of other cars or pedestrians, to operate reliably and safely~\cite{thrun2010toward}. However, building models based on aggregated data stored in a centralized location may be complicated in these scenarios due to the private nature of the data and the limited connectivity of each device. FL might be a viable solution to training models that efficiently adjust to changes in these systems while respecting user privacy~\cite{samarakoon2018federated}.

\section{Problem formulation}
\label{sec:problem_formulation}

The most widespread approach to Federated Learning is to learn a single global model $x \in \R^d$ from decentralized data $\cbr{\cD_i}_{i = 1}^n$ stored on $n$ remote clients. This objective can be written in the form 
\begin{align}
\label{eq:fl_objective}
    \min_{x \in \R^d} \sbr*{f(x) \eqdef \sum_{i = 1}^n w_i f_i(x)}\;,
\end{align}
where $x \in \R^d$ represents the parameters defining the global model, $f(x): \R^d \rightarrow \R$ denotes the global objective function composed of the local objective functions $\cbr{f_i(x): \R^d \rightarrow \R}_{i = 1}^n$ weighted by non-negative weights $\cbr{w_i}_{i = 1}^n$ such that $\sum_{i=1}^n w_i = 1$. A typical way to define the weights is to make them proportional to the size of the local datasets, i.e., 
\begin{align}
\label{eq:weights}
    w_i = \frac{|\cD_i|}{\sum_{j=1}^n |\cD_j|}, \text{ for } i = 1, 2, \hdots, n.
\end{align}
The local loss functions $\cbr{f_i(x)}$ can in general be written in the form
\begin{align}
    \label{eq:fl_local_loss}
    f_i(x) \eqdef \EE{\xi \sim \cD_i}{f_i(x, \xi)}\;,
\end{align}
where $\E{\cdot}$ denotes mathematical expectation, and $\cD_i$ characterizes the local data distribution. While it is common to consider the same loss function across the clients, the functions $\cbr*{f_i}$ are typically very different. The difference comes from the variability in the local data distributions $\cbr{\cD_i}$ capturing data heterogeneity. If the local data distributions $\cbr{\cD_i}$ have finite supports, i.e., $\cbr{\cD_i}$ are sets with $\cbr{|\cD_i| < \infty}$ elements, each $f_i(x)$ can be written as an average loss across the local training data points, i.e.,
\begin{align}
    \label{eq:fl_local_loss_finite_sum}
    f_i(x) = \frac{1}{|\cD_i|}\sum_{j = 1}^{|\cD_i|} \sbr*{f_i(x, \xi_j) \eqdef f_{ij}(x)} \;.
\end{align}
%
The standard choice of weights~\eqref{eq:weights} leads to the global loss $f(x)$ being the average loss with respect to the global dataset $\cD$ defined as the union of all local datasets (global empirical risk minimization (ERM)), i.e., $\cD \eqdef \bigcup_{i=1}^{n} \cD_i$.

In Federated Learning, we aim to learn the global model $x$ under the constraint that each client's data is kept and processed locally. The learning objective is meant to be achieved using intermediate client updates intended for immediate aggregation by a central server. With the demand for privacy guarantees, such a setting differs significantly from traditional distributed learning. Therefore, FL poses new challenges that call for fundamental advances in privacy, large-scale machine learning, and distributed optimization. We discuss this in more detail in the next section. 

\section{Challenges}
\label{sec:challenges}
As previously discussed, FL comes with specific challenges which differentiate it from the more general field of distributed learning. In particular, we now list four core challenges that make the federated setting distinct from traditional setups, including but not limited to distributed optimization or privacy-preserving data analytics.

    \subsection{Communication bottleneck} It is well understood that communication cost can be the primary bottleneck in FL. Indeed, wireless links and other end-user internet connections typically operate at lower speeds and bandwidths than intra-datacenter or inter-datacenter links, and can be expensive and unreliable (e.g., clients in remote areas). Therefore, communication between the clients and the orchestrating server can be slower than local computation, often by many orders of magnitude~\cite{van2009multi,huang2013depth}. In addition, while federated networks, especially in the cross-device regime, are usually composed of a massive number of clients in the range of millions, the capacity of the aggregating server and other important system considerations impose direct or indirect constrains on the number of clients that can effectively participate in each communication round. Therefore, in order to be able to fit a model to data located on the devices in a decentralized federated network, it is necessary to develop communication-efficient methods. Several approaches to achieving this goal were proposed in the literature, including:
    \begin{enumerate}[label=(\roman*)]
        \item \textit{communication compression} (i.e., reducing the size of communicated messages in each communication round),
        \item \textit{more efficient methods} (i.e., the reduction of the total number of communication rounds),
        \item \textit{client sampling} (i.e., the reduction of the number of clients that are required to participate in each communication round).
    \end{enumerate}
    
    \subsection{Clients' systems heterogeneity} While traditional cloud-based distributed training can rely on powerful high-end compute architectures~\cite{fbdatacenter2018hpca}, this is impossible in federated learning due to the data locality paradigm, and the fact that FL as clients are typically \textit{resource-constrained}, \textit{heterogeneous} and \textit{unreliable} devices.
    In this respect, FL deployment is currently hindered by the vast \textit{heterogeneity} of client hardware~\cite{facebook2019hpca,ai_benchmark_2019,sysdesign_fl2019mlsys}. 
    For example, hardware heterogeneity often implies a high variability in processing speed~\cite{embench2019emdl}, which leads to slower aggregation if the server is required to wait for the stragglers.
    Furhter, mid and low-tier devices might not even be able to support larger models due to severe memory and processing limitations, and are therefore often excluded entirely from the FL training process, or dropped upon timeouts, together with their unique and potentially highly valuable data. More interestingly, the memory and processing capabilities of participating devices may also reflect on demographic and socio-economic information of owners, which might make the exclusion of such clients unfair~\cite{kairouz2019advances}. 
    A similar trend can be traced in the downstream (model) and upstream (updates) network communication in FL, which can be an additional substantial bottleneck for the training procedure \cite{comms_fl2020tnnls}. Therefore, practical FL methods must be robust to
    \begin{itemize}
        \item heterogeneity of the devices (in terms of the network connection speed, bandwidth, storage and compute capacity, and battery level),
        \item partial participation of clients, and
        \item stragglers and device dropping.
    \end{itemize} 
    
    \subsection{Statistical heterogeneity} The fact that clients generate the data locally typically leads to non-identical data distributions across the federated network, e.g., the type of images captured by different people for an object recognition task. Therefore, one needs to account for this extra challenge as it introduces obstacles to designing efficient training or optimization methods for FL. On top of that, statistical heterogeneity also raises questions about the utility of a single global model~\eqref{eq:fl_objective} to individual users. Thess issues are often handled by the introduction of suitable FL personalization techniques~\cite{smith2017federated, Hanzely2020, hanzely2020lower, gasanov2021flix, PFL_universal2021}. For instance, the utility of a next-word prediction model can typically be enhanced by taking into account specific speech patterns of individual users.
    
    \subsection{Privacy} Privacy is often a major concern in FL applications and one of the reasons Federated Learning has become popular in recent years. While Federated Learning offers a solution that makes a step towards protecting local data by communicating model updates, e.g., gradient information, instead of the raw data~\cite{duchi2014privacy, dwork2014algorithmic,carlini2018secret}, it is not automatically guaranteed that such updates do not contain sensitive information about the clients.
    In particular scenarios, it was shown that the central server or a third-party entity can reveal private information from the shared model updates~\cite{mcmahan18learning}.
    There are several tool and techniques for enhancing FL algorithms with formal privacy guarantees, including differential privacy~\cite{chaudhuri2011differentially, jain2014near, geyer2017differentially}, secure multi-party computation (SMPC)~\cite{du2001secure, goryczka2015comprehensive, bonawitz2017practical, so2020turbo}, homomorphic encryption (HE)~\cite{hardy2017private}, and trusted execution environments (TEEs)~\cite{mo2020darknetz, mo2021ppfl}. However, these approaches often come with the cost of (often dramatically) reduced performance and slower training. Therefore, practically and theoretically, finding the right balance of all of these techniques and methods is still a fundamental challenge in achieving practical private FL systems.
\begin{figure}[t]
    \centering
    \includegraphics[width=1.\textwidth]{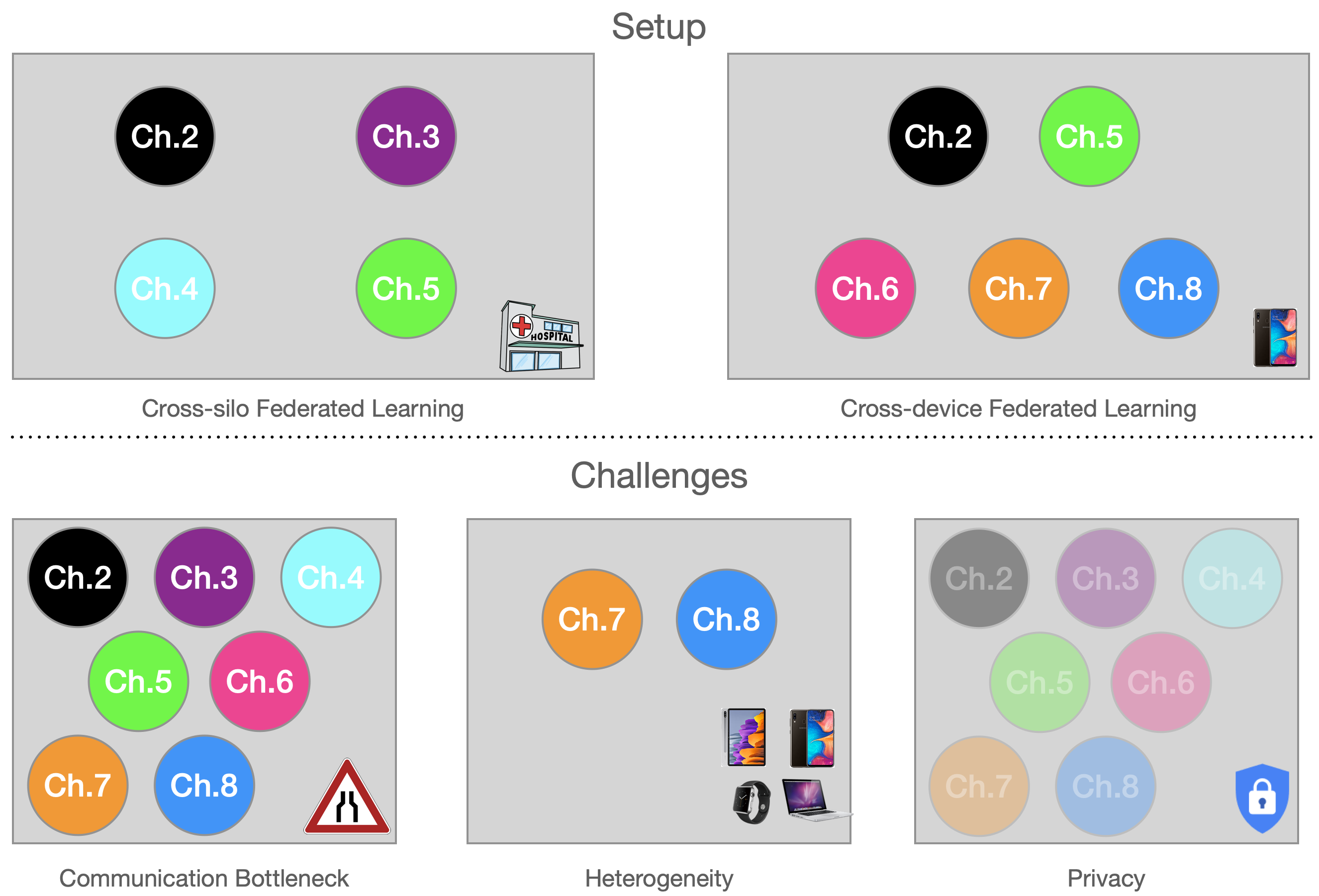}
    \caption{The scope and particular challenges addressed in the chapters of this thesis. Limited transparency in the privacy box indicates that while we do not directly address privacy issues in this thesis, our techniques are compatible with privacy-enhancing tools.}
    \label{fig:chapters}
\end{figure}
\subsection{Thesis focus}

We directly focus on the first three challenges in this thesis. In Chapters 2--5, we introduce novel algorithms and tools for reducing the communication bottleneck in FL, or more generally in distributed (deep) learning  via novel and provably efficient compression mechanisms. 
In Chapter 6, we tackle the communication bottleneck in cross-device FL by proposing the first importance client sampling that is compatible with FL privacy requirements such as secure aggregation and client statelessness. 
In Chapter 7, we shift our focus to client systems heterogeneity and propose the Ordered Dropout technique that enables federated training in this challenging setup.
Finally, in Chapter 8, we take a closer look at the most popular techniques (e.g., local steps with random shuffling of the data and variance reduction), devised to tackle client heterogeneity and communication bottleneck in cross-device FL.

While we do not directly address privacy challenges, all the methods presented in this thesis are compatible and combinable with popular techniques for enhancing the privacy guarantees in FL.
The scope and particular challenges addressed in the chapters of this thesis are visualized in Figure~\ref{fig:chapters}.

\section{Basic federated optimization algorithms}

Assuming access to local gradients, problem \eqref{eq:fl_objective} can potentially be solved by the most popular gradient type method--gradient descent (\texttt{GD}). Each step of \texttt{GD} is of the form
$$ x^{k+1} = x^{k} - \eta^k \nabla f(x^{k}), k=0,1,2,\dots,$$
where $\eta^k > 0$ is an appropriately chosen learning rate. In terms of local gradients, we can write
$$\nabla f(x^k) =  \sum_{i=1}^n w_i \nabla f_i(x^k),$$ which means that the gradient of the global loss function $f$ is equal to the weighted average of the gradients of the local functions $f_i$.
While in standard FL settings the clients can't directly share gradients among themselves, they can communicate with an orchestrating server that handles the aggregation and model distribution; see Figure~\ref{fig:fl_google} for a visualization.

 Under standard assumptions, \texttt{GD} is a very well understood method and, therefore, we will not discuss it further here. Instead, we refer readers to the book of Nesterov for more details~\cite{nesterov2013introductory}. 

While in theory \texttt{GD} can be applied in the context of FL, it is not used in practice due to various FL-specific constraints and considerations discussed in Section~\ref{sec:challenges}. Several techniques can make \texttt{GD} efficient as a method for solving federated optimization problems. These enhancements and their possible combinations are still the subject of active research. The majority of the results presented in this thesis can be seen as attempts to make \texttt{GD} a theoretically and practically more efficient method for FL. 

Below we discuss in detail several of the requirements and popular considerations in FL:

\begin{itemize}
    \item  \textbf{Partial client participation.} In every communication round $k$, only a subset $S^k \subset [n]$ of clients can participate, i.e., communicate with the server, and the model update rule is therefore limited to
    $$ x^{k+1} = x^{k} - \eta^k \frac{1}{\sum_{i\in S^k} w_i}\sum_{i\in S^k} w_i \nabla f_i(x^k).$$
    The availability of clients $\cbr{S^k}$  is hard to model as it is controlled by unknown events beyond the control of the orchestrating server. For instance, it is common to consider mobile devices to be ready to participate only when idle, charging and connected to a fast network. In order to guarantee theoretical convergence with partial participation, one needs to make assumptions on client sampling. Finally, it should not be surprising that partial participation typically leads to an increase in the number of communication rounds to reach any specific performance criterion. 
    \item \textbf{Stochastic approximation of gradients (SA).} Generally, replacing the exact gradient with its stochastic approximation is preferred in many machine learning (ML) applications. Theoretically, it is often convenient to work with unbiased stochastic gradients $g_i(x)$, i.e., $\EE{\cD_i}{g_i(x)} = \nabla f_i(x)$. However, practitioners often choose to approximate the gradient using without-replacement sampling over the local dataset $\cD_i$ (if $\cD_i$ is a finite set). SA must be applied when computing the expectation over the local distribution $\cD_i$ is impossible. On the other hand, SA is preferred when the local dataset $\cD_i$ is large, and calculating the exact gradient in each step is computationally expensive. 
    \item \textbf{Local training.} Local training is a popular FL technique for reducing the communication costs by allowing each client to run several steps of a local training method before communicating back to the orchestrating server. While local training is a practically useful trick often used in FL, its theoretical benefits are still not well understood\footnote{The only exception is a very recent work by Mishchenko et al.~\cite{proxskip} that presents the first local method with better communication complexity than gradient descent for strongly convex functions in the heterogeneous data regime.}.
    \item \textbf{Communication compression.} Reducing the size of messages transmitted between the clients and the server helps to alleviate the communication bottleneck. 
    \item \textbf{Other considerations.} There are many other helpful techniques for enhancing \texttt{GD} in the FL setting, including but not limited to \textit{momentum} (\textit{acceleration}), \textit{control variates}, and \textit{adaptive methods}. 
    
\end{itemize}

\section{Detailed review of federated learning}

This chapter's goal is not an extensive detailed review of federated learning. Instead, this chapter only acts as an introduction to the areas of FL considered in this thesis. We do not discuss many topics in detail, and we omit several directions here, including but not limited to fully decentralized/peer-to-peer distributed learning or split learning.  For a comprehensive overview of current research in FL, we refer readers to review papers~\cite{kairouz2019advances, li2020federated, wang2021field, ding2022federated}.

\section{Chapters organization and outline of contributions}

This introductory chapter is followed by seven chapters, each based on a single paper. Some of these papers are already published, while others are at the various stages of the peer-review process. Each chapter is organized as follows. We first introduce a specific problem and elaborate on the challenges involved. For reader's convenience, we recall the specific notation used in the chapter. Subsequently, we discuss the relevant literature, followed by a summary of our contributions. In some chapters, we first discuss our contributions, and only then the related work. Next, we describe the proposed methods and the theoretical analysis. Finally, we conclude with experimental evaluations to verify our theoretical claims and practical performance. Each chapter, together with the corresponding appendix containing detailed proofs of our claims and extra details, is mostly self-contained. 

Each chapter identifies and proposes a solution to a critical FL challenge as presented in Section~\ref{sec:challenges}. The chapters are mainly complementary to each other, i.e., our proposed methods can be combined for the improved overall effect to address multiple challenges, or they build on top of each other, in which case, we specifically mention this. The provided results are meant to bring us closer to the ultimate goal of our research--efficient collaborative machine learning without centralized training data.

Let us now briefly discuss the outline and scope of each of the chapters individually. 

\subsection{Chapter 2: Natural compression for distributed deep learning}

Modern deep learning models are often trained in parallel on a collection of distributed machines in order to reduce training time. In such settings, communication of model updates among devices becomes a significant performance bottleneck, and various lossy update compression techniques have been proposed in the literature to alleviate this problem. In this chapter, we introduce a new, simple, yet theoretically and practically effective compression technique: {\em natural compression ($\NC$)}. Our method is applied individually to all entries of the to-be-compressed update vector. It works by randomized rounding to the nearest (negative or positive) power of two, which can be computed in a ``natural'' way by ignoring the mantissa. We show that compared to no compression,  $\NC$ increases the second moment of the compressed vector by not more than the tiny factor  $\nicefrac{9}{8}$, which means that the effect of $\NC$ on the convergence speed of popular training algorithms, such as distributed SGD, is negligible. However, the communications savings enabled by $\NC$ are substantial, leading to {\em $3$-$4\times$ improvement in overall theoretical running time}. For applications requiring more aggressive compression, we generalize $\NC$ to {\em natural dithering}, which we prove is {\em exponentially better} than the common random dithering technique. Our compression operators can be used on their own or combined with existing operators for a more aggressive combined effect and offer new state-of-the-art both in theory and practice.

This chapter is based on:
\\
\\
\cite{Cnat}: Samuel Horv\'{a}th, Chen-Yu Ho, Ľudov\'{i}t Horv\'{a}th, Atal Narayan Sahu, Marco Canini, and Peter Richt\'{a}rik, {``Natural compression for distributed deep learning"}, arXiv preprint arXiv:1905.10988, 2019. Workshop on AI Systems at Symposium on Operating Systems Principles (SOSP 2019).

\subsection{Chapter 3: Stochastic distributed learning with gradient quantization  and double variance reduction}

We consider distributed optimization over several devices, sending incremental model updates to a central server. This setting is considered, for instance, in Federated Learning. Various schemes have been designed to compress the model updates to reduce the overall communication cost. However, existing methods suffer from a significant slowdown due to additional variance $\omega > 0$ coming from the compression operator and, as a result, only converge sublinearly. What is needed is a variance reduction technique for taming the variance introduced by compression. We propose the first methods that achieve linear convergence for arbitrary compression operators.
 For strongly convex functions with condition number $\kappa$, distributed among $n$ machines  with a finite-sum structure, each worker having less than $m$ components, we  
We also (ii)~give analysis for the weakly convex and non-convex cases and
(iii)~verify that our novel variance reduced schemes are more efficient than the baselines in experiments. Moreover, we show theoretically that as the number of devices increases, higher compression levels are possible without this affecting the overall number of communications compared to methods that do not perform any compression. This leads to a significant reduction in communication costs. Our general analysis allows picking the most suitable compression for each problem, finding the right balance between additional variance and communication savings. Finally, we also (iv) give analysis for arbitrary quantized updates. 

This chapter is based on:
\\
\\
\cite{diana2}: Samuel Horv\'{a}th, Dmitry Kovalev, Konstantin Mishchenko, Sebastian Stich, and Peter Richt\'{a}rik, {``Stochastic distributed learning with gradient quantization and variance reduction"}, arXiv preprint arXiv:1904.05115, 2019. \textbf{The Best Poster Prize}: Control, Information and Optimization Summer School, Voronovo, Russia, presented by D. Kovalev.

\subsection{Chapter 4: On biased compression for distributed learning}

In the last few years, various communication compression techniques have emerged as an indispensable tool helping to alleviate the communication bottleneck in distributed learning. However, despite the fact {\em biased} compressors often show superior performance in practice when compared to the much more studied and understood {\em unbiased} compressors,  very little is known about them. This chapter studies three classes of biased compression operators, two of which are new, and their performance when applied to  (stochastic) gradient descent and distributed (stochastic) gradient descent. We show for the first time that biased compressors can lead to linear convergence rates both in the single node and distributed settings. We prove that distributed compressed SGD method, employed with error feedback mechanism, enjoys the ergodic rate $\cO\left( \delta L \exp[-\frac{\mu K}{\delta L}]  + \frac{(C + \delta D)}{K\mu}\right)$, where $\delta\ge1$ is a compression parameter which grows when more compression is applied, $L$ and $\mu$ are the smoothness and strong convexity constants, $C$ captures stochastic gradient noise ($C=0$ if full gradients are computed on each node) and $D$ captures the variance of the gradients at the optimum ($D=0$ for over-parameterized models).  Further, we shed light on why and by how much biased compressors outperform their unbiased variants via a theoretical study of several synthetic and empirical distributions of communicated gradients. Finally, we propose several new biased compressors with promising theoretical guarantees and practical performance.

This chapter is based on:
\\
\\
\cite{beznosikov2020biased}: Aleksandr Beznosikov, Samuel Horv\'{a}th, Peter Richt\'{a}rik, and Mher Safaryan, {``On biased compression for distributed learning"}, arXiv preprint arXiv:2002.12410, 2020. \textbf{Oral Presentation}: \emph{Workshop on Scalability, Privacy, and Security in Federated Learning} (NeurIPS 2020).

\subsection{Chapter 5: A better alternative to error feedback for communication-efficient distributed learning}

Modern large-scale machine learning applications require stochastic optimization algorithms to be implemented on distributed compute systems. A key bottleneck of such systems is the communication overhead for exchanging information (e.g., stochastic gradients) across the workers. Among the many techniques proposed to remedy this issue, one of the most successful is the framework of compressed communication with error feedback (EF). EF remains the only known technique that can deal with the error induced by contractive compressors which are not unbiased, such as Top-$K$ or PowerSGD.  In this chapter, we propose a new and theoretically and practically better alternative to EF for dealing with contractive compressors. In particular, we propose a construction which can transform any contractive compressor into an induced unbiased compressor. Following this transformation, existing methods able to work with unbiased compressors can be applied. We show that our approach leads to vast improvements over EF, including reduced memory requirements, better communication complexity guarantees and fewer assumptions. We further extend our results to Federated Learning with partial participation following an arbitrary distribution over the nodes, and demonstrate the benefits thereof.

This chapter is based on:
\\
\\
\cite{horvath2021a}: Samuel Horv\'{a}th, and Peter Richt\'{a}rik, {``A better alternative to error feedback for communication-efficient distributed learning"}, in \emph{International Conference on Learning Representations} (ICLR 2021), 2020. \textbf{The Best Paper Award}: \emph{Workshop on Scalability, Privacy, and Security in Federated Learning} (NeurIPS 2020).

\subsection{Chapter 6: Optimal client sampling for federated learning}

It is well understood that client-master communication can be a primary bottleneck in Federated Learning. In this work, we address this issue with a novel client subsampling scheme, where we restrict the number of clients allowed to communicate their updates back to the master node. All participating clients compute their updates in each communication round, but only the ones with ``important'' updates communicate back to the master. We show that importance can be measured using only the norm of the update and give a formula for optimal client participation. This formula minimizes the distance between the full update, where all clients participate, and our reduced update, which restricts the number of participating clients. In addition, we provide a simple algorithm that approximates the optimal formula for client participation, which only requires secure aggregation and thus does not compromise client privacy. We show both theoretically and empirically that for Distributed {\tt SGD} ({\tt DSGD}) and Federated Averaging ({\tt FedAvg}), the performance of our approach can be close to full participation and superior to the baseline where participating clients are sampled uniformly. Moreover, our system is orthogonal to and compatible with existing methods for reducing communication overhead, such as local and communication compression methods. 

This chapter is based on:
\\
\\
\cite{chen2020optimal}: Wenlin Chen, Samuel Horv\'{a}th, and Peter Richt\'{a}rik, {``Optimal client sampling for federated learning"}, arXiv preprint arXiv:2010.13723, 2020. \emph{Privacy Preserving Machine Learning Workshop} (NeurIPS 2020).

\subsection{Chapter 7: Fair and accurate federated learning under heterogeneous targets with ordered dropout}

Federated Learning (FL) has been gaining significant traction across different ML tasks, ranging from vision to keyboard predictions. In large-scale deployments, client heterogeneity is a fact, and constitutes a primary problem for fairness, training performance and accuracy. Although significant efforts have been made into tackling statistical data heterogeneity, the diversity in the processing capabilities and network bandwidth of clients, termed as system heterogeneity, has remained largely unexplored. Current solutions either disregard a large portion of available devices or set a uniform limit on the model's capacity, restricted by the least capable participants.
In this chapter, we introduce Ordered Dropout, a mechanism that achieves an ordered, nested representation of knowledge in Neural Networks and enables the extraction of lower footprint submodels without the need of retraining. We further show that for linear maps our Ordered Dropout is equivalent to SVD.  We employ this technique, along with a self-distillation methodology, in the realm of FL in a framework called \tool. \tool alleviates the problem of client system heterogeneity by tailoring the model width to the client's capabilities. 
Extensive evaluation on both CNNs and RNNs across diverse modalities shows that \tool consistently leads to significant performance gains over state-of-the-art baselines, while maintaining its nested structure.

This chapter is based on:
\\
\\
\cite{horvath2021fjord}: Samuel Horv\'{a}th, Stefanos Laskaridis, Mario Almeida, Ilias Leontiadis, Stylianos I. Venieris, Nicholas D. Lane, {``FjORD: Fair and accurate federated learning under heterogeneous targets with ordered dropout"}, in \emph{Conference on Neural Information Processing Systems} \textbf{Spotlight} (NeurIPS 2021), 2021. \emph{International Workshop on Federated Learning for User Privacy and Data Confidentiality} (ICML 2021).

 \subsection{Chapter 8: Recipes for better use of local work in federated learning}

The practice of applying several local updates before aggregation across clients has been empirically shown to be a successful approach to overcoming the communication bottleneck in Federated Learning (FL). 
In this chapter, we propose a general recipe--\fedshuffle, that better utilizes the local updates in FL, especially in the heterogeneous regime.
Unlike many prior works, \fedshuffle does not assume any uniformity in the number of updates per device.
Our \fedshuffle recipe comprises four simple-yet-powerful ingredients: 1) local shuffling of the data, 2) adjustment of the local learning rates, 3) update weighting, and 4) momentum variance reduction~\cite{cutkosky2019momentum}.
We present a comprehensive theoretical analysis of \fedshuffle and show that both theoretically and empirically, our approach does not suffer from the objective function mismatch that is present in FL methods which assume homogeneous updates in heterogeneous FL setups, e.g., \fedavg~\cite{mcmahan17fedavg}. In addition, by combining the ingredients above, \fedshuffle improves upon \fednova~\cite{fednova2020neurips}, which was previously proposed to solve this mismatch. We also show that \fedshuffle with momentum variance reduction can improve upon non-local methods under a Hessian similarity assumption.
Finally, through experiments on  synthetic and real-world datasets, we illustrate how each of the four ingredients used in \fedshuffle\ helps improve the use of local updates in FL.

This chapter is based on:
\\
\\
\cite{fedshuffle}: Samuel Horv\'{a}th, Maziar Sanjabi, Lin Xiao, Peter Richt\'{a}rik, and Michael Rabbat {``FedShufffle: Recipes for better use of local work in federated learning"}, Technical Report, 2022.

\begin{table}[t]
\caption{Excluded papers.}
\label{tab:excluded}
\begin{tabular}{|c|c|c|}
\hline
Ref.                                               & Title                                                                                                                                    & FL Paper              \\ \hline
\cite{horvath2018nonconvex}       & \begin{tabular}[c]{@{}c@{}}Nonconvex variance Reduced Optimization\\  with Arbitrary Sampling\end{tabular}                               & \xmark \\ \hline
\cite{Kovalev2019:svrg}           & \begin{tabular}[c]{@{}c@{}}SVRG and Katyusha are Better\\  Without the Outer Loop\end{tabular}                                           & \xmark \\ \hline
\cite{horvath2020adaptivity}      & \begin{tabular}[c]{@{}c@{}}Adaptivity of Stochastic Gradient Methods\\  for Nonconvex Optimization\end{tabular}                          & \xmark \\ \hline
\cite{hanzely2020lower}           & \begin{tabular}[c]{@{}c@{}}Lower Bounds and Optimal Algorithms\\  for Personalized Federated Learning\end{tabular}                       & \cmark \\ \hline
\cite{horvath2021hyperparameter}  & \begin{tabular}[c]{@{}c@{}}Hyperparameter Transfer Learning\\  with Adaptive Complexity\end{tabular}                                     & \xmark \\ \hline
\cite{burlachenko2021fl_pytorch} & \begin{tabular}[c]{@{}c@{}}FL\_PyTorch: Optimization Research Simulator\\  for Federated Learning\end{tabular}                           & \cmark \\ \hline
\cite{wang2021field}              & \begin{tabular}[c]{@{}c@{}}A Field Guide\\  to Federated Optimization\end{tabular}                                                       & \cmark \\ \hline
\cite{gasanov2021flix}            & \begin{tabular}[c]{@{}c@{}}FLIX: A Simple and Communication-Efficient Alternative\\  to Local Methods in Federated Learning\end{tabular} & \cmark \\ \hline
\cite{pogran2021long}             & \begin{tabular}[c]{@{}c@{}}Long-term Outcome in Patients\\  with Takotsubo Syndrome\end{tabular}                                         & \xmark \\ \hline
\end{tabular}
\end{table}

\subsection{Excluded papers}
During my PhD, I co-authored nine more papers which are not a part of this thesis. The list includes:
\begin{itemize}
    \item three works on variance reduction for stochastic optimization~\cite{horvath2018nonconvex, Kovalev2019:svrg, horvath2020adaptivity},
    \item four papers on FL---one focusing on lower bounds~\cite{hanzely2020lower}, the second on personalization~\cite{gasanov2021flix}, the third work introduced a research simulator for FL~\cite{burlachenko2021fl_pytorch}, and the last one is a field guide to federated optimization~\cite{wang2021field}, which is joint work with more than 50 authors from academia and industry,
    \item a paper that I co-authored while in Amazon looking into transfer learning in black-box optimization~\cite{horvath2021hyperparameter},
    \item a medical study focusing on takotsubo syndrome~\cite{pogran2021long}.
\end{itemize}
This is summarized in Table~\ref{tab:excluded}.

\section{Basic facts and notations} 
Before proceeding with the main results, let us elaborate on the most common notation used in the thesis and some theoretical background. 

Apart from the already introduced notation in Section~\ref{sec:problem_formulation}, we denote by $x^\star$ an optimal solution of~\eqref{eq:fl_objective} and let $f^\star\eqdef f(x^\star)$. We always assume that our objective is lower bounded, i.e., $f^\star > - \infty$. We assume that we can compute gradient of $f_i(x)$ or $f_{ij}(x)$, denoted by $\nabla f_i(x)$ or $\nabla f_{ij}(x)$, respectively. In some scenarios, we can't directly access the local gradient $\nabla f_i(x)$ and we only see its stochastic estimator $g_i(x)$, which we commonly assume to be unbiased, i.e., $\E{g_i(x)|x} = \nabla f_i(x)$, with bounded variance. Details about the variance upper bound can be found in the respective chapters.  

In most cases, when we consider a regularization term, e.g., $\ell_2$ penalty, we implicitly assume it is hidden in $f_i(x)$ or $f_{ij}(x)$. In some cases, we explicitly mention the regularization term, denoted by $R \colon \R^d \to \R \cup \{\infty\}$, so instead of $f(x)$, we explicitly minimize $f(x) + R(x)$. In this case, we assume that $R$ is a proper closed and convex regularizer and we can evaluate its prox operator as defined below.

\begin{definition}
\label{def:prox}
The prox operator $\prox_{\gamma R} \colon \R^d \to \R^d$ is defined as
\begin{align*}
 \prox_{\gamma R}(x) \eqdef \argmin_{y \in \R^d} \left\{\gamma R(y) + \tfrac{1}{2}\norm{y-x}^2 \right\}\,,
\end{align*}
for $\gamma > 0$ and a proper closed and convex  regularizer $R \colon \R^d \to \R \cup \{\infty\}$.
\end{definition}

Regarding vector operations, we use $\lin{ x,y } \eqdef \sum_{i=1}^d x_i y_i$ to denote standard inner product of  two vectors $x, y\in\R^d$, where $x_i$ corresponds to the $i$-th component of $x$ in the standard basis in $\R^d$. This induces the $\ell_2$-norm in $\R^d$ in the following way $\norm{x} \eqdef\sqrt{\lin{ x, x }}$.  We denote $\ell_p$-norms as $\|x\|_p \eqdef (\sum_{i=1}^d|x_i|^p)^{\nicefrac{1}{p}}$ for $p\in(1,\infty)$. For any $x, y \in \R^d$, $x \circ y$ denotes their element-wise multiplication.

When talking about convergence guarantees, we often use big $\cO$ notation for the ease of presentation denoted by $\cO(\cdot)$.

\subsection{Smoothness and convexity}

Throughout the thesis, for the sake of facilitating convergence/complexity analysis, we assume the global loss $f(x)$, the local loss $f_i(x)$, and its elements $f_{ij}(x)$ to satisfy certain properties. Two of the most basic properties are defined next. 

\begin{definition}[Convexity]
\label{ass:1_optimal}
Differentiable function $h:\R^d\to\R$ is $\mu$-(strongly) convex with $\mu \geq 0$ if
    \begin{equation}
    \label{eq:def_strongly_convex}
    h(y) \geq h(x) + \dotprod{\nabla h(x)}{y - x} + \frac{\mu}{2}\norm{y - x}^2,\quad\forall x,y \in \R^d.
    \end{equation}
We say $h$ is convex if it satisfies \eqref{eq:def_strongly_convex} with $\mu = 0$.
\end{definition}

\begin{definition}[Smoothness]
\label{ass:smooth}
Differentiable function $h:\R^d\to\R$ is $L$-smooth if
    \begin{equation*}
    \label{eq:smooth}
    \norm{\nabla h(x) - \nabla h(y)}\leq L\norm{x - y},\quad\forall x,y \in \R^d.
    \end{equation*}
\end{definition}

The above are standard definitions of convexity and smoothness. In Appendix~\ref{appendix:technicalities}, we provide weaker versions of these notions, and also derive some consequences that are useful for the analysis of algorithms developed in this thesis.

\subsection{Communication compression}
\label{sec:quant_and_comp}

One way to alleviate the communication bottleneck in FL is to employ message/gradient compression before communication.  We provide a detailed overview of different approaches that use compression techniques in Chapters 2--5.

We start with the definition of unbiased and general compression operators as commonly defined in literature~\cite{cordonnier2018convex, stich2018sparsified, koloskova2019decentralized}.

\begin{definition}[Unbiased Compression Operator]
\label{def:omegaquant} A randomized mapping $\cC\colon \R^d \to \R^d$  is an {\em unbiased compression operator (unbiased compressor)}  if there exists $\omega \geq 0$ such that
\begin{equation}
\label{eq:omega_quant}
 \E{\cC(x)}=x, \qquad \E{\norm{\cC(x)}^2} \leq (\omega + 1) \norm{x}^2, \qquad \forall x \in \R^d.
 \end{equation}
If this holds, we will for simplicity write $\cC\in \U(\omega)$.
\end{definition}

\begin{remark}
As $\E{ \norm{X-\E{X}}^2} = \E{\norm{X}^2} - \norm{\E{X}}^2$ for any random vector $X$, e.q.,\eqref{eq:omega_quant} implies 
\begin{align}
\E{\norm{\cC(x) - x}^2} \leq \omega \norm{x}^2\,, \qquad \forall x \in \R^d. \label{def:omega}
\end{align}
\end{remark}

\begin{definition}[General Compression Operator]
\label{def:deltaquant} A (possibly) randomized mapping $\cC\colon \R^d \to \R^d$  is a {\em general compression operator (general compressor)} if there exists $\delta \geq 1$ such that
\begin{equation}
\label{eq:quant} 
\E{\norm{\cC(x) - x}^2} \leq \lp 1 - \frac{1}{\delta}\rp \norm{x}^2, \qquad \forall x \in \R^d.
 \end{equation}
If this holds, we will for simplicity write $\cC\in \B(\delta)$ ($\B$ is supposed to invoke the word ``biased'', i.e., not necessarily unbiased).
\end{definition}

We use the term general compression operator because unbiased compression operators are a subset of general compression operators, as the lemma below states.
\begin{lemma}
$\cC\in \U(\omega)$ implies that $\frac{1}{\omega + 1}\cC\in \B(\omega + 1)$.
\end{lemma}

A particularly popular technique for compressing vectors is sparsification. Below, we include an example that leads to an unbiased compression operator. The second example is a sparsification operator that belongs to the class of general compression operators, but not to the class of unbiased compression operators.

\begin{example}[Random (aka Rand-$K$) sparsifiction]
The random sparsification operator $\cC:\R^d\to\R^d$ with sparsity parameter $k \in \{1,2, \hdots, d\}$ is defined by
$$\cC(x) \eqdef \tfrac{d}{k} \cdot \xi \circ x$$ 
where $\xi \sim_{\rm u.a.r.} \{y \in  \{0,1\}^d \colon \norm{y}_0 = k\}$ is a random vector with $k$ non-zero elements (i.e.,  $\norm{y}_0 = k$) equal to one sampled uniformly at random ({\rm u.a.r.}). $\cC \in \U(\omega)$ for $\omega = \tfrac{d}{k}-1$ unbiased compression operator.
\end{example}

\begin{example}[Greedy (aka Top-$k$) sparsifiction]
The Top-$k$ sparsification operator is defined as
\begin{equation*}
\cC(x) \eqdef \sum \limits_{i=d-k+1}^d x_{(i)} e_{(i)},
\end{equation*}
where $e_{j} \in \R^d$ denotes the all zeros vector with $1$ on $j$-th coordinate, $x_{(j)}$ is the $j$-th largest coordinate in magnitude, i.e., $\abs{x_{(1)}} \leq \abs{x_{(2)}} \leq \cdots \leq \abs{x_{(d)}}$, and $k \in \{1,2, \hdots, d\}$ is a sparsity parameter. The Top-$k$ sparsification operator $\cC \in \B(\delta)$ for $\delta = \tfrac{d}{k}$.
\end{example}

Other popular approaches include quantization and low-rank approximation. We discuss quantization, in particular random dithering, in Chapters 2 and 3. 
Lastly, we see that $\omega = 0$ in Definition~\ref{def:omegaquant} or $\delta=1$ in  Definition~\ref{def:deltaquant} imply that $\cC(x)$ is equal to $x$ almost surely.

\begin{remark}
Besides the variance bound, Definitions~\ref{def:omegaquant} and \ref{def:deltaquant} do not impose any further restrictions on $\cC$. However, for all applications it is advisable to consider operators $\cC$ that achieve a suitable level of compression, i.e., $\cC(x)$ should be cheaper to transmit/encode than $x$.
\end{remark}

\subsection{Arbitrary partial participation}
\label{sec:partial_participation}

This section introduces a novel client participation framework that we use throughout Chapters 5 and 7. It was initially introduced in the work presented in Chapter~\ref{chapter5:induced} in the context of FL, but the main theoretical results date back to our earlier work on importance sampling for non-convex variance reduced \texttt{SGD}~\cite{horvath2018nonconvex} that borrows some notions and tools from the previous works focusing on  randomized coordinate descent methods~\cite{richtarik2016optimal, csiba2018importance, hanzely2019accelerated}. 
In this framework, only a subset of all nodes communicates to the master node in each communication round. Methods of this type were analyzed before, but for the case of uniform subsampling~\cite{comms_fl2020tnnls, reisizadeh2020fedpaq}. In our work, we consider a more general partial participation framework: we assume that the subset of participating clients is determined by a fixed but otherwise arbitrary  random set-valued mapping $\Sam$ (a ``sampling'') with values in $2^{[n]}$, where $[n] = \{1,2, \dots,n\}$.

Note that the sampling $\Sam$ is uniquely defined by assigning probabilities to all $2^n$ subsets of $[n]$. With each sampling $\Sam$ we associate the {\em probability matrix} $\mP \in \R^{n\times n}$  defined by $\mP_{ij} \eqdef \Prob(\{i,j\}\subseteq \Sam )$. The {\em probability vector} associated with $\Sam$ is the vector composed of the diagonal entries of $\mP$: $p = (p_1,\dots,p_n)\in \R^n$, where $p_i\eqdef \Prob(i\in \Sam )$. We say that $\Sam$ is {\em proper} if $p_i>0$ for all $i$. It is easy to show that $b \eqdef \Exp{|\Sam|} = \trace{\mP} = \sum_{i=1}^n p_i$, and hence  $b$ can be seen as the expected number of clients  participating in each communication round. 

With a proper sampling $\Sam$ we associate a vector $v = [v_1, \hdots, v_n]^\top$ for which the following inequality holds
 \begin{equation}
 \label{eq:ESO_main}
    \mP -pp^\top \preceq {\rm \bf Diag}(p_1v_1,p_2v_2,\dots,p_n v_n),
\end{equation}
where $\mA \preceq \mB$ means that $\mB - \mA$ is a positive semi-definite matrix, i.e., $x^\top (\mB - \mA) x \geq 0$ for all $x\in \R^n$. 
Equation~\eqref{eq:ESO_main} is a technical tool that allows us to analyze the variance of the sampling by disentangling the joint effect of client sampling to individual clients. This is a key technique that enables us to perform complexity analysis and to obtain convergence guarantees. We observe that larger $v_i$'s lead to slower convergence, and we show that Equation~\eqref{eq:ESO_main} always holds for $v_i = n(1 - p_i)$. However, one can obtain tighter bound for specific samplings, for instance, one can show that uniform sampling with $b$ participating clients admits $v_i = \nicefrac{n-b}{n-1}$, and full participation allows to set $v_i = 0$ as $\mP$ is the all ones matrix; see \cite{horvath2018nonconvex} for details.
Equipped with this background, we are ready to proceed with the chapters containing this thesis' main results.


\singlespacing

\chapter{Natural compression for distributed deep learning}
\label{chapter2:c_nat}

\section{Introduction}
\label{sec:intro}

Modern deep learning models~\cite{resnet} are almost invariably trained in parallel or distributed environments, which is necessitated by the enormous size of the data sets and dimension and complexity of the models required to obtain state-of-the-art performance.  In this chapter, the focus is on the \emph{data-parallel} paradigm, in which the training data is split  across several workers capable of operating in parallel~\cite{bekkerman2011scaling,recht2011hogwild}.  Formally, we consider optimization problems of the form
\begin{align}
  \min \limits_{x \in \R^d}  \sbr*{f(x) \eqdef \frac{1}{n} \sum \limits_{i=1}^n f_i(x)}  \,, \label{eq:probR}
\end{align}
where  $x\in \R^d$ represents the parameters of the model, $n$ is the number of workers, and $f_i \colon \R^d \to \R$ is a loss function composed of data stored on worker $i$. 

\begin{figure}[t] 
    \centering
    \includegraphics[width=0.4\textwidth]{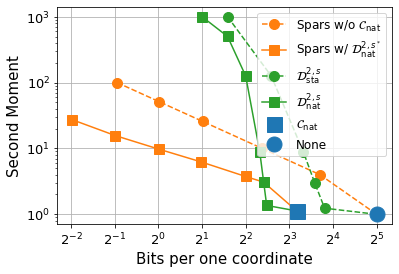}
    \caption{Communication (in bits) vs. the second moment $\omega+1$ (see Equation~\eqref{eq:omega_quant}) for several state-of-the-art compressors applied to a gradient  of size $d=10^6$. Our methods ($\NC$ and $\ND$) are depicted with a square marker. For any fixed communication budget, natural dithering offers an exponential improvement on standard dithering, and when used in composition with sparsification, it offers an order of magnitude improvement.  }
    \label{fig:comp_var_bits}
    \end{figure}
  
 \section{Related work}
 In this section, we discuss previous approaches to tackle \eqref{eq:probR}.
\subsection{Distributed learning}
Typically, problem \eqref{eq:probR} is solved by distributed stochastic gradient descent ({\tt \texttt{SGD}}) \cite{SGD}.

A key bottleneck of the above algorithm, and of its many variants (e.g.,  variants utilizing mini-batching~\cite{Goyal2017:large}, importance sampling~\cite{horvath2018nonconvex}, momentum ~\cite{nesterov2013introductory}, or variance reduction~\cite{johnson2013accelerating}), is the cost of communication of the typically dense gradient vector $g_i(x^k)$, and in a parameter-sever implementation with a master node,  also the cost of broadcasting the aggregated gradient. These are  $d$ dimensional vectors of floats, with $d$ being very large in modern deep learning. 
It is well-known \cite{1bit,qsgd2017neurips,zhang2016zipml,deepgradcompress2018iclr,lim20183lc} that in many practical applications with common computing architectures, communication takes much more time than computation, creating a bottleneck of the entire training system.
\subsection{Communication reduction} Several solutions were suggested in the literature as a remedy to this problem. In one strain of work, the issue is addressed by giving each worker ``more work'' to do, which results in  a better communication-to-computation ratio. For example, one may use mini-batching to construct more powerful gradient estimators~\cite{Goyal2017:large}, define local problems for each worker to be solved by a more advanced  local solver~\cite{Shamir2014:approxnewton, Hydra, Reddi:2016aide}, or reduce communication frequency (e.g., by communicating only once~\cite{Mann2009:parallelSGD,Zinkevich2010:parallelSGD} or once every few iterations~\cite{local_SGD_stich_18}). An orthogonal approach to the above efforts aims to reduce the size of the communicated vectors instead~\cite{1bit,qsgd2017neurips,terngrad,tonko,hubara2017quantized} using various  lossy (and often randomized) {\em compression} mechanisms, commonly  known in the literature as quantization techniques. In their most basic form, these schemes decrease the \# bits used to represent floating point numbers forming the communicated $d$-dimensional vectors~\cite{gupta2015deep, Na2017:limitedprecision}, thus reducing the size of the communicated message by a constant factor. Another possibility is to apply randomized \emph{sparsification} masks to the gradients~\cite{Suresh2017, RDME, alistarh2018sparse, stich2018sparsified},  or to rely on coordinate/block descent updates-rules, which are sparse by design~\cite{Hydra2}. 
\newline
One of the most important considerations in the area of compression operators is the  {\em compression-variance} trade-off \cite{RDME, qsgd2017neurips, diana2}. For instance, while random dithering approaches attain up to $\cO(d^{\nicefrac{1}{2}})$ compression~\cite{1bit,qsgd2017neurips,terngrad}, the most aggressive schemes reach $\cO(d)$ compression by sending a constant number of bits per iteration only~\cite{Suresh2017,RDME,alistarh2018sparse,stich2018sparsified}.  However, the more compression is applied, the more information is lost, and the more will the quantized vector differ from the original vector we  want to communicate, increasing its statistical variance.  Higher variance implies slower convergence \cite{qsgd2017neurips, mishchenko2019distributed}, i.e., more communication rounds. So, ultimately, compression approaches offer a trade-off between the communication cost per iteration and the number of communication rounds. 
\newline
Outside of the optimization for machine learning, compression operators are very relevant to optimal quantization theory and control theory~\cite{elia2001stabilization, sun2011scalar, sun2012framework}.
%

\section{Contributions} 
The key contributions of this chapter are following:

    \subsection{New compression operators} We construct a new  {\em ``natural compression''} operator ($\NC$; see Sec.~\ref{sec:nat_compression}) based on a randomized rounding scheme in which each float of the compressed vector is rounded to a (positive or negative) power of 2. 
    This compression has a provably small variance, at most $\nicefrac{1}{8}$ (see Theorem~\ref{lem:br_quant}), which implies that theoretical convergence results of \texttt{SGD}-type methods are essentially unaffected (see Theorem~\ref{thm:arbSGD_short}). At the same time,  substantial savings are obtained in the amount of communicated bits per iteration ($3.56\times$ less for {\em float32} and $5.82\times$ less for {\em float64}). In addition, we utilize these insights and develop a new random dithering operator---{\em natural dithering} ($\ND$; see Sec.~\ref{sec:ND})---which is {\em exponentially better} than the very popular ``standard'' random dithering operator (see Theorem~\ref{thm:expon_better}). We remark that $\NC$ and the identity operator arise as limits of $\ND$ and $\SD$ as $s \to \infty$, respectively.  Importantly, our new compression techniques can be {\em combined} with existing compression and sparsification operators for a more dramatic effect as we argued before. 
    \subsection{State-of-the-art compression}
    When compared to previous state-of-the-art compressors such as (any variant of) sparsification and dithering---techniques used in methods such as  Deep Gradient Compression~\cite{deepgradcompress2018iclr}, \texttt{QSGD}~\cite{qsgd2017neurips} and
    \texttt{TernGrad}~\cite{terngrad}, our compression 
    operators offer provable and often large improvements in practice, thus 
    leading to {\em new state of the art}. In particular, given a budget on the second moment $\omega+1$ (see Equation~\eqref{eq:omega_quant}) of a compression operator, which is the main factor influencing the increase in the number of communications when communication compression is applied compared to no compression,
    our compression operators offer the largest compression factor, resulting in fewest bits transmitted (see Figure~\ref{fig:comp_var_bits}).  
    \subsection{Lightweight \& simple low-level implementation} We show that apart from a randomization procedure (which is inherent in all unbiased compression operators), natural compression is {\em computation-free}. Indeed, natural compression essentially amounts to the trimming of the mantissa and possibly increasing the exponent by one. This is the first compression mechanism with such a  ``natural'' compatibility with binary floating point types. 
    \subsection{Proof-of-concept system with in-network aggregation (INA)} The recently proposed SwitchML~\cite{switchML} system alleviates the communication  bottleneck via  in-network  aggregation (INA) of gradients. Since current programmable network switches are only capable of adding integers, new update compression methods are needed which can supply outputs in an integer format. Our {\em natural compression} mechanism is the first that is provably able to operate in the SwitchML framework as it communicates integers only: the sign, plus the bits forming the exponent of a float. Moreover,  having bounded (and small) variance, it is compatible with existing distributed training methods. 
    \subsection{Bidirectional compression for \texttt{SGD}} We provide convergence theory for distributed \texttt{SGD} which allows for {\em compression both at the worker and master side}  (see Algorithm~\ref{alg:arbSGD}).  The compression operators compatible with our theory form a large family (operators $\cC\in \U(\omega)$ for some finite $\omega\geq 0$; see Definition~\ref{def:omegaquant}). This enables safe experimentation with existing and facilitates the development of new compression operators fine-tuned to specific deep learning model architectures.  Our convergence result (Theorem~\ref{alg:arbSGD})  applies to  smooth and non-convex functions, and our rates predict  linear speed-up with respect to the number of machines. 
    \subsection{Better total complexity} Most importantly, we are the first to {\em prove} that the increase in the number of iterations caused by (a carefully designed) compression is more than compensated by the savings in communication, which leads to an overall provable speedup in training time. Read Theorem~\ref{thm:arbSGD_short}, the discussion following the theorem and Table~\ref{tab:algo-comparison_2} for more details.
    To the best of our knowledge,  standard dithering (\texttt{QSGD}~\cite{qsgd2017neurips}) is the only previously known compression technique able to achieve this with our distributed \texttt{SGD} with bi-directional compression. Importantly,  our natural dithering is exponentially better than standard dithering, and hence provides for state-of-the -art performance in connection with Algorithm~\ref{alg:arbSGD}.
    \subsection{Experiments} We show that  popular compression methods such as random sparsification and random dithering are enhanced by combination with natural compression or natural dithering. The combined compression technique reduces the number of communication rounds without any noticeable impact on convergence providing the same quality solution.

\begin{figure}[t]
    \centering
    \includegraphics[width=0.8\textwidth]{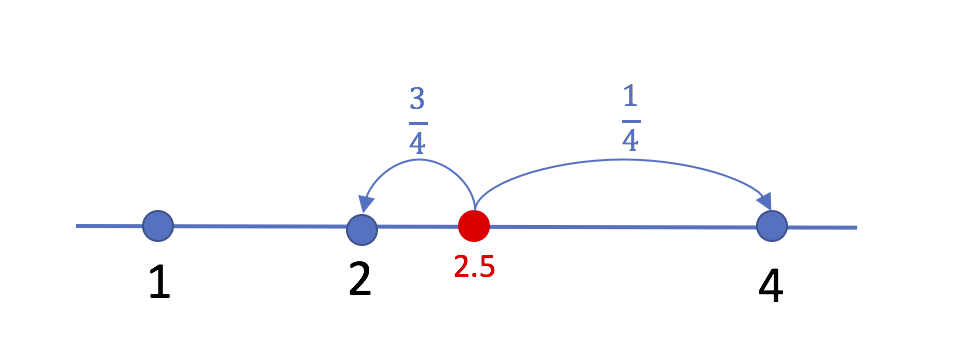}
    \caption{
    An illustration of nat.\ compression applied to $t=2.5$: $\NC(2.5) = 2$ with probability $\frac{4-2.5}{2}=0.75$, and $\NC(2.5) = 4$ with prob.\ $\frac{2.5-2}{2}=0.25$. This choice of probabilities ensures that the compression operator is unbiased, i.e., $\E{\NC(t)}\equiv t$.}
    \label{fig:nur}
\end{figure}

\section{Natural compression}\label{sec:nat_compression}

We define a new  (randomized) compression technique, which we call {\em natural compression}. This is fundamentally a function mapping $t\in \R$ to a random variable $\NC(t)\in \R$. In case of vectors  $x=(x_1,\dots,x_d)\in \R^d$ we apply it in an element-wise fashion: $(\NC(x))_i = \NC(x_i)$. Natural compression $\NC$ performs a randomized logarithmic rounding of its input $t\in \R$. Given nonzero $t$, let $\alpha\in \R$ be such that $|t|=2^\alpha$ (i.e., $\alpha = \log_2 |t|$). Then $ 2^{\lfloor  \alpha \rfloor} \leq |t| = 2^\alpha \leq 2^{\lceil \alpha \rceil}$ and we round $t$ to either $\signum (t) 2^{\lfloor \alpha \rfloor}$, or to $\signum (t) 2^{\lceil \alpha \rceil}$. When $t=0$, we set $\NC(0)=0$. The probabilities are chosen so that $\NC(t)$ is an unbiased estimator of $t$, i.e., $\E{\NC(t)} = t$ for all $t$. For instance, $t=-2.75$ will be rounded to either $-4$ or $-2$ (since $-2^2 \leq -2.75 \leq -2^1$), and  $t = 0.75$ will be rounded to either   $\nicefrac{1}{2}$ or $1$ (since $2^{-1} \leq 0.75 \leq 2^0$). As a consequence, if  $t$ is an integer power of 2, then $\NC$ will leave $t$ unchanged, see Figure.~\ref{fig:nur}.  
\begin{definition}[Natural compression]\label{def:NC}
Natural compression is a random function $\NC: \R \mapsto \R$ defined as follows. We set $\NC(0)=0$. If $t\neq 0$, we let
 \begin{equation}
\NC(t) \eqdef \begin{cases} \signum (t) \cdot 2^{\floor{\log_2 \abs{t}}} , \text{ with }  p(t),\\
\signum(t) \cdot  2^{\ceil{\log_2 \abs{t}}} , \text{ with } 1-p(t),
\end{cases}
\label{def:rr_vec}
 \end{equation}
 where probability $ p(t)\eqdef \frac{2^{\ceil{\log_2 \abs{t}}}-\abs{t}}{2^{\floor{\log_2 \abs{t}}}}.$
\end{definition}
Alternatively, \eqref{def:rr_vec} can be written as 
$
 \NC(t)= \signum (t) \cdot 2^{\floor{\log_2 \abs{t}}}(1+\lambda(t)),
$
  where $\lambda(t)\sim {\rm Bernoulli}(1-p(t))$; that is, $\lambda(t)=1$ with prob.\ $1 - p(t)$ and  $\lambda(t)=0$ with prob.\ $p(t)$. The key properties of any (unbiased) compression  operator are variance, ease of implementation, and compression level. We characterize the remarkably low variance of  $\NC$  and describe an (almost) effortless and {\em natural}  implementation, and the compression it offers in rest of this section.
\newline
\subsection{$\NC$ has a negligible variance: $\omega=\nicefrac{1}{8}$}  We identify natural  compression as belonging to a large class of unbiased  compression operators with bounded second moment as in Definition~\ref{def:omegaquant}. 


If Definition~\ref{def:omegaquant} holds, we say that ``$\cC$ has variance $\omega$''. 
The importance of $\U(\omega)$ stems from two observations. First, operators from this class are known to be compatible with several optimization algorithms~\cite{khirirat2018distributed,diana2}. Second, this class includes most compression operators used in practice~\cite{qsgd2017neurips,terngrad,tonko,mishchenko2019distributed}.   In general, the larger $\omega$ is, the higher compression level might be achievable, and the worse impact compression has on the convergence speed. 
The main result of this section says that the natural compression operator $\NC$  has variance $\nicefrac{1}{8}$.
\begin{theorem}
\label{lem:br_quant} $\NC\in \U(\nicefrac{1}{8})$. 
\end{theorem}


Consider now a similar unbiased randomized rounding operator to $\NC$; but one that rounds to one of the nearest integers (as opposed to integer powers of 2). We call it $\Qint$. At first sight, this may seem like a reasonable alternative to $\NC$. However, as we show next, $\Qint$ does not have a finite second moment and is hence incompatible with existing optimization methods.

\begin{theorem} \label{prop:introunding_negative}
There is no $\omega\geq 0$ such that $\Qint \in \U(\omega)$.
\end{theorem}

\subsection{From 32 to 9 bits, with lightning speed} We now explain that performing natural compression of a real number in a binary floating point format is computationally cheap. In particular, excluding the randomization step, $\NC$ amounts to simply dispensing off the mantissa in the binary representation. The most common computer format for real numbers, {\em binary$32$} (resp.\  {\em binary$64$}) of the IEEE 754 standard, represents each number with $32$ (resp.\ $64$) bits, where the first bit represents the sign, $8$ (resp.\ $11$) bits are used for the exponent, and the remaining $23$ (resp. $52$) bits  are used for the mantissa. A scalar $t\in \R$ is represented in the form $(s,e_7,e_6,\dots,e_0,m_1,m_2,\dots,m_{23})$, where $s,e_i,m_j \in \{0,1\}$ are bits, via the relationship
$$
  t  = (-1)^s \times 2^{e-127} \times (1 + m),\;
 e  = \sum \limits_{i=0}^7 e_i 2^i, \; m = \sum \limits_{j=1}^{23} m_j 2^{-j},
$$
where $s$ is the {\em sign}, $e$ is the {\em exponent} and $m$ is the {\em mantissa}.
 A  {\em binary$32$} representation of  $t=-2.75$ is visualized in Figure~\ref{fig:float32}.  In this case, $s=1$, $e_7=1$, $m_2=m_3=1$ and hence $t = (-1)^s \times 2^{e-127} \times (1+m) = -1 \times 2 \times (1+2^{-2}+2^{-3}) = -2.75$.
\begin{figure}[t]
\centering
\includegraphics[width=0.8\textwidth]{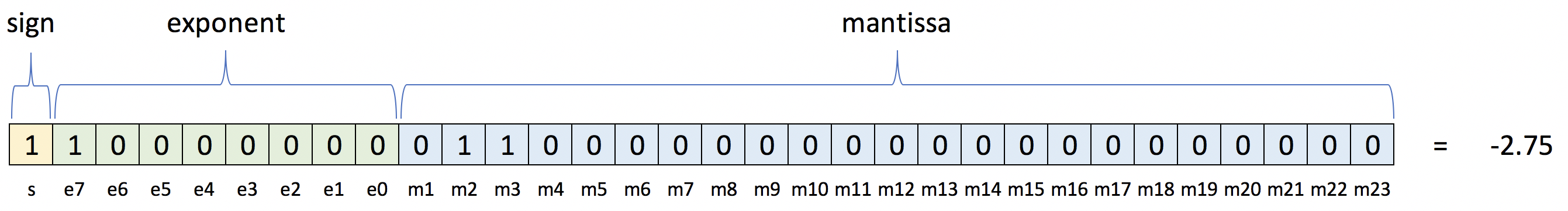}
\caption[]{IEEE 754 single-precision binary floating-point format: {\em binary$32$}.}
\label{fig:float32}
\end{figure}

It is clear that $0\leq m < 1$, and hence $ 2^{e-127} \leq |t| < 2^{e-126}$. Moreover, $p(t) = \frac{2^{e-126} - |t|}{2^{e-127} } = 2 - |t|2^{127-e} = 1 - m.$ Hence, natural compression of $t$ represented as {\em binary$32$} is given as follows: 
\[\NC(t) = \begin{cases}(-1)^s \times 2^{e-127},  \text{ with probability }  1 - m, \\
(-1)^s \times 2^{e-126},  \text{ with probability }  m. \\
\end{cases}\]
Observe that $(-1)^s \times 2^{e-127}$ is obtained from $t$ by setting the mantissa $m$ to zero, and keeping both the sign $s$ and exponent $e$ unchanged. Similarly,  $(-1)^s \times 2^{e-126}$ is obtained from $t$ by setting the mantissa $m$ to zero, keeping the sign $s$, and increasing the exponent by one.
Hence, {\em both values can be computed from $t$ essentially without any computation.}
\newline
\subsection{Communication savings} In summary, in case of {\rm binary$32$}, the output $\NC(t)$ of natural compression is encoded using the 8 bits in the exponent and an extra bit for the sign. {\em This is $3.56\times$ less communication.} In case of {\rm binary$64$}, we only need 11 bits for the exponent and 1 bit for the sign, and {\em this is $5.82\times$ less communication.}

\subsection{Compatibility with other compression techniques} We start with a simple but useful observation about composition of compression operators.
\begin{theorem}
\label{thm:composition}
If $\cC_1\in \U(\omega_1)$ and $\cC_2 \in \U(\omega_2)$, then  $\cC_1 \circ \cC_2 \in \U(\omega_{12})$, where $\omega_{12} = \omega_1\omega_2+\omega_1 + \omega_2$, and  $ \cC_1\circ \cC_2$ is the composition defined by $(\cC_1\circ \cC_2)(x) = \cC_1(\cC_2(x))$. 
\end{theorem}
Combining this result with Theorem.~\ref{lem:br_quant}, we observe that for any $\cC\in \U(\omega)$, we have $\NC \circ \cC\in \U(\nicefrac{9 \omega}{8}+ \nicefrac{1}{8})$. Since $\NC$ offers substantial communication savings with only a negligible effect on the variance of $\cC$, a key use for natural compression beyond applying it as the sole compression strategy is to deploy it with other effective techniques as a final compression mechanism (e.g., with sparsifiers~\cite{stich2018sparsified}), boosting the performance of the system even further. However, our technique will be useful also as a post-compression mechanism for compressions that do not belong to $\U(\omega)$ (e.g., TopK sparsifier~\cite{alistarh2018sparse}). The same comments apply to the {\em natural dithering} operator $\ND$, defined in the next section.

\begin{figure}[t]
        \centering
		\includegraphics[width=.8\textwidth]{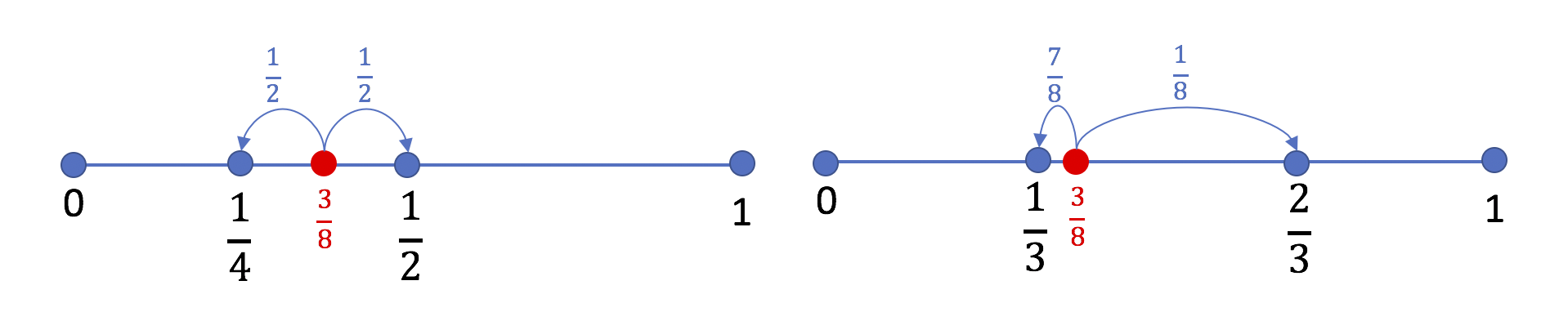}
		\caption{Randomized rounding for natural (left) and standard  (right) dithering ($s=3$ levels).}
\label{fig:rd_vs_brd}
\end{figure}

\section{Natural dithering} \label{sec:ND}

Motivated by the natural compression introduced in Sec~\ref{sec:nat_compression}, here we propose a new random dithering operator which we call {\em natural dithering}.  However, it will be useful to introduce a more general dithering operator, one generalizing both the natural and the standard dithering operators. For $1\leq p \leq +\infty$, let $\norm{x}_p$ be $p$-norm: $\|x\|_p \eqdef (\sum_i |x_i|^p)^{\nicefrac{1}{p}}$.

\begin{definition}[General dithering] The {\em general dithering} operator with respect to the $p$ norm and with $s$ levels $0 = l_s <l_{s-1} < l_{s-2} < \dots < l_{1} < l_0 = 1$, denoted $\GD$, is defined as follows. Let $x\in \R^d$.  If $x=0$, we let $\GD(x)=0$. If $x\neq 0$,  we let $y_i \eqdef |x_i|/\|x\|_p$ for all $i\in [d]$. Assuming $l_{u+1} \leq y_i \leq l_{u}$ for some $u \in \{0,1,\dots,s-1\}$, we let
$
\left(\GD(x)\right)_i = \cC(\norm{x}_p) \times \signum(x_i) \times \xi(y_i) \; , \
$
where $\cC \in \U(\omega)$ for some $\omega \geq 0$ and $\xi(y_i)$ is a random variable equal to $l_u$ with probability $\tfrac{y_i-l_{u+1}}{l_{u}-l_{u+1}}$, and to $l_{u+1}$ with probability $\tfrac{l_{u}-y_i}{l_{u}-l_{u+1}}$. Note that $\E{\xi(y_i)} = y_i$.
\end{definition}

Standard (random) dithering, $\SD$,  \cite{goodall1951television, roberts1962picture} is obtained as a special case of general dithering (which is also novel) for a linear partition of the unit interval, $l_{s-1} = \nicefrac{1}{s}$, $l_{s-2} = \nicefrac{2}{s}$, \dots, $l_1 = \nicefrac{(s-1)}{s}$ and $\cC$ equal to the identity operator. $\cD_{\rm sta}^{2,s}$ operator was used in \texttt{QSGD}~\cite{qsgd2017neurips} and $\cD_{\rm sta}^{\infty,1}$ in Terngrad~\cite{terngrad}.  {\em Natural dithering}---a novel compression operator introduced in this chapter---arises as a special case of general dithering for $\cC$ being an identity operator and a binary geometric partition of the unit interval: $l_{s-1} = 2^{1-s}$, $l_{s-2} = 2^{2-s}$, \dots, $l_1 = 2^{-1}$. For the  INA application, we apply $\cC = \NC$ to have output always in powers of $2$, which would introduce extra factor of $\nicefrac{9}{8}$ in the second moment. A comparison of the $\xi$ operators for the standard and natural  dithering with $s=3$ levels applied to $t=\nicefrac{3}{8}$ can be found in Figure~\ref{fig:rd_vs_brd}.
When $\GD$ is used to compress gradients, each worker communicates the norm (1 float), vector of signs ($d$ bits) and efficient encoding of the effective levels for each entry $i=1,2,\dots,d$. 
Note that $\ND$ is essentially an application of $\NC$ to all normalized entries of $x$, with two differences: i) we can also communicate the compressed norm $\|x\|_p$, ii) in $\NC$ the interval  $[0,2^{1-s}]$ is subdivided further, to machine precision, and in this sense  {\em $\ND$ can be seen as a limited precision variant of $\NC$.} As is the case with $\NC$, the mantissa is ignored, and one communicates exponents only. The norm compression is particularly useful on the master side since multiplication by a naturally compressed norm is just summation of the exponents.
\newline
The main result of this section establishes natural dithering as belonging to the class $\U(\omega)$:
\begin{theorem}
\label{thm:natural_dithering}
 $\ND \in \U(\omega)$, where   
$
\omega = \nicefrac{1}{8}   +  d^{\nicefrac{1}{r}} 2^{1-s} \min\left\{1, d^{\nicefrac{1}{r} }   2^{1-s}\right\},
$
and  $r = \min\{p,2\}$. 
\end{theorem} 
\begin{figure}[t]
\centering
\includegraphics[width=0.8\textwidth]{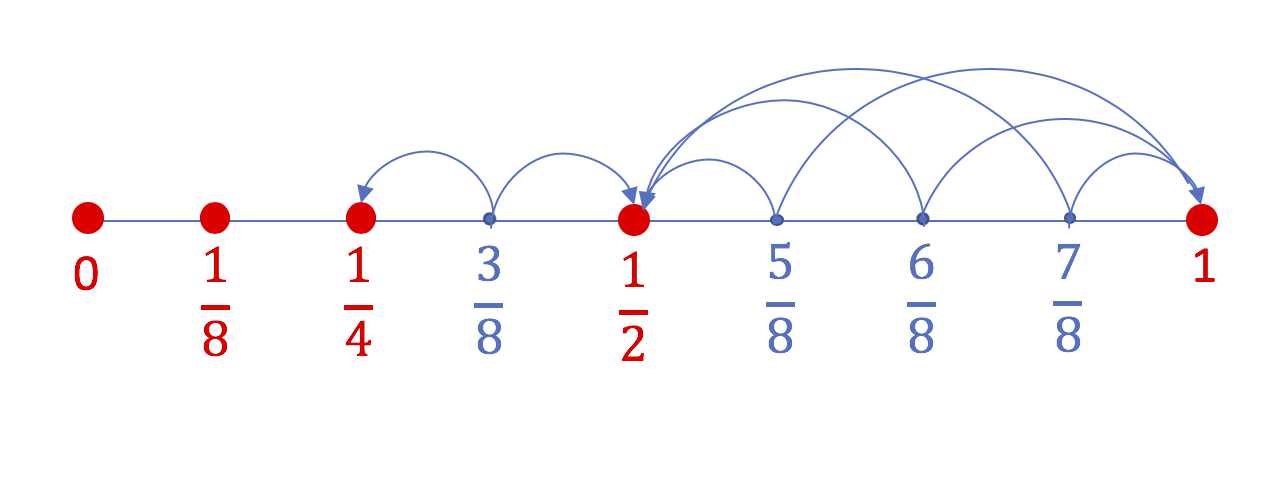}
\caption{1D (i.e., $d=1$) visualization of the workings of natural dithering $\ND$  and standard dithering $\SDs{u}$ with $u = 2^{s-1}$, with $s=4$. Notice that the numbers standard dithering rounds to, i.e., $0,\nicefrac{1}{8}, \nicefrac{2}{8}, \dots,\nicefrac{7}{8}, 1$,  form a {\em superset} of the numbers  natural dithering rounds to, i.e., $0,2^{-3}, 2^{-2}, 2^{-1}, 1$. Importantly,  while standard dithering uses $u=2^{4-1} = 8$ levels (i.e., intervals) to achieve a certain fixed variance, natural dithering only needs $s=4$ levels to achieve the same variance. This is an exponential improvement in compression (see Theorem~\ref{thm:expon_better} for the formal statement). }
\label{fig:brrd_with_rd}
\end{figure}
To illustrate the strength of this result, we now compare natural dithering $\ND$ to  standard dithering $\SD$ and show that {\em natural dithering is exponentially better than  standard dithering.}  In particular, for the same level of variance, 
$\ND$ uses only $s$ levels while $\SDs{u}$ uses $u=2^{s-1}$ levels. Note also that  the levels used by $\ND$ form a {\em subset} of the levels used by $\SD$ (see Figure~\ref{fig:brrd_with_rd}). We also confirm this empirically (see Section~\ref{sec:emp_variance}).

\begin{theorem}\label{thm:expon_better}
Fixing $s$, natural dithering $\ND$  has $\cO(2^{s-1}/s)$ times smaller variance than standard dithering $\SD$. Fixing $\omega$, if $u = 2^{s-1}$, then $\SDs{u} \in \U(\omega)$ implies that $\ND\in \U(\nicefrac{9}{8}(\omega+1) - 1)$.
\end{theorem}

\begin{algorithm}[t]
\begin{algorithmic}[1]
 \STATE {\bfseries Input:} learning rates $\{\eta^k\}_{k=0}^{T} > 0$, initial vector $x^0$
 \FOR{$k=0,1,\dots T$}
	\STATE {\bfseries Parallel: Worker side}
  	\FOR{$i=1,\dots,n$ }
		\STATE compute  a stochastic gradient $g_i(x^k)$ (of $f_i$ at $x^k$) 
		\STATE compress it $\Delta_i^k = \cC_{W_i}(g_i(x^k))$\;
 	\ENDFOR
 	\STATE {\bfseries Master side}
 	\STATE aggregate $ \Delta^k = \sum_{i=1}^n \Delta_i^k$
 	\STATE compress $g^k = \cC_M(\Delta^k)$ and broadcast to each worker
 	\STATE {\bfseries Parallel: Worker side}
  	\FOR{$i=1,\dots,n$ }
		\STATE $x^{k+1} = x^k - \tfrac{\eta^{k}}{n}g^k$\; 
 	\ENDFOR
 \ENDFOR
\end{algorithmic}  
\caption{Distributed \texttt{SGD} with bidirectional compression}
\label{alg:arbSGD}
\end{algorithm}


\begin{table}[t]
\begin{center}
\small
\begin{tabular}{|c|c|c|c|c|}
\hline
Approach &$\cC_{W_i}$  &  No. iterations  &  Bits per $1$ Iter.  & Speedup  \\
 &   & $T'(\omega_W) = \cO((\omega_W+1)^\theta)$  &  $W_i \mapsto M$ & Factor  \\
\hline 
Baseline & identity  & $1$  & $32d$  & 1   \\
\bf {\color{blue} New} & {\color{blue}$\NC$ } & {\color{blue}$(\nicefrac{9}{8})^{\theta}$} &  {\color{blue}$9d$} & {\color{blue}$3.2 \times$--$3.6\times$ } \\
\hline
Sparsification & $\cS^q$ & $(\nicefrac{d}{q})^{\theta}$ & $(33 +\log_2d)q$ & $0.6\times$--$6.0\times$   \\
\bf {\color{blue} New} & ${\color{blue}\NC} \circ \cS^q$  &  {\color{blue}$(\nicefrac{9d}{8q})^{\theta}$} & {\color{blue}$(10+\log_2d)q$} & {\color{blue}$1.0\times$--$10.7\times$ } \\
\hline
Dithering & $\SDs{2^{s-1}}$  & $(1 +\kappa d^{\nicefrac{1}{r}}2^{1-s})^{\theta}$  & $31+d(2+s)$ & $1.8\times$--$15.9\times$  \\
\bf {\color{blue} New} & {\color{blue}$\ND$}  & {\color{blue}$(\nicefrac{9}{8} + \kappa d^{\frac{1}{r}}2^{1-s})^\theta$} &  {\color{blue}$31+d(2+\log_2s  )$} & {\color{blue}$4.1\times$--$16.0\times$} \\
\hline
\end{tabular}
\end{center} 
\caption{
The overall speedup of distributed \texttt{SGD} with compression on nodes via $\cC_{W_i}$ over a Baseline variant without compression. Speed is measured by multiplying the \# communication rounds (i.e., iterations  $T(\omega_W)$) by the bits sent from worker $i$ to master ($W_i \mapsto M$) per 1 iteration. We neglect $M \mapsto W_i$ communication as in practice this is often much faster (see e.g. \cite{mishchenko2019distributed}, for other cost/speed model see Appendix~\ref{sec:dif_regimes}). We do not just restrict to this scenario and . We assume {\em binary$32$} representation. The relative \# iterations sufficient to guarantee $\varepsilon$ optimality is $T'(\omega_W) \eqdef(\omega_W+1)^\theta$, where  $\theta\in (0,1]$ (see Theorem~\ref{thm:arbSGD_short}). Note that in the big $n$ regime the iteration bound $T(\omega_W)$ is better due to $\theta\approx 0$ (however, this is not very practical as $n$ is usually small), while for small $n$ we have $\theta\approx 1$. For dithering, $r=\min\{p,2\}$,  $\kappa =  \min \{1, \sqrt{d}   2^{1-s}\}$. The lower bound for the Speedup Factor is obtained for $\theta=1$, and the upper bound for $\theta=0$. The Speedup Factor $\left(\frac{T(\omega_W)\cdot\text{\# Bits}}{T(0)\cdot 32d}\right)$ figures were calculated for $d=10^6$, $q=0.1d$ (10\% sparsity), $p=2$ and optimal choice of $s$ with respect to speedup.}
\label{tab:algo-comparison_2}
\end{table}

\section{Distributed \texttt{SGD} } \label{sec:SGD}

There are several stochastic gradient-type methods~\cite{SGD, bubeck2015convex,ghadimi2013stochastic,mishchenko2019distributed} for solving \eqref{eq:probR} that are compatible with compression operators  $\cC\in \U(\omega)$, and hence also with our natural compression ($\NC$) and natural dithering ($\ND$) techniques. However, as none of them support compression at the master node
we propose a distributed \texttt{SGD} algorithm that allows for {\em bidirectional compression} (Algorithm~\ref{alg:arbSGD}). 
First, each worker computes its stochastic gradient $g_i(x^k)$, this is then compressed using a compression operator $\cC_{W_i}$ (this can be different for every node, for simplicity, one can assume that they are all the same) and send to the master node. The master node then aggregates the updates from all the workers, compresses them with its operator $\cC_M$ and broadcasts updates back to the workers, updating their local copy of the solution parameter $x$. 
We note that there are two concurrent papers to ours (all appeared online in the same month and year) proposing the use of bidirectional compression, albeit in conjunction with different underlying algorithms, such as \texttt{SGD} with  error feedback or local updates~\cite{DoubleSqueeze2019, zheng2019communication}.  Since we instead focus on vanilla distributed \texttt{SGD} with bidirectional compression, the algorithmic part of our work is complementary to theirs. Moreover, our key contribution---the highly efficient natural compression and dithering compressors---can be used within their algorithms as well, which expands their impact further.
\\
 We assume repeated access to unbiased stochastic gradients $g_i(x^k)$ with bounded variance $\sigma_i^2$  for every worker $i$. We also assume {\em node similarity} represented by constant $\zeta_i^2$,  and that $f$ is $L$-smooth (gradient is $L$-Lipschitz). 
 Formal definitions can be found in Appendices~\ref{appendix:technicalities} and \ref{sec:gen_sgd}.  
 We denote $\zeta^2 = \frac{1}{n} \sum_{i=1}^n \zeta_i^2$, $\sigma^2 = \frac{1}{n} \sum_{i=1}^n \sigma_i^2$ and 
\begin{equation}
\begin{split}
\alpha = \tfrac{(\omega_M+1)(\omega_W+1)}{n}\sigma^2 + \tfrac{(\omega_M+1)\omega_W}{n} \zeta^2, \quad
\beta = 1+ \omega_M + \tfrac{(\omega_M+1)\omega_W}{n}  \,,
\end{split}
\label{eq:alpha_beta_out}
\end{equation}
where $\cC_M \in \U(\omega_M)$ is the compression operator used by the master node,  $\cC_{W_i} \in \U(\omega_{W_i})$ are the compression operators used by the workers and $\omega_W \eqdef \max_{i \in [n]}\omega_{W_i}$. 

The main theorem follows:

\begin{theorem}
 \label{thm:arbSGD_short}
 Let  $\cC_M \in\U(\omega_M)$, $\cC_{W_i} \in \U(\omega_{W_i})$ and $\eta^k \equiv \eta \in (0,\nicefrac{2}{\beta L})$,
where $\alpha, \beta$ are as in \eqref{eq:alpha_beta_out}. 
 If $a$ is picked uniformly at random from $\{0,1, \cdots, T-1\}$, then 
\begin{equation}\label{eq:main_thm_SGD}
\E{\|\nabla f(x^a)\|^2}
\leq 
\frac{2(f(x^0)-f(x^{\star}))}{\eta (2-\beta L \eta) T} + \frac{\alpha L \eta }{2-\beta L \eta},
\end{equation}
where $x^{\star}$ is an optimal solution of \eqref{eq:probR}.  In particular, if we fix any $\varepsilon>0$ and choose 
$
\eta = \frac{ \epsilon}{L(\alpha + \varepsilon \beta)}
$ 
and 
$
T\geq \frac{2L(f(x^0)-f(x^{\star}))(\alpha + \epsilon \beta)}{\varepsilon^2}
$,
then $\E{\|\nabla f(x^a)\|^2} \leq \varepsilon$.
\end{theorem}

The above theorem has some interesting consequences. First, notice that \eqref{eq:main_thm_SGD} posits a $\cO(\nicefrac{1}{T})$ convergence of the gradient norm to the value $\tfrac{\alpha L \eta }{2-\beta L \eta}$, which depends linearly on $\alpha$. In view of \eqref{eq:alpha_beta_out}, the more compression we perform, the larger this value becomes. More interestingly, assume now that the same compression operator is used at each worker: $\cC_W =\cC_{W_i}$. Let $\cC_W \in \U(\omega_W)$ and  $\cC_M \in \U(\omega_M)$ be the compression on master side. Then,  $T(\omega_M, \omega_W) \eqdef 2L(f(x^0)-f(x^{\star}))\varepsilon^{-2}(\alpha + \varepsilon \beta)$  is its iteration complexity.  In the special case of equal data on all nodes, i.e., $\zeta=0$, we get $\alpha =  \nicefrac{(\omega_M+1)(\omega_W+1) \sigma^2}{n}$ and $\beta = (\omega_M+1)\left(1+\nicefrac{\omega_W}{n}\right)$. If no compression is used, then $\omega_W = \omega_M = 0$ and $\alpha +\varepsilon \beta = \nicefrac{\sigma^2}{n} +\varepsilon$. So, the {\em relative slowdown} of Algorithm~\ref{alg:arbSGD} used {\em with} compression compared to Algorithm~\ref{alg:arbSGD} used {\em without} compression is given by
\begin{eqnarray*}  \tfrac{T(\omega_M, \omega_W)}{T(0,0)} 
= (\omega_M+1)  \nicefrac{\left(\tfrac{(\omega_W+1)\sigma^2}{n} +(1+\nicefrac{\omega_W}{n})\varepsilon\right)}{\left(\nicefrac{\sigma^2}{n} +\varepsilon\right)} 
 \in  \left( \omega_M+1, (\omega_M+1)(\omega_W+1) \right].
 \end{eqnarray*}
The upper bound is achieved for $n=1$ (or for any $n$ and $\varepsilon\to 0$), and the lower bound is achieved in the limit as $n\to \infty$. So, {\em the slowdown caused by compression on worker side decreases with $n$.} More importantly, {\em the savings in communication due to compression can outweigh the iteration slowdown, which leads to an overall speedup!} See Table~\ref{tab:algo-comparison_2} for the computation of the overall worker to master speedup achieved by our compression techniques (also see Appendix~\ref{sec:dif_regimes} for additional similar comparisons under different cost/speed models). Notice that, however, standard sparsification does {\em not} necessarily improve the overall running time --- it can make it worse. Our methods have the desirable property of significantly uplifting the minimal speedup comparing to their ``non-natural'' version. The minimal speedup is more important as usually the number of nodes $n$ is not very big.


\section{Experiments}
\label{sec:Exp}

\begin{figure}[t]
\centering
\includegraphics[width=0.32\textwidth]{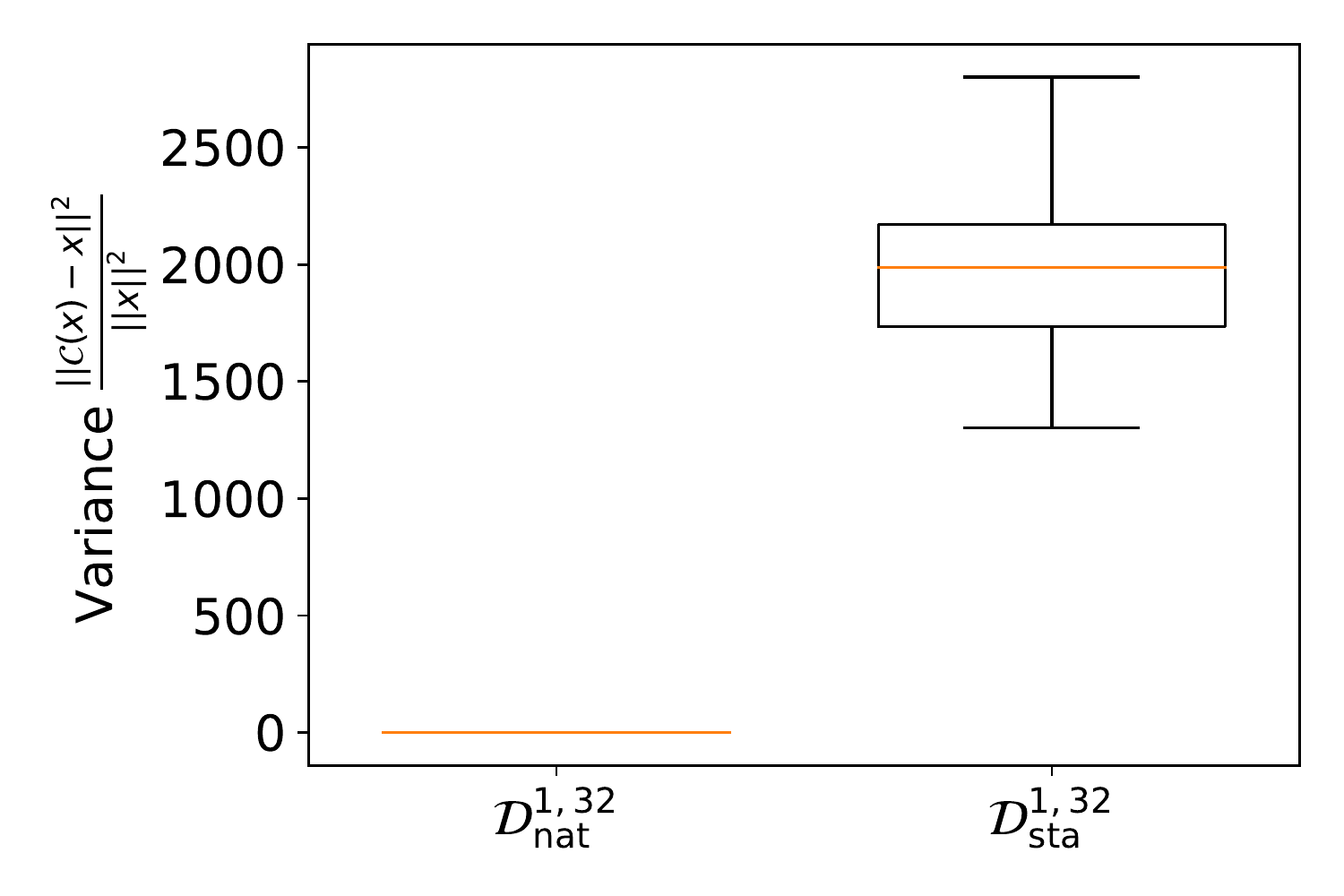}
\includegraphics[width=0.32\textwidth]{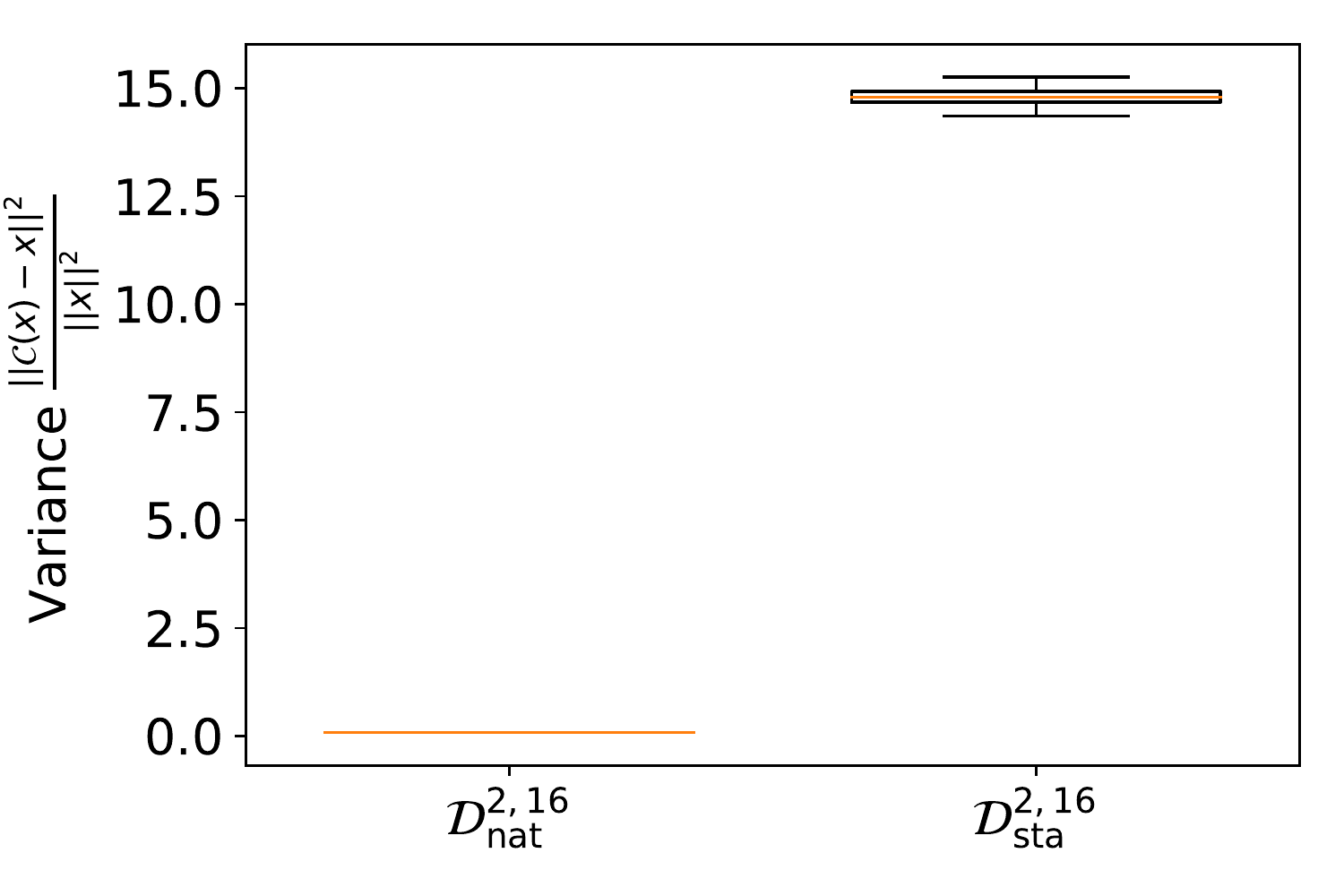}
\includegraphics[width=0.32\textwidth]{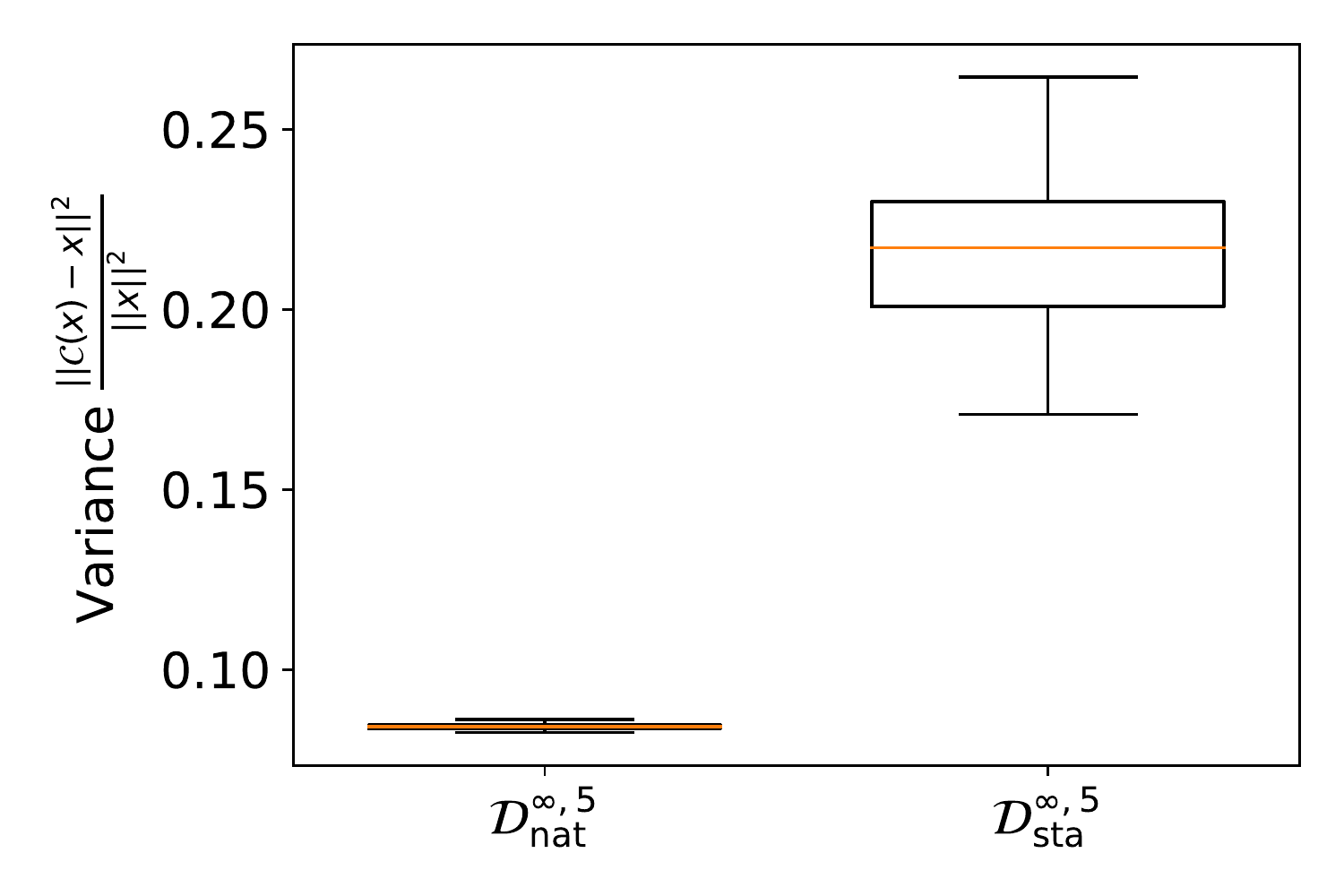}
\caption{$\ND$ vs. $\SDs{u}$ with $u = s$. }
\label{fig:bin_comparison_var}
\end{figure}
\begin{figure}[t]
\centering
\includegraphics[width=0.32\textwidth]{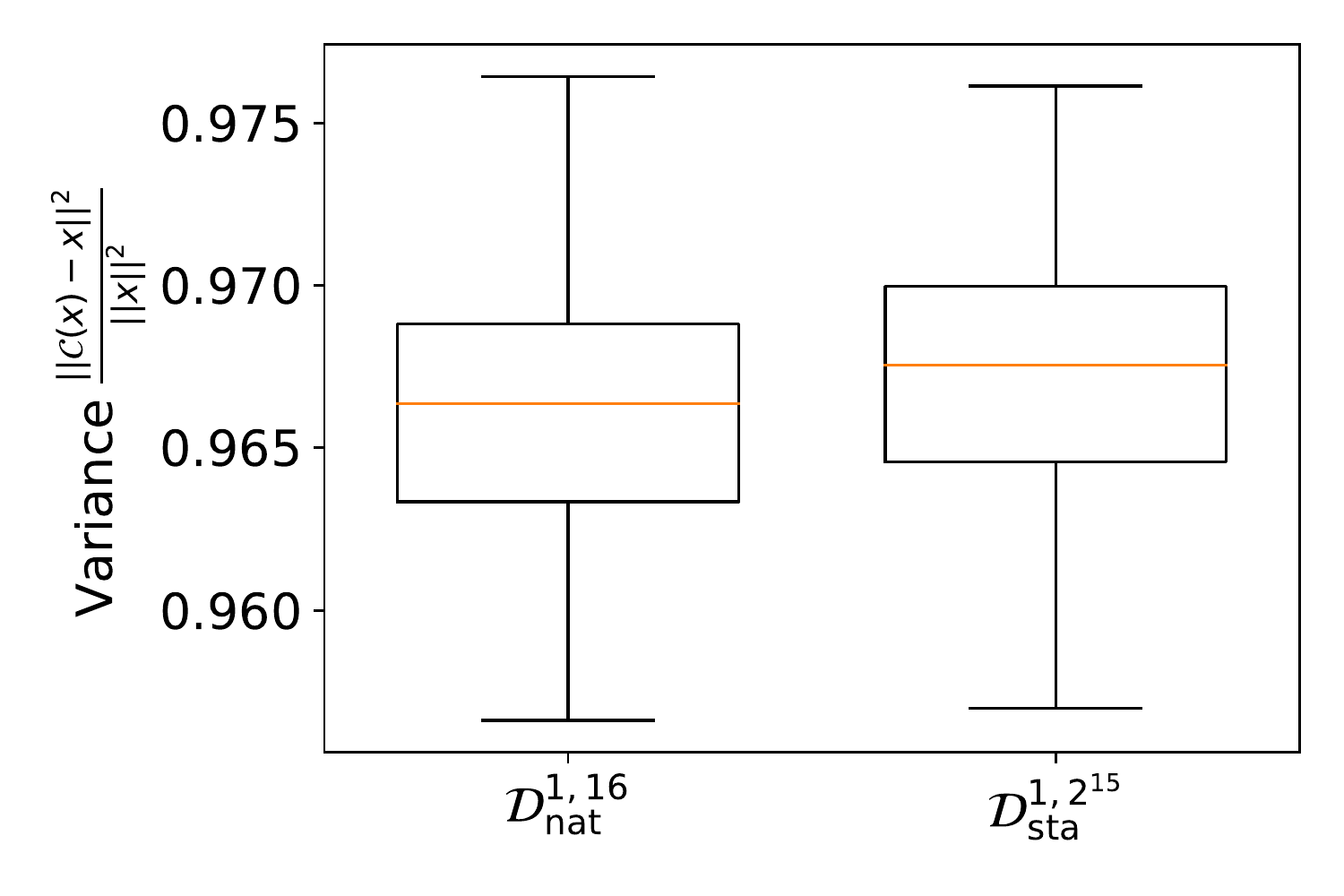}
\includegraphics[width=0.32\textwidth]{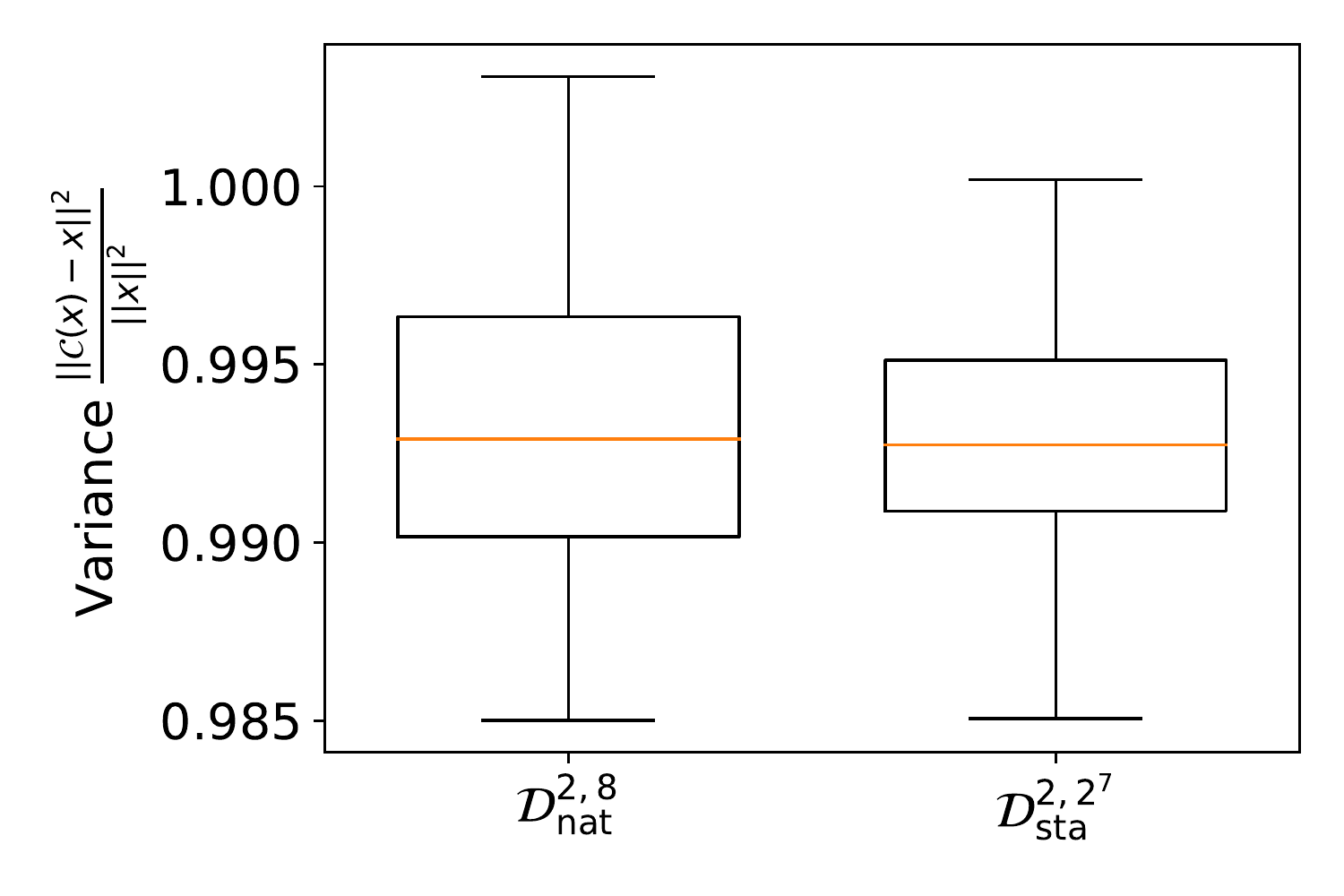}
\includegraphics[width=0.32\textwidth]{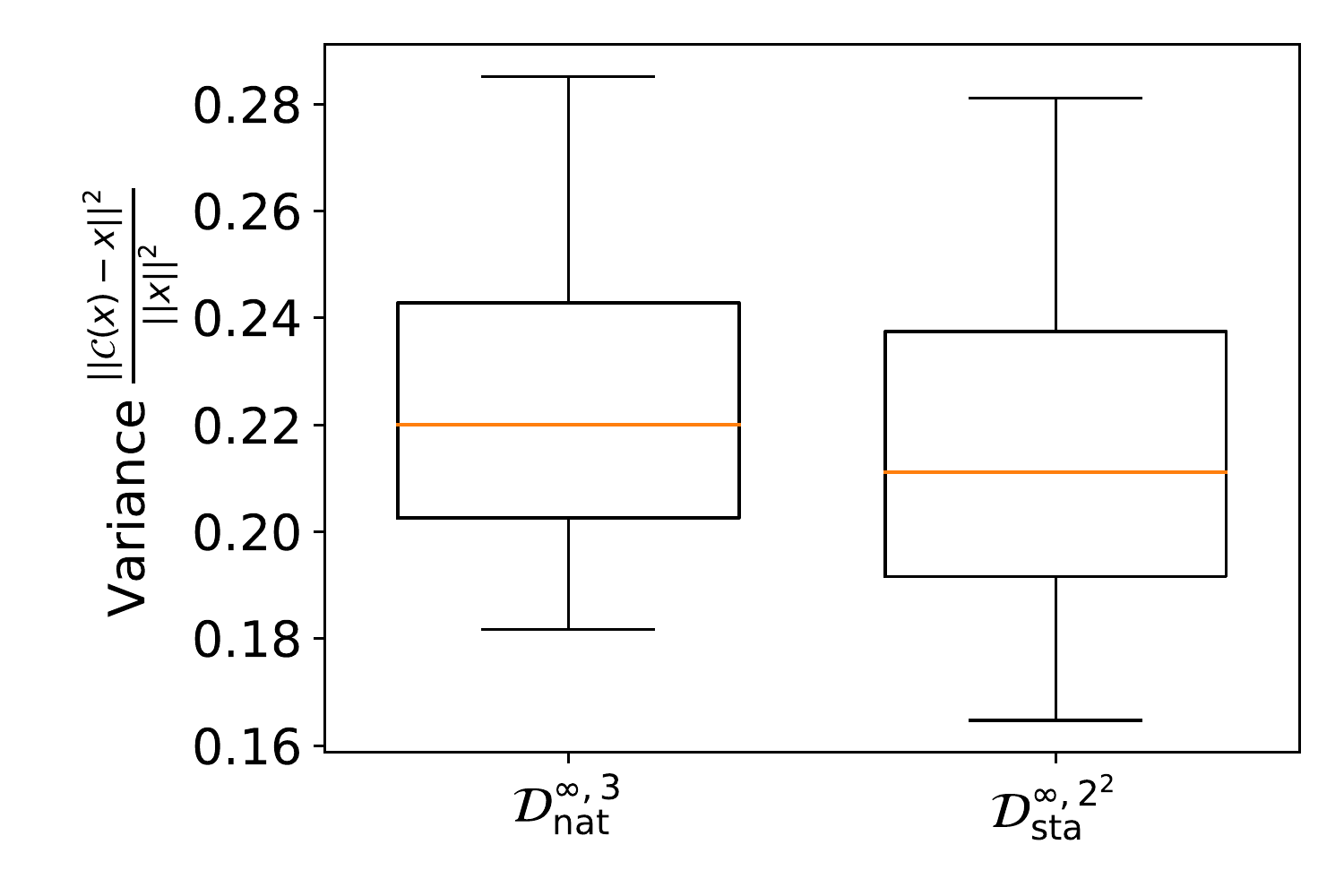}
\caption{$\ND$ vs. $\SDs{u}$ with $u = 2^{s-1}$.}
\label{fig:bin_comparison_var_xx}
\end{figure}
\begin{figure}[t]
\centering
\includegraphics[width=0.32\textwidth]{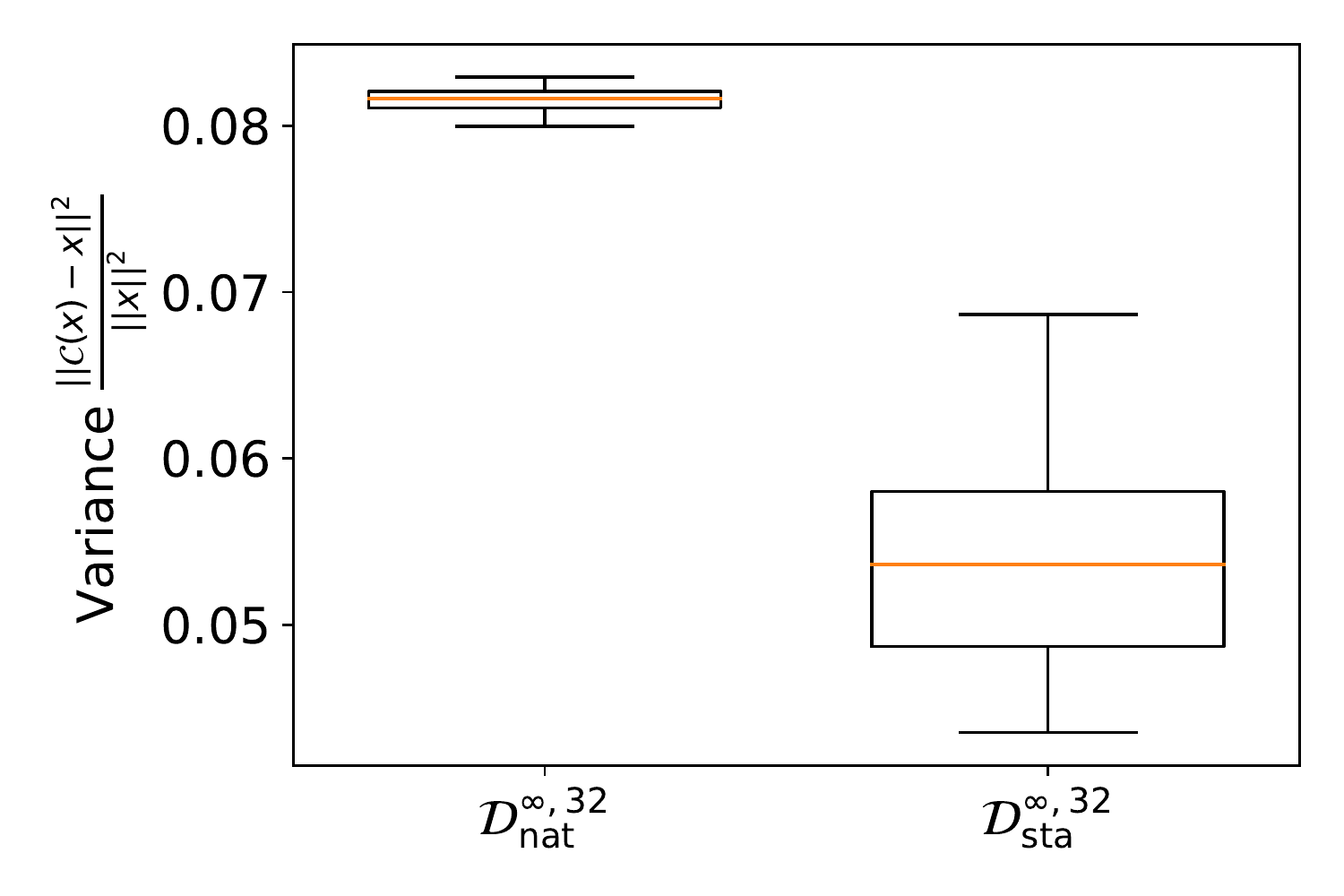}
\caption{When $p=\infty$ and $s$ is very large, the empirical variance of  $\SD$ can be  smaller than that of $\ND$. However, in this case, the variance of $\ND$ is already negligible.}
\label{fig:row_c}
\end{figure}

\subsection{$\ND$ vs. $\SDs{u}$: Empirical variance}
\label{sec:emp_variance}

Firstly, we perform experiments to confirm that  $\ND$ level selection brings not just theoretical but also practical performance speedup in comparison to $\SDs{u}$. We measure the empirical variance of $\SDs{u}$ and $\ND$. For $\ND$, we do not compress the norm to compare just variance introduced by level selection. Our experimental setup is the following. We first generate a random vector $x$ of size $d=10^5$, with independent entries with Gaussian distribution of zero mean and unit variance  (we tried other distributions, the results were similar, thus we report just this one) and then we measure normalized {\em empirical variance} \[\omega(x)\eqdef \frac{\norm{\cC(x)-x}^2}{\norm{x}^2}.\] We provide boxplots for 100 randomly generated vectors $x$ using the above procedure. We perform this for $p=1$, $p=2$ and $p=\infty$. We report our findings in Figure~\ref{fig:bin_comparison_var}, Figure~\ref{fig:bin_comparison_var_xx} and Figure~\ref{fig:row_c}. These experimental results support our theoretical findings.

{\bf $\ND$ has exponentially better variance.} 
In Figure~\ref{fig:bin_comparison_var}, we compare $\ND$ and $\SDs{u}$ for $u=s$, i.e., we use the same number of levels for both compression strategies. In each of the three plots, we generated vectors $x$ with a different norm. We find that natural dithering has a dramatically smaller variance, as predicted by Theorem~\ref{thm:expon_better}.

{\bf $\ND$ needs exponentially less levels to achieve the same variance.} 
In Figure~\ref{fig:bin_comparison_var_xx}, we set the number of levels for $\SDs{u}$ to  $u=2^{s-1}$. We give standard dithering an exponential advantage in terms of the number of levels (which also means that it will need more bits for communication). We now study the effect of this change on the variance. We observe that the empirical variance is essentially the same for both, as predicted by Theorem~\ref{thm:expon_better}.

{\bf $\SD$ can outperform $\ND$ in the big $s$ regime.} 
We now remark on the situation when the number of levels $s$ is chosen to be very large (see Figure~\ref{fig:row_c}). While this is not a practical setting as it does not provide sufficient compression, it will illustrate a fundamental theoretical difference between $\SD$ and $\ND$ in the $s\to \infty$ limit, which we want to highlight. Note that while $\SD$ converges to the identity operator as $s\to \infty$, which enjoys zero variance, $\ND$ converges to $\NC$ instead, with the variance that can't reduce below $\omega=\nicefrac{1}{8}$. Hence, for large enough $s$, one would expect, based on our theory,  the variance of $\ND$ to be around $\nicefrac{1}{8}$, while the variance of $\SD$ to be closer to zero. In particular, this means that $\SD$ {\em can}, in a practically meaningless regime, outperform $\ND$. In Figure~\ref{fig:row_c} we choose $p=\infty$ and  $s=32$ (this is large). Note that, as expected, the empirical variance of both compression techniques is small and that, indeed, $\SD$ outperforms $\ND$.

{\bf Compressing gradients.} 
 We performed identically to those reported above but with a different generation technique of the vectors $x$. In particular, instead of a synthetic Gaussian generation, we used gradients generated by our optimization procedure as applied to the problem of training several deep learning models. Our results were essentially the same as those reported above, so we do not include them.

 \subsection{Neural networks}

\begin{figure}[t]
  \centering
\includegraphics[width=0.24\textwidth]{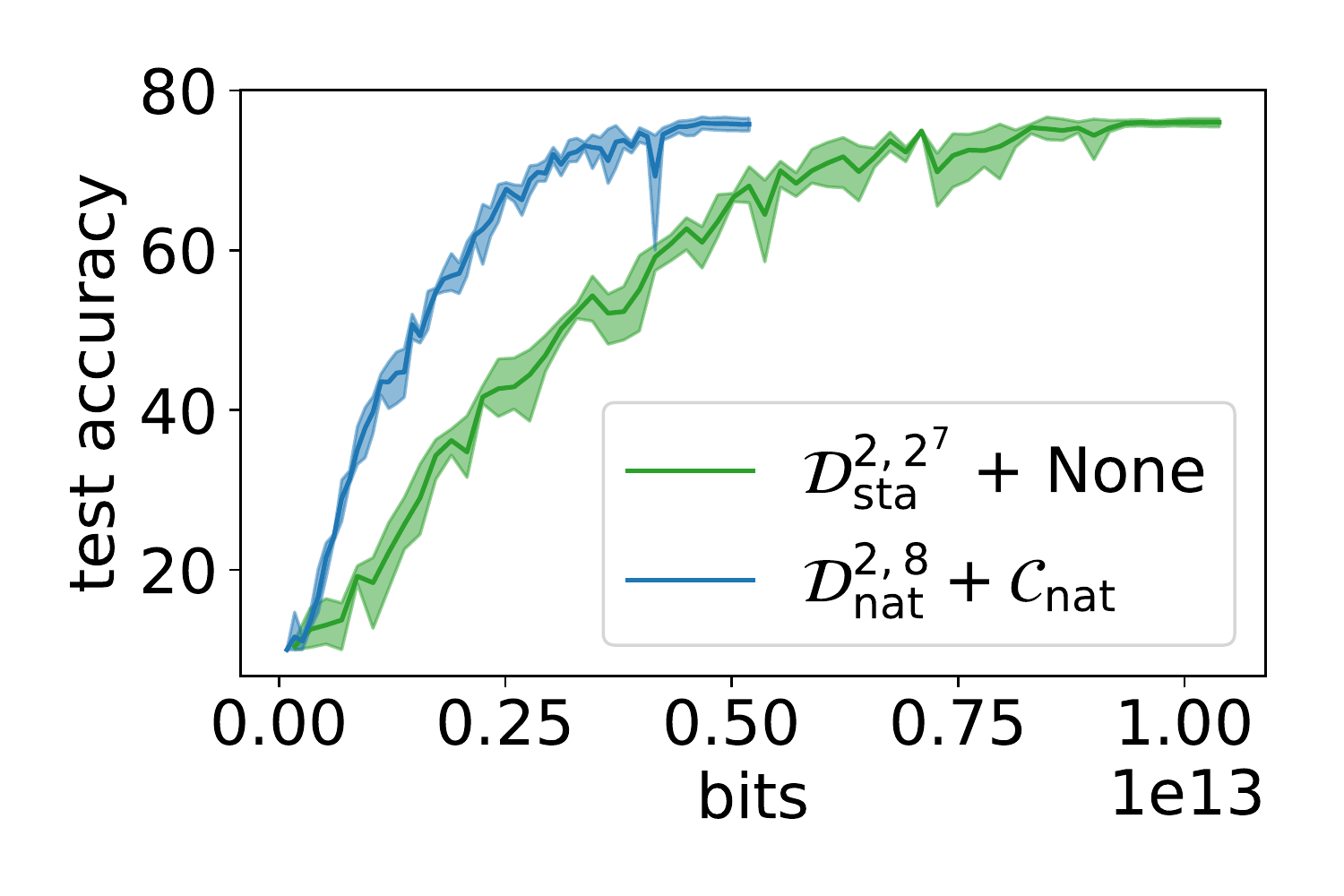}
\includegraphics[width=0.24\textwidth]{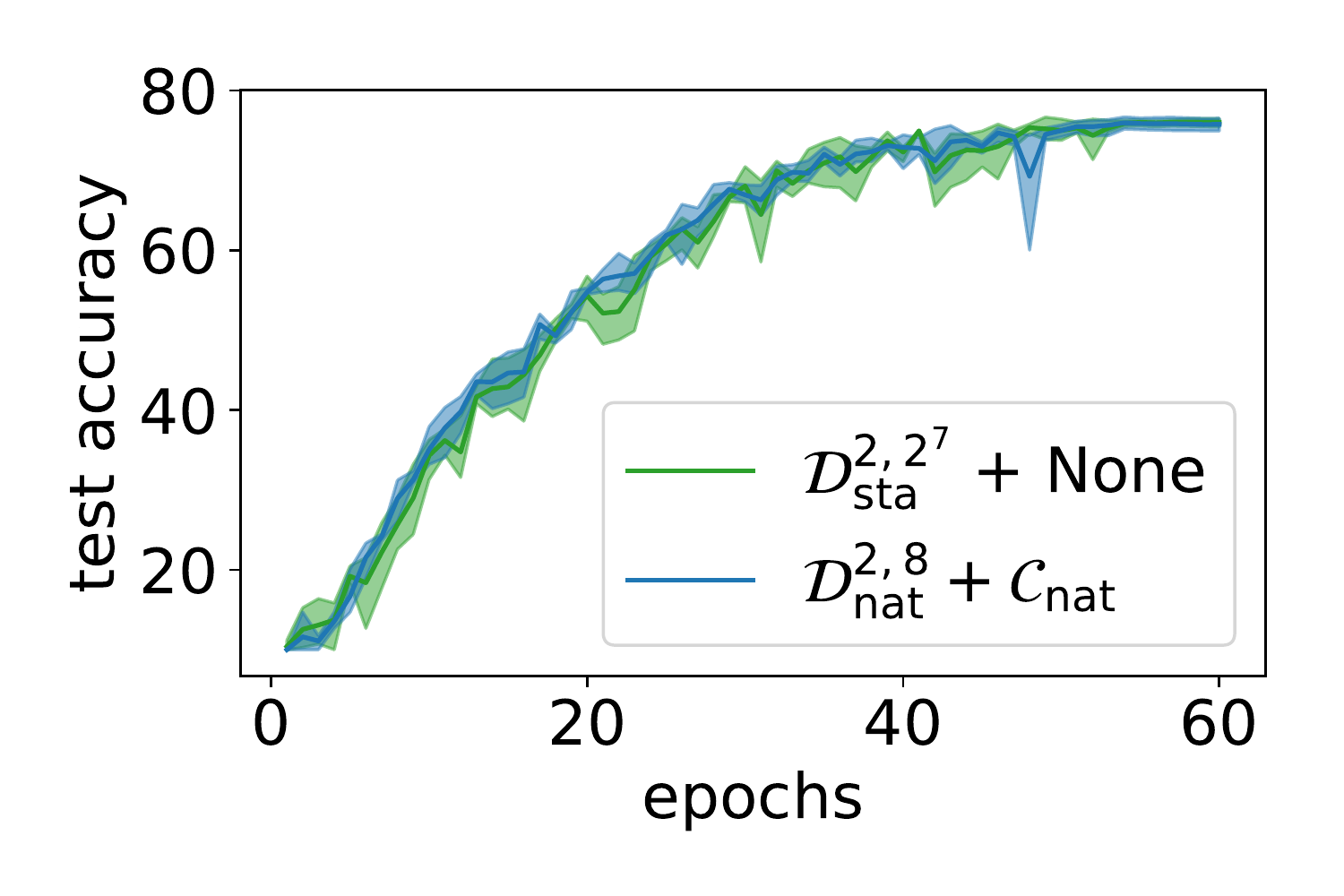}
\includegraphics[width=0.24\textwidth]{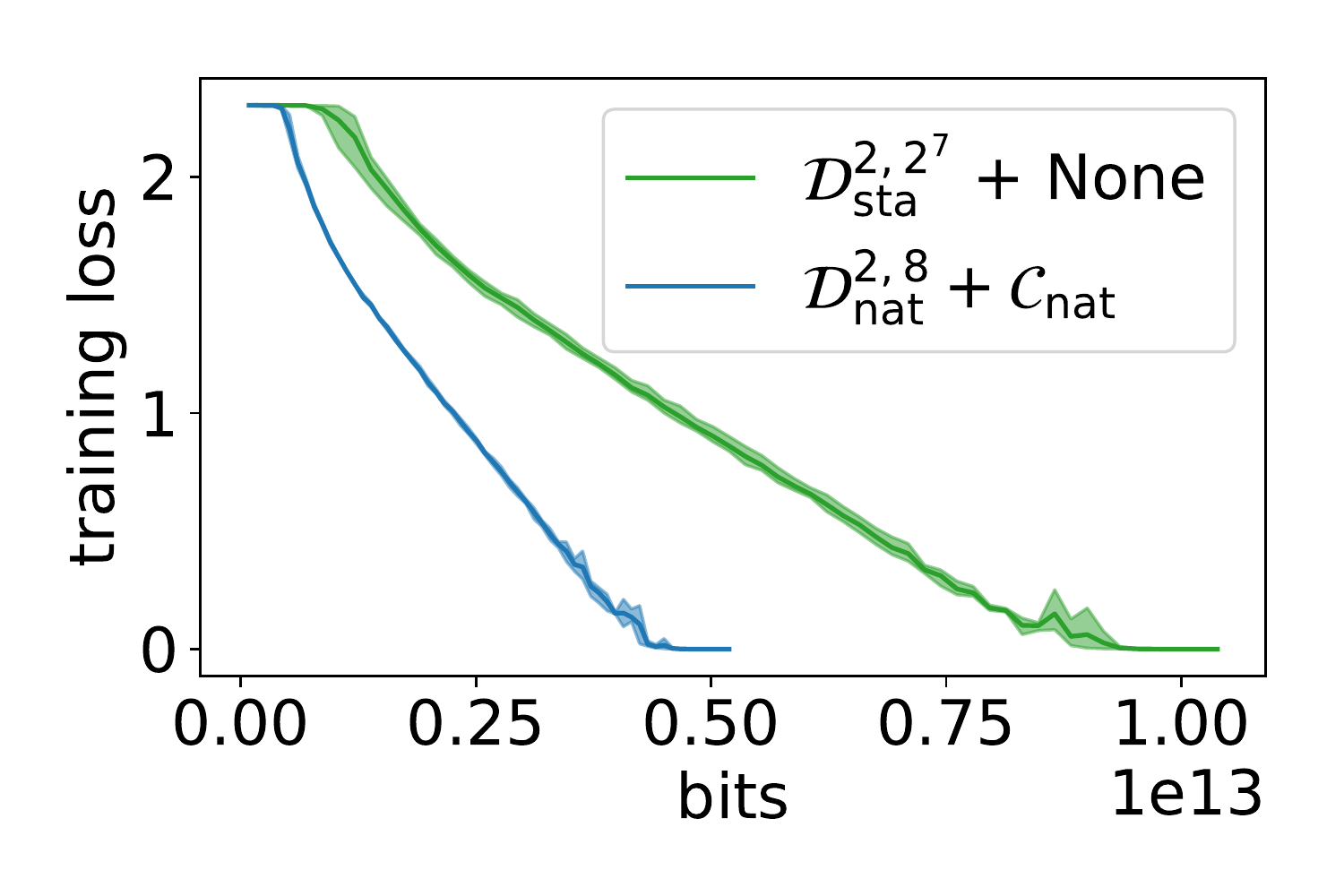}
\includegraphics[width=0.24\textwidth]{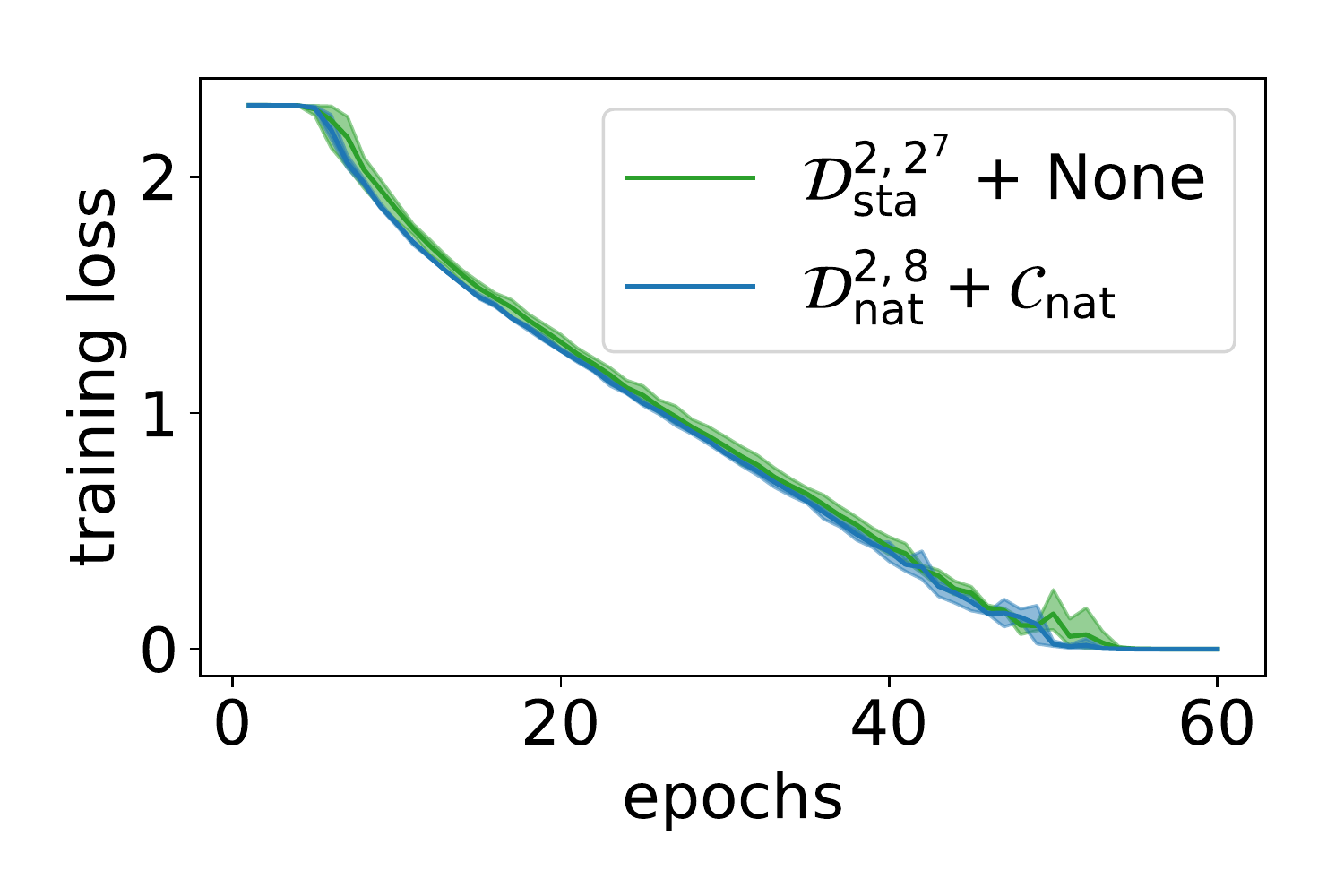}  
\caption{Train loss and test Aacuracy of VGG11 on CIFAR10.  Green line: $\cC_{W_i} =\cD_{\rm sta}^{2,2^7}$, $\cC_M = $ identity. Blue line: $\cC_{W_i} =\cD_{\rm nat}^{2,8}$, $\cC_M = \NC$. }
\label{fig:dith_cifar}
\end{figure}

\begin{figure}[t]
\centering
\includegraphics[width=0.24\textwidth]{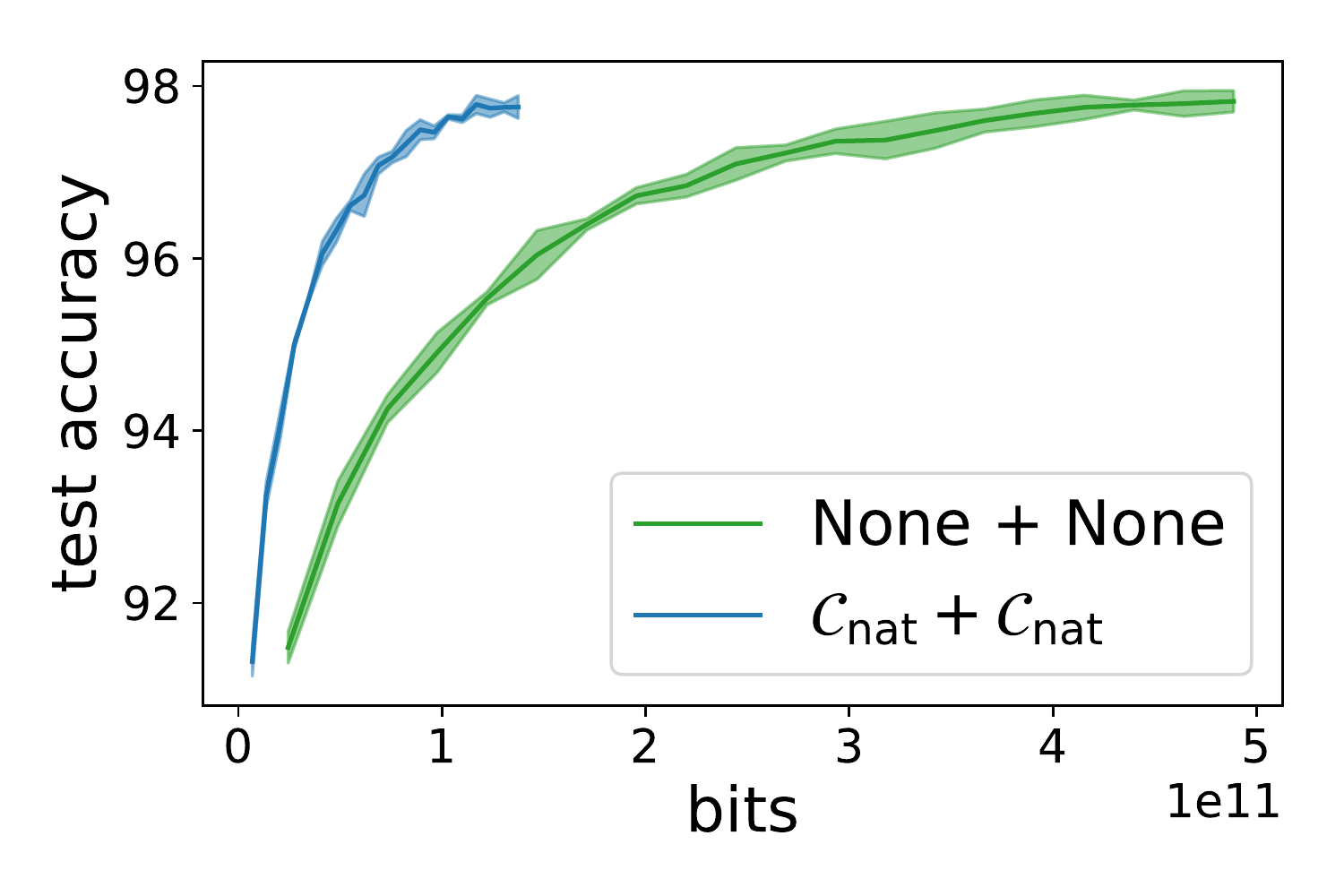}
\includegraphics[width=0.24\textwidth]{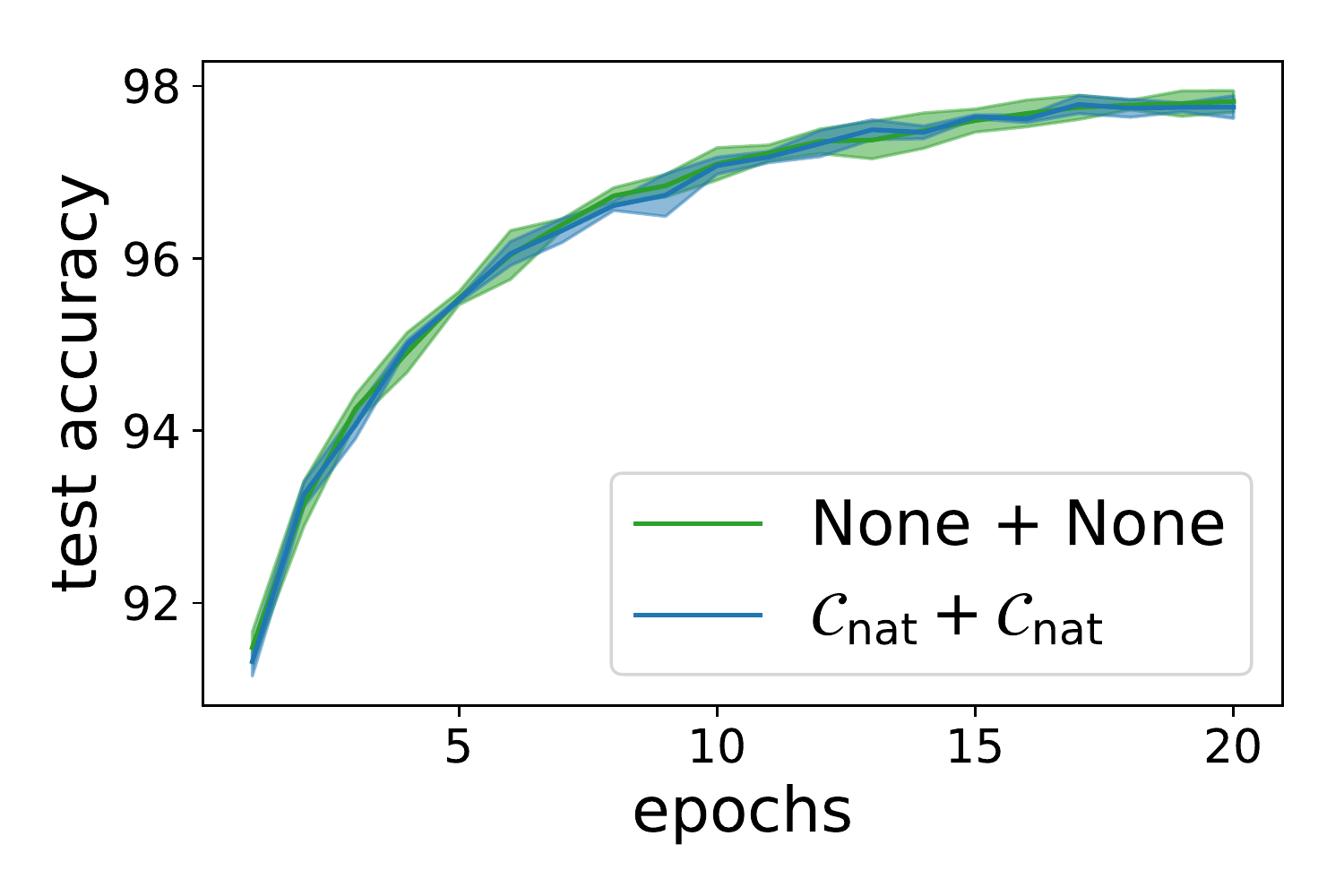}
\includegraphics[width=0.24\textwidth]{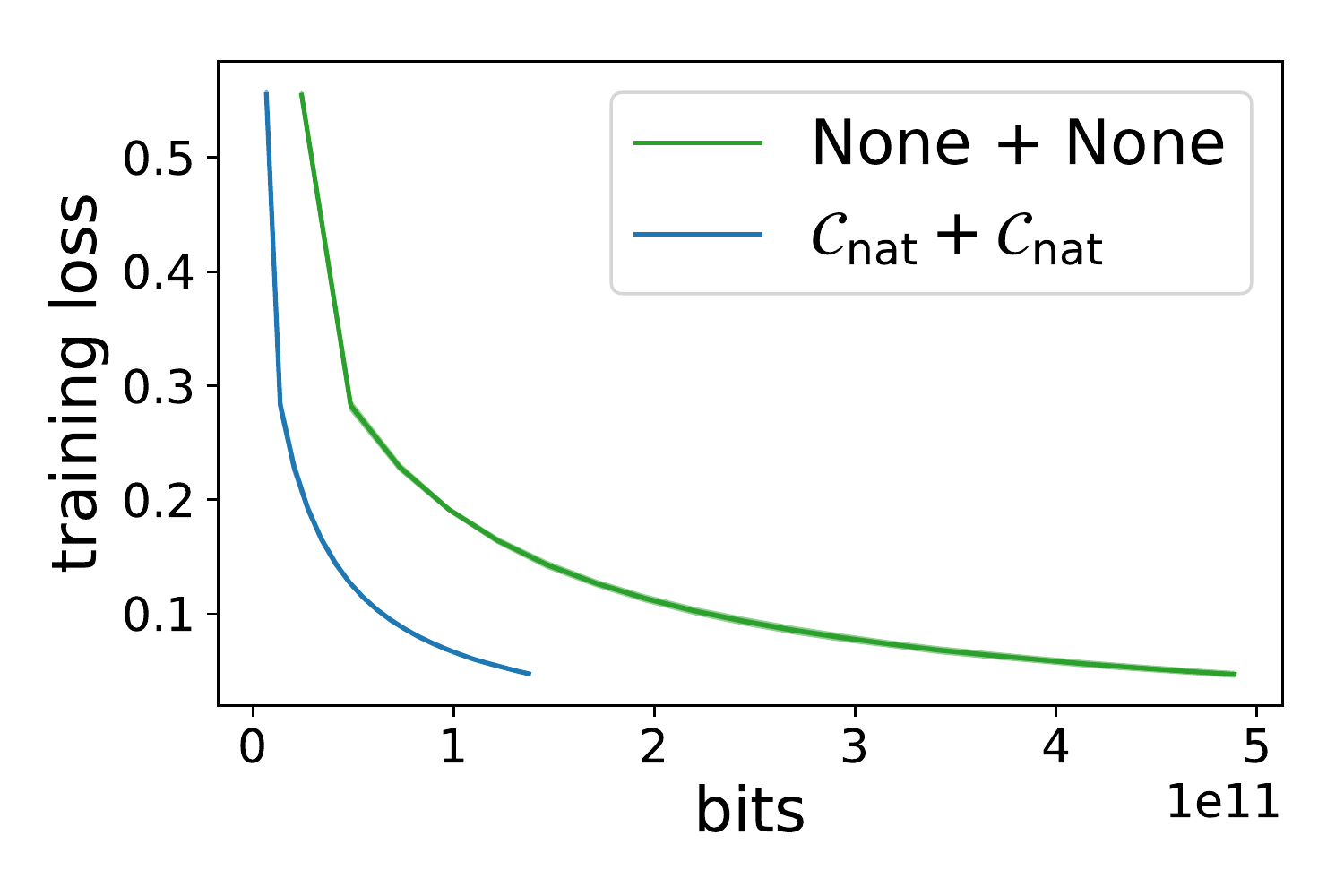}
\includegraphics[width=0.24\textwidth]{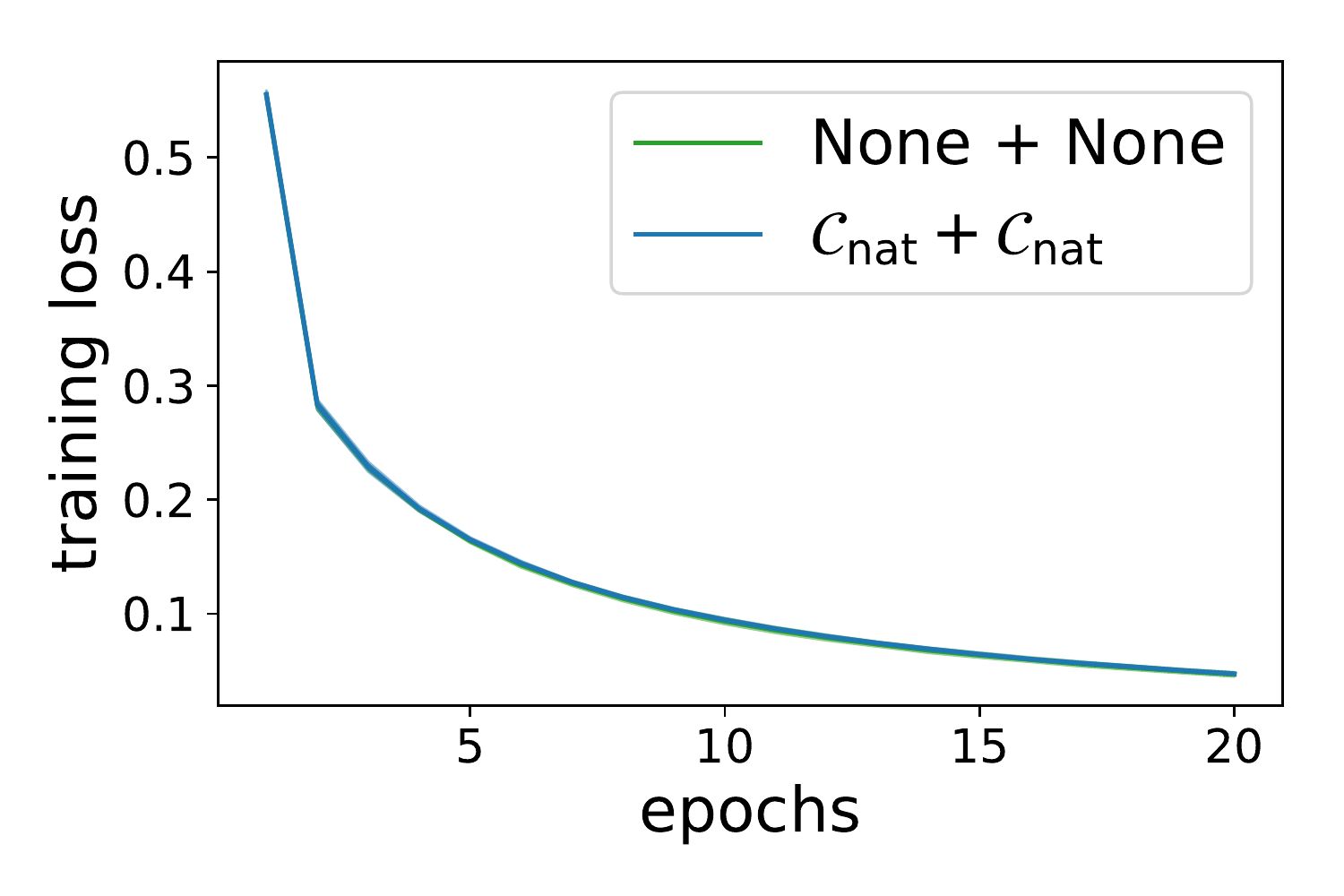}
\caption{Train loss and test Aacuracy of neural network with $2$ fully conected layers  on MNIST.  Green line: $\cC_{W_i} = \cC_M = $ identity. Blue line: $\cC_{W_i} = \cC_M = \NC$. }
\label{fig:c_nat_mnist}
\end{figure}
 
Next, we compare our compression operator to their ``non-natural'' counterexamples. These results are based on a Python implementation of our methods running in PyTorch~\cite{pytorch} as this enabled a rapid direct comparison against the prior methods. We compare against no compression, random sparsification, and random dithering methods. We compare on MNIST~\cite{lecun1998gradient} and CIFAR10~\cite{cifar} datasets. For MNIST, we use a two-layer fully connected neural network with RELU activation function. For CIFAR10, we use VGG11~\cite{simonyan2014very} with one fully connected layer as the classifier. We run these experiments with $4$ workers and batch size $32$ for MNIST and $64$ for CIFAR10. The results are averaged over three independent runs with one standard deviation.

We tune the step size for \texttt{SGD} for a given ``non-natural'' compression. Then we use the same step size for the ``natural'' method. 

Figures~\ref{fig:dith_cifar} and \ref{fig:c_nat_mnist} (complete plots in Figures~\ref{fig:bin_comparison_cifar} and \ref{fig:bin_comparison_mnist} in Appendix~\ref{sec:extra_exp}) illustrate the results. Each row contains four plots that illustrate, left to right, (1) the test accuracy vs the volume of data transmitted from workers to master (measured in bits), (2) the test accuracy over training epochs, (3) the training loss vs the volume of data transmitted, and (4) the training loss over training epochs.

One can see that in terms of epochs, we obtain almost the same result in terms of training loss and test accuracy, sometimes even better. On the other hand, our approach significantly impacts the number of bits transmitted from workers to master, which is the main speedup factor and the speedup in aggregation if we use In-Network Aggregation (INA). Moreover, with INA, we compress updates from master to nodes; hence, we also send fewer bits. 



\chapter{Stochastic distributed learning with gradient quantization and double variance reduction}
\label{chapter3:diana_2}

\section{Introduction}
\label{sec:intro_diana}
The training of large scale machine learning models poses many challenges that stem from the mere size of the available training data. The \emph{data-parallel} paradigm focuses on distributing the data across different compute nodes, which operate in parallel on the data. 
For example in \emph{federated learning} the data is generated and stored directly on the user devices and aggregation of all user data on a single server is impossible~\cite{kairouz2019advances}. In order to still be able to train machine learning models only model updates (but not the data) are shared among the devices. 
The training of a machine learning model then crucially relies on frequent communication between the devices to aggregate these updates and the training time strongly correlates with the time for communication (e.g.\ size, frequency and total number of messages).

\subsection{Optimization problem}
 Formally, we consider optimization problems distributed across $n$ nodes of the form
\begin{align}
\min_{x \in \R^d}  \underbrace{
			\frac{1}{n} \sum_{i=1}^n f_i(x)
		}_{f(x)}  + R(x) \,, \label{eq:probR_diana}
\end{align}
where  $f_i \colon \R^d \to \R$ for $i=1,\dots,n$ are given as
$
f_i(x) \eqdef \EE{\zeta \sim \cD_i}{f_i(x,\zeta)}
$
with $f_i\colon \R^d \times \Omega \to \R$ being a general (non-convex) loss function and $R \colon \R^d \to \R \cup \{\infty\}$ being a closed convex regularizer.
Our main focus is on the problems, where each $f_i$ can be written as a finite sum (the data samples on each device in e.g.\ federated learning), i.e.,
$
f_i(x) = \sum_{j=1}^m f_{ij}(x),
$
but for completeness we also provide some results in the general loss setting.

\subsection{Quantization for communication reduction}
In typical computing architectures, communication is much slower than computation and the communication bottleneck between workers is a major limiting factor for many large scale applications (such as e.g.\ deep neural network training), as reported in~\cite{1bit,qsgd2017neurips,zhang2016zipml,deepgradcompress2018iclr}. 
Especially in federated learning the communication is heavily constrained (for instance due to metered or slow mobile connections).
A possible remedy to tackle this communication issue are approaches that focus on increasing the computation  to communication ratio, such as increased mini-batch sizes~\cite{Goyal2017:large}, defining local problems for each worker~\cite{Shamir2014:approxnewton,Reddi:2016aide} or reducing the communication frequency~\cite{Mann2009:parallelSGD, Zinkevich2010:parallelSGD, you2017imagenet, local_SGD_stich_18, kairouz2019advances}.

A direction orthogonal to these approaches tries to reduce the size of the messages---typically gradient vectors---that are exchanged between the nodes~\cite{1bit,strom2015scalable,qsgd2017neurips,terngrad,grishchenko2018asynchronous}.
These \emph{quantization} techniques rely on lossy compression of the gradient vectors. 
In the simplest form these schemes limit the number of bits that are used to represent floating point numbers~\cite{gupta2015deep,Na2017:limitedprecision}, reducing the size of a $d$-dimensional (gradient) vector by a constant factor. Random dithering approaches attain up to $\cO(\sqrt{d})$ compression~\cite{1bit,qsgd2017neurips,terngrad}. The most aggressive schemes reach $\cO(d)$ compression by only sending a constant number of bits per iteration~\cite{Suresh2017,RDME,alistarh2018sparse,stich2018sparsified}. An alternative approach is to not first compute the gradient and subsequently compress it, but to update a subset of elements of the iterate $x$ using coordinate descent type updates~\cite{Hydra, Hydra2, mishchenko201999}.
Recently, Mishchenko et al., 2019~\cite{mishchenko2019distributed} proposed the first method that successfully applies the gradient quantization technique to the distributed optimization problem~\eqref{eq:probR_diana} with a non-smooth regularizer.

\subsection{Variance reduction to reduce communication rounds}
Stochastic gradient descent converges sublinearly at rate $\cO(\frac{1}{\epsilon})$. That is, to double the target accuracy, the number of iterations---and consequently the number of communication rounds---has to be doubled as well. However, in communication restricted settings much faster linearly converging schemes that converge at rate $\cO(\log \frac{1}{\epsilon})$ instead have an advantage, as only a constant number of additional communication rounds are required to double the target accuracy.

For the distributed optimization problem~\eqref{eq:probR_diana} such linear rates can be achieved by variance reduction techniques such as  \texttt{SAG},  \texttt{SAGA} or  \texttt{SVRG}~\cite{LeRoux:2012sag, johnson2013accelerating, Defazio2014:saga}. 
Whilst these schemes are in general well-understood in the communication-unrestricted setting, it is non-trivial to achieve variance reduction in the communication-restricted setting:\ the aforementioned schemes crucially rely on \emph{control variates} to  reduce the variance of the individual (stochastic) gradient updates. For e.g.\  \texttt{SVRG}, this control variate is the \emph{full} gradient of $f$ that can only be obtained when each device communicates the (uncompressed) gradients $\nabla f_i$ to the server. 
The method fails (i.e.\ only converges at rate $\cO(\frac{1}{\epsilon})$ instead of $\cO(\log \frac{1}{\epsilon})$) when only compressed gradients $\nabla f_i$ are communicated instead.
Alistarh et al., 2017~\cite{qsgd2017neurips} proposed a quantized version of  \texttt{SVRG}~\cite{johnson2013accelerating}, but their scheme relies on broadcasting exact (unbiased) gradients every epoch to update the control variates. In this chapter we show how to achieve  the fast $\cO(\log \frac{1}{\epsilon})$ convergence rate when the control variates are only updated by (arbitrary) compressed updates. This is a substantial technical contribution and could be of independent interest.

\section{Contributions}
We now briefly outline the key contributions of this chapter:

\subsection{Double variance reduction} We present the first methods for solving problem~\eqref{eq:probR_diana} at a linear $\cO(\log \frac{1}{\epsilon})$ rate under arbitrary communication compression. Such result has been not obtained for any compression operator before. The only method with linear rate is the quantized  \texttt{SVRG} method~ \cite{qsgd2017neurips} which did not only rely on a specific compression scheme, but also on exact communication from time to time (epoch gradients), both restrictions been overcome here. On the top of that, our more general approach allows to choose freely the compression operators that perform best on the available system resources and can thus offer gains in training time over the previous method that is tied to a single class of operators. Our variance reduced schemes employ \emph{double-variance reduction}: on top of the standard variance reduction techniques by control variates, we show how the control variates can also be updated without relying on dense communication rounds.
\subsection{Convex and non-convex problems} We provide concise convergence analysis for our novel  schemes for the strongly-convex, the (weakly) convex and non-convex setting. Our analysis recovers the respective rates of earlier schemes without communication compression and shows that compression can provide huge benefits when communication is a bottleneck.
\subsection{Experiments} We compare the novel variance reduced schemes with various baselines. In the experiments we leverage the flexibility of our approach---allowing to freely chose a quantization scheme with optimal parameters---and demonstrate on par performance with the baselines in terms of iterations, but considerable savings in terms of total communication cost.
\subsection{General \texttt{DIANA}} Lastly, we also generalize the method presented in \cite{mishchenko2019distributed} allowing for {\em arbitrary compression} (e.g., quantization and sparsification) operators (Appendix~\ref{sec:generalDIANA}). Our analysis is tight, i.e., we recover their convergence rates as a special case.

\section{Related work}
Gradient compression schemes have successfully been used in many implementations as a heuristic to reduce communication cost, such as for instance in \texttt{1BitSGD}~\cite{1bit,strom2015scalable} that rounds the gradient components to either -1 or 1, and error feedback schemes aim at reducing quantization errors among iterations, for instance as in~\cite{deepgradcompress2018iclr}.
In the discussion below we focus in particular on schemes that enjoy theoretical convergence guarantees. 

\subsection{Quantization and sparsification}
\label{sec:intro_dianaducequant}
A class of very common quantization operators is based on random dithering~\cite{goodall1951television,roberts1962picture} and can be described as the random operators $\cC \colon \R^d \to \R^d$,
\begin{align}
\cC(x) =  \sign(x) \cdot \norm{x}_p \cdot \tfrac{1}{s}\cdot \left\lfloor s \tfrac{ \abs{x}}{\norm{x}_p} + \xi \right\rfloor \label{eq:Q}
\end{align}
for random variable $\xi \sim_{\rm u.a.r.} [0,1]^{d}$, parameter $p \geq 1$, and $s \in \N_+$, denoting the \emph{levels} of the rounding. Its unbiasedness property, $\EE{\xi}{\cC(x)}=x$, $\forall x \in \R^d$, is the main catalyst of the theoretical analyses. 
The quantization~\eqref{eq:Q} was for instance used in  \texttt{QSGD} for $p=2$~\cite{qsgd2017neurips}, in  \texttt{TernGrad} for $s=1$ and $p=\infty$~\cite{terngrad} or for general $p \geq 1$ and $s=1$  in \texttt{DIANA}~\cite{mishchenko2019distributed}. 
For $p=2$ the expected sparsity is $\EE{\xi}{\norm{\cC(x)}_0} = \cO(s(s+\sqrt{d}))$~\cite{qsgd2017neurips} 
and encoding a nonzero coordinate of $\cC(x)$ requires $\cO(\log(s))$ bits. 

Much sparser vectors can be obtained by random sparsification techniques that randomly mask the input vectors and only preserve a constant number of coordinates~\cite{Suresh2017, RDME,tonko,stich2018sparsified}. 
Experimentally it has been shown that deterministic masking---for instance selecting the largest components in absolute value~\cite{Dryden2016:topk,aji2017sparse}---can outperform the random techniques. 
However, as these schemes are biased, they resisted careful analysis until very recently~\cite{alistarh2018sparse,stich2018sparsified}. 
We will not further distinguish between sparsification and quantization approaches, and refer to both of these compression schemes them as \emph{quantization} in the following.

Schemes with error compensation techniques have recently been successfully vanquished, for instance for unbiased quantization on quadratic functions~\cite{errorSGD} and for unbiased and biased quantization on strongly convex functions~\cite{stich2018sparsified}. 
These schemes suffer much less from large quantization errors, and can tolerate higher variance (both in practice an theory) than the methods without error compensation.

\subsection{Quantization in distributed learning}
In centralized approaches, all nodes communicate with a central node (parameter sever) that coordinates the optimization process. An algorithm of specific interest 
is mini-batch \texttt{SGD}~\cite{Dekel2012:minibatch, pegasos2}, a parallel version of stochastic gradient descent (\texttt{SGD})~\cite{SGD,Nemirovski-Juditsky-Lan-Shapiro-2009}. In mini-batch \texttt{SGD}, full gradient vectors have to be communicated to the central node and hence it is natural to incorporate gradient compression to reduce the cost of the communication rounds.
in \cite{khirirat2018distributed}, the authors study quantization in the deterministic setting, i.e.\ when the gradients $\nabla f_i(x)$ can be computed without noise, and show that parallel gradient descent with unbiased quantization converges to a neighborhood of the optimal solution on strongly convex functions. 
\cite{mishchenko2019distributed} consider the stochastic setting as in~\eqref{eq:probR_diana} and show convergence of \texttt{SGD} with unbiased quantization to a neighborhood of a stationary point for non-convex and strongly convex functions; the analysis in~\cite{errorSGD} only applies to quadratic functions. The method of \cite{stich2018sparsified} can also be parallelized as shown in \cite{cordonnier2018convex} and converges 
on strongly convex functions.

\subsection{Quantization and variance reduction}
Variance reduced methods~\cite{LeRoux:2012sag, johnson2013accelerating, ShalevShwartz:2013sdca, Defazio2014:saga, qu2015quartz, SDCA-dual-free, SDNA,  dfSDCA, ADASDCA, nguyen2017sarah, JacSketch, zhou2018direct} admit linear convergence on empirical risk minimization problems, surpassing the rate of vanilla \texttt{SGD}~\cite{Bottou2010:sgd}. A quantized version of  \texttt{SVRG}~\cite{johnson2013accelerating} was proposed in~\cite{qsgd2017neurips}, but the proposed scheme relies on broadcasting exact (unbiased) gradients every epoch.
Kuenstner et al., 2017~\cite{kuenstner2017:svrg} studies an alternative scheme that crucially relies on high-precision quantization and only mildly improves the communication complexity over \cite{qsgd2017neurips}. Here we alleviate all these restrictions and present quantization not only for  \texttt{SVRG} but also for  \texttt{SAGA}~\cite{Defazio2014:saga}. Our analysis also supports a version of  \texttt{SVRG} whose epoch lengths are not fixed, but random, similar as e.g.\ in~\cite{Lei2017:singlepass,Hannah2018:span}. Our base version (denoted as  \texttt{L-SVRG}) is slightly different and inspired by observations made in \cite{Hofmann2015:saga,Raj2018:ksvrg} and following closely~\cite{Kovalev2019:svrg}.

\subsection{Orthogonal approaches}
There are other approaches aiming to reduce the communication cost, such as
 increased mini-batch sizes~\cite{Goyal2017:large}, defining local problems for each worker \cite{Shamir2014:approxnewton,cocoa, COCOA+, Reddi:2016aide, COCOA+journal}, reducing the communication frequency~\cite{Mann2009:parallelSGD,Zinkevich2010:parallelSGD,%
 you2017imagenet,local_SGD_stich_18,Li2019:decentralized} or distributing along features~\cite{Hydra, Hydra2}. However, we are not considering such approaches here.

Frequently we will assume $L$-smoothness or $\mu$-strong convexity, and will work with the proximal operator.
Our analysis depends on a general notion of unbiased quantization operators with bounded variance as in Definition~\ref{def:omegaquant}.



\begin{table}[t]
\begin{center}
\caption{Comparison of variance reduced methods with quantization.  $\hat{\cO}(\cdot)$ omits $\log \frac{1}{\varepsilon}$ factors and we assume $\omega \leq m$ ($\omega \leq m^{2/3}$ for non-convex case) for the ease of presentation. For block quantization (last row), the convergence rate is identical to the method without quantization (first row) but the savings in communication is at least $d/n$, assuming $d\geq n$  (in number of total coordinates, here we did for simplicity not consider further savings that the random dithering approach offers in terms of less bits per coordinate). This table shows that quantization is meaningful and can provide huge benefits, especially when communication is a bottleneck.}
\label{tab:algo-comparison_diana_diana}
\scriptsize
\resizebox{\textwidth}{!}{
\begin{tabular}{|c||c|c|c|c|}
\hline
\bf \multirow{2}{*}{Algorithm} & \bf \multirow{2}{*}{$\omega$} & \bf Convergence rate & \bf Convergence rate & \bf Communication \\
&  & \bf strongly convex & \bf non-convex & \bf cost per iter. \\
\hline 
\hline
VR without   & \multirow{2}{*}{$1$} & \multirow{2}{*}{$\hat{\cO}\left(
	\kappa + m
	\right)$}  &  \multirow{2}{*}{$\cO\left(
	\frac{ m^{2/3}}{\varepsilon}
	\right)$}   &  \multirow{2}{*}{$\cO(dn)$}\\
quantization &   & & &\\
\hline
VR with random & \multirow{2}{*}{$\sqrt{d}$} & \multirow{2}{*}{$\hat{\cO}\left(
	\kappa + \kappa \frac{\sqrt{d}}{n} + m + \sqrt{d}
	\right)$}  &  \multirow{2}{*}{$\cO\left(
	\left(\frac{\sqrt{d}}{n}\right)^{1/2}\frac{m^{2/3}}{\epsilon} 
	\right)$}  &  \multirow{2}{*}{$\cO(n\sqrt{d})$}\\
 dithering ($p=2, s =1$) & & & &\\
\hline
VR with random sparsification & \multirow{2}{*}{$\frac{d}{r}$} &  \multirow{2}{*}{$\hat{\cO}\left(
	\kappa + \kappa \frac{d}{n} + m + d
	\right)$}  &  \multirow{2}{*}{$\cO\left(
	 \sqrt{\frac{d}{n}}\frac{m^{2/3}}{\epsilon}
	\right)$}  &  \multirow{2}{*}{$\cO(n)$} \\
  ($r= \text{const}$, \# of non-zeros) & & & & \\
 \hline
VR with  block quantization  &  \multirow{2}{*}{$n$} & \multirow{2}{*}{$\hat{\cO}\left(
	\kappa  + m + n
	\right)$}  &  \multirow{2}{*}{$\cO\left(
	\frac{m^{2/3}}{\epsilon}
	\right)$}  &  \multirow{2}{*}{$\cO(n^2)$}\\
  ($t = d/n^2$,  block size) & & & &\\
\hline
\end{tabular}
}
\end{center} 
\end{table}

\section{Variance reduction for quantized updates}

We are now ready to present the main contribution of the chapter. One of the main disadvantages of distributed stochastic gradient methods, such as e.g.\ the \texttt{DIANA} algorithm~\cite{mishchenko2019distributed}, is the fact that these algorithms only converge sublinearly, i.e. they require $\cO(\frac{1}{\epsilon})$ iterations to reach an arbitrary small $\epsilon$-neighborhood of the optimal solution.
One can only guarantee (e.g.\ for \texttt{DIANA}) linear convergence to a $ \frac{2}{\mu(\mu+L)}  \frac{\sigma^2}{n}$-neighborhood of the optimal solution. The same holds even for our more general algorithm that straightforwardly extends \texttt{DIANA}, see Appendix~\ref{sec:generalDIANA}. In particular, the size of the neighborhood depends on the size of $\sigma^2$, which  measures the average variance of the stochastic gradients $g_i^k$ across the workers $i\in [n]$.

 In contrast, variance reduced methods converge linearly for arbitrary accuracy $\epsilon > 0$, i.e.\ they require only $\cO(\log \frac{1}{\epsilon})$ iterations and consequently much fewer communication rounds (for small $\epsilon$).
We present here the first variance reduced methods that converge under arbitrary communcation compression.
To the best of our knowledge the only method with compression that converges for arbitrary accuracy is  \texttt{QSVRG}~\cite{qsgd2017neurips}, but not only this method depends on the specific compression scheme, it also requires to communicate full non-compressed updates in each epoch.

As we mentioned earlier, we focus in this section on the case where each component $f_i$ of $f$ in~\eqref{eq:probR_diana} has finite-sum structure, and assume that the number of components $m$ is the same (this might be extended to $m_i$ components on each node with  $m =\max_{i \in [n]} m_i$)
for all functions $f_i$, i.e.
$
f_i(x) = \tfrac{1}{m}\sum\limits_{j=1}^{m} f_{ij}(x)\,.
$

\begin{algorithm}[t]
\begin{algorithmic}[1]
		\STATE {\bfseries Input:} learning rates $\alpha > 0$ and $\gamma > 0$, initial vectors $x^0, w^0=x^0, h_{1}^0, \dots, h_{n}^0$, $h^0 = \frac{1}{n}\sum_{i=1}^n h_i^0$
		\FOR{$k = 0,1,\ldots$}
			\STATE sample random 
			$
				u^k = \begin{cases}
					1,& \text{with probability } \frac{1}{m}\\
					0,& \text{with probability } 1 - \frac{1}{m}\\
				\end{cases}
			$\;  
			\STATE broadcast $x^k$, $u^k$ to all workers\;
			\STATE $\triangleright$ \textit{Worker side}
			\FOR{$i = 1, \ldots, n$}
				\STATE pick random $j_i^k \sim_{\rm u.a.r.} [m]$\;
				\STATE $\mu_i^k = \frac{1}{m} \sum_{j=1}^{m} \nabla f_{ij}(w_{ij}^k)$\label{ln:mu} \;
				\STATE $g_i^k = \nabla f_{ij_i^k}(x^k) - \nabla f_{ij_i^k}(w_{ij_i^k}^k) + \mu_i^k$\;
				\STATE $\hat{\Delta}_i^k = \cC(g_i^k - h_i^k)$\;
				\STATE $h_i^{k+1} = h_i^k + \alpha \hat{\Delta}_i^k$\;
				\FOR{$j = 1, \ldots, m$}
					\STATE $
					(\text{Var 1})\; w_{ij}^{k+1} =
					\begin{cases}
						x^k, & \text{if } u^k = 1 \\
						w_{ij}^k, &\text{if } u^k = 0\\
					\end{cases} 
					$ 
					\STATE $
					(\text{Var 2})\; w_{ij}^{k+1} =
					\begin{cases}
					x^k, & j = j_i^k\\
					w_{ij}^k, & j \neq j_i^k\\
					\end{cases}
					$\;
				\ENDFOR
			\ENDFOR 
			\STATE  $g^k = h^k + \frac{1}{n}\sum_{i=1}^{n} \hat{\Delta}_i^k$\;
			\STATE $x^{k+1} = \prox_{\gamma R}(x^k - \gamma g^k)$\;
			\STATE $h^{k+1} = h^k + \frac{\alpha}{n} \sum_{i=1}^n \hat{\Delta}_i^k \left( = \frac1n \sum_{i=1}^n h_i^{k+1}\right)$ 
		\ENDFOR
\end{algorithmic}  
\caption{  \texttt{VR-DIANA} based on  \texttt{L-SVRG} (Var 1),  \texttt{SAGA} (Var 2)}
\label{alg:VR-DIANA}
\end{algorithm}

\subsection{The main challenge} Let us recall that the variance reduced method  \texttt{SVRG}~\cite{johnson2013accelerating} computes the full gradient $\nabla f(x)$ in every epoch and uses this vector as a control variate to bias the stochastic updates in the next epoch.
 In order to compute this control variate in the distributed setting~\eqref{eq:probR_diana}, each worker $i$ must compute the gradient $\nabla f_i(x)$ and send this vector (exactly) either to the master node or broadcast it to all other workers.  
 It might be tempting to replace this expensive communication step with quantized gradients instead, i.e., to rely on the aggregate $y_\cC \eqdef \frac{1}{n}\sum_{i=1}^n \cC(\nabla f_i(x))$ instead. However, the error $\norm{y_\cC-\nabla f(x)}$ can be arbitrarily large (e.g.\ it will even not vanish for $x=x^\star$) and this simple scheme \emph{does not} achieve variance reduction (cf.\ also the method and discussion in~\cite{kuenstner2017:svrg}).
 
 Our approach to tackle this problem is via the {\em quantization of gradient differences}. We propose that each worker maintains an individual control variate $h_i^k \in \R^d$, and only quantizes  \emph{differences} that are needed to update the control variate.
 For instance, for a  \texttt{SVRG}-like update with desired control variate $\nabla f_i(x^k)$, one could update $h_i^{k+1} = h_i^k + \cC(\nabla f_i(x^k) - h_i^k)$ (we assume this for  ease of exposition, the actual scheme is slightly different). We can now use the average over all devices, $h^k = \frac{1}{n}\sum_{i=1}^n h_i^k$ as the next control variate. In general $h^k \neq \nabla f(x^k)$, however, we are able to show that $h^k$ can still be used as a control variate to achieve variance reduction.  
 Essentially, by proving that $\norm{h_i^k - \nabla f_i(x^\star)} \to 0$ for $(k \to \infty)$ we are able to derive the {\em first variance reduced method that only exchanges quantized gradient updates among workers.}
\begin{algorithm}[t]
	\begin{algorithmic}[1]
		\STATE {\bfseries Input:}$p \geq 1$, learning rates $\alpha, \gamma>0$, initial vectors $x^0, h_{1}^0, \ldots, h_{n}^0 \in \R^d$, $p_0, p_1, \cdots, p_{l-1} \in \R$
		\STATE $s = 0$\\
		\STATE $x^0 = x^0$\\
		\STATE $z^0 = x^0$ \\
		\FOR{$k = 1,2,\ldots$}
			\STATE Broadcast $x^k$ to all workers\\
			\FOR{$i = 1, \ldots, n$}
				\IF{$k \equiv 0 \mod l$}
    					\STATE $s = s+1$ \\
    					\STATE $z^s =\sum_{r=0}^{l-1} p_r x^{(s-1)l+r}$\\
 				\ENDIF
				\STATE Pick random $j_i^k \in [m]$ uniformly\\
				\STATE $g_i^k = \nabla f_{ij_i^k}(x^k) - \nabla f_{ij_i^k}(z^s)  + \nabla f_{i}(z^s)$ \\ 
				\STATE $\hat{\Delta}_i^k = \cC( g_i^k - h_{i}^k)$\\
				\STATE $h_{i}^{k+1} = h_{i}^k + \alpha \hat{\Delta}_i^k$\\
			\ENDFOR
			\STATE $g^k = \frac{1}{n}\sum\limits_{i=1}^{n} (\hat{\Delta}_i^k + h_i^k)$\\
			\STATE $x^{k+1} =  \prox_{\gamma R}(x^k - \gamma g^k)$\\
		\ENDFOR
	\end{algorithmic}	
	\caption{\texttt{SVRG}-\texttt{DIANA}}
	\label{alg:SVRG_DIANA}
	\end{algorithm}
\subsection{Three new  algorithms} We propose in total three variance reduced algorithms, which are derived from either  \texttt{SAGA} (displayed in Algorithm~\ref{alg:VR-DIANA}, Variant 2),  \texttt{SVRG} (Algorithm~\ref{alg:SVRG_DIANA}) and  \texttt{L-SVRG} (Algorithm~\ref{alg:VR-DIANA}, Variant 1), a variant of  \texttt{SVRG} with random epoch length
described in~\cite{Kovalev2019:svrg}.
We prove  global linear convergence in the strongly convex case and $\cO(1/k)$ convergence in convex and non-convex cases. Moreover, our analysis is very general in the sense that the original complexity results for all three algorithms can be obtained by setting $\omega = 0$ (no quantization).
\begin{remark}
For the ease of exposition, we present all three algorithms in a way that each worker uses the same compression operator. We can extend our results to work with arbitrary (possibly different) $\omega_i$-quantization operators $\cC_i$. The constant $\omega$ that appears in the rates would be replaced by $\max_{i\in [n]} \omega_i$ Note that one cannot do better than maximum (e.g. average) as  each worker might have completely different function to optimize.
\end{remark}

\textbf{Comments on Algorithm~\ref{alg:VR-DIANA}.} 
In analogy to Algorithm~\ref{alg:improvedDIANA}, each node maintains a state $h_i^k \in \R^d$ that aims to reduce the variance introduced by the quantized communication. In contrast to the variance reduced method in~\cite{qsgd2017neurips} that required the communication of the uncompressed gradients for every iteration of the inner loop (\texttt{SVRG} based scheme), here we communicate only quantized vectors.

Each worker $i=1,\dots,n$ computes its stochastic gradient $g_i^k$ in iteration $k$ by the formula given by the specific variance reduction type (\texttt{SVRG}, \texttt{SAGA}, \texttt{L-SVRG}), such that $\E{g_i^k \mid x^k} = \nabla f_i(x^k)$. 
It is important to note that the workers only quantize the difference $g_i^k - h_i^k$ instead of $g_i^k$ directly. This quantized vector  $\hat{\Delta}_i^k$ is then sent to the master node which in turn updates its local copy of~ $h_i^k$. Thus both the $i$-th worker and the master node have access to $h_i^k$ even though it has never been transmitted in full (but instead incrementally constructed from the $\hat{\Delta}_i^k$'s).
The algorithm based on \texttt{SAGA} maintains on each worker a table of gradients, $\nabla f_{ij}(w_{ij}^k)$ (so the computation of $\mu_i^k$ on line~\ref{ln:mu} can efficiently be implemented). The algorithms based on \texttt{SAGA} just need to store the epoch gradients $\frac{1}{m}\sum_{i=1}^n \nabla f_{ij}(w_{ij}^k)$ that needs only to be recomputed whenever the $w_{ij}^k$'s change. Which is either after a fixed number of steps in \texttt{SVRG}, or after a random number of steps as in \texttt{L-SVRG}. These aspects of the algorithm are not specific to quantization and we refer the readers to e.g.~\cite{Raj2018:ksvrg} for a more detailed exposition.

For pseucode, see Algorithm~\ref{alg:VR-DIANA}

\section{Convergence of  \texttt{VR-DIANA} (Algorithm~\ref{alg:VR-DIANA})}

We make the following technical assumptions (out of which only the first one is shared among all theorems in this section) and then proceed with the main theorems.

\begin{assumption}\label{assumption:VRdiana}
Finite-sum structure in problem~\eqref{eq:probR_diana}, and each $f_{ij} \colon \R^d \to \R$ to be $L$-smooth.
\end{assumption}

\begin{assumption}\label{assumption:VRdianasc}
Each $f_i \colon \R^d \to \R$ in~\eqref{eq:probR_diana} $\mu$-strongly convex, $\mu > 0$ and each $f_{ij} \colon \R^d \to \R$  convex.
\end{assumption}

\begin{assumption}\label{assumption:VRdianac}
Each function $f_{ij} \colon \R^d \to \R$ is convex.
\end{assumption}

\subsection{Strongly convex case} First, we analyze  \texttt{VR-DIANA} on strongly convex functions.

\begin{theorem}\label{thm:VR-DIANA}
    Consider Algorithm~\ref{alg:VR-DIANA} with $\omega$-quantization $\cC$,
	and step size $\alpha \leq \frac{1}{\omega + 1}$.
	For 
	$b = \frac{4(\omega+1)}{\alpha n^2}$,
	$c = \frac{16(\omega+1)}{\alpha n^2}$,
	$\gamma = \frac{1}{L\left(1 + 36(\omega+1)/n\right)}$,
	define the Lyapunov function
\begin{equation*}
	\psi^k = \norm{x^k - x^\star}^2 + b \gamma^2 H^k + c \gamma^2 D^k\,,
\end{equation*}
\par\vspace{-2ex}
where

\begin{equation*}
	 H^k =  \sum_{i=1}^{n} \norm{h_i^k - \nabla f_i(x^\star)}^2\,, \quad
	D^k = 
	\sum_{i=1}^{n} \sum_{j=1}^{m}
		\norm{\nabla f_{ij}(w_{ij}^k) - \nabla f_{ij}(x^\star)}^2\,.
\end{equation*}

\noindent Then	
under Assumption~\ref{assumption:VRdiana} and \ref{assumption:VRdianasc}
	\begin{equation*}
		\E{\psi^{k+1}} \leq
		(1 - \rho)\psi^k,
	\end{equation*}
	where $\rho\eqdef \min\bigl\{
			\frac{\mu}{L\left(1 + 36\frac{\omega + 1}{n} \right)},
			 \frac{\alpha}{2},
			\frac{3}{8m}
		\bigr\}$ and the expectation is conditioned on the previous iterate.
\end{theorem}

\begin{corollary}
    Let $\alpha = \tfrac{1}{\omega + 1}$.
	To achieve precision $\E{\|x^k - x^\star\|^2} \leq \varepsilon \psi^0$  \texttt{VR-DIANA} needs
	\begin{align*}
	\cO\rbr*{
	\rbr*{\kappa + \kappa \frac{\omega}{n} + m + \omega}\log\frac{1}{\epsilon}
	}
	\end{align*}
	iterations.
\end{corollary}

Recall that variance reduced methods (such as  \texttt{SAGA} or  \texttt{SVRG}) converge at rate $\hat{\cO}(\kappa + m)$ in this setting (cf. Table~\ref{tab:algo-comparison_diana_diana}). The additional variance of the quantization operator enters the convergence rate in two ways: firstly, (i), as an additive component 
affecting the rate only mildly.
Secondly, (ii), and more severely, as a multiplicative component $\frac{\kappa \omega}{n}$. However, we see that this factor is discounted by $n$, the number of workers. Thus by choosing a quantization operator with $\omega = \cO(n)$, this term can be controlled. In summary, by choosing a quantization operator with $\omega = \cO(\min\{n,m\})$, the rate of \texttt{VR-DIANA} becomes identical to the convergence rate of the vanilla variance reduced schemes without quantization.
This shows the importance of algorithms that support arbitrary quantization schemes and that do not depend on a specific quantization scheme.

\subsection{Convex case}
Let us now look at the convergence under (standard) convexity assumption, that is, $\mu = 0$. By taking output to be some iterate $x^k$ with uniform probability instead of the last iterate, one gets the following convergence rate.

\begin{theorem}[Convex case]\label{thm:VR-DIANAweak}
Let Assumptions~\ref{assumption:VRdiana} and~\ref{assumption:VRdianac} hold, then a randomly chosen iterate  $x^a$ of Algorithm~\ref{alg:VR-DIANA}, i.e.\  $x^a \sim_{u.a.r.} \{x^0, x^1,\dots, x^{k-1}\}$, satisfies
\begin{equation*}
\E{B_f(x^a, x^\star)} \leq \frac{\psi_0}{2k\left(
		\gamma - 
		L\gamma^2
		\left[
			1 + \frac{36(\omega+1)}{n}	 		
		\right]
	\right)}, 
\end{equation*} 
where $k$ denotes the number of iterations.
\end{theorem}

\begin{corollary}
    Let $\gamma = \frac{1}{2L\sqrt{m}\left(1+36\frac{\omega + 1}{n}\right)}$, $b = \frac{2(\omega+1)}{\alpha n^2}$, $c= \frac{6(\omega+1)}{n^2}$ and $\alpha = \frac{1}{(\omega + 1)}$.
	To achieve precision $\E{B_f(x^a, x^\star)} \leq\varepsilon$  \texttt{VR-DIANA} needs
	$
	\cO\left(
	\frac{\left(1+\tfrac{\omega}{n}\right)\sqrt{m} + \tfrac{\omega}{\sqrt{m}}}{\epsilon}
	\right)
	$
	iterations.
\end{corollary}
Here we see that along the quantization operator is chosen to satisfy \\ $\omega = \cO(\min\{m,n\})$ the convergence rate is not worsened compared to a scheme without quantization.

\subsection{Non-convex case}
Finally, we also show convergence in the non-convex setting.

\begin{theorem}\label{thm:VR-DIANAnc}
Let Assumption~\ref{assumption:VRdiana} hold and $R \equiv 0$. Moreover, let   \[\gamma = \frac{1}{10L\left( 1 + \frac{\omega}{n} \right)^{1/2} (m^{2/3} + \omega + 1)}\] and $\alpha = \frac{1}{(\omega+1)}$, then  a randomly chosen iterate $x^a \sim_{u.a.r.} \{x^0, x^1,\dots, x^{k-1}\}$ of Algorithm~\ref{alg:VR-DIANA} satisfies
\begin{align*}
\E{\norm{\nabla f(x^a)}^2} &\leq  \frac{40(f(x^0) - f^\star)L\left(1 +  \frac{\omega}{n} \right)^{1/2}(m^{2/3} + \omega + 1)}{k}, 
\end{align*}
where $k$ denotes the number of iterations.
\end{theorem}

\begin{corollary}
	To achieve $\E{\norm{\nabla f(x^a)}^2} \leq\varepsilon$  \texttt{VR-DIANA} needs
	\[
	\cO\left(
	\left(1 + \frac{\omega}{n} \right)^{1/2} \frac{( m^{2/3} + \omega)}{\varepsilon}
	\right)
	\]
	iterations.
\end{corollary}

As long as $\omega\le m^{2/3}$, the iteration complexity above is $\cO(\omega^{1/2})$ assuming the other terms are fixed. At the same time, the communication complexity is proportional to the number of nonzeros, which for random dithering and random sparsification decreases as  $\cO(1/\omega)$. Therefore, one can trade-off iteration and communication complexities by using quantization (cf.\ Table~\ref{tab:algo-comparison_diana_diana} above.)

\begin{remark}
Note that assuming the total number of data points to be fixed $M = nm$  increasing number of nodes $n$ by factor $\lambda > 1$ reduces both $m \rightarrow  \frac{m}{\lambda}$ and $\frac{\omega}{n} \rightarrow  \frac{\omega}{\lambda n}$, thus improving overall rate of convergence.
\end{remark}

\section{Convergence of  \texttt{SVRG}-\texttt{DIANA} (Algorithm~\ref{alg:SVRG_DIANA})}
We also provide the convergence for the Algorithm~\ref{alg:SVRG_DIANA} in all three cases---strongly convex, convex and non-convex. Since these rates are the same in terms of $\cO$-notation as the rates for  \texttt{VR-DIANA}, we don't include them in the main text. The detailed  theorems with the proofs can be found in Appendix~\ref{sec:SVRG_DIANA}.

\section{Experiments}

\begin{figure}[t]
\centering
\hfill
\subfigure[Real-sim, \texttt{$\lambda_2 = 6\cdot 10^{-5}$}]{\includegraphics[width=0.3\textwidth]{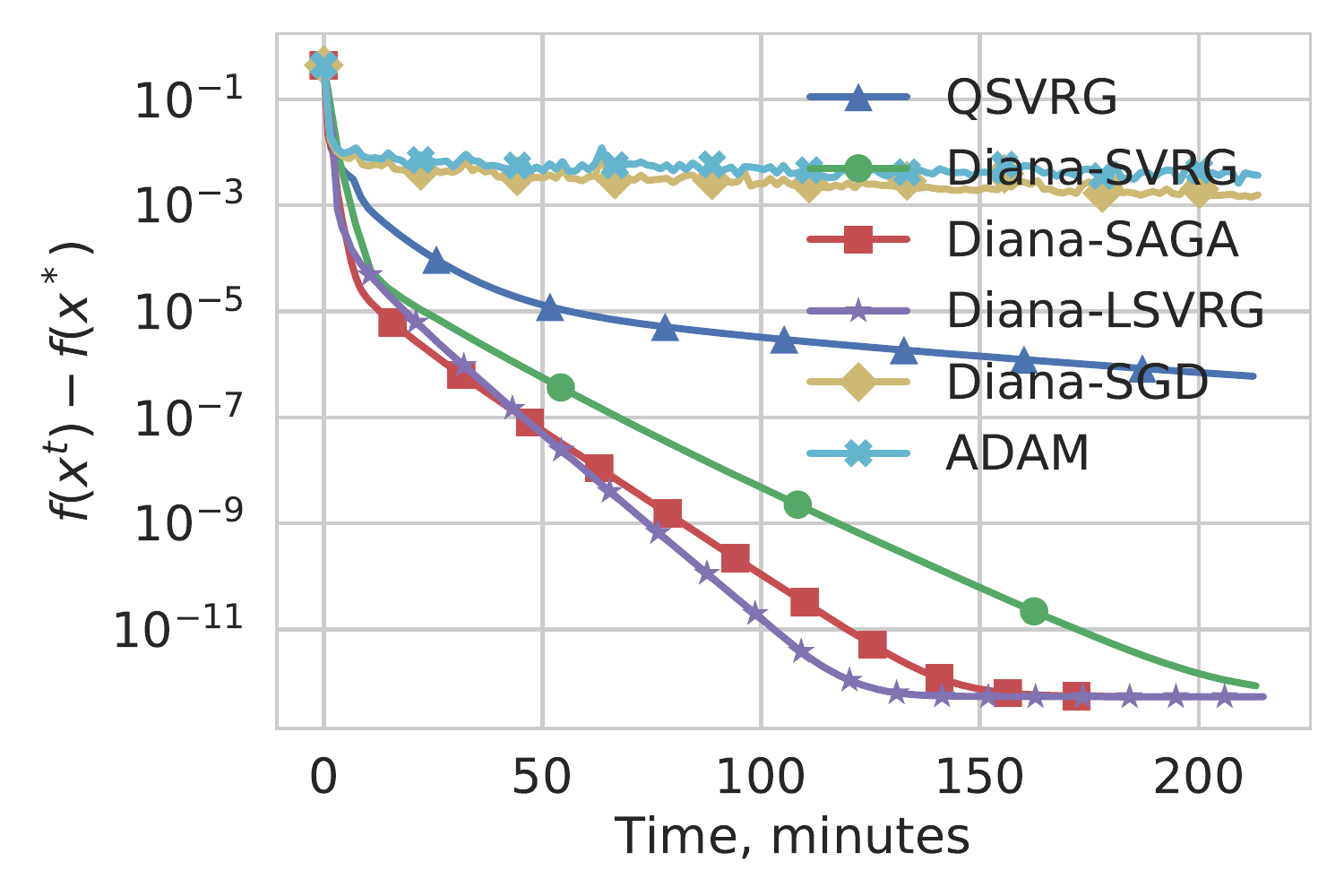}}
\hfill
\subfigure[Real-sim, \texttt{$\lambda_2 = 6\cdot 10^{-5}$}]{\includegraphics[width=0.3\textwidth]{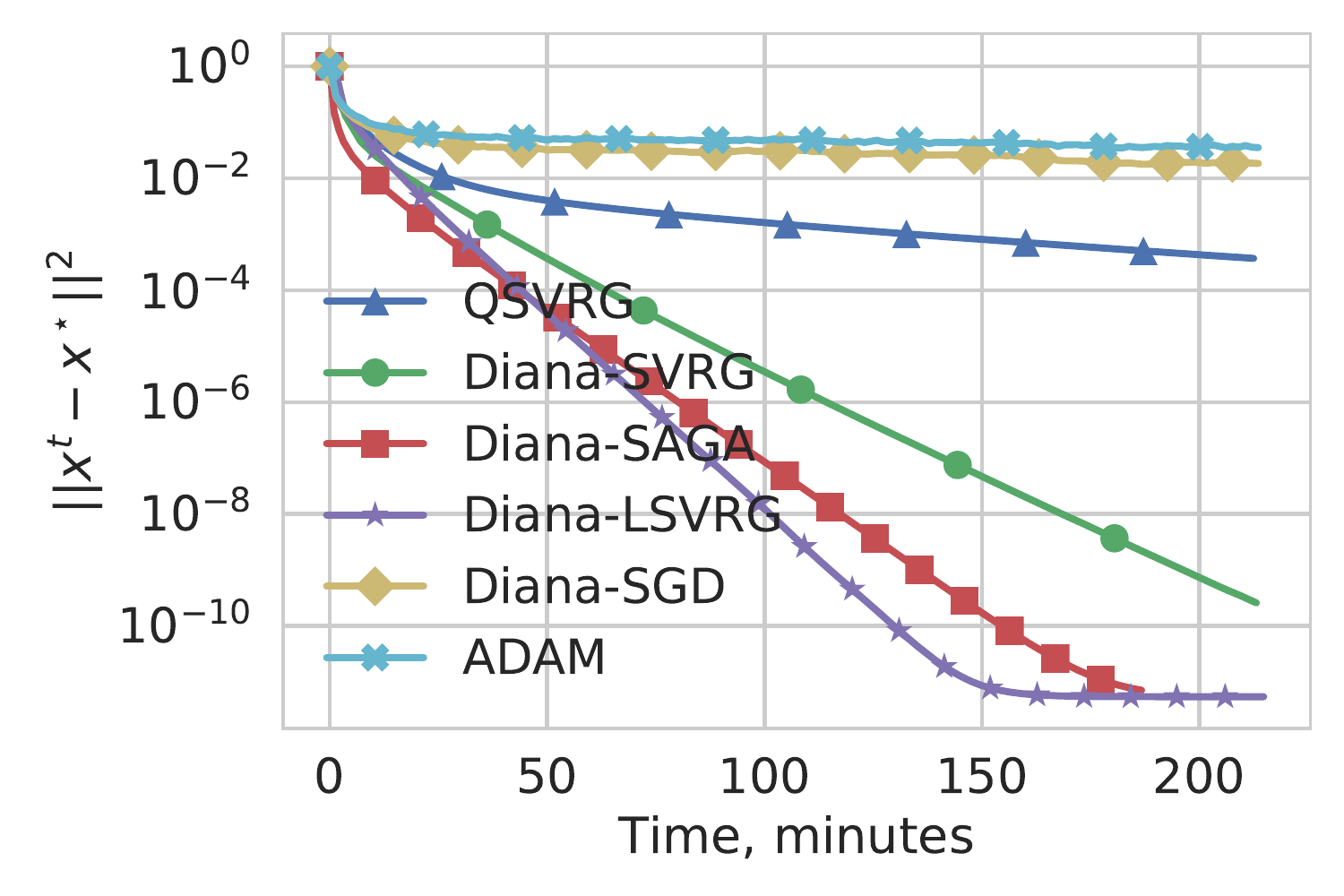}}
\hfill\null
\caption{Comparison of  \texttt{VR-DIANA}, \texttt{DIANA}-\texttt{SGD}, \texttt{QSVRG} and  \texttt{TernGrad}-\texttt{Adam} with $n=12$ workers on \texttt{real-sim} in suboptimality (left) and distance from the optimum (right).
$\ell_{\infty}$ dithering is used for every method except for \texttt{QSVRG}, which uses $\ell_2$ dithering.
}
\label{fig:with_adam}
\end{figure}

\begin{figure}[t]
\centering
\hfill
\subfigure[\texttt{SAGA}]
{\includegraphics[trim={3mm 0 3mm 0},clip,width=0.25\textwidth]{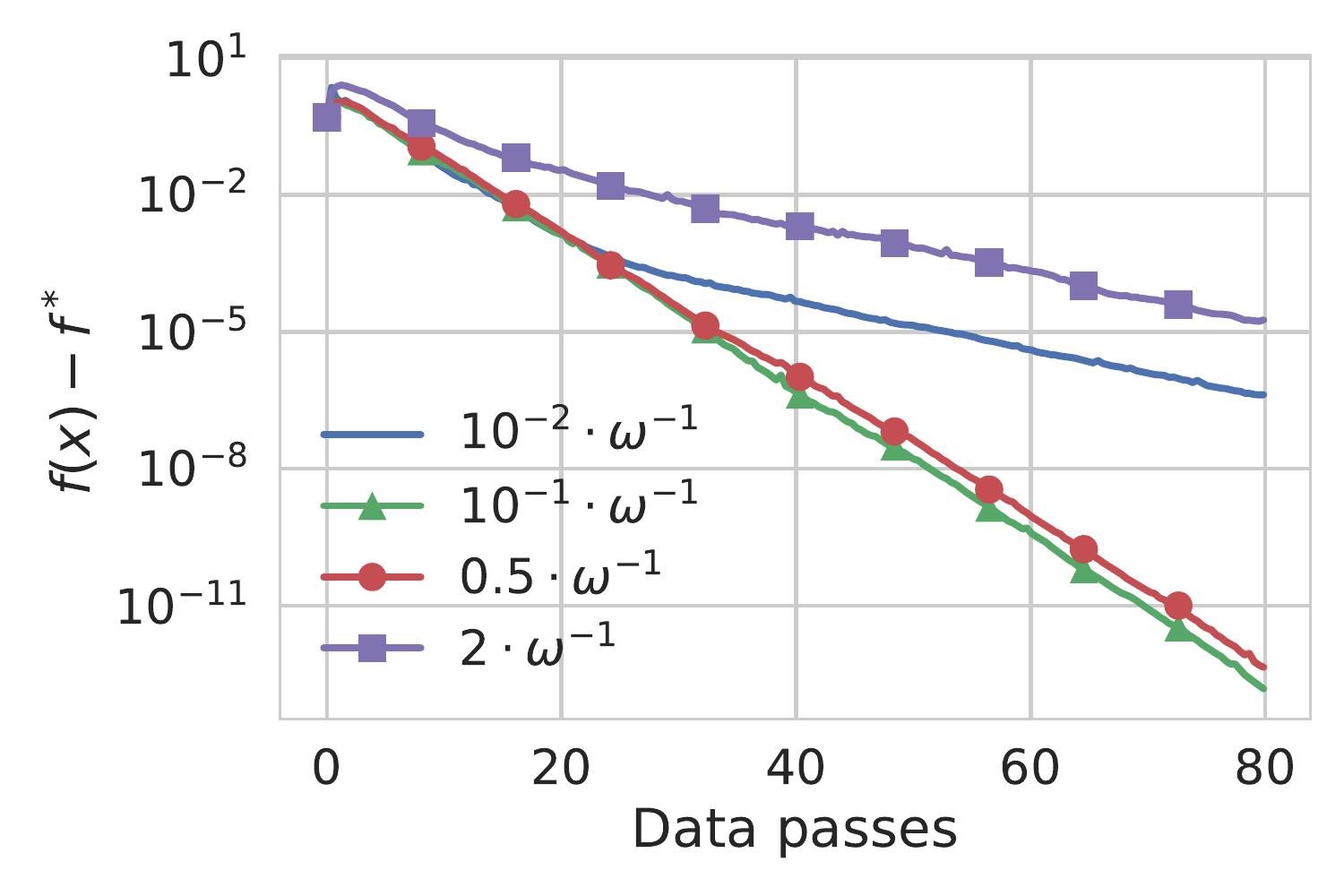}}
\hfill
\subfigure[\texttt{SVRG}]
{\includegraphics[trim={3mm 0 3mm 0},clip,width=0.25\textwidth]{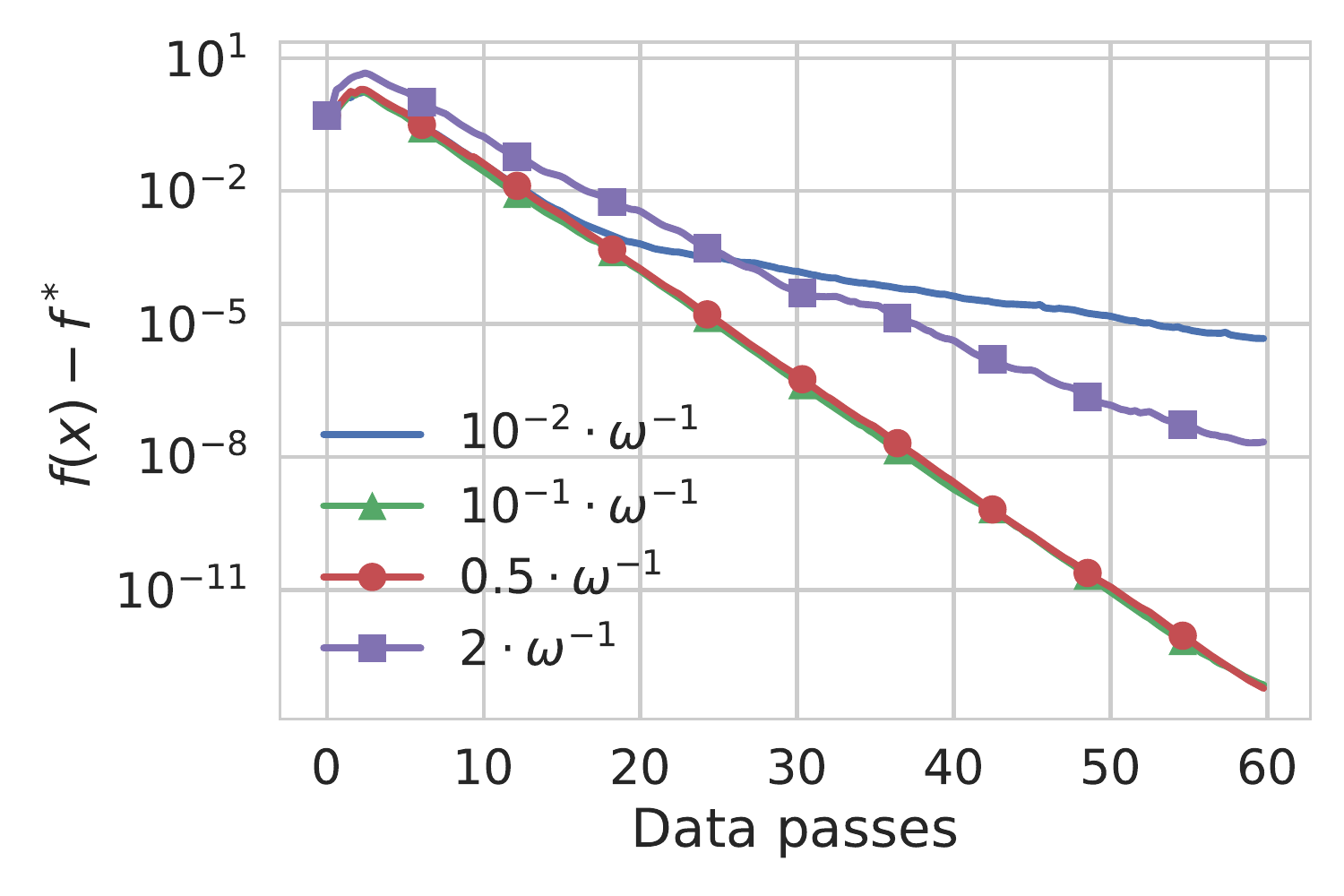}}
\hfill
\subfigure[\texttt{L-SVRG}]
{\includegraphics[trim={3mm 0 3mm 0},clip,width=0.25\textwidth]{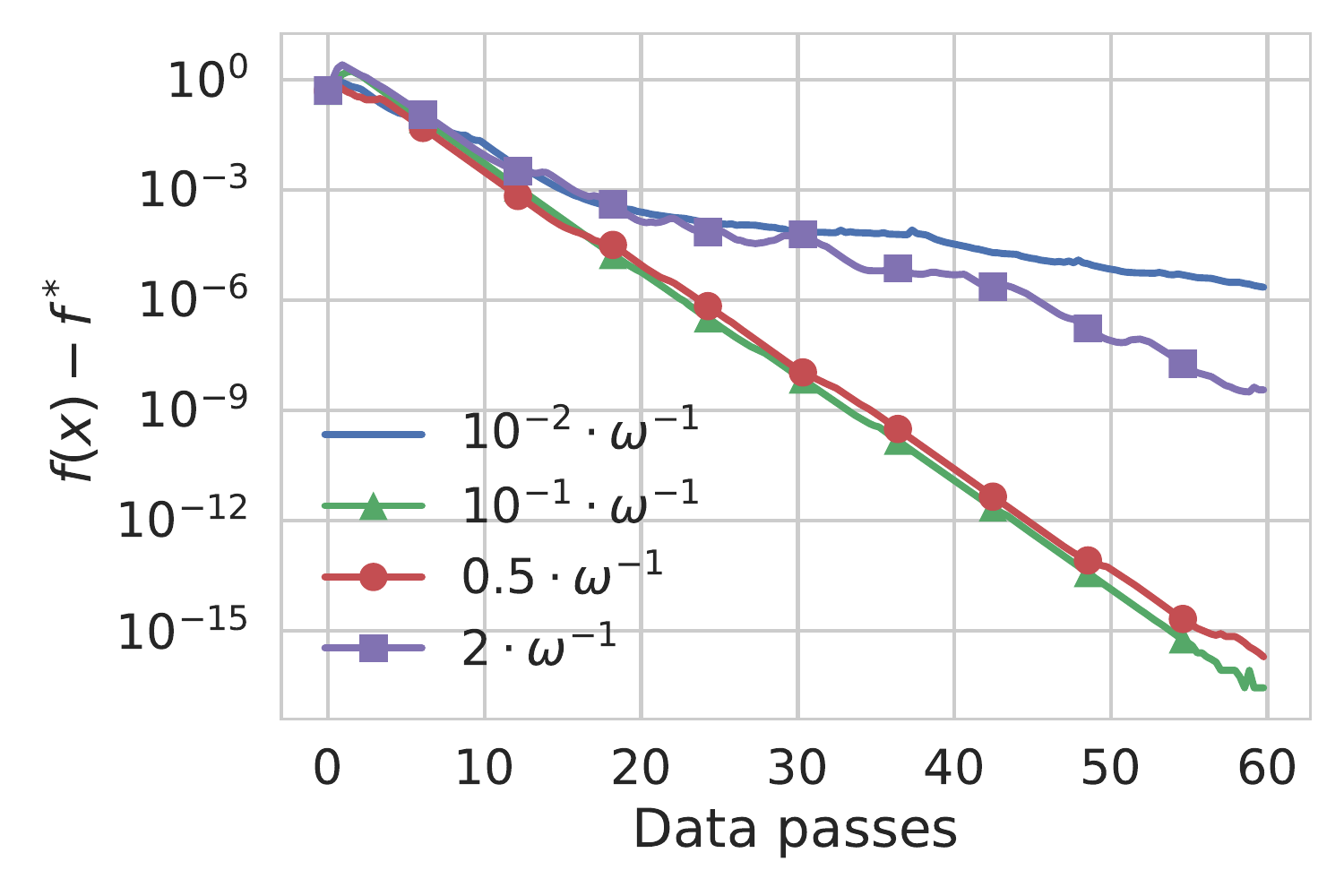}}
\hfill\null
\caption{Comparison of VR methods with different parameter $\alpha$ for solving \texttt{gisette} with block size 2000, $\ell_2$-penalty $\lambda_2=2\cdot 10^{-1}$, and $\ell_2$ random dithering.\label{fig:alpha_comparison}}
\end{figure}

\begin{figure}[t]
\centering
\subfigure[\texttt{Mushrms,$\lambda_2 = 6\!\cdot\!10^{-4}$}]{
\includegraphics[width=0.24\textwidth]{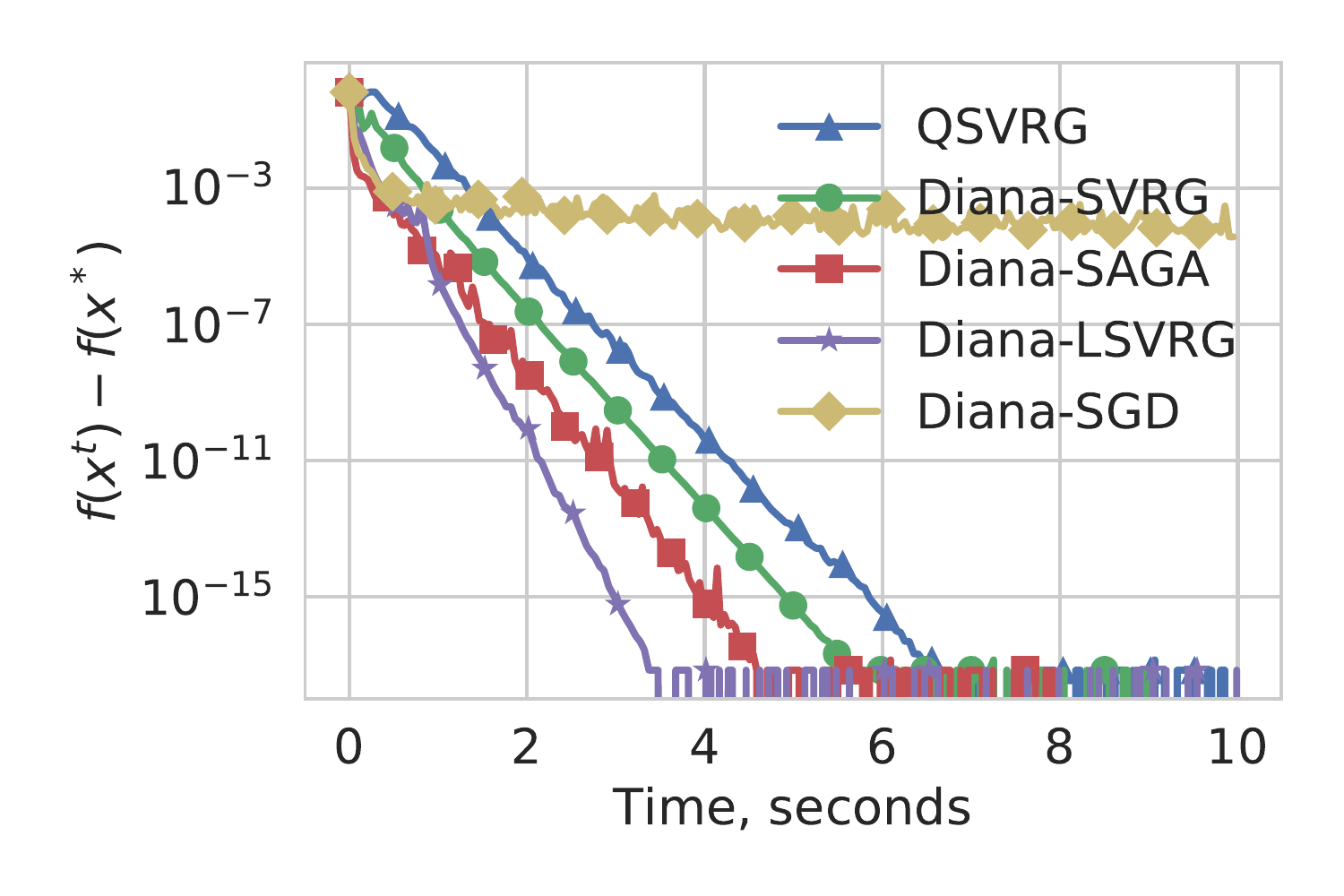}}
\hfill
\subfigure[\texttt{Mushrms,$\lambda_2 =\! 6\!\cdot\!10^{-5}$}]{\includegraphics[width=0.24\textwidth]{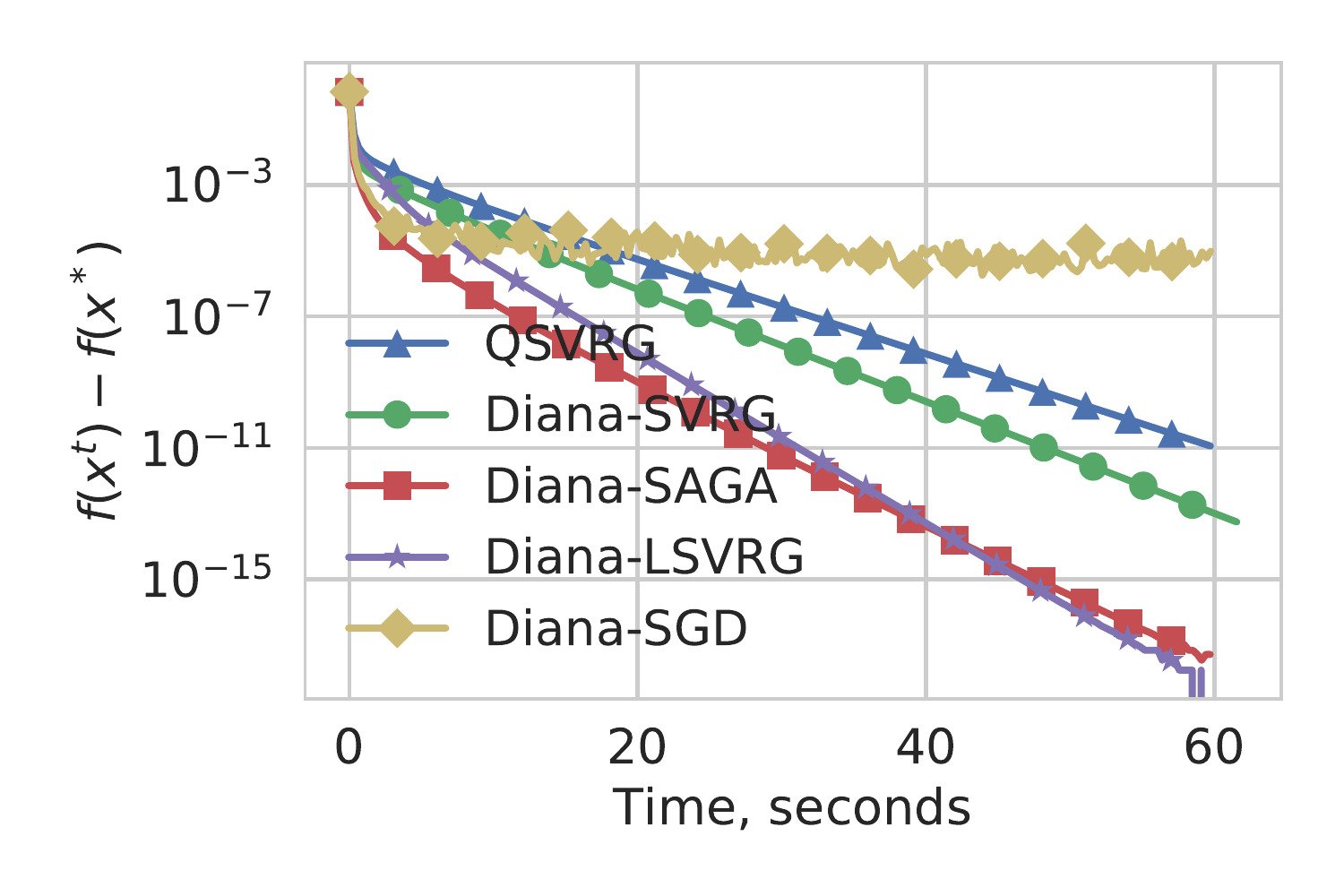}}
\hfill
\subfigure[\texttt{a5a,$\lambda_2 = 5\cdot 10^{-4}$}]{\includegraphics[width=0.245\textwidth]{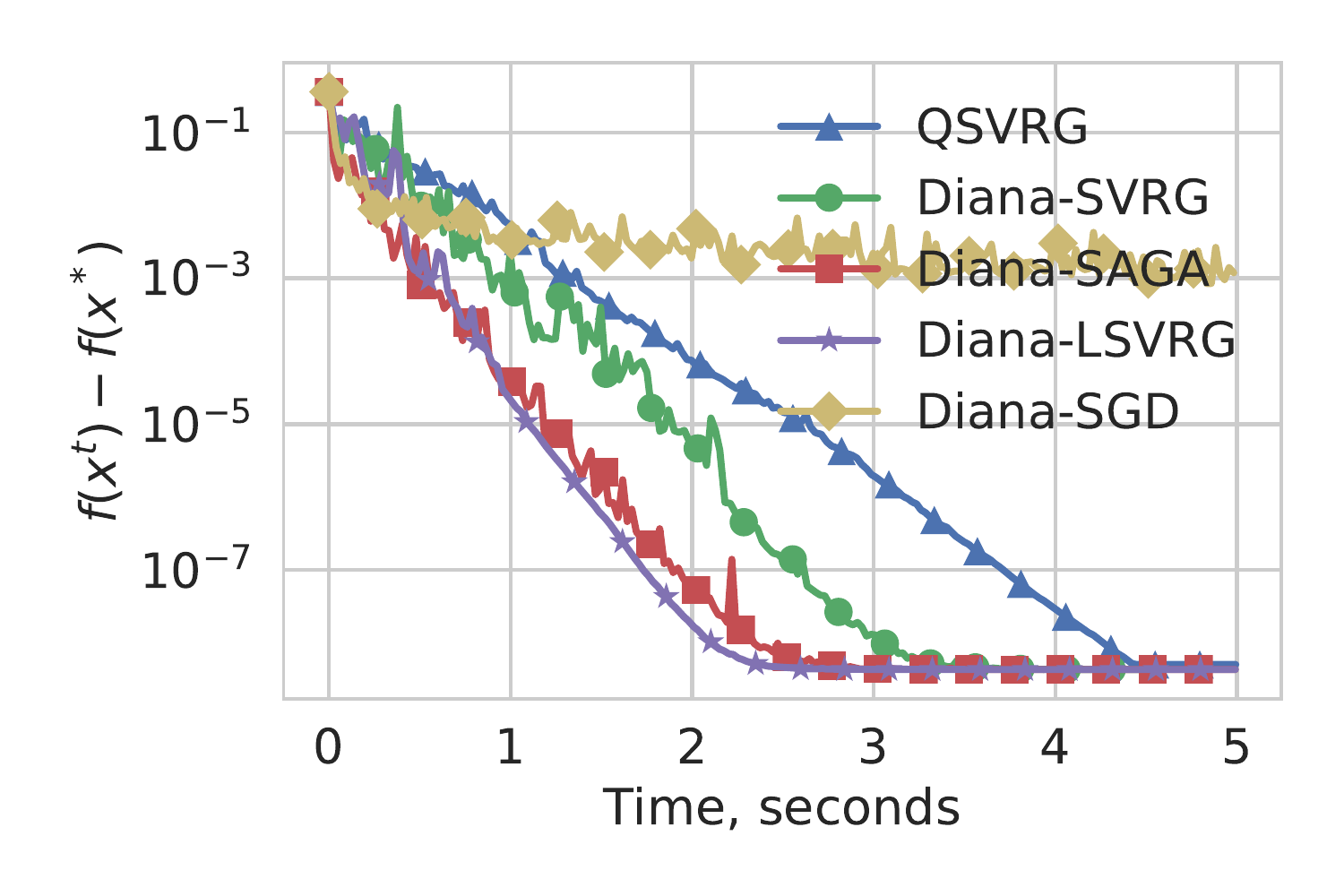}}
\hfill
\subfigure[\texttt{a5a,$\lambda_2 = 5\cdot 10^{-5}$}]{\includegraphics[width=0.24\textwidth]{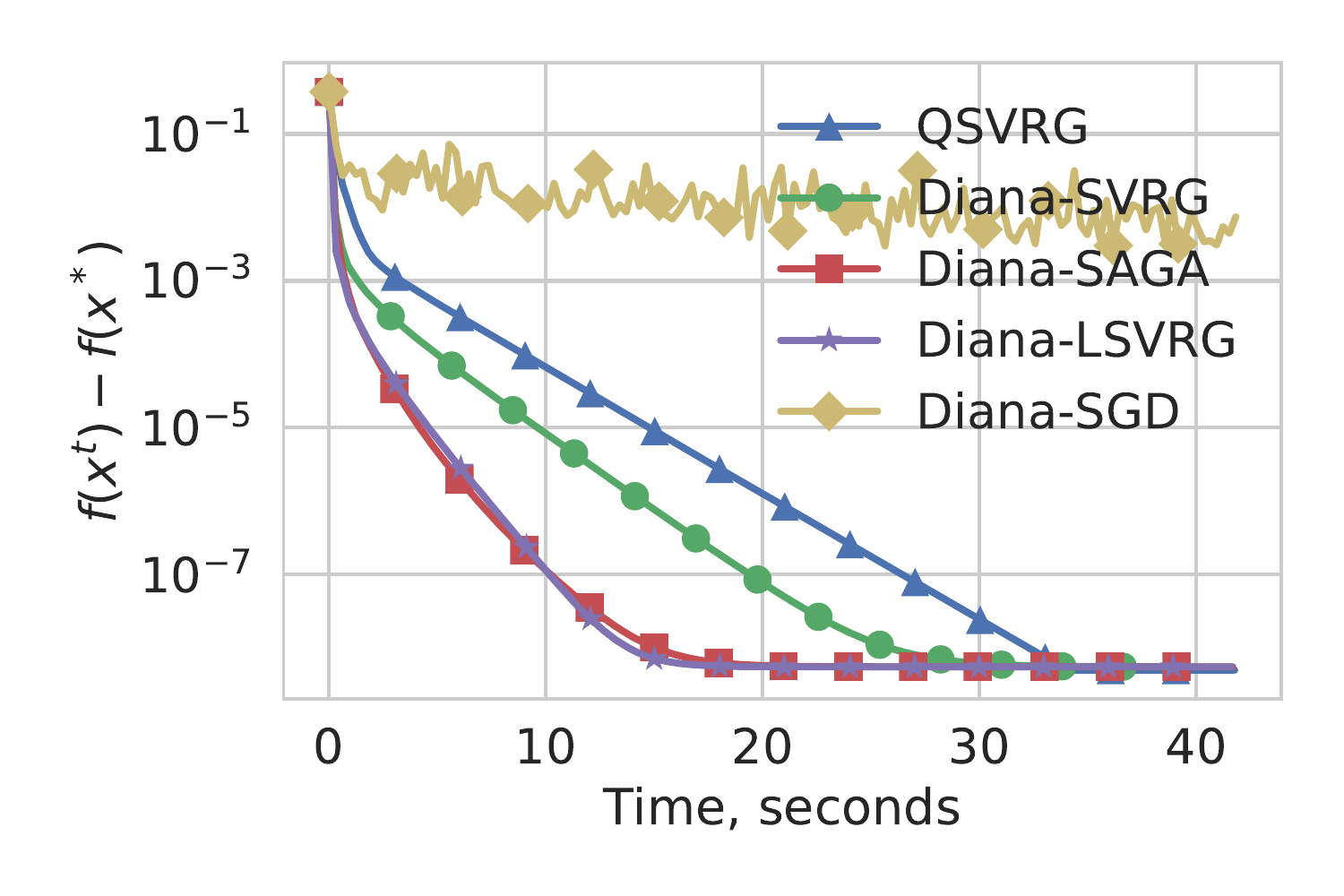}}\\
\vspace{-0.3cm}
\subfigure[\texttt{Mushrms,$\lambda_2 =\! 6\!\cdot\! 10^{-4}$}]{\includegraphics[width=0.24\textwidth]{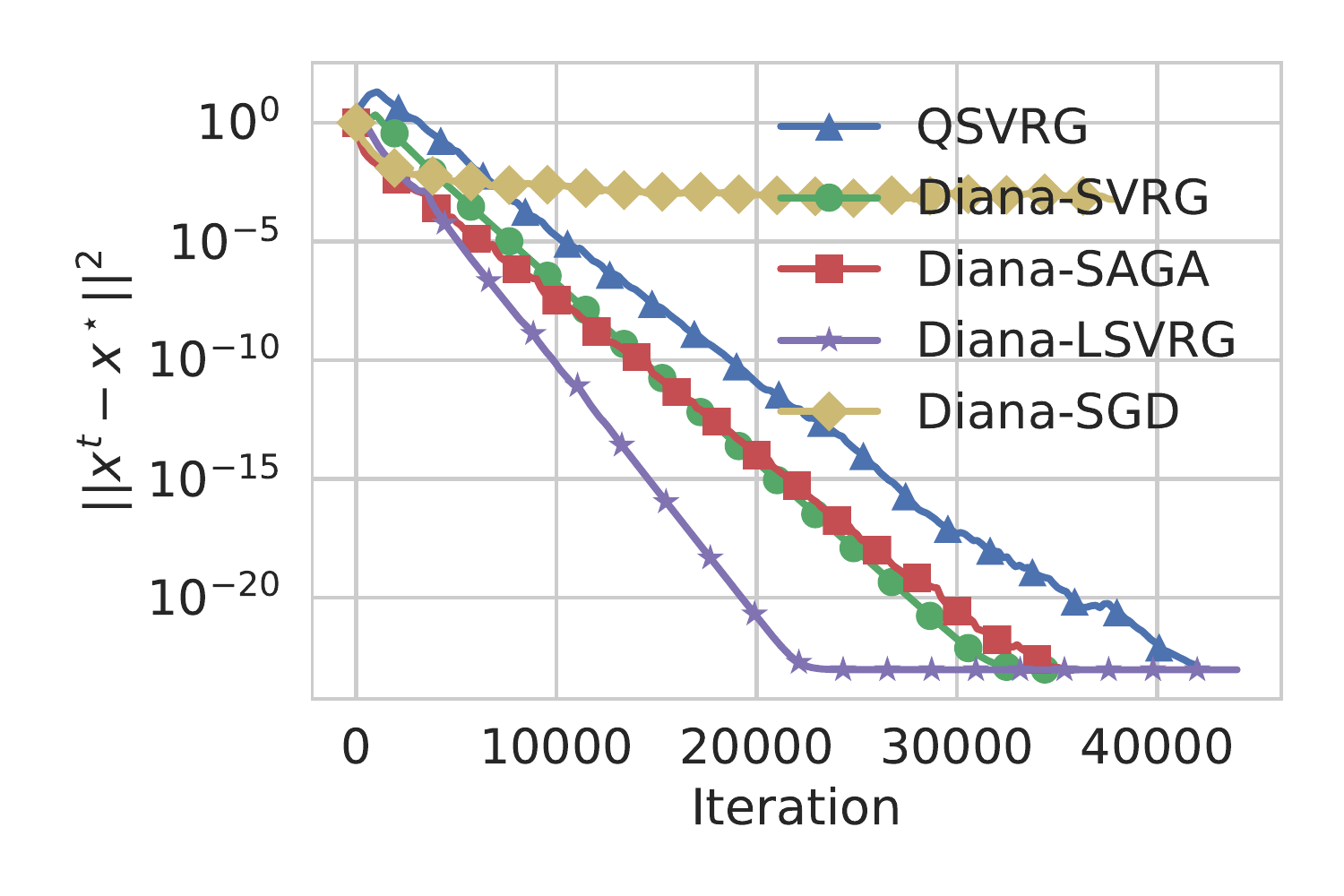}}
\hfill
\subfigure[\texttt{Mushrms,$\lambda_2 = 6\!\cdot\! 10^{-5}$}]{\includegraphics[width=0.24\textwidth]{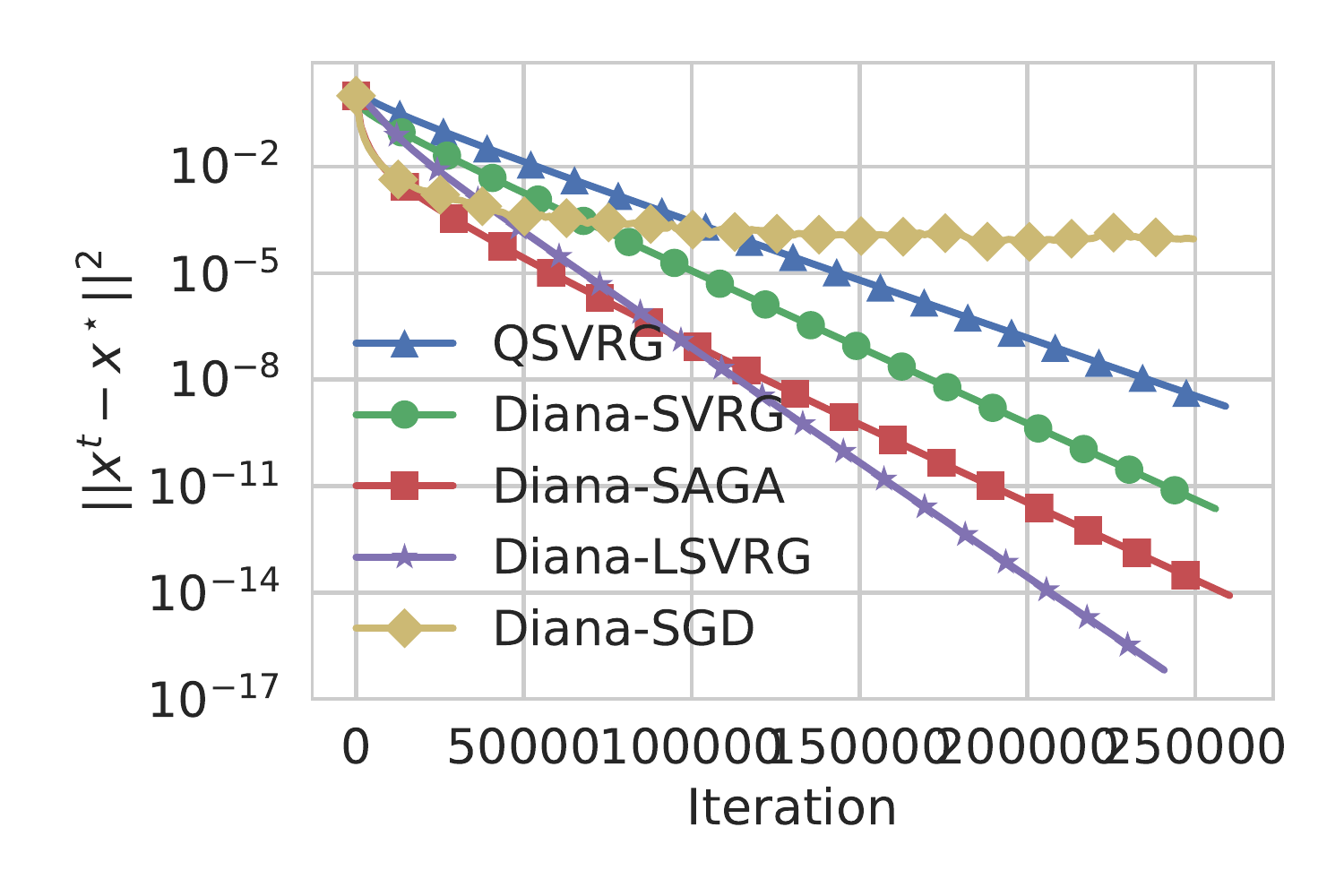}}
\hfill
\subfigure[\texttt{a5a,$\lambda_2 = 5\cdot 10^{-4}$}]{\includegraphics[width=0.24\textwidth]{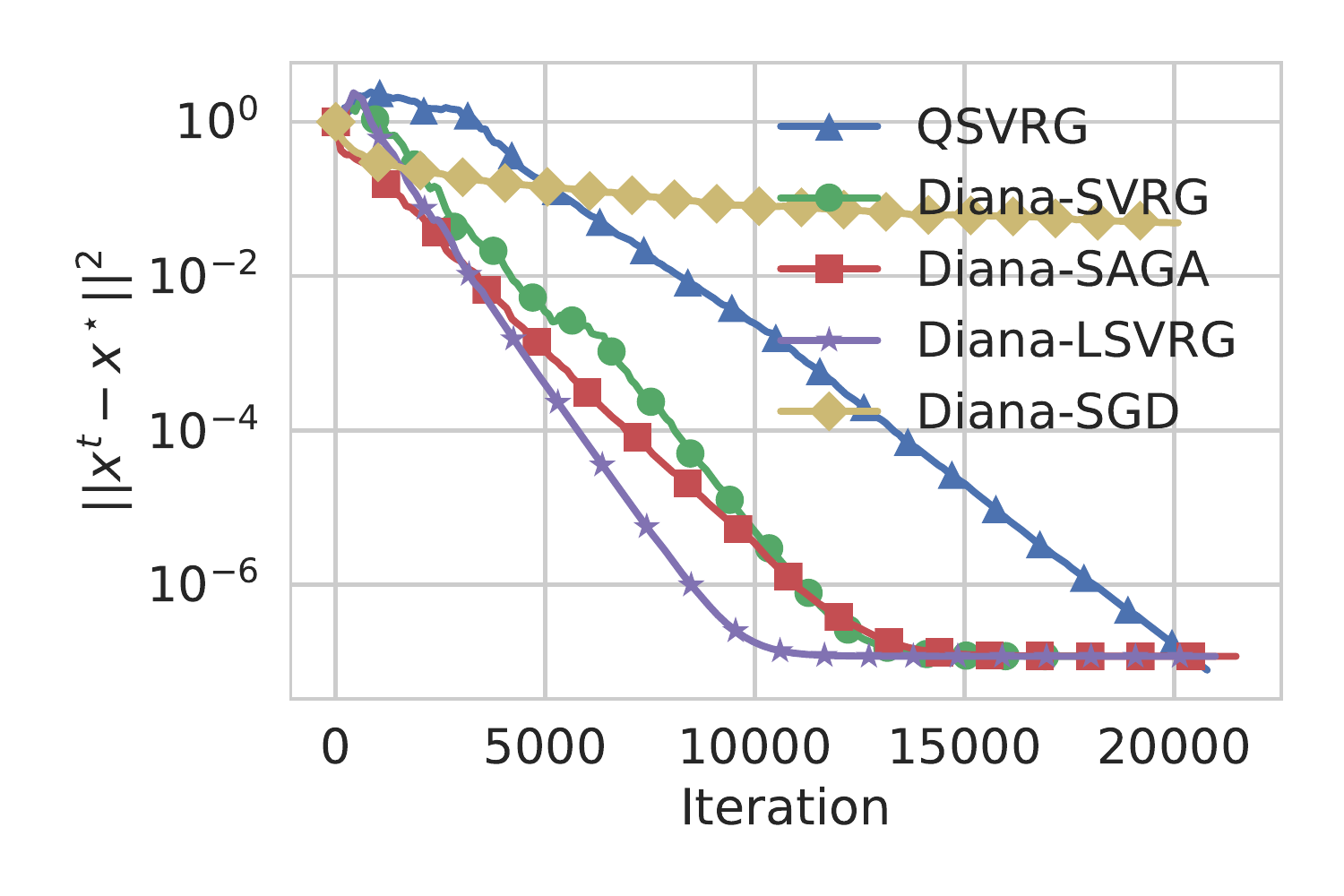}}
\hfill
\subfigure[\texttt{a5a,$\lambda_2 = 5\cdot 10^{-5}$}]{\includegraphics[width=0.24\textwidth]{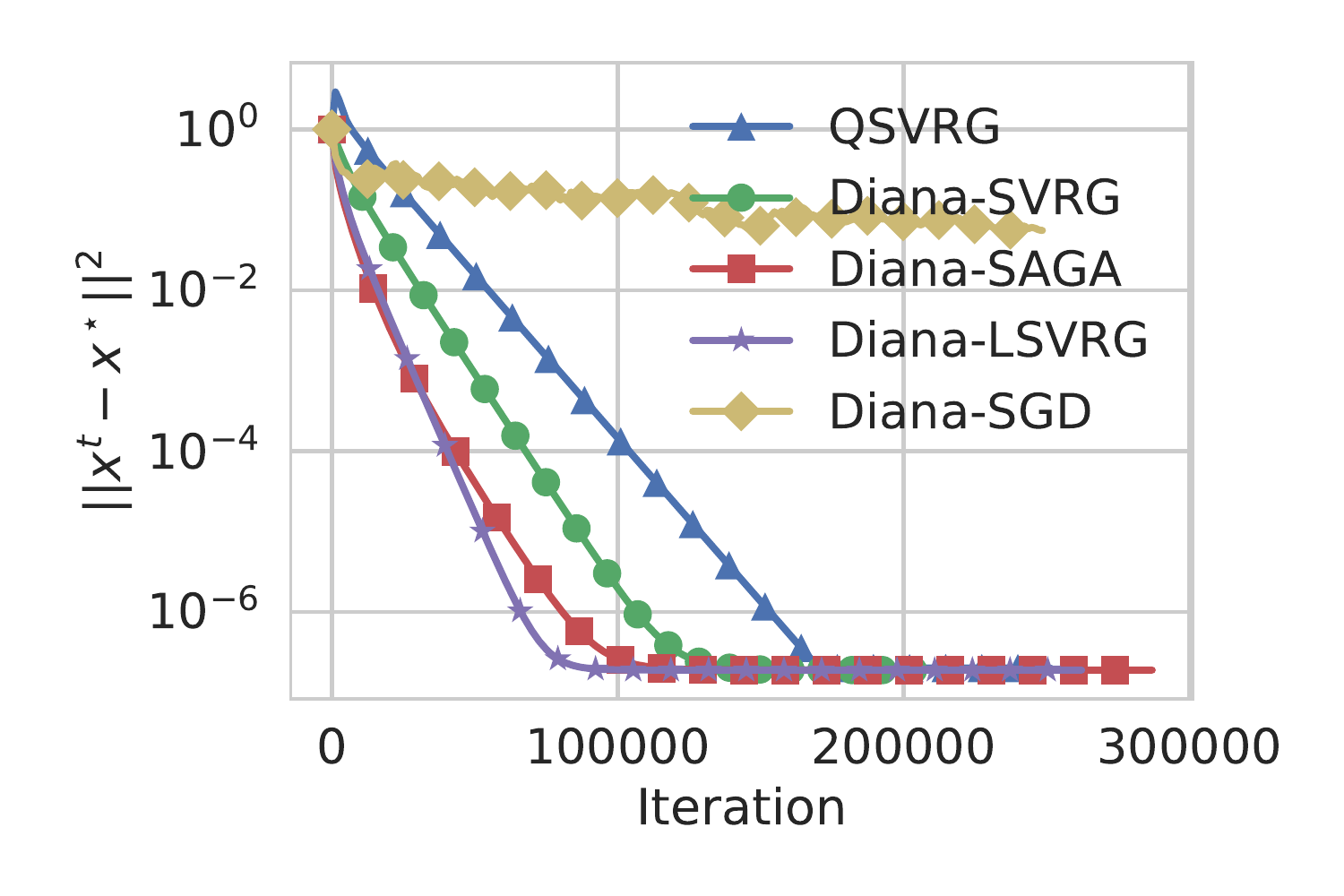}}
\caption{Comparison of  \texttt{VR-DIANA} and \texttt{DIANA}-\texttt{SGD} against  \texttt{QSVRG} \cite{qsgd2017neurips} on \texttt{mushrooms} and \texttt{a5a} in suboptimality (top) and distance to the solution (bottom).}
\label{fig:qsgd}
\end{figure}

We illustrate the performance of the methods on standard logistic regression for binary classification, given as 
$
\log(1 + \exp(-b_{ij} A_{ij}^\top x)) + \frac{\lambda_2}{2}\|x\|^2\,,
$
where $A_{ij}, b_{ij}$ are data points and $\lambda_2 \approx \tfrac{1}{nm}$ regularization. This problem is attractive because some of the methods to which we want to compare do not have convergence guarantees for non-convex or non-smooth objective.
We use datasets from the LIBSVM library~\cite{CC01a}.
All methods are implemented in Python using MPI4PY~\cite{dalcin2011parallel} for collective communication operations and run on a machine with 24 Intel(R) Xeon(R) Gold 6146 CPU @ 3.20GHz cores. The cores are connected to two sockets, 12 cores to each of them, and we run workers and the parameter server on different cores. 

To have a trade-off between communication and iteration complexities, we use in most experiments the block quantization scheme with block sizes equal $n^2$, and in one additional experiments we explore what changes under different block sizes.
We analyze the effect of changing the parameter $\alpha$ over different values in Figure~\ref{fig:alpha_comparison} and find out that it is very easy to tune, and in most cases the speed of convergence is roughly the same unless $\alpha$ is set too close to 1 or to 0. For instance, we see almost no difference between choosing any element of $\{10^{-2}, 10^{-3}, 5\cdot 10^{-4}\}$ when $\omega^{-1} = 2\cdot 10^{-2}$, although for tiny values the convergence becomes much slower. For more detailed consideration of block sizes and choices of $\alpha$ on a larger dataset see Figure~\ref{fig:blocks_and_alphas}. 

We also provide a comparison between the most relevant quantization methods ---  \texttt{QSVRG}~\cite{qsgd2017neurips},  \texttt{TernGrad}-\texttt{Adam} and \texttt{DIANA}~\cite{mishchenko2019distributed}, in Figure~\ref{fig:qsgd}. Since  \texttt{TernGrad}-Adam was the slowest in our experiments, we provide it only in Figure~\ref{fig:with_adam}.

We don't perform experiments for the more complex problems such as Deep Nets as the extra variance is not usually the problem because such models are over-parameterized, thus there is no variance of stochastic gradients in the optimum as one can optimize network to zero loss. In addition, because of over-parameterization, there is no extra variance imposed by compression. In terms of our analysis, this would mean that there is no need to have terms $D^k$ and $H^k$ and one would obtain the same result without our tricks. This does not undermine our results but sheds light on the optimization of Deep Nets and explains why they perform so well in practice.


\chapter{On biased compression for distributed learning}
\label{chapter4:biased}

\section{Introduction}\label{sec:intro_biased}

In order to achieve state-of-the-art performance, modern machine learning models need to be trained using large corpora of training data, and often feature an even larger number of trainable parameters \cite{vaswani2018fast,Brown2020fewshot}. The data is typically collected in a distributed manner and stored across a network of edge devices, as is the case in federated learning \cite{FEDLEARN2016, mcmahan17fedavg, li2020federated, kairouz2019advances}, or collected centrally   in a data warehouse composed of a large collection of commodity clusters. In either scenario, communication among the workers is typically the bottleneck.

Motivated by the need for more efficient training methods in traditional distributed and emerging federated environments, we consider optimization problems of the form
\begin{align}
\min \limits_{x \in \R^d}  \sbr*{f(x) \eqdef \frac{1}{n} \sum \limits_{i=1}^n f_i(x) } \,, \label{eq:probR_biased_biased}
\end{align}
where  $x\in \R^d$ collects the parameters of a statistical  model to be trained, $n$ is the number of workers/devices, and $f_i(x)$ is the loss incurred by model $x$ on data stored on worker $i$. 

As discussed in the introduction, fundamental baseline for  solving problem  \eqref{eq:probR_biased_biased} is (distributed) gradient descent ({\tt GD}).
Due to the communication issues inherent to distributed systems, several enhancements to this baseline have been proposed that can better deal with the communication cost challenges of distributed environments, including
acceleration \cite{nesterov2013introductory, beck2009fista, allen2017katyusha}, reducing the number of iterations via momentum,
local methods \cite{mcmahan17fedavg,bayoumi2020tighter,karimireddy2020scaffold}, reducing the number of communication rounds via performing multiple local updates before each communication round, and
communication compression \cite{1bit, qsgd2017neurips,zhang2016zipml,lim20183lc, alistarh2018sparse, deepgradcompress2018iclr, sign_descent_2019}, reducing the size of communicated messages via compression operators.

\section{Contributions} In this work we contribute to a better  understanding of the latter approach to alleviating the communication bottleneck: {\em communication compression}. In particular, we study the theoretical properties of gradient-type methods which employ {\em biased} gradient compression operators, such as Top-$k$ sparsification~\cite{alistarh2018sparse}, or deterministic rounding~\cite{switchML}. Surprisingly, current
\footnote{Here we refer to the \href{https://arxiv.org/abs/2002.12410}{initial online appearance of this work} on February of 2020, after which several enhancements were developed. See Section \ref{sec:related-work} for more details.}
theoretical understanding of such methods is very limited. For instance, there is no general theory of such methods even in the $n=1$ case, only a handful of biased compression techniques have been proposed in the literature, we do not have any theoretical understanding of why biased compression operators could outperform their unbiased counterparts and when. More importantly, there is no good convergence theory for any gradient-type method with a biased compression in the crucial $n>1$ setting.

In this chapter, we address all of the above problems. In particular, our main contributions are:

\subsection{Three parametric classes}
We define and study three parametric classes of  biased compression operators (see Section~\ref{sec:biased_compression_operators}), which we denote $\mathbb{B}^1(\alpha,\beta)$, $\mathbb{B}^2(\gamma,\beta)$ and $\mathbb{B}^3(\delta)$, the first two of which are new. We prove that, despite they are alternative parameterization of the same collection of operators, the last two are more favorable than the first one, thus highlighting the importance of parametrization and providing further reductions. We show how is the commonly used class of unbiased compression operators, which we denote $\mathbb{U}(\zeta)$, relates to these biased classes. We also study scaling and compositions of such compressors.

\subsection{New compression operators}
We then proceed to give a long list of new and known  biased (and some unbiased) compression operators which belong to the above classes in Section~\ref{sec:examples}. A summary of all compressors considered can be found in Table~\ref{table:compressor-examples}.

\subsection{Convergence analysis}
In Section~\ref{sec:analysis_of_biased_GD} we analyze  compressed {\tt GD} in the $n=1$ case for compressors belonging to all three classes under smoothness and strong convexity assumption. Our theorems generalize existing results which hold for unbiased operators in a tight manner, and also recover the rate of {\tt GD} in this regime. Our linear convergence results are summarized in Table~\ref{table:iter_complexity}.

\begin{table}[t]
\begin{center}
\begin{tabular}{cccc}
\hline
Compressor & $\cC \in  \mathbb{B}^1(\alpha,\beta)$ & $\cC \in  \mathbb{B}^2(\gamma,\beta)$ & $ \cC\in  \mathbb{B}^3(\delta)$\\
Theorem & Theorem~\ref{thm:main-I} &  Theorem~\ref{thm:main-II} &  Theorem~\ref{thm:main-III} \\
Complexity & $\displaystyle \cO\left(\frac{\beta^2}{\alpha} \frac{L}{\mu} \log \frac{1}{\epsilon} \right)$ & $\displaystyle \cO\left(\frac{\beta}{\gamma} \frac{L}{\mu} \log \frac{1}{\epsilon}\right)$ & $\displaystyle \cO\left(\delta \frac{L}{\mu} \log \frac{1}{\epsilon}\right)$ \\
\hline
\end{tabular}
\end{center}
\caption{Complexity results for {\tt GD} with biased compression. The identity compressor $\cC(x)\equiv x$ belongs to all  classes with $\alpha=\beta=\gamma=\delta=1$; all three results recover standard rate of {\tt GD}.}
\label{table:iter_complexity}
\end{table}

\subsection{Superiority of biased compressors}
We ask the question: do biased compressors outperform their unbiased counterparts in theory, and by how much? We answer this question by studying the performance of compressors under various synthetic and empirical statistical assumptions on the distribution of the entries of gradient vectors which need to be compressed. Particularly, we quantify the gains of the Top-$k$ sparsifier when compared against the unbiased Rand-$k$ sparsifier in Section~\ref{sec:stat}.

\subsection{Distributed setting}
Finally, we study the important $n>1$ setting in Section~\ref{sec:distributed} and argue by giving a counterexample that a naive application of biased compression to distributed {\tt GD} might diverge.  We then show that distributed \texttt{SGD} method equipped with an error-feedback mechanism \cite{stich2020error} provably handles biased compressors. In our main result (Theorem~\ref{thm:sparsified}; also see Table~\ref{table:distributed_thm}) we consider three learning schedules and iterate averaging schemes to provide three distinct convergence rates. Our analysis provides the first convergence guarantee for distributed gradient-type method which provably converges for biased compressors, and we thus solve a major open problem in the literature. 

\begin{table}
\begin{center}
\begin{tabular}{ccc}
\hline
Stepsizes & Weights & Rate\\
\hline
$\cO(\frac{1}{k})$  & $\cO(k)$ & $\displaystyle \cO \left(\frac{A_1}{K^2}  +  \frac{A_2}{ K} \right)$  \\
$\cO(1)$ & $\cO(e^{-k})$ & $\displaystyle \tilde \cO \left( A_3 \exp \left[- \frac{K}{A_4} \right] + \frac{A_2}{ K}\right)$ \\
$\cO(1)$ & $1$ & $\displaystyle \cO \left(  \frac{A_3}{ K } + \frac{A_5}{\sqrt{K}}  \right)$ \\
\hline
\end{tabular}
\end{center}

\caption{Ergodic convergence of distributed \texttt{SGD} with biased compression and error-feedback (Algorithm~\ref{alg}) for $L$-smooth and $\mu$-strongly convex functions ($K$ communications). Details are given in Theorem~\ref{thm:sparsified}. }
\label{table:distributed_thm}
\end{table}

\section{Related work}\label{sec:related-work}

There has been extensive work related to communication compression, mostly focusing on unbiased compressions~\cite{qsgd2017neurips} as these are much easier to analyze. Works concerning biased compressions show strong empirical results with limited or no analysis~\cite{vogels2019powersgd,deepgradcompress2018iclr, sun}. There have been several attempts trying to address this issue, e.g., \cite{errorSGD} provided analysis for quadratics in distributed setting, \cite{zhao} gave analysis for momentum \texttt{SGD} with a specific  biased  compression, but under unreasonable assumptions, i.e., bounded gradient norm and memory. The first result that obtained linear rate of convergence for biased compression was done by \cite{karimireddy2019error}, but only for one node and under bounded gradient norm assumption, which was later overcome by \cite{stich2020error}.  

After the initial online appearance of our work~\cite{biased2020} that this chapter is based on, there has been several enhancements in the literature. In particular, \cite{AjallStich2021biased} developed theory for non-convex objectives in the single node setup, \cite{Gorbunov2020EF-SGD} designed a novel error compansated \texttt{SGD} algorithm converging linearly in a more relaxed setting with the help of additional unbiased compressor, \cite{horvath2021a} proposed a simple trick to convert any biased compressor to corresponding induced (unbiased) compressor leading to improved theoretical guarantees. Recently, a new variant of error feedback mechanism was introduced in \cite{EF21,EF21-ext} showing an improved rates for distributed non-convex problems.

\section{Biased compressors}\label{sec:biased_compression_operators}

By compression operator we mean a (possibly random) mapping $\cC\colon\R^d\to\R^d$ with some constraints.
Typically, literature considers {\em unbiased} compression operators $\cC$ with a bounded second moment as in Definition~\ref{def:omegaquant}.



\subsection{Three classes of biased compressors}

We instead focus on understanding {\em biased} compression operators, or ``compressors'' in short. We now introduce three classes of biased compressors, the first two are new, which can be seen as natural extensions of unbiased compressors.

\begin{definition}\label{def:comp_I}
We say that $\cC\in \mathbb{B}^1(\alpha,\beta)$ for some $\alpha,\beta>0$ if
\begin{equation} \label{eq:alpha-beta} \alpha \norm*{x}^2 \leq \Exp{ \norm*{\cC(x)}^2 } \leq \beta  \langle \Exp{\cC(x)},x \rangle, \qquad \forall x\in\R^d \; .
\end{equation}
\end{definition}

As we shall show next, the second inequality in \eqref{eq:alpha-beta} implies $\Exp{ \norm*{\cC(x)}^2 } \leq \beta^2 \norm*{x}^2$. 
\begin{lemma}\label{lem:second_ineq_implies} For any $x\in \R^d$, if  $\Exp{ \norm*{\cC(x)}^2 } \leq  \beta \langle  \Exp{\cC(x)}, x \rangle$, then\begin{equation}\label{eq:noiuhfu93hufbuf}\Exp{ \norm*{\cC(x)}^2 } \leq \beta^2\norm*{x}^2.\end{equation} 
\end{lemma}
\begin{proof}
 Fix any $x\in \R^d$.  Applying Jensen's inequality, the second inequality in \eqref{eq:alpha-beta} and Cauchy-Schwarz, we get 
\begin{equation}\label{eq:bd8b38bdhbvhbyf}  \norm*{ \Exp{\cC(x)}}^2   \leq \Exp{ \norm*{\cC(x)}^2 } \overset{\eqref{eq:alpha-beta}}{\leq} \beta \langle  \Exp{\cC(x)}, x \rangle  \leq \beta   \norm*{\Exp{ \cC(x) }} \norm*{x}  .
\end{equation}
If $\Exp{ \cC(x) } \neq 0$, this implies
$\norm*{\Exp{\cC(x)}} \leq \beta \norm*{x}$. Plugging this back into \eqref{eq:bd8b38bdhbvhbyf}, we get \eqref{eq:noiuhfu93hufbuf}. If $\Exp{ \cC(x) } = 0$, then from \eqref{eq:alpha-beta} we see that $\Exp{ \norm*{\cC(x)}^2 }=0$, and \eqref{eq:noiuhfu93hufbuf} holds trivially.
\end{proof}

In the second class, we require the inner product between uncompressed $x$ and compressed $\cC(x)$ vectors to dominate the squared norms of both vectors in expectation.

\begin{definition}\label{def:comp_II}
We say that $\cC\in \mathbb{B}^2(\gamma,\beta)$ for some $\gamma,\beta>0$ if
\begin{equation} \label{eq:alpha-betaII}
\max\left\{ \gamma \norm*{x}^2 , \tfrac{1}{\beta} \Exp{\norm*{\cC(x)}^2}\right\} \leq \lin{\Exp{\cC(x)},x } \qquad \forall x\in\R^d \,.
\end{equation}
\end{definition}

Finally, in the third class, we consider standard definition of biased compression \cite{stich2018sparsified, cordonnier2018convex} as in Definition~\ref{def:deltaquant}. We denote this as $\mathbb{B}^3(\delta)$.



All three definitions require the compressed vector $\cC(x)$ to be in the neighborhood of the uncompressed vector $x$ so that initial information is preserved with some accuracy. We now establish several basic properties and connections between the classes. We first show that the three classes of biased compressors defined above are equivalent in the following sense: a compressor from any of those three classes can be shown to belong to all three classes with different parameters and after possible scaling.

\begin{theorem}[Equivalence between biased compressors]\label{thm:compression_properties}
Let $\lambda>0$ be a free scaling parameter.
\begin{enumerate}
\item If $\cC\in \mathbb{B}^1(\alpha,\beta)$, then 
\begin{itemize}
\item[(i)] $\beta^2\geq \alpha$ and $\lambda \cC \in \mathbb{B}^1(\lambda^2 \alpha  , \lambda\beta )$, and
\item[(ii)] $\cC\in \mathbb{B}^2(\alpha,\beta^2)$ and $ \frac{1}{\beta} \cC \in \mathbb{B}^3(\beta^2/\alpha) $.
\end{itemize}

\item If $\cC \in \mathbb{B}^2(\gamma, \beta)$, then 
\begin{itemize}
\item[(i)] $\beta \geq \gamma$ and $\lambda \cC \in \mathbb{B}^2(\lambda \gamma, \lambda \beta)$, and
\item[(ii)]  $\cC \in \mathbb{B}^1(\gamma^2, \beta)$ and $\frac{1}{\beta}\cC \in \mathbb{B}^3(\beta/\gamma)$.
\end{itemize}

\item If $\cC \in \mathbb{B}^3(\delta)$, then 
\begin{itemize}
\item[(i)] $\delta \geq 1$, and 
\item[(ii)]  $\cC \in \mathbb{B}^2\left(\frac{1}{2\delta}, 2\right)\subseteq \mathbb{B}^1\left(\frac{1}{4\delta^2}, 2\right)$.
\end{itemize}
\end{enumerate}
\end{theorem}

\begin{proof} 
Let us prove this implications for each class separately.
\begin{enumerate}
\item Case $\cC\in \mathbb{B}^1(\alpha,\beta)$:
\begin{itemize}
\item[(i)] Let us choose any $x\neq 0$ and observe that \eqref{eq:alpha-beta} implies that $\Exp{\cC(x)}\neq 0$. Further, from \eqref{eq:alpha-beta} we get the bounds
$$\frac{\Exp{\norm*{\cC(x)}^2}}{\langle \Exp{\cC(x)}, x\rangle } \leq \beta, \qquad \alpha \leq  \frac{\Exp{\norm*{\cC(x)}^2}}{\norm*{x}^2}.$$
Finally, \[
\beta^2 \geq \left(\frac{\Exp{\norm*{\cC(x)}^2}}{\langle \Exp{\cC(x)}, x\rangle }\right)^2  \geq \frac{\Exp{\norm*{\cC(x)}^2} \Exp{\norm*{\cC(x)}^2}}{ \norm*{ \Exp{\cC(x)}}^2 \norm*{x}^2} \geq \alpha \frac{\Exp{\norm*{\cC(x)}^2}}{ \norm*{ \Exp{\cC(x)}}^2 } \geq \alpha,
\]
where the second inequality is due to Cauchy-Schwarz, and the last inequality follows by applying Jensen inequality.

The scaling property $\lambda \cC \in \mathbb{B}^1(\alpha \lambda^2 , \beta \lambda)$ follows directly from \eqref{eq:alpha-beta}.

\item[(ii)] In view of (i), $\lambda \cC \in \mathbb{B}^1(\lambda^2 \alpha  ,  \lambda \beta)$. If we choose $\lambda \leq \frac{2}{\beta}$, then
\begin{eqnarray*}
\Exp{\norm*{\lambda \cC(x)-x}^2} &=& \Exp{\norm*{\lambda \cC(x)}^2} - 2\lin{ \Exp{\lambda \cC(x)}, x } + \norm*{x}^2 \\
& \overset{\eqref{eq:alpha-beta} }{ \leq } & (\beta \lambda - 2) \lin{ \Exp{\lambda \cC(x)}, x } + \norm*{x}^2\\
& \overset{\eqref{eq:alpha-beta} }{ \leq } & (\beta \lambda - 2) \frac{\alpha \lambda^2}{\beta \lambda}\norm*{x}^2 + \norm*{x}^2\\
& \overset{\eqref{eq:alpha-beta} }{ \leq } & \left( \alpha \lambda^2 - 2\frac{\alpha}{\beta} \lambda + 1\right) \norm*{x}^2.
\end{eqnarray*}
Minimizing the above expression in $\lambda$, we get $\lambda=\frac{1}{\beta}$, and the result follows.

\end{itemize}

\item Case $\cC \in \mathbb{B}^2(\gamma, \beta)$. 

\begin{itemize}
\item[(i)] Using \eqref{eq:alpha-betaII} we get 
\begin{align*}
\gamma \leq \frac{\dotprod{\Exp{\cC(x)}}{x}}{\norm*{x}^2} \leq \frac{\Exp{\norm*{\cC(x)}^2}}{\sqrt{\Exp{\norm*{\cC(x)}^2}\norm*{x}^2}} \leq \beta \frac{\dotprod{\Exp{\cC(x)}}{x}}{\sqrt{\Exp{\norm*{\cC(x)}^2}\norm*{x}^2}} \leq \beta ,
\end{align*}
where the first and third inequalities follow from \eqref{eq:alpha-betaII} and the third and the last from Cauchy-Schwarz inequality with Jensen inequality.

The scaling property $\lambda \cC \in \mathbb{B}^2(\lambda \gamma, \lambda \beta)$ follows directly from \eqref{eq:alpha-betaII}.

\item[(ii)]  If $\cC \in \mathbb{B}^2(\gamma, \beta)$, then 
$
\Exp{ \norm*{\cC(x)}^2 } \leq \beta  \langle \Exp{\cC(x)},x \rangle
$
and
\[
\gamma^2 \norm*{x}^4 \overset{\eqref{eq:alpha-betaII}}{\leq} \langle \Exp{\cC(x)},x \rangle^2 \leq   \norm*{\Exp{\cC(x)}}^2\norm*{x}^2   \leq  \Exp{\norm*{\cC(x)}^2} \norm*{x}^2,
\]
where the second inequality is Cauchy-Schwarz, and the third is Jensen. Therefore, $\cC \in \mathbb{B}^1(\gamma^2, \beta)$.

Further, for any $\lambda >0$, we get
\begin{eqnarray*}
\Exp{\norm*{\lambda \cC(x)-x}^2} &=& \Exp{\norm*{\lambda \cC(x)}^2} - 2\lin{ \Exp{\lambda \cC(x)}, x } + \norm*{x}^2 \\
 &=&\lambda^2  \Exp{\norm*{ \cC(x)}^2} - 2\lambda \lin{ \Exp{ \cC(x)}, x } + \norm*{x}^2 \\
& \overset{\eqref{eq:alpha-betaII} }{ \leq } & (\lambda \beta - 2) \lambda \lin{ \Exp{ \cC(x)}, x } + \norm*{x}^2.
\end{eqnarray*}
If we choose $\lambda = \frac{1}{\beta}$, then we can continue as follows:
\begin{eqnarray*}
\Exp{\norm*{\lambda \cC(x)-x}^2}  & \leq  & -\frac{1}{\beta} \lin{ \Exp{ \cC(x)}, x } + \norm*{x}^2\\
& \overset{\eqref{eq:alpha-betaII} }{ \leq } & \left(  1-\frac{\gamma}{\beta}\right) \norm*{x}^2,
\end{eqnarray*}
whence $\frac{1}{\beta}\cC \in \mathbb{B}^3(\beta/\gamma)$.

\end{itemize}


\item Case $\cC \in \mathbb{B}^3(\delta)$. 

\begin{itemize}
\item[(i)] Pick $x\neq 0$. Since $0\leq \Exp{\norm*{\cC(x)-x}^2} \leq \left(1-\frac{1}{\delta}\right)\norm*{x}^2$ and we assume $\delta>0$, we must necessarily have $\delta\geq 1$.

\item[(ii)] If $\cC \in \mathbb{B}^3(\delta)$ then 
\begin{equation*}
 \Exp{ \norm*{\cC(x)}^2 } - 2 \dotprod{\Exp{\cC(x)}}{x} + \frac{1}{\delta}\norm*{x}^2 \leq 0,
\end{equation*}
which implies  that
\begin{equation*}
\frac{1}{2\delta}\norm*{x}^2 \leq  \dotprod{\Exp{\cC(x)}}{x} \qquad \text{and} \qquad  \Exp{ \norm*{\cC(x)}^2} \leq  2 \dotprod{\Exp{\cC(x)}}{x}.
\end{equation*}
Therefore, $\cC \in \mathbb{B}^2\left(\frac{1}{2\delta}, 2\right)\subseteq \mathbb{B}^1\left(\frac{1}{4\delta^2}, 2\right)$.
\end{itemize}
\end{enumerate}

\end{proof}

Next, we show that, with a proper scaling, any unbiased compressor also belongs to all the three classes of biased compressors.

\begin{theorem}[From unbiased to biased with scaling]\label{thm:unbiased_to_biased}
If $\cC\in \mathbb{U}(\zeta)$, then for the scaled operator $\lambda\cC$ we have
\begin{itemize}
\item[(i)] $\lambda \cC\in \mathbb{B}^1(\lambda^2,\lambda \zeta )$ for $\lambda>0$,
\item[(ii)] $\lambda \cC\in \mathbb{B}^2(\lambda,\lambda \zeta )$ for $\lambda>0$, 
\item[(iii)] $\lambda \cC\in \mathbb{B}^3\left(\frac{1}{\lambda(2 - \zeta \lambda)}\right)$ for $\zeta \lambda < 2$.
\end{itemize}
\end{theorem}

\begin{proof} 
Let $\cC\in \mathbb{U}(\zeta)$.

\begin{itemize}


\item [(i)] Given any $\lambda>0$, consider the scaled operator $\lambda \cC$. We have
\begin{eqnarray*} \lambda^2 \norm*{x}^2 = \norm*{\Exp{\lambda \cC(x)}}^2  \leq    \Exp{\norm*{\lambda \cC(x)}^2}  \leq \lambda^2 \zeta \norm*{x}^2  =  \lambda \zeta \langle \Exp{\lambda \cC(x)},x \rangle,
\end{eqnarray*}
whence $\cC\in \mathbb{B}^1(\lambda^2,\lambda \zeta)$.

\item [(ii)] Given any $\lambda>0$, consider the scaled operator $\lambda \cC$. We have
\begin{eqnarray*} \lambda \norm*{x}^2 &=& \langle \Exp{\lambda \cC(x)},x \rangle, \\
\Exp{\norm*{\lambda \cC(x)}^2} 
& \leq & \lambda^2 \zeta \norm*{x}^2 = \lambda \zeta \langle \Exp{\lambda \cC(x)},x \rangle,
\end{eqnarray*}
whence $\lambda \cC\in \mathbb{B}^2(\lambda,\lambda \zeta)$.

\item [(iii)] Given $\lambda>0$ such that $\lambda \zeta < 2$, consider the scaled operator $\lambda \cC$. We have
\begin{eqnarray*} 
\Exp{\norm*{\lambda \cC(x)- x}^2} &=& \Exp{\norm*{\lambda \cC(x)}^2} - 2 \lin{\Exp{\lambda \cC(x)}, x} + \norm*{x}^2 \\
&\leq& (\zeta \lambda^2 - 2 \lambda + 1)\norm*{x}^2   
\end{eqnarray*}
whence $\lambda \cC\in \mathbb{B}^3\left(\frac{1}{\lambda(2 - \zeta \lambda)}\right)$.

\end{itemize}

\end{proof}

 \subsection{Examples of biased compressors: old and new} \label{sec:examples}

We  now give some examples of compression operators belonging to the classes $\mathbb{B}^1$,  $\mathbb{B}^2$,  $\mathbb{B}^3$ and  $\mathbb{U}$. Several of them are new. For a summary, refer to Table~\ref{table:compressor-examples}.

\begin{table*}[!t]
\resizebox{\columnwidth}{!}{
\begin{tabular}{lcccccc}
\hline
Compression Operator $\cC$    &  Unbiased? &  $\alpha$ & $\beta$ & $\gamma$ & $\delta$ & $\zeta$\\
\hline 
Unbiased random sparsification  & \cmark  &  &  &  &  & $\nicefrac{d}{k}$ \\ 
Biased random sparsification  {\bf [NEW]} & \xmark  & $q$ & $1$ & $q$ & $\nicefrac{1}{q}$ & \\ 
Adaptive random sparsification {\bf [NEW]} & \xmark & $\nicefrac{1}{d}$ & $1$ & $\nicefrac{1}{d}$ & $d$ & \\
Top-$k$ sparsification \cite{alistarh2018sparse}  & \xmark & $\nicefrac{k}{d}$ & $1$ & $\nicefrac{k}{d}$ & $\nicefrac{d}{k}$ & \\
General unbiased rounding {\bf [NEW]} & \cmark & & & & & 
$\tfrac{1}{4} \sup\left(\tfrac{a_k}{a_{k+1}} + \tfrac{a_{k+1}}{a_k} + 2\right)$ \\
Unbiased exponential rounding {\bf [NEW]} & \cmark & & & & & $\tfrac{1}{4}\left( b+\frac{1}{b}+2\right)$ \\
Biased exponential rounding {\bf [NEW]} & \xmark  & $\left(\frac{2}{b+1}\right)^2$ & $\frac{2b}{b+1}$ & $\frac{2}{b+1}$ & $\frac{(b+1)^2}{4b}$ & \\
Natural compression \cite{Cnat} & \cmark & & & & & $\nicefrac{9}{8}$ \\
General exponential dithering {\bf [NEW]} & \cmark & & & & & $\zeta_b$ \\
Natural dithering  \cite{Cnat} & \cmark & & & & & $\zeta_2$ \\
Top-$k$ + exponential dithering  {\bf [NEW]} & \xmark & $\nicefrac{k}{d}$ & $\zeta_b$ & $\nicefrac{k}{d}$ & $\zeta_b\nicefrac{d}{k}$ &  \\
\hline
\end{tabular}
}
\caption{Compressors $\cC$ described in Section~\ref{sec:examples} and their membership in  $\mathbb{B}^1(\alpha,\beta)$, $\mathbb{B}^2(\gamma,\beta)$,  $\mathbb{B}^3(\delta)$ and $\mathbb{U}(\zeta)$.}
\label{table:compressor-examples}
\end{table*}

\begin{itemize}

\item[(a)] For $k \in [d]\eqdef \{1,\dots,d\}$, the {\bf unbiased random  (aka Rand-$k$) sparsification} operator is defined via
\begin{equation}\label{ex:ur-sparse}
\cC(x) \eqdef \frac{d}{k}\sum \limits_{i\in S}x_ie_i,
\end{equation}
where $S\subseteq [d]$ is the $k$-nice sampling; i.e., a subset of $[d]$ of cardinality $k$ chosen uniformly at random, and $e_1,\dots,e_d$ are the standard unit basis vectors in $\R^d$.

\begin{lemma}\label{lem-ex:ur-sparse}
The Rand-$k$ sparsifier (\ref{ex:ur-sparse}) belongs to $\U(\tfrac{d}{k})$.
\end{lemma}

\item[(b)] Let $S\subseteq [d]$ be a random set, with probability vector $p\eqdef (p_1,\dots,p_d)$, where  $p_i \eqdef \Prob(i\in S)>0$ for all $i$ (such a set is called a proper sampling~\cite{PCDM}). Define {\bf biased random sparsification} operator via
\begin{equation}\label{ex:br-sparse}
 \cC(x) \eqdef \sum \limits_{i\in S} x_i e_i.
\end{equation}

\begin{lemma}\label{lem-ex:br-sparse}
Letting $q \eqdef \min_i p_i$, the biased random sparsification operator (\ref{ex:br-sparse}) belongs to $\mathbb{B}^1(q, 1)$, $\mathbb{B}^2(q, 1)$, $\mathbb{B}^3(\nicefrac{1}{q})$.
\end{lemma}

\item[(c)] {\bf Adaptive random sparsification} is defined via
\begin{equation}\label{ex:ar-sparse}
\cC(x) \eqdef x_i e_i \quad \text{ with probability } \quad \frac{|x_i|}{\onenorm{x}}.
\end{equation}

\begin{lemma}\label{lem-ex:ar-sparse}
Adaptive random sparsification operator (\ref{ex:ar-sparse}) belongs to \\ $\mathbb{B}^1(\frac{1}{d},1)$, $\mathbb{B}^2(\frac{1}{d},1)$, $\mathbb{B}^3(d)$.
\end{lemma}

\item[(d)] {\bf Greedy (aka Top-$k$) sparsification} operator is defined via
\begin{equation}\label{ex:top-sparse}
\cC(x) \eqdef \sum \limits_{i=d-k+1}^d x_{(i)} e_{(i)},
\end{equation}
where coordinates are ordered by their magnitudes so that $\abs{x_{(1)}} \leq \abs{x_{(2)}} \leq \cdots \leq \abs{x_{(d)}}$.

\begin{lemma}\label{lem-ex:top-sparse}
Top-$k$ sparsification operator (\ref{ex:top-sparse}) belongs to $\mathbb{B}^1(\frac{k}{d},1)$, \\ $\mathbb{B}^2(\frac{k}{d},1)$, and $\mathbb{B}^3(\frac{d}{k})$.
\end{lemma}

\item[(e)] Let $(a_k)_{k\in\Z}$ be an arbitrary increasing sequence of positive numbers such that $\inf a_k = 0$ and $\sup a_k = \infty$. Then {\bf general unbiased rounding} $\cC$ is defined as follows: if $a_k\le |x_i|\le a_{k+1}$ for some coordinate $i\in[d]$, then
\begin{equation}\label{ex:gu-rounding}
\cC(x)_i = 
\begin{cases}
    \sign(x_i)a_k     & \text{ with probability } \quad \frac{a_{k+1}-|x_i|}{a_{k+1}-a_k}\\ 
    \sign(x_i)a_{k+1} & \text{ with probability } \quad \frac{|x_i| - a_{k}}{a_{k+1}-a_k}\\
\end{cases}
\end{equation}

\begin{lemma}\label{lem-ex:gu-rounding}
General unbiased rounding operator (\ref{ex:gu-rounding}) belongs to $\U(\zeta)$, where
$$
\zeta = \frac{1}{4} \sup_{k\in\Z}\left(\frac{a_k}{a_{k+1}} + \frac{a_{k+1}}{a_k} + 2\right).
$$
\end{lemma}

Notice that $\zeta$ is minimizing for exponential roundings $a_k=b^k$ with some basis $b>1$, in which case $\zeta = \tfrac{1}{4}\left(b+\nicefrac{1}{b}+2\right)$.

\item[(f)] Let $(a_k)_{k\in\Z}$ be an arbitrary increasing sequence of positive numbers such that $\inf a_k = 0$ and $\sup a_k = \infty$. Then {\bf general biased rounding} is defined via
\begin{equation}\label{ex:gb-rounding}
\cC(x)_i \eqdef \sign(x_i)\arg \min \limits_{t\in(a_k)} |t-|x_i||,\qquad i\in[d].
\end{equation}

\begin{lemma}\label{lem-ex:gb-rounding}
General biased rounding operator (\ref{ex:gb-rounding}) belongs to $\B^1(\alpha, \beta)$, $\B^2(\gamma, \beta)$, and $\B^3(\delta)$, where $$\beta = \sup_{k\in\Z}\frac{2a_{k+1}}{a_k+a_{k+1}}, \quad \gamma = \inf_{k\in\Z}\frac{2a_k}{a_k+a_{k+1}}, \quad \alpha = \gamma^2, \quad \delta = \sup_{k\in\Z}\frac{\left(a_k + a_{k+1}\right)^2}{4a_k a_{k+1}}.$$
\end{lemma}

In the special case of exponential rounding $a_k=b^k$ with some base $b>1$, we get
$$\alpha = \left(\frac{2}{b+1}\right)^2, \quad \beta = \frac{2b}{b+1}, \quad \gamma = \frac{2}{b+1}, \quad \delta = \frac{(b+1)^2}{4b}.$$

\begin{remark}
Plugging these parameters into the iteration complexities of Table~\ref{table:iter_complexity}, we find that the class $\B^3$ gives the best iteration complexity as $\tfrac{\beta^2}{\alpha} = b^2 > \tfrac{\beta}{\gamma} = b > \delta = \tfrac{(b+1)^2}{4b}$.
\end{remark}

\item[(g)] {\bf  Natural compression} operator $\cC_{nat}$ of \cite{Cnat} is the special case of general unbiased rounding operator (\ref{ex:gu-rounding}) when $b=2$. So, $$\cC_{nat}\in\U\left(\frac{9}{8}\right).$$

\item[(h)] For $b>1$, define {\bf general exponential dithering} operator with respect to $l_p$-norm and with $s$ exponential levels $0<b^{1-s}<b^{2-s}<\dots<b^{-1}<1$ via
\begin{equation}\label{ex:ge-dithering}
\cC(x) \eqdef \|x\|_p \times \sign(x) \times \xi\left(\frac{|x_i|}{\|x\|_p}\right),
\end{equation}
where the random variable $\xi(t)$ for $t\in[b^{-u-1},b^{-u}]$ is set to either $b^{-u-1}$ or $b^{-u}$ with probabilities proportional to $b^{-u}-t$ and $t-b^{-u-1}$, respectively.

\begin{lemma}\label{lem-ex:ge-dithering}
General exponential dithering operator (\ref{ex:ge-dithering}) belongs to $\U(\zeta_b)$, where, letting $r = \min(p,2)$,
\begin{equation}\label{ge-dithering-zeta}
\zeta_b = \frac{1}{4}\left(b+\frac{1}{b}+2\right) + d^{\frac{1}{r}}b^{1-s}\min(1, d^{\frac{1}{r}}b^{1-s}).
\end{equation}
\end{lemma}

\item[(i)] {\bf  Natural dithering} introduced by \cite{Cnat} without norm compression is the spacial case of general exponential dithering (\ref{ex:ge-dithering}) when $b=2$.

\item[(j)] {\bf Top-$k$ combined with exponential dithering.} Let $\cC_{top}$ be the Top-$k$ sparsification operator (\ref{ex:top-sparse}) and $\cC_{dith}$ be general exponential dithering operator (\ref{ex:ge-dithering}) with some base $b>1$ and parameter $\zeta_b$ from (\ref{ge-dithering-zeta}). Define a new compression operator as the composition of these two:
\begin{equation}\label{ex:top-gendith}
\cC(x) \eqdef \cC_{dith}\left(\cC_{top}(x)\right).
\end{equation}

\begin{lemma}\label{lem-ex:top-gendith}
The composition operator \eqref{ex:top-gendith} of Top-$k$ sparsification and exponential dithering with base $b$ belongs to $\B^1(\tfrac{k}{d}, \zeta_b)$, $\B^2(\tfrac{k}{d}, \zeta_b)$, $\B^3(\tfrac{d}{k}\zeta_b)$, where $\zeta_b$ is as in \eqref{ge-dithering-zeta}.
\end{lemma}

\end{itemize}

\section{Gradient descent with biased compression}\label{sec:analysis_of_biased_GD}

As we discussed in previous section, compression operators can have different equivalent parametrizations. Next, we aim to investigate the influence of those parametrizations on the theoretical convergence rate of an algorithm employing compression operators. To achieve clearer understanding of the interaction of compressor parametrization and convergence rate, we first consider the single node, unconstrained optimization problem 
$$ \min_{x\in \R^d} f(x),$$
where $f:\R^d\to \R$ is $L$-smooth and $\mu$-strongly convex. We study the method
\begin{equation} \label{eq:CGD} \tag{\texttt{CGD}} \boxed{x^{k+1} =x^k -\stepsize \cC^k(\nabla f(x^k))}\, ,\end{equation}
where $\cC^k:\R^d\to \R^d$ are (potentially biased) compression operators belonging to one of the classes $\mathbb{B}^1$, $\mathbb{B}^2$ and $\mathbb{B}^3$ studied in the previous sections, and $\eta>0$ is a stepsize. We refer to this method as  \texttt{CGD}: Compressed Gradient Descent.
 
\subsection{Complexity theory} We now establish three theorems, performing complexity analysis for each of the three classes $\mathbb{B}^1$, $\mathbb{B}^2$ and $\mathbb{B}^3$ individually. Let $\cE_k\eqdef \Exp{f(x^k)} - f(x^\star)$, with $\cE_0 = f(x^0) - f(x^\star)$.
  
\begin{theorem}\label{thm:main-I} Let $\cC\in \mathbb{B}^1(\alpha,\beta)$. Then as long as $0\leq \stepsize \leq \frac{2}{\beta L}$, we have
$\cE_k \leq  \left(1- \frac{\alpha}{\beta} \stepsize \mu ( 2 - \stepsize \beta  L ) \right)^k \cE_0.$
If we choose $\stepsize = \frac{1}{\beta L}$, then
$$\cE_k\leq  \left(1- \frac{\alpha}{\beta^2}\frac{\mu}{ L}  \right)^k \cE_0.$$
\end{theorem}

\begin{theorem}\label{thm:main-II}  Let $\cC\in \mathbb{B}^2(\gamma,\beta)$. Then as long as $0\leq \stepsize \leq \frac{2}{\beta L}$, we have
$\cE_k \leq  \left(1-  \gamma \stepsize \left(2 -\stepsize \beta)  L \right) \right)^k \cE_0.$
If we choose $\stepsize = \frac{1}{\beta L}$, then
$$\cE_k \leq  \left(1- \frac{\gamma}{\beta }\frac{\mu}{ L}  \right)^k \cE_0.$$
\end{theorem}

\begin{theorem}\label{thm:main-III} Let $ \cC\in \mathbb{B}^3(\delta)$. Then as long as  $0\leq \stepsize \leq \frac{1}{L}$, we have
$\cE_k \leq  \left(1- \frac{1}{\delta} \stepsize \mu \right)^k \cE_0.$ If we choose $\eta=\frac{1}{L}$, then $$\cE_k \leq  \left(1- \frac{1}{\delta} \frac{\mu}{L}  \right)^k \cE_0.$$
\end{theorem}

The iteration complexity for these results can be found in Table~\ref{table:iter_complexity}. Note that the identity compressor $\cC(x)\equiv x$ belongs to $ \mathbb{B}^1(1,1), \mathbb{B}^2(1,1)$, and  $\mathbb{B}^3(1)$, hence all these result exactly recover the rate of {\tt GD}. In the first two theorems, scaling the compressor by a positive scalar $\lambda >0$ does not influence the rate (see Theorem~\ref{thm:compression_properties}).

\subsection{$\mathbb{B}^3$ and $\mathbb{B}^2$ are better than $\mathbb{B}^1$} In light of the results above, we make the following observation. If $\cC\in \mathbb{B}^1(\alpha,\beta)$, then by Theorem~\ref{thm:compression_properties}, $\frac{1}{\beta}\cC\in \mathbb{B}^3(\frac{\beta^2}{\alpha})$. Applying Theorem~\ref{thm:main-III}, we get the  bound $\cO \left(\frac{\beta^2}{\alpha}\frac{L}{\mu}\log \frac{1}{\varepsilon} \right) $. This is the same result as that obtained by Theorem~\ref{thm:main-I}. On the other hand, if $ \cC\in \mathbb{B}^3(\delta)$, then by Theorem~\ref{thm:compression_properties}, $ \cC\in \mathbb{B}^1(\frac{1}{4\delta^2},2)$. Applying Theorem~\ref{thm:main-I}, we get the  bound $\cO \left(16 \delta^2 \frac{L}{\mu}\log \frac{1}{\varepsilon} \right) $.  This is a worse result than what Theorem~\ref{thm:main-III} offers by a factor of $16 \delta$. Hence, while $\mathbb{B}^1$ and $\mathbb{B}^3$ describe the same classes of compressors, for the purposes of \texttt{CGD} it is better to parameterize them as members of $\mathbb{B}^3$.

\section{Superiority of biased compressors under statistical assumptions} \label{sec:stat}

Here we highlight some advantages of biased compressors by comparing them with their unbiased cousins. We evaluate compressors by their average capacity of preserving the gradient information or, in other words, by expected approximation error they produce. In the sequel, we assume that gradients have i.i.d.\ coordinates drawn from some distribution.

\subsection{Top-$k$ vs Rand-$k$}
We now compare two  sparsification operators: Rand-$k$  \eqref{ex:ur-sparse} which is unbiased and which we denote as $\cC_{rnd}^k$, and Top-$k$  \eqref{ex:top-sparse} which is biased and which we denote as $\cC_{top}^k$. We define variance of the approximation error of $x$ via
$$\omega_{rnd}^k(x) \eqdef \Exp{\norm*{\frac{k}{d}\cC_{rnd}^k(x)-x}^2} = \left(1-\frac{k}{d}\right)\norm*{x}^2 $$ and
$$\omega_{top}^k(x) \eqdef   \norm*{\cC_{top}^k(x)-x}^2 = \sum_{i=1}^{d-k} x_{(i)}^2$$ and the energy ``saving''  via
$$s_{rnd}^k(x) \eqdef \norm*{x}^2 - \omega_{rnd}^k(x) = \Exp{\norm*{\frac{k}{d}\cC_{rnd}^k(x)}^2} = \frac{k}{d}\norm*{x}^2$$ and
$$s_{top}^k(x) \eqdef \norm*{x}^2 - \omega_{top}^k(x) = \norm*{\cC_{top}^k(x)}^2 = \sum_{i=d-k+1}^d x_{(i)}^2.$$ 

Expectations in these expressions are taken with respect to the randomization of the compression operator rather than input vector $x$.
Clearly, there exists $x$ for which these two operators incur identical variance, e.g. $x_1=\dots=x_d$.  However,  in practice we apply compression to gradients $x$ which evolve in time, and which may have heterogeneous components. In such situations, $\omega_{top}^k(x)$ could be much smaller than $\omega_{rnd}^k(x)$. This motivates a {\em quantitative study} of the {\em average  case} behavior in which we make an {\em  assumption} on the distribution of the coordinates of the compressed vector.

\paragraph{Uniform and exponential distribution.}
We first consider the case of uniform and exponentially distributed entries, and quantify the difference.
\begin{lemma}\label{lem:stat-top-random}
Assume the coordinates of $x\in \R^d$ are i.i.d.\ \\ (a) If they follow uniform distribution over $[0,1]$, then
\begin{equation*}
\frac{ \Exp{\omega^k_{top}} }{ \Exp{\omega^k_{rnd}} } = \left(1-\frac{k}{d+1}\right)\left(1-\frac{k}{d+2}\right),\qquad \frac{ \Exp{s^1_{top}} }{ \Exp{s^1_{rnd}} } = \frac{3d}{d+2}.
\end{equation*}
(b) If they follow standard exponential distribution, then
$$
 \frac{ \Exp{s^1_{top}} }{ \Exp{s^1_{rnd}} } = \frac{1}{2}\sum\limits_{i=1}^d \frac{1}{i^2} + \frac{1}{2}\left( \sum\limits_{i=1}^d \frac{1}{i}\right)^2 = \cO(\log^2 d).
$$
\end{lemma}

\renewcommand{\arraystretch}{1.8} 
\renewcommand{\tabcolsep}{0.1cm}   
\begin{table}[th]
\centering
\begin{tabular}{|c|c|c|c|c|c|c|c|c|}
\hline
\multicolumn{1}{@{\setlength{\arrayrulewidth}{1pt}\vline}c@{\setlength{\arrayrulewidth}{0.5pt}\vline}}{} & \multicolumn{4}{@{\setlength{\arrayrulewidth}{0.5pt}\vline}c@{\setlength{\arrayrulewidth}{0.5pt}\vline}}{\small Top-3} & \multicolumn{4}{@{\setlength{\arrayrulewidth}{0.5pt}\vline}c@{\setlength{\arrayrulewidth}{1pt}\vline}}{\small Top-5} \\ \hline
\small$d$ & $10^2$  & $10^3$  & $10^4$  & $10^5$ & $10^2$  & $10^3$  & $10^4$ & $10^5$  \\ \hline
\small$\mathcal{N}(0;1)$ & \multicolumn{4}{|c|}{$3 \cdot (\sigma^2 + \mu^2) = 3$}  & \multicolumn{4}{|c|}{$5 \cdot (\sigma^2 + \mu^2) = 5$}\\ \hline
$\Exp{s_{top}^k(x)}$ & 18.65  & 31.10  & 43.98  & 57.08  &  27.14 & 47.70  & 69.07  & 90.85\\ \hline
\small$\mathcal{N}(2;1)$ & \multicolumn{4}{|c|}{$3 \cdot (\sigma^2 + \mu^2) = 15$}  & \multicolumn{4}{|c|}{$5 \cdot (\sigma^2 + \mu^2) = 25$}\\ \hline
$\Exp{s_{top}^k(x)}$   & 53.45  &  75.27 &  95.81 & 115.53  & 81.60  & 118.56  & 153.13  & 186.22  \\ \hline
\end{tabular}
\caption{Information savings of greedy and random sparsifiers for $k=3$ and $k=5$.}
\label{normtable_main}
\end{table}

\paragraph{Empirical comparison.} Now we compare these two sparsification methods on an empirical bases and show the significant advantage of greedy sparsifier against random sparsifier. We assume that coordinates of to-be-compressed vector are i.i.d.\ Gaussian random variables.

\begin{figure}[t]
\vskip 0.2in
\begin{center}
\centerline{\includegraphics[width=0.5\columnwidth]{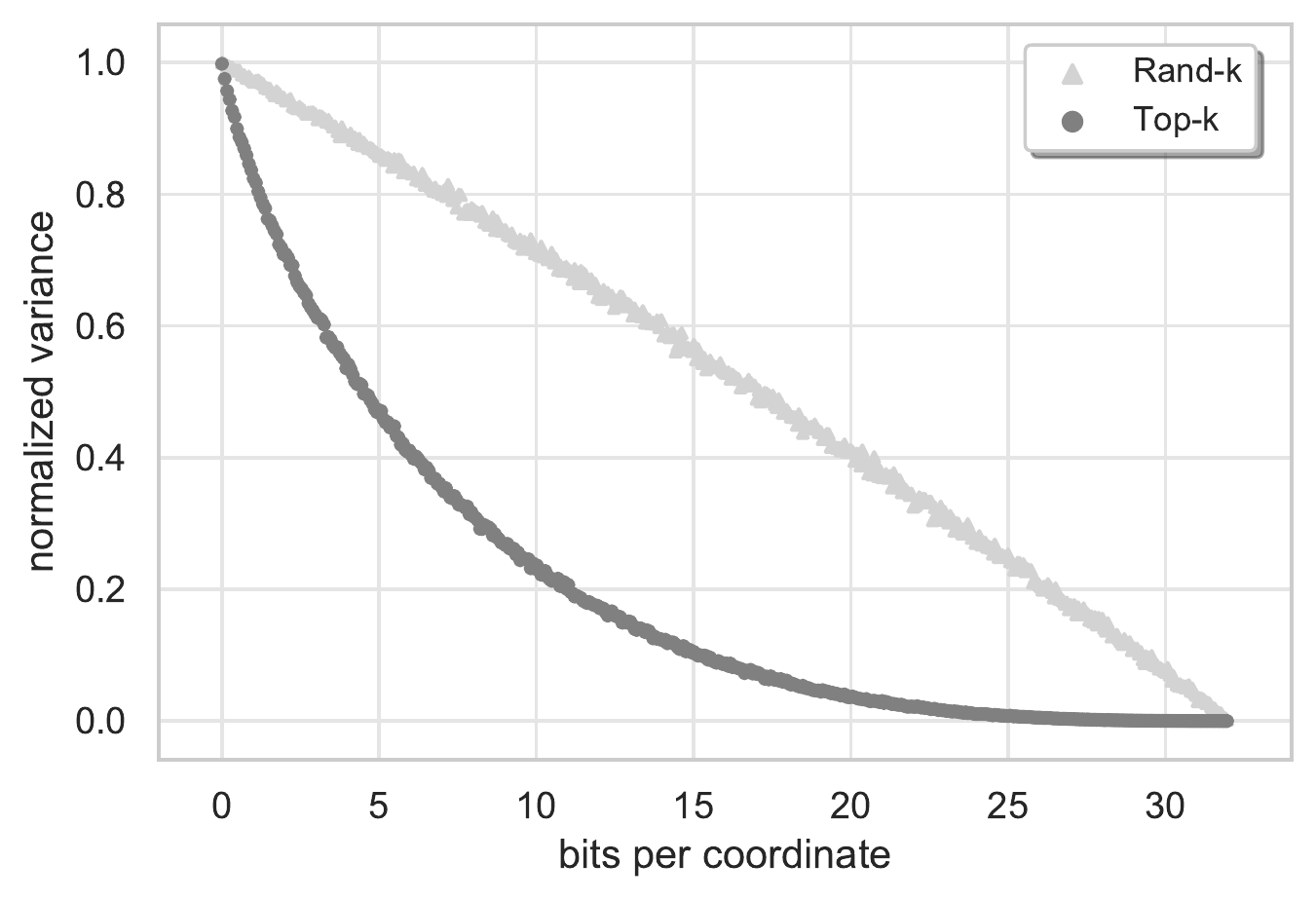}}
\caption{The comparison of Top-$k$ and Rand-$k$ sparsifiers w.r.t.\  normalized variance and the number of encoding bits used for each coordinate on average. Each point/marker represents a single $d=10^4$ dimensional vector drawn form Gaussian distribution and then compressed by the specified operator. Plots for different $d$ look very similar. Notice that, for random sparsification the normalized variance is perfectly linear with respect to the number of bit per coordinate. Letting $b$ be the total number of bits to encode the compressed vector (say in {\em binary32} system), the normalized variance produced by random sparsifier is almost $1-\tfrac{\nicefrac{b}{d}}{32}$. However, greedy sparsifier achieves exponentially lower variance $\approx 0.86^{\nicefrac{b}{d}}$ utilizing the same amount of bits.}
\label{fig:vb_top-random}
\end{center}
\vskip -0.2in
\end{figure}

First, we compare the savings $s_{top}^k$ and $s_{rnd}^k$ of these compressions. For random sparsification, we have $$\Exp{s_{rnd}^k(x) }= k \cdot (\sigma^2 + \mu^2),$$ where $\mu$ and $\sigma^2$ are the mean and variance of the Gaussian distribution. For computing $\Exp{s_{top}^k(x)}$, we use the probability density function of $k$-th order statistics (see e.g. \cite{arnold}). Table~\ref{normtable_main} shows that Top-$3$ and Top-$5$ sparsifiers ``save'' $3\times$--$40\times$ more information in expectation and the factor grows with the dimension.

Next we compare normalized variances $\frac{\omega_{top}^k(x)}{\norm*{x}^2}$ and $\frac{\omega_{rnd}^k(x)}{\norm*{x}^2}$ for randomly generated Gaussian vectors. In an attempt to give a dimension independent comparison, we compare them against the average number of encoding bits per coordinate, which is quite stable with respect to the dimension. Figure \ref{fig:vb_top-random} reveals the superiority of greedy sparsifier against the random one.

\paragraph{Practical distribution.} We obtained various gradient distributions via logistic regression (\textit{mushrooms} LIBSVM dataset) and least squares. We used the sklearn package and built Gaussian smoothing of the practical gradient density. The second moments, i.e.  energy ``saving'', were already calculated from it by formula for density function of $k$-order statistics, see \cite{arnold} for reference. We conclude experiments for Top-5 and Rand-5, see Figure \ref{exper2_main} for details.

\begin{figure}[t]
\centering
\hfill
\subfigure[\newline$s_{rnd}^5 = 9 \cdot 10^5$, \newline $s_{top}^5 = 28 \cdot 10^5$]
  {\includegraphics[width=0.22\linewidth]{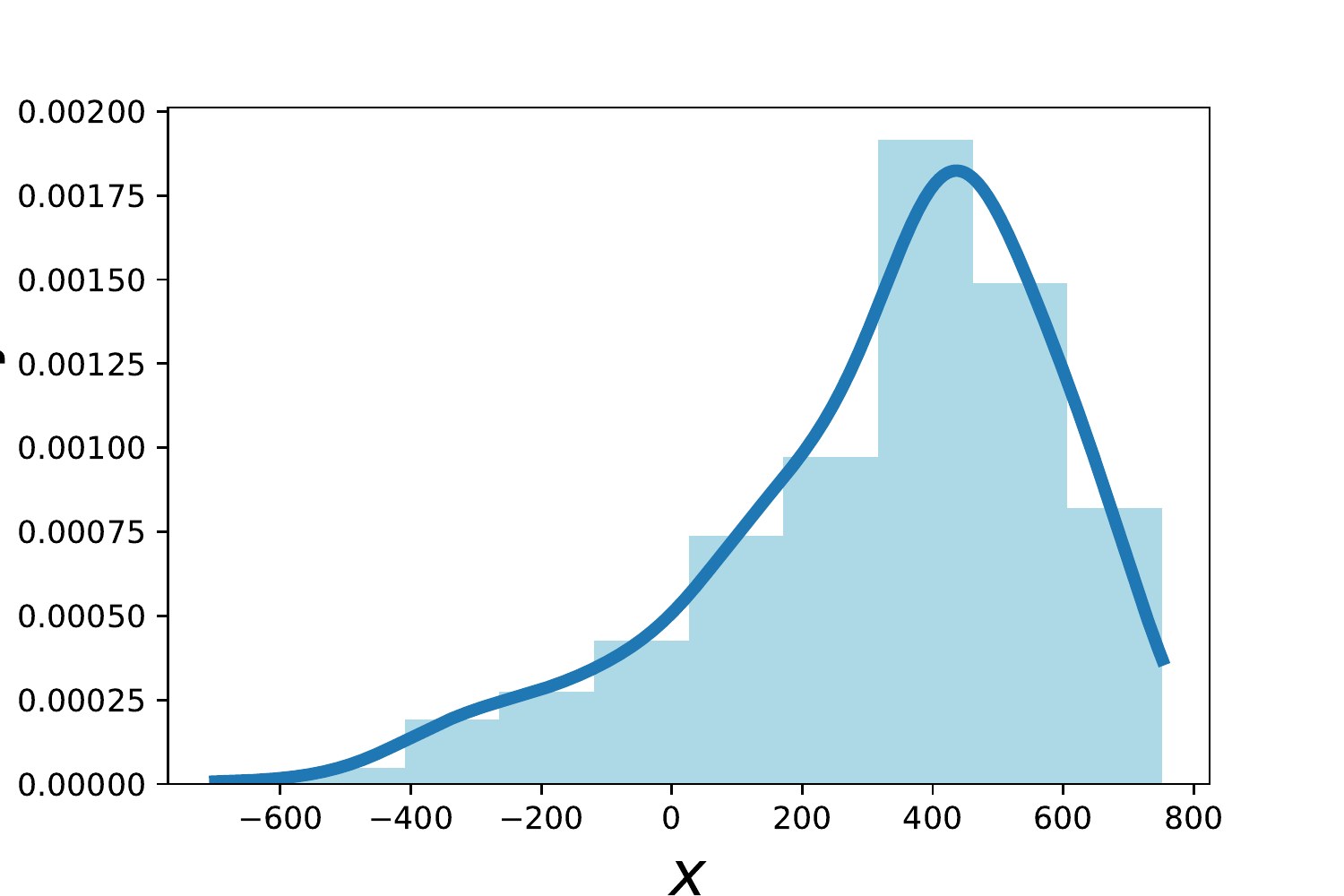}}
 \hfill
\subfigure[\newline$s_{rnd}^5 = 624$,\newline $s_{top}^5 = 1357$]
  {\includegraphics[width=0.22\linewidth]{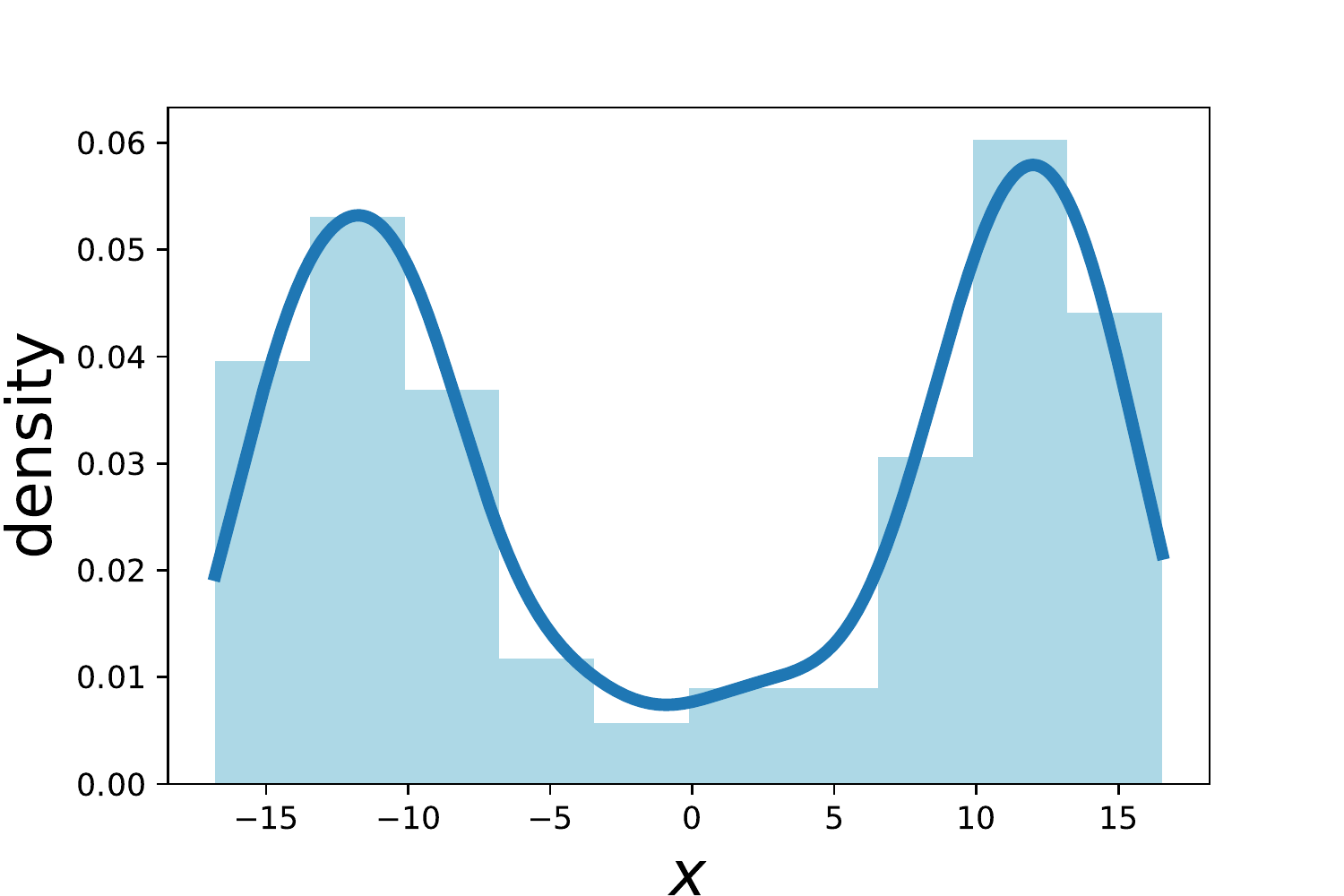}}
\hfill
  \subfigure[\newline$s_{rnd}^5 = 11921$,\newline $s_{top}^5 = 23488$]
  {\includegraphics[width=0.22\linewidth]{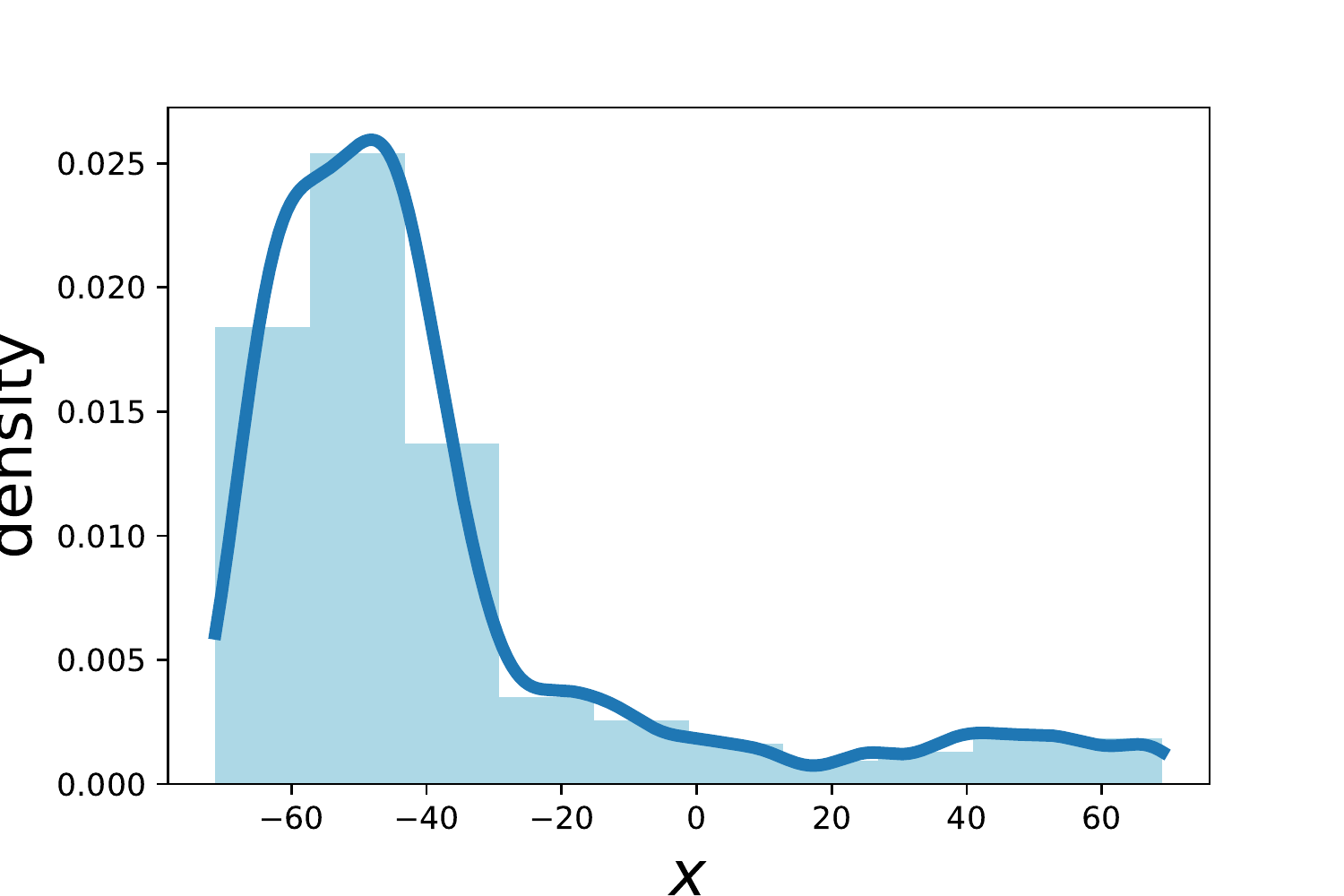}}
 \hfill
  \subfigure[\newline$s_{rnd}^5 = 2 \cdot 10^{-4}$,\newline $s_{top}^5 = 17 \cdot 10^{-4}$]
  {\includegraphics[width=0.22\linewidth]{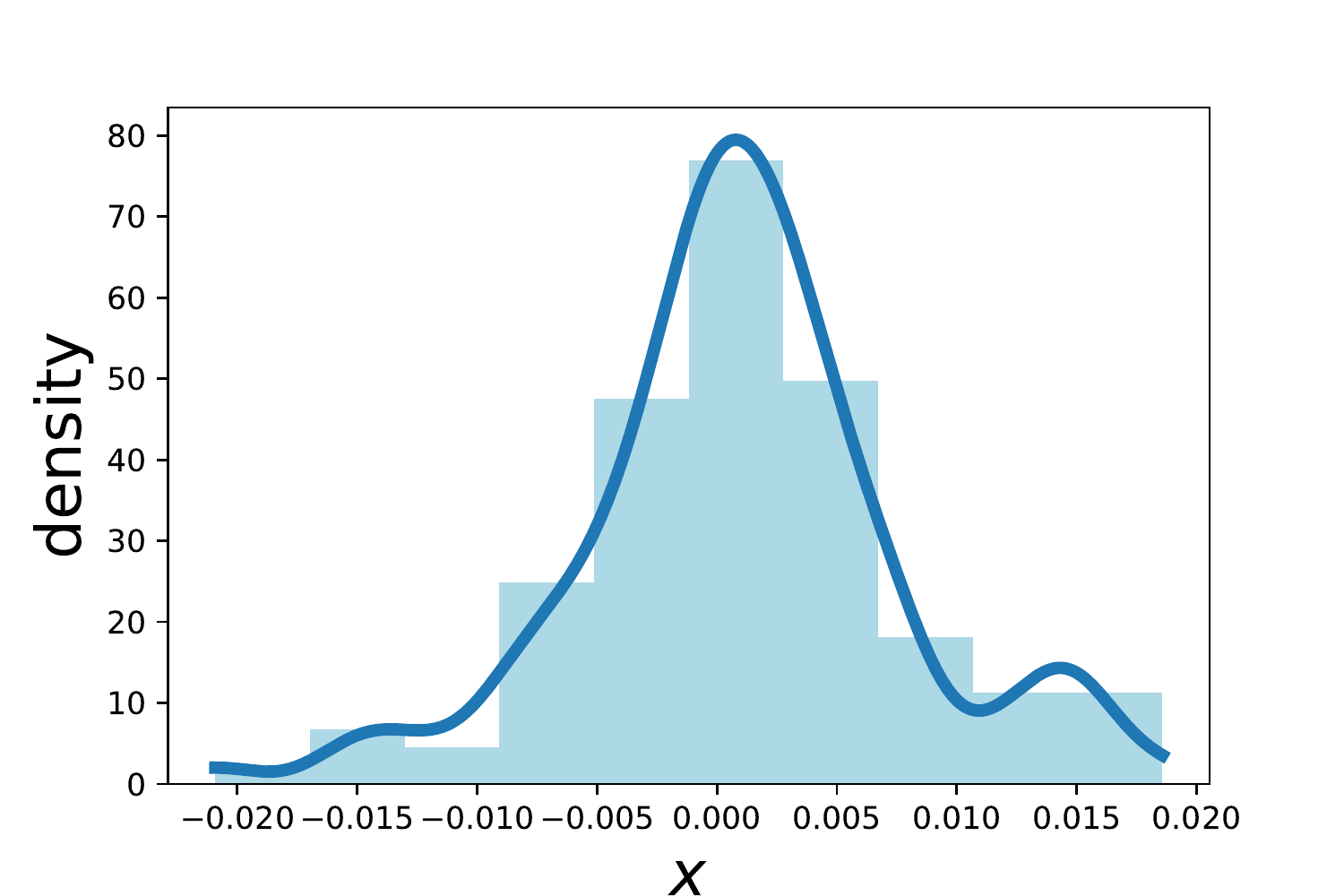}}
\caption{Calculations of the Rand-5 and Top-5 energy ``saving'' for practical gradient distributions ((a),(b),(c):  quadratic problem, (d):  logistic regression). The results of Top-5 are 3--5$\times$ better.}
\label{exper2_main}
\end{figure}

\subsection{New compressor: Top-$k$ combined with dithering}
In Section~\ref{sec:examples} we gave a new biased compression operator (see \eqref{ex:top-gendith}), where we combined Top-$k$ sparsification operator (see \eqref{ex:top-sparse}) with the general exponential dithering (see \eqref{ex:ge-dithering}). Consider the composition operator with  natural dithering, i.e., with base $b=2$. We showed that it belongs to $\B^1(\frac{k}{d},\frac{9}{8})$, $\B^2(\frac{k}{d},\frac{9}{8})$ and $\B^3(\frac{9d}{8k})$. Figure~\ref{fig:vb_comparison} empirically confirms that it attains the lowest compression parameter $\delta\ge1$  among all other known compressors (see \eqref{def:deltaquant}). Furthermore, the iteration complexity $\cO\left(\delta\tfrac{L}{\mu}\log\tfrac{1}{\varepsilon}\right)$ of \texttt{CGD} for  $\cC\in \B^3(\delta)$ implies that it enjoys fastest convergence.

\begin{figure}[ht]
\centering
\includegraphics[width=0.5\columnwidth]{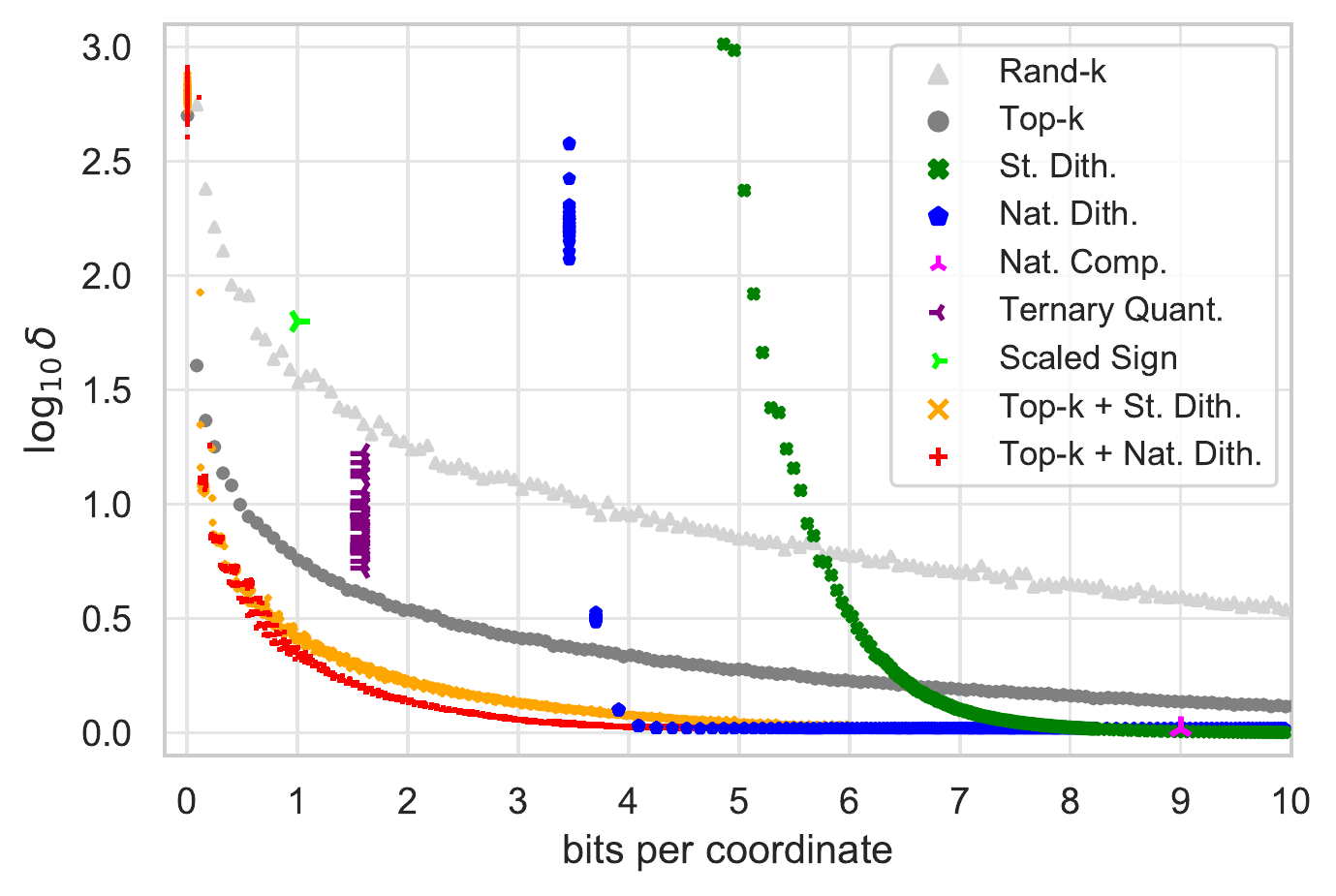}
\caption{Comparison of various compressors with respect to the parameter $\delta\ge1$ in $\log_{10}-$scale and the number of encoding bits used for each coordinate on average. Each point/marker represents a single $d=10^4$ dimensional vector  $x$ drawn from Gaussian distribution and then compressed by the specified operator.}
\label{fig:vb_comparison}
\end{figure}

\section{Distributed setting} \label{sec:distributed}

We now focus attention on a distributed setup with $n$ machines, each of which owns {\em non-iid} data defining one loss function $f_i$. Our goal is to minimize the average loss:
\begin{equation}
    \label{distr probl_main}
 \min \limits_{x \in \R^d} \sbr*{f(x) \eqdef \frac{1}{n} \sum\limits_{i=1}^n f_i(x)}.
\end{equation}

\subsection{Distributed \texttt{CGD} with unbiased compressors}
Perhaps the most straightforward extension of \texttt{CGD} to the distributed setting is to consider the method
\begin{equation}\label{eq:DCGD} \tag{\texttt{DCGD}} x^{k+1} = x^k - \stepsize \frac{1}{n}\sum \limits_{i=1}^n \cC^k_i(\nabla f_i(x^k)).\end{equation}
Indeed, for $n=1$ this method reduces to \texttt{CGD}. For unbiased compressors belonging to $\mathbb{U}(\zeta)$, this method converges under suitable assumptions on the functions. For instance, if $f_i$ are $L$-smooth and $f$ is $\mu$-strongly convex, then as long as the stepsize is chosen appropriately, the method converges  to a $\cO\left(\frac{\eta D(\zeta-1)}{\mu n}\right)$ neighborhood of the (necessarily unique) solution $x^\star$ with the linear rate $$\cO\left(\left(\frac{L}{\mu}+\frac{L(\zeta-1) }{ \mu n}\right) \log \frac{1}{\epsilon}\right),$$ where $D\eqdef \frac{1}{n}\sum \limits_{i=1}^n \norm*{\nabla f_i(x^\star)}^2$ \cite{gorbunov2020unified}. In particular, in the overparameterized setting when $D=0$, the method converges to the exact solution, and does so at the same rate as {\tt GD} as long as $\zeta = \cO(n)$. These results hold even if a regularizer is considered, and a proximal step is added to \texttt{DCGD}. Moreover, as shown by \cite{mishchenko2019distributed} and \cite{diana2}, a variance reduction technique can be devised to remove the neighborhood convergence and replace it by convergence to $x^\star$, at the negligible additional cost of $\cO((\zeta-1)\log \frac{1}{\epsilon})$.

\subsection{Failure of \texttt{DCGD} with biased compressors}  However, as we now demonstrate by giving some counter-examples,  \texttt{DCGD} may fail if the compression operators are allowed to be biased. In the first example below, \texttt{DCGD} used with the Top-1 compressor diverges at an exponential rate.

\begin{example}\label{ex:counter-example}
Consider $n=d=3$ and define
$$f_1(x) = \lin{a,x}^2 + \frac{1}{4}\norm*{x}^2, \qquad f_2(x) = \lin{b,x}^2 + \frac{1}{4}\norm*{x}^2, \qquad f_3(x) = \lin{c,x}^2 + \frac{1}{4}\norm*{x}^2,$$ where $a=(-3,2,2)$, $b=(2,-3,2)$ and $c=(2,2,-3)$.
Let the starting iterate be $x^0=(t,t,t)$, where $t > 0$. Then
$$
\nabla f_1(x^0) =\frac{t}{2} (-11, 9, 9), \qquad \nabla f_2(x^0) =\frac{t}{2} (9,-11, 9), \qquad \nabla f_3(x^0) = \frac{t}{2} (9,9,-11).$$ Using the Top-1 compressor, we get $\cC(\nabla f_1(x^0)) = \frac{t}{2}(-11, 0, 0)$, \\
$\cC(\nabla f_2(x^0)) = \frac{t}{2}(0,-11, 0)$ and
$\cC(\nabla f_3(x^0)) = \frac{t}{2}(0,0,-11)$. The next iterate of DCGD is
$$ x^1 = x^0 - \stepsize \frac{1}{3} \sum_{i=1}^3 \cC(\nabla f_i(x^0)) =  \left(1+\frac{11 \eta}{6}\right)x^0. $$
Repeated application gives $$x^k = \left(1+\frac{11 \eta}{6}\right)^k x^0.$$  Since $\eta>0$, the entries of $x^k$ diverge exponentially fast to $+\infty$.
\end{example}

The above counter-example can be extended to the case of Top-$d_1$ when $d_1 < \ceil*{\frac{d}{2}}$.

\begin{example}\label{ex:counter-example-ext}
Fix the dimension $d\ge 3$ and let $n=\binom{d}{d_1}$ be the number of nodes, where $d_1 < \ceil*{\frac{d}{2}}$ and $d_2 = d-d_1 > d_1$. Choose positive numbers $b,\;c > 0$ such that
$$
-b d_1 + c d_2 = 1,\; b > c+1.
$$
One possible choice could be $b = d_2+\frac{d_2}{d_1},\; c = d_1+\frac{1}{d_2}+1$. Define vectors $a_j\in\R^d,\;j\in[n]$ via
$$
a_j = \sum_{i\in I_j} (-b) e_i + \sum_{i\in [d]\setminus I_j} c e_i,
$$
where sets $I_j\subset[d],\;j\in[n]$ are all possible $d_1$-subsets of $[d]$ enumerated in some way. Define
$$
f_j(x) = \lin{a_j,x}^2 + \frac{1}{2}\norm*{x}^2, \quad j\in[n]
$$
and let the initial point be $x^0 = t e,\; t>0$, where $e=\sum_{i=1}^d e_i$ is the vector of all $1$s. Then
$$
\nabla f_j(x^0) = 2\lin{a_j, x^0} \cdot a_j + x^0 = 2t(-b d_1 + c d_2)\cdot a_j + te = t(2 a_j + e).
$$
Since $|2(-b)+1| > |2c+1|$, then using the Top-$d_1$ compressor, we get 
$$
\cC(\nabla f_j(x^0)) = -t(2b-1) \sum_{i\in I_j} e_i.
$$
Therefore, the next iterate of \texttt{DCGD} is
\begin{eqnarray*}
x^1
&=& x^0 - \stepsize \frac{1}{n} \sum_{j=1}^n \cC(\nabla f_j(x^0))
= x^0 + \frac{\stepsize t(2b-1)}{n} \sum_{j=1}^n\sum_{i\in I_j}e_i \\
&=& x^0 + \frac{\stepsize(2b-1)}{n} \binom{d}{d_1-1} x^0
= \left(1+ \frac{\stepsize(2b-1)d_1}{d_2+1}\right) x^0,
\end{eqnarray*}
which implies
$$x^k = \left(1+ \frac{\stepsize(2b-1)d_1}{d_2+1}\right)^k x^0.$$
Since $\eta>0$ and $b>1$, the entries of $x^k$ diverge exponentially fast to $+\infty$.
\end{example}

Finally, we present more general counter-example with different type of failure for \texttt{DCGD} when non-randomized compressors are used.

\begin{example}
Let $\cC\colon\R^d\to\R^d$ be {\em any deterministic} mapping for which there exist vectors $v_1,v_2,\dots,v_m\in\R^d$ such that
\begin{equation}\label{eq:vanisher}
\sum_{i=1}^m v_i \ne 0, \qquad\text{and}\qquad \sum_{i=1}^m \cC(v_i) = 0.
\end{equation}
Consider the distributed optimization problem \eqref{distr probl_main} with $n=m$ devices and with the following strongly convex loss functions
$$
f_i(x) = \<v_i, x\> + \frac{1}{2}\|x\|^2, \quad i\in[m].
$$
Then $\nabla f_i(x) = v_i + x$ and $\nabla f(x) = \frac{1}{n}\sum_{i=1}^n v_i + x$. Hence, the optimal point $x^* = - \frac{1}{n}\sum_{i=1}^n v_i \neq 0$. However, it can be easily checked that, with initialization $x^0=0$, we have
$$
x^1 = x^0 - \gamma\frac{1}{n}\sum_{i=1}^n \cC(\nabla f_i(x^0)) = x^0 - \gamma\frac{1}{n}\sum_{i=1}^n \cC(v_i + x^0) = x^0.
$$
Thus, when initialized at $x^0=0$, not only \texttt{DCGD} does not converge to the solution $x^*\neq 0$, it remains stuck at the same initial point for all iterates, namely $x^k = x^0 = 0$ for all $k\ge1$.

Condition \eqref{eq:vanisher} can be easily satisfied for specific biased compressors. For instance, Top-$1$ satisfies \eqref{eq:vanisher} with $v_1 = \left[\begin{smallmatrix} 1 \\ 4 \end{smallmatrix}\right]$, $v_2 = \left[\begin{smallmatrix} -1 \\ -2 \end{smallmatrix}\right]$, $v_3 = \left[\begin{smallmatrix} 1 \\ -2 \end{smallmatrix}\right]$.
\end{example}
 
The above examples suggests that one needs to devise a different approach to solving the distributed problem \eqref{distr probl_main} with biased compressors. We resolve this problem by employing a memory feedback mechanism.

\subsection{Error feedback} We show that distributed version of Distributed \texttt{SGD} wtih Error-Feedback~\cite{karimireddy2019error}, displayed in Algorithm~\ref{alg}, is able to resolve the issue. Moreover, this algorithm allows for the computation of stochastic gradients. Each step starts with all machines $i$ in parallel computing a stochastic gradient $g_i^k$ of the form
\begin{equation}\label{stgr1_main}
    g_i^k = \nabla  f_i(x^k) + \xi_i^k,
\end{equation}
where $\nabla  f_i(x^k)$ is the true gradient, and $\xi_i^k$ is a stochastic error. Then, this is multiplied by a stepsize $\eta^k$ and added to the memory/error-feedback term $e_i^k$, and subsequently compressed. The compressed messages are communicated and aggregated. The difference of message we wanted to send and its compressed version becomes stored as $e_i^{k+1}$ for further correction in the next communication round. The output $\overline x^K$  is an ergodic average of the form
\begin{equation}
    \label{answer}
  \overline x^K \eqdef \frac{1}{W^K}\sum_{k=0}^K w^k x^k, \qquad W^K \eqdef \sum_{k=0}^K w^k.
\end{equation}

\begin{algorithm} [th]
    \caption{Distributed \texttt{SGD} with Biased Compression and Error Feedback}
    \label{alg}
    \begin{algorithmic}
\STATE 
\noindent {\bf Parameters:}  Compressors  $\cC_i^k \in \mathbb{B}^3(\delta)$; Stepsizes $\{\eta^k\}_{k\geq 0}$; Iteration count $K$  \\
\noindent {\bf Initialization:} Choose  $x^0\in \R^d$ and $e_i^0 = 0$ for all $i$
\FOR {$k=0,1, 2, \ldots, K$ } 
\item Server sends  $x^k$ to all $n$ machines
\item All machines in parallel perform these updates:
\begin{eqnarray}
    \label{comp_grad}
    & \tilde{g}_i^k = \cC^k_i(e_i^k+ \stepsize^k g_i^k) \\
    \label{error}
    & e_i^{k+1} = e_i^k + \stepsize^k g_i^k - \tilde{g}_i^k 
\end{eqnarray}
\item Each machine $i$ sends  $\tilde{g}_i^k$ to the server
\item Server performs aggregation:
\begin{eqnarray}
    \label{step}
    x^{k+1} = x^k -\frac{1}{n} \sum\limits_{i=1}^n \tilde{g}_i^k
\end{eqnarray}
\ENDFOR

\noindent {\bf Output:} Weighted average of the iterates: $\overline x^K$ \eqref{answer}
    \end{algorithmic}
\end{algorithm}

\subsection{Complexity theory} 

We assume the stochastic error $\xi_i^k$ in \eqref{stgr1_main}  satisfies the following condition.
\begin{assumption} Stochastic error  $\xi_i^k$ is unbiased, i.e. $\Exp{\xi_i^k} = 0$, and for some constants $B,C\geq 0$
\begin{equation}
    \label{stgr3_main}
    \Exp{\norm*{\xi_i^k}^2} \leq B \norm*{\nabla f_i(x^k)}^2 + C, \qquad \forall i\in[n], k\ge0.
\end{equation}
\end{assumption}

Note that this assumption is much weaker than the bounded variance assumption (i.e., $\Exp{\norm*{\xi_i^k}^2} \leq  C$) and bounded gradient assumption (i.e., $\Exp{\norm*{g_i^k}^2} \leq  C$).
We can now state the main result of this section. To the best of our knowledge, {\em  this was an open problem}: we are not aware of any convergence results for distributed optimization that tolerate general classes of {\em biased} compression operators and have reasonable assumptions on the stochastic gradient.
\begin{theorem}[Main]
\label{thm:sparsified}
Let $\{x^k\}_{k \geq 0}$ denote the iterates of Algorithm~\ref{alg} for solving problem \eqref{eq:probR_biased_biased}, where each $f_i$ is $L$-smooth and $\mu$-strongly convex. Let $x^\star$ be the minimizer of $f$ and let $f^\star\eqdef f(x^\star)$ and
$$
D\eqdef \frac{1}{n}\sum_{i=1}^n \norm*{\nabla f_i(x^\star)}^2.
$$
Assume the compression operator used by all nodes is in $\mathbb{B}^3(\delta)$. Then we have the following convergence rates under three different stepsize and iterate weighting regimes:

\begin{itemize}

\item[(i)] {\bf $\cO(\frac{1}{k})$ stepsizes \& $\cO(k)$ weights.} Let, for all $k\ge0$, the stepsizes and weights be set as $\stepsize^k = \frac{4}{\mu(\kappa + k)}$ and  $w^k = \kappa +k$, respectively, where $\kappa = \frac{56(2\delta + B)L}{\mu}$. Then
\begin{eqnarray*}
   \Exp{f(\bar x^K)}-f^\star = 
   \cO \left(\frac{A_1}{K^2}  +  \frac{A_2}{ K} \right)\,,
\end{eqnarray*}
where $A_1 \eqdef \frac{L^2(2\delta+B)^2}{\mu}\norm*{x^0-x^\star}^2 $ and $A_2\eqdef  \frac{C\left(1+\nicefrac{1}{n}\right) + D \left(\nicefrac{2B}{n} +3\delta\right)}{\mu}$.

\item[(ii)] {\bf  $\cO(1)$ stepsizes \& $\cO(e^{-k})$ weights.} Let, for all $k\ge0$, the stepsizes and weights be set as $\stepsize^k = \stepsize$ and $w^k = (1-\mu \stepsize/2)^{-(k+1)}$, respectively, where $\stepsize \leq \frac{1}{14(2\delta + B) L}$. Then
\begin{eqnarray*}
  \Exp{f(\bar x^K)}-f^\star = \tilde \cO \left( A_3 \exp \left[- \frac{K}{A_4} \right] + \frac{A_2}{ K}\right)\,,
\end{eqnarray*}
where $A_3\eqdef L(2\delta + B)\norm*{x^0-x^\star}^2$ and  $A_4 \eqdef  \frac{28L(2\delta + B)}{\mu} $.
 
\item[(iii)]  {\bf $\cO(1)$ stepsizes \& equal weights.} Let, for all $k\ge0$, the stepsizes and weights be set as $\stepsize^k = \stepsize$ and $w^k = 1$, respectively,  where $\stepsize \leq \frac{1}{14(2\delta + B) L}$. Then
\begin{eqnarray*}
\Exp{f(\bar x^K)}-f^\star = \cO \left(  \frac{A_3}{ K } + \frac{A_5}{\sqrt{K}}  \right)\,,
\end{eqnarray*}
where $A_5\eqdef  \sqrt{C\left(1 + \nicefrac{1}{n}\right) + D \left(\nicefrac{2B}{n} +3\delta\right)}\norm*{x^0-x^\star} $.

\end{itemize}

\end{theorem}

Let us make a few observations on these results.
First, Algorithm \ref{alg} employing general biased compressors and error feedback mechanism indeed {\em resolves} convergence issues of \texttt{DCGD} method by {\em converging the optimal solution $x^*$}.
Second, note that the choice of stepsizes $\stepsize^k$ and weights $w^k$ leading to convergence is not unique and {\em several schedules are feasible}.
Third, all the rates are sublinear and based on the second rate {\it (ii)} above, {\em linear convergence} is guaranteed if $C=D=0$. Based on \eqref{stgr3_main}, one setup when the condition $C=0$ holds is when all devices compute full local gradients (i.e., $g_i^k = \nabla f_i(x^k)$). Furthermore, the condition $D=0$ is equivalent to $\nabla f_i(x^\star) = 0$ for all $i\in[n]$, which is typically satisfied for over-parameterized models.
Lastly, under these two assumptions (i.e., devices can compute full local gradients and the model is over-parameterized), we show that distributed \texttt{SGD} method with error feedback converges with the same $\cO\left(\delta \frac{L}{\mu} \log \frac{1}{\epsilon}\right)$ linear rate as single node \texttt{CGD} algorithm. {\em To the best of our knowledge, this was the first regime where distributed first order method with biased compression is guaranteed to converge linearly.}

\begin{figure*}[t]
\centering
\includegraphics[width=.35\linewidth]{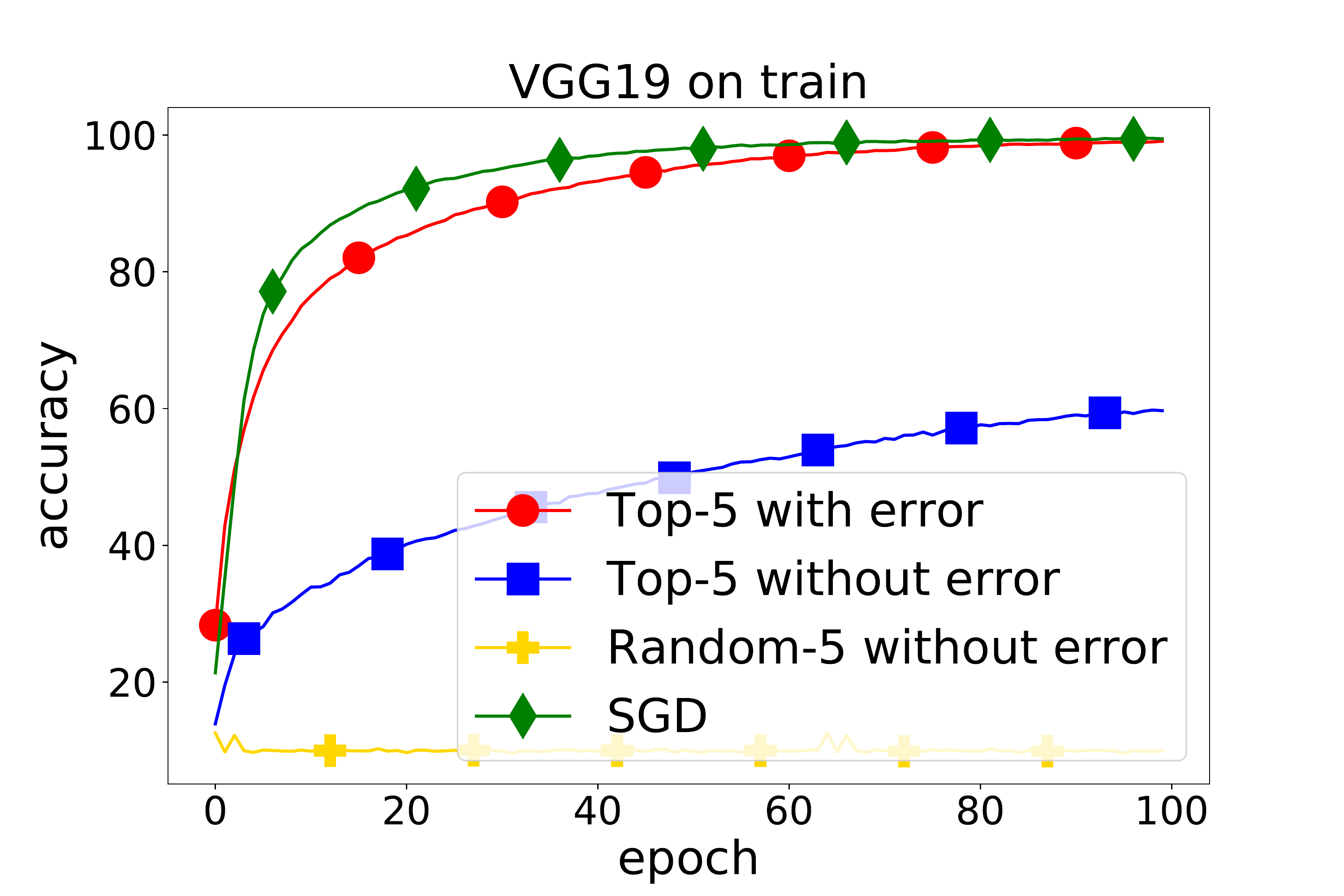}
\includegraphics[width=.35\linewidth]{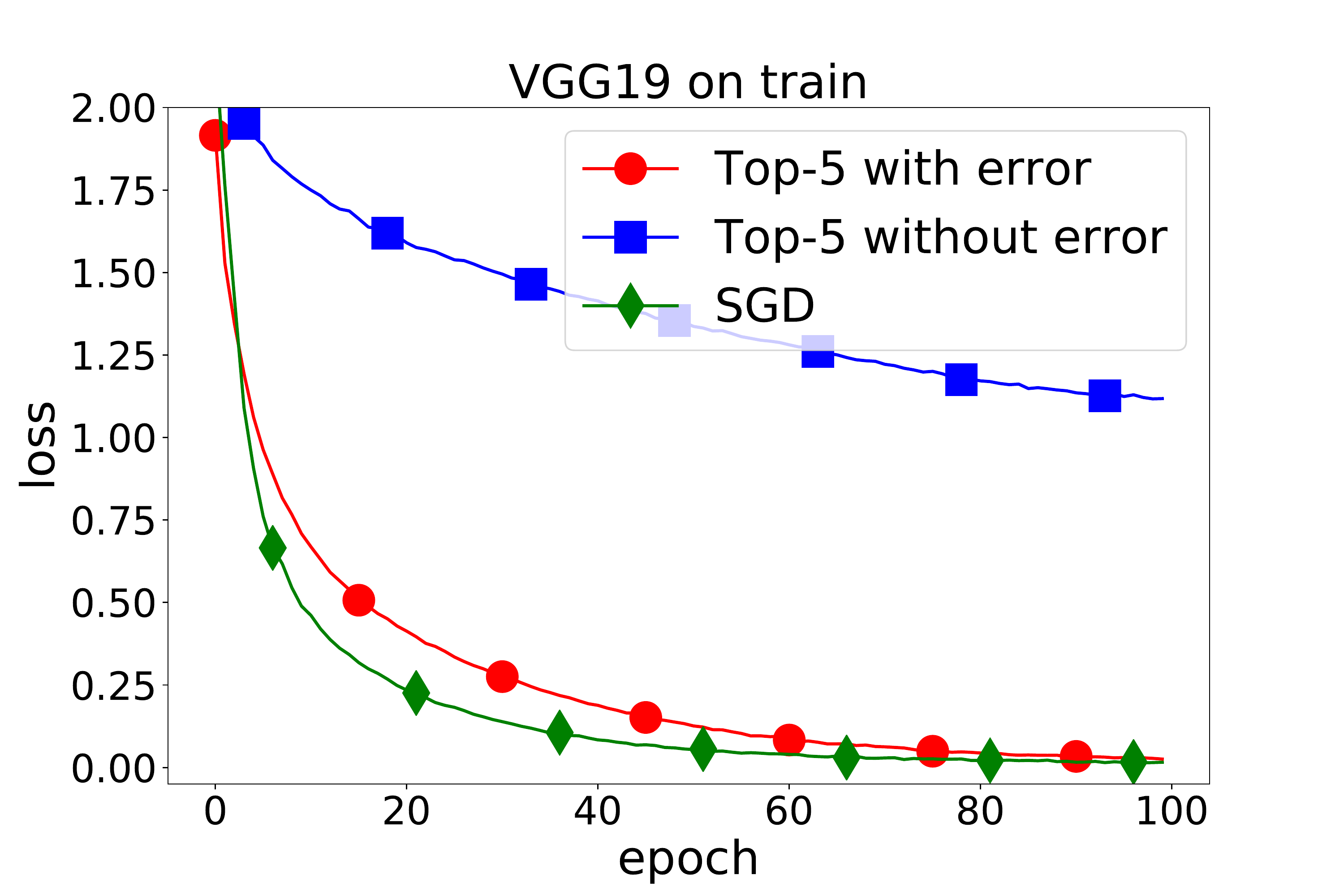}
\includegraphics[width=.35\linewidth]{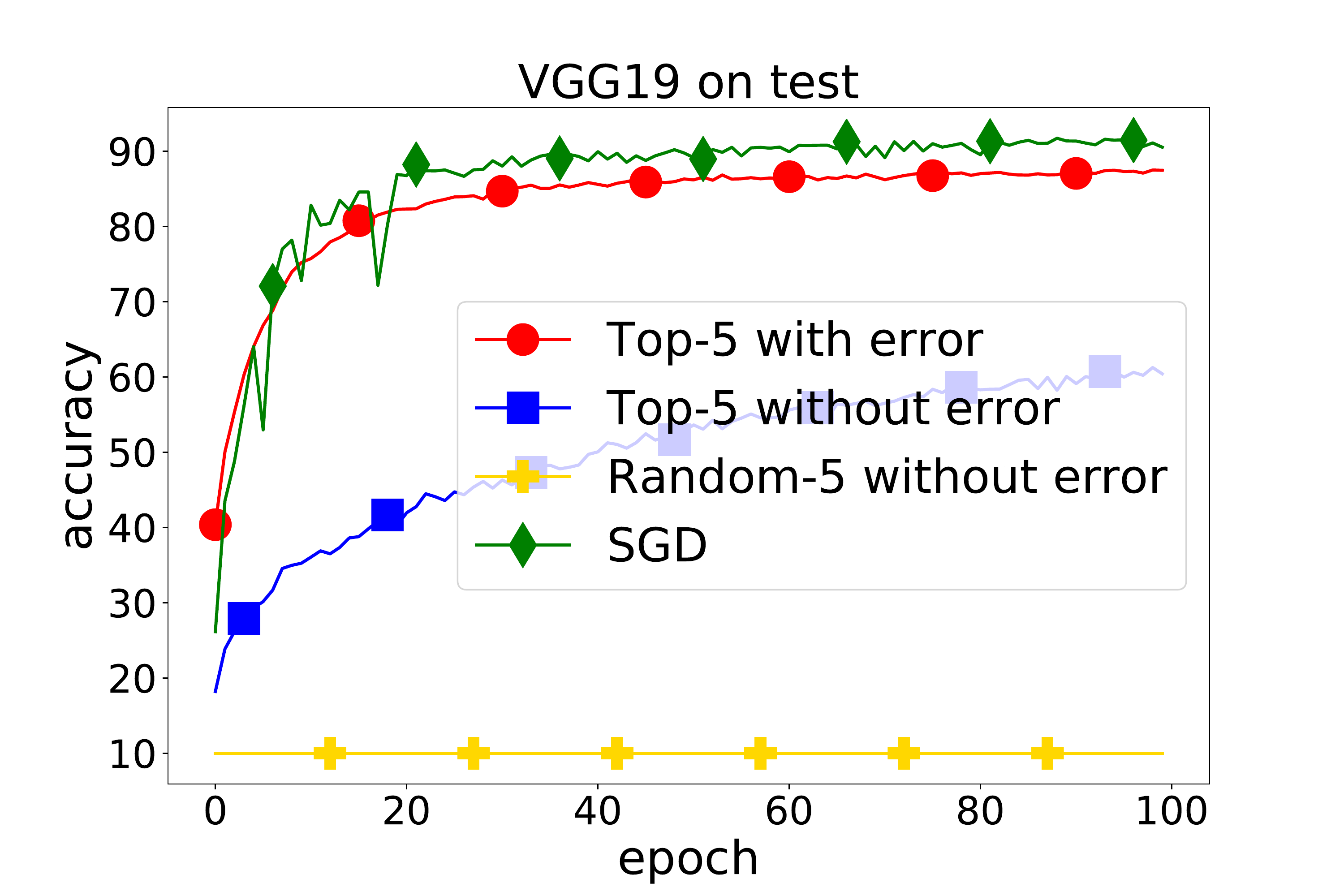}
\includegraphics[width=.35\linewidth]{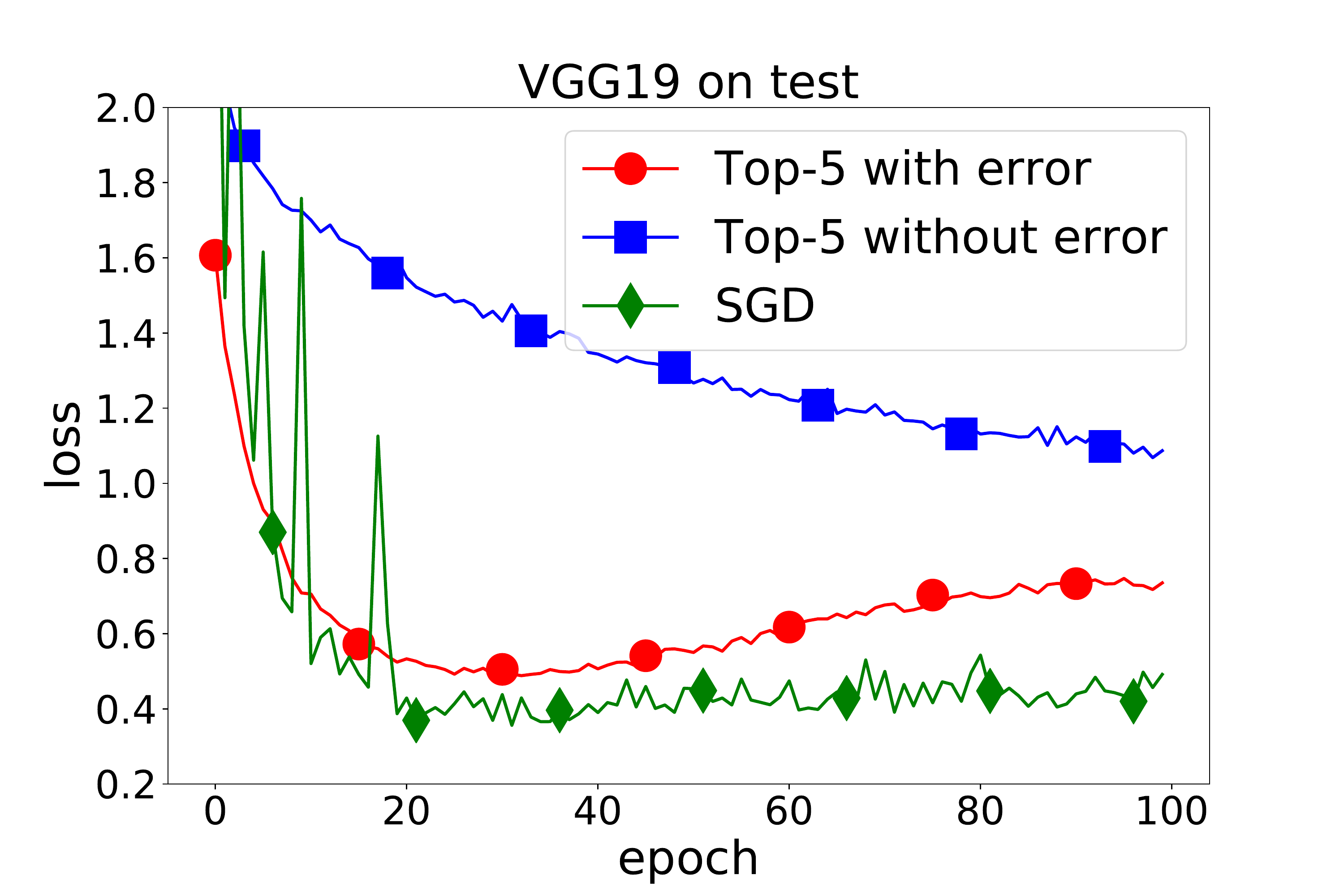}
\caption{Training/Test loss and accuracy  for VGG19 on CIFAR10 distributed among $4$ nodes for $4$ different compression operators.}
\label{fig:exper4_main}
\end{figure*}

\begin{figure*}[t]
\centering
\includegraphics[width=.35\linewidth]{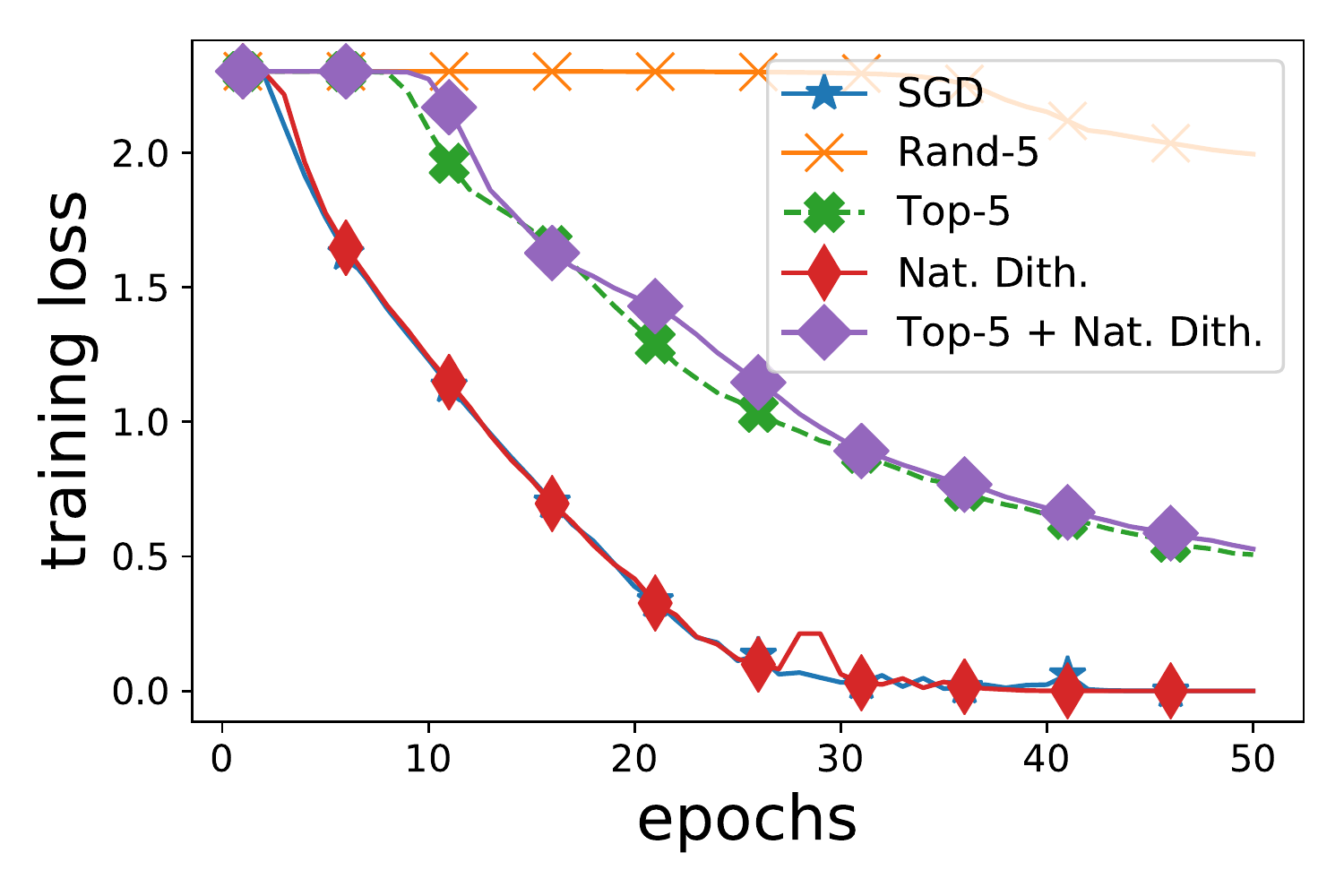}
\includegraphics[width=.35\linewidth]{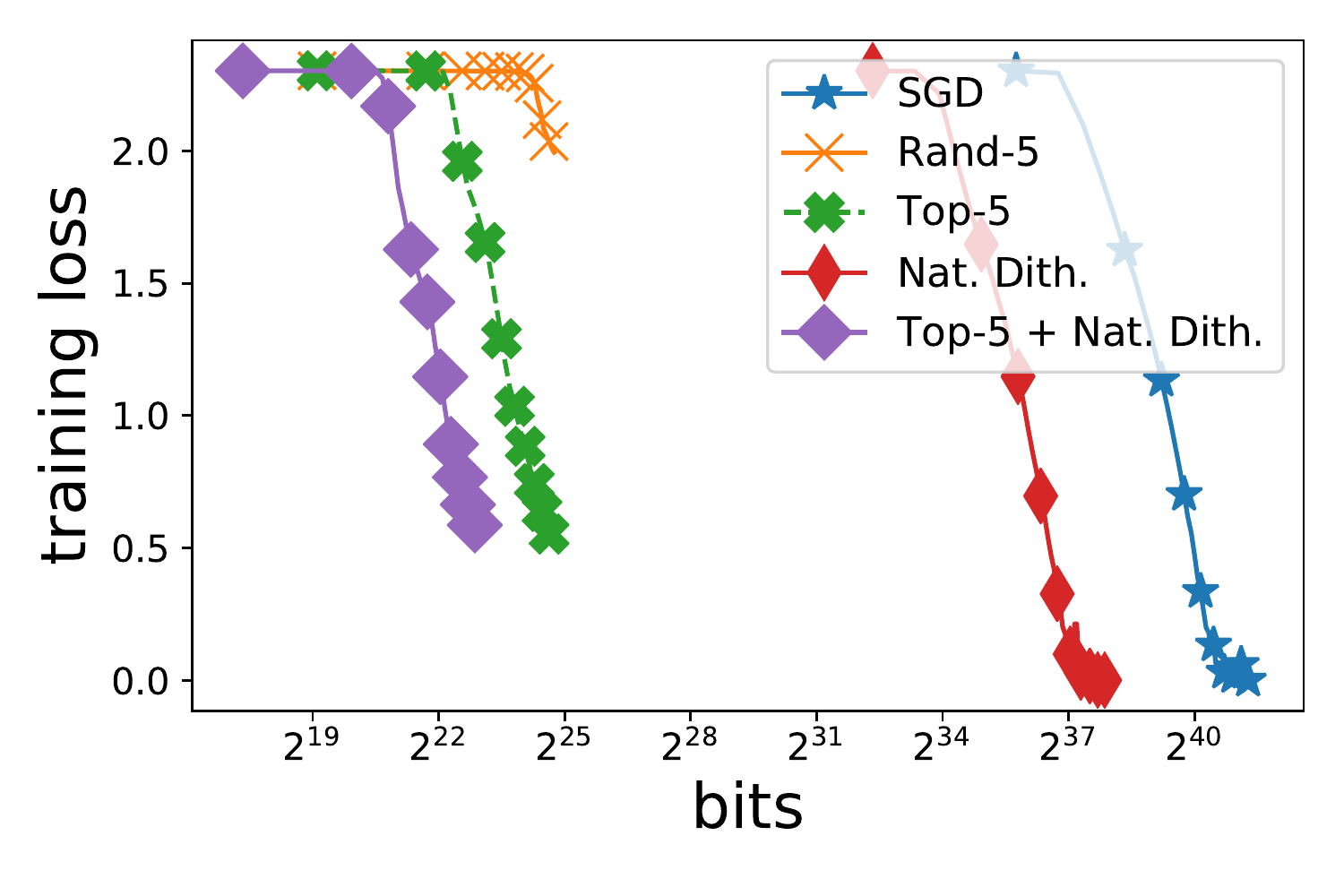}
\includegraphics[width=.35\linewidth]{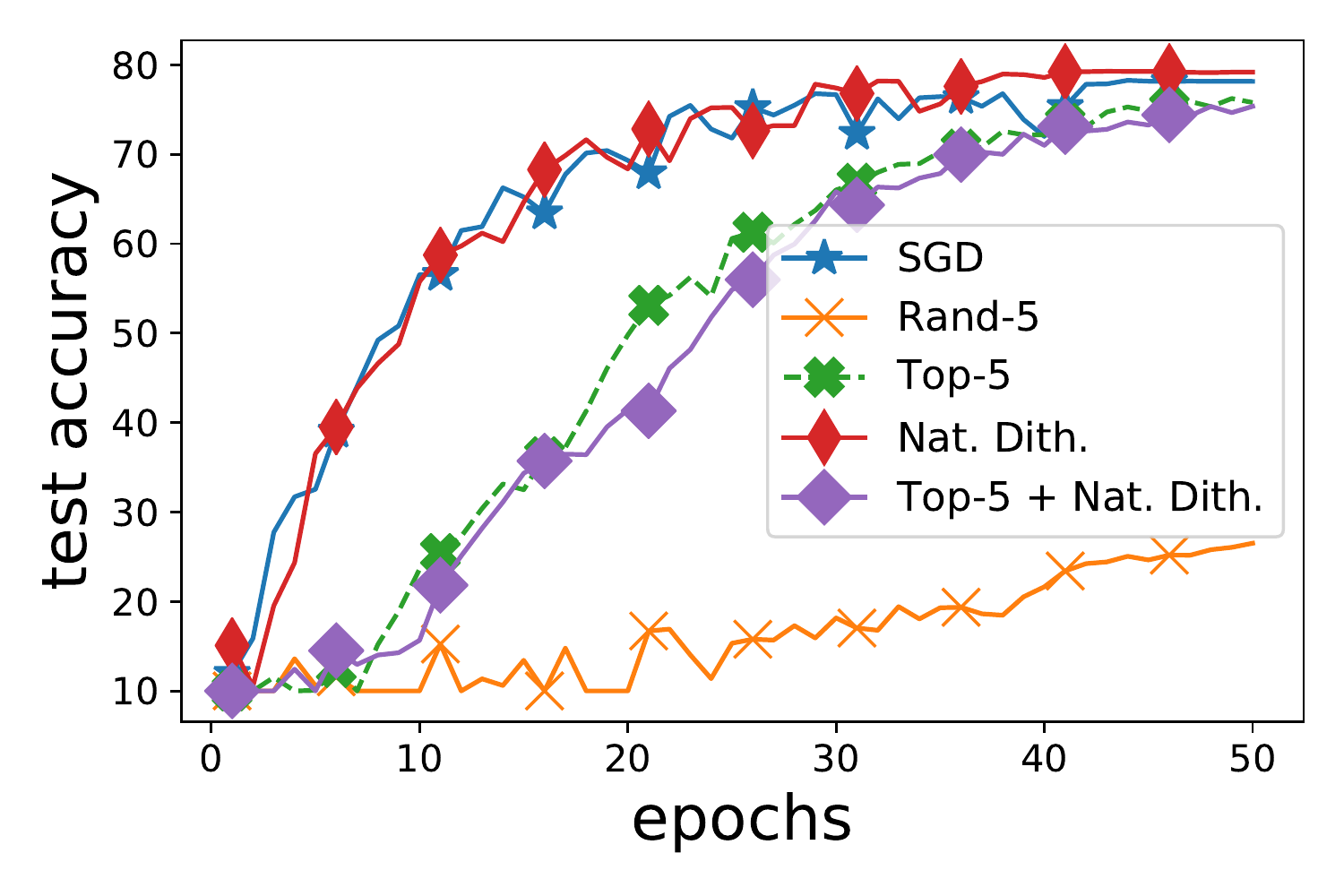}
\includegraphics[width=.35\linewidth]{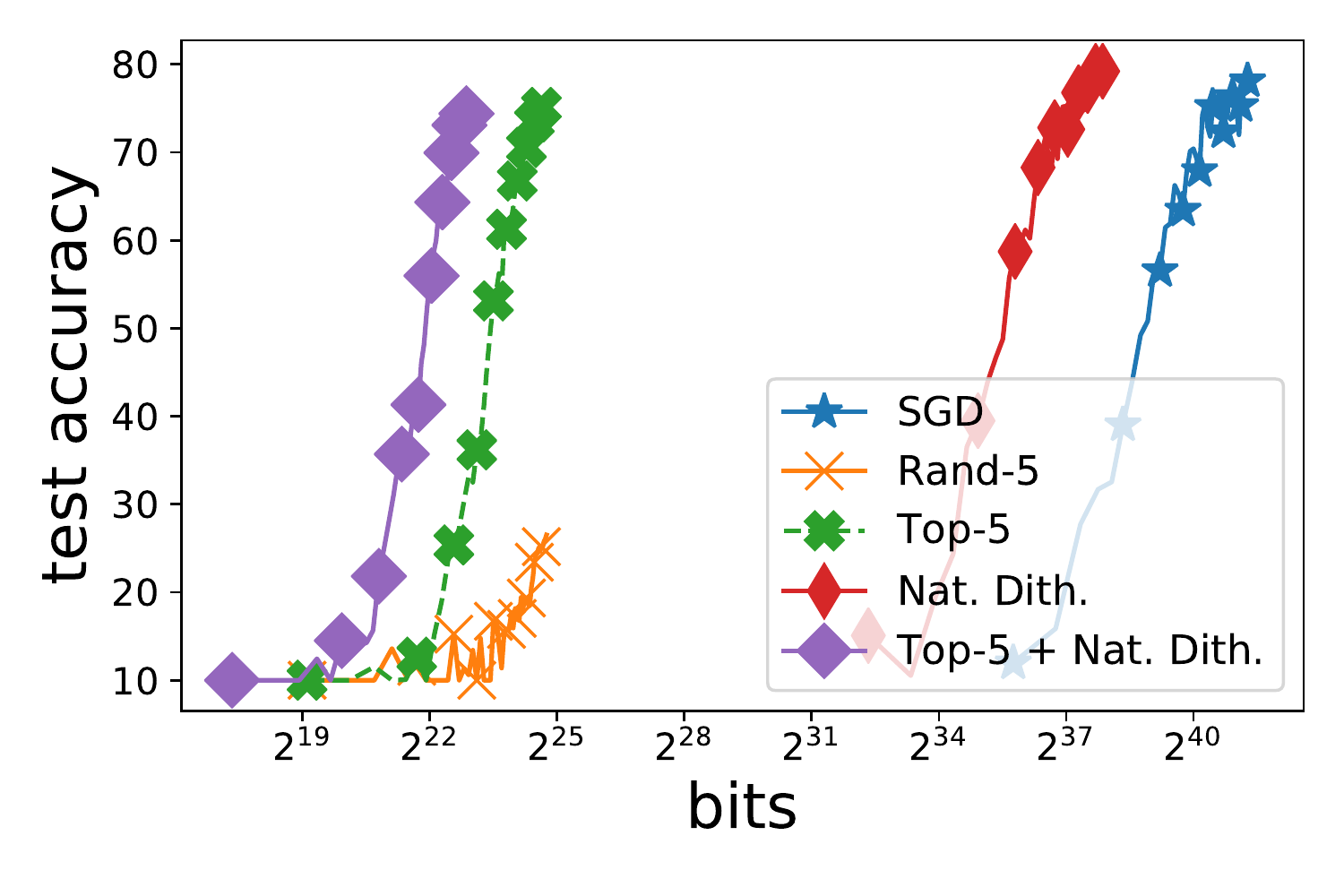}
\caption{Training loss and test accuracy  for VGG11 on CIFAR10 distributed among $4$ nodes for $5$ different compression operators.}
\label{fig:top_k_plus_nat_dit}
\end{figure*}

\section{Experiments}

In this section, we present our experimental results to support our theoretical findings. 

\subsection{Experimental setup}
We implement all methods in Python 3.7 using Pytorch~\cite{pytorch} and run on a machine with 24 Intel(R) Xeon(R) Gold 6146 CPU @ 3.20GHz cores, GPU @GeForce GTX 1080 Ti with memory 11264 MB (Cuda 10.1).
As biased compressions were already shown to perform better in distributed settings~\cite{deepgradcompress2018iclr, lim20183lc}, we rather focus on the reasoning why this is the case.  We conduct simulated experiments on one machine which enable us to do rapid direct comparisons against the prior methods. Another issue is that for many methods, there is no public implementation available, which makes it hard to do a fair comparison in distributed settings, thus we focus on simulated experiments.

\subsection{Lower empirical variance induced by biased compressors during deep network training}
Motivated by our theoretical results in Section~\ref{sec:stat}, we show that similar behaviour can be seen in the empirical variance of gradients. We run 2 sets of experiments with Resnet18 on CIFAR10 dataset. In Figure~\ref{fig:emp_variance}, we display empirical variance, which is obtained by running a training procedure with specific compression. We compare unbiased and biased compressions with the same communication complexities--deterministic with classic/unbiased $\NC$ and Top-$k$ with Rand-$k$ with  $k$ to be $\nicefrac{1}{5}$ of coordinates. One can clearly see, that there is a gap in empirical variance between biased and unbiased methods, similar to what we have shown in theory, see Section~\ref{sec:stat}. 

\subsection{Error-feedback is needed in distributed training with biased compression}
The next experiment shows the need of error-feedback for methods with biased compression operators. Based on Example 1, error feedback is necessary to prevent divergence from the optimal solution. Figure~\ref{fig:exper4_main} displays training/test loss and accuracy for VGG19 on CIFAR10 with data equally distributed among $4$ nodes. We use plain \texttt{SGD} with a default step size equal to $0.01$ for all methods, i.e. Top-$5$ with and without error feedback, Rand-$5$ and no compression. As suggested by the counterexample, not using error feedback can really hurt the performance when biased compressions are used. Also note, that performance of Rand-$5$ is significantly worse than Top-$5$.

\begin{figure}[t!]
\centering
\includegraphics[width=.3\linewidth]{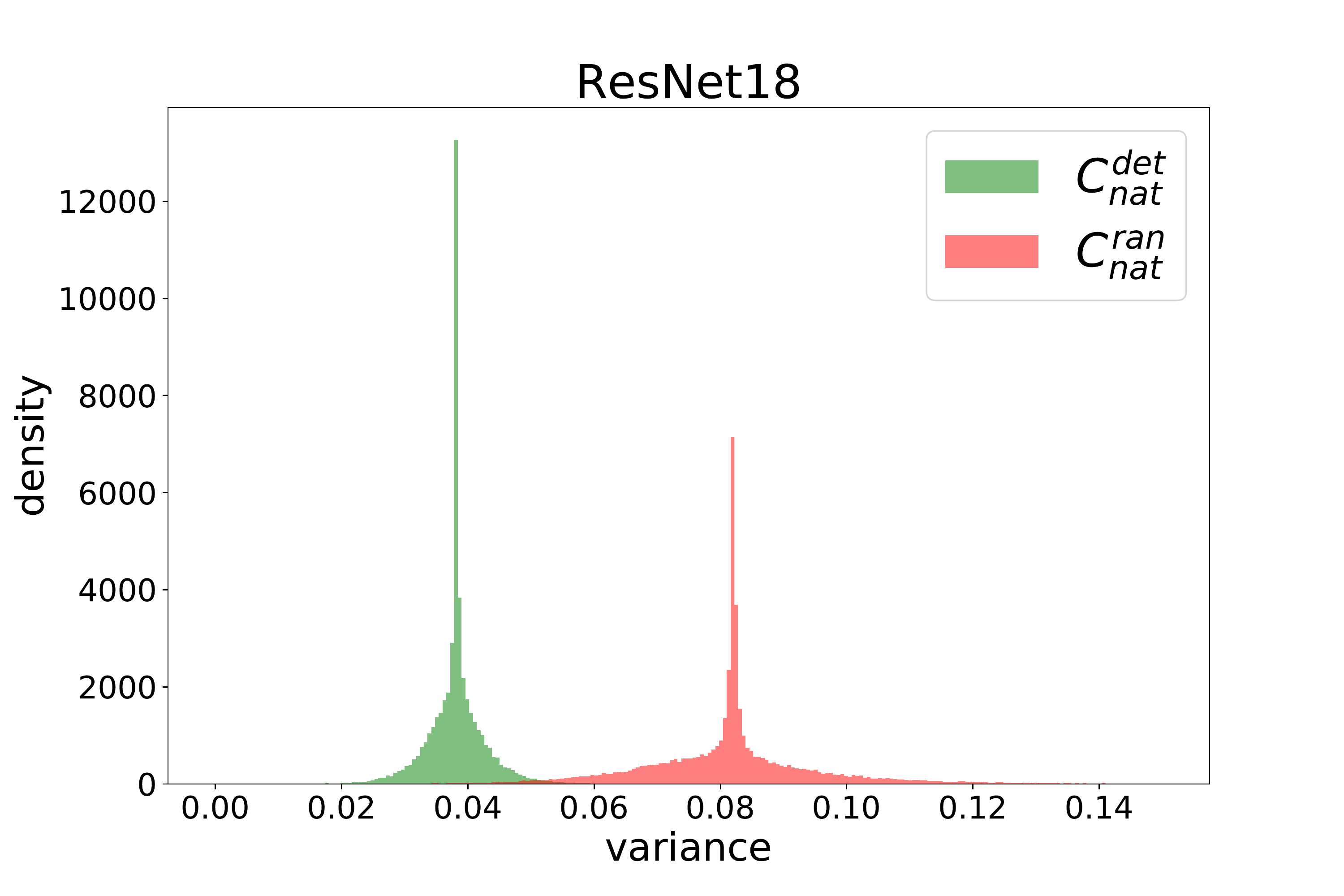}
\includegraphics[width=.3\linewidth]{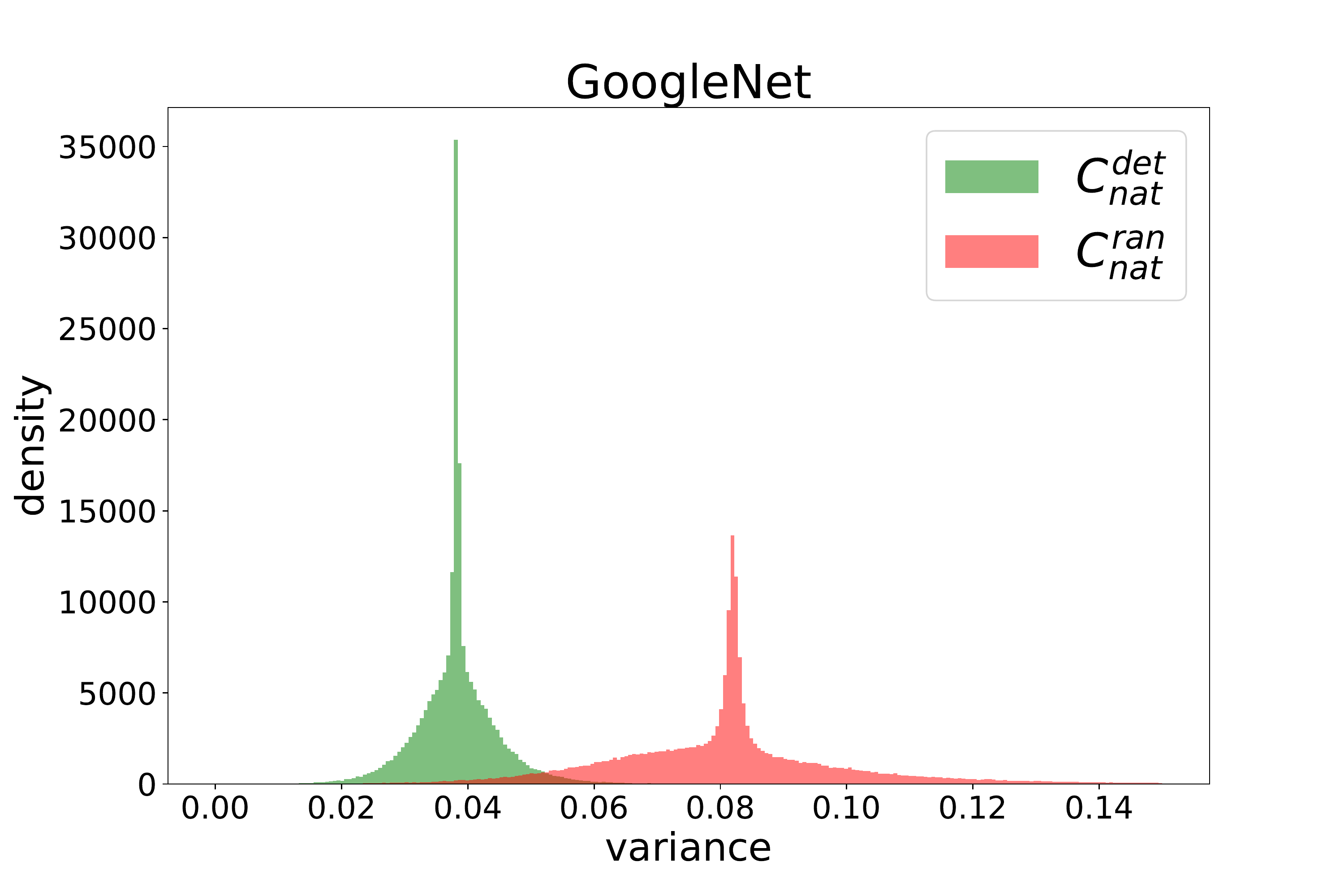}
\includegraphics[width=.3\linewidth]{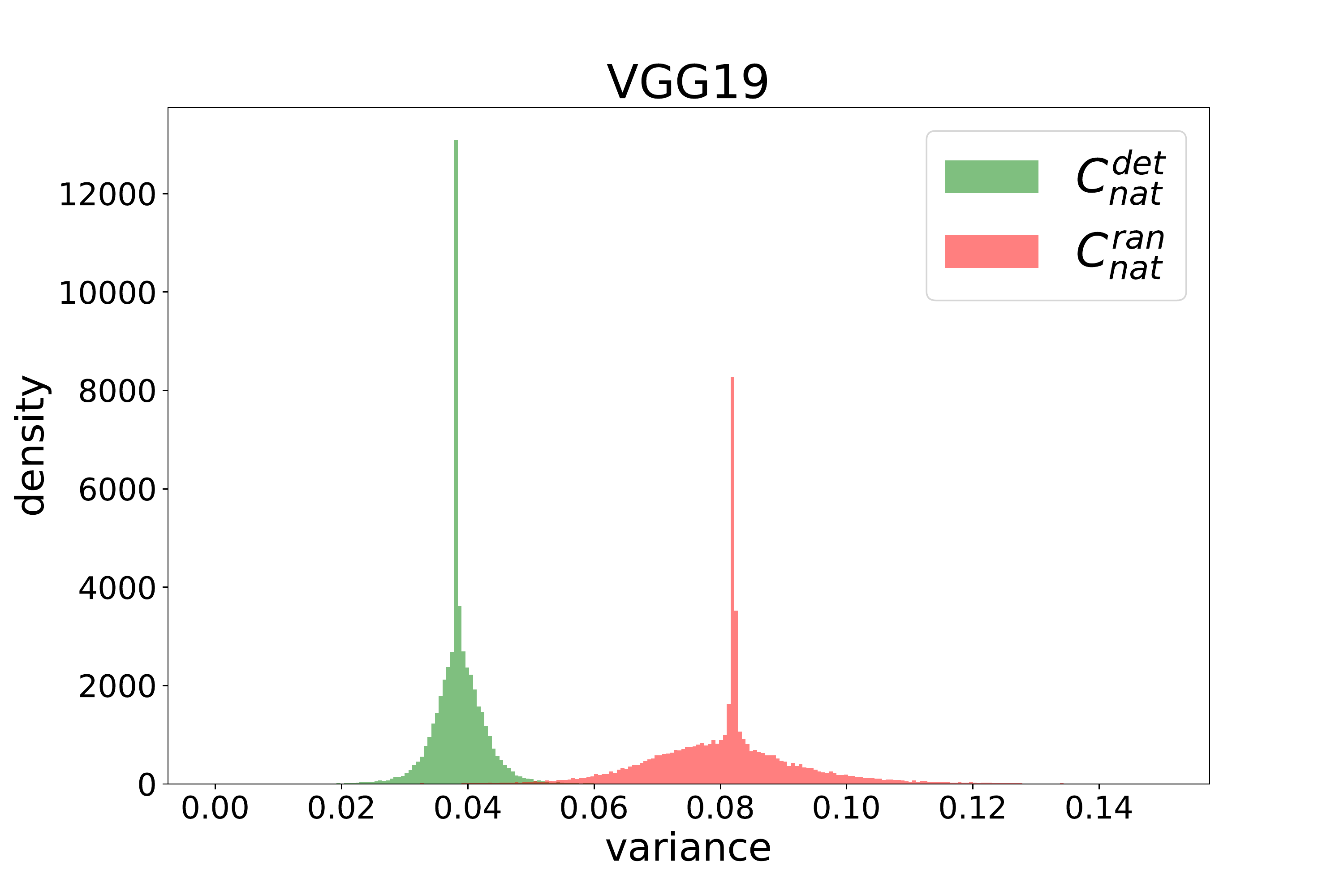}
\\
\includegraphics[width=.3\linewidth]{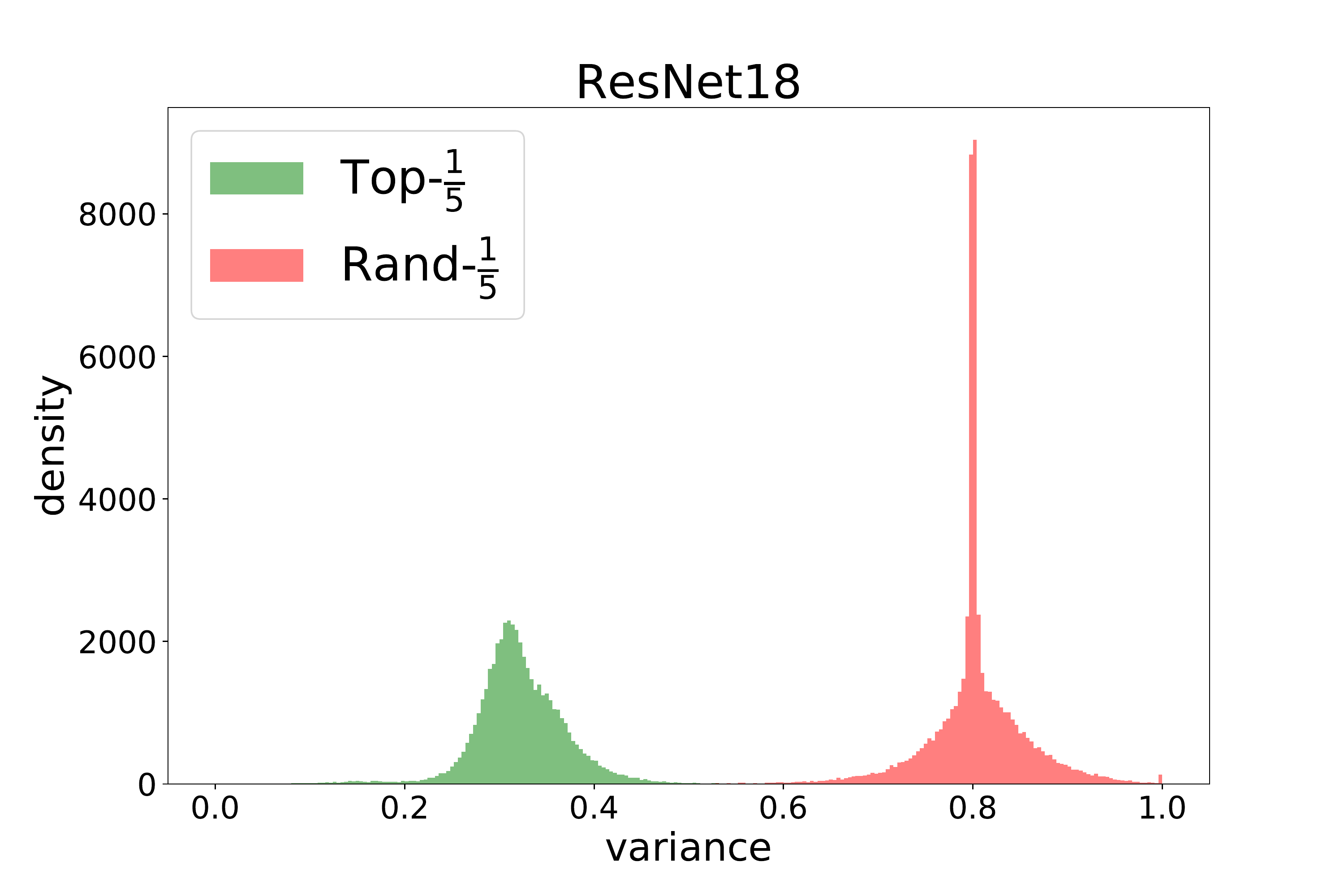}
\includegraphics[width=.3\linewidth]{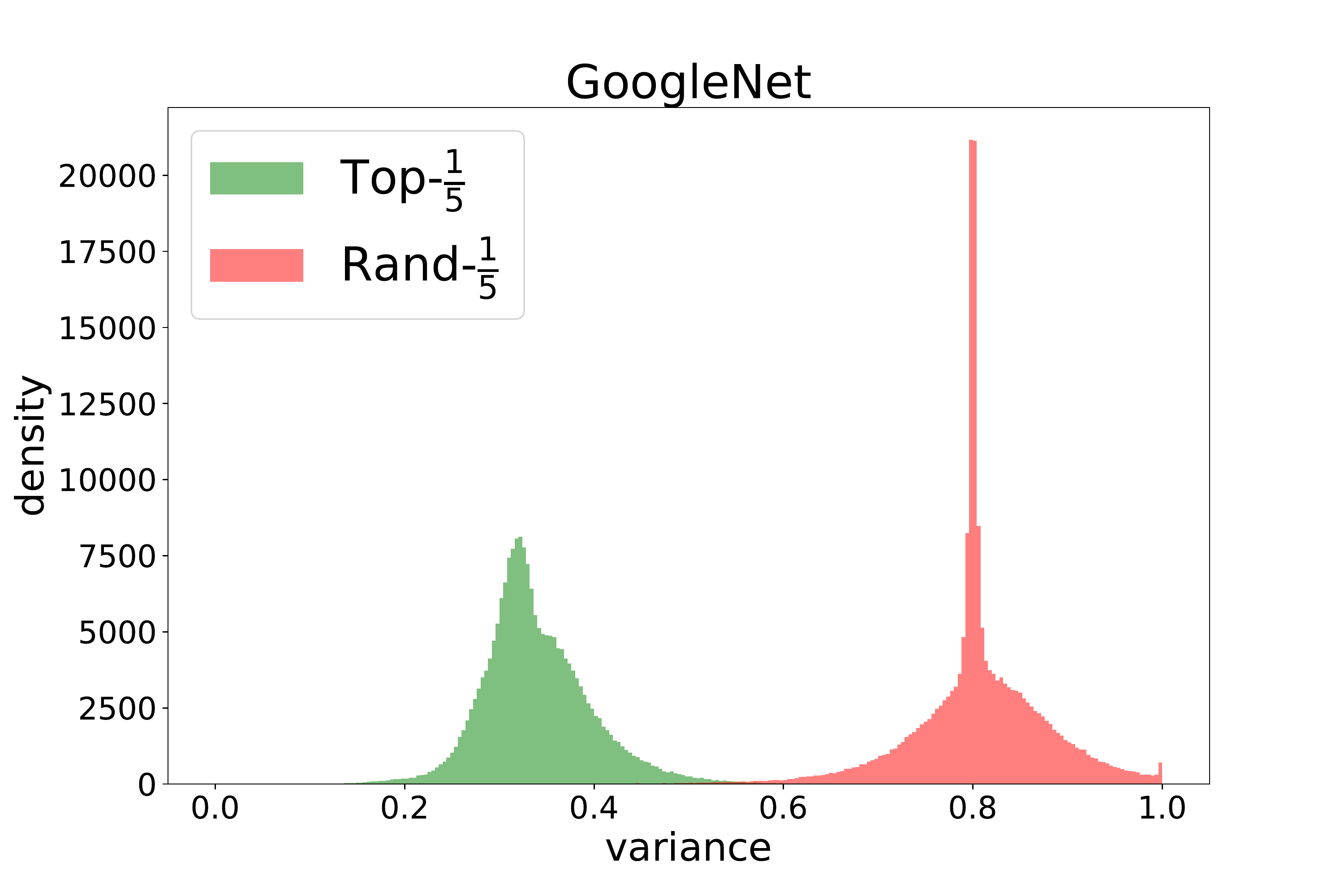}
\includegraphics[width=.3\linewidth]{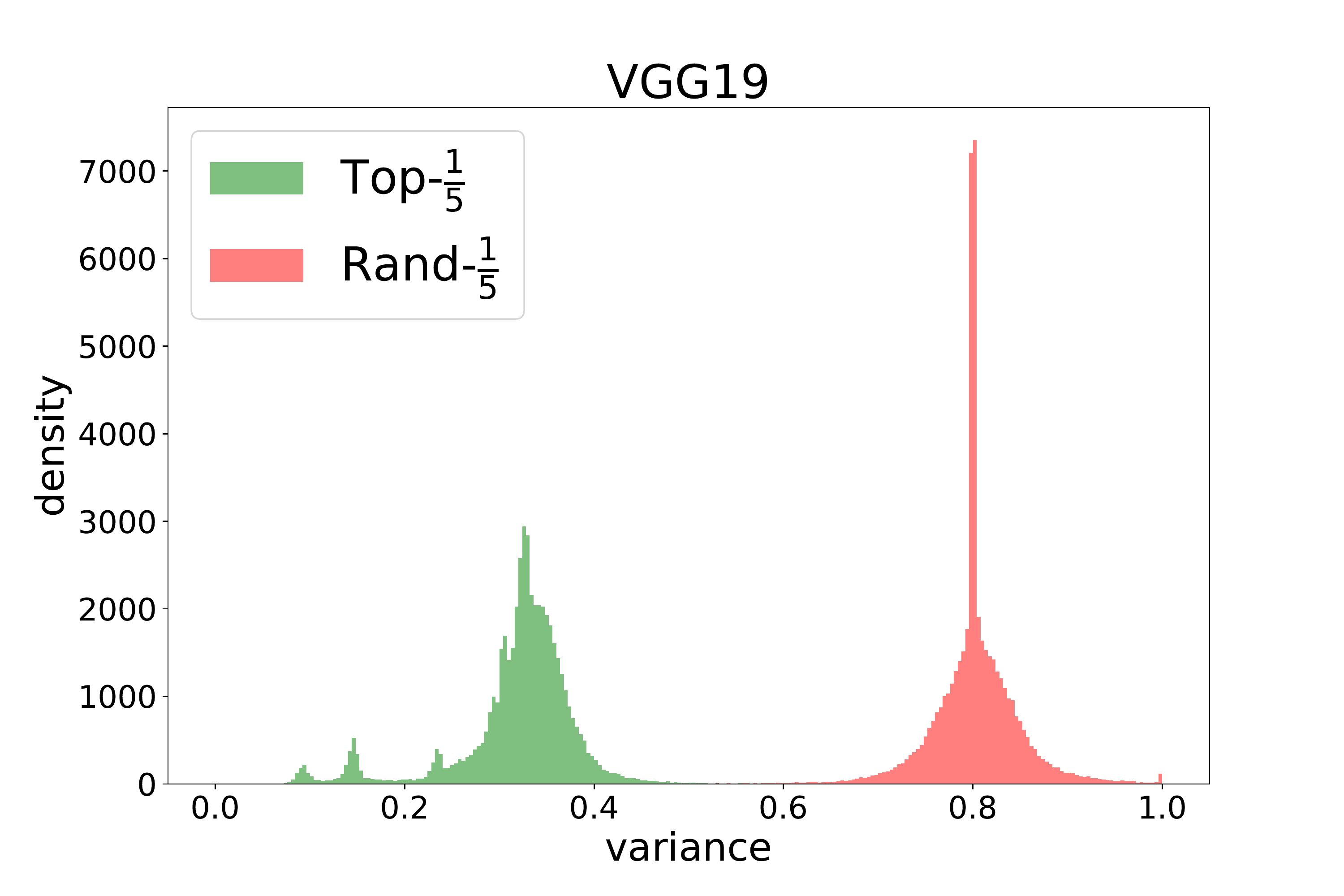}
\caption{Comparison of empirical variance $\norm*{\cC(x) - x}^2/\norm*{x}^2$ during training procedure for two pairs of method-- deterministic with classic/unbiased $\NC$ and Top-$k$ with Rand-$k$ with Top-$\nicefrac{1}{5}$ of coordinates. Both of the plots were produced using ResNet18, GoogleNet, and VGG19 on CIFAR10 dataset.}
\label{fig:emp_variance}
\end{figure}

\subsection{Top-$k$ mixed with natural dithering saves in communication significantly}
Next, we experimentally show the superiority of our newly proposed compressor--Top-$k$ combined with natural dithering. We compare this against current state-of-the-art for low bandwidth approach Top-$k$ for some small $k$. In Figure~\ref{fig:top_k_plus_nat_dit}, we plot comparison of $5$ methods--Top-$k$, Rand-$k$, natural dithering, Top-$k$ combined with natural dithering and plain \texttt{SGD}. We use $2$ levels with infinity norm for natural dithering and $k=5$ for sparsification methods. For all the compression operators, we train VGG11 on CIFAR10 with plain \texttt{SGD} as an optimizer and default step size equal to $0.01$.  We can see that adding natural dithering after Top-$k$ has the same effect as the natural dithering comparing to no compression, which is a significant reduction in communications without almost no effect on convergence or generalization. Using this intuition, one can come to the conclusion that Top-$k$  with natural dithering is the best compression operator for any bandwidth, where we adjust to given bandwidth by adjusting $k$. This exactly matches with our previous theoretical variance estimates displayed in Figure~\ref{fig:vb_comparison}.

\begin{figure}[t!]
\centering
\includegraphics[width=.32\linewidth]{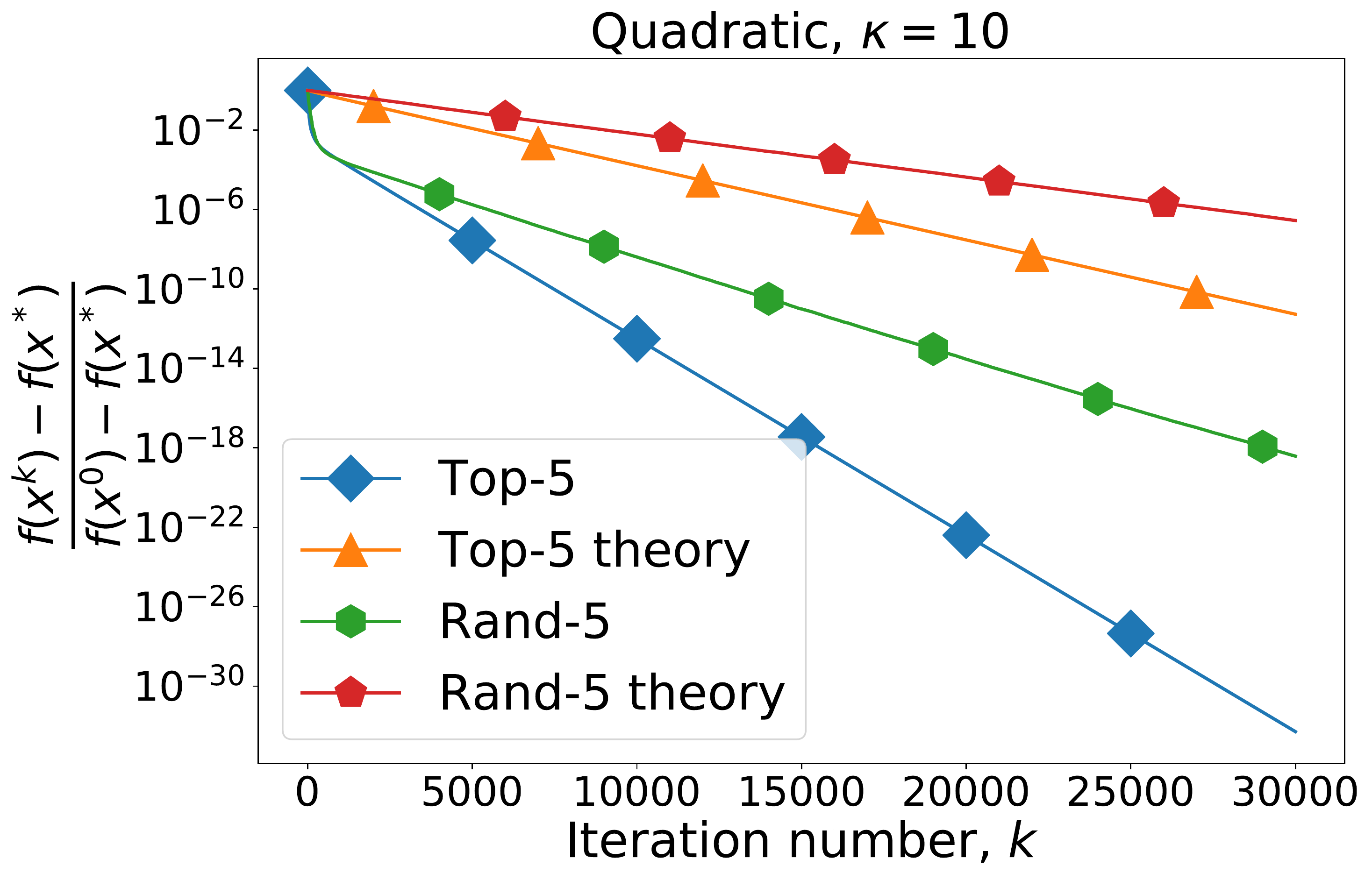}
\includegraphics[width=.32\linewidth]{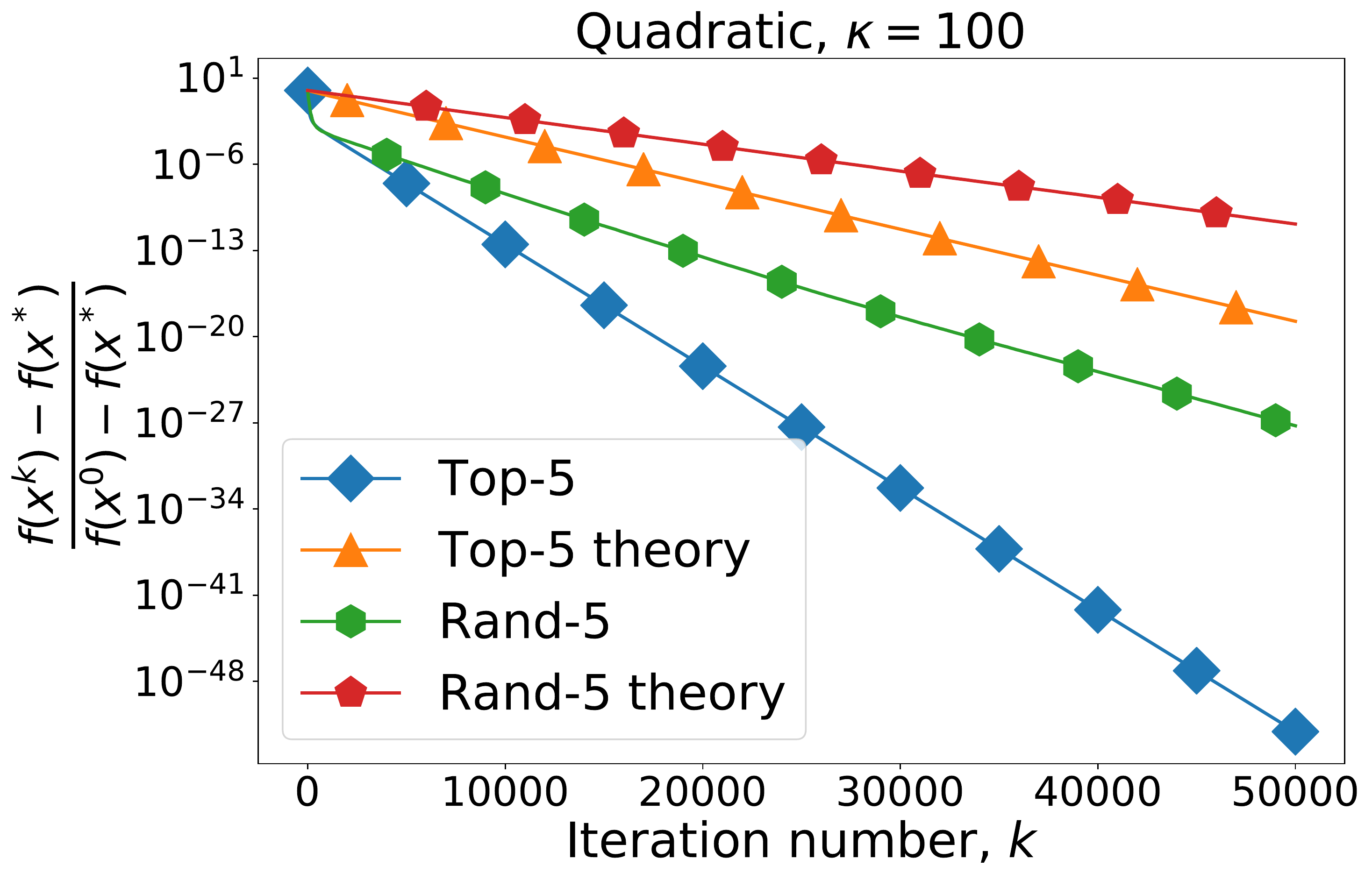}
\includegraphics[width=.32\linewidth]{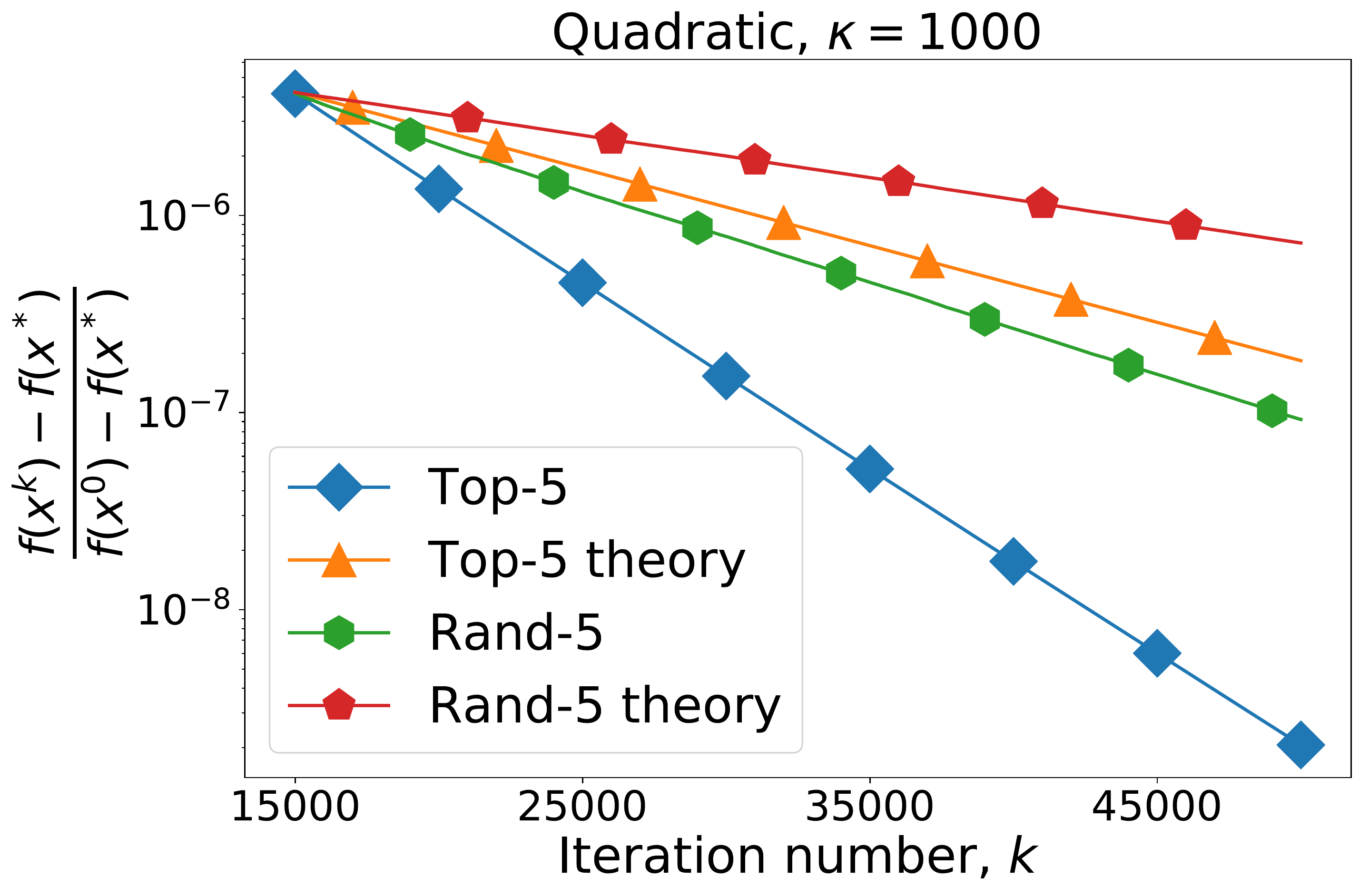}
\caption{Theoretical vs. Practical Convergence of Compressed Gradient Descent on Quadratics problem with different condition number $\kappa$ for Top-5 and Rand-5 compression operators.}
\label{fig:exper7}
\end{figure}

\begin{figure}[t]
\centering
\includegraphics[width=.4\linewidth]{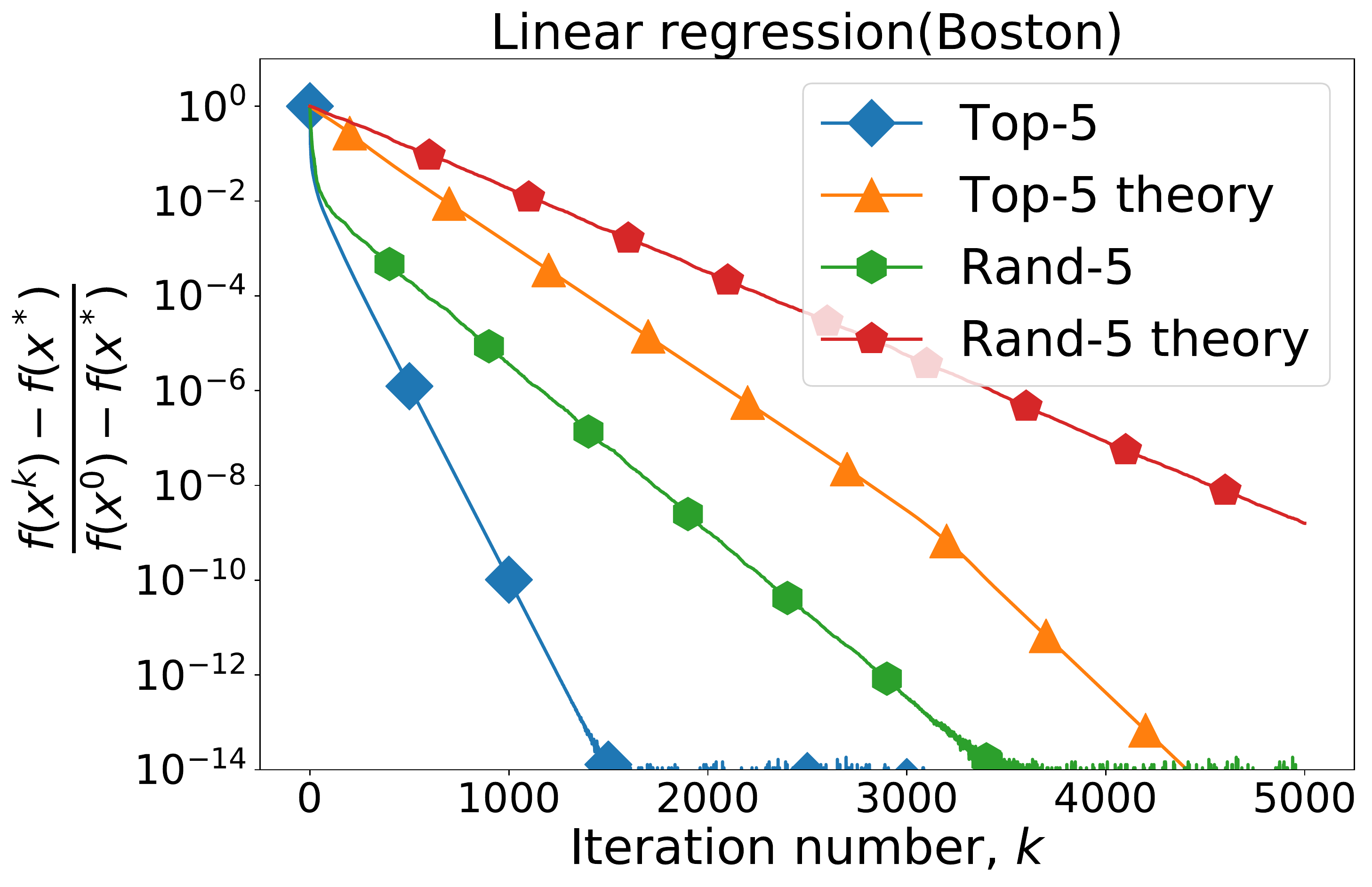}
\includegraphics[width=.4\linewidth]{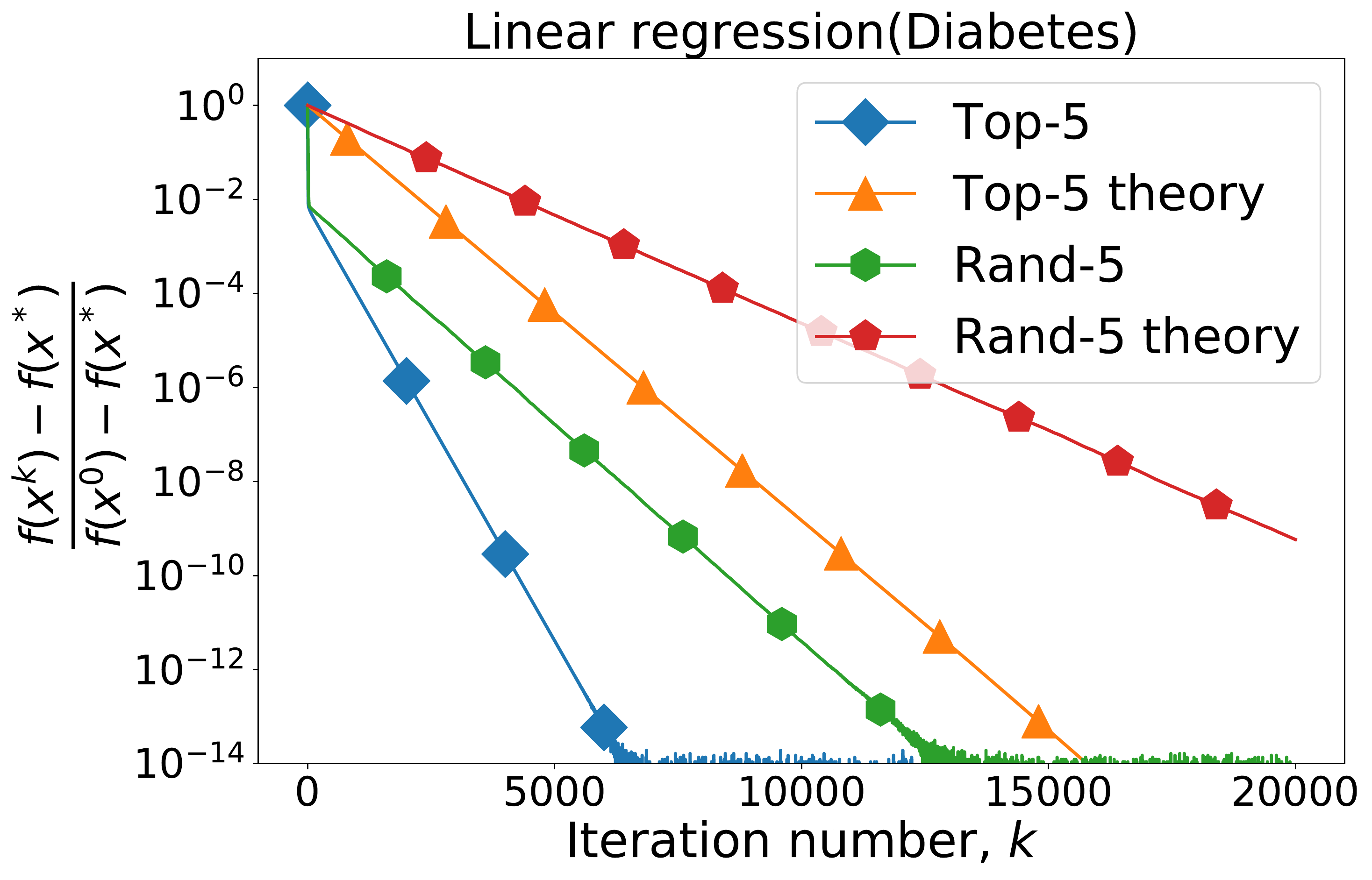}
\caption{Theoretical vs. Practical Convergence of Compressed Gradient Descent on Linear Regression problem for \textit{Boston} and \textit{Diabetes} datasets with Top-5 and Rand-5 compression operators.}
\label{fig:exper8}
\end{figure}

\subsection{Theoretical behavior predicts the actual performance in practice}

In the last experiment, we provide numerical results to further show that our predicted theoretical behavior matches the actual performance observed in practice. We run two regression experiments optimized by gradient descent with step-size $\eta = \frac{1}{L}$. We use a slightly adjusted version of Theorem~\ref{thm:main-III} with adaptive step-sizes, namely
$$
\frac{f(x^k) - f(x^\star)}{f(x^0) - f(x^\star)} \leq  \prod \limits_{i=1}^k \left(1 - \frac{\mu}{L \delta_i}\right),
 $$
where 
$$
1 - \frac{1}{\delta_i} = \frac{\norm*{\cC(\nabla f(x^i)) - \nabla f(x^i)}^2}{\norm*{\nabla f(x^i)}^2}.
$$
Note that this is the direct consequence of our analysis. We apply this property to display the theoretical convergence. In the first experiment depicted in Figure \ref{fig:exper7}, we randomly generate random square matrix $\mA$ of dimension $100$ where it is constructed in the following way: we sample random diagonal matrix $\mD$, which elements are independently sampled from the uniform distribution $(1,10)$,  $(1,100)$, and  $(1,1000)$, respectively. $\mA$ is then constructed using $\mQ^\top \mD \mQ$, where $\mP = \mQ\mR$ is a random matrix and $\mQ\mR$ is obtained using QR-decomposition. The label $y$ is generated the same way from the uniform distribution $(0,1)$. The optimization objective is then 
 $$
 \min_{x \in \R^d} x^\top \mA x - y^\top x.
 $$

For the second experiment shown in Figure \ref{fig:exper8}, we run standard linear regression on two scikit-learn datasets--\textit{Boston} and \textit{Diabetes}--and applied data normalization as the preprocessing step.

Looking into Figures~\ref{fig:exper7} and \ref{fig:exper8}, one can clearly see that as predicted by our theory, biased compression with less empirical variance leads to better convergence in practice and the gap almost matches the improvement.

\chapter{A better alternative to error feedback for communication-efficient distributed learning}
\label{chapter5:induced}

\section{Introduction}

We consider distributed optimization problems of the form
\begin{align} 
\min \limits_{x \in \R^d} \sbr*{ f(x) \eqdef \frac{1}{n} \sum \limits_{i=1}^n f_i(x)}  \,, \label{eq:probR_induced}
\end{align}
where  $x\in \R^d$ represents the weights of a statistical model we wish to train, $n$ is the number of nodes, and $f_i \colon \R^d \to \R$ is a smooth differentiable loss function composed of data stored on worker $i$.  

\section{Related work}
\textbf{Communication Bottleneck.} In distributed training, model updates (or gradient vectors) have to be exchanged in each iteration.  Due to the size of the communicated messages for commonly considered deep models~\cite{qsgd2017neurips}, this represents significant bottleneck of the whole optimization procedure. To reduce the amount of data that has to be transmitted, several strategies were proposed.  

One of the most popular strategies is to incorporate local steps and  communicated updates every few iterations only~\cite{local_SGD_stich_18, lin2018don, stich2020error, karimireddy2020scaffold, bayoumi2020tighter}. Unfortunately, despite their practical success, local methods are poorly understood and their theoretical foundations are currently lacking. Almost all existing error guarantees are dominated by a simple baseline, minibatch \texttt{SGD}~\cite{woodworth2020local}.  

In this chapter, we focus on another popular approach: {\em gradient compression}. In this approach, instead of transmitting the full dimensional (gradient) vector $g \in \R^d$, one transmits a compressed vector $\cC(g)$, where $\cC: \R^d \rightarrow \R^d$ is a (possibly random) operator chosen such that $\cC(g)$ can be represented using fewer bits, for instance by using limited bit representation (quantization) or by enforcing sparsity. A particularly popular class of quantization operators is based on random dithering~\cite{goodall1951television, roberts1962picture}; see  \cite{qsgd2017neurips, terngrad, zhang2016zipml, Cnat, ramezani2019nuqsgd}. Much sparser vectors can be obtained by random sparsification techniques that randomly mask the input vectors and only preserve a constant number of coordinates~\cite{tonko, konevcny2018randomized,stich2018sparsified,mishchenko201999, vogels2019powersgd}. There is also a line of work~\cite{Cnat,basu2019qsparse} in which a combination of sparsification and quantization was proposed to obtain a more aggressive  effect. We will not further distinguish between sparsification and quantization approaches, and refer to all of them as compression operators hereafter. 

Considering both practice and theory, compression operators can be split into two groups: biased and unbiased.  For the unbiased compressors, $\cC(g)$ is required to be an unbiased estimator of the update $g$. Once this requirement is lifted, extra tricks are necessary for Distributed Compressed Stochastic Gradient Descent (\texttt{DCSGD}) \cite{qsgd2017neurips, alistarh2018sparse, khirirat2018distributed} employing such a compressor to work, even if the full gradient is computed by each node. Indeed, the naive approach can lead to exponential divergence~\cite{beznosikov2020biased}, and Error Feedback (EF)~\cite{1bit, karimireddy2019error} is the only known mechanism able to remedy the situation.

\section{Contributions}

Our contributions can be summarized as follows:
    \subsection{Induced compressor}  When used within the stabilizing EF framework,  biased compressors (e.g., Top-$K$) can often achieve superior performance when compared to their unbiased counterparts (e.g., Rand-$K$). This is often attributed to their low variance. However, despite ample research in this area,  EF remains the only known mechanism that allows the use of these powerful biased compressors. Our key contribution is the development of a simple but remarkably effective alternative---and this is the only alternative we know of---which we argue leads to better and more versatile methods both in theory and practice. In particular, we propose a  general construction that can transform any biased compressor, such as Top-$K$,  into an unbiased one for which we coin the name {\em induced compressor} (Section~\ref{sec:construction}). Instead of using the desired biased compressor within EF, our proposal is to  instead use the induced compressor within an appropriately chosen existing method designed for  unbiased compressors, such as distributed compressed \texttt{SGD} (\texttt{DCSGD}) \cite{khirirat2018distributed}, variance reduced \texttt{DCSGD} (\texttt{DIANA}) \cite{mishchenko2019distributed} or accelerated \texttt{DIANA} (\texttt{ADIANA})~\cite{li2020acceleration}. While EF can bee seen as a version of \texttt{DCSGD} which can work with biased compressors, variance reduced nor accelerated variants of EF were not known at the time of writing the original paper~\cite{horvath2021a}.
    \subsection{Better theory for \texttt{DCSGD}} As a secondary contribution, we provide a new and tighter theoretical analysis of \texttt{DCSGD} under weaker assumptions. If $f$ is $\mu$-quasi convex (not necessarily convex) and local functions $f_i$ are $(L, \sigma^2)$-smooth (weaker version of $L$-smoothness with strong growth condition),  we obtain the rate 
    $$\cO \left(\delta_n L r^0 \exp \left[-\frac{\mu T}{4 \delta_n L} \right] + \frac{(\delta_n-1) D + \delta \frac{\sigma^2}{n}}{\mu T}  \right),$$ where $\delta_n = 1 + \frac{\delta-1}{n}$ and $\delta\geq 1$ is the parameter which bounds the second moment of the compression operator, and $T$ is the number of iterations. This rate has linearly decreasing dependence on the number of nodes $n$, which is strictly better than the best-known rate for \texttt{DCSGD} with EF, whose convergence does not improve as the number of nodes increases, which is one of   the main disadvantages of using EF. Moreover, EF requires extra assumptions. In addition, while the best-known rates for EF~\cite{karimireddy2019error,beznosikov2020biased} are expressed in terms of functional values, our theory guarantees convergence in both iterates and functional values.  Another practical implication of our findings is the reduction of the memory requirements by half; this is because in \texttt{DCSGD} one does not need to store the error vector.
    \subsection{Partial participation} We further extend our results to obtain the first convergence guarantee for partial participation with arbitrary distributions over nodes,  which plays a key role in Federated Learning (FL). 
    \subsection{Experimental validation} Finally, we provide an experimental evaluation on an array of classification tasks with CIFAR10 dataset corroborating our theoretical findings.

\section{Error feedback is not a good idea when using unbiased compressors}
\label{sec:1_d_unbiased_better}

In this section we first introduce the notions of unbiased and general compression operators, and then compare Distributed Compressed \texttt{SGD} (\texttt{DCSGD}) without (Algorithm~\ref{alg:UC_SGD}) and with (Algorithm~\ref{alg:EF_SGD}) Error Feedback.

\subsection{Unbiased vs general compression operators}
The following  lemma provides a link between the unbiased and general compression operators as defined in Definitions~\ref{def:omegaquant} and \ref{def:deltaquant}.

\begin{lemma}
\label{lem:subset}
If $\cC \in \U(\delta - 1)$, then \eqref{eq:quant} holds with  $\lambda = \frac{1}{\delta}$, i.e.,  $\cC \in \B(\delta)$. That is, $\U(\delta - 1)\subset \B(\delta)$. 
\end{lemma}

Note that the opposite inclusion to that established in the above lemma does not hold. For instance, the Top-$K$ operator belongs to $\B(\delta)$, but does not belong to $\U(\delta - 1)$.  In the next section we develop a procedure for transforming any mapping $\cC:\R^d\to \R^d$ (and in particular,  any general compressor) into a closely related  {\em induced} unbiased compressor.

\begin{figure}[!t]
\centering
\begin{minipage}[t]{0.49\textwidth}
\begin{algorithm}[H]
\begin{algorithmic}[1]
 \STATE {\bfseries Input:} $\{\eta^k\}_{k=0}^{T} > 0$, $x_0$
 \FOR{$k=0,1,\dots T$}
	\STATE {\bfseries Parallel: Worker side}
  	\FOR{$i=1,\dots,n$ }
		\STATE obtain $g_i^k$
		\STATE  send $\Delta_i^k = \cC^k(g_i^k)$ \;
		\STATE [no need to keep track of errors]
 	\ENDFOR
 	\STATE {\bfseries Master side}
 	\STATE aggregate $ \Delta^k = \frac{1}{n} \sum_{i=1}^n \Delta_i^k$
 	\STATE broadcast $ \Delta^k$ to each worker
 	\STATE {\bfseries Parallel: Worker side}
  	\FOR{$i=1,\dots,n$ }
		\STATE $x^{k+1} = x^k - \eta^{k} \Delta^k$\;
 	\ENDFOR
 \ENDFOR
\end{algorithmic}
\caption{\texttt{DCSGD}}
\label{alg:UC_SGD}
\end{algorithm}
\end{minipage}
\hfill
\begin{minipage}[t]{0.49\textwidth}
\begin{algorithm}[H]
\begin{algorithmic}[1]
 \STATE {\bfseries Input:} $\{\eta^k\}_{k=0}^{T} > 0$, $x_0$, $\cbr{e_i^0}_{i \in [n]}$
 \FOR{$k=0,1,\dots T$}
	\STATE {\bfseries Parallel: Worker side}
  	\FOR{$i=1,\dots,n$ }
		\STATE obtain $g_i^k$
		\STATE  send $\Delta_i^k = \cC^k(\eta^k g_i^k + e_i^k)$\;
		\STATE  $e_i^{k+1} = \eta^k g_i^k + e_i^k -  \Delta_i^k$\;
 	\ENDFOR
 	\STATE {\bfseries Master side}
 	\STATE aggregate $ \Delta^k = \frac{1}{n} \sum_{i=1}^n \Delta_i^k$
 	\STATE broadcast $ \Delta^k$ to each worker
 	\STATE {\bfseries Parallel: Worker side}
  	\FOR{$i=1,\dots,n$ }
		\STATE $x^{k+1} = x^k -  \Delta^k$\;
 	\ENDFOR
 \ENDFOR
\end{algorithmic}
\caption{\texttt{DCSGD} w/ Error Fedback}
\label{alg:EF_SGD}
\end{algorithm}
\end{minipage}
\end{figure}

\subsection{Distributed \texttt{SGD} with vs without error feedback}
In the rest of this section, we compare the convergence rates for \texttt{DCSGD} (Algorithm~\ref{alg:UC_SGD}) and \texttt{DCSGD} with EF (Algorithm~\ref{alg:EF_SGD}).  We do this comparison under standard assumptions~\cite{karimi2016linear, bottou2018optimization, necoara2019linear, gower2019sgd, stich2019unified, stich2020error}, listed next. 


\begin{assumption}[$\mu$-quasi convexity]
\label{ass:1}
$f$ is $\mu$-quasi convex.
%
\end{assumption}

\begin{assumption}[unbiased gradient oracle]
\label{ass:2} The stochastic gradient used in Algorithms~\ref{alg:UC_SGD} and \ref{alg:EF_SGD} are unbiased.
%
\end{assumption}


\begin{assumption}[$(L,\sigma^2)$-expected smoothness]
\label{ass:3}
Function $f$ is $(L,\sigma^2)$-smooth if there exist constants $L>0$  and $\sigma^2 \geq 0$ such that $\forall i \in [n]$  and $\forall x^k \in \R^d$ 
\begin{equation}
\label{eq:L_smooth_f_i}
\E{\norm*{g^k_i}^2} \leq 2L (f_i(x^k)-f_i^\star) + \sigma^2,
\end{equation}
\begin{equation}
\label{eq:L_smooth_f}
\E{\norm*{\frac1n \sum_{i=1}^n g^k_i}^2} \leq 2L (f(x^k)-f^\star) + \frac{\sigma^2}{n},
\end{equation}
where $f_i^\star$ is the minimum functional value of $f_i$ and $[n] = \{1,2, \dots, n\}$.
\end{assumption}
This assumption generalizes standard smoothness and boundedness of variance assumptions.   
For more details and discussion, see the works of \cite{gower2019sgd, stich2019unified}. Equipped with these assumptions, we are ready to proceed with the convergence theory.  

\begin{theorem}[Convergence of \texttt{DCSGD}]
\label{thm:u_n}
Consider the \texttt{DCSGD} algorithm  with $n\geq 1$ nodes. Let Assumptions~\ref{ass:1}--\ref{ass:3} hold and $\cC \in \U(\delta - 1)$. Let $D \eqdef \frac{2L}{n}\sum_{i=1}^n (f_i(\xs) - f _i^\star)$ and $\delta_n =  \frac{\delta - 1}{n} + 1 $. Then there exist stepsizes $\eta^k  \leq \frac{1}{2\delta_n L}$ and weights $w^k \geq 0$ such that for all $T\geq 1$ we have
\begin{equation*}
\label{eq:conv_u_n}
 \E{f(\bar{x}^T) - \fs} + \mu \E{\norm*{x^T - \xs}^2}\leq  64 \delta_n L r^0 \exp \left[-\frac{\mu T}{4 \delta_n L} \right] + 36 \frac{(\delta_n -1) D +  \nicefrac{\delta\sigma^2}{n}}{\mu T} \,,
\end{equation*}
where $r^0 = \norm*{x^0 - \xs}^2$, $W^T = \sum_{k=0}^T w^k$, and $\Prob(\bar{x}^T = x^k) = \nicefrac{w^k}{W^T}$.
\end{theorem}

If $\delta = 1$ (no compression), Theorem~\ref{thm:u_n} recovers the optimal rate of Distributed \texttt{SGD}~\cite{stich2019unified}. If $\delta > 1$, there is an extra term $(\delta_n - 1)D$ in the convergence rate, which appears due to heterogenity of data ($\sum_{i = 1}^n \nabla f_i (x^\star) = 0$, but $\sum_{i = 1}^n \cC(\nabla f_i (x^\star)) \neq 0$ in general). In addition, the rate is negatively affected by extra variance due to presence of compression which leads to $L \rightarrow \delta_n L$ and $\nicefrac{\sigma^2}{n} \rightarrow \nicefrac{\delta\sigma^2}{n}$.

Next we compare our rate to the best-known result for Error Feedback~\cite{stich2020error}~($n=1$), \cite{beznosikov2020biased}~($n \geq 1$) used with $\cC\in \U(\delta - 1)\subset \B(\delta)$
$$
\E{f(\bar{x}^T) - \fs} = \tilde{\cO} \lp \delta L r^0 \exp \left[-\frac{\mu T}{ \delta L} \right] + \frac{\delta D +  \sigma^2}{\mu T} \rp
$$
One can note several disadvantages of Error Feedback (Algorithm~\ref{alg:EF_SGD}) with respect to plain \texttt{DCSGD} (Algorithm~\ref{alg:UC_SGD}). The first major drawback is that the effect of compression $\delta$ is not reduced with an increasing number of nodes. Another disadvantage is that Theorem~\ref{thm:u_n} implies convergence for both the functional values and the last iterate, rather than for functional values only as it is the case for EF. On top of that,  our rate of \texttt{DCSGD} as captured by Theorem~\ref{thm:u_n} does not contain any hidden polylogarithmic factor comparing to EF. Another practical supremacy of \texttt{DCSGD} is that there is no need to store an extra vector for the error, which reduces the storage costs by a factor of two, making Algorithm~\ref{alg:UC_SGD} a viable choice for Deep Learning models with millions of parameters. Finally, one does not need to assume standard $L$-smoothness in order to prove convergence in Theorem~\ref{thm:u_n}, while, one the other hand, $L$-smoothness is an important building block for proving convergence for general compressors due to the presence of bias~\cite{stich2020error, beznosikov2020biased}. The only term in which EF might outperform plain \texttt{DCSGD} is $\cO(\nicefrac{\sigma^2}{\mu T})$ for which the corresponding term is $\cO(\nicefrac{\delta\sigma^2}{n\mu T})$. This is due to the fact that EF compensates for the error, while standard compression introduces extra variance. Note that this is not major issue as it is reasonable to assume $\nicefrac{\delta}{n} =  \cO(1)$ or, in addition, $\sigma^2 = 0$ if weak growth condition holds~\cite{vaswani2018fast}, which is quite standard assumption, or one can remove effect of $\sigma^2$ by either computing full gradient locally or by incorporating variance reduction such as \texttt{SVRG}~\cite{johnson2013accelerating}.  In Section~\ref{sec:linear_covergence}, we also discuss the way how to remove the effect of $D$ in Theorem~\ref{thm:u_n}. Putting all together, this suggests that standard \texttt{DCSGD} (Algorithm~\ref{alg:UC_SGD}) is strongly preferable, in theory, to \texttt{DCSGD} with Error Feedback (Algorithm~\ref{alg:EF_SGD}) for $\cC \in \U(\delta - 1)$.

\section{Induced compressor: Fixing bias with error-compression}
\label{sec:construction}

In the previous section, we showed that compressed \texttt{DCSGD} is theoretically preferable to \texttt{DCSGD} with Error Feedback for $\cC \in \U(\delta - 1)$. Unfortunately,  $\B(\delta) \not \subset \U(\delta - 1)$, an example being the Top-$K$ compressor~\cite{alistarh2018sparse, stich2018sparsified}. This compressors belongs to $\B(\frac{d}{K})$, but does not belong to $\U(\delta - 1)$ for any $\delta$. On the other hand, multiple unbiased alternatives to Top-$K$ have been proposed in the literature, including gradient sparsification~\cite{tonko} and adaptive  random sparsification~\cite{beznosikov2020biased}.

\subsection{Induced compressor}
We now propose a {\em general mechanism for constructing an unbiased compressor $\cC\in \U$ from  any biased compressor $\cC_1 \in \B$.} We shall argue that it is preferable to use this {\em induced compressor} within DCSGD, in both theory and practice, to using the original biased compressor $\cC_1$ within DCSGD + Error Feedback.

\begin{theorem}
\label{thm:biased_to_unbiased}
For $\cC_1 \in \B(\delta_1)$ with $\lambda = 1$, choose $\cC_2 \in \U(\delta_2)$ and define the induced compressor via $$\cC(x) \eqdef \cC_1(x) + \cC_2(x - \cC_1(x)).$$ The induced compression operator satisfies $\cC \in \U(\delta - 1)$ with $\delta = \delta_2 \lp 1 - \nicefrac{1}{\delta_1}\rp + \nicefrac{1}{\delta_1}$.
\end{theorem}

To get some intuition about this procedure,  recall the structure used in Error Feedback. The gradient estimator is first compressed with $\cC_1(g)$ and the error $ e = g - \cC_1(g)$ is stored in memory and used to modify the gradient in the next iteration. In our proposed approach, instead of storing the error $e$, we compress it with an unbiased compressor $\cC_2$ (which can be seen as a parameter allowing flexibility in the design of the induced compressor) and communicate {\em both} of these compressed vectors. Note that this procedure results in extra variance as we do not work with the exact error, but with its unbiased estimate only.  On the other hand, there is no bias and error accumulation that one needs to correct for. In addition, due to our construction, at least the same amount of information is sent to the master as in the case of plain $\cC_1(g)$: indeed, we send both $\cC_1(g)$ and $\cC_2(e)$. The drawback of this  is the necessity to send more bits. However, Theorem~\ref{thm:biased_to_unbiased} provides the freedom in generating the induced compressor through the choice of the unbiased compressor $\cC_2$. In theory, it makes sense to choose $\cC_2$ with similar compression factor to the compressor $\cC_1$ we are transforming as this way the total number of communicated bits per iteration is preserved, up to the factor of two. 

{\em Remark:} The $\mbox{rtop}_{k_1, k_2} (x,y)$ operator proposed by \cite{elibol2020variance} can be seen as a special case of our induced compressor with $x=y$, $\cC_1 = \mbox{Top-}k_1$ and $\cC_2 = \mbox{Rand-}k_2$.

\subsection{Benefits of the induced compressor}
In the light of the results in Section~\ref{sec:1_d_unbiased_better}, we argue that one should always prefer unbiased compressors to biased ones as long as their variances $\delta$ and communication complexities are the same, e.g., Rand-$K$ over Top-$K$. In practice, biased/greedy compressors are in some settings observed to perform better due to their lower empirical variance~\cite{beznosikov2020biased}.  These considerations give a practical significance to  Theorem~\ref{thm:biased_to_unbiased} as we demonstrate on the following example. Let us consider two compressors: one biased $\cC_1 \in \B(\delta_1)$ and one unbiased $\cC_2 \in \U(\delta_2)$, such that $\delta_1 = \delta_2 = \delta$, having identical communication complexity, e.g., Top-$K$ and Rand-$K$. The induced compressor $\cC(x) \eqdef \cC_1(x) + \cC_2(x - \cC_1(x))$ belongs to $\U(\delta_3)$, where $\delta_3 = \delta - \left(1 - \frac1\delta \right) < \delta.$ While the size of the transmitted message is doubled,  one can use Algorithm~\ref{alg:UC_SGD} since $\cC$ is unbiased,  which provides  better convergence guarantees than Algorithm~\ref{alg:EF_SGD}. Based on the construction of the induced compressor, one might expect that we need extra memory as ``the error'' $ e = g - \cC_1(g)$ needs to be stored, but during computation only.  This is not an issue as compressors for DNNs are always applied layer-wise~\cite{dutta2019discrepancy}, and hence the size of the extra memory is negligible. It does not help EF, as the error needs to be stored at any time for each layer.

\section{Extensions}

We now develop several extensions of Algorithm~\ref{alg:UC_SGD} relevant to distributed optimization in general, and to Federated Learning in particular. This is all possible due to the simplicity of our approach. Note that in the case of Error Feedback, these extensions have either not been obtained yet, or similarly to Section~\ref{sec:1_d_unbiased_better}, the results are worse when compared to our derived bounds for unbiased compressors.

\subsection{Partial participation with arbitrary distribution over nodes}
In this section, we extend our results to a variant of \texttt{DCSGD} utilizing {\em partial participation}, which is of key relevance to Federated Learning. We assume clients' partial participation framework as defined in Section~\ref{sec:partial_participation}.

The following theorem establishes the convergence rate for Algorithm~\ref{alg:UC_SGD} with partial participation.

\begin{theorem}
\label{thm:u_n_p}
Let Assumptions~\ref{ass:1}--\ref{ass:3} hold and $\cC \in \U(\delta - 1)$, then there exist stepsizes $\eta^k  \leq \frac{1}{2\delta_\Sam L}$ and weights $w^k \geq 0$ such that
\begin{equation*}
\label{eq:conv_u_n_p}
 \E{f(\bar{x}^T) - \fs} + \mu \E{\norm*{x^T - \xs}^2}\leq  64 \delta_\Sam L r^0 \exp \left[-\frac{\mu T}{4 \delta_\Sam L} \right] + 36\frac{(\delta_\Sam - 1) D + \lp 1 + a_\Sam \rp \nicefrac{\delta\sigma^2}{n}}{\mu T} \,,
\end{equation*}
where $r^0, W^T, \bar{x}^T$, and $D$ are defined in Theorem~\ref{thm:u_n}, $a_\Sam = \max_{i \in [n]}\{\nicefrac{v_i}{p_i}\}$, and $\delta_\Sam = \frac{\delta a_\Sam + (\delta-1)}{n} + 1$.
\end{theorem}

For the case $\Sam = [n]$ with probability $1$, one can show that Lemma~\ref{LEM:UPPERV} holds with $v = 0$, and hence we exactly recover the results of Theorem~\ref{thm:u_n}. In addition, we can quantify the slowdown factor with respect to full participation regime (Theorem~\ref{thm:u_n}), which is $\delta\max_{i \in [n]} \frac{v_i}{p_i}$. While in our framework we assume the distribution $\Sam$ to be fixed,  it can be easily extended to several proper distributions $\Sam_j$'s or we can even handle a block-cyclic structure with each block having an arbitrary proper distribution $\Sam_j$ over the given block $j$ combining our analysis with the results of~\cite{eichner2019semi}.

\subsection{Obtaining linear convergence}
\label{sec:linear_covergence}
Note that in all the previous theorems, we can only  guarantee a sublinear $\cO(\nicefrac{1}{T})$  convergence rate. Linear rate is obtained in the special case when $D = 0$ and $\sigma^2 = 0$. The first condition is satisfied, when $f_i^\star = f_i(x^\star)$ for all $i \in [n]$, thus when $x^\star$ is also minimizer of every local function $f_i$.  Furthermore, the effect od $D$ can be removed using compression of gradient differences, as pioneered in the DIANA algorithm~\cite{mishchenko2019distributed}.
Note that $\sigma^2 = 0$ if weak growth condition holds~\cite{vaswani2018fast}. Moreover, one can remove effect of $\sigma^2$ by either computing full gradients locally or by incorporating variance reduction such as \texttt{SVRG}~\cite{johnson2013accelerating}. It was shown by~\cite{Cnat} that both $\sigma^2$ and $D$ can be removed for the setting of Theorem~\ref{thm:u_n}. These results can be easily extended to partial participation using our proof technique for Theorem~\ref{thm:u_n_p}. Note that this reduction is not possible for Error Feedback as the analysis of the \texttt{DIANA} algorithm is heavily dependent on the unbiasedness property. This points to another advantage of the induced compressor framework introduced in Section~\ref{sec:construction}.

\subsection{Acceleration}
We now comment on the combination of compression and acceleration/momentum. This setting is very important to consider as essentially all state-of-the-art methods for training deep learning models, including Adam \cite{kingma2014adam, reddi2019convergence}, rely on the use of momentum in one form or another. One can treat the unbiased compressed gradient as a stochastic gradient \cite{gorbunov2020unified} and the theory for momentum \texttt{SGD}~\cite{yang2016unified, gadat2018stochastic, loizou2017momentum} would be applicable with an extra smoothness assumption. Moreover, it is possible to remove the variance caused by stochasticity and obtain linear convergence with an accelerated rate, which leads to the Accelerated \texttt{DIANA} method~\cite{li2020acceleration}. Similarly to our previous discussion, both of these techniques are heavily dependent on the unbiasedness property. It is an intriguing question, but out of the scope of the paper, to investigate the combined effect of momentum and Error Feedback and see whether these techniques are compatible theoretically.

\begin{figure}[t]
\centering
\includegraphics[width=0.35\textwidth]{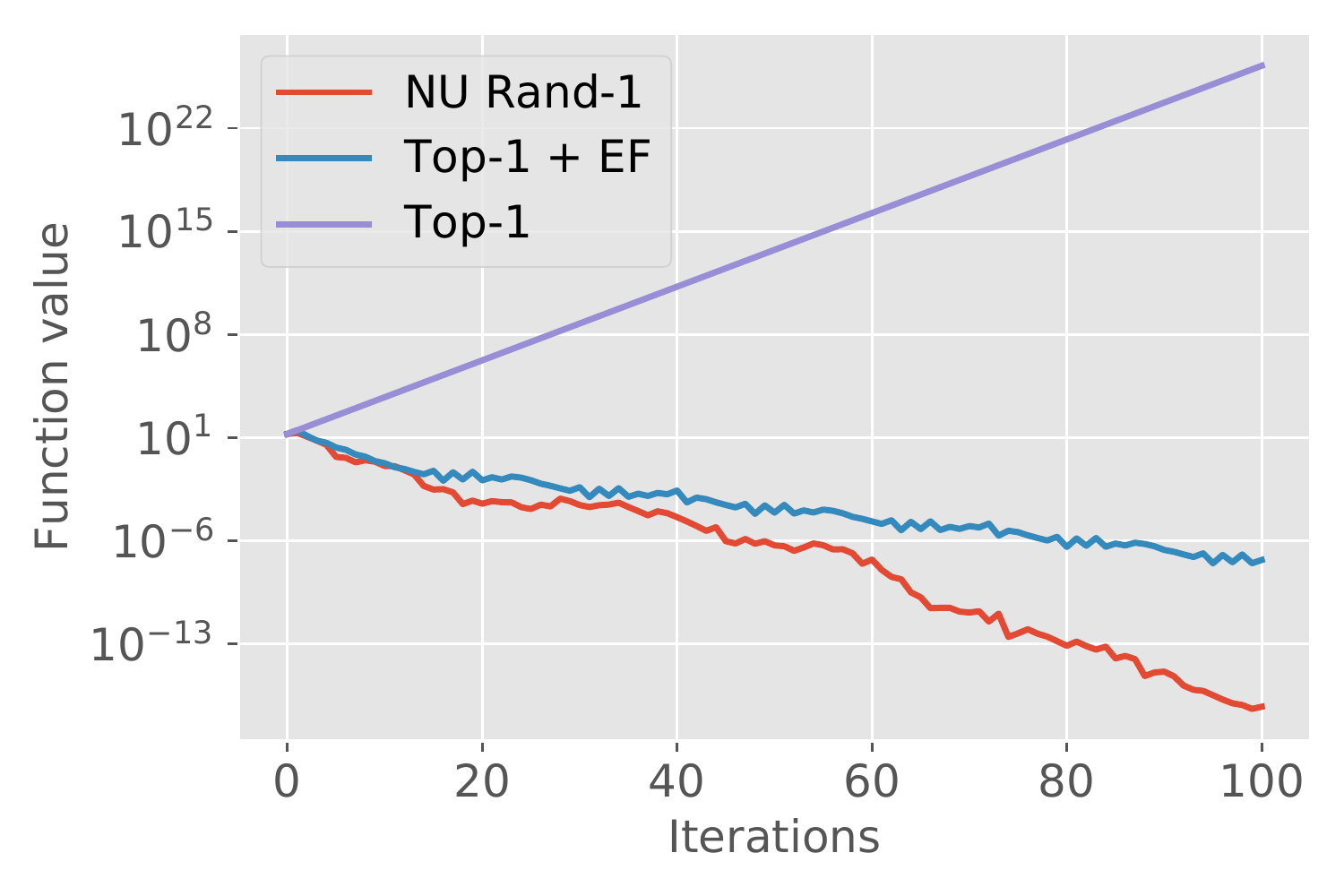}
\includegraphics[width=0.35\textwidth]{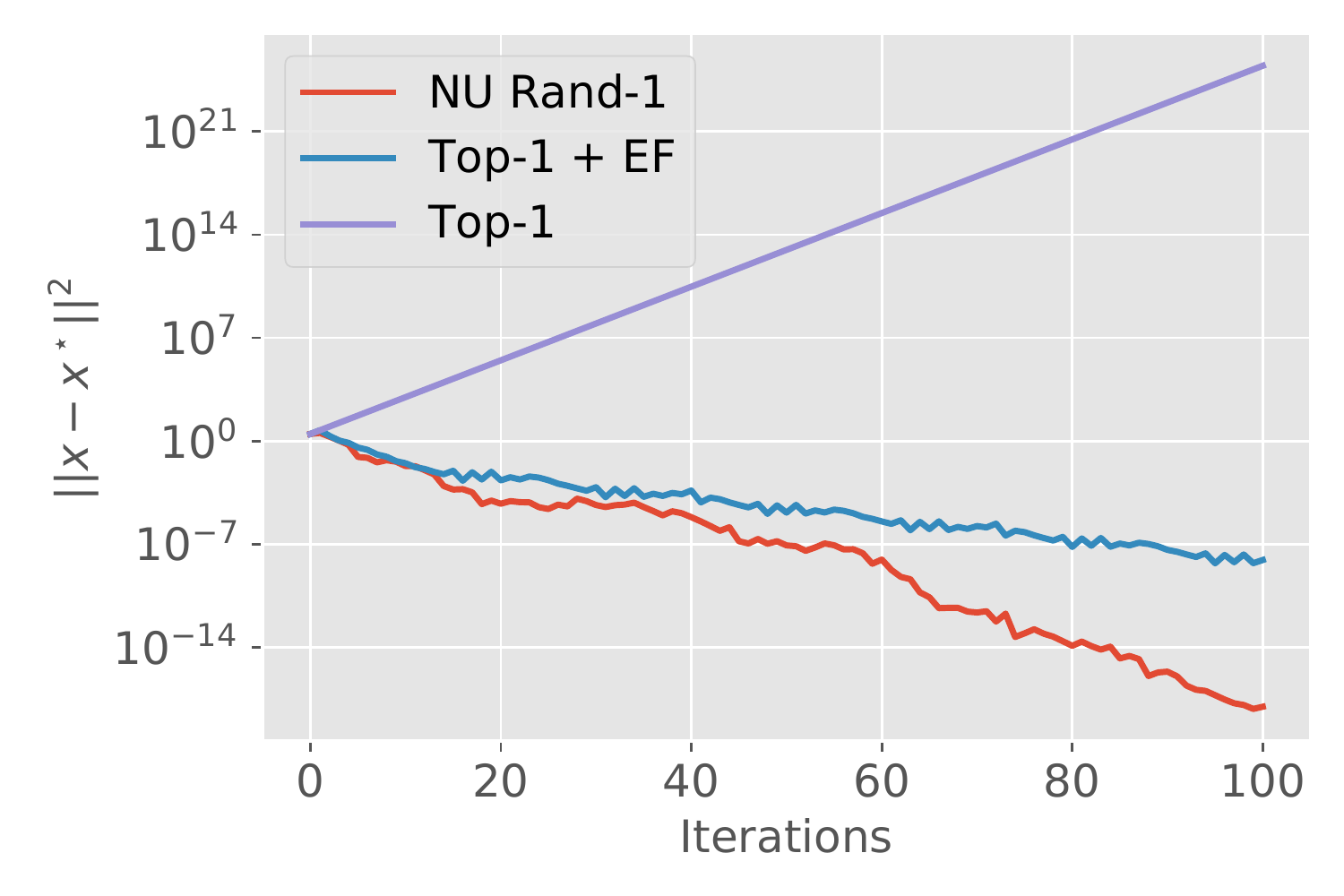}
\caption{Comparison of Top-$1$ (+ EF) and  NU Rand-$1$ on Example 1 from~\cite{beznosikov2020biased}.}
\label{fig:exp4}
\end{figure}

\section{Experiments}

\begin{figure}[t]
\center
\includegraphics[width=0.36\textwidth]{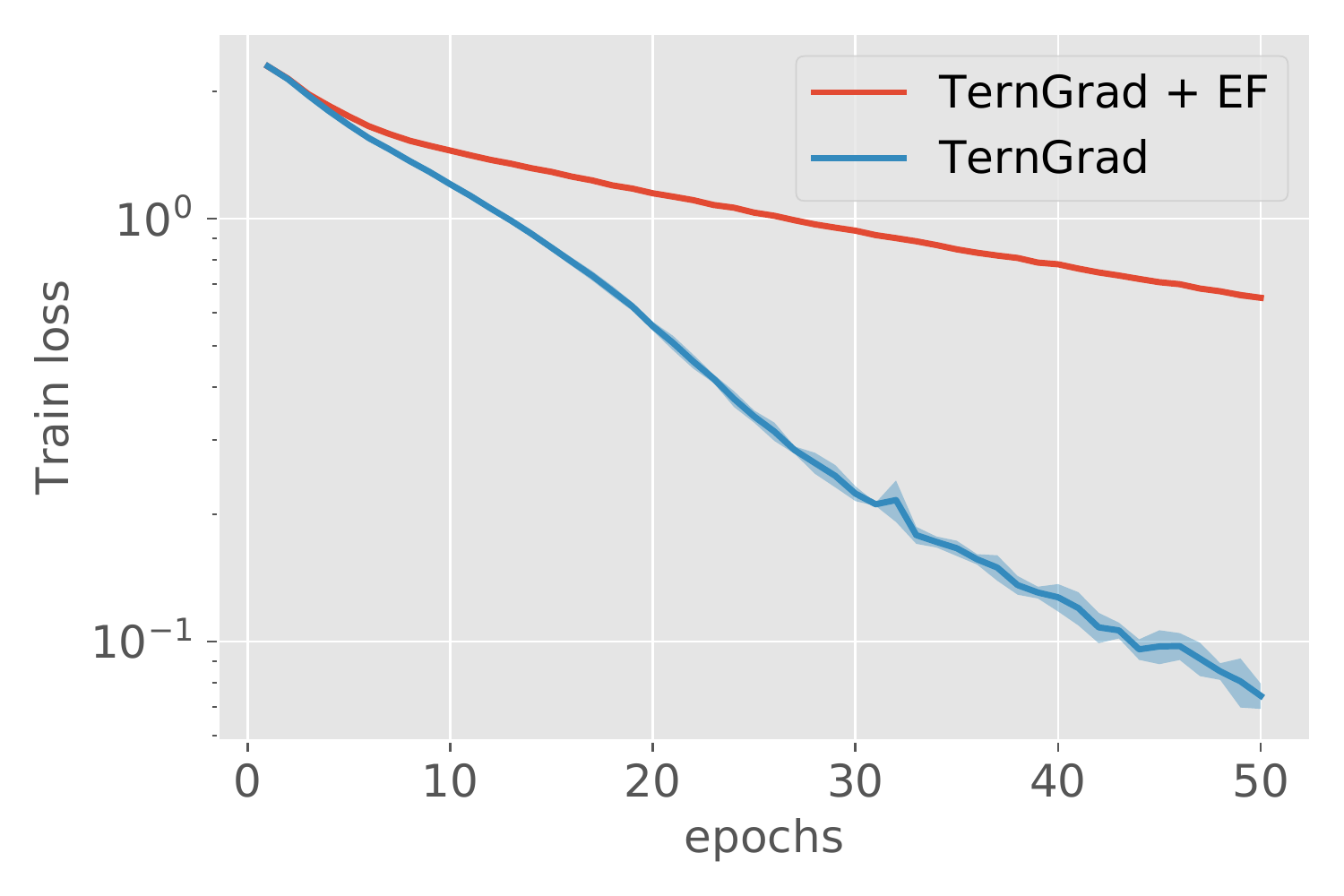}
\includegraphics[width=0.36\textwidth]{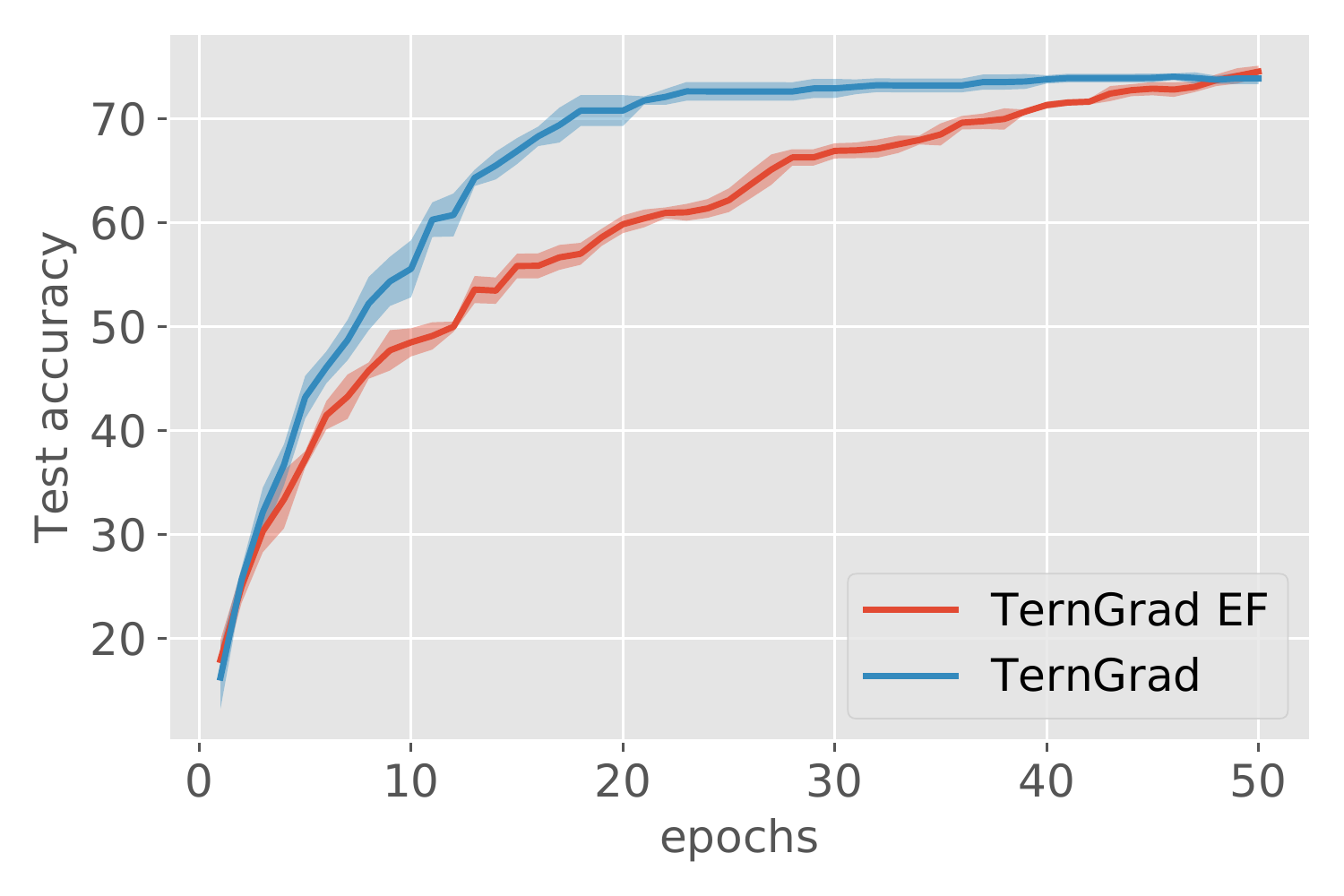} \\
\includegraphics[width=0.36\textwidth]{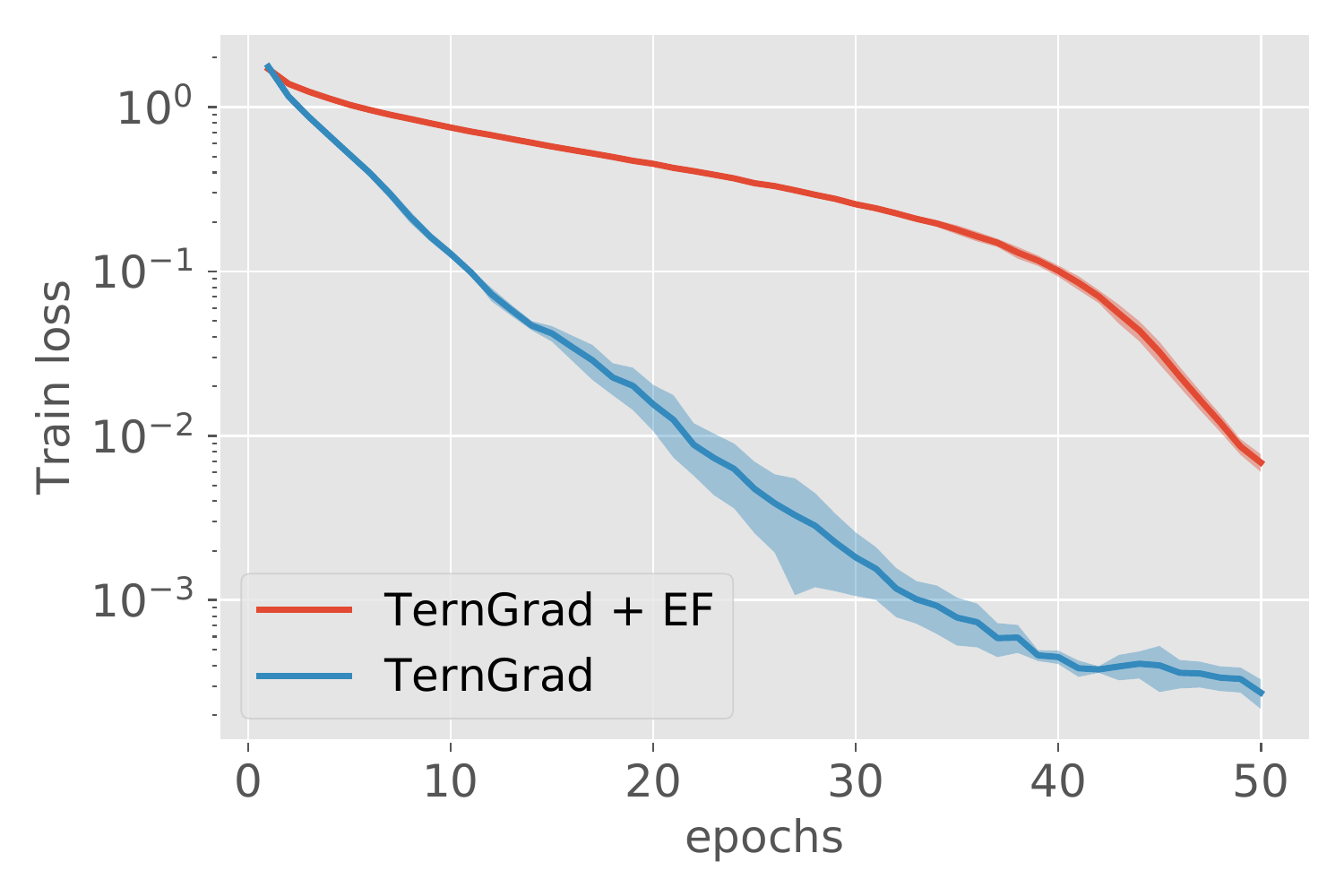}
\includegraphics[width=0.36\textwidth]{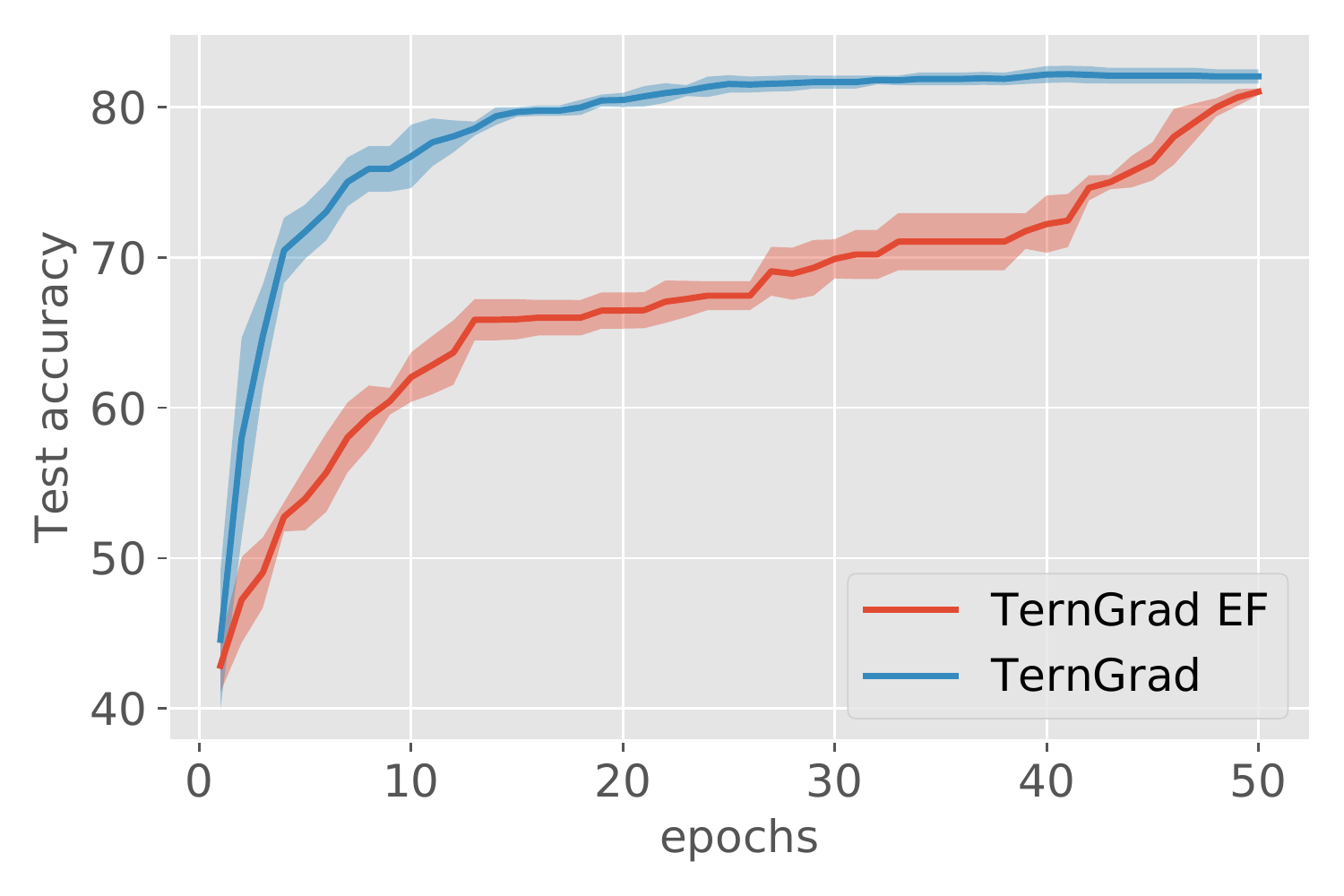}
\caption{Algorithm~\ref{alg:UC_SGD} vs. Algorithm~\ref{alg:EF_SGD} on CIFAR10 with ResNet18 (bottom), VGG11 (top) and TernGrad as a compression.}
\label{fig:exp1}
\end{figure}

In this section, we compare Algorithms~\ref{alg:UC_SGD} and~\ref{alg:EF_SGD} for several compression operators. If the method contains `` + EF '', it means that EF is applied, thus Algorithm~\ref{alg:EF_SGD} is applied. Otherwise, Algorithm~\ref{alg:UC_SGD} is displayed. To be fair, we always compare methods with the same communication complexity per iteration. All experimental details can be found in the Appendix.

\begin{figure}[t]
\centering
\includegraphics[width=0.36\textwidth]{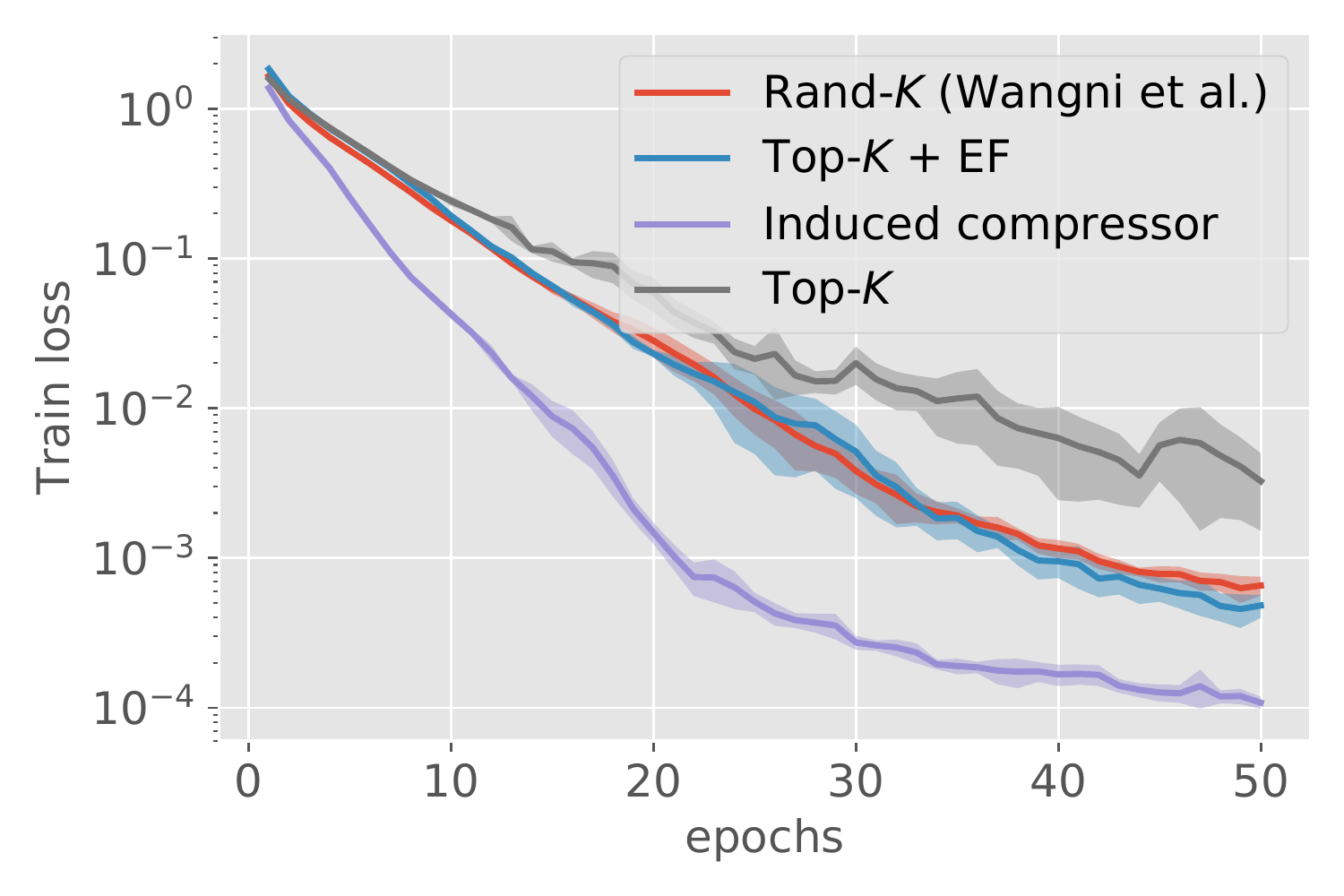}
\includegraphics[width=0.36\textwidth]{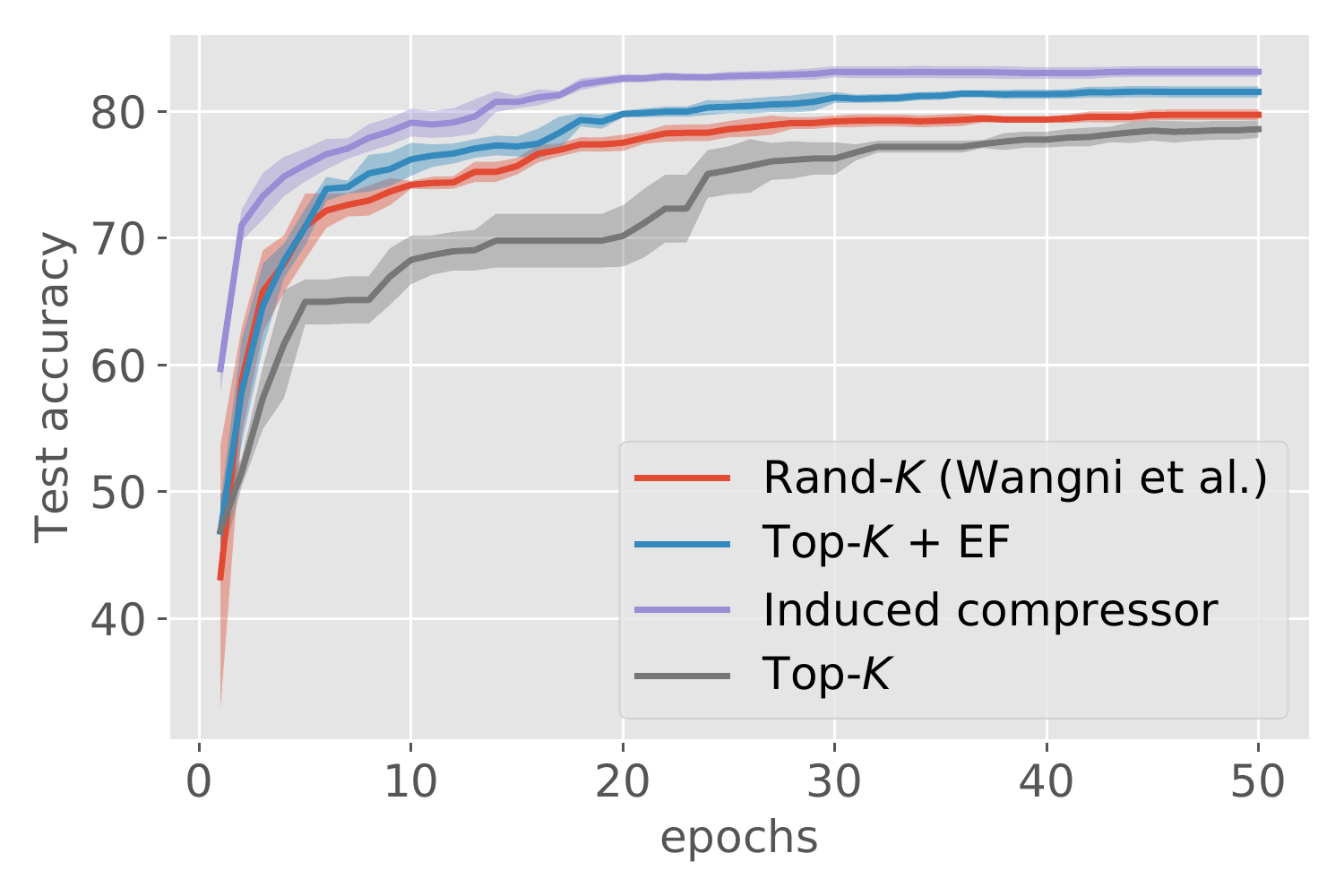} \\
\includegraphics[width=0.36\textwidth]{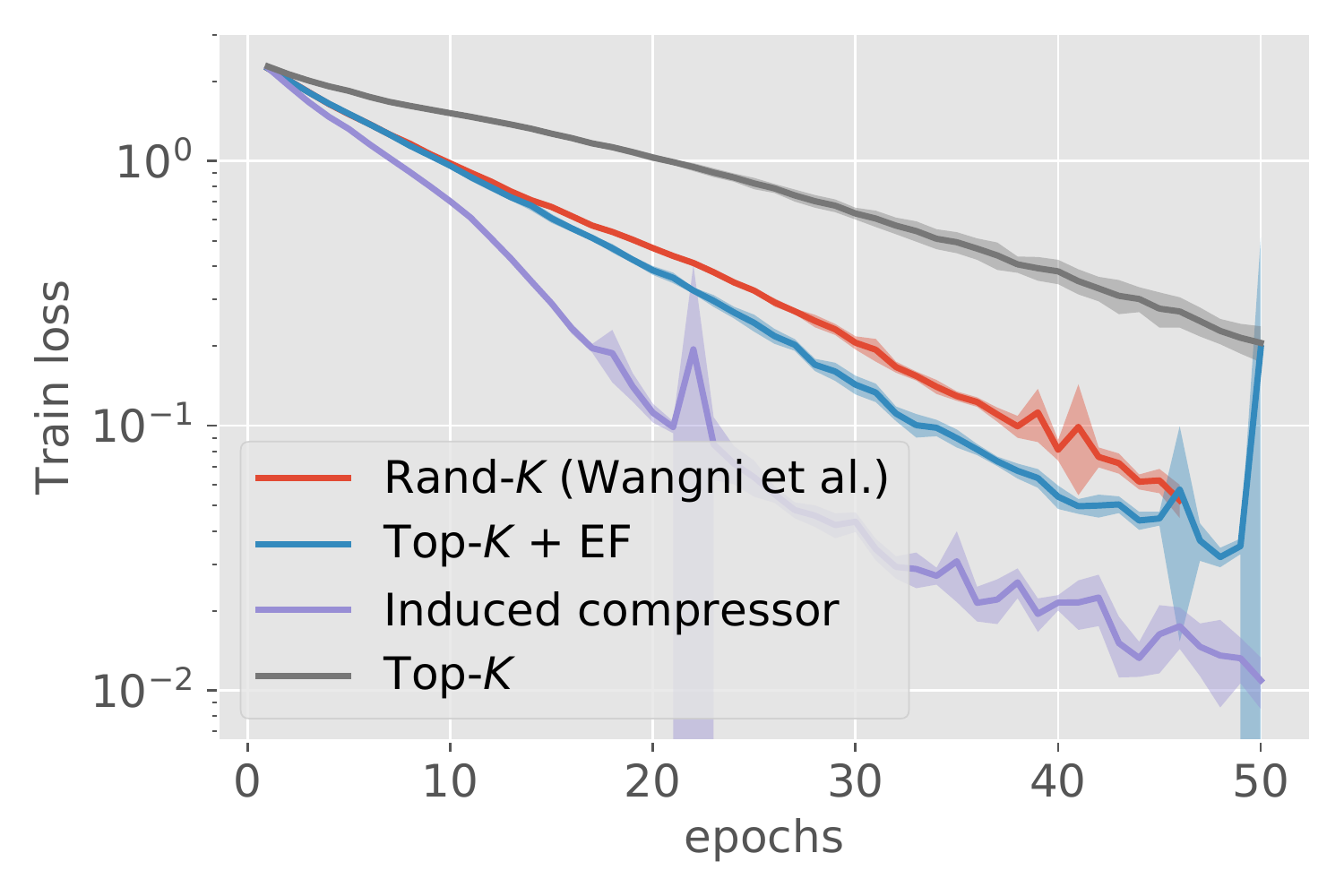}
\includegraphics[width=0.36\textwidth]{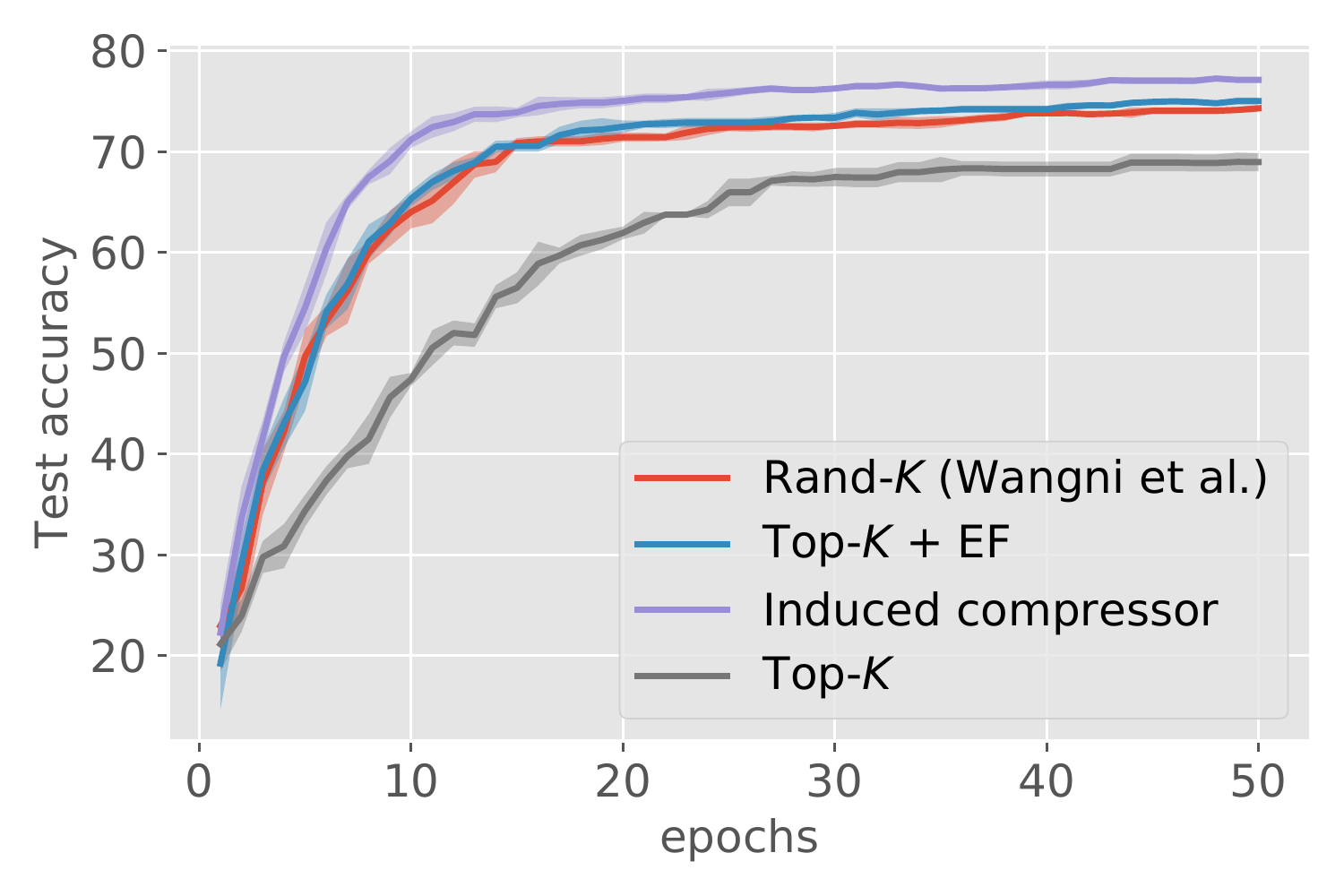}
\caption{Comparison of different sparsification techniques with and without usage of Error Feedback on CIFAR10 with Resnet18 (top) and VGG11 (bottom). $K =  5\% * d$, for Induced compressor $\cC_1$ is Top-$\nicefrac{K}{2}$ and $\cC_2$ is Rand-$\nicefrac{K}{2}$ (Wangni et al.).}
\label{fig:exp2}
\end{figure}

\subsection{Failure of \texttt{DCSGD} with biased Top-$\mathbf{1}$ sparsifier} 
In this experiment, we present example considered in~\cite{beznosikov2020biased}, which was used as a counterexample to show that some form of error correction is needed in order for biased compressors to work/provably converge.  In addition, we run experiments on their construction and show that while Error Feedback fixes divergence, it is still significantly dominated by unbiased non-uniform sparsification(NU Rand-$1$), which works by only keeping one non-zero coordinate sampled with probability equal to $\nicefrac{|x|}{\sum_{i=1}^d |x|_i}$, where $|x|$ denotes element-wise absolute value, as can be seen in Figure~\ref{fig:exp4}. The details can be found in the Appendix.

\subsection{Error feedback for unbiased compression operators}
In our second experiment, we compare the effect of Error Feedback in the case when an unbiased compressor is used. Note that unbiased compressors  are theoretically guaranteed to work both with Algorithm~\ref{alg:UC_SGD} and~\ref{alg:EF_SGD}. We can see from Figure~\ref{fig:exp1} that adding Error Feedback can hurt the performance; we use TernGrad~\cite{terngrad} (coincides with \texttt{QSGD}~\cite{qsgd2017neurips} and natural dithering~\cite{Cnat} with the infinity norm and one level) as compressors. This agrees with our theoretical findings. In addition, for sparsification techniques such as Random Sparsification or Gradient Sparsification~\cite{tonko}, we observed that when sparsity is set to be 10 \%, Algorithm~\ref{alg:UC_SGD} converges for all the selected values of step-sizes, but Algorithm~\ref{alg:EF_SGD} diverges and a smaller step-size needs to be used. This is an important observation as many practical works~\cite{li2014scaling, wei2015managed, aji2017sparse, hsieh2017gaia, deepgradcompress2018iclr, lim20183lc} use sparsification techniques mentioned in this section, but proposed to use EF, while our work shows that using unbiasedness property leads  not only to better convergence but also to memory savings.

\subsection{Unbiased alternatives to biased compression}
In this section, we investigate candidates for unbiased compressors than can compete with Top-$K$, one of the most frequently used compressors. Theoretically, Top-$K$ is not guaranteed to work by itself and might lead to divergence~\cite{beznosikov2020biased}  unless Error Feedback is applied. One would usually compare the performance of Top-$K$ with EF to Rand-$K$, which keeps $K$ randomly selected coordinates and then scales the output by $\nicefrac{d}{K}$ to preserve unbiasedness. Rather than  naively comparing to Rand-$K$, we propose to use more nuanced unbiased approaches. The first one is Gradient Sparsification proposed by Wagni et al.~\cite{tonko}, which we refer to here as Rand-$K$ (Wangni et al.), where the probability of keeping each coordinate scales with its magnitude and communication budget. As the second alternative, we propose to use our induced compressor, where $\cC_1$ is Top-$a$ and unbiased part $\cC_2$ is Rand-$(K-a)$ (Wangni et al.) with communication budget $K-a$. It should be noted that $a$ can be considered as a hyperparameter to tune. For our experiment, we chose it to be $\nicefrac{K}{2}$ for simplicity. Figure~\ref{fig:exp2} suggests that our induced compressor  outperforms all of its competitors as can be seen for both VGG11 and Resnet18. Moreover, induced compressor  as well as Rand-$K$ do not require extra memory to store the error vector. Finally, Top-$K$ without EF suffers a significant decrease in performance, which stresses the necessity of error correction.

\chapter{Optimal client sampling for federated learning}
\label{chapter6:optimal_sampling}

\section{Introduction}

We consider the standard cross-device Federated Learning (FL) setting~\cite{kairouz2019advances}, where the objective is of the form
\begin{align}
 \min \limits_{x\in\mathbb{R}^d} \left[ f(x)\eqdef \sum \limits_{i=1}^n w_i f_i(x)\right] \,, \label{eq:probR_optimal}
\end{align}
where  $x\in \R^d$ represents the parameters of a statistical model we aim to find, $n$ is the total number of clients, each $f_i \colon \R^d \to \R$ is a continuously differentiable local loss function which depends on the data distribution $\cD_i$ owned by client $i$ via $f_i(x) = \EE{\xi \sim\cD_i}{ f(x, \xi)}$, and  $w_i\geq 0$ are client weights such that $\sum_{i=1}^n w_i=1$. We assume the classical FL setup in which a central master (server) orchestrates the training by securely aggregating updates from clients without seeing the raw data.

\subsection{Communication as the bottleneck}
 It is well understood that communication cost can be the primary bottleneck in FL. Indeed, wireless links and other end-user internet connections typically operate at lower rates than intra-datacenter or inter-datacenter links and can be potentially expensive and unreliable. Moreover, the capacity of the aggregating master and other FL system considerations imposes direct or indirect constrains on the number of clients allowed to participate in each communication round. These considerations have led to significant interests in reducing the communication bandwidth of FL systems.

\paragraph{Local methods} 
One of the most popular strategies is to reduce the frequency of communication and put more emphasis on computation. This is usually achieved by asking the devices to perform multiple local steps before communicating their updates. A prototype method in this category is the Federated Averaging ({\tt FedAvg}) algorithm~\cite{mcmahan17fedavg}, an adaption of local-update to parallel \texttt{SGD}, where each client runs some number of \texttt{SGD} steps locally before local updates are averaged to form the global update for the global model on the master. The original work was a heuristic, offering no theoretical guarantees, which motivated the community to try to understand the method and various existing and new variants theoretically~\cite{local_SGD_stich_18, lin2018don, karimireddy2020scaffold, stich2020error, bayoumi2020tighter, Hanzely2020}. 

\paragraph{Communication compression} 
Another popular approach is to reduce the size of the object (typically gradients) communicated from clients to the master. These techniques are usually referred to as gradient/communication {\em  compression}. In this approach, instead of transmitting the full-dimensional gradient/update vector $g \in \R^d$, one transmits a compressed vector $\cC(g)$, where $\cC: \R^d \rightarrow \R^d$ is a (possibly random) operator chosen such that $\cC(g)$ can be represented using fewer bits, for instance by using limited bit representation (quantization) or by enforcing sparsity (sparsification). A particularly popular class of quantization operators is based on random dithering~\cite{goodall1951television, roberts1962picture}; see \cite{qsgd2017neurips, terngrad, zhang2016zipml,  ramezani2019nuqsgd}. A new variant of random dithering developed in \cite{Cnat} offers an exponential improvement on standard dithering. Sparse vectors can be obtained by random sparsification techniques that randomly mask the input vectors and preserve a constant number of coordinates~\cite{tonko, konevcny2018randomized,stich2018sparsified,mishchenko201999, vogels2019powersgd}. There is also a line of work~\cite{Cnat,basu2019qsparse} which propose to combine sparsification and quantization to obtain a more aggressive combined effect. For a comprehensive overview of the literature on communication compression, we refer the reader to the related work sections in previous chapters (Chapters \ref{chapter2:c_nat}-\ref{chapter5:induced}).

 \section{Related work}
 Importance sampling methods for optimization have been studied extensively in the last few years in several contexts, including convex optimization and deep learning. {\tt LASVM} developed in \cite{bordes2005fast} is an online algorithm that uses importance sampling to train kernelized support vector machines. The first importance sampling for randomized coordinate descent methods was proposed in a seminal paper in \cite{nesterov2012efficiency}. It was showed in~\cite{richtarik2014iteration} that the proposed sampling is optimal. Later, several extensions and improvements followed~\cite{shalev2014accelerated, lin2014accelerated, fercoq2015accelerated, qu2015quartz, allen2016even, stich2017safe}. Another branch of work studies sample complexity. In \cite{needell2014stochastic, zhao2015stochastic}, the authors make a connection with the variance of the gradient estimates of {\tt SGD} and show that the optimal sampling distribution is proportional to the per-sample gradient norm. However, obtaining this distribution is as hard as the computation of the full gradient in terms of computation, thus it is not practical. For simpler problems, one can sample proportionally to the norms of the inputs, which can be linked to the Lipschitz constants of the per-sample loss function for linear and logistic regression. For instance, it was shown in \cite{horvath2018nonconvex} that static optimal sampling can be constructed even for mini-batches and the probability is proportional to these Lipschitz constants under the assumption that these constants of the per-sample loss function are known. Unfortunately, importance measures such as smoothness of the gradient are often hard to compute/estimate for more complicated models such as those arising in deep learning, where most of the importance sampling schemes are based on heuristics. A manually designed sampling scheme was proposed in~\cite{bengio2009curriculum}. It was inspired by the perceived way that human children learn; in practice, they provide the network with examples of increasing difficulty in an arbitrary manner. In a diametrically opposite approach, it is common for deep embedding learning to sample hard examples because of the plethora of easy non-informative ones \cite{schroff2015facenet, simo2015discriminative}. Other approaches use a history of losses
for previously seen samples to create the sampling distribution and sample either proportionally to the loss or based on the loss ranking~\cite{schaul2015prioritized, loshchilov2015online}. \cite{katharopoulos2018not} proposes to sample based on the gradient norm of a small uniformly sampled subset of samples.

In our work, we avoid all the issues mentioned above, as our motivation is not to reduce computation, which is not the main bottleneck of FL, but to {\em use importance sampling to reduce the number of bits communicated}. This, as we show in Section~\ref{sec:main},  allows us to construct  {\em optimal adaptive sampling}; that is, we do not need to rely on any heuristics, historical losses, or partial information. In comparison to the optimal sampling in \cite{zhao2015stochastic} where batch size is restricted to one, our theoretical results can be seen as a non-trivial extension to any batch size (in our context any number of sampled clients), which gives rise to a practical approach to reducing communication costs in FL.

\section{Contributions} 

In this chapter, we propose a new approach to address the communication bandwidth issues appearing in FL. Our approach is based on the observation that in the situation where partial participation is desired and a budget on the number of participating clients is applied,  {\em careful selection of the participating clients can lead to better communication complexity, and hence faster training.} In other words, we claim that in any given communication round, some clients will have ``more informative'' updates than others and that the training procedure will benefit from capitalizing on this fact by ignoring some of the worthless updates. 

In particular, we propose a principled {\em optimal client sampling scheme}, capable of identifying the most informative clients in any given communication round. Our scheme works by minimizing the variance of the stochastic gradient produced by the partial participation procedure, which then translates to a reduction in the number of communication rounds. To the best of our knowledge, this approach was not considered before. Moreover, our proposal is orthogonal to and hence compatible with existing approaches to communication reduction such as communication compression and/or local updates (see Section~\ref{sec: FedAvg}). Our contributions can be summarized as follows:

\subsection{Novel adaptive partial participation strategy}
We propose a  {\em novel adaptive partial participation strategy for reducing communication in FL} that works by a careful selection of the  clients that are allowed to communicate their updates to the master node in any given communication round.
\subsection{Optimality}
  Our  {\em adaptive client sampling procedure is optimal} in the sense that it minimizes the variance of the master update.
\subsection{Approximation}
  We obtain an approximation to our optimal adaptive sampling strategy which only requires aggregation, thus allows for {\em secure aggregation} and {\em stateless clients}.
\subsection{Convergence}
 We show theoretically that our approach allows for {\em larger learning rates} for Distributed {\tt SGD} and {\tt FedAvg} algorithms than the baseline which performs uniform client sampling, which as a result leads to {\em better communication complexity} and hence {\em faster convergence}.
\subsection{Experiments}
We show empirically that the performance of our approach is superior to uniform sampling and is close to full participation.

\section{Smart client sampling for reducing communication}
\label{sec:main}

We now describe our client sampling strategy for reducing the communication bottleneck in Federated Learning. 
Each client $i$ participating in round $k$ computes an update vector $\U\in\R^d$. For simplicity and ease of exposition, we assume that all clients $i\in [n]\eqdef \{1,2,\dots,n\}$ are available in each round. However, we would like to point out that this is not a limiting factor, and all presented theory can be easily extended to the case of partial participation with an arbitrary distribution. In our framework, only a subset of clients communicates their updates to the master node in each communication round in order to reduce the number of transmitted bits. 

In order to provide an analysis in this framework, we consider a general partial participation framework~\cite{horvath2021a}, where we assume that the subset of participating clients is determined by an arbitrary  random set-valued mapping $\Sam$ (i.e., a ``sampling'') with values in $2^{[n]}$. A sampling $\Sam$ is uniquely defined by assigning probabilities to all $2^n$ subsets of $[n]$. With each sampling $\Sam$ we associate a {\em probability matrix} $\mP \in \R^{n\times n}$  defined by $\mP_{ij} \eqdef \Prob(\{i,j\}\subseteq \Sam)$. The {\em probability vector} associated with $\Sam$ is the vector composed of the diagonal entries of $\mP$: $p = (p_1,\dots,p_n)\in \R^n$, where $p_i\eqdef \Prob(i\in \Sam)$. We say that $\Sam$ is {\em proper} if $p_i>0$ for all $i$. It is easy to show that $b \eqdef \Exp{|\Sam|} = \trace{\mP} = \sum_{i=1}^n p_i$, and hence  $b$ can be seen as the expected number of clients  participating in each communication round. Given parameters $p_1,\dots,p_n \in [0,1]$, consider a random set $\Sam\subseteq [n]$ generated as follows: for each $i\in [n]$, we  include $i$ in $\Sam$ with probability $p_i$. This is called {\em independent sampling}, since the event $i\in \Sam$ is independent of  $j\in \Sam$ for any $i\neq j$.

While our client sampling strategy can be adapted to essentially any underlying learning method, we give details here for {\tt DSGD}:
\begin{equation}
\label{eq:SGD_step}
x^{k+1} = x^k - \eta^k \mG^k, \quad \mG^k \eqdef \sum \limits_{i \in S^k} \frac{w_i}{p_i^k} \U,
\end{equation}
 where $S^k \sim \Sam^k$ and $\U = g_i^k $ is an unbiased estimator of $\nabla f_i(x^k)$. The scaling factor $\frac{1}{p_i^k}$ is necessary in order to obtain an unbiased  estimator of the true update, i.e.,
$\EE{S^k}{\mG^k} =  \sum_{i = 1}^n w_i \U$. 

\subsection{Optimal client sampling} 
A simple observation is that the variance of our gradient estimator $\mG^k$ can be decomposed into

\begin{equation}
\label{eq:dif}
\E{\norm*{\mG^k- \nabla f(x^k)}^2} = \E{\norm*{\mG^k - \sum \limits_{i=1}^n w_i \U}^2}    + \E{\norm*{\sum \limits_{i=1}^n w_i \U - \nabla f(x^k)}^2},
\end{equation}
where the second term on the right-hand side is independent of the sampling procedure, and the first term is zero if every client sends its update (i.e., if $p_i^k=1$ for all $i$). In order to provide meaningful results, we restrict the expected number of clients to communicate in each round by bounding $b^k \eqdef \sum_{i=1}^n p_i^k$ by some positive integer $m \leq n$.  This raises the following question: {\em What is the sampling procedure that  minimizes \eqref{eq:dif} for any given $m$?} We answer this question using the following  technical lemma (see Appendix \ref{appendix:lemma proof} for a proof):

\begin{lemma} 
\label{LEM:UPPERV}
Let $\zeta_1,\zeta_2,\dots,\zeta_n$ be vectors in $\mathbb{R}^d$ and $w_1,w_2,\dots,w_n$ be non-negative real numbers such that $\sum_{i=1}^n w_i=1$. Define $\Tilde{\zeta} \coloneqq \sum_{i=1}^n w_i \zeta_i$. Let $S$ be a proper sampling. If $v\in\mathbb{R}^n$ is  such that
        \begin{equation}\label{eq:ESO}
            \mP -pp^\top \preceq {\rm \bf Diag}(p_1v_1,p_2v_2,\dots,p_nv_n),
        \end{equation}
        then
        \begin{equation}\label{eq:key_inequality_optimal}
            \E{ \norm*{ \sum \limits_{i\in S} \frac{w_i\zeta_i}{p_i} - \Tilde{\zeta} }^2 } \leq \sum \limits_{i=1}^n w_i^2\frac{v_i}{p_i} \norm*{\zeta_i}^2,
        \end{equation}
        where the expectation is taken over $S$. Whenever \eqref{eq:ESO} holds, it must be the case that $v_i \geq 1-p_i$. 
\end{lemma}
It turns out that given probabilities  $\{p_i\}$, among all samplings $S$ satisfying $p_i=\Prob(i\in S)$, the independent sampling (i.e., $p_{ij}=\Prob(i, j \in S) = \Prob(i \in S) \Prob(j \in S) = p_i p_j$)  minimizes the left-hand side of \eqref{eq:key_inequality_optimal}. This is due to two nice properties: a) any independent sampling admits optimal choice of $v$, i.e., $v_i=1-p_i$ for all $i$, and b) for independent sampling \eqref{eq:key_inequality_optimal} holds as equality. In the context of our method, these properties can be written as
\begin{align}
&\E{\norm*{\mG^k - \sum \limits_{i=1}^n w_i \U}^2} = \E{\sum \limits_{i=1}^n w_i^2 \frac{1-p_i^k}{p_i^k} \norm*{\U}^2}.
\label{varianceToBeOptimized}
\end{align}
It now only remains to find the parameters $\{p_i^k\}$ defining the optimal independent sampling, i.e., one that  minimizes \eqref{varianceToBeOptimized} subject to  the constraints $0 \leq p_i^k \leq 1$ and $b^k \eqdef \sum_{i=1}^n p_i^k \leq m$. It turns out that this problem  has the following closed-form solution (see Appendix~\ref{appendix:improvement factor}):
\begin{equation}
\label{eq:unique_sol_main}
p_i^k = 
\begin{cases}
       (m +  l - n)\frac{\norm*{\tilde{U}_i^k}}{\sum_{j=1}^l \norm*{\tilde{U}_{(j)}^k}}, &\quad\text{if } i \notin A^k,  \\
       1, &\quad\text{if } i \in A^k,
\end{cases}
\end{equation} 
where $\tilde{U}_i^k \eqdef w_i \U$, and $\norm*{\tilde{U}_{(j)}^k}$ is the $j$-th largest value in $\left\{\norm*{\tilde{U}_i^k}\right\}_{i=1}^n$, $l$ is the largest integer  for which  $0<m +  l - n \leq \frac{\sum_{i =1}^l \norm*{\tilde{U}_{(i)}^k}}{\norm*{\tilde{U}_{(l)}^k}}$ (note that this inequality at least holds for $l= n-m+1$), and $A^k$ contains indices $i$ such that $\norm*{\tilde{U}_i^k} \geq \norm*{\tilde{U}_{(l+1)}^k}$. We summarize this procedure in Algorithm~\ref{alg:OCS}. 

\textbf{Remark.} We note that optimizing the left-hand side of~\eqref{eq:key_inequality_optimal} does not guarantee the proposed sampling to be optimal with respect to the right-hand side in the general case. For this to hold, we need our sampling to be independent, which is not very restrictive, especially considering that enforcing independent sampling across clients accommodates the privacy requirements of FL. In addition, \eqref{eq:key_inequality_optimal} is tight and therefore, if one is allowed to communicate only norms (i.e., one float per client) as extra information, then our sampling is optimal. We further note that requiring optimality with respect to the right-hand side of~\eqref{eq:key_inequality_optimal} in the full general case is not practical, as it cannot be obtained without revealing, i.e., communicating, all clients' full updates to the master.

\begin{figure}[!t]
\centering
\begin{algorithm}[H]
\begin{algorithmic}[1]
    \STATE {\bfseries Input:} expected batch size $m$
	\STATE each client $i$ computes a local update $\U$ (in parallel)
	\STATE each client $i$ sends the norm of its update $u_i^k = w_i\norm*{\U}$ to the master (in parallel)
 	\STATE master computes optimal probabilities $p_i^k$ using equation \eqref{eq:unique_sol_main}
 	\STATE master broadcasts $p_i^k$ to all clients
	\STATE each client $i$ sends its update $\frac{w_i}{p_i^k}\U$ to the master with probability $p_i^k$ (in parallel)
\end{algorithmic}
\caption{Optimal Client Sampling ({\tt OCS}).}
\label{alg:OCS}
\end{algorithm}
\end{figure}

\subsection{Secure aggregation} 
In the case $l = n$, the optimal probabilities $p_i^k = \frac{m\norm*{\tilde{U}_i^k}}{\sum_{j=1}^n \norm*{\tilde{U}_{j}^k}}$ can be computed easily: the master aggregates the norm of each update and then sends the sum back to the clients. However, if $l < n$, in order to compute optimal probabilities, the master would need to identify the norm of every update and perform partial sorting, which can be computationally expensive and also violates the client privacy requirements in FL. Therefore, we develop an algorithm for approximately solving this problem, which only requires to perform aggregation at the master node without compromising the privacy of any client. The construction of this algorithm is similar to \cite{tonko}. We first set $\tilde{p}_i^k = \frac{m\norm*{\tilde{U}_i^k}}{\sum_{j=1}^n \norm*{\tilde{U}_{j}^k}}$ and $p_i^k = \min\{\tilde{p}_i^k, 1\}$. In the ideal situation where every $\tilde{p}_i^k$ equals the optimal solution \eqref{eq:unique_sol_main}, this would be sufficient. However, due to the truncation operation, the expected mini-batch size  $b^k = \sum_{i=1}^n p_i^k \leq \sum_{i=1}^n \frac{m\norm*{g_i^k}}{\sum_{j=1}^n \norm*{g_{j}^k}} = m$ can be strictly less than $m$ if $\tilde{p}_i^k > 1$ holds true for at least one $i$.  Hence, we employ an iterative procedure to fix this gap by rescaling the probabilities which are smaller than $1$, as summarized in Algorithm~\ref{alg:AOCS}. This algorithm is much easier to implement and computationally more efficient on parallel computing architectures. In addition, it only requires a secure aggregation procedure on the master, which is essential in privacy preserving FL, and thus it is compatible with existing FL software and hardware. 

\textbf{Extra Communication Costs.} We acknowledge that Algorithm~\ref{alg:AOCS} brings extra communication costs, as it requires all clients to send the norms of their updates $u_i^k$'s and probabilities $p_i^k$'s in each round. However, since these are single floats, this only costs $\cO(j_{\max})$ extra floats for each client. Picking $j_{\max} = \cO(1)$, this is negligible for large models of size $d$.

\textbf{Fairness.} Based on our sampling strategy, it might be tempting to assume that the obtained solution could exhibit fairness issues. In our convergence analyses below, we show that this is not the case, as our proposed methods converge to the optimal solution. Hence, as long as the original objective has no inherent issue with fairness, our methods do not exhibit any fairness issues. Besides, our algorithm can be used in conjunction with other ``more fair'' objectives (e.g., tilted ERM~\cite{li2021tilted}) if needed.

\section{Convergence guarantees}
In this section, we provide convergence analyses for {\tt DSGD} and {\tt FedAvg} with our optimal client sampling scheme in both convex and non-convex settings. We compare our scheme with full participation and independent uniform sampling with sample size $m$. We use standard assumptions~\cite{karimi2016linear}, assuming throughout that
$f_i$'s are $L$-smooth. We also make assumptions of the gradient oracles:

\begin{assumption}[Gradient oracle for {\tt DSGD}]
\label{ass:2_optimal} The stochastic gradient estimator $g_i^k=\nabla f_i(x^k)+\xi_i^k$ of the local gradient $\nabla f_i(x^k)$, for each round $k$ and all $i=1,\dots,n$,  satisfies
\begin{equation}
\label{eq:grad_unbiased_optimal}
\E{\xi_i^k } = 0
\end{equation}
 and
 \begin{equation}
\label{eq:grad_bounded}
\E{\norm*{\xi_i^k}^2| x^k_i } \leq M\norm*{\nabla f_i(x^k)}^2 + \sigma^2,\quad\text{for some $M\geq 0$}.
\end{equation}
This further implies  that $\E{\frac{1}{n}\sum_{i=1}^n g_i^k \;| \;x^k} = \nabla f(x^k)$.
\end{assumption}

\begin{assumption}[Gradient oracle for {\tt FedAvg}]
\label{ass:3_optimal} The stochastic gradient estimator $g_i(y_{i,r}^k)=\nabla f_i(y_{i,r}^k)+\xi_{i,r}^k$ of the local gradient $\nabla f_i(y_{i,r}^k)$, for each round $k$, each local step $r=0,\dots,R$ and all $i=1,\dots,n$,  satisfies
\begin{equation}
\label{eq:grad_unbiased_optimal_fedavg}
\E{\xi_{i,r}^k } = 0
\end{equation}
 and
 \begin{equation}
\label{eq:grad_bounded_fedavg}
\E{\norm*{\xi_{i,r}^k}^2| y_{i,r}^k } \leq M\norm*{\nabla f_i(y_{i,r}^k)}^2 + \sigma^2,\quad\text{for some $M\geq 0$},
\end{equation}
where $y_{i,0}^k=x^k$ and $y_{i,r}^k=y_{i,r-1}^k-\eta_l g_i(y_{i,r}^k),~ r=1,\cdots,R$.
\end{assumption}

\begin{figure}[!t]
\centering
\begin{algorithm}[H]
\begin{algorithmic}[1]
 \STATE {\bfseries Input:} expected batch size $m$, maximum number of iteration $j_{\max}$
	\STATE each client $i$ computes an update $\U$ (in parallel)
    \STATE each client $i$ sends the norm of its update $u_i^k= w_i \norm*{\U}$ to the master (in parallel)
 	\STATE master aggregates $u^k =  \sum_{i=1}^n u_i^k$
 	\STATE master broadcasts $u^k$ to all clients
 	\STATE each client $i$ computes $p_i^k = \min\{\frac{m u_i^k}{u^k}, 1\}$ (in parallel)
 	\FOR{$j=1,\cdots,j_{max}$}
		\STATE each client $i$ sends $t_i^k = (1, p_i^k)$ to the master if $p_i^k < 1$; else sends $t_i^k =(0, 0)$ (in parallel)
		\STATE master aggregates $ (I^k, P^k)=  \sum_{i=1}^n t_i^k$
		\STATE master computes $C^k = \frac{m - n + I^k}{P^k}$
		\STATE master broadcasts $C^k$ to all clients
		\STATE each client $i$ recalibrates $p_i^k = \min\{C^k p_i^k, 1\}$ if $p_i^k < 1$ (in parallel)
		\IF{$C^k \leq 1$}
		    \STATE break
		\ENDIF
	\ENDFOR
	\STATE each clients $i$ sends its update $\frac{w_i}{p_i^k}\U$ to master with probability $p_i^k$ (in parallel)
\end{algorithmic}
\caption{Approximate optimal client sampling ({\tt AOCS}).}
\label{alg:AOCS}
\end{algorithm}
\end{figure}

For non-convex loss functions, we employ a standard assumption of local gradients:
\begin{assumption}[Similarity among local gradients]\label{ass:local grad similarity}
    The gradients of local loss functions $f_i$ satisfy
    \begin{align*}
        \sum_{i=1}^n w_i \norm*{\nabla f_i(x)-\nabla f(x)}^2\leq \rho,\quad \text{for some $\rho\geq0$}.
    \end{align*}
\end{assumption}

We define three quantities, which will appear in our convergence analyses:
\begin{align*}
 W\eqdef\max_{i \in [n]}\{w_i\} \quad\text{and}\quad R_i \eqdef f_i(x^\star) -  f_i^\star \quad\text{and}\quad r^k \eqdef x^k - x^\star,
\end{align*}
where $f_i^\star$ is the functional value of $f_i$ at its optimum. $R_i$ represents the mismatch between the local and global minimizer, and $r^k$ captures the distance between the current point and the minimizer  of $f$. 

We are now ready to proceed with our convergence analyses. We first define the improvement factor
\begin{align}
\label{eq:alpha_k}
\alpha^k \eqdef 
\frac{
                \E{ \norm*{ \sum_{i\in S^k}\frac{w_i}{p_i^k}\U - \sum_{i=1}^n w_i \U }^2 }
            }
            {
                \E{ \norm*{ \sum_{i\in U^k}\frac{w_i}{p_i^U}\U - \sum_{i=1}^n w_i \U }^2 }
            },
\end{align}
where $S^k \sim \Sam^k$ with $p_i^k$ defined in \eqref{eq:unique_sol_main} and $U^k \sim \mathbb{U}$ is an independent uniform sampling with  $p_i^U = \nicefrac{m}{n}$. By construction, $\alpha^k\leq 1$, as $\Sam^k$ minimizes the variance term (see Appendix \ref{appendix:improvement factor}). Note that $\alpha^k$ can reach zero in the case where there are at most $m$ non-zero updates. If $\alpha^k = 0$,  our method performs as if all updates were communicated. In the worst-case $\alpha^k = 1$,  our method performs as if we picked $m$ updates uniformly at random, and one could not do better in theory due to the structure of the updates $\U$. In the following subsections, we provide convergence analyses of specific methods for solving the optimization problem \eqref{eq:probR_optimal}. The proofs of the theorems are deferred to Appendices \ref{appendix:proof_dsgd} and \ref{appendix:proof_fedavg}. For simplicity of notation and analysis, we will denote
\begin{align*}
    \gamma^k\eqdef\frac{m}{\alpha^k(n-m)+m}\in\left[\frac{m}{n},1\right],\quad k=0,\dots,K-1.
\end{align*}

\subsection{Distributed {\tt SGD} ({\tt DSGD}) with optimal client sampling}
We obtain analyses for {\tt DSGD} \eqref{eq:SGD_step} with optimal client sampling in both convex and non-convex settings. 

 \begin{theorem}
 \label{THM:DSGD_MAIN}
           Let $f_i$ be $L$-smooth and convex for all $i=1,\dots,n$. Let $f$ be $\mu$-strongly convex. Suppose that Assumption~\ref{ass:2_optimal} holds. Choose $\eta^k \in\left(0,\frac{\gamma^k}{(1 + W M)L}\right]$. Define
            \begin{align*}
            \beta_1 \eqdef \sum_{i=1}^n w_i^2 (2L(1+M)R_i + \sigma^2)\quad\text{and}\quad
            \beta_2 \eqdef 2L\sum_{i=1}^n w_i^2R_i.
            \end{align*}
            Then, the iterates of {\tt DSGD} with optimal client sampling \eqref{eq:unique_sol_main} satisfy
            \begin{align}
            \label{eq:DSGD_conv}
          \E{\norm*{r^{k+1}}^2}
            \leq (1-\mu\eta^k)\E{\norm*{r^k}^2} +(\eta^k)^2 \left( \frac{\beta_1}{\gamma^k} - \beta_2 \right).
        \end{align}
        \end{theorem}

\textbf{Interpretation.} We first look at the best and worst case scenarios. In the best case scenario, we have $\gamma^k = 1$ for all $k$'s. This implies that there is no loss of speed comparing to the method with full participation. It is indeed confirmed by our theory as our obtained recursion recovers the best-known rate of {\tt DSGD} in the full participation regime~\cite{gower2019sgd}. Similarly, in the worst case, we have $\gamma^k = \nicefrac{m}{n}$ for all $k$'s, which corresponds to uniform sampling with sample size $m$, and our recursion recovers the best-known rate for {\tt DSGD} in this regime. This is expected as \eqref{eq:alpha_k} implies that every update $\U$ is equivalent, and thus it is theoretically impossible to obtain a better rate than that of uniform sampling in the worst case scenario. In the general scenario, our obtained recursion sits somewhere between full and uniform partial participation, where the actual position is determined by $\gamma^k$'s which capture the distribution of updates (here gradients) on the clients. For instance, with a larger number of $\gamma^k$'s tending to $1$, we are closer to the full participation regime. Similarly, with more $\gamma^k$'s tending to $\nicefrac{m}{n}$, we are closer to the rate of uniform partial participation.

\begin{theorem}\label{THM:DSGD_NONCONVEX}
        Let $f_i$ be $L$-smooth for all $i=1,\dots,n$. Suppose that Assumptions~\ref{ass:2_optimal} and \ref{ass:local grad similarity} hold. Let $\eta^k$ be the step size and define
        \begin{align*}
            \beta^k\eqdef  \frac{L}{2\gamma^k}\left((1+M-\gamma^k)W\rho+\sum_{i=1}^n w_i^2\sigma^2\right).
        \end{align*}
        Then, the iterates of {\tt DSGD} with optimal client sampling \eqref{eq:unique_sol_main} satisfy
        \begin{align}\label{eq:DSGD_nonconvex}
            \E{f(x^{k+1})}&\leq
        \E{f(x^{k})} - \eta^k\left(1-\frac{(1+M)L}{2\gamma^k}\eta^k\right)\E{\norm*{\nabla f(x^k)}^2} + (\eta^k)^2\beta^k.
        \end{align}
\end{theorem}
\textbf{Interpretation.} The iterate \eqref{eq:DSGD_nonconvex} recovers the standard form of the convergence result of {\tt DSGD} for one recursion step in the non-convex setting. Similar to the previous results, this convergence bound sits between the best-known rate of full participation and uniform sampling~\cite{bottou2018optimization}.

\subsection{Federated averaging ({\tt FedAvg}) with Optimal client sampling}
\label{sec: FedAvg}

    \begin{figure}[!t]
	\centering
    \begin{algorithm}[H]
        \begin{algorithmic}[1]
        \STATE {\bfseries Input:}  initial global model $x^1$, global and local step-sizes $\eta_g^k$, $\eta_l^k$ 
        \FOR{each round $k=1,\dots,K$}
            \STATE master broadcasts $x^{k}$ to all clients $i\in[n]$
            \FOR{each client $i\in[n]$ (in parallel)}
                \STATE initialize local model $y_{i,0}^k\leftarrow x^{k}$
                \FOR{$r=1,\dots,R$}
                   \STATE  compute mini-batch gradient $g_i(y_{i,r-1}^k)$
                    \STATE update $y_{i,r}^k \leftarrow y_{i,r-1}^k - \eta_l^k g_i(y_{i,r-1}^k)$
                \ENDFOR
                \STATE compute $\U \eqdef \Delta y_i^k=x^{k}-y_{i,R}^k$
                \STATE compute $p_i^k$ using Algorithm~\ref{alg:OCS} or~\ref{alg:AOCS}
                \STATE send $\frac{w_i}{p_i^k}\Delta y_i^k$ to master with probability $p_i^k$
            \ENDFOR
           \STATE  master computes $\Delta x^k=\sum_{i\in S^k} \frac{w_i}{p_i^k}\Delta y_i^k$
            \STATE master updates global model $x^{k+1} \leftarrow x^{k} - \eta_g^k \Delta x^k$
        \ENDFOR
        \end{algorithmic}
        \caption{{\tt FedAvg} with Optimal Client Sampling.}
        \label{alg:FedAvg}
    \end{algorithm}
    \end{figure}

Pseudo-code that adapts the standard {\tt FedAvg} algorithm to our framework is provided in Algorithm~\ref{alg:FedAvg}. We obtain analyses for {\tt FedAvg} with optimal client sampling in both convex and non-convex settings.

\begin{theorem}
\label{THM:FEDAVG_MAIN}
   Let $f_i$ be $L$-smooth and $\mu$-strongly convex for all $i=1,\dots,n$. Suppose that Assumption~\ref{ass:3_optimal} holds. Let $\eta^k \eqdef R\eta_l^k \eta_g^k$ be the effective step-size and $\eta_g^k \geq \sqrt{\frac{\gamma^k}{\sum_{i}w_i^2}}$. Choose $\eta^k \in\left(0, \frac{1}{8}\min\left\{ \frac{1}{L(2 +\nicefrac{M}{R})},\frac{\gamma^k}{(1 + W(1 + \nicefrac{M}{R}))L} \right\}\right]$. Define
   \begin{align*}
\beta^k_1 \eqdef  \frac{2\sigma^2}{\gamma^k R} \sum_{i=1}^n w_i^2 + 4L \left(\frac{M}{R}+1   - \gamma^k \right)  \sum_{i=1}^n w_i^2 R_i,\;
\beta_2 \eqdef  72L^2 \left(1 + \frac{M}{R} \right) \sum_{i=1}^n w_i R_i.
\end{align*}
   Then, the iterates of {\tt FedAvg} ($R\geq 2$) with optimal client sampling~\eqref{eq:unique_sol_main} satisfy
    \begin{align*}
\frac{3}{8} \E{(f(x^k)-f^\star)}  \leq \frac{1}{\eta^k}\left( 1 - \frac{\mu\eta^k}{2} \right) \E{\norm*{r^k}^2}  - \frac{1}{\eta^k}\E{\norm*{r^{k+1}}^2} 
+  \eta^k \beta^k_1 + (\eta^k)^2 \beta_2.
    \end{align*}
\end{theorem}

\begin{theorem}
\label{THM:FEDAVG_MAIN_nonconvex}
   Let $f_i$ be $L$-smooth for all $i=1,\dots,n$. Suppose that Assumptions~\ref{ass:3_optimal} and \ref{ass:local grad similarity} hold. Let $\eta^k \eqdef R\eta_l^k \eta_g^k$ be the effective step-size and $\eta_g^k \geq \sqrt{\frac{5\gamma^k}{4\sum_{i}w_i^2}}$. Choose $\eta^k \in\left(0, \frac{1}{8L(2 +\nicefrac{M}{R})} \right]$. Then, the iterates of {\tt FedAvg} ($R\geq 2$) with optimal client sampling~\eqref{eq:unique_sol_main} satisfy
   $\E{f(x^{k+1})} \leq$
    \begin{align*}
            \E{f(x^k)} - \frac{3\eta^k}{8}\left( 1-\frac{10\eta^k L}{3} \right) \E{\norm*{\nabla f(x^k)}^2}+\frac{\eta^k\rho}{8}+(\eta^k)^2\left(\frac{\rho}{4}+\frac{\sigma^2 }{R\gamma^k}\sum_{i=1}^n w_i^2\right)L.
    \end{align*}
\end{theorem}
\begin{figure}[!t]
\centering
\begin{subfigure}
    \centering
    \includegraphics[width=0.32\textwidth]{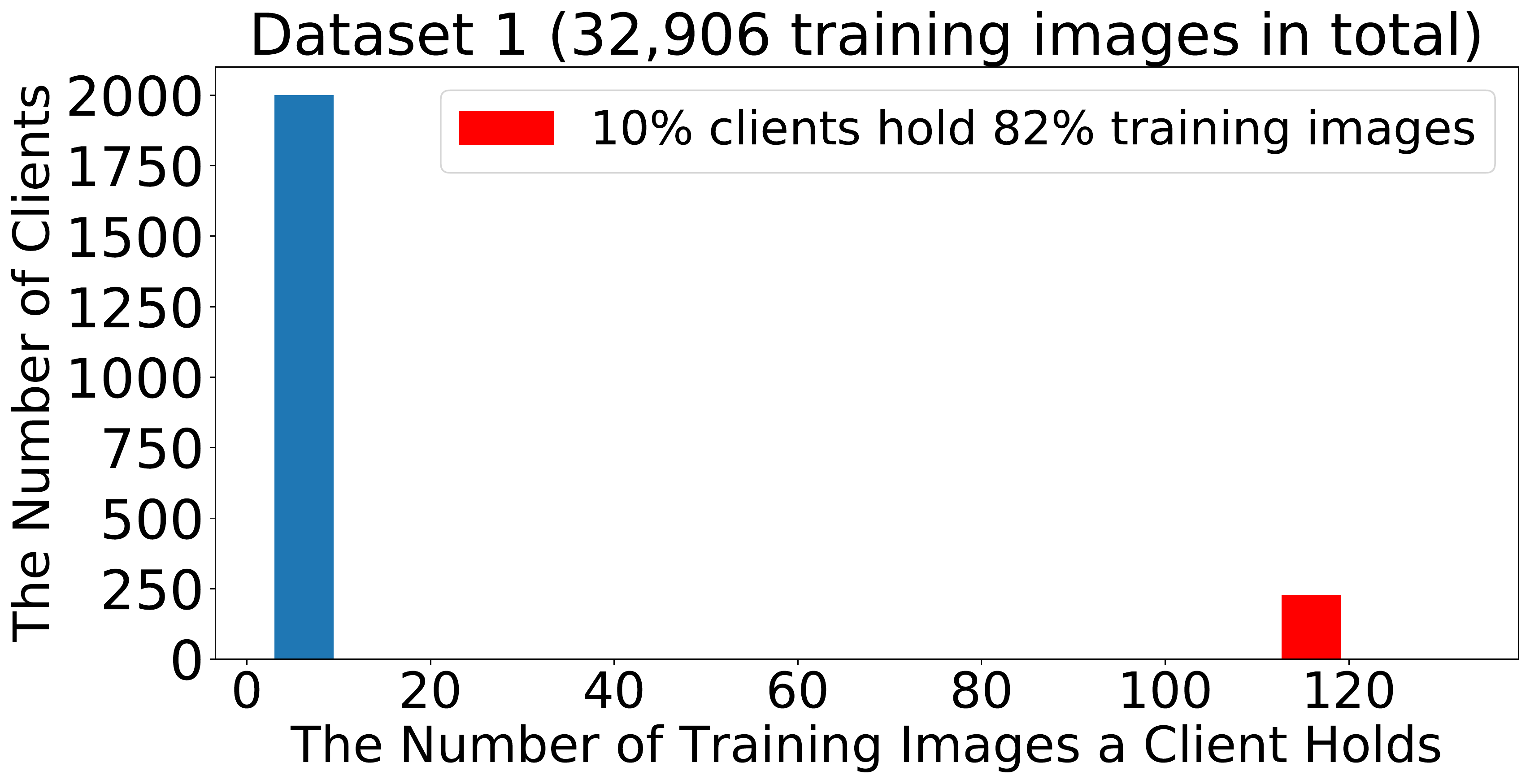}
\end{subfigure}
\begin{subfigure}
    \centering
    \includegraphics[width=0.32\textwidth]{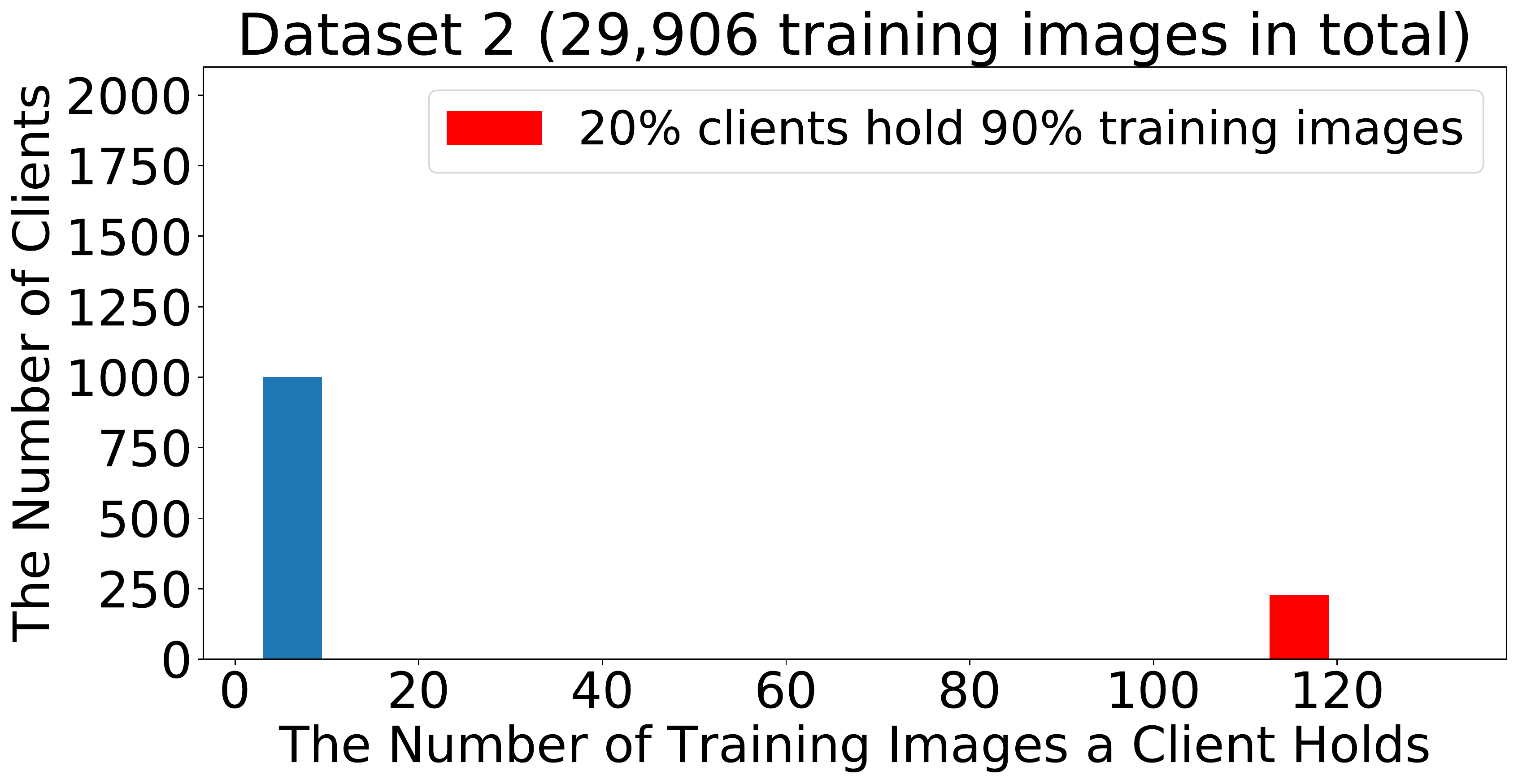}
\end{subfigure}
\begin{subfigure}
    \centering
    \includegraphics[width=0.32\textwidth]{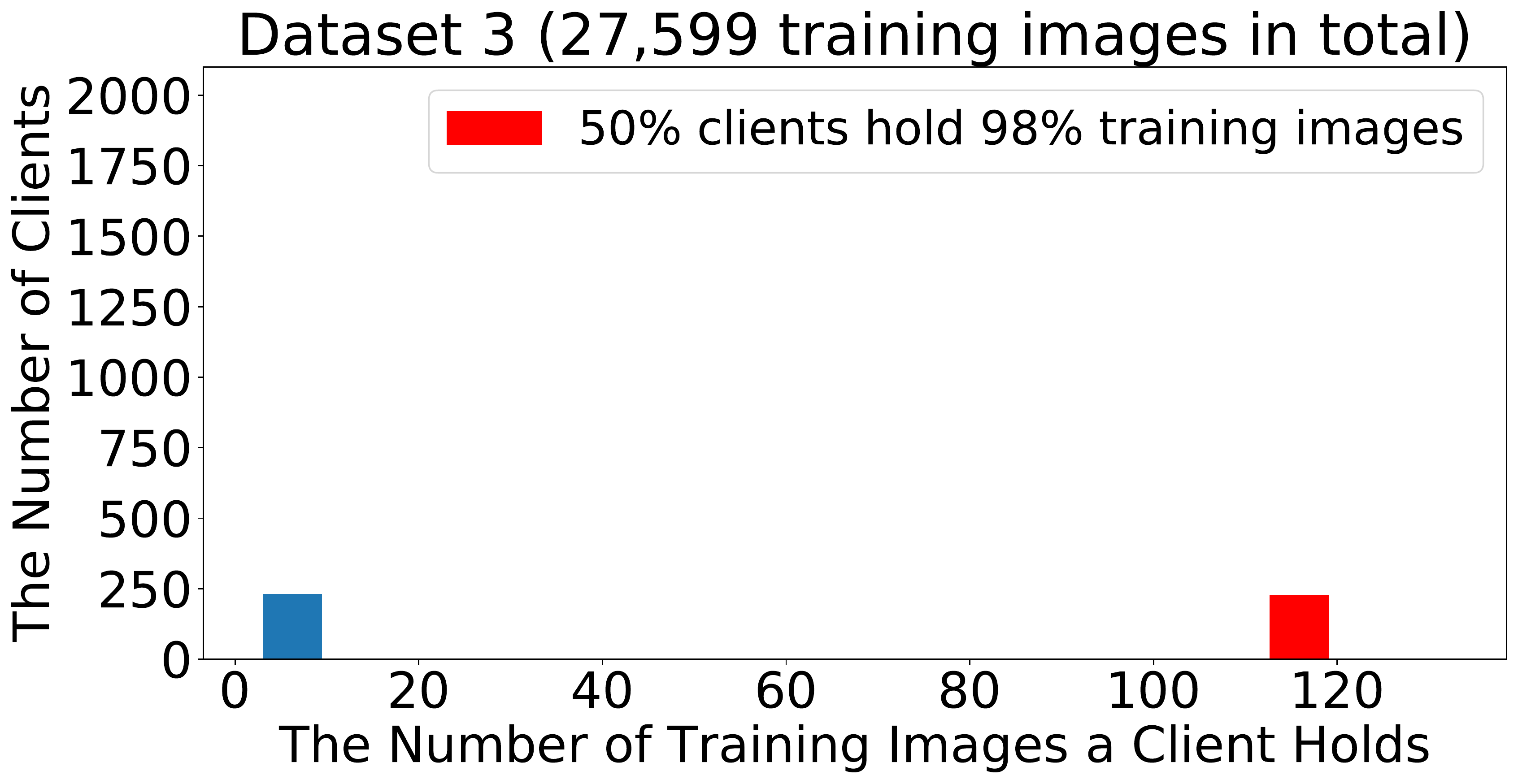}
\end{subfigure}
\caption{Distributions of the three modified Federated EMNIST training sets.}
\label{fig:datasets}
\end{figure}
\textbf{Interpretation.} The convergence guarantee sits somewhere between the performances of those with full and uniform partial participation. The actual position is again determined by the distribution of updates linked to $\gamma^k$'s. In the edge cases, i.e., $\gamma^k = 1$ (best case) or $\gamma^k = \nicefrac{m}{n}$ (worst case), we recover the state-of-the-art complexity guarantees provided in \cite{karimireddy2020scaffold} in both regimes. Note that our results are slightly more general, as \cite{karimireddy2020scaffold} assumes $M = 0$ and $w_i = \nicefrac{1}{n}$.

\begin{figure}[!t]
\centering
\begin{subfigure}
    \centering
    \includegraphics[width=0.4\textwidth]{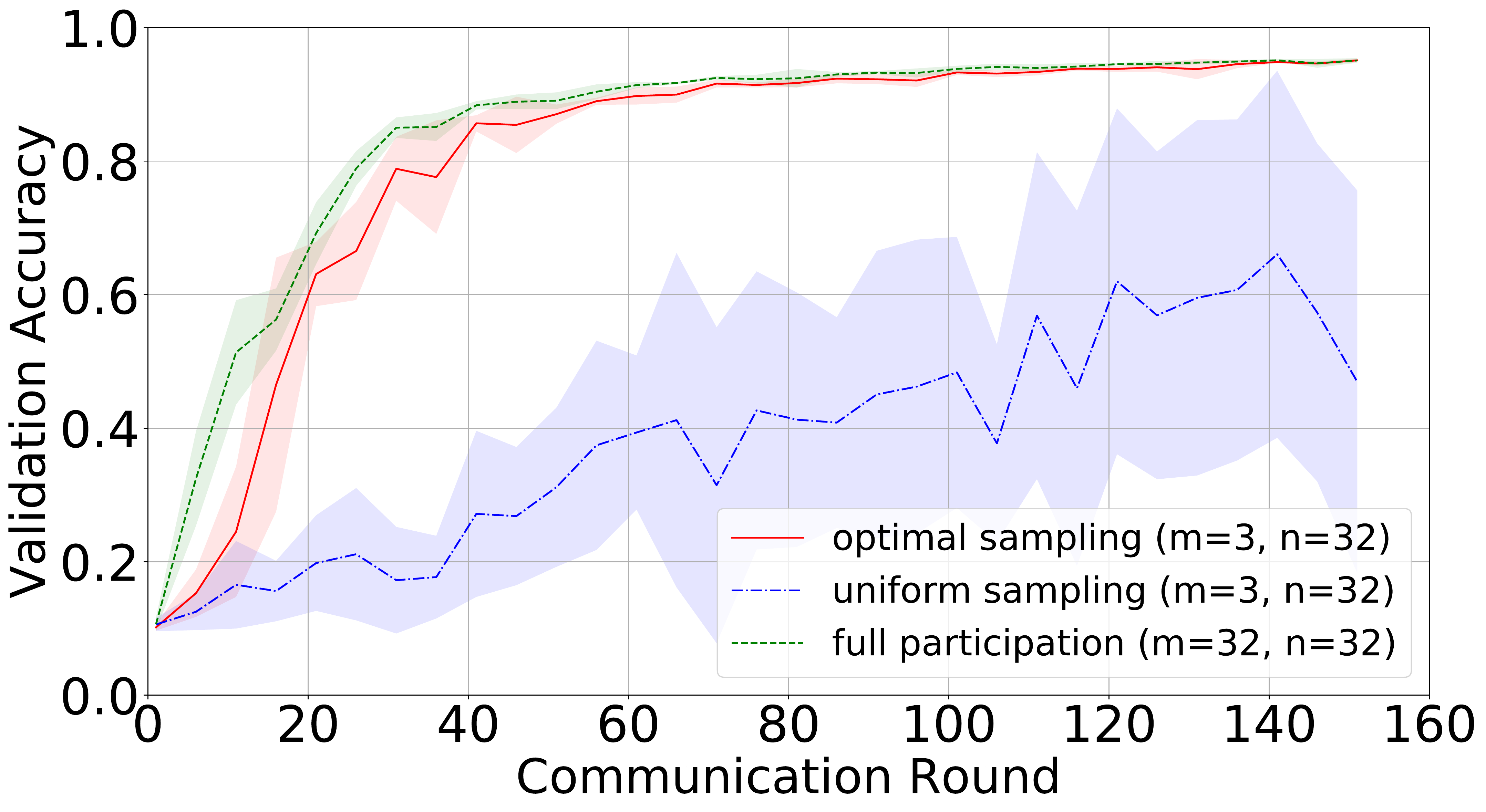}
\end{subfigure}
\begin{subfigure}
    \centering
    \includegraphics[width=0.4\textwidth]{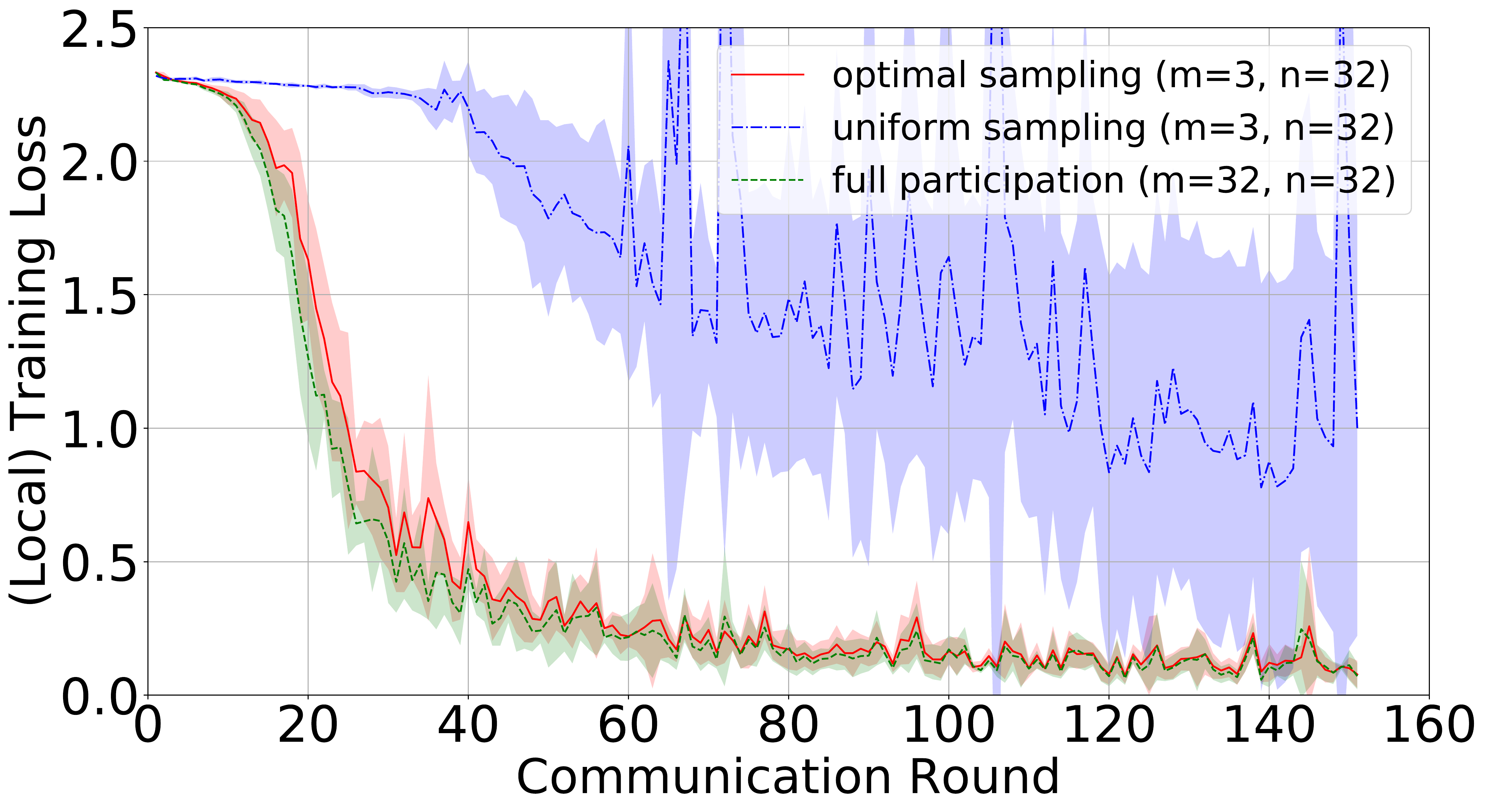}
\end{subfigure}
\begin{subfigure}
    \centering
    \includegraphics[width=0.4\textwidth]{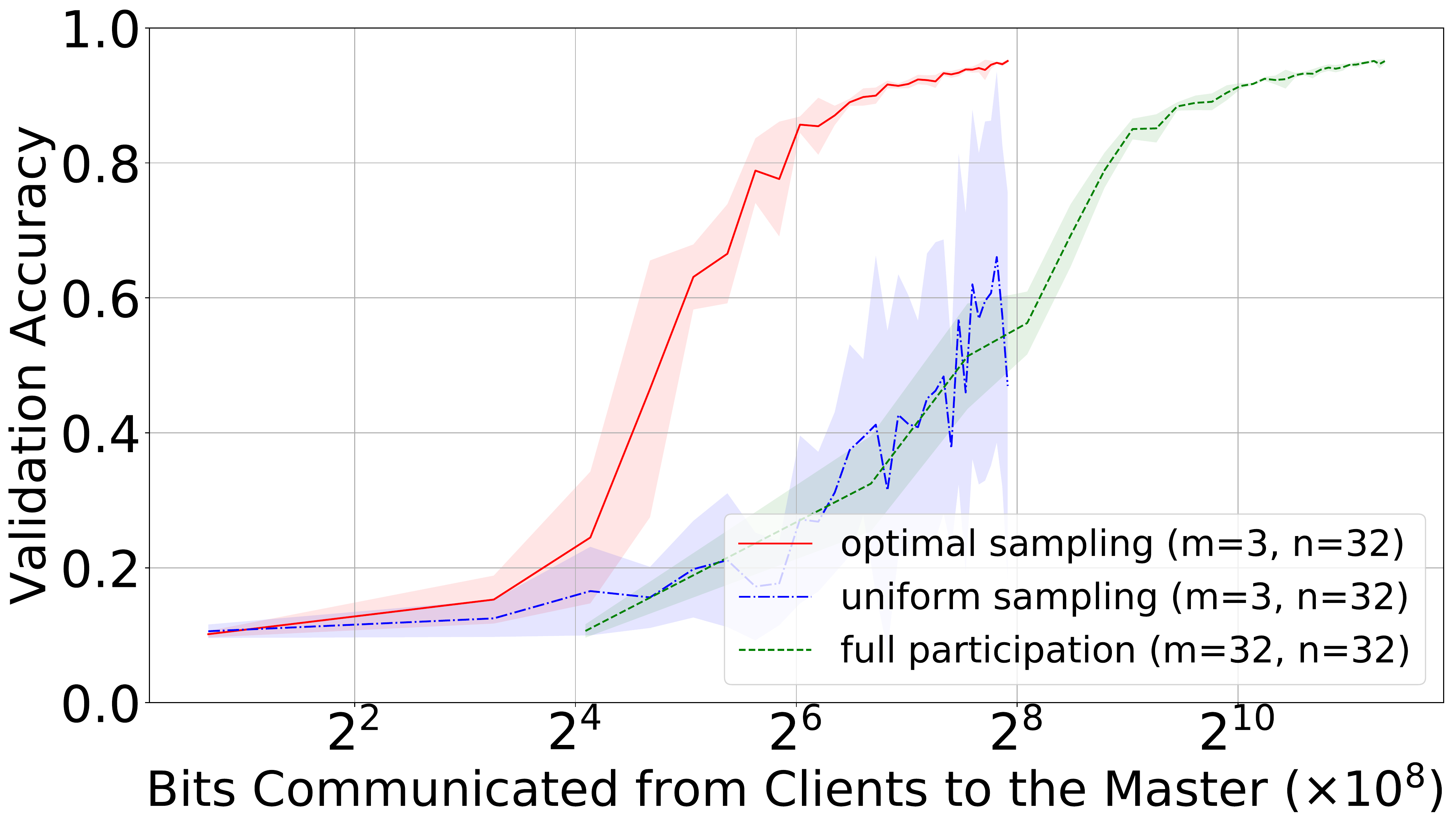}
\end{subfigure}
\begin{subfigure}
    \centering
    \includegraphics[width=0.4\textwidth]{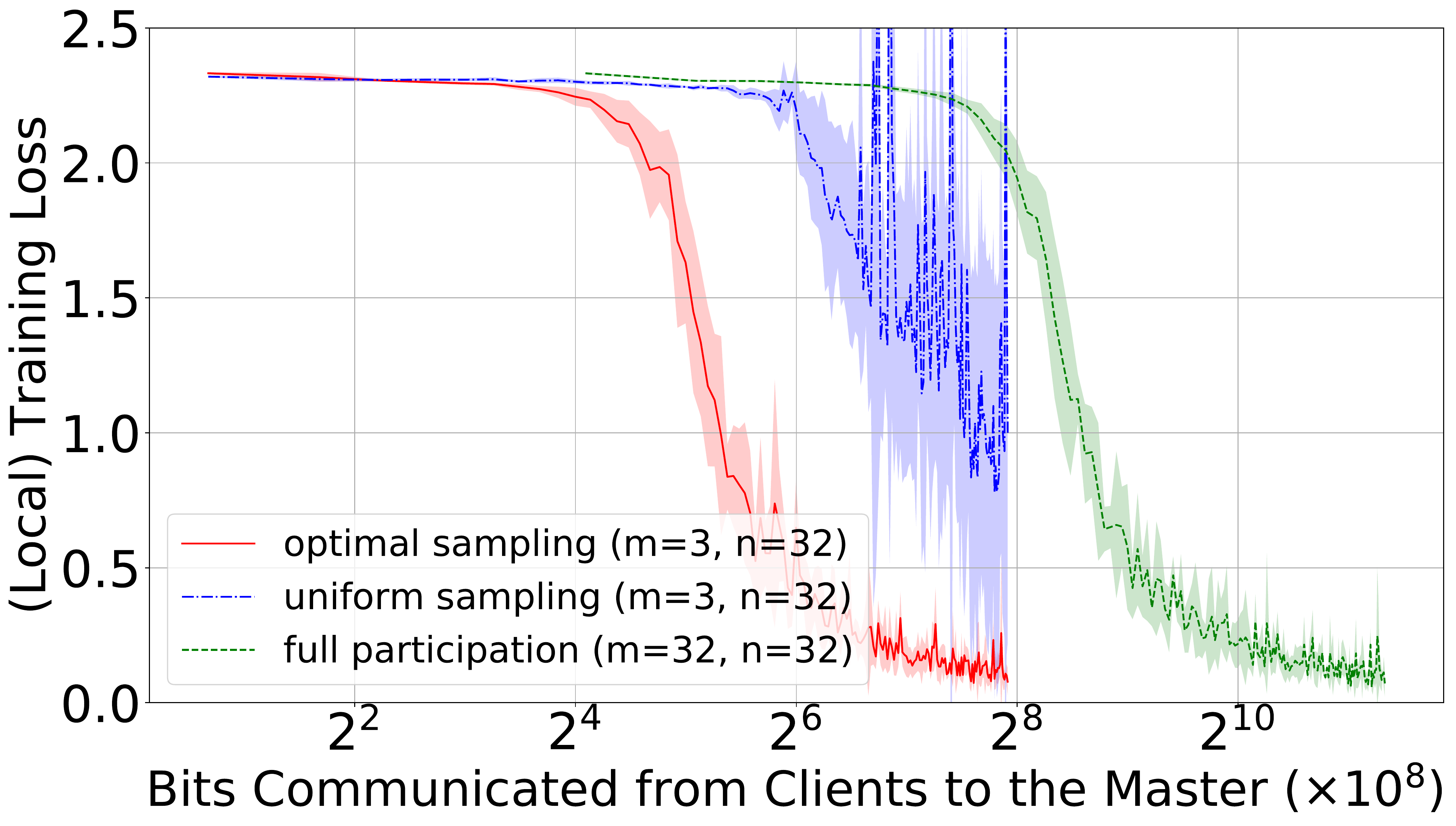}
\end{subfigure}

\caption{(FEMNIST Dataset 1) validation accuracy and (local) training loss as a function of the number of communication rounds and the number of bits communicated from clients to the master.}
\label{fig:cookup1}
\end{figure}

\section{Experiments}
In this section, we empirically evaluate our optimal client sampling method on two standard federated datasets from \cite{leaf}. We compare our method with 1) the baseline where participating clients are sampled uniformly from available clients in each round and 2) full participation where all available clients participate. We simulate the cross-device FL setting and train our models using TensorFlow Federated (TFF)\footnote{\url{https://github.com/tensorflow/federated}}. We conclude our evaluations using {\tt FedAvg} (Algorithm~\ref{alg:AOCS}), as it supports stateless clients and secure aggregation. We extend the TFF implementation of {\tt FedAvg}\footnote{\url{https://github.com/tensorflow/federated/tree/master/tensorflow_federated/python/examples/simple_fedavg}} to fit our framework. For all three methods, we report validation accuracy and (local) training loss (vertical axis) as a function of the number of communication rounds and the number of bits communicated from clients to the master (horizontal axis). Each figure displays the mean performance with standard error over 5 independent runs. For a fair comparison, we use the same random seed for the three compared methods in a single run and vary random seeds across different runs. Our implementation together with datasets is available at \url{https://github.com/SamuelHorvath/FL-optimal-client-sampling}. 

\subsection{Federated EMNIST dataset} 
We first evaluate our method on the Federated EMNIST (FEMNIST) image dataset with only digits for image classification. Since it is a well-balanced dataset with data of similar quality on each client, we modify the training set by removing some clients or some of their images, in order to better simulate the conditions in which our proposed method brings significant theoretical improvements. As a result, we produce three unbalanced training sets as summarized in Figure~\ref{fig:datasets}. The CNN model architecture we use is the same as the one used in \cite{mcmahan17fedavg}. For validation, we use the unchanged EMNIST validation set, which consists of $40,832$ validation images. In each communication round, $n=32$ clients are sampled uniformly from the client pool, each of which then performs several vanilla \texttt{SGD} steps on its local training images for $1$ epoch with batch size $20$. For partial participation, the expected number of clients allowed to communicate their updates back to the master is set to $m=3$ for all the experiments. We use vanilla \texttt{SGD} optimizers with constant step sizes for both clients and the master, where we set $\eta_g=1$ and tune $\eta_l$ using a holdout set. For full participation and optimal sampling, it turns out that $\eta_l=2^{-3}$ is the optimal local step size for all three datasets. For uniform sampling, the optimal is $\eta_l=2^{-5}$ for Dataset 1 and $\eta_l=2^{-4}$ for Datasets 2 and 3. We set $j_{\max}=4$ and include the extra communication costs in our results. The main results are shown in Figures~\ref{fig:cookup1}, \ref{fig:cookup2} and \ref{fig:cookup3}. In Appendix~\ref{appendix:EMNIST experiment}, we include more details of the experimental settings and extra results.

%
%

\begin{figure}[!t]
\centering
\begin{subfigure}
    \centering
    \includegraphics[width=0.4\textwidth]{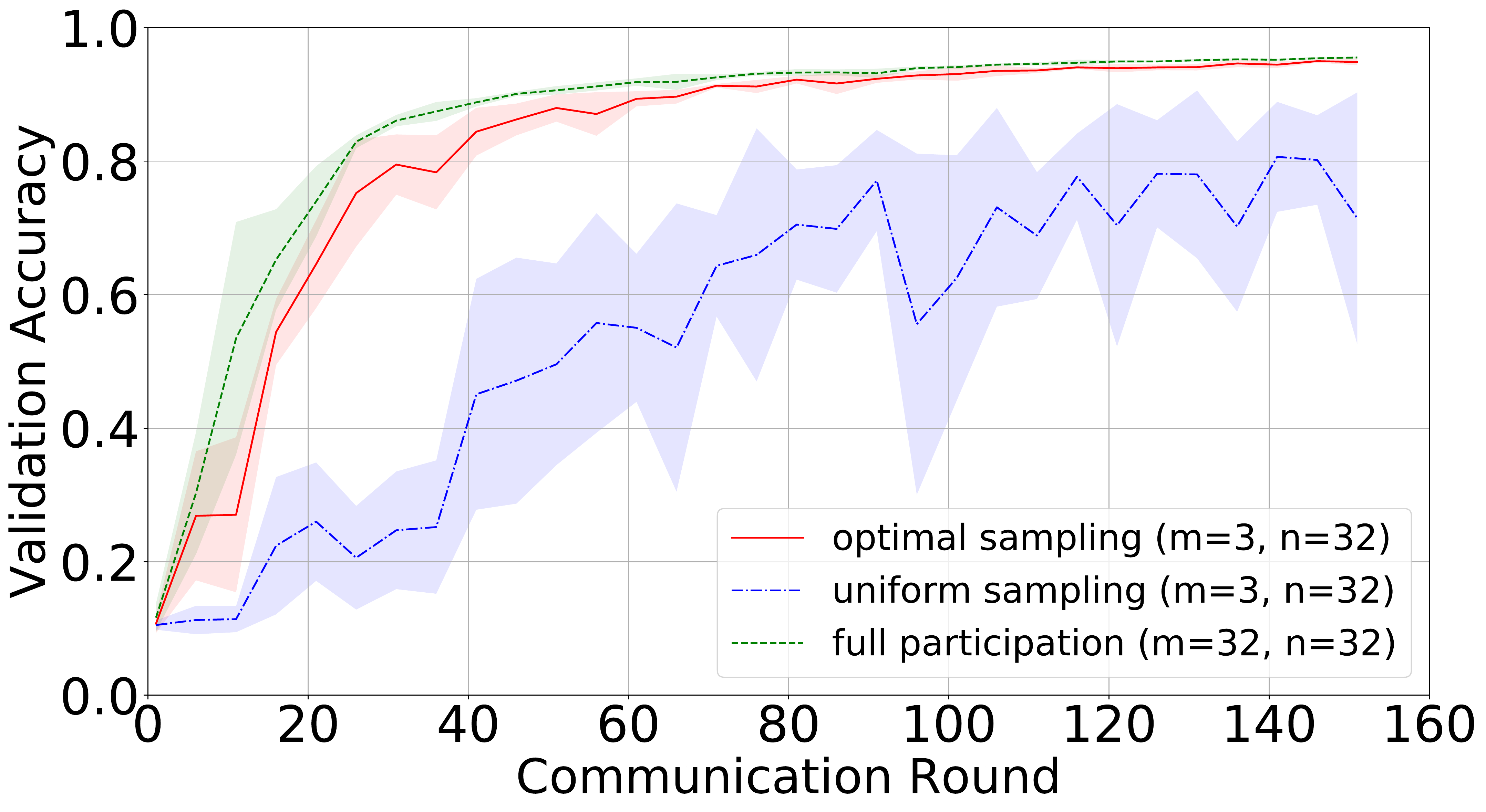}
\end{subfigure}
\begin{subfigure}
    \centering
    \includegraphics[width=0.4\textwidth]{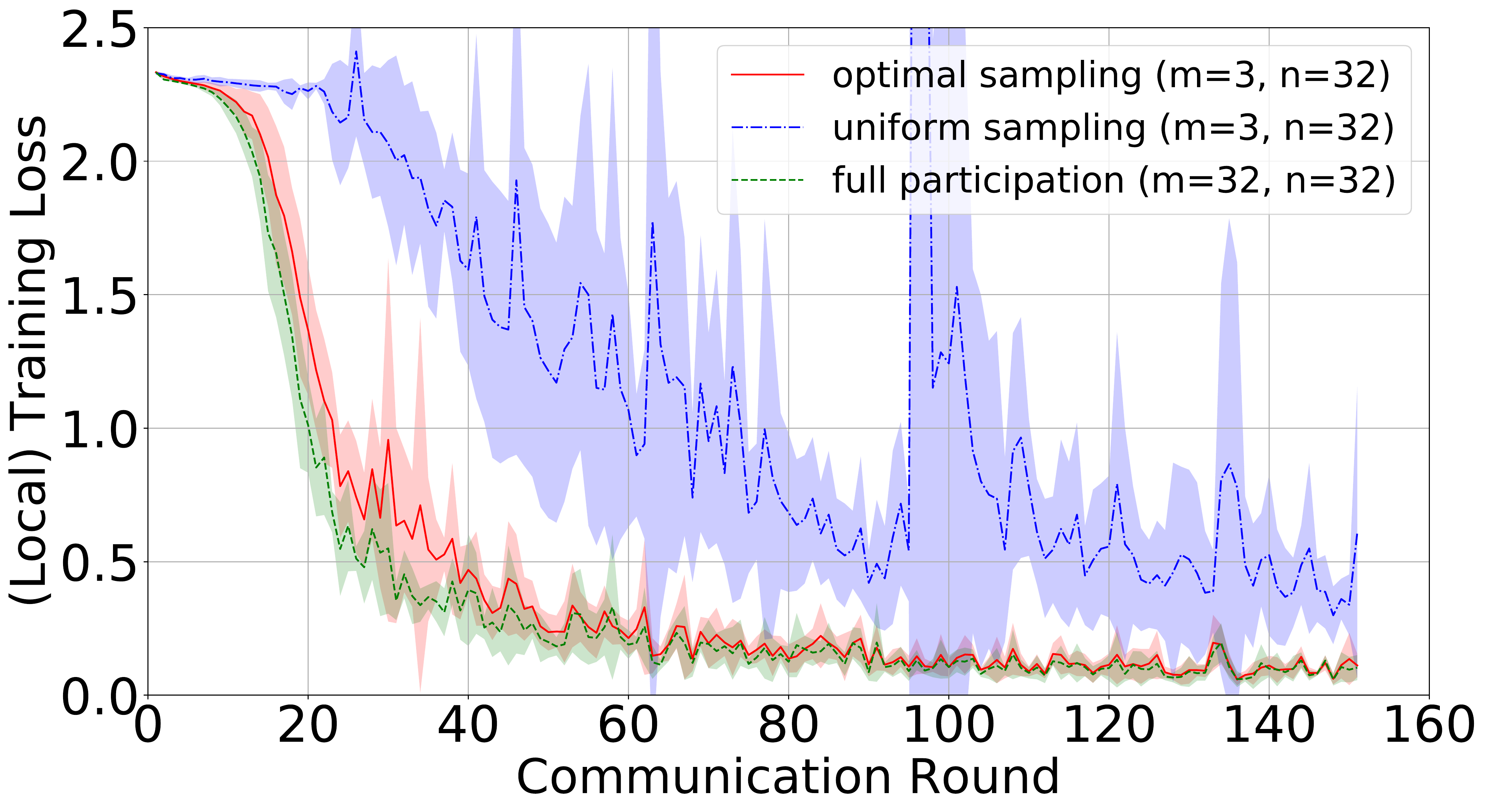}
\end{subfigure}
\begin{subfigure}
    \centering
    \includegraphics[width=0.4\textwidth]{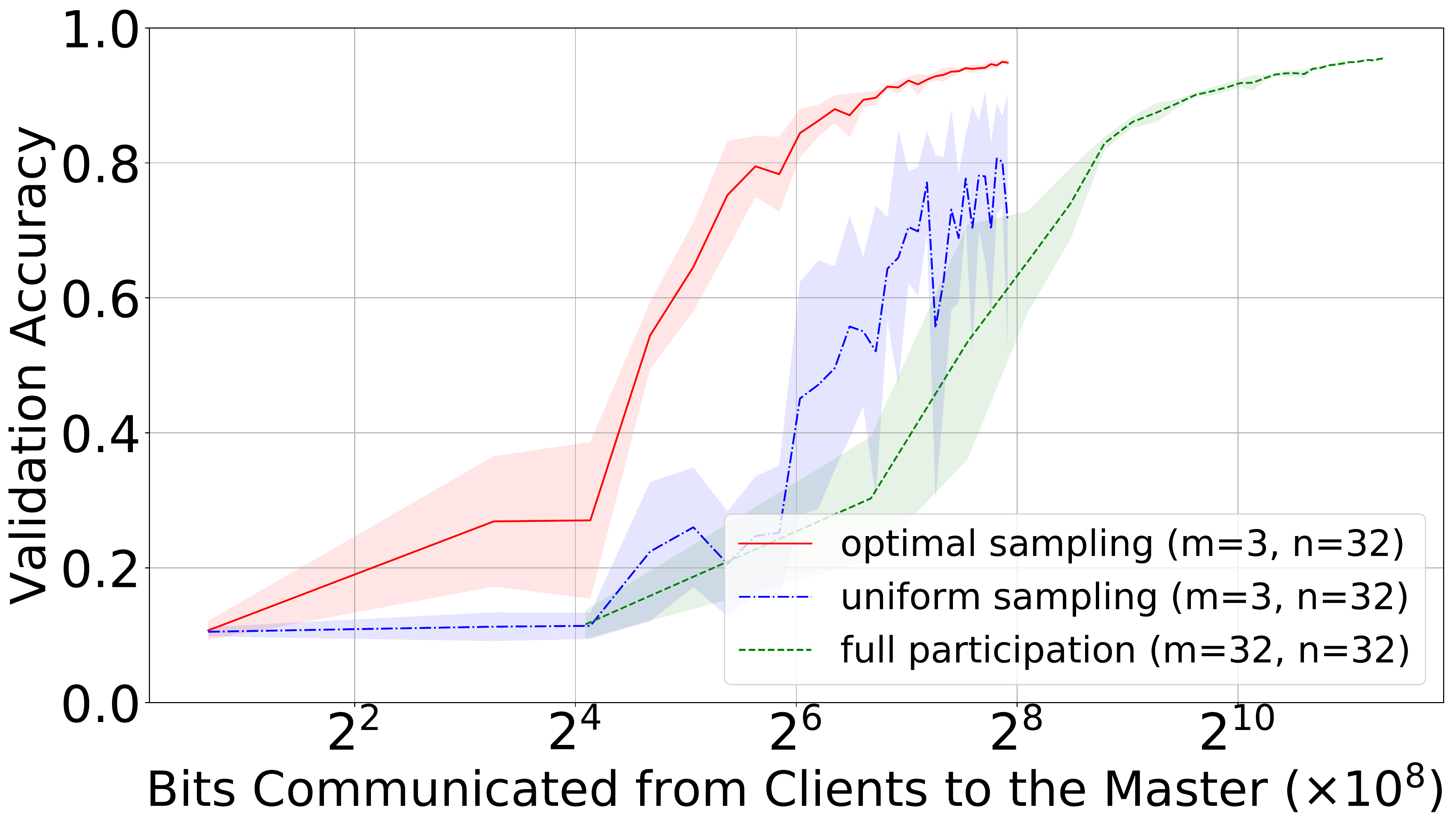}
\end{subfigure}
\begin{subfigure}
    \centering
    \includegraphics[width=0.4\textwidth]{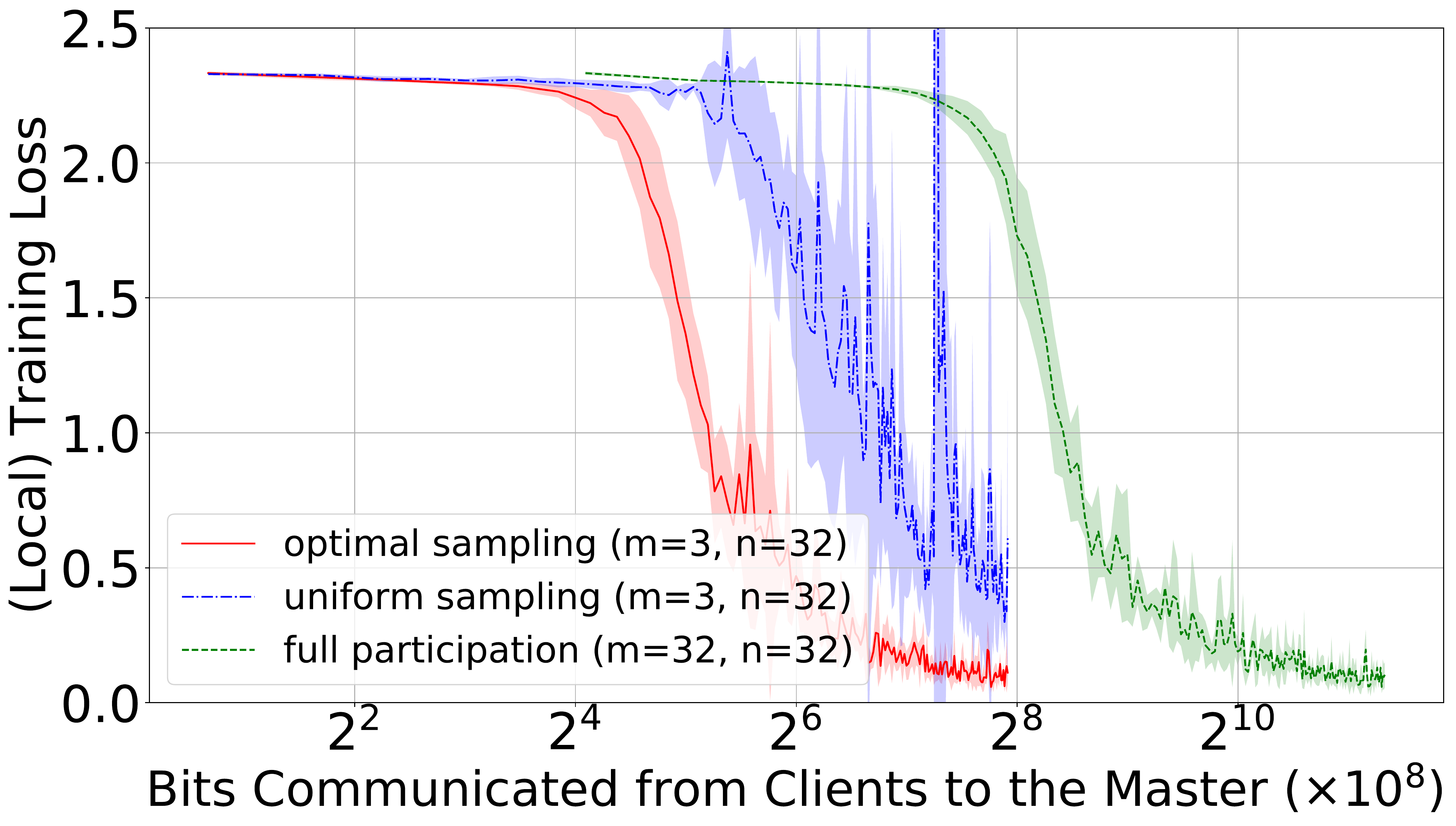}
\end{subfigure}

\caption{(FEMNIST Dataset 2) validation accuracy and (local) training loss as a function of the number of communication rounds and the number of bits communicated from clients to the master.}
\label{fig:cookup2}
\end{figure}
\begin{figure}[!t]
\centering
\begin{subfigure}
    \centering
    \includegraphics[width=0.4\textwidth]{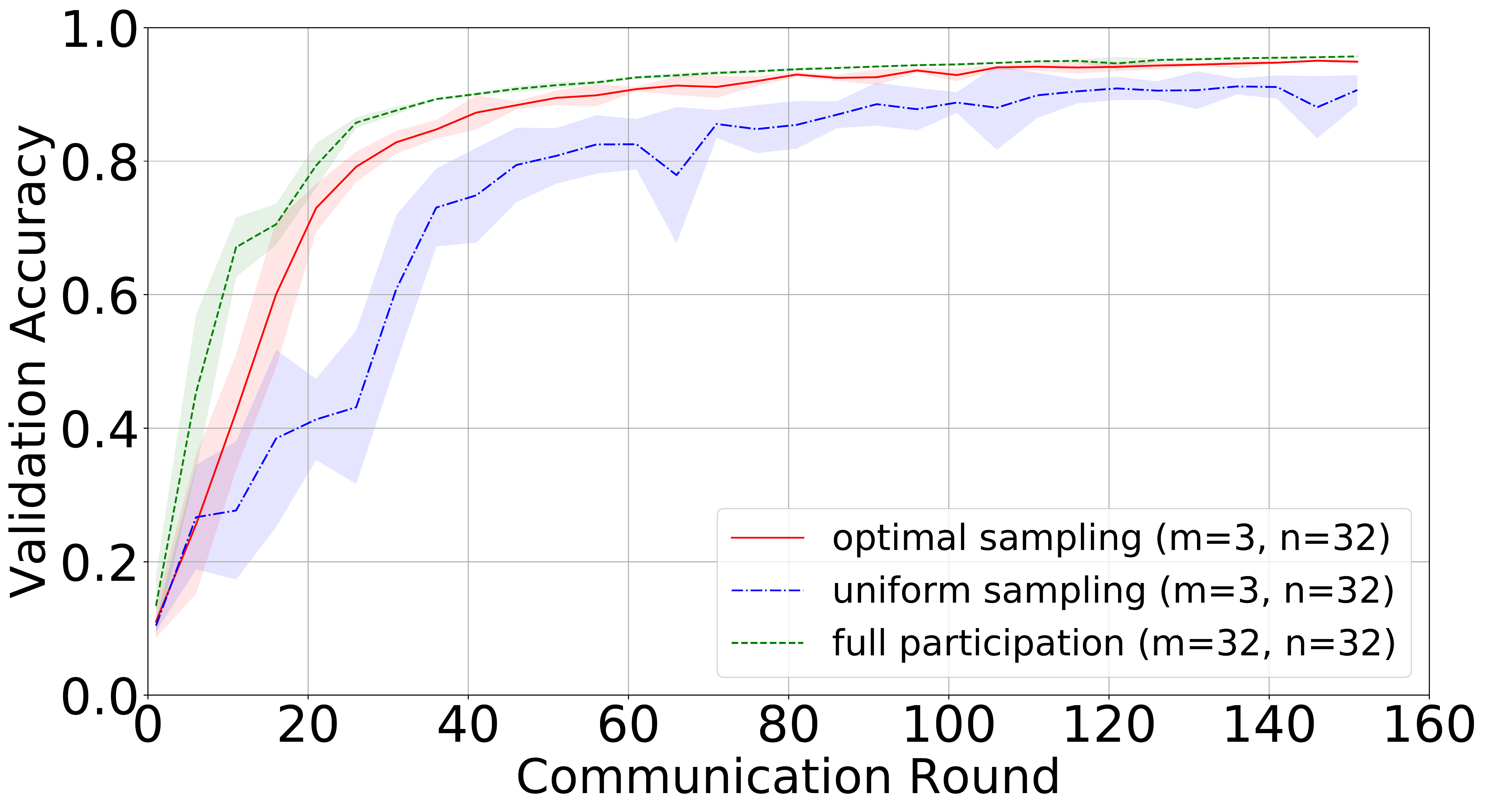}
\end{subfigure}
\begin{subfigure}
    \centering
    \includegraphics[width=0.4\textwidth]{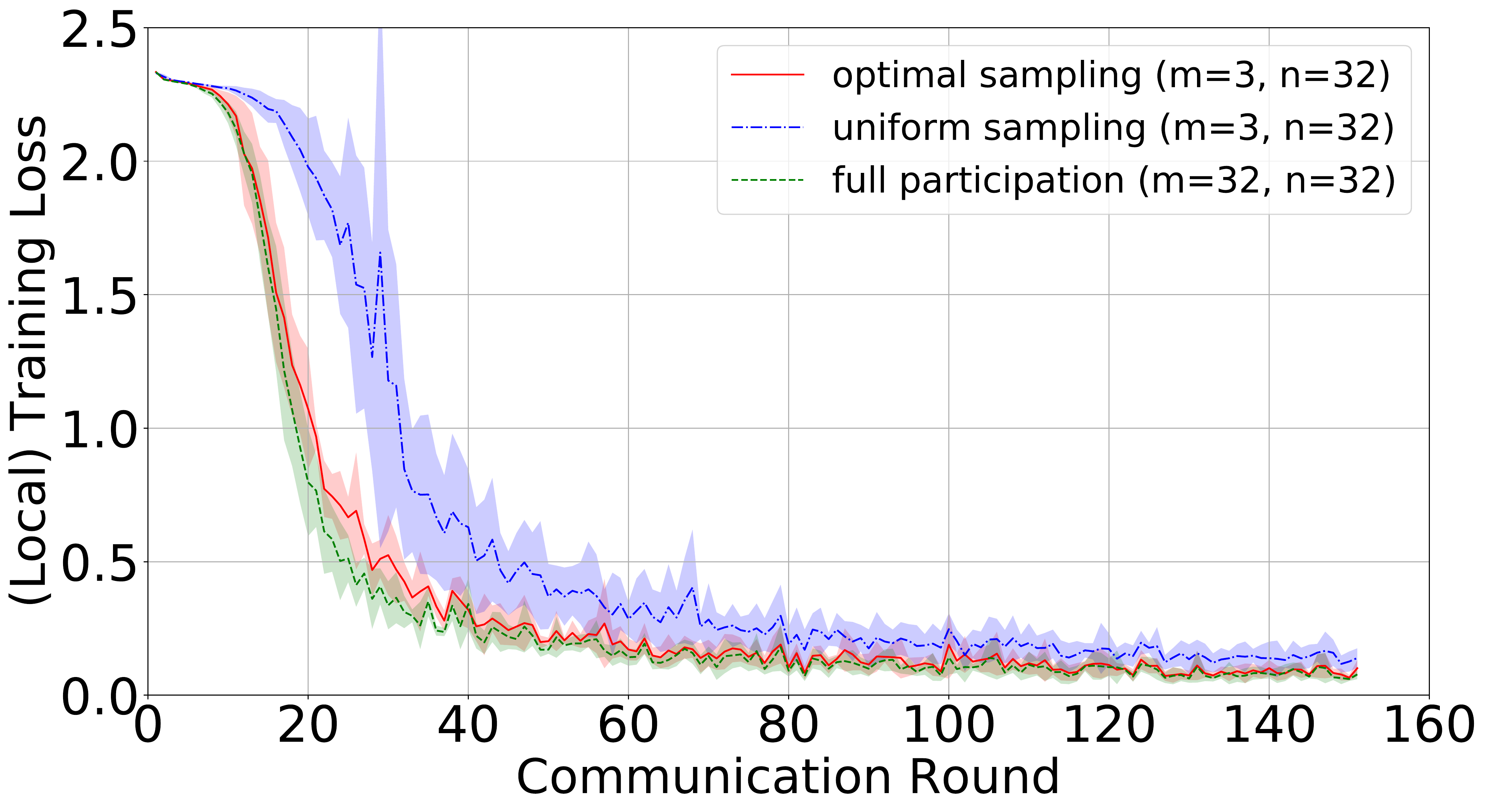}
\end{subfigure}
\begin{subfigure}
    \centering
    \includegraphics[width=0.4\textwidth]{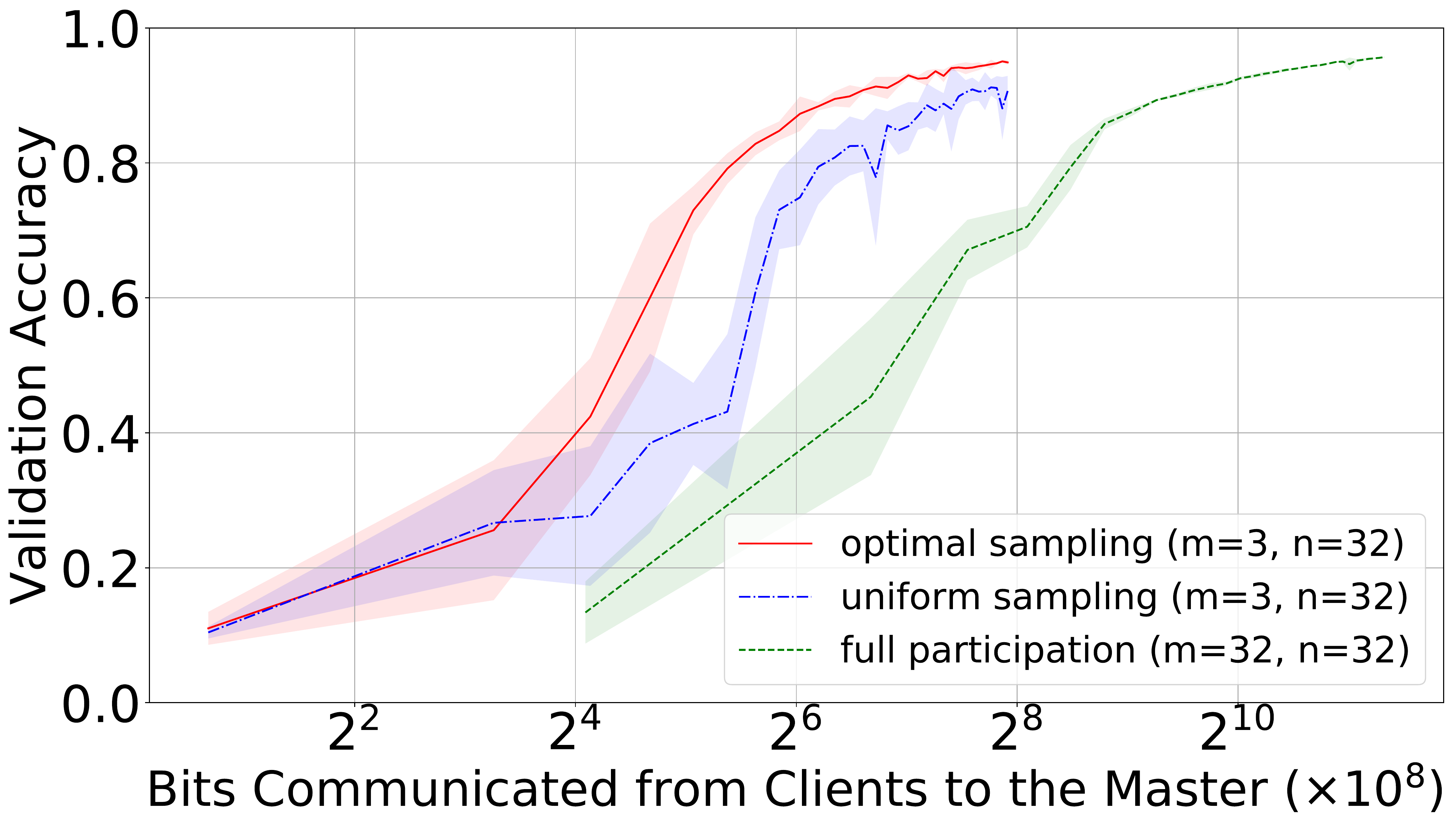}
\end{subfigure}
\begin{subfigure}
    \centering
    \includegraphics[width=0.4\textwidth]{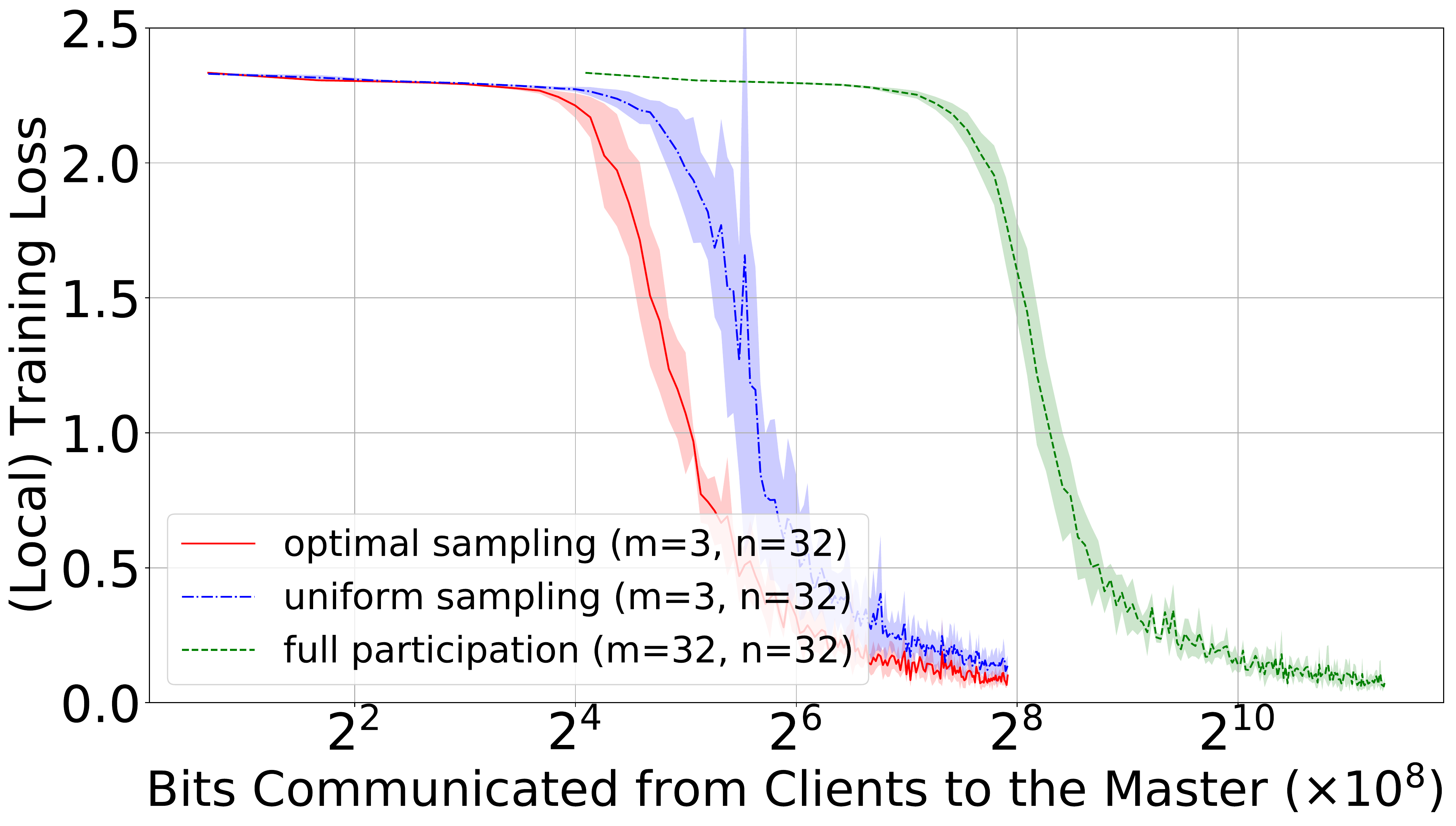}
\end{subfigure}
\caption{(FEMNIST Dataset 3) validation accuracy and (local) training loss as a function of the number of communication rounds and the number of bits communicated from clients to the master.}
\label{fig:cookup3}
\end{figure}

\subsection{Shakespeare dataset}
We also evaluate our method on the (unchanged) Shakespeare text dataset for next character prediction. The vocabulary set for this task consists of $86$ unique characters. The dataset contains $715$ clients, each corresponding to a character in Shakespeare's plays. We divide the text into batches in such a way that each batch contains $8$ example sequences of length $5$. The RNN model architecture we use is a two-layer GRU model. We set $n=32,~m=2,~j_{max}=4$ and run several venilla \texttt{SGD} steps for $1$ epoch on each client's local dataset in every communication round. We set $\eta_g=1$ and use the same strategy to tune $\eta_l$ as in the FEMNIST experiments. For full participation and optimal sampling, it turns out that $\eta_l=2^{-2}$ is the optimal local step size. For uniform sampling, the optimal is $\eta_l=2^{-3}$. The main result is shown in Figure~\ref{fig:ss}. In Appendix~\ref{appendix:shakespeare experiment}, we include more details of the experimental settings and extra results.
\begin{figure}[!t]
\centering
\begin{subfigure}
    \centering
    \includegraphics[width=0.4\textwidth]{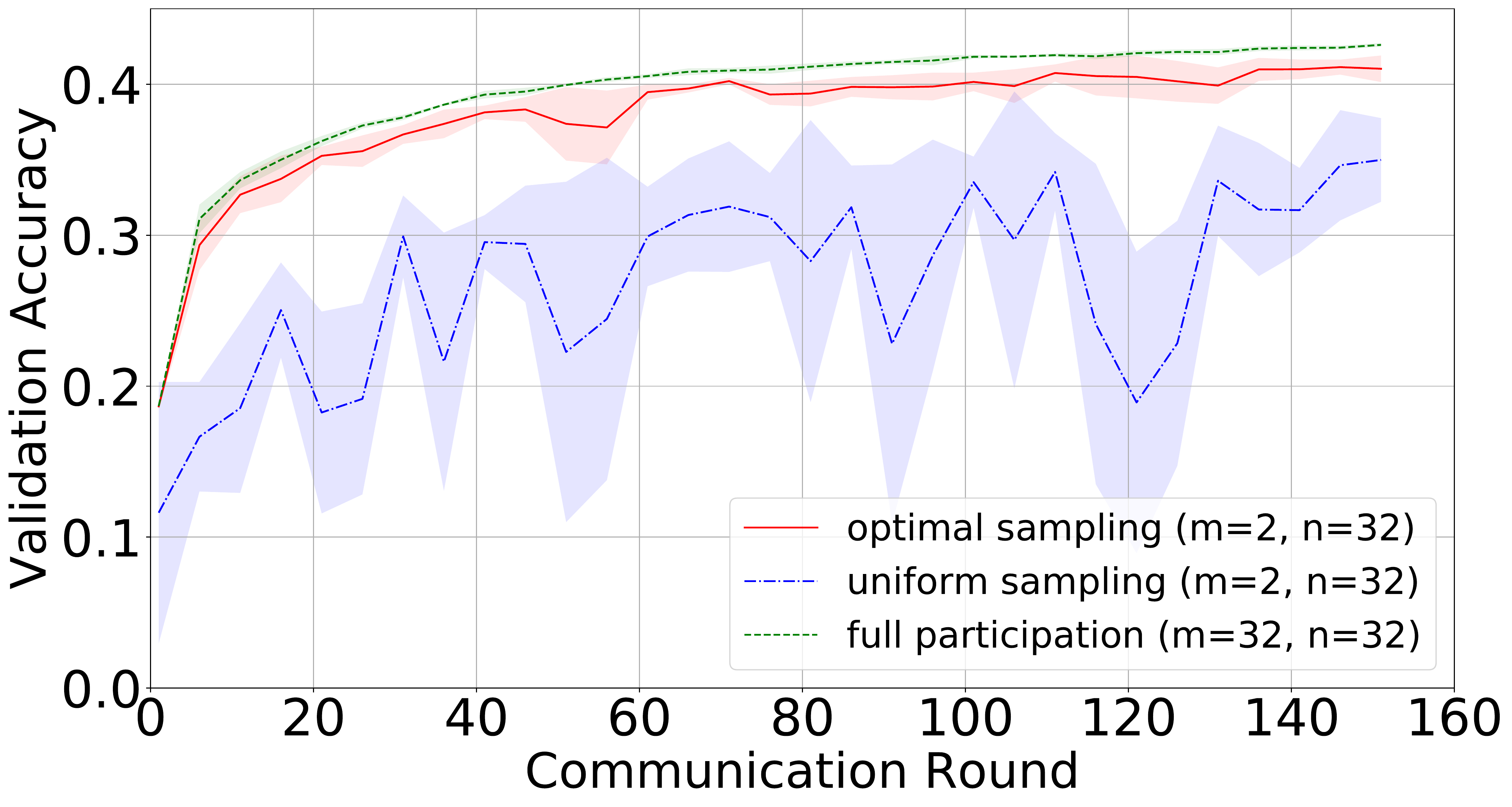}
\end{subfigure}
\begin{subfigure}
    \centering
    \includegraphics[width=0.4\textwidth]{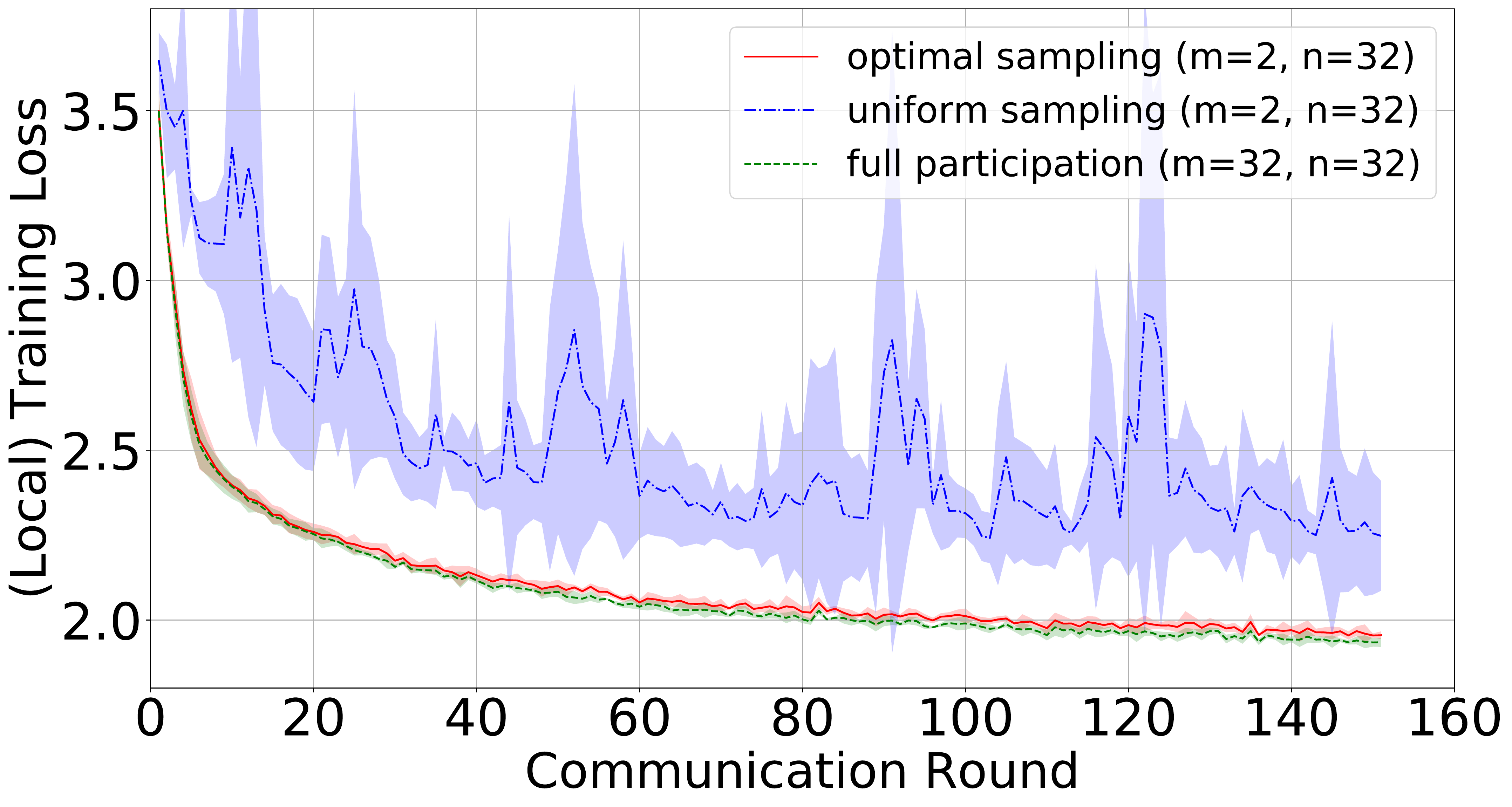}
\end{subfigure}
\begin{subfigure}
    \centering
    \includegraphics[width=0.4\textwidth]{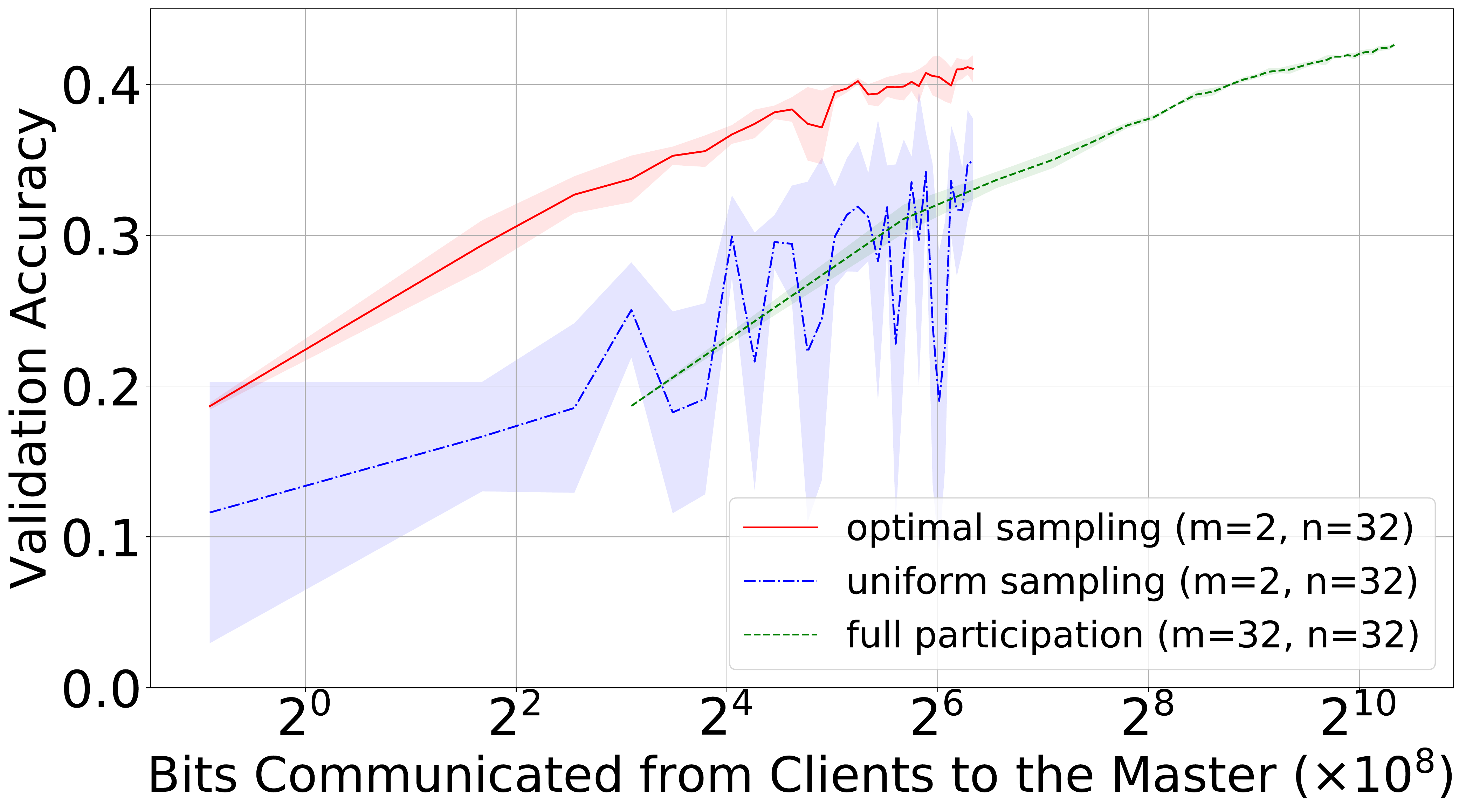}
\end{subfigure}
\begin{subfigure}
    \centering
    \includegraphics[width=0.4\textwidth]{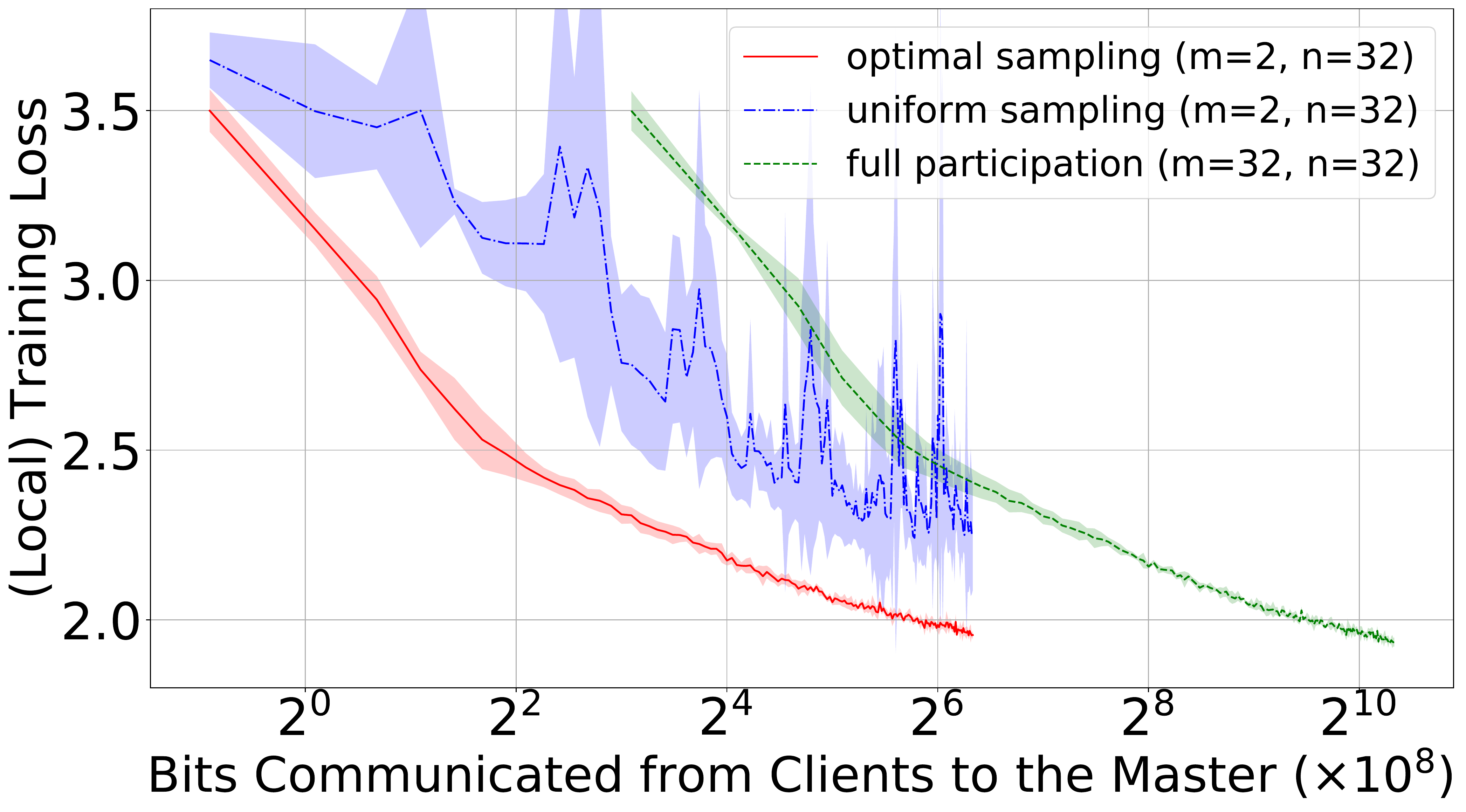}
\end{subfigure}
\caption{(Shakespeare Dataset) validation accuracy and (local) training loss as a function of the number of communication rounds and the number of bits communicated from clients to the master.}
\label{fig:ss}
\end{figure}

\subsection{Discussion}
As predicted by our theory, the performance of {\tt FedAvg} with our proposed optimal client sampling strategy is in between that with full and uniform partial participation. For all datasets, the optimal sampling strategy performs slightly worse than but is still competitive with the full participation strategy in terms of the number of communication rounds -- it almost reached the performance of full participation while only less than $10\%$ of the available clients communicate their updates back to the master. Note that the uniform sampling strategy performs significantly worse, which indicates that a careful choice of sampling probabilities can go a long way towards closing the gap between the performance of naive uniform sampling and full participation. 

More importantly, and this was the main motivation of our work, our optimal sampling strategy is significantly better than both the uniform sampling and full participation strategies when we compare validation accuracy as a function of the number of bits communicated from clients to the master. For instance, on FEMNIST Dataset 1 (Figure~\ref{fig:cookup1}), while our optimal sampling approach reached around 85\% validation accuracy after $2^6 \times 10^8$ communicated bits, neither the full nor the uniform sampling strategies are able to exceed 40\% validation accuracy within the same communication budget. Indeed, to reach the same 85\% validation accuracy, full participation approach needs to communicate more than $2^9 \times 10^8$ bits, i.e., $8\times $ more, and uniform sampling approach needs to communicate about the same number of bits as full participation or even more. The results for FEMNIST Datasets 2 and 3 and for the Shakespeare dataset are of a similar qualitative nature, showing that these conclusions are robust across the datasets considered.
\chapter{Fair and accurate federated learning under heterogeneous targets with ordered dropout}
\label{chapter7:fjord}

\section{Introduction}
\label{sec:intro_fjord}

Over the past few years, advances in deep learning have revolutionised the way we interact with everyday devices.
Much of this success relies on the availability of large-scale training infrastructures and the collection of vast amounts of training data.
However, users and providers are becoming increasingly aware of the privacy implications of this ever-increasing data collection, leading to the creation of various privacy-preserving initiatives by service providers~\cite{apple} and government regulators~\cite{gdpr}.

Federated Learning (FL)~\cite{mcmahan17fedavg} is a relatively new subfield of machine learning (ML) that allows the training of models without the data leaving the users' devices; instead, FL allows users to collaboratively train a model by moving the computation to them. At each round, participating devices download the latest model and compute an updated model using their local data. These locally trained models are then sent from the participating devices back to a central server where updates are aggregated for next round's global model.
Until now, a lot of research effort has been invested with the sole goal of maximising the accuracy of the global model~\cite{mcmahan17fedavg,liang2019think,fedprox2020mlsys,karimireddy2020scaffold,fednova2020neurips}, while complementary mechanisms have been proposed to ensure privacy and robustness~\cite{bonawitz2017practical,geyer2017differentially,mcmahan18learning,feature_leak2019sp,diff_priv_fl2020jiot,backdoor_fl2020aistats}. 
\begin{figure}[t]
\centering
\includegraphics[width=0.6\columnwidth]{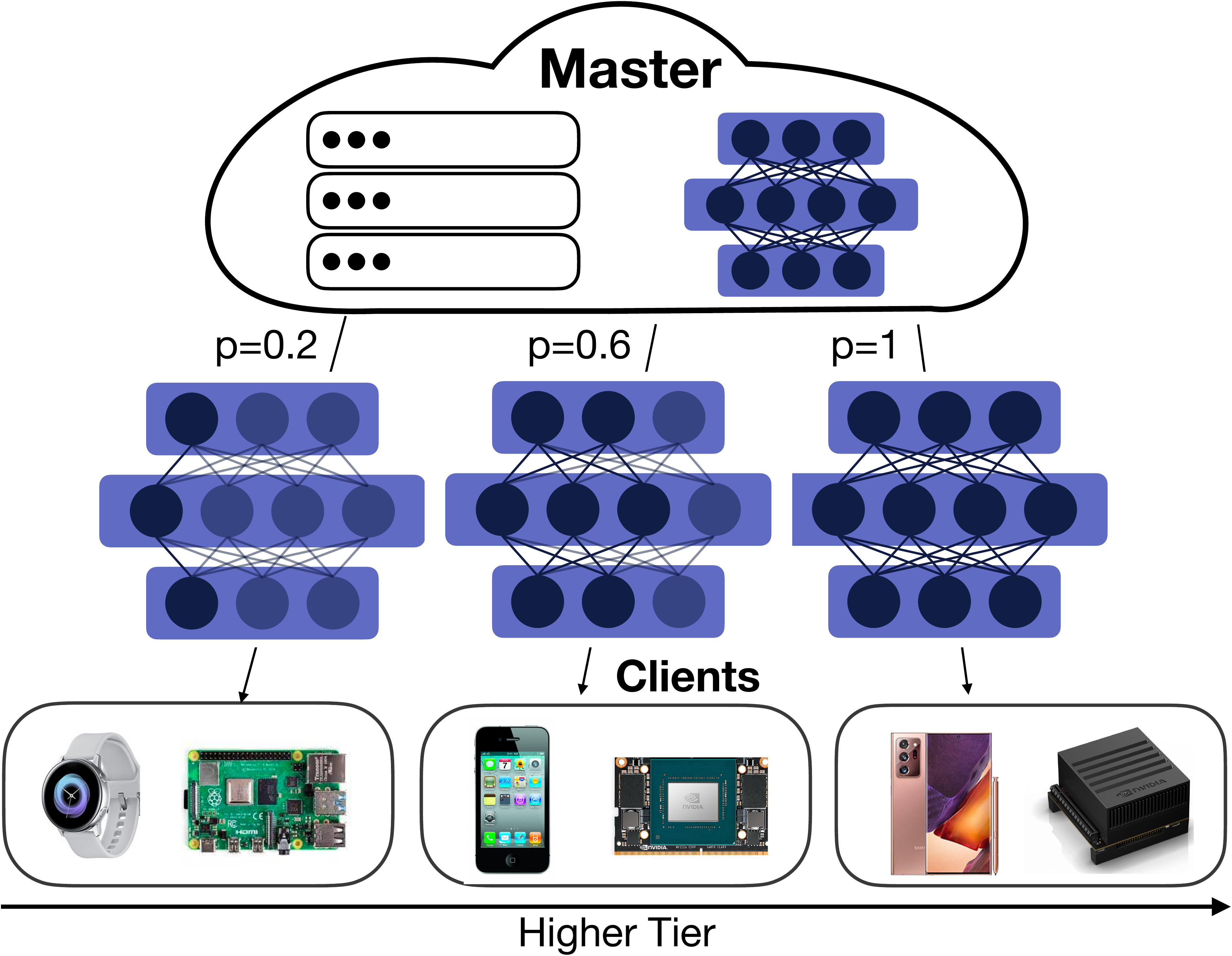}
\caption{\tool employs OD to tailor the amount of computation to the capabilities of each participating device.}
\label{fig:fjord}
\end{figure}

A key challenge of deploying FL in the wild is the vast heterogeneity of devices~\cite{li2020federated}, ranging from low-end IoT to flagship mobile devices. 
Despite this fact, the widely accepted norm in FL is that the local models have to share the \textit{same} architecture as the global model.
Under this assumption, developers typically opt to either drop low-tier devices from training, hence introducing training bias due to unseen data~\cite{kairouz2019advances}, or limit the global model's size to accommodate the slowest clients, leading to degraded accuracy due to the restricted model capacity~\cite{caldas2018expanding}. In addition to these limitations, variability in sample sizes, computation load and data transmission speeds further contribute to a very unbalanced training environment. 
Finally, the resulting model might not be as efficient as models specifically tailored to the capabilities of each device tier to meet the minimum processing-performance requirements~\cite{hapi2020iccad}.
\section{Contributions}
In this chapter, we introduce \tool (Figure~\ref{fig:fjord}), a novel adaptive training framework that enables heterogeneous devices to participate in FL by dynamically adapting model size -- and thus computation, memory and data exchange sizes -- to the available client resources. To this end, we introduce Ordered Dropout (OD), a mechanism for run-time ordered (importance-based) pruning, which enables us to extract and train submodels in a nested manner. As such, OD enables \textit{all} devices to participate in the FL process independently of their capabilities by training a submodel of the original DNN, while still contributing knowledge to the global model. Alongside OD, we propose a self-distillation method from the maximal supported submodel on a device to enhance the feature extraction of smaller submodels. Finally, our framework has the additional benefit of producing models that can be dynamically scaled during inference, based on the hardware and load constraints of the device. 

Our evaluation shows that \tool enables significant accuracy benefits over the baselines across diverse datasets and networks, while allowing for the extraction of submodels of varying FLOPs and sizes without the need for retraining.
%
%
%

\section{Related work}
\label{sec:related_work}

\subsection{Motivation}
\label{sec:background}


Despite the progress on the accuracy front, the unique deployment challenges of FL still set a limit to the attainable performance. FL is typically deployed on either siloed setups, such as among hospitals, or on mobile devices in the wild~\cite{sysdesign_fl2019mlsys}. In this chapter, we focus on the latter setting.
Hence, while cloud-based distributed training uses powerful high-end clients~\cite{fbdatacenter2018hpca}, in FL these are commonly substituted by resource-constrained and heterogeneous embedded devices. 

In this respect, FL deployment is currently hindered by the vast heterogeneity of client hardware~\cite{facebook2019hpca,ai_benchmark_2019,sysdesign_fl2019mlsys}. On the one hand, different mobile hardware leads to significantly varying processing speed~\cite{embench2019emdl}, in turn leading to longer waits upon aggregation of updates (\textit{i.e.} stragglers).  
At the same time, devices of mid and low tiers might not even be able to support larger models, \textit{e.g.}~the model does not fit in memory or processing is slow, and, thus, are either excluded or dropped upon timeouts from the training process, together with their unique data. More interestingly, the resource allocation to participating devices may also reflect on demographic and socio-economic information of owners, that makes the exclusion of such clients unfair~\cite{kairouz2019advances}. 
Analogous to the device load and heterogeneity, a similar trend can be traced in the downstream (model)
and upstream (updates)
network communication in FL, which can be an additional substantial bottleneck for the training procedure \cite{comms_fl2020tnnls}. 
%
%

\subsection{Dropout techniques}
Contrary to conventional Random Dropout~\cite{dropout2014jmlr}, which  stochastically drops a different, random set of a layer's units in every batch and is typically applied for regularisation purposes, OD employs a \textit{structured} ordered dropping scheme 
that aims primarily at tunably reducing the computational and memory cost of training and inference. However, OD can still have an implicit regularization effect since  
we encourage learning towards the top-ranked units (\textit{e.g.}~the left-most units in the example of Figure~\ref{fig:od}), as these units will be dropped less often during training. Respectively, at inference time, the load of a client can be dynamically adjusted by dropping the least important units, \textit{i.e.} adjusting the width of the network. 

To the best of our knowledge, the only similar technique to OD is Nested Dropout \cite{rippel2014learning}, where the authors proposed a similar construction, which is applied to the representation layer in autoencoders in order to enforce identifiability of the learned representation. In our case, we apply OD to every layer to elastically adapt the computation and memory requirements during training and inference.

\subsection{Traditional pruning}
Conventional non-FL compression techniques can be applicable to reduce the network size and computation needs. The majority of pruning methods~\cite{pruning2015neurips,dnn_surgery2016neurips,pruning_filters2016iclr,snip2019iclr,molchanov2019importance} aim to generate a \textit{single} pruned model 
and require access to labelled data in order to perform a costly fine-tuning/calibration for \textit{each} pruned variant.
Instead, \tool's Ordered Dropout enables the deterministic extraction of \textit{multiple} pruned models with \textit{varying} resource budgets directly after training. In this manner, we remove both the excessive overhead of fine-tuning and the need for labelled data availability, which is crucial for real-world, privacy-aware applications~\cite{privacy_learning2012neurips,privacy_dl2015ccs}.
Finally, other model compression methods~\cite{nestdnn2018mobicom,haq2019cvpr,shrinkml2019interspeech} remain orthogonal to \tool.

\subsection{System heterogeneity}
So far, although substantial effort has been devoted to alleviating the statistical heterogeneity~\cite{li2020federated} among clients~\cite{smith2017federated,fedmd2019neuripsw,hsieh2020noniid,personalised_fl2020neurips,fair_fl2020iclr}, the system heterogeneity has largely remained unaddressed.  
Considering the diversity of client devices, techniques on client selection~\cite{clientsel_fl2019icc} and control of the per-round number of participating clients and local iterations~\cite{cost_eff_fl2021infocom,adaptivefl2019jsac} have been developed.
Nevertheless, as these schemes are restricted to allocate a uniform amount of work to each selected client, they either limit the model complexity to fit the lowest-end devices or exclude slow clients altogether.
From an aggregation viewpoint, FedProx~\cite{fedprox2020mlsys} allows for partial results to be integrated to the global model, thus enabling the allocation of different amounts of work across heterogeneous clients. Despite the fact that each client is allowed to perform a different number of local iterations based on its resources, large models still cannot be accommodated on the more constrained devices. 

\subsection{Communication optimization}
The majority of existing work has focused on tackling the communication overhead in FL.
\cite{FEDLEARN2016} proposed using structured and sketched updates to reduce the transmitted data.
ATOMO~\cite{wang2018atomo} introduced a generalised gradient decomposition and sparsification technique, aiming to reduce the gradient sizes communicated upstream. 
\cite{adapt_grad_sparse_fl2020icdcs} adaptively select the gradients' sparsification degree based on the available bandwidth and computational power. 
Building upon gradient quantisation methods~\cite{deepgradcompress2018iclr, Cnat, muppet2020icml, horvath2021a},
\cite{quant_grad_fl2020arxiv} proposed using quantisation in the model sharing and 
aggregation steps. However, their scheme requires the \textit{same} clients to participate across all rounds, and is, thus, unsuitable for realistic settings where clients' availability cannot be guaranteed.

Despite the bandwidth savings, these communication-optimizing approaches do not offer computational gains 
nor do they address device heterogeneity.
Nonetheless, they remain orthogonal to our work and can be complementarily combined to further alleviate the communication cost. 

\subsection{Computation-communication co-optimization}
A few works aim to co-optimize both the computational and bandwidth costs.
PruneFL \cite{prunefl2020neuripsw} proposes an unstructured pruning method. Despite the similarity to our work in terms of pruning, this method assumes a \textit{common} pruned model across \textit{all} clients at a given round, thus not allowing more powerful devices to update more weights. 
Hence, the pruned model needs to meet the constraints of the least capable devices, which severely limits the model capacity. Moreover, the adopted unstructured sparsity is difficult to translate to processing speed gains~\cite{balancedsparse2019aaai}.
Federated Dropout~\cite{caldas2018expanding} randomly sparsifies the global model, before sharing it to the clients.
Similarly to PruneFL, Federated Dropout does not consider the system diversity and distributes the \textit{same} model to all clients. Thus, it is restricted by the low-end devices or excludes them altogether from the FL process.


Contrary to the presented works, our framework embraces the client heterogeneity, instead of treating it as a limitation, and thus pushes the boundaries of FL deployment in terms of fairness, scalability and performance by tailoring the model size to the device at hand. 



\section{Ordered dropout}
\label{sec:ordered_dropout}

\begin{figure}[t]
    \centering
    \includegraphics[width=\columnwidth]{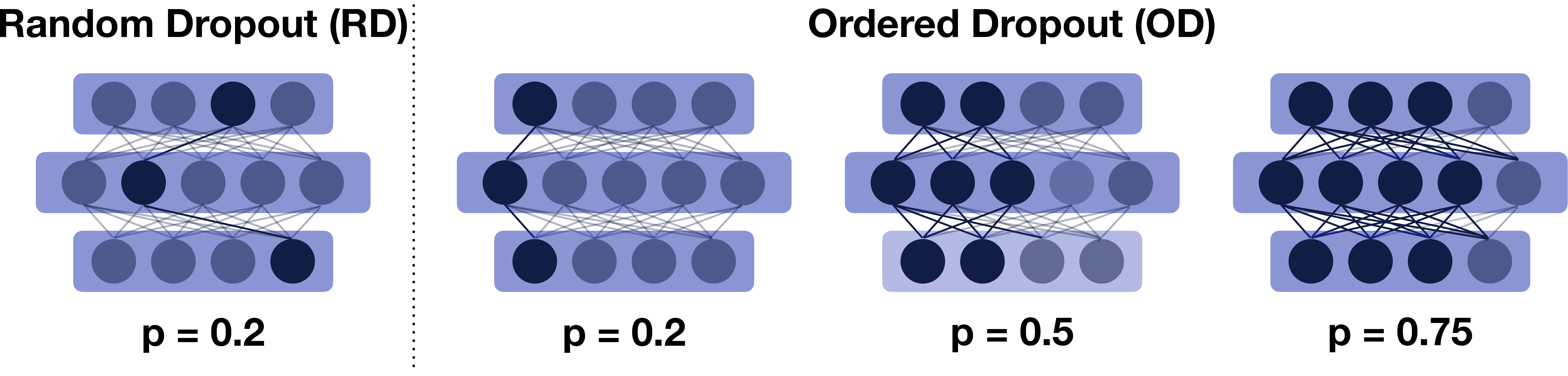}
    \caption{Ordered vs. Random Dropout. In this example, the left-most features are used by more devices during training, creating a natural ordering to the importance of these features.}
    \label{fig:od}
\end{figure}
In this work, we firstly introduce the tools that act as enablers for heterogeneous federated training. Concretely, we have devised a mechanism of importance-based pruning for the easy extraction of subnetworks from the original, specially trained model, each with a different computational and memory footprint. We name this technique \textbf{Ordered Dropout} (OD), as it orders knowledge representation in nested submodels of the original network.

More specifically, our technique starts by sampling a value (denoted by~$p$) from a distribution of candidate values. Each of these values corresponds to a specific submodel, which in turn gets translated to a specific computational and memory footprint (see Table~\ref{tab:macs}). Such sampled values and associations are depicted in Figure~\ref{fig:od}. Contrary to conventional dropout (RD), our technique drops adjacent components of the model instead of random neurons, which translates to computational benefits in today's linear algebra libraries and higher accuracy as shown later.


\subsection{Ordered dropout mechanics}
\label{sec:od_mechanics}



The proposed OD method is parametrised with respect to:
\begin{enumerate}[label=\roman*)]
    \item the value of the dropout rate $p \in (0,1]$ per layer,
    \item the set of candidate values $\cP$, such that $p\in\cP$, and
    \item the sampling method of $p$ over the set of candidate values, such that $p\sim \Dp$, where $\Dp$ is the distribution over $\cP$.
\end{enumerate}

A primary hyperparameter of OD is the dropout rate $p$ which defines how much of each layer is to be included, with the rest of the units dropped in a structured and ordered manner. The value of $p$ is selected by sampling from the dropout distribution $\Dp$ which is represented by a set of discrete values $$\cP = \{s_1, s_2, \dots, s_{|\cP|}\}$$ such that $0$$<$$s_1$$<$$\dots$$<$$s_{|\cP|} \leq 1$ and probabilities $$\Prob(p=s_i) > 0, \; \forall i \in [|\cP|]$$ such that  $\sum_{i=1}^{|\cP|}\Prob(p=s_i) = 1$. For instance, a uniform distribution over $\cP$ is denoted by $p \sim \cU_{\cP}$ (\textit{i.e.}~$D=\cU$). In our experiments we use uniform distribution over the set $ \cP = \{\nicefrac{i}{k}\}_{i=1}^k$, which we refer to as $\cU_k$ (or \textit{uniform-$k$}).
The discrete nature of the distribution stems from the innately discrete number of neurons or filters to be selected. The selection of set $\cP$ is discussed in the next subsection.

The dropout rate $p$ can be constant across all layers or configured individually per layer $l$, leading to $p_l \sim D^l_{\cP}$.
As such an approach opens the search space dramatically, we refer the reader to NAS techniques~\cite{nas2017iclr} and continue with 
the same $p$ value across network layers for simplicity, without hurting the generality of our approach.

Given a $p$ value, a pruned $p$-subnetwork can be directly obtained as follows. For each\footnote{Notice that OD is not applied on the last layer in order to maintain the same output dimensionality.} layer $l$ with width\footnote{\textit{i.e.}~neurons for fully-connected layers (linear and recurrent) and filters for convolutional layers.} $K_l$ , the submodel for a given $p$ has all neurons/filters with index $\{0, 1, \dots, \ceil{ p\cdot K_l}-1\}$ included and \mbox{$\{\ceil{p \cdot K_l}, \dots, K_l-1\}$} pruned. Moreover, the unnecessary connections between pruned neurons/filters are also removed\footnote{
For BatchNorm, we maintain a separate set of statistics for every dropout rate $p$. 
This has only a marginal effect on \#parameters and can be used in a privacy-preserving manner~\cite{li2021fedbn}.}.
%
%
We denote a pruned $p$-subnetwork $\F_p$ with its weights $\w_p$, where $\F$ and $\w$ are the original network and weights, respectively.
Importantly, contrary to existing pruning techniques~\cite{pruning2015neurips,snip2019iclr,molchanov2019importance}, a $p$-subnetwork from OD can be directly obtained post-training without the need to fine-tune, thus eliminating the requirement to access any labelled data.


\subsection{Training ordered dropout formulation}
\label{sec:od-training}

We propose two ways to train an OD-enabled network: i)~\textit{plain OD} and ii)~\textit{knowledge distillation} OD training (OD w/ KD).
%
In the first approach, in each step we first sample $p \sim \Dp$; then we perform the forward and backward pass using the $p$-reduced network $\F_p$; finally we update the submodel's weights using the selected optimizer.
Since sampling a $p$-reduced network provides us significant computational savings on average, we can exploit this reduction to further boost accuracy.
Therefore, in the second approach we exploit the nested structure of OD, \textit{i.e.}~\mbox{$p_1 < p_2 \implies \F_{p_1} \subset \F_{p_2}$} and allow for the bigger capacity supermodel to teach the sampled $p$-reduced network at each iteration via knowledge distillation (teacher $p_{\M} > p$, $p_{\M}  = \max \cP$). 
%
In particular, in each iteration, the loss function  consists of two components as follows: 
%

\begin{align}
    \begin{split}
        \mathcal{L}_d&(\sm_{p}, ~\sm_{p_{\M}}, \boldsymbol{y}_{\text{label}}) = \\
        & (1 - \alpha)\text{CE}(\max(\sm_{p}), \boldsymbol{y}_{\text{label}}) 
     + \alpha \text{KL}(\sm_{p}, \sm_{p_{\M} }, T)
    \end{split}
    \label{eq:distill}
\end{align}
where $\sm_p$ is the \textit{softmax} output of the sampled $p$-submodel, $\boldsymbol{y}_{\text{label}}$ is the ground-truth label, $\text{CE}$ is the cross-entropy function, $\text{KL}$ is the KL divergence, $T$ is the distillation temperature~\cite{hinton2014distilling} and $\alpha$ is the relative weight of the two components. We observed in our experiments always backpropagating also the teacher network further boosts performance. Furthermore, the best performing values for distillation were $\alpha = T = 1$, thus smaller models exactly mimic the teacher output. 

\subsection{Ordered dropout exactly recovers SVD}

We further show that our new OD formulation can recover the Singular Value Decomposition (SVD) in the case where there exists a linear mapping from features to responses. We formalise this claim in the following theorem.

\begin{theorem}
\label{thm:od_is_svd}
Let $\F: \R^n \rightarrow \R^m $ be a neural network with two fully-connected linear layers with no activation or biases and $K = \min\{m,n\}$ hidden neurons. Moreover, let data $\mathcal{X}$ come from a uniform distribution on the $n$-dimensional unit ball and $A$ be an $m \times n$ full rank matrix with $K$ distinct singular values. If response $y$ is linked to data $\mathcal{X}$ via a linear map:  $x \rightarrow Ax$ and distribution $\Dp$ is such that for every $b \in [K]$ there exists $p \in \cP$ for which $b = \ceil{ p\cdot K}$, then for the optimal solution of 
$$\min_{U,V} \EE{x \sim \mathcal{X},\ p \sim \Dp}{\|\F_p(x) - y \|^2}$$
it holds $\F_p(x) = A_b x$, where $A_b$ is the best $b$-rank approximation of $A$ and $b = \ceil{ p\cdot K}$.
\end{theorem}

Theorem~\ref{thm:od_is_svd} shows that our OD formulation exhibits not only intuitively, but also theoretically ordered importance representation.
Proof of this claim is deferred to the Appendix.

\subsection{Model-device association}

\textbf{Computational and Memory Implications.}
\label{sec:od_comp_mem_gains}
The primary objective of OD is to alleviate the excessive computational and memory demands of the training and deployment processes.
When a layer is shrunk through OD, there is no need to perform the forward and backward passes or gradient updates on the pruned units. 
As a result, OD offers gains both in terms of FLOP count and model size. 
In particular, for every fully-connected and convolutional layer, the number of FLOPs and weight parameters is reduced by $\nicefrac{K_1 \cdot K_2}{\ceil{ p \cdot K_1 } \cdot \ceil{ p \cdot K_2 }} \sim \nicefrac{1}{p^2}$, where $K_1$ and $K_2$ correspond to the number of input and output neurons/channels, respectively. Accordingly, the bias terms are reduced by a factor of $\nicefrac{K_2}{\ceil{ p \cdot K_2 }} \sim \nicefrac{1}{p}$. The normalisation, activation and pooling layers are compressed in terms of FLOPs and parameters similarly to the biases in fully-connected and convolutional layers. This is also evident in Table~\ref{tab:macs}.
Finally, smaller model size also leads to reduced memory footprint for gradients and the optimizer's state vectors such as momentum. However, how are these submodels related to devices in the wild and how is this getting modelled?

\textbf{Ordered Dropout Rates Space.}
Our primary objective with OD is to tackle device heterogeneity. Inherently, each device has certain capabilities and can run a specific number of model operations within a given time budget.
Since each $p$ value defines a submodel of a given width, we can indirectly associate a $\pim$ value with the $i$-th device capabilities, such as memory, processing throughput or energy budget. As such, each participating client is given at most the \mbox{$\pim$-submodel} it can handle. 

Devices in the wild, however, can have dramatically different capabilities; a fact further exacerbated by the co-existence of previous-generation devices. Modelling discretely each device becomes quickly intractable at scale. 
Therefore, we cluster devices of similar capabilities together and subsequently associate a single $\pim$ value with each cluster.
This clustering can be done heuristically (\textit{i.e.}~based on the specifications of the device) or via benchmarking of the model on the actual device and is considered a system-design decision for our work.
As smartphones nowadays run a multitude of simultaneous tasks~\cite{starfish2015mobisys},
our framework can further support modelling of transient device load by reducing its associated $\pim$, which essentially brings the capabilities of the device to a lower tier at run time, thus bringing real-time adaptability to \tool.


Concretely, the discrete candidate values of $\cP$ depend on i)~the number of clusters and corresponding device tiers, ii)~the different load levels being modelled and iii)~the size of the network itself, as \textit{i.e.}~for each tier $i$ there exists $\pim$ beyond which the network cannot be resolved. In this chapter, we treat the former two as invariants (assumed to be given by the service provider), but provide results across different number and distributions of clusters, models and datasets.

\subsection{Preliminary results}
\label{sec:non-fl-exps}
\begin{figure}[t]
    \centering
    \begin{tabular}{ccc}
        \subfigure[ResNet18 - CIFAR10]{
            \includegraphics[width=0.3\textwidth]{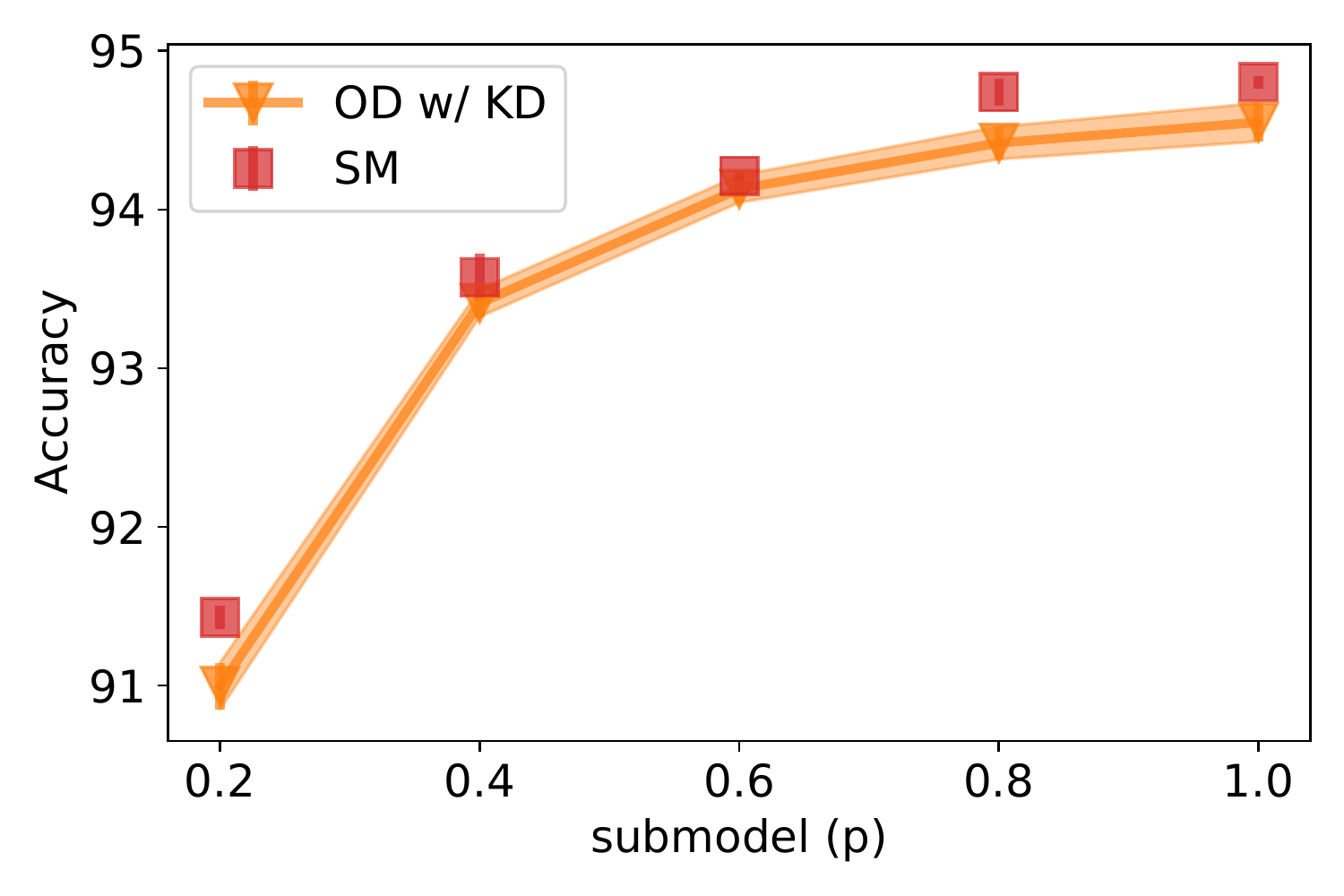}
        }
        \subfigure[CNN - EMNIST]{
            \includegraphics[width=0.3\textwidth]{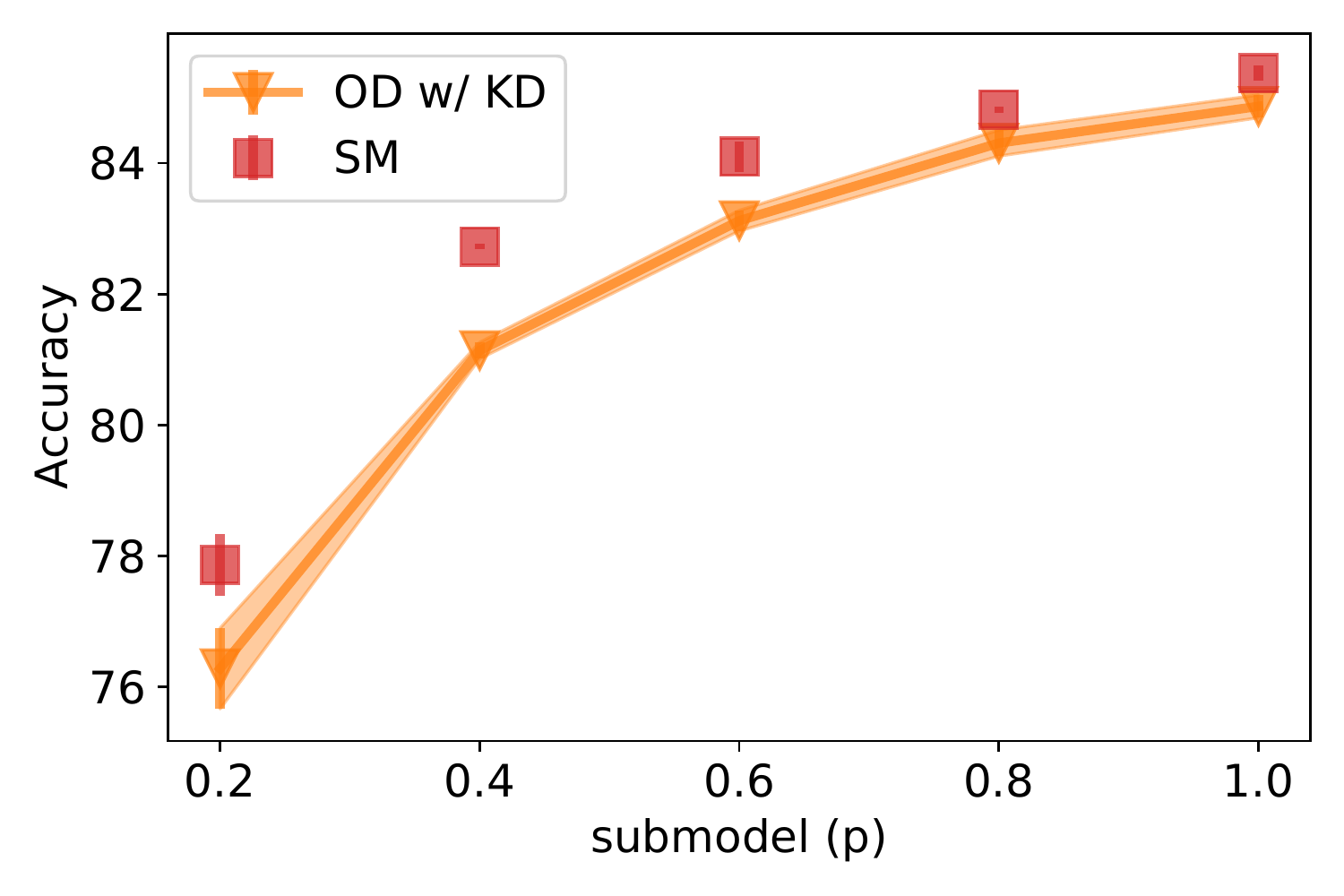}
        }
        \subfigure[RNN - Shakespeare]{
            \includegraphics[width=0.3\textwidth]{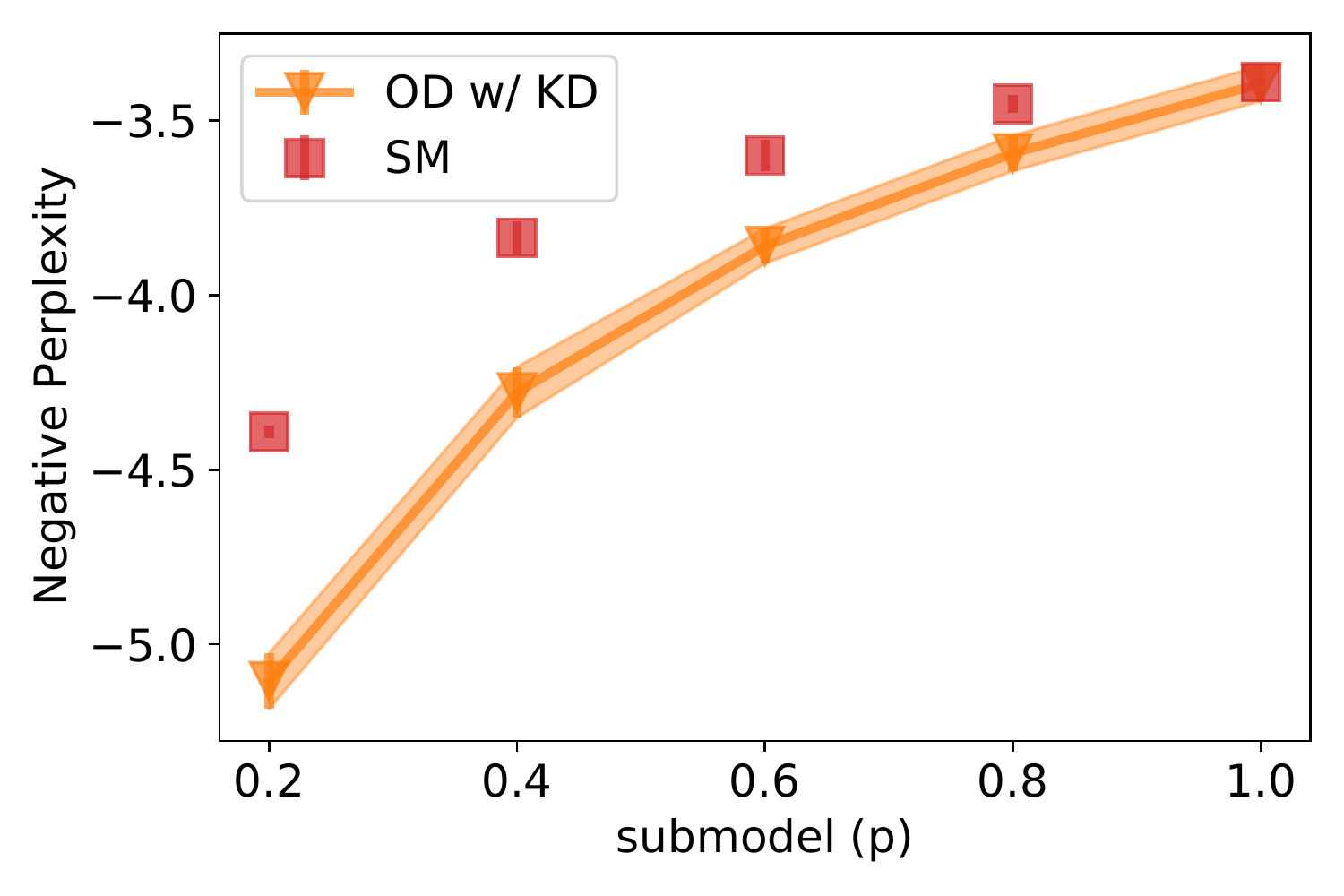}
        }
    \end{tabular}
    \caption{Full non-federated datasets. OD-Ordered Dropout with $\Dp = \cU_5$, SM-single independent models, KD-knowledge distillation.}
    \label{fig:non-fl-baseline}
\end{figure}
In this section, we present some results to showcase the performance of OD in the centralised non-FL training setting (i.e. the server has access to all training data) across three tasks, explained in detail in Section~\ref{sec:evaluation}.

We run OD with distribution $\Dp = \cU_5$ (uniform distribution over the set $\{\nicefrac{i}{5}\}_{i=1}^5$) and compare it with end-to-end trained submodels (SM) trained in isolation for the given width of the model. Figure~\ref{fig:non-fl-baseline} shows that across the three datasets, the best attained performance of OD along every width $p$ is very close to the performance of the baseline models.
We note at this point that the submodel baselines are trained from scratch, explicitly optimized to that given width with no possibility to jump across them, while our OD model was trained using a single training loop and offers the ability to switch between accuracy-computation points without the need to retrain. 

\section{\tool}
\label{sec:fl_w_od}



Building upon the shoulders of OD, we introduce \tool, a framework for federated training over heterogenous clients. We subsequently describe the workflow of \tool, further documented in Algorithm~\ref{alg:flanders_training}.


As a starting point, the global model architecture, $\F$, is initialised with weights $\w^0$, either randomly or via a pretrained network. The dropout rates space $\cP$ is selected along with distribution $\Dp$ with $|\cP|$ discrete candidate values, with each $p$ corresponding to a subnetwork of the global model with varying FLOPs and parameters. Next, the devices to participate are clustered into $|\cC_\ti|$ tiers and a $\pcm$ value is associated with each cluster $c$. The resulting $\pcm$ represents the maximum capacity of the network that the devices in this cluster can handle without violating a latency or memory footprint constraint.

At the beginning of each communication round $t$,  the set of participating devices $\mS_t$ is determined, which either consists of all available clients $\mA_t$ or contains only a random subset of $\mA_t$ based on the server's capacity. 
Next, the server broadcasts the current model to the set of clients $\mS_t$ and each client $i$ receives $\w_{\pim}$.

On the client side, each client runs $E$ local iterations and at each local iteration $k$, the device $i$ samples $p_{(i,k)}$ from conditional distribution  $\Dp|\Dp \leq \pim$ which accounts for its limited capability. Subsequently, each client updates the respective weights ($\w_{p_{(i,k)}}$) of the local submodel using the \texttt{FedAvg}~\cite{mcmahan17fedavg} update rule. In this step, other strategies~\cite{fedprox2020mlsys,fednova2020neurips,karimireddy2020scaffold} can be interchangeably employed.
At the end of the local iterations, each device sends its update back to the server.

Finally, the server aggregates these communicated changes and updates the global model, to be distributed in the next global federated round to a different subset of devices.
Heterogeneity of devices leads to heterogeneity in the model updates and, hence, we need to account for that in the global aggregation step. To this end, we utilise the following aggregation rule

\begin{flalign}
    \begin{small}
        \w^{t+1}_{s_j} \setminus \w^{t+1}_{s_{j - 1}} =  \text{WA} \left(\left\{ \w^{(i, t, E)}_{i_{s_j}} \setminus \w^{(i, t, E)}_{s_{j - 1}}\right\}_{i \in \mS^j_t}\right)
        \label{eq:global_update}
    \end{small}
\end{flalign}

where $\w_{s_j} \setminus \w_{s_{j - 1}}$ are the weights that belong to $\F_{s_j}$ but not to $\F_{s_{j - 1}}$, $\w^{t+1}$ the global weights at communication round $t+1$, $\w^{(i, t, E)}$ the weights on client $i$ at communication round $t$ after $E$ local iterations, \mbox{\small $\mS^j_t = \{i \in \mS_t: \pim \geq s_j\}$} a set of clients that have the capacity to update $\w_{s_j}$, and WA stands for weighted average, where weights are proportional to the amount of data on each client.  

\begin{algorithm}[t]
    \begin{algorithmic}[1]
    	\STATE {\bfseries Input:}  $\F, \w^0, \Dp, T, E$
    	\FOR{$t \gets 0$ \dots $T-1$}
    	   \STATE Server selects clients as a subset $\mS_t \subset \mA_t$ \;
    	   \STATE Server broadcasts weights to each client  \;
    	   \FOR{$k \gets 0$ \dots $E-1$}
    	        \STATE $\forall i \in \mS_t$: Device $i$ samples $p_{(i,k)} \sim \Dp | \Dp \leq \pim$ and updates the weights of local model
    	    \ENDFOR
    	    \STATE $\forall i \in \mS_t$: device $i$ sends to the server the updated weights \STATE $\w^{(i, t, E)}$  \;
    	    \STATE Server updates $\w^{t+1}$ as in Eq.~(\ref{eq:global_update}) \;
    	\ENDFOR
    \end{algorithmic}
    \caption{\footnotesize \mbox{\textbf{\tool} (Proposed Framework)}}
    \label{alg:flanders_training}
\end{algorithm}

\subsection{Communication savings}
In addition to the computational savings (\S\ref{sec:od_comp_mem_gains}), OD provides additional communication savings. First, for the server-to-client transfer, every device with $\pim < 1$ observes a reduction of approximately $\nicefrac{1}{(\pim)^2}$ in the downstream transferred data due to the smaller model size (Section~\ref{sec:od_comp_mem_gains}). Accordingly, the upstream client-to-server transfer is decreased by $\nicefrac{1}{(\pim)^2}$ as only the gradient updates of the unpruned units are transmitted.


\subsection{Subnetwork knowledge transfer}
In the Section~\ref{sec:od-training}, we introduced knowledge distillation for our OD formulation. We extend this approach to \tool, where instead of the full network, we employ width \mbox{$\max\{p \in \cP: p \leq \pim \}$} as a teacher network in each local iteration on device $i$.

\begin{algorithm}[t]
    \begin{algorithmic}[1]	
	\STATE {\bfseries Input:}  $\F_p, d, \pim$
	\IF{$p \leq \pim$}{
    \STATE do nothing \;
  }
  \ELSIF{$p \leq \nicefrac{\pim}{(1-d)}$}
    \STATE use RD s.t. each pruned layer $l$ has exactly $\ceil{\pim \cdot K_l}$ non-masked neurons / filters   \;
  \ELSE
    \STATE $p$ is too big, it can't be evaluated \;
  \ENDIF
\end{algorithmic}
\caption{ Adaptive Dropout \textbf{\mbox{AD}}}
\label{alg:ad}
\end{algorithm}

\subsection{Resource-efficient large-model update}
It can be noted that only clients with large $\pim$ can update the weights corresponding to the neurons/channels with higher indices. This might be an issue that degrades performance along larger widths. 
To remedy such a situation, we propose an approach that allows any device to update a submodel that is larger than its maximum $\pim$ without exceeding its resource budget. In this manner, more devices update a larger part of the global network, while still respecting the device cluster's resource availability. 
Our approach which we refer to as \emph{Adaptive Dropout} (AD), is inspired by Federated Dropout (FD)~\cite{caldas2018expanding}, where the main difference is that AD employs variable level of unstructured Random Dropout (RD). It is parametrised by $d \in [0,1)$ which controls the maximum allowed value of RD, and it is described in Algorithm~\ref{alg:ad}. This technique shifts $\pim \rightarrow \pim{'} = \min\left\{\nicefrac{\pim}{(1-d)}, 1\right\}$ that can be used during training. For \tool, we first sample $p \sim \Dp|\Dp \leq \pim{'}$ and then apply AD.

\section{Evaluation of \tool}
\label{sec:evaluation}

In this section, we provide a thorough evaluation of \tool and its components across different tasks, datasets, models and device cluster distributions to showcase its performance, elasticity and generality.

\subsection{Datasets and models} 
We evaluate \tool on 
two vision and one text prediction task, shown in  Table~\ref{tab:datasets}. For {CIFAR10}~\cite{cifar}, we use the ``CIFAR'' version of ResNet18~\cite{resnet}. We federate the dataset by randomly dividing it into equally-sized partitions, each allocated to a specific client, and thus remaining IID in nature.
For {FEMNIST}, we use a CNN with two convolutional layers followed by a
softmax layer. For {Shakespeare}, we employ a RNN with an embedding layer (without dropout) followed by two LSTM~\cite{hochreiter1997long} layers and a
softmax layer. 
We report the model's performance of the last epoch on the test set which is constructed by combining the test data for each client.
We report top-$1$ accuracy vision tasks and negative perplexity for text prediction.
Further details, such as hyperparameters, description of datasets and models are available in the Appendix.

\subsection{Experimental setup}
\label{sec:exp_setup}

\begin{table}[t]
\centering
    \caption{Datasets}
    \renewcommand{\arraystretch}{1.}
        \begin{tabular}{llrrll}
            \toprule
            Dataset     & Model    & \# Clients & \# Samples & Task \\ \midrule
            CIFAR10     & ResNet18 & $100$         & $50,000$      & Image classification \\
            CIFAR100     & ResNet18 & $500$         & $50,000$      & Image classification \\
            FEMNIST     & CNN  & $3,400$             & $671,585$       & Image classification      \\
            Shakespeare & RNN      & $715$          & $38,001$      & Next character prediction \\
            \bottomrule
        \end{tabular}
    \label{tab:datasets}
\end{table}
\begin{table}[t]
\centering
    \caption{MACs and Parameters per $p$-reduced network}
    \label{tab:macs}
    \renewcommand{\arraystretch}{1.}
        \begin{tabular}{lrrrrr}
         & $p=  0.2$  & $0.4$  & $0.6$   & $0.8$   & $1.0$   \\ \bottomrule
        \multicolumn{6}{c}{CIFAR10 / ResNet18}                           \\
        MACs   & 23M & 91M & 203M& 360M & 555M \\
        Params & 456K  & 2M & 4M   & 7M   & 11M  \\ \bottomrule
        \multicolumn{6}{c}{FEMNIST / CNN}                                \\ 
        MACs   & 47K  & 120K   & 218K    & 342K    & 491K    \\
        Params & 5K     & 10K     & 15K     & 20K     & 26K     \\ \bottomrule
        \multicolumn{6}{c}{Shakespeare / RNN}                            \\ 
        MACs   & 12K    & 40K    & 83K     & 143K    & 216K    \\
        Params & 12K    & 40K    & 82K     & 142K    & 214K    \\ \bottomrule
        \end{tabular}
\end{table}

\textbf{Infrastructure.}
\tool was implemented on top of the \ttt{Flower} (v0.14dev)~\cite{beutel2020flower} framework and PyTorch (v1.4.0) \cite{pytorch}.
We run all our experiments on a private cloud cluster, consisting of Nvidia V100 GPUs. To scale to hundreds of clients on a single machine, we optimized \ttt{Flower} so that clients only allocate GPU resources when actively participating in a federated client round. We report average performance and the standard deviation across three runs for all experiments. 
To model client availability, we run up to $100$ \ttt{Flower} clients in parallel and sample 10\% at each global round, with the ability for clients to switch identity at the beginning of each round to overprovision for larger federated datasets.
Furthermore, we model client heterogeneity by assigning each client to one of the device clusters. We provide the following setups:
\begin{itemize}[leftmargin=*,label={},noitemsep,topsep=0pt]
    \item \textbf{Uniform-\{5,10\}:} This refers to the distribution $\Dp$, \textit{i.e.}~$p \sim \cU_{k}$, with $k=5$ or $10$.
    \item \textbf{Drop Scale $\in \{0.5,1.0\}$:} This parameter affects a possible skew in the number of devices per cluster.
    It refers to the drop in clients per cluster of devices, as we go to higher $p$'s. 
    Formally, for \textit{uniform-n} and drop scale~$ds$, the high-end cluster $n$ contains $1$$-$$\sum_{i=0}^{n-1}\nicefrac{ds}{n}$ of the devices and the rest of the clusters contain $\nicefrac{ds}{n}$ each.
    Hence, for $ds$$=$$1.0$ of the \mbox{\textit{uniform-5}} case, all devices can run the $p=0.2$ subnetwork, 80\% can run the $p=0.4$ and so on, leading to a device distribution of $(0.2, ..., 0.2)$. This percentage drop is half for the case of $ds$$=$$0.5$, 
    resulting in a larger high-end cluster, \textit{e.g.}~$(0.1, 0.1, ..., 0.6)$.
\end{itemize}

\textbf{Baselines.}
To assess the performance of our work against the state-of-the-art, we compare \tool with the following set of baselines: i) Extended Federated Dropout (eFD), ii) \tool with eFD (\tool w/ eFD). 
%
%

eFD builds on top of the technique of Federated Dropout (FD)~\cite{caldas2018expanding}, which adopts a Random Dropout (RD) at neuron/filter level for minimising the model's footprint. However, FD does not support adaptability to heterogeneous client capabilities out of the box, as it inherits a \textit{single} dropout rate across devices.
For this reason, we propose an extension to FD, allowing to adapt the dropout rate to the device capabilities, defined by the respective cluster membership. 
It is clear that eFD dominates FD in performance 
and provides a tougher baseline, as the latter needs to impose the same dropout rate to fit the model at hand on all devices, leading to larger dropout rates (\textit{i.e.}~uniform dropout of 80\% for full model to support the low-end devices).
We provide empirical evidence for this in the Appendix.
For investigative purposes, we also applied eFD on top of \tool, as a means to update a larger part of the model from lower-tier devices, \textit{i.e.}~allow them to evaluate submodels beyond their $\pim$ during training.

\subsection{Performance evaluation}
\label{sec:perf_eval}

\begin{figure}[t]
    \centering
    \begin{tabular}{ccc}
        \subfigure[ResNet18 - CIFAR10]{
            \includegraphics[width=0.3\textwidth]{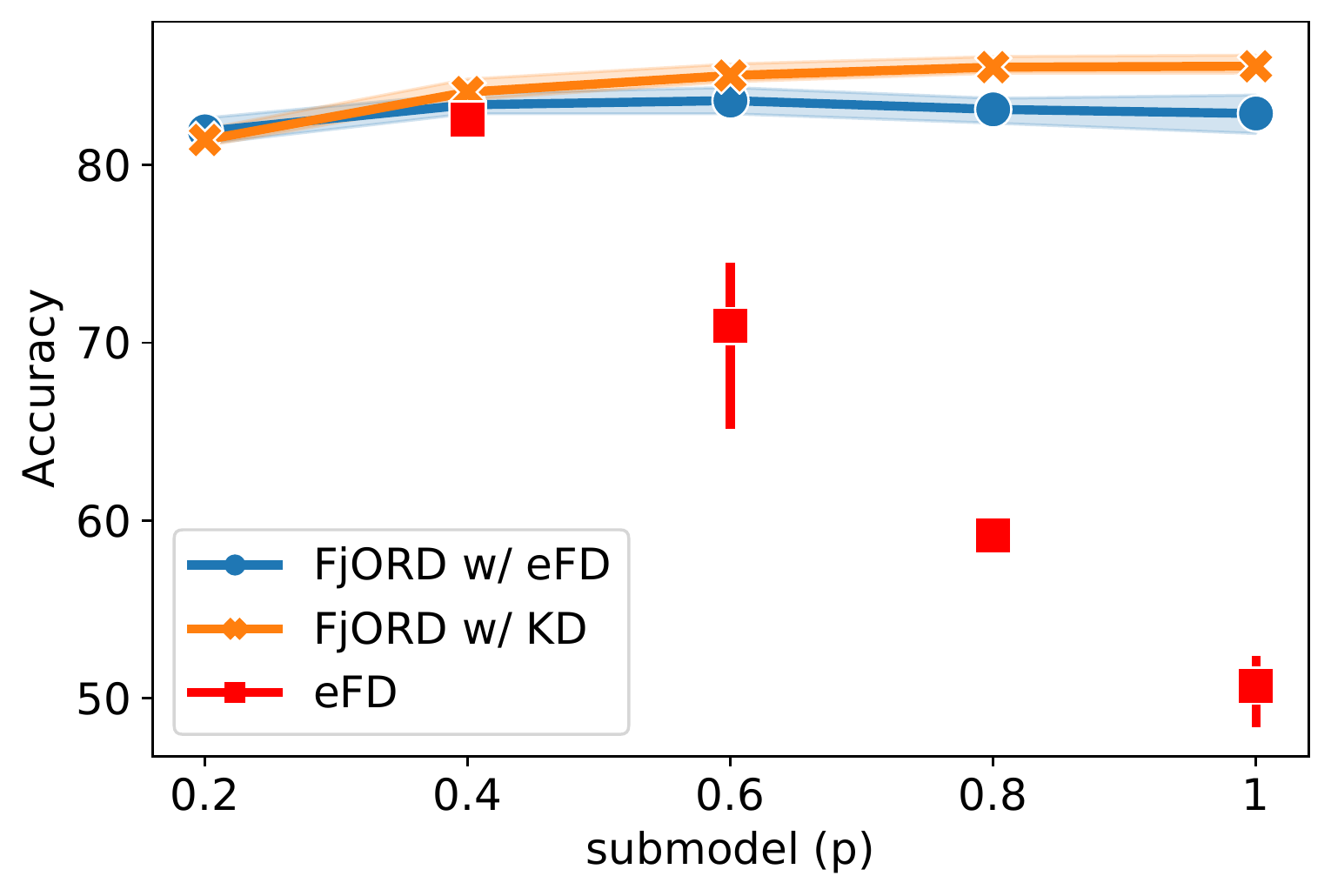}
            }
        \hfill
        \subfigure[CNN - FEMNIST]{
            \includegraphics[width=0.3\textwidth]{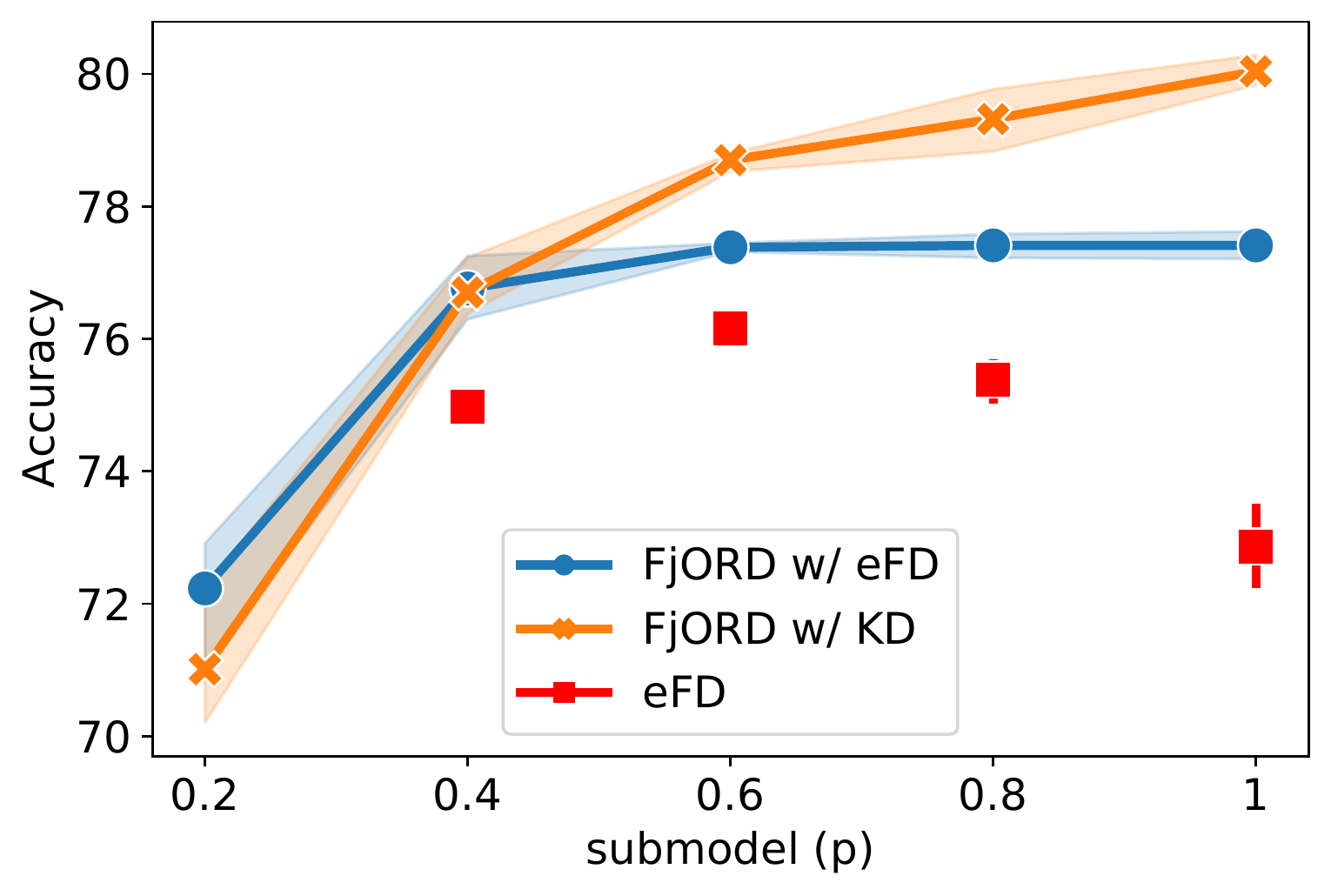}
            }
        \hfill
        \subfigure[RNN - Shakespeare]{
            \includegraphics[width=0.3\textwidth]{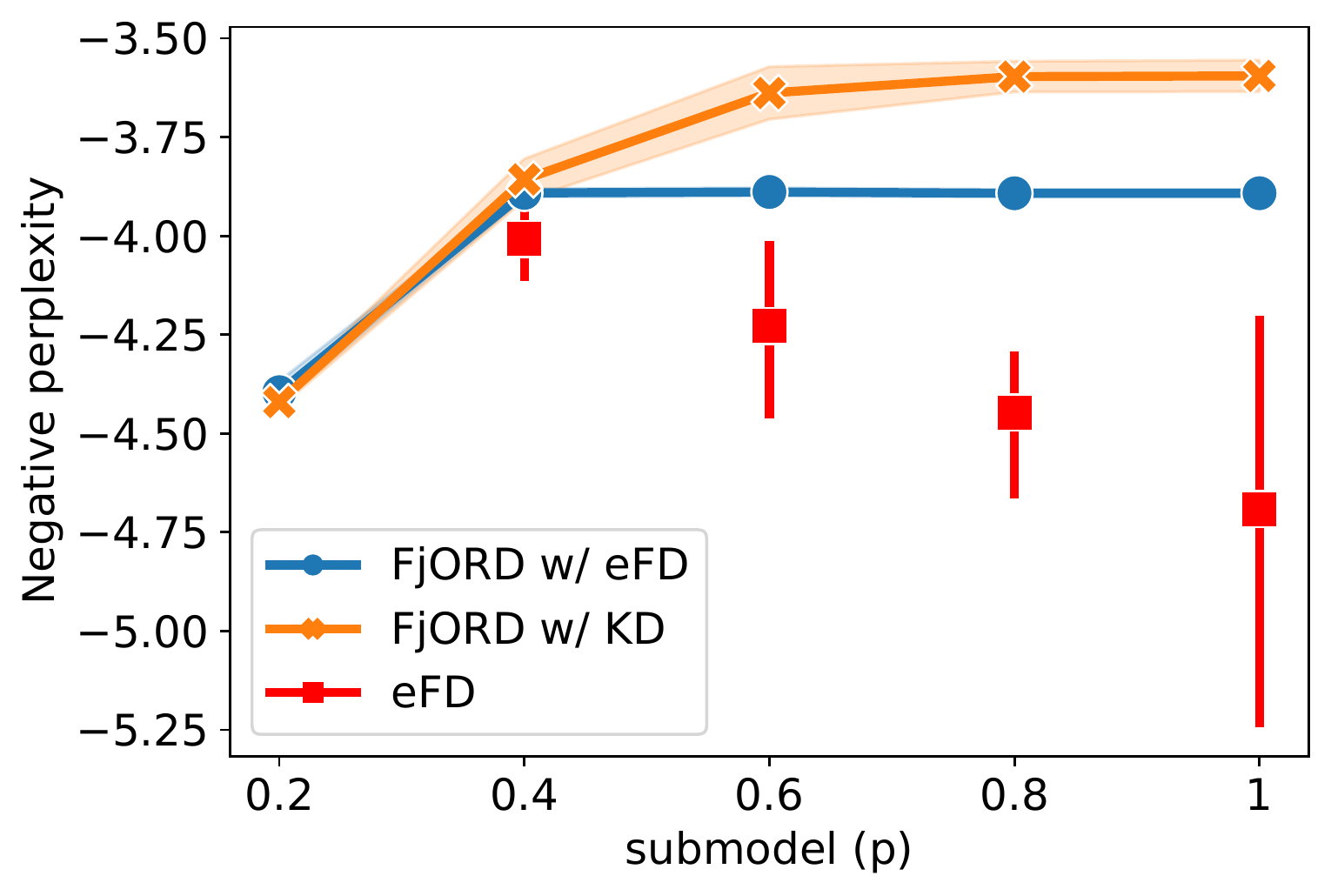}
            }
    \end{tabular}
    \caption{Ordered Dropout with KD vs eFD baselines. Performance vs dropout rate \textit{p} across different networks and datasets. $\Dp = \cU_5$}
    \label{fig:fl-perf-comparison}
\end{figure}

In order to evaluate the performance of \tool, we compare it to the two baselines, eFD and OD+eFD. We consider the \textit{uniform-5} setup with drop scale of 1.0 (\textit{i.e.}~uniform clusters).
For each baseline, we train one independent model $\F_p$, end-to-end, for each $p$.
For eFD, what this translates to is that the clusters of devices that cannot run model $\F_p$ compensate by randomly dropping out neurons/filters. We point out that $p=0.2$ is omitted from the eFD results as it is essentially not employing any  dropout whatsoever.
For the case of \tool + eFD, we control the RD by capping it to $d=0.25$. This allows for larger submodels to be updated more often -- as device belonging to cluster $c$ can now have $\pcm \rightarrow p_{\M}^{c+1}$ during training where $c$$+$$1$ is the next more powerful cluster --
while at the same time it prevents the destructive effect of too high dropout values shown in the eFD baseline.

Figure~\ref{fig:fl-perf-comparison} presents the achieved accuracy for varying values of $p$ across the three target datasets. 
\tool (denoted by \tool w/ KD) outperforms eFD across all datasets with improvements between $1.53$-$34.87$ percentage points (pp) ($19.218$ pp avg. across $p$ values) on CIFAR10, $1.72$-$7.17$ pp ($3.84$ p avg.) on FEMNIST and $0.15$-$1.09$ points (p) ($0.67$ p avg.) on Shakespeare.
Compared to \tool+eFD, \tool achieves performance gains of $0.71$-$2.66$ pp ($1.79$ avg.), up to $2.62$ pp ($1.44$ pp avg.) on FEMNIST and $0.03$-$0.29$ p (0.22 p avg.) on Shakespeare.
\begin{figure}[t]
    \centering
    \begin{tabular}{ccc}
        \subfigure[ResNet18 - CIFAR10]{\includegraphics[width=0.32\textwidth]{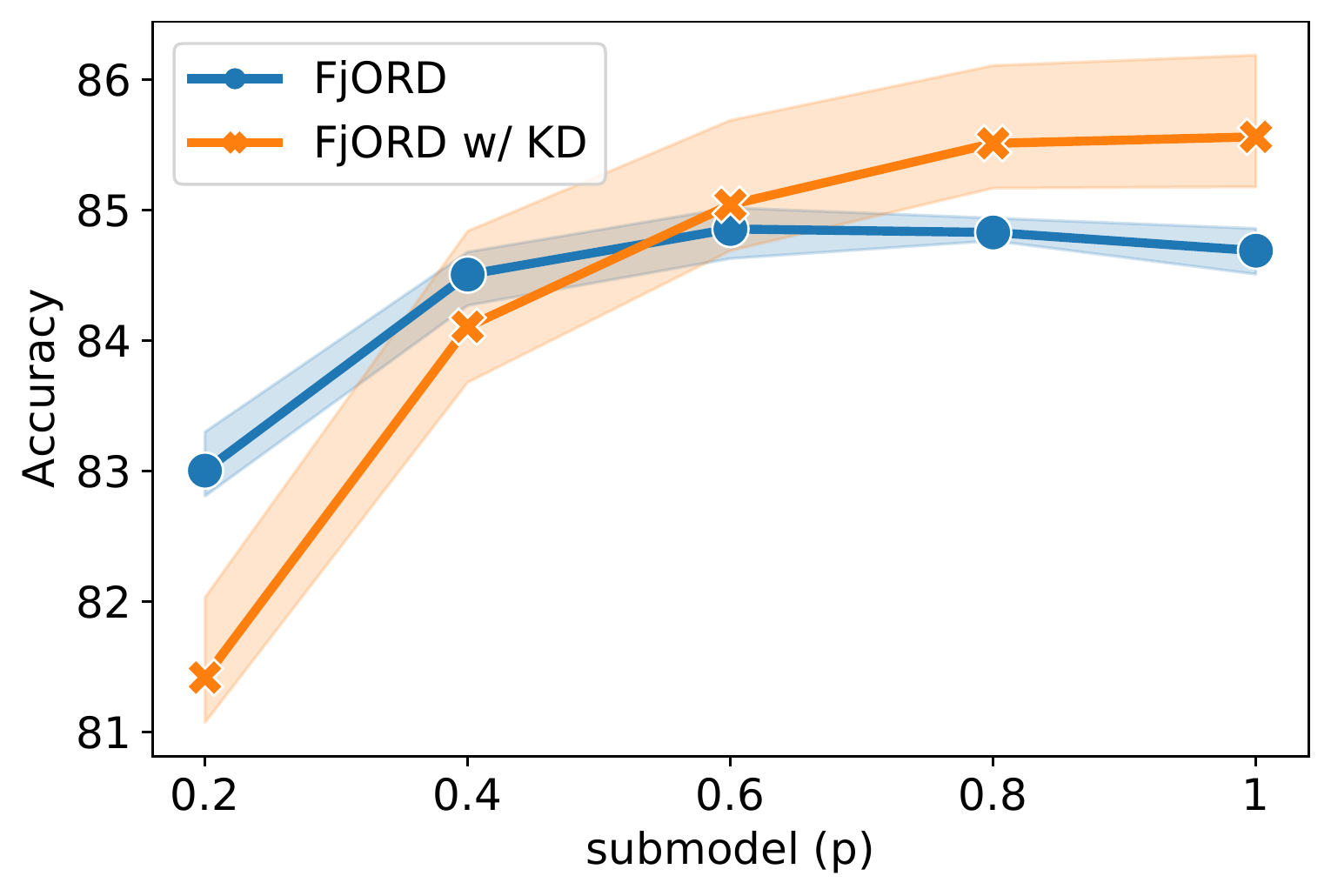} \label{fig:cifar10_scalability}}
        \hfill
        \subfigure[CNN - FEMNIST]
        {\includegraphics[width=0.32\textwidth]{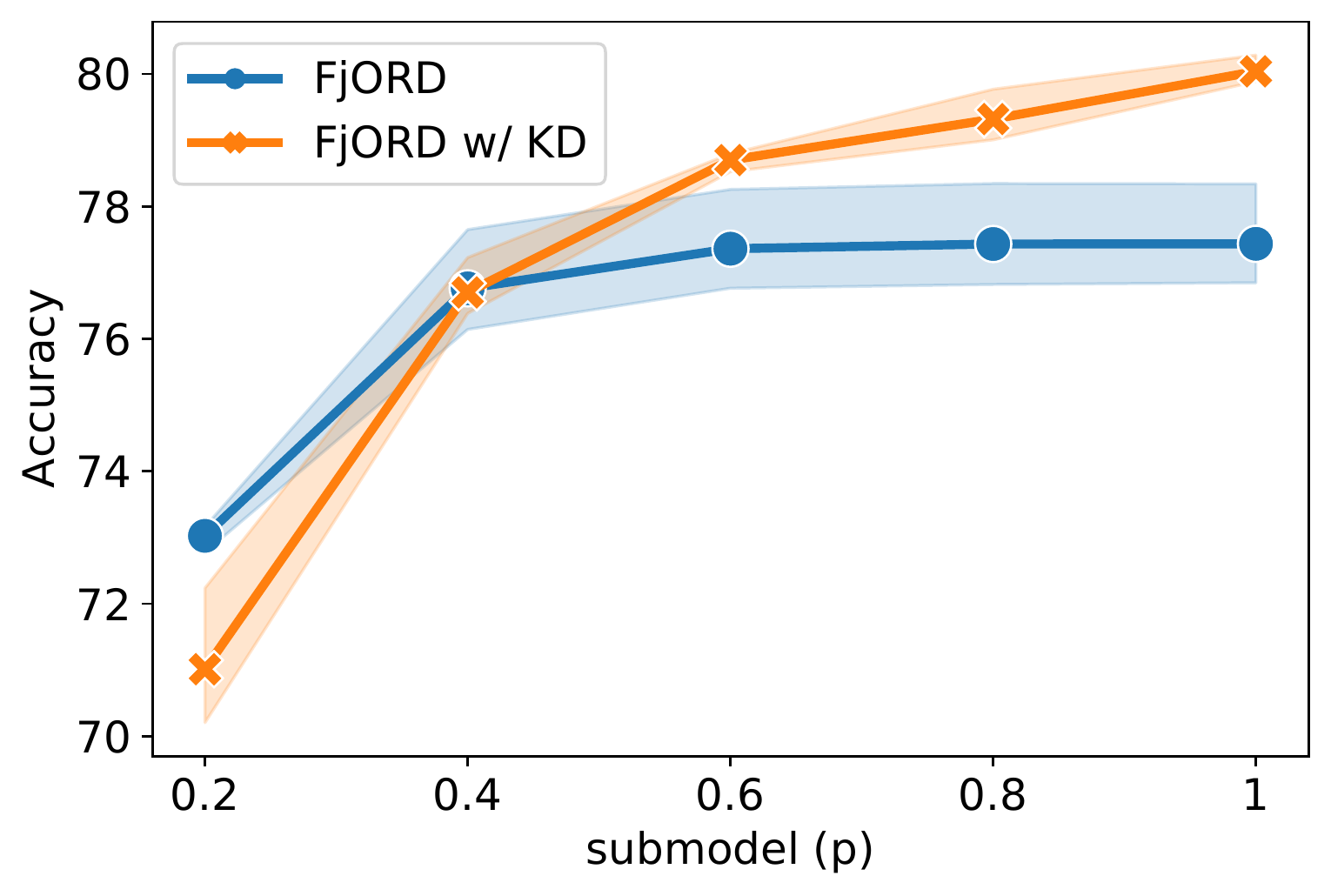}
        \label{fig:femnist_scalability}}
        \hfill
        \subfigure[RNN - Shakespeare]{
        \includegraphics[width=0.32\textwidth]{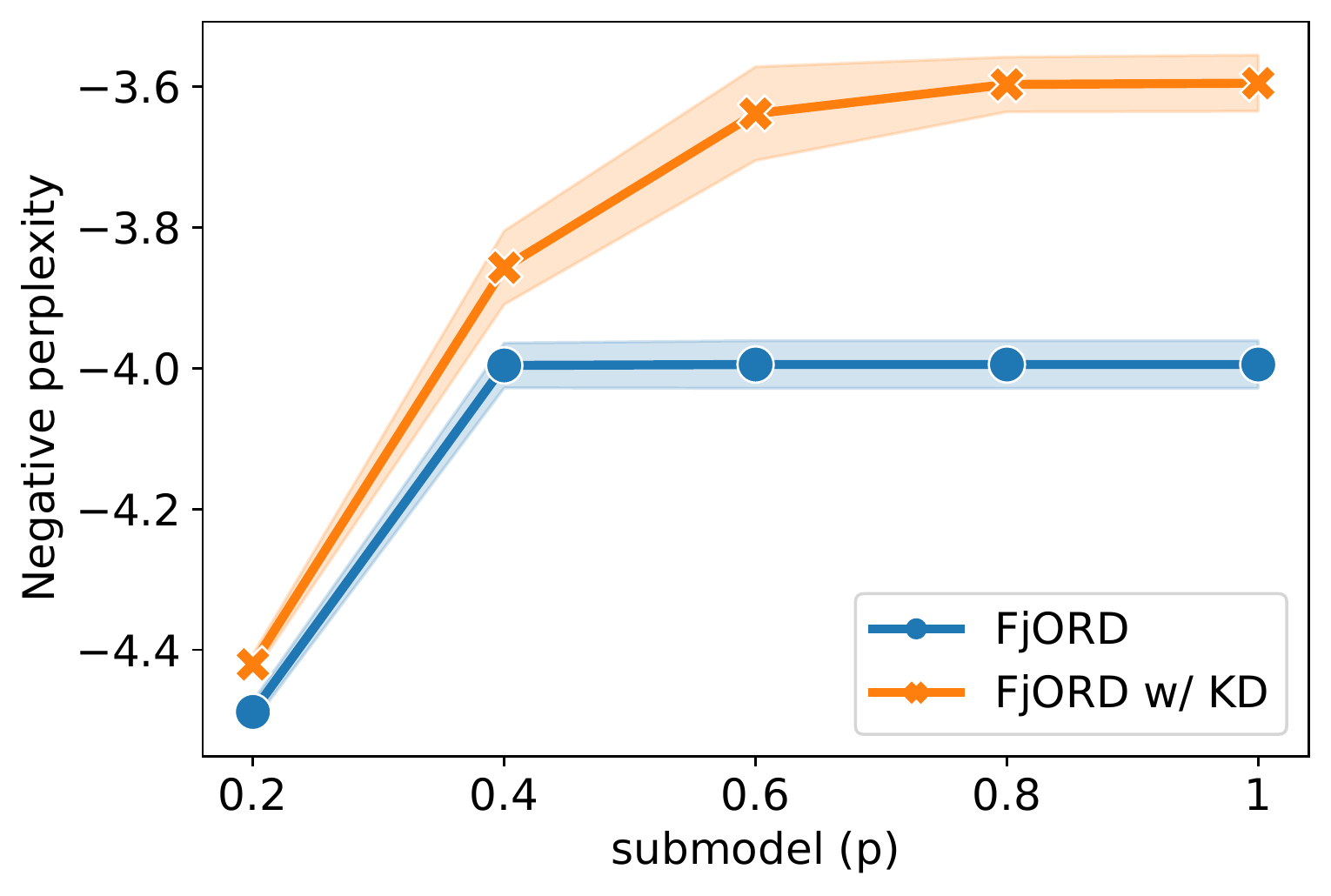}}
    \end{tabular}
    \caption{Ablation analysis of \tool with Knowledge Distillation. Ordered Dropout with $\Dp = \cU_5$, KD - Knowledge distillation.}
    \label{fig:fl-ablation}
\end{figure}

Across all tasks, we observe that \tool is able to improve its performance with increasing $p$ due to the nested structure of its OD method. We also conclude that eFD on top of \tool does not seem to lead to better results.
More importantly though, given the heterogeneous pool of devices, to obtain the highest performing model for eFD, multiple models have to be trained (\textit{i.e.}~one per device cluster). For instance, the highest performing models for eFD are $\F_{0.4}$, $\F_{0.6}$ and $\F_{0.4}$ for CIFAR10, FEMNIST and Shakespeare respectively, which can be obtained only \textit{a posteriori}; after all model variants have been trained. Instead, despite the device heterogeneity, \tool requires a single training process that leads to a global model that significantly outperforms the best model of eFD (by $2.98$ and $3.87$ pp for CIFAR10 and FEMNIST respectively and $0.41$ p for Shakespeare), while allowing the direct, seamless extraction of submodels due to the nested structure of OD.

\subsection{Ablation study of knowledge distillation in \tool}
\label{sec:ablation}
To evaluate the contribution of our knowledge distillation method to the attainable performance of \tool, we conduct an ablative analysis on all three datasets. We adopt the same setup of \textit{uniform-5} and $\text{drop scale}=1.0$ as in the previous section and compare \tool with and without KD.


Figure~\ref{fig:fl-ablation} shows the efficacy of \tool's KD in FL settings. \tool's KD consistently improves the performance across all three datasets when $p > 0.4$, with average gains of $0.66, 0.68$ and $0.87$ pp for submodels of size $0.6, 0.8$ and $1$ on CIFAR-10, $1.33, 1.88$ and $2.60$ pp for FEMNIST and $0.35, 0.39$ and $0.40$ pp for Shakespeare. 
For the cases of $p \le 0.4$, the impact of KD is fading, especially in the two vision tasks. We believe this to be a side-effect of optimizing for the average accuracy across submodels, which also yielded the $T=\alpha=1$ strategy. We leave the exploration of alternative weighted KD strategies as future work. 
Overall, the use of KD significantly improves the performance of the global model, yielding gains of $0.87$ and $2.60$ pp for CIFAR10 and FEMNIST respectively and $0.40$ p for Shakespeare.



\subsection{\tool's deployment flexibility}
\label{sec:flexibility}

\paragraph{Device clusters scalability}
\label{sec:scalability}


An important characteristic of \tool is its ability to scale to a larger number of device clusters or, equivalently, perform well with higher granularity of $p$ values.
%
To illustrate this, we fall back to the local setup for simplicity and timeliness of experiment results and test the performance of OD across two setups, \textit{uniform-5} and \textit{-10} (defined in Section~\ref{sec:exp_setup}).

As shown in Figure~\ref{fig:non-fl-scalability}, \tool sustains its performance even under the higher granularity of $p$ values. This means that for applications where the modelling of clients needs to be more fine-grained, \tool can still be of great value, without any significant degradation in achieved accuracy per submodel. This further supports the use-case where device-load needs to be modelled explicitly in device clusters (\textit{e.g.}~modelling device capabilities and load with deciles).



\begin{figure}[t]
    \centering
    \begin{tabular}{ccc}
        \subfigure[CNN - EMNIST]{
            \includegraphics[width=0.33\textwidth]{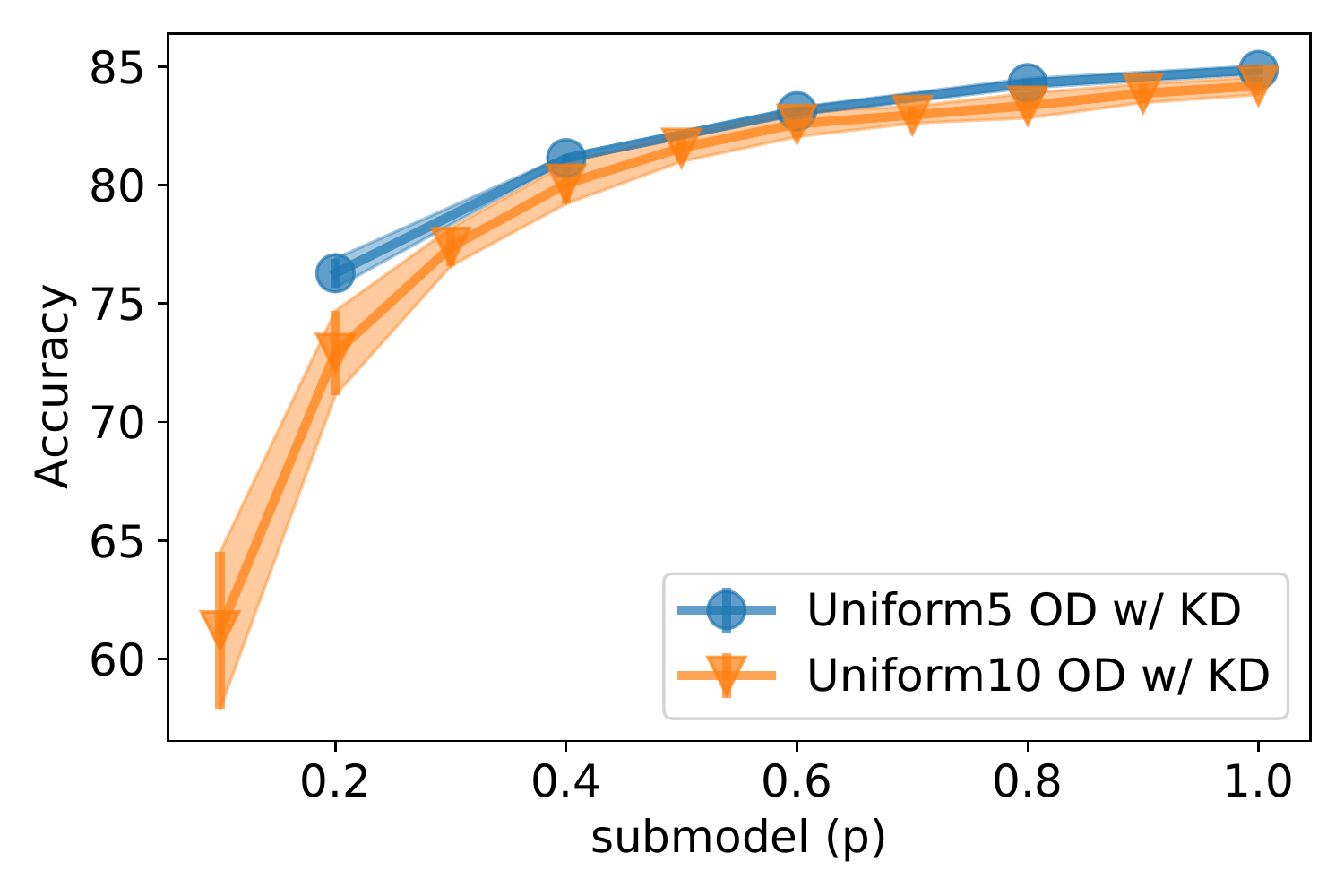}
        }
        \subfigure[RNN - Shakespeare]{
            \includegraphics[width=0.33\textwidth]{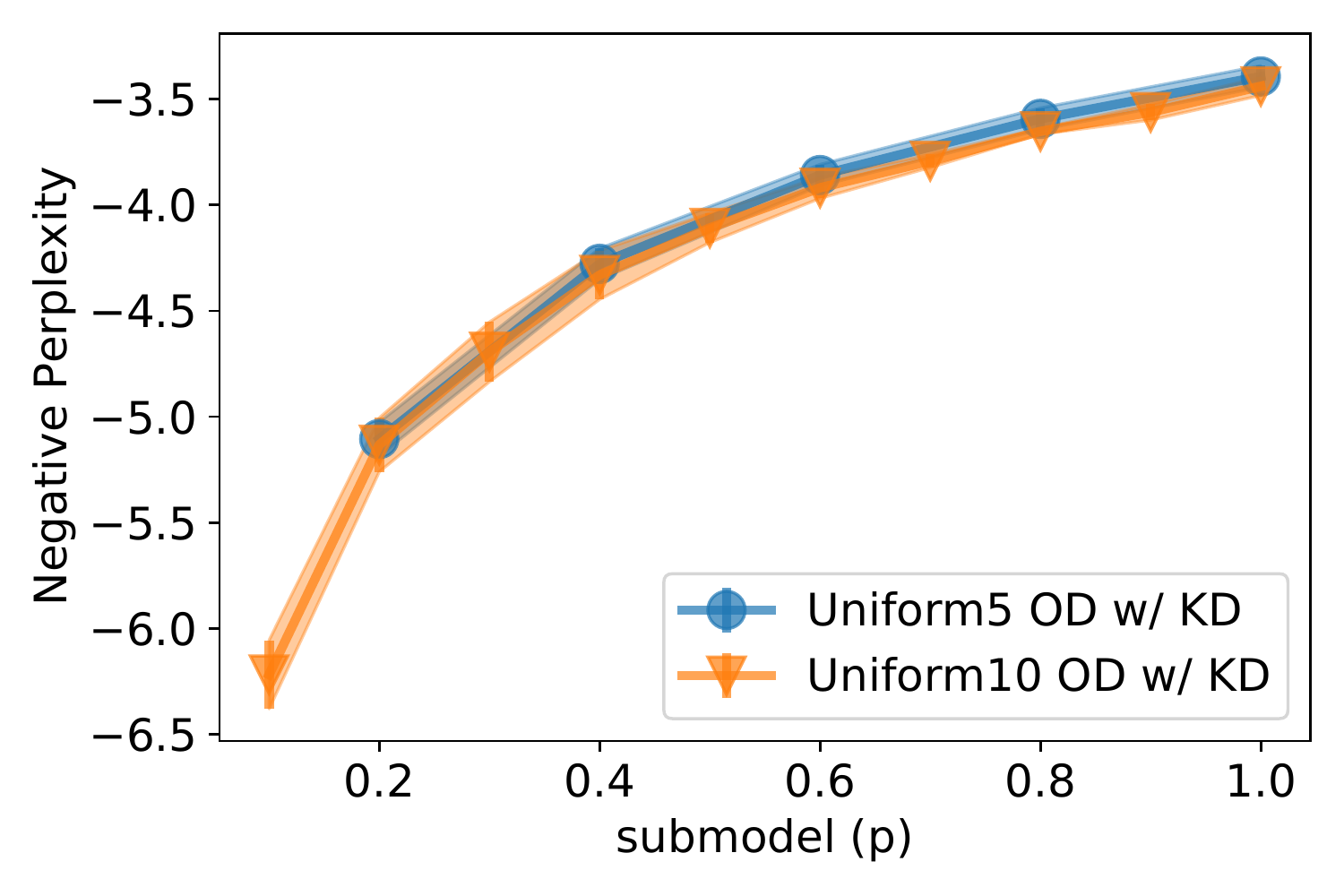}
        }
    \end{tabular}
    \caption{Demonstration of \tool's scalability with respect to the number of device clusters on non-federated datasets. 
    }
    \label{fig:non-fl-scalability}
\end{figure}

\begin{figure}[t]
    \centering
    \begin{tabular}{ccc}
        \subfigure[CNN - FEMNIST]{
            \includegraphics[width=0.33\textwidth]{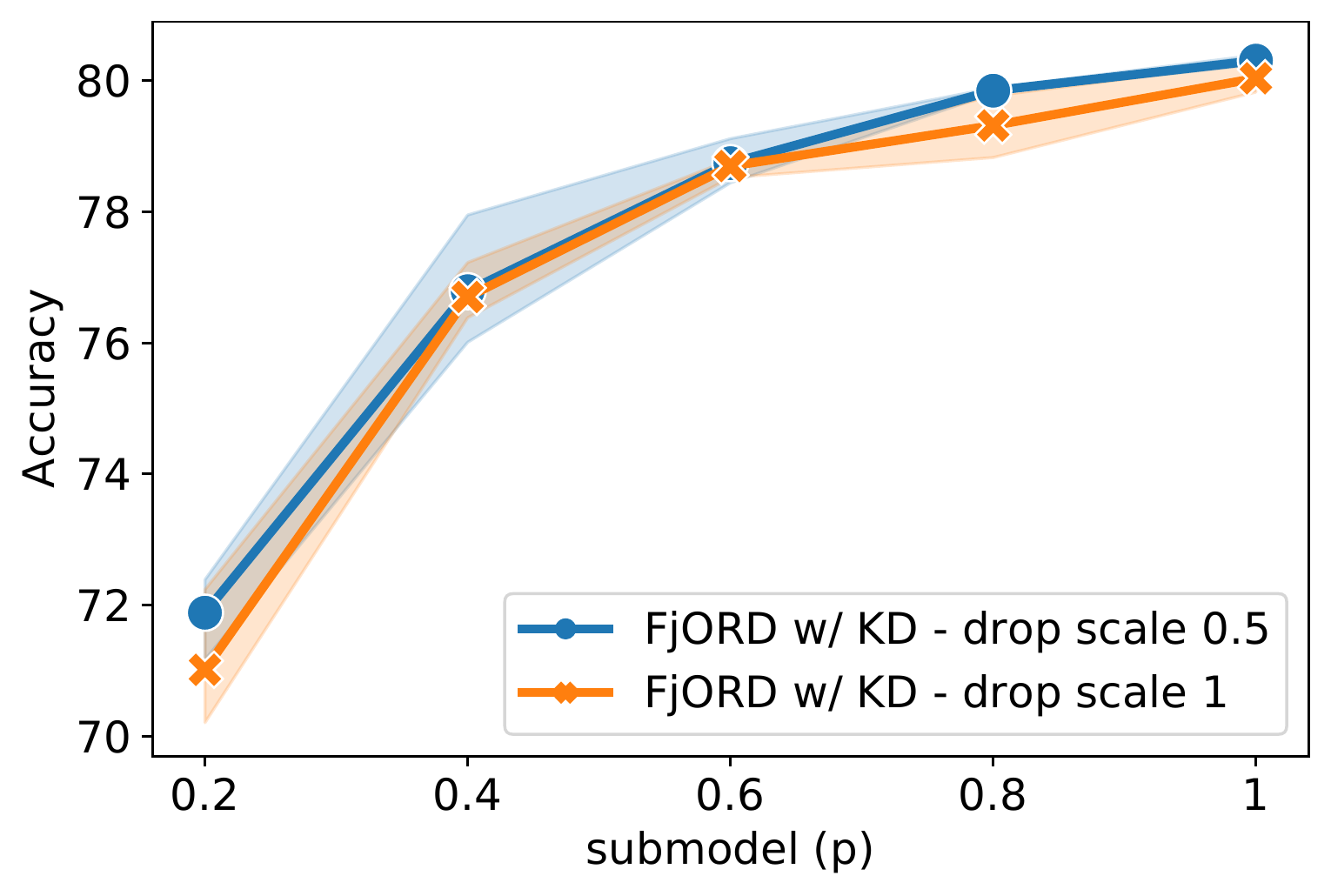}
        }
        \subfigure[RNN - Shakespeare]{
           \includegraphics[width=0.33\textwidth]{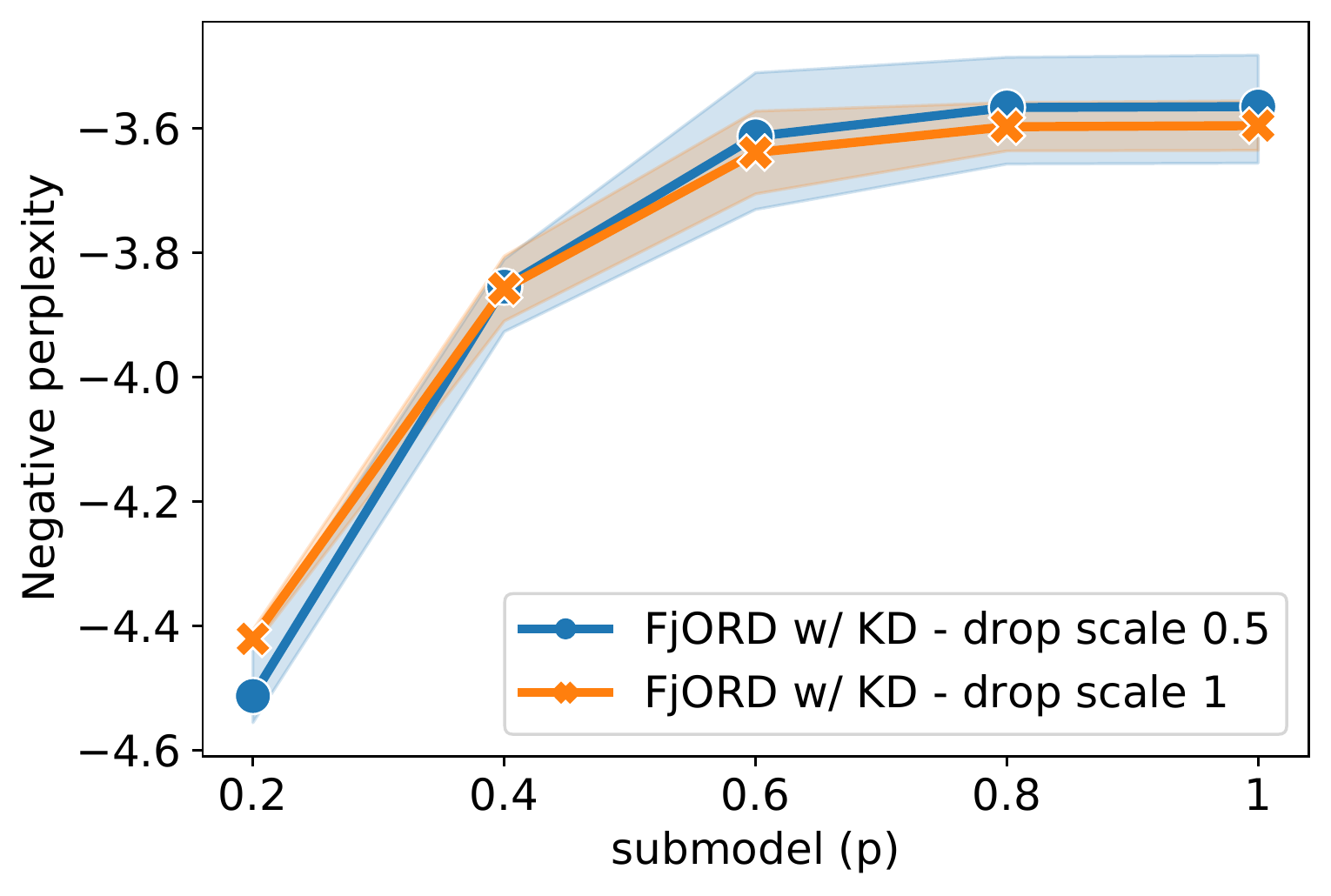}
        }
    \end{tabular}
    \caption{Demonstration of the adaptability of \tool across different device distributions.}
    \label{fig:fl-adaptability}
\end{figure}

\paragraph{Adaptability to device distributions}
\label{sec:adaptability}

In this section, we make a similar case about \tool's elasticity with respect to the allocation of available devices to each cluster. We adopt the setup of \textit{uniform-5} once again, but compare across drop scales $0.5$ and $1.0$. In both cases, clients that can support models of {\small $\pim \in \{0.2, \dots, 0.8\}$} are equisized, but the former halves the percentage of devices and allocates it to the last cluster. Hence, when $\text{drop scale}=0.5$, the high-end cluster accounts for 60\% of the devices. The rationale behind this is that the majority of participating devices are able to run the whole original model.

The results depicted in Figure~\ref{fig:fl-adaptability} show that in both cases, \tool is able to perform up to par, without significant degradation in accuracy due to the presence of more lower-tier devices in the FL setting. We should note that we did not alter the uniform sampling in this case on the premise that high-end devices are seen more often, exactly to illustrate \tool's adaptability to latent user device distribution changes of which the server is not aware.


\chapter{Recipes for better use of local work in federated learning}
\label{chapter8:fedshuffle}

\section{Introduction}

\emph{Federated learning} (FL) aims to train models in a decentralized manner, preserving the privacy of client data by leveraging edge-device computational capabilities. Clients' data never leaves their devices; instead, the clients coordinate with a server to train a global model. 
Due to such advantages and promises, 
FL is now deployed in a variety of applications~\cite{hard18gboard, apple19wwdc}.

In this chapter, we consider the standard FL problem formulation of solving an empirical risk minimization problem over the data available from all devices; i.e., 
\begin{align}
    \label{eq:objective}
    \min_{x \in \R^d} \sbr*{f(x) \eqdef \sum_{i = 1}^n w_i \underbrace{\frac{1}{|\cD_i|}\sum_{j=1}^{|\cD_i|} f_{ij}(x)}_{f_i(x)}}, \;
\end{align}
where $f_{ij}$ corresponds to the loss of the model with parameters $x$ evaluated on the $j$-th data point of the $i$-th client. The weight $w_i$ assigned to device $i$'s empirical risk is $w_i = \nicefrac{|\cD_i|}{|\cD|}$ where $|\cD_i|$ is size of the training dataset at device $i$ and $|\cD| = \sum_{i=1}^n |\cD_i|$. This choice of weights $w_i$ places equal weight on all training samples.\footnote{Although we focus on the setting with weights $w_i = \nicefrac{|\cD_i|}{|\cD|}$ so that the overall objective is equivalent to a standard, centralized empirical risk minimization problem using the data from all devices, this is not essential to our analysis, which could accommodate any choice of $w_i$. Of course, using a different choice of $w_i$ will change the solution.}

To solve~\eqref{eq:objective}, optimization methods must contend with several challenges that are unique to FL: 
\textit{heterogeneity} with respect to data and compute capabilities of each device, \textit{data imbalance} across devices, and \textit{limited device availability}. 
In addition, considering a large number of devices ($n$ for practical cross-device FL tasks can be in order of millions), privacy requirements, and necessity of device sampling, the \emph{stateless} setting is of particular interest, where each device might only participate once or only few times in the entire training process.

The most widely studied and used methods in this challenging setting have devices perform multiple steps locally, on their data, before communicating updates to the server (\emph{a la} local \sgd / \fedavg)~\cite{kairouz2019advances}. Most practical implementations, going back to the original description of \fedavg~\cite{mcmahan17fedavg}, have devices perform one or more local epochs over their training data, and the number of training samples per device may vary by many orders of magnitude~\cite{kairouz2019advances}. In contrast, most analyses assume that all participating devices perform the same number of local steps per round, and that gradients obey standard stochastic approximation assumptions (unbiased, with some notion of bounded variance).

Previous work of \cite{fednova2020neurips} identified that performing different numbers of local steps per device per round in \fedavg leads to the \emph{objective inconsistency} problem --- effectively minimizing a different objective than~\eqref{eq:objective}, where terms are reweighted by the number of samples per device. \cite{fednova2020neurips} propose the \fednova method to address this by rescaling updates during aggregation. \fedlin~\cite{mitra2021fedlin} addresses objective inconsistency by scaling local learning rates, while assuming full participation (all devices participate in every round) which is not practical for cross-device FL. Neither \fednova nor \fedlin incorporate reshuffling of device data. \cite{mishchenko2021proximal} and \cite{yun2021minibatch} analyze random reshuffling in the context of FL, while assuming that all devices have the same number of training samples. Improved convergence rates for FL optimization have been achieved in~\cite{karimireddy2020mime} by combining a Hessian similarity assumption with momentum variance reduction~\cite{cutkosky2019momentum}, but without accounting for random reshuffling or data imbalance.

In this work we aim to provide a unified view of local update methods while accounting for random reshuffling, data imbalance (through non-identical numbers of local steps), and client subsampling, while also obtaining faster convergence by incorporating momentum variance reduction. Table~\ref{tab:related-work} summarizes key differences mentioned above.

\begin{table}[]
\centering
\renewcommand{\arraystretch}{1.}
\caption{Comparison of characteristics considered in previous work and the methods analyzed in this work. Notation: 
\textit{HD} = Heterogenous Data,
\textit{CS} = Client Sampling,
\textit{RR} = Random Reshuffling,
\textit{NL} = Non-identical Local Steps, 
\textit{VR} = Variance Reduction,
\textit{GM} = Global Momentum,
\textit{SS} = Server Step Size. The bottom four methods are proposed and/or analyzed in this work.} \label{tab:related-work}
\begin{tabular}{llllllll}
                 & HD
                 & CS
                 & RR 
                 & NL
                 & VR 
                 & GM 
                 & SS 
                 \\ \midrule
\fednova          & \cmark             & \cmark          & \xmark             & \cmark                    & \xmark             & \xmark            & \cmark           \\
\fedlin          & \cmark             & \xmark          & \xmark             & \cmark                    & \cmark             & \xmark             & \xmark           \\
\textsc{Mime} & \cmark             & \cmark          & \xmark             & \xmark                    & \cmark             & \cmark             & \xmark           \\
\fedrr /\locrr & \cmark             & \xmark          & \cmark             & \xmark                    & \xmark             & \xmark             & \xmark           \\ \midrule
\fedavgrr        & \cmark             & \cmark          & \cmark             & \xmark$^*$                   & \xmark             & \xmark      & \cmark           \\
\fednovarr        & \cmark             & \cmark          & \cmark             & \cmark                    & \xmark             & \xmark   & \cmark           \\
\fedshuffle       & \cmark             & \cmark          & \cmark             & \cmark                    & \xmark             & \xmark     & \cmark           \\
\fedshufflemvr    & \cmark             & \cmark          & \cmark             & \cmark                    & \cmark             & \cmark          & \cmark           \\ \bottomrule
\end{tabular} \\
{\footnotesize $^*$With NL, \fedavgrr optimizes the wrong objective.}
\end{table}


\section{Contributions}

We make the following contributions in this chapter:
   \subsection{New algorithm: \fedshuffle, an improved way to remove objective inconsistency} In Section~\ref{sec:fedshuffle} we introduce and analyze \fedshuffle to account for random reshuffling and address the objective inconsistency problem. \fedshuffle fixes the objective inconsistency problem by adjusting the local step size for each client and redesigning the aggregation step, enabling a larger theoretical step size than either \fedavg or \fednova, while also benefiting from lower variance from random reshuffling. 
    \subsection{New algorithm: \fedshufflemvr, beating non-local methods} In Section~\ref{sec:mvr} we extend the results of \cite{karimireddy2020mime} by accounting for random reshuffling and data imbalance. Under a Hessian similarity assumption, we show that incorporating momentum variance reduction (MVR) with \fedshuffle leads to better convergence rates than the lower bounds for methods that don't use local updates.
    \subsection{General framework: \texttt{FedShuffleGen}} The above results are obtained by first considering a general framework (see Algorithm~\ref{alg:FedGen}) which accounts for data heterogeneity, different numbers of local updates per device, arbitrary client sampling, different local learning rates, and heterogeneity in the aggregation weights. To the best of our knowledge, this work is the first to tackle the challenge of random reshuffling in FL at this level of generality. Similar to \cite{fednova2020neurips}, our analysis reveals how heterogeneity in the number of local updates can lead to objective inconsistency. Within this framework, we obtain the first analysis of \fedavg and \fednova with random reshuffling (\fedavgrr and \fednovarr respectively in Table~\ref{tab:related-work}).

Finally, our theoretical results and insights are corroborated by experiments, both in a controlled setting with synthetic data, and using commonly-used real datasets for benchmarking and comparison with other methods from the literature.

\section{Related work}
\label{sec:related_work_ferdshuffle}

\subsection{Federated optimization and local update methods} As we mentioned earlier local update methods are at the heart of FL. As a result, many prior works have tried to analyze the local update methods, e.g.~\cite{wang2021cooperative, local_SGD_stich_18,zhou2017convergence,yu2019parallel,li2019convergence,haddadpour2019convergence,haddadpour2019trading,haddadpour2019local,bayoumi2020tighter,stich2020error,wang2019adaptive,woodworth2020local,koloskova2020unified,khaled2019first,woodworth2018graph,xie2019local,lin2018don}. These analyses are done under the assumption that every client performs the same number of local updates in each round. However this is usually not the case in practice, due to the heterogeneity of the data and system in FL; note that practical FL algorithms run a fixed number of \emph{epochs} (not steps) per device. Moreover, forcing the fast and slow devices to run the same number of iterations would slow-down the training. This problem was also noted by~\cite{fednova2020neurips}, where it is shown that having a heterogeneous number of updates, which is inevitable in FL, leads to an inconsistency between the target loss~\eqref{eq:objective} and the loss that the methods optimize. Moreover, as shown in~\cite{fednova2020neurips}, other approaches such as {\tt FedProx}~\cite{li2020federated}, {\tt VRLSGD}~\cite{liang2019variance}
and {\tt SCAFFOLD}~\cite{karimireddy2020scaffold} that are designed for heterogeneous data can partially alleviate the problem but not completely eliminate it. 

In this work, we propose \fedshuffle, that combines update weighting and learning rate adjustments, to deal with this issue. Compared to the \fednova~\cite{fednova2020neurips} which only uses update weighting, our approach is more general and as we show both analytically and experimentally it outperforms \fednova and does not slow down the convergence of \fedavg.

\subsection{Communication compression} The need for communication efficiency in FL led not only to significant research interest in methods that perform more local work such as \fedavg, but it also led to an orthogonal stream of works focusing on reducing the size of the communicated updates that are exchanged between the nodes in the network~\cite{1bit, strom2015scalable, qsgd2017neurips, terngrad, gorbunov2021marina, philippenko2021preserved}. In the one of the most popular forms, these techniques limit the floating point precision, i.e., they reduce the number of bits that are used to represent the communicated values~\cite{gupta2015deep, Cnat}. In the most extreme case, the coordinates are replaced with their signs~\cite{bernstein2018signsgd, sign_descent_2019}. Other notable examples include low-rank approximation~\cite{wang2018atomo, vogels2019powersgd},  sparsification techniques~\cite{Suresh2017, alistarh2018sparse, xu2021deepreduce}, or limiting the number of clients to communicate back to the master~\cite{chen2020optimal}. 

While we do not directly consider communication compression in this work, our results can be extended in the straightforward way, at least for the unbiased class of compressor operators. This is because unbiased compression can be seen as an extra variance to the update during the aggregation, see \cite{gorbunov2020unified} for details, and this can be captured by our Assumption~\ref{ass:grad_sim}. Combination with popular unbiased compression operators with error feedback~\cite{karimireddy2019error, EF21, sahu2021rethinking} is also possible, but for the sake of brevity we omit this extension and leave it as a future work.

For a comprehensive overview of the literature on communication compression, we refer the reader to the related work sections in previous chapters (Chapters \ref{chapter2:c_nat}-\ref{chapter5:induced}). 

\subsection{Random reshuffling} A particularly successful technique to optimize the empirical risk minimization objective is the idea of random reshuffling of the training data followed by full epoch using this permutation~\cite{bottou2012stochastic} instead of randomly selecting a data point (or a subset of data points) in each step as for standard \sgd. This process is repeated several times and the resulting method is usually referred to as Random Reshuffling ({\tt RR}). {\tt RR} is often observed to speed up the convergence which can be intuitively attributed to the fact that while with-replacement sampling can miss to learn from some data points in any given epoch, {\tt RR} cannot by construction. Furthermore, {\tt RR} enables a significantly faster implementation. Properly understanding the random reshuffling trick, and why it works, has been an open challenging problem for a long time~\cite{bottou2009curiously, gurbuzbalaban2021random, ahn2020tight} until recent advances that lead to a significant simplification of the convergence analysis technique~\cite{mishchenko2020random}. The difficulty of analysing {\tt RR} comes from the biased gradient estimator, which is not the case for its cousin--with-replacement sampling. Apart from this, FL comes with an extra challenge---the heterogeneity in number of samples that leads to the heterogeneity in number of local updates. To the best of our knowledge, analyzing {\tt RR} in FL setting remains largely unexplored in the literature despite {\tt RR} being the most popular technique used in practice for FL, e.g., it is a default option for \textit{TensorFlow Federated}. We are only aware of handful amount of papers analyzing {\tt RR} for FL~\cite{mishchenko2021proximal, yun2021minibatch} and both of these works require all the clients to participate in each round and to have the same amount of data. In addition, \cite{mishchenko2021proximal} analyze only (strongly) convex setting and \cite{yun2021minibatch} require the Polyak-{\L}oyasiewicz condition to hold for the global function. In this work, building on shoulders of the recent advances~\cite{mishchenko2020random}, we address all the challenges that come from applying {\tt RR} for FL.

\subsection{Heterogeneous devices} FL involves user devices such as smartphones with a wide range of device capabilities. Common practice is to leave out a certain class of low-end devices during the FL process, as they may not be capable of performing the local training step and communicating the updates before timeout. This leads to device tier-based selection bias and, therefore, it might cause serious issues. Several approaches were proposed to accommodate heterogeneous clients including model distillation~\cite{lin2020ensemble, he2020group} or reducing the model size evaluated by clients~\cite{caldas2018expanding, diao2020heterofl, horvath2021fjord}. We do not consider these techniques, but we note that the algorithms presented in this work can be directly adjusted to be compatible with the aforementioned methods.

\section{Notation and assumptions}
We assume that client $i$ has access to an oracle that takes $(j, x)$ as an input and returns the gradient $\nabla f_{ij}(x)$ as an output. 
In order to provide convergence guarantees, we make the following standard assumptions and will discuss how they relate to other commonly used assumptions in the literature.

We provide convergence guarantees for three common classes of smooth objectives: strongly-convex, general convex, and non-convex.

\begin{assumption}
    \label{ass:strong-convexity}
    The functions $\cbr{f_{ij}}$ are $\mu$-(strongly) convex.
\end{assumption}

\begin{assumption}
\label{ass:smoothness}
    The functions $\{f_{ij}\}$ are $L$-smooth.
\end{assumption} 
      
Next, we state two standard assumptions which quantify heterogeneity. The first bounds the gradient dissimilarity among local functions $\cbr*{f_i}$ at different clients, and the second controls the variance of local gradients at each client.
  
\begin{assumption}[Gradient Similarity]
\label{ass:grad_sim}
The local gradients $\{\nabla f_i\}$ are $(G,B)$-bounded, i.e., for all $x \in \R^d$,
\begin{equation}
\label{eq:heterogeneity}
 \sum_{i=1}^n w_i \norm{\nabla f_i(x)}^2 \leq G^2 + B^2 \norm{\nabla f(x)}^2 \,.
\end{equation}
If $\{f_i\}$ are convex, then we can relax the assumption to
\begin{equation}
\label{eq:heterogeneity-convex}
  \sum_{i=1}^n w_i \norm{\nabla f_i(x)}^2 \leq G^2 + 2L  B^2(f(x) - f^\star) \,.
\end{equation}  
\end{assumption}

\begin{assumption}[Bounded Variance]
\label{ass:bounded_var}
    The local stochastic gradients $\{\nabla f_{ij}\}$ have $(\sigma_i, P_i)$-bounded variance, i.e., for all $x \in \R^d$,
    \begin{equation}\label{eqn:bounded_var}
         \frac{1}{|\cD_i|}\sum_{j=1}^{|\cD_i|} \norm{\nabla f_{ij}(x) - \nabla f_i(x)}^2 \leq \sigma_i^2 + P_i^2\norm{\nabla f_i(x)}^2  \,.
    \end{equation}
 \end{assumption}
 
 Note that we do not require gradient norms to be bounded by constant.
 Moreover, we do not require the global or local variance to be bounded by constants either, but we allow them to be proportional to the gradient norms. In stochastic optimization, these assumptions are referred to as relaxed growth condition~\cite{bottou2018optimization}. Furthermore, one can show that for smooth and convex $\cbr*{f_{ij}}$, these are not actually assumptions, but rather properties~\cite{stich2019unified}.
 
 Following \cite{karimireddy2020mime}, we also characterize the variance in the Hessian that is the essential assumption to show superiority of local methods.
 
 \begin{assumption}[Hessian Similarity]
\label{ass:hessian_sim}
    The local gradients $\{\nabla f_{i}\}$ have $\delta$-Hessian similarity, i.e., for all $x \in \R^d$ and $i \in [n]$, ($\norm{\cdot}$ represents spectral norm for matrices)
    \begin{equation}\label{eqn:hessian_sim}
         \norm{\nabla^2 f_i(x) - \nabla^2 f(x)}^2 \leq \delta^2  \,.
    \end{equation}
 \end{assumption}
 
 Note that if $\cbr*{f_i}$ are L-smooth then it has to hold that $\norm{\nabla^2 f_i(x)} \leq L $ for all for all $x \in \R^d$ and $i \in [n]$ and, therefore, Assumption~\ref{ass:hessian_sim} is satisfied with $\delta \leq 2L$. In realistic examples, one might expect the clients to be similar and hence $\delta \lll L$.
 
 For the client participation, we assume a fixed \textit{arbitrary participation framework} introduced in Section~\ref{sec:partial_participation}.
 

\begin{figure}[t]
	\centering
    \begin{algorithm}[H]
        \begin{algorithmic}[1]
        \STATE {\bfseries Input:}  initial global model $x^0$, global and local step sizes $\eta_g^r$, $\eta_l^r$, proper distribution $\cS$ 
        \FOR{each round $r=0,\dots,R-1$}
            \STATE master broadcasts $x$ to all clients $i\in \mathcal{S}^r \sim S$
            \FOR{each client $i\in \cS^r$ (in parallel)}
                \STATE initialize local model $y_i \leftarrow x$
                \FOR{$e=1,\dots, E$}
                    \STATE Sample permutation $\{\Pi_0, \hdots, \Pi_{|\cD_i|-1}\}$ of $\{1, \hdots, |\cD_i| \}$
                    \FOR{$j=1,\dots, |\cD_i|$}
                        \STATE update $y_i \leftarrow y_i - \nicefrac{\eta_l^r}{|\cD_i|} \nabla f_{i\Pi_{j-1}}(y_i)$
                    \ENDFOR
                \ENDFOR
                \STATE send $\Delta_i = y_i - x$ to master
            \ENDFOR
           \STATE  master computes $\Delta = \sum_{i\in \cS^r} \frac{w_i}{p_i}\Delta_i$
            \STATE master updates global model $x \leftarrow x - \eta_g^r \Delta$
        \ENDFOR
        \end{algorithmic}
        \caption{\fedshuffle}
        \label{alg:FedShuffle_simple}
    \end{algorithm}
\end{figure}

\section{The \fedshuffle algorithm}
\label{sec:fedshuffle}

In this section, we formally introduce our proposed \fedshuffle. Its pseudocode is provided in Algorithm~\ref{alg:FedShuffle_simple} (simple) and Algorithm~\ref{alg:FedShuffle} (precise). The main inspiration for \fedshuffle\ is the default optimization strategy used in Federated Learning: \fedavg. As described in~\cite{mcmahan17fedavg}, one starts each communication round by sampling $b$ clients to participate uniformly at random. These clients then receive the global model from the server and update it by training the model for $E$ epochs based on their local data. The model updates are communicated back to the server, which aggregates them and updates the global model to be communicated in the next communication round.  We provide the pseudocode for this procedure in Algorithm~\ref{alg:FedAvg} in the Appendix. \fedshuffle involves two modifications: we scale the local step size by $\frac{1}{|\cD_i|}$, and we also adjust the aggregation step. We also allow each client to run different number of epochs $\cbr{E_i}$, in that case, the local step size is scaled proportionally to $\frac{1}{E_i|\cD_i|}$, see Section~\ref{sec:fedgen}.  Both of these changes are theoretically motivated to make sure that the updates stay consistent, i.e., we optimize the true objective~\eqref{eq:objective}.

\subsection{Heterogeneity in the number of local updates}
\label{sec:heterogeneity}

We introduce the first adjustment, that is the step size scaling. Consider the following quadratic minimization problem 
\begin{align}
\label{eq:quad_loss}
    \min_{x \in \R^d} \frac{1}{|\cD|}\sum_{i=1}^N \norm*{x - e_i}^2,
\end{align}
where $\cbr{e_i}$ are given vectors. Clearly, this is a strongly convex objective with the unique minimizer $x^\star = \frac{1}{n}\sum_{i=1}^n e_i$. For simplicity, let us assume that we solve this objective using standard \fedavg\ with local shuffling and full client participation, i.e., $b=n$. Since each local function has only one element, this is equivalent to running Gradient Descent (\gd), and therefore for small enough step size this algorithm converges linearly to the optimal solution $x^\star$. Now, suppose instead that that only $\cbr{e_i}_{i=1}^n$ are unique and each client $i$ has $|\cD_i|$ copies of $e_i$ locally. Then, we can write the objective as
\begin{align}
\label{eq:quad_loss_local}
    \min_{x \in \R^d} \sum_{i=1}^n \frac{|\cD_i|}{|\cD|} f_i(x);\quad f_i(x) \eqdef \frac{1}{|\cD_i|}\sum_{j=1}^{|\cD_i|}\norm*{x - e_i}^2.
\end{align}

Applying \fedavg\ with local shuffling is equivalent to running \fedavg\ with $E |\cD_i|$ local steps since all the local data are the same.  We show in Section~\ref{sec:fedgen} that \fedavg\ with $E |\cD_i|$ local steps converges linearly to $\tilde x = \frac{1}{\sum_{i=1}^n |\cD_i|^2} \sum_{i=1}^n |\cD_i|^2 e_i$ for sufficiently small step size $\eta_l$. We note that one can choose $\cbr{|\cD_i|}$ and $\cbr{e_i}$ arbitrarily; thus, the difference between $\tilde x$ and $x^\star$ can be arbitrary large. 

To tackle this first issue that causes the objective inconsistency, we propose to scale the step size proportionally to $\nicefrac{1}{|\cD_i|}$, which removes the aforementioned inconsistency. 
In Section~\ref{sec:fedgen}, we extend these results. We introduce and analyze a general shuffling algorithm---\fedgen. \fedgen includes \fedavg, \fednova and our \fedshuffle as special cases. As a byproduct, we obtain a detailed theoretical comparison of \fednova and \fedshuffle. In a nutshell, we show that \fedshuffle balances the progress made by each client and keeps the aggregation weights unaffected while \fednova diminishes the weights for the client that makes the most progress. As a consequence, \fedshuffle allows larger theoretical local step sizes than both \fedavg and \fednova while preserving the worst-case convergence rate.
We refer reader to Section~\ref{sec:fedgen}, particularly Section~\ref{sec:comparison}, for the extended discussion and a detailed comparison of all three methods.
Lastly, one might fix the inconsistencies in the \fedavg by running the same number of local steps $K$ per each client. Note that universally choosing a fixed number of steps for all clients is not straightforward. We will compare heuristics based on fixed number of steps with our proposed approaches in the experiments. To be comparable to other baselines, we use two heuristics to select $K$ : (1) Set $K$ based on the client with minimum number of data points in the round (\fedmin), which ensures that such a round will not result in any additional stragglers compared to other baselines. As we will see, \fedmin does not result in great performance as it does not utilize all the data on most of the clients. (2) Set $K$ to be the average number of steps that the selected clients would have taken in that round if they were running other baselines (\fedmean). This makes sure that the total number of local steps for all clients is the same across all baselines. Note that \fedmin and \fedmean are not practical since they require additional coordination among the selected clients to determine the number of local steps to take; we consider them as a heuristic to show the difficulty of choosing a fixed number of local steps for all clients. As we will see, \fedmean under-performs \fedshuffle and even in some cases \fedavg, especially in terms of test accuracy in heterogeneous settings.

\subsection{Removing bias in aggregation}
The second algorithmic change compared to \fedavg has been, to the best of our knowledge, overlooked and it is related to the aggregation step. The original aggregation that is widely used in practice, see Algorithm~\ref{alg:FedAvg} for \fedavg practical implementation, contains the step (line 15) where the local weights from the client $i \in \cS$ are normalized to sum to one by $\frac{w_i}{\sum_{i\in \cS} w_i}$; we refer to this as the \textit{Sum One (SO)} aggregation. Such aggregation can lead to a biased contribution from workers and therefore to an inconsistent solution that optimizes a different objective as we show in the following example. Suppose that there are three clients and they hold, respectively, $1$, $2$ and $3$ data points. In each round, we sample two clients uniformly at random. Then, the normalized expected contribution from client $i$ is $\nicefrac{2}{3}\EE{i}{\frac{w_i}{\sum_{j\in \cS\text{; s.t. }i \in \cS}w_j}}$. It is easy to verify that this is equal to $\nicefrac{7}{36}$, $\nicefrac{16}{45}$ and $\nicefrac{9}{20}$, respectively. One can note that this is not proportional to the weights $\cbr{w_i}$ of the objective \eqref{eq:objective}. Furthermore, this proposed aggregation cannot be simply fixed by changing the client sampling scheme, e.g., by sampling clients with probability proportional to the amount of data they hold, since one can always find a simple counterexample. The problem of the aggregation scheme is sample dependent normalization $\sum_{i\in \cS} w_i$ that makes sampling biased in the presence of non-uniformity with respect to the number of data samples per client.

To solve this issue, we use $\nicefrac{w_i}{p_i}$ in the scaling step, where $p_i$ is the probability that client $i$ is selected. This a very standard aggregation scheme that results in unbiased aggregation with respect to the worker contribution since $\EE{i}{\nicefrac{w_i}{p_i}} = w_i$. It is easy to see that if $\cbr*{p_i}$ are proportional to $\cbr*{w_i}$ then the aggregation step would be simply taking a sum. This can be achieved by each client being sampled independently using a probability proportional to its weight $w_i$, i.e., its dataset size if the master has access to this information.\footnote{It may not be possible for the server to know the number of samples per client because of privacy constraints, in which case one can always default a uniform sampling scheme with $p_i=1/n$.} If not all clients are available at all time one can use Approximate Independent Sampling~\cite{horvath2018nonconvex} that leads to the same effect.

\subsection{Extensions}

As mentioned previously, we introduce \fedgen (Algorithm~\ref{alg:FedGen}) in Section~\ref{sec:fedgen} which encapsulates \fedavg, \fednova and \fedshuffle\ as special cases and unifies the convergence analysis of these three methods. As an advantage, we use this unified framework to show that it is better to handle objective inconsistency by scaling the step sizes rather than scaling the updates, i.e., it is better to run \fedshuffle rather than \fednova as \fedshuffle allows for larger theoretical step sizes, see Remark~\ref{rem:fedshuffle_is_better} for details.
In addition, our general analysis allows for different extensions such as each client running different arbitrary number of local epochs. \fedgen also allows us to run and analyze hybrid approaches of mixing step size scaling with update scaling to overcome the objective inconsistency. 
These hybrid approaches would be efficient when applying step size scaling only, i.e. \fedshuffle, might not overcome objective inconsistency due to system challenges.
For example, such a scenario could happen when some clients cannot finish their predefined number of epochs due to a time-out enforced by the master, e.g., large variance in computing time, random drop-off, or interruption during local training. In such scenarios, \fedgen allows additional adjustments through update scaling.
\vspace{-1em}
\section{Convergence guarantees}
\label{sec:convergence}

In our main theorem, we establish the convergence guarantees for the Algorithm~\ref{alg:FedShuffle_simple}.
\subsection{Main result}
Before proceeding with the theorem, we define several quantities derived from the constants that appear in the assumptions 
\begin{align*}
     &M \eqdef \max_{i \in [n]} \cbr*{\frac{s_i}{p_i} w_i}, \qquad 
     P^2 \eqdef \max_{i \in [n]} \frac{P_i^2}{|\cD_i|}, \\
     &\sigma^2 \eqdef \frac{1}{|\cD|}\sum_{i \in [n]} \sigma_i^2, \qquad 
     \beta \eqdef 1 + (1+P)B + MB^2,
\end{align*}
 and the ones that reflect the quality of the initial solution
 $D \eqdef \norm{x^0 - x^\star}^2$ and $F \eqdef f(x^0) - f^\star$.
    
\begin{theorem}\label{thm:fedavg-full}
Suppose that the Assumptions~\ref{ass:smoothness}-\ref{ass:bounded_var}
holds. Then, in each of the following cases, there exist weights
$\{s_r\}$, local step sizes $\eta_l^r \eqdef \eta_l$ and effective step sizes $\tilde \eta^r \eqdef \tilde \eta = E \eta_g \eta_l$
such that for any $\eta_g^r \eqdef \eta_g \geq 1$ the output of \fedshuffle\ (Algorithm~\ref{alg:FedShuffle_simple})
\begin{equation}\label{eqn:fedshuffle-output}
    \bar x^R = x^{r} \text{ with prob. } \frac{s_r}{\sum_{\tau}s_{\tau}} \text{ for } r\in\{0,\dots,R-1\} \,
\end{equation}
satisfies
\begin{itemize}
        \item \textbf{Strongly convex:} $\cbr{f_{ij}}$ satisfy \eqref{eq:def_strongly_convex} for $\mu >0$, $\tilde \eta \leq \frac{1}{4\beta L}$,  $R \geq\frac{4\beta L}{\mu}$ then 
        \[
\E{f(\bar{x}^R) - f(x^\star)} \leq \tilde\cO\rbr*{\frac{MG^2}{\mu R}  + \frac{(E^2 + P^2)G^2 + \sigma^2}{\mu^2 R^2 \eta_g^2 E^2} + \mu D^2 \exp(-\frac{\mu}{8\beta L} R)}\,,
    \]
    \item \textbf{General convex:} $\cbr{f_{ij}}$ satisfy \eqref{eq:def_strongly_convex} for $\mu = 0$,
    \[
\E{f(\bar{x}^R) - f(x^\star)} \leq \cO\rbr*{\frac{\sqrt{DM}G}{\sqrt{R}} + \frac{D^{2/3}((E^2 + P^2)G^2 + \sigma^2)^{1/3}}{R^{2/3}\eta_g^{2/3} E^{2/3}}+ \frac{L D \beta}{R}}\,,
    \]
    \item \textbf{Non-convex:} $\tilde \eta \leq \frac{1}{4\beta L}$,  then
    \[
\E{\norm{\nabla f(\bar{x}^R)}^2} \leq \cO\rbr*{\frac{\sqrt{FML}G}{\sqrt{R}} + \frac{F^{2/3}L^{1/3}((E^2 + P^2)G^2 + \sigma^2)^{1/3}}{R^{2/3}\eta_g^{2/3} E^{2/3}}+ \frac{L F \beta}{R}}\,.
    \]
    \end{itemize}
\end{theorem}

Let us discuss the obtained rates. First, note that for a sufficiently large number of communications rounds, the first term is the leading term. This term together with the last term correspond to the rate of Distributed \gd\ with partial participation, where each sampled client returns its gradient as the update. If each client participates then $M=0$ and the first term vanishes. The second term comes from random reshuffling. Note here that the dependency on the number of communication rounds is $R^2$ and $R^{2/3}$, respectively, while for independent stochastic gradients, this would be $R$ and $R^{1/3}$. This shows that the variance is decreased when one employs random reshuffling instead of with-replacement sampling. We further note that the middle term can be completely removed in the limit where $\eta_g \rightarrow \infty$, and the local variance $\sigma^2$ vanishes when $E \rightarrow \infty$. We note that such property was not observed for \fednova. The limit $\eta_g = \infty$ implies $\eta_l = 0$ and thus \fedshuffle reduces to \gd\ with partial participation. To analyze the effect of the cohort size $b$ (number of sampled clients) on the convergence rate, we look at the special case where each client is sampled independently with probability $p_i = b w_i$ (assume $b w_i \leq 1$ for simplicity) for all $i \in [n]$. We refer to this sampling as \textit{importance sampling} as it is easy to see that the $M$ term is minimized for this sampling~\cite{horvath2018nonconvex}. In this particular case, $M = \nicefrac{(1 - \min\cbr*{w_i})}{b}$ and, thus, we obtain theoretical linear speed with respect to the expected cohort size $b$.


Lastly, we note that the obtained rates do not asymptotically improve upon distributed \gd\ with partial participation, but this is the case for every local method with local steps based only on the local dataset, i.e., no global information is exploited. 

\begin{figure}[t]
	\centering
    \begin{algorithm}[H]
        \begin{algorithmic}[1]
        \STATE {\bfseries Input:}  initial global model $x^0$, global and local step sizes $\eta_g^r$, $\eta_l^r$, proper distribution $\cS$ 
        \FOR{each round $r=0,\dots,R-1$}
            \STATE master broadcasts $x^{r}$ to all clients $i\in \mathcal{S}^r \sim S$
            \FOR{each client $i\in \cS^r$ (in parallel)}
                \STATE initialize local model $y_{i, 0, 0}^r\leftarrow x^{r}$
                \FOR{$e=1,\dots, E_i$}
                    \STATE Sample permutation $\{\Pi^r_{i, e, 0}, \hdots, \Pi^k_{i, e, |\cD_i|-1}\}$ of $\{1, \hdots, |\cD_i| \}$
                    \FOR{$j=1,\dots, |\cD_i|$}
                        \STATE local step size $\eta_{l, i}^r = \nicefrac{\eta_l^r}{b_i}$
                        \STATE update $y_{i, e, j}^r = y_{i, e, j-1}^r - \eta_{l, i}^r \nabla f_{i\Pi^r_{i, e, j-1}}(y_{i, e, j-1}^r)$
                    \ENDFOR
                    \STATE $y_{i, e+1, 0}^r = y_{i, e, |\cD_i|}^r$
                \ENDFOR
                \STATE send $\Delta_i^r = y_{i,E_i, |\cD_i|}^r - x^r$ to master
            \ENDFOR
           \STATE  master computes $\Delta^r = \sum_{i\in \cS^r} \frac{\tw_i}{q^{\cS^r}_i}\Delta_i^r$
            \STATE master updates global model $x^{r+1} = x^{r} - \eta_g^r \Delta^r$
        \ENDFOR
        \end{algorithmic}
        \caption{\fedgen $(b, \tw, q)$}
        \label{alg:FedGen}
    \end{algorithm}
\end{figure}

\subsection{\fedgen: General shuffling method}
\label{sec:fedgen}

In this section, we introduce and analyze \fedgen. \fedgen\ is a class of algorithms parametrized by local and global step sizes $\cbr*{\eta_l^r}_{r=0}^{R-1}$ and $\cbr*{\eta_g^r}_{r=0}^{R-1}$, step size normalization $b = \cbr*{b_i}_{i=1}^n$, where each of its element $b_i$ is the step size normalization for client $i$, aggregation weights $\cbr*{\tw_i}_{i=1}^n$ and the aggregation normalization constants $\cbr*{\cbr*{q_i^{S^r}}_{i \in S, S^r \sim \cS}}_{r=0}^{R-1}$. We include the pseudocode in Algorithm~\ref{alg:FedGen}. Later in this section, we show that \fedgen\ with its parametrization covers a wide variety of the FL algorithms in the form as they are implemented in practice.   

For the convergence analysis, we first introduce the function $\hf(x)$, which we show is the true objective optimized by Algorithm~\ref{alg:FedGen}.

\begin{equation}
    \label{eq:wrong_objective}
    \begin{split}
        &\hf(x) \eqdef \sum_{i \in [n]} \hw_i f_i(x),\; \text{where } \\ &\quad \hw_i \eqdef \frac{\tw_i |\cD_i|E_i}{W q_i b_i} \text{ with } \frac{1}{q_i} = \EE{\cS}{\frac{1}{q^{\cS}_i} 1_{i\in \cS}} \text{ and } W = \sum_{i \in [n]}  \frac{\tw_i |\cD_i|E_i}{q_i b_i}.
    \end{split}
\end{equation}

We denote $\hx$ to be an optimal solution of $\hf$ and $\hfs$ to be its functional value.

Furthermore, \fedgen allows the normalization to depend on the sampled clients (line 16 of Algorithm~\ref{alg:FedGen}, which requires more general notion of variance. For this purpose, we introduce a $n \times n$ matrix $\mH$, where its elements $\mH_{i,j} = \EE{\cS}{\frac{q_i q_j}{q^\cS_i q^\cS_j} 1_{i,j\in \cS}}$ and $h$ is its diagonal with $h_i =  \EE{\cS}{\frac{q_i^2}{\rbr{q^\cS_i}^2} 1_{i\in \cS}}$. We further define vector $s \in \R^n$ to be such vector that it holds 
\[
\mH - ee^\top \preceq {\rm \bf Diag}(h_1 v_1, h_2 v_2,\dots, h_n v_n),
\]
where $e \in \R^n$ is all ones vector. Note that this is not an assumption as such upper-bound always exists due to Gershgorin circle theorem, e.g., for $v_i = n$ for all $i \in [n]$ this holds. Equipped with these extra quantities, we proceed with the convergence guarantees for (strongly-)convex and non-convex functions.

\begin{theorem}\label{thm:fedshuffle_gem-full}
Suppose that the Assumptions~\ref{ass:smoothness}-\ref{ass:bounded_var}
holds. Then, in each of the following cases, there exist weights
$\{v_r\}$ and local step sizes $\eta_l^r \eqdef \eta_l$
such that for any $\eta_g^r \eqdef \eta_g \geq 1$ the output of \fedgen\ (Algorithm~\ref{alg:FedGen})
\begin{equation}\label{eqn:fedgen-output}
    \bar x^R = x^{r} \text{ with probability } \frac{v_r}{\sum_{\tau}v_{\tau}} \text{ for } r\in\{0,\dots,R-1\} \,.
\end{equation}
satisfies
\begin{itemize}
        \item \textbf{Strongly convex:} $\cbr{f_{ij}}$ satisfy \eqref{eq:def_strongly_convex} for $\mu >0$, $\eta_l \leq \frac{1}{4\beta L \eta_g}$,  $R \geq\frac{4\beta L}{\mu}$ and $b_i \geq E_i |\cD_i|$ then
        \[
\E{\hf(\bar{x}^R) - \hf(\hx)} \leq
\tilde\cO\rbr*{\frac{M_1G^2}{\mu R}  + \frac{(1 + P^2)G^2 + \sigma^2}{\mu^2 R^2 \eta_g^2} + \mu D^2 \exp(-\frac{\mu}{8\beta L} R)}\,,
    \]
    \item \textbf{General convex:} $\cbr{f_{ij}}$ satisfy \eqref{eq:def_strongly_convex} for $\mu = 0$, $\eta_l \leq \frac{1}{4\beta L \eta_g}$, $R \geq 1$ and $b_i \geq E_i |\cD_i|$ then
    \[
\E{\hf(\bar{x}^R) - \hf(\hx)} \leq
\cO\rbr*{\frac{\sqrt{DM_1}G}{\sqrt{R}} + \frac{D^{2/3}((1 + P^2)G^2 + \sigma^2)^{1/3}}{R^{2/3}\eta_g^{2/3} }+ \frac{L D \beta}{R}}\,,
    \]
    \item \textbf{Non-convex:} $\eta_l \leq \frac{1}{4\beta L \eta_g}$, $R \geq 1$ and $b_i \geq E_i |\cD_i|$ then 
    \[
\E{\norm{\nabla \hf(\bar{x}^R)}^2} \leq \cO\rbr*{\frac{\sqrt{FM_1L}G}{\sqrt{R}} + \frac{F^{2/3}L^{1/3}((1 + P^2)G^2 + \sigma^2)^{1/3}}{R^{2/3}\eta_g^{2/3} }+ \frac{L F \beta}{R}}\,,
    \]
    \end{itemize}
    where $\beta \eqdef 1 + M_2 + (1+P)B + M_1 B^2$, $P^2 \eqdef \max_{i \in [n]}\frac{P_i^2}{3|\cD_i|E_i^2}$, $\sigma^2 \eqdef \sum_{i \in [n]}\hw_i\frac{\sigma^2_i}{3|\cD_i|E_i^2}$, $M_1 \eqdef \max_{i \in [n]} \cbr*{h_i v_i \hw_i}$, $M_2 \eqdef \rbr*{\sum_{i \in [n]} E_i |\cD_i|} \max_{i \in [n]} \cbr*{\frac{\tw_i}{W q_i b_i}}$, $D \eqdef \norm{x^0 - x^\star}^2$, and $F \eqdef f(x^0) - f^\star$.
\end{theorem}

The above rate is a generalization of results presented in Section~\ref{sec:convergence} and we refer the reader to this section for the detailed discussion.

\subsection{Special cases}
\label{sec:comparison}
In this section, we show how \fedgen captures not only our proposed \fedshuffle, but also both \fedavg\ and \fednova\ with heterogeneous data, arbitrary client sampling, random reshuffling, non-identical local steps, stateless clients and server and local step sizes. To the best of our knowledge, our work is the first to provide such comprehensive analysis. 

We start with the \fedavg\ algorithm. Algorithm~\ref{alg:FedAvg} contains a detailed pseudocode of how it is usually implemented in practice. It is easy to verify that \fedgen\ covers this implementation with the following selection of parameters
\[
\hw = w,\quad b=\max_{i \in [n]}\cbr*{E_i |\cD_i| } e \text{ and } q^{S^r}_i = \frac{b}{n}\sum_{j \in S^r} w_j.
\]

Unfortunately, it is not guaranteed that $\hf = f$. For instance, when all clients participate in each round, each client runs $E$ local epochs and weight of each client is proportional to its dataset size, i.e. $w_i = \nicefrac{|\cD_i|}{|\cD|}$, then $\hw_i = \nicefrac{|\cD_i|^2}{\sum_{j \in [n]} |\cD_j|^2}$, which means that the objective that we end up optimizing favours clients with larger amount of data. This inconsistency is partially removed, when client sampling is introduced as the clients with the larger number of data points are in average normalized with larger numbers, i.e. $q_i$ grows. The inconsistency is only fully removed when only one client is sampled uniformly at random, which does not reflect standard FL systems where a reasonably large number of clients is selected to participate in each round. 

\fednova\ tackles the objective inconsistency issue by re-weighting the local updates during the aggregation step. For the same example as before with the full participation, \fednova\ uses the same parameters as \fedavg with a difference that $\hw = \nicefrac{e}{n}$ that leads to $\hf = f$. Apart from full participation, \cite{fednova2020neurips} analyze a client sampling scheme with one client per communication round sampled with probability $\nicefrac{|\cD_i|}{|\cD|}$.  

In our work, apart from providing a general theory for shuffling methods in FL, we introduce \fedshuffle that is motivated by insights obtained from our general theory provided in Theorem~\ref{thm:fedshuffle_gem-full}. \fedshuffle\ preserves the original objective aggregation weights, i.e., $\hw = w$, it uses unbiased normalization weights during the aggregation step and it sets $b_i = E_i |\cD_i|$. Such parameters choice guarantees that  $\hf = f$. 

\begin{remark}[\fedshuffle is better than \fednova]
\label{rem:fedshuffle_is_better}
\fedshuffle preserves the original objective aggregation weights, i.e., $\hw = w$, and it uses unbiased normalization weights during the aggregation step. In addition, contrary to \fedavg and \fednova, \fedshuffle\ uses $b_i = E_i |\cD_i|$ (\fedavg and \fednova require $b_i = \max_{j \in [n]}\cbr*{E_j |\cD_j| }$), which implies that it allows larger local step sizes than both \fedavg and \fednova, but it does not introduce any inconsistency as it still holds that $\hf = f$ as for \fednova. Differently from \fednova, \fedshuffle\ does not degrade the local progress made by clients by diminishing their contribution via decreased weights in the aggregation step. Still, it achieves the objective consistency by balancing the clients' progress at each step while preserving the worst-case convergence rate. 
\end{remark}

\subsection{Improving upon non-local (master-only) methods}
\label{sec:mvr}
 Contrary to the relatively negative worst-case results presented in the previous section, local methods have been observed to perform significantly better in practice~\cite{mcmahan17fedavg} when compared to non-local (i.e., one local step) methods. To overcome this issue, Karimireddy et al.~\cite{karimireddy2020scaffold} proposed to use a Hessian similarity assumption~\cite{Arjevani2015:lowerbound}, and they showed that local steps bring improvement when the objective is \textit{quadratic}, \textit{all clients} participate in each round and the local steps are corrected using \texttt{SAGA}-like \textit{variance reduction}~\cite{defazio2014saga}. Later, Karimireddy et al.~\cite{karimireddy2020mime} proposed \textsc{MimeMVR} that uses the Momentum Variance Reduction (\mvr) technique~\cite{cutkosky2019momentum, tran2019hybrid} and extended the prior results to smooth non-convex functions with uniform partial participation. 
 In our work, we build upon these results and show that \fedshuffle\ can also improve in terms of communication rounds complexity. To achieve this, we introduce \fedmvr, a \fedshuffle type algorithm that is extended with \textsc{MimeMVR}'s momentum technique. Each local update of \fedmvr\ has the following form
\begin{equation*}
 y_{i, e, j}^r = y_{i, e, j-1}^r - \frac{\eta_l^r}{|\cD_i|} d_{i, e, j-1},
\end{equation*}
where
\begin{align}
    \label{eq:mvr_loc_update}
    d_{i, e, j} = a\nabla f_{i\Pi^r_{i, e, j}}(y_{i, e, j}^r) + (1-a) m^r 
    + (1-a)\rbr*{\nabla f_{i\Pi^r_{i, e, j}}(y_{i, e, j}^r) - \nabla f_{i\Pi^r_{i, e, j}}(x^r)}
\end{align}
where the momentum term $m^{r}$ is updated at the beginning of each communication round as
\begin{align}
    \label{eq:mvr_mom_update}
    m^{r} = a\sum_{i\in \cS^r} \frac{w_i}{p_i}\nabla f_i(x^r) + (1-a) m^{r-1} 
    + (1-a)\rbr*{\sum_{i\in \cS^r} \frac{w_i}{p_i}\rbr{\nabla f_i(x^r) - \nabla f_i(x^{r-1})}}.
\end{align}
For the notation details, we refer reader to Algorithm~\ref{alg:FedShuffle} in the Appendix. The above equations can be seen as the standard momentum (first two terms) with an extra correction term (the last term). For the detailed explanation about the motivation behind the used momentum technique, we refer reader to \cite{cutkosky2019momentum}. Note that the momentum term is only updated \textit{once} in each communication round, this is to reduce the local drift as proposed by \cite{karimireddy2020mime}.

A convergence guarantee of \fedmvr\ in the non-convex regime follows.

\begin{theorem}\label{thm:mvr_is_better}
  Let us run \fedmvr\ with step sizes $$\eta_l = \frac{1}{40E}\min\cbr*{\frac{1}{ \delta }, \rbr*{\frac{f(x^{0}) - f^\star}{R\delta^2(G^2 + \sigma^2)}}^{1/3}}$$, $\eta_g = 1$, momentum parameter $a = \max\rbr*{1152 E^2 \delta^2 \eta_l^2, \frac{1}{R}}$, and local epochs $E \geq \frac{L}{\delta}$. Then, given that Assumptions~\ref{ass:smoothness}-\ref{ass:hessian_sim} hold with $P_i = 0$ for all $i \in [n]$, $B = 1$, $\delta > 0$ and one client is sampled with probabilities $\cbr*{w_i}$, we have 
  \begin{align*}
      \frac{1}{RE}\sum_{r=0}^{R-1}\sum_{e=0}^{E-1}\E{\norm*{\nabla f(y^r_{i^r,e})}^2} \leq \cO\rbr[\bigg]{ \frac{\delta^{2/3} F^{2/3} \rbr*{G^2 + \sigma^2}^{1/3}}{R^{2/3}} +  \frac{\delta F + G^2}{R} }\,.
  \end{align*}
\end{theorem}
Note that our rate is independent of $L$ and only depends on the Hessian similarity constant $\delta$. Because $\delta \leq L$, this rate improves upon the rate of the centralized \mvr $\cO\rbr{\nicefrac{L^{2/3}}{R^{2/3}}}$. We note that the improvement is only for the number of communication rounds, and the number of gradient calls is at least the same as for the non-local centralized \mvr\ since $E\delta \leq L$, but our main concern here is the communication efficiency. 
It is worth noting that our results are qualitatively similar to those provided by \cite{karimireddy2020mime}, but in our work, we consider a more challenging setting. The reason is that each client runs local epochs using random reshuffling. Therefore, we work with biased gradients. In addition, we allow for heterogeneity in the number of samples per client, which brings another challenge that needs to be properly addressed in the analysis. 
Lastly,  we note that our step size scaling is essential in the provided \fedshufflemvr convergence theory as it requires balancing the progress made by each client. Therefore, it is not clear whether a combination of \fednova and {\sc MVR} can lead to a similar improvement.
\section{Experimental evaluation}
\begin{figure}[t]
    \centering
    \includegraphics[width=0.4\textwidth]{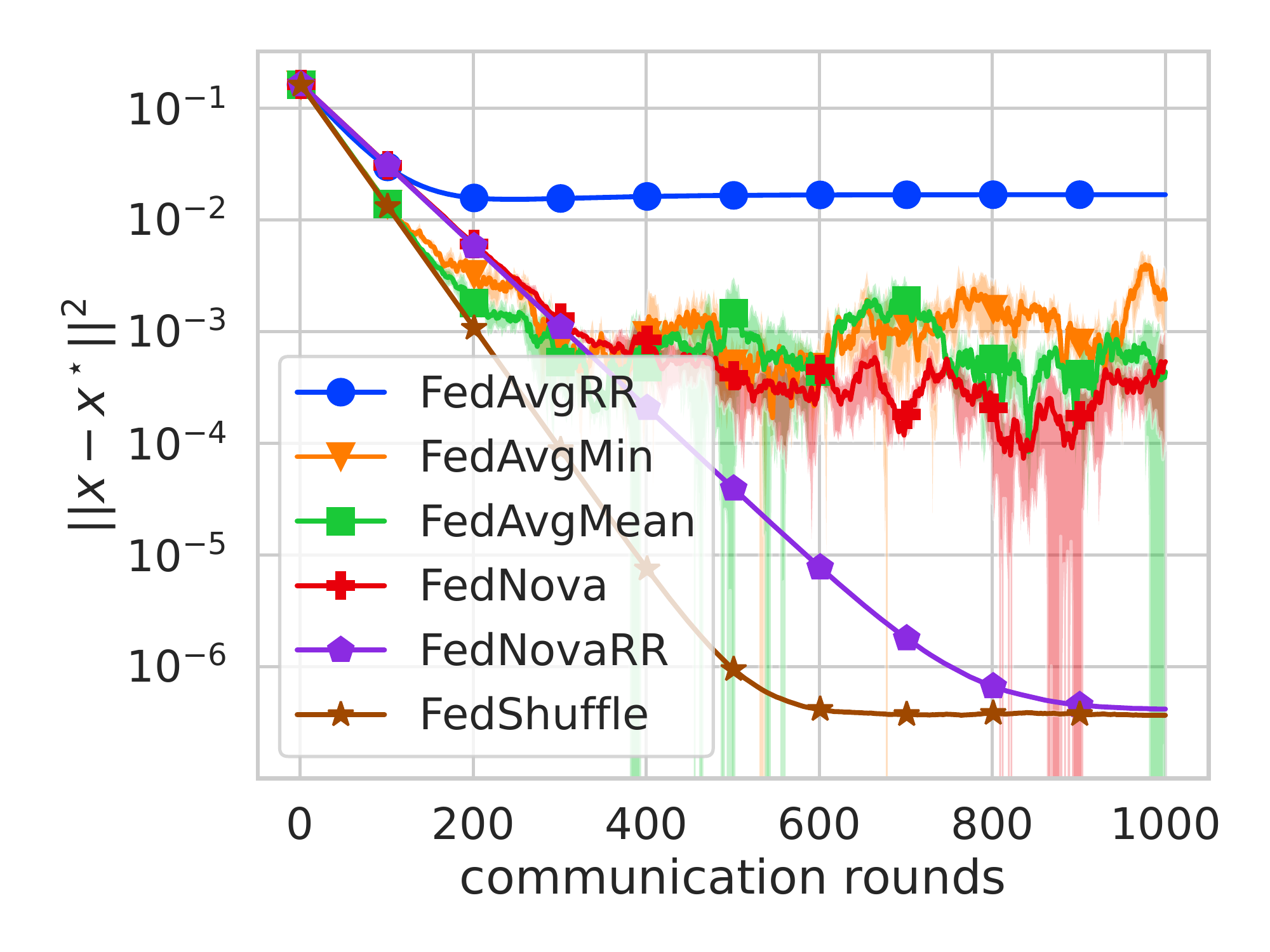}
    \includegraphics[width=0.4\textwidth]{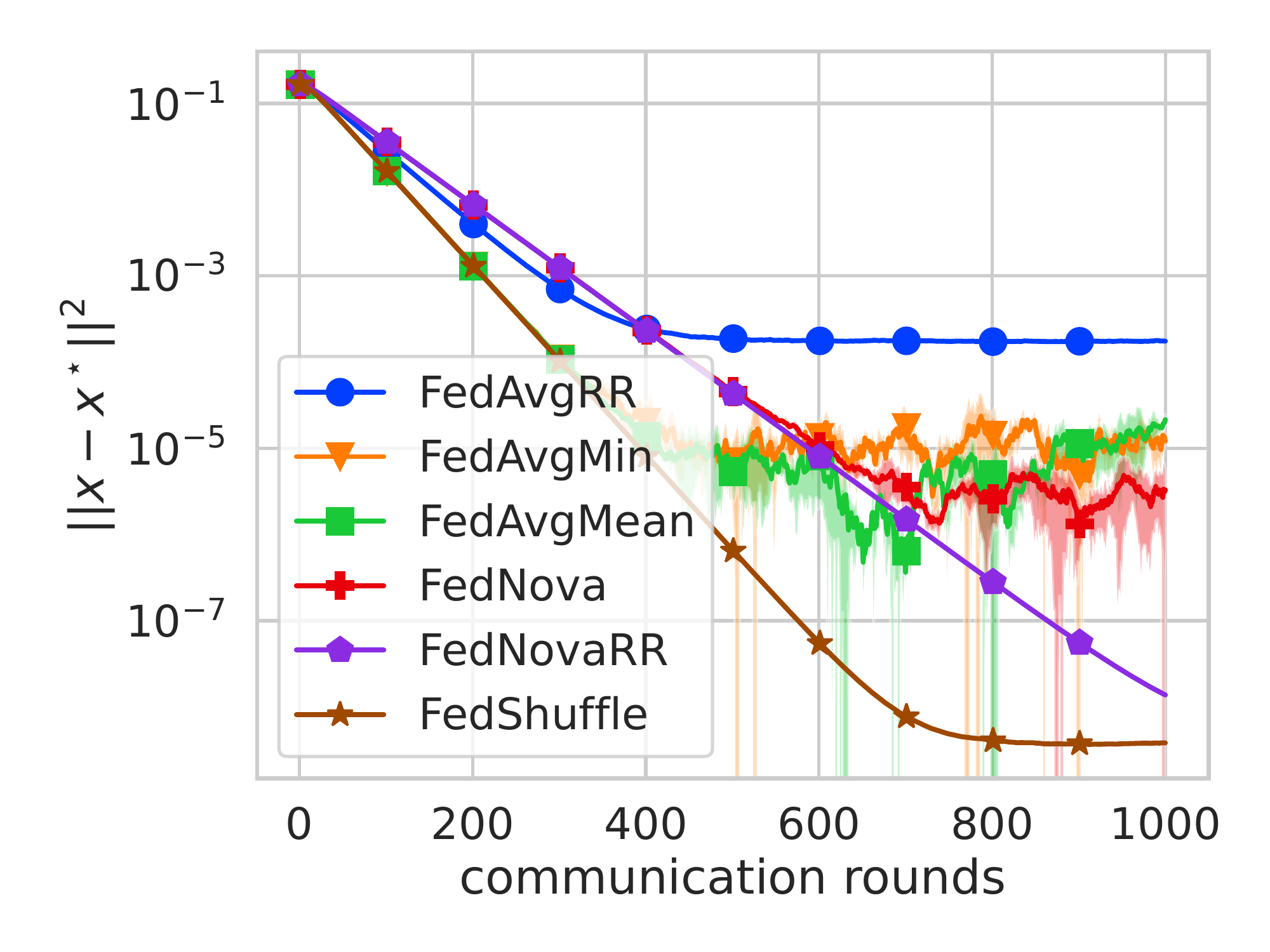}
    \includegraphics[width=0.4\textwidth]{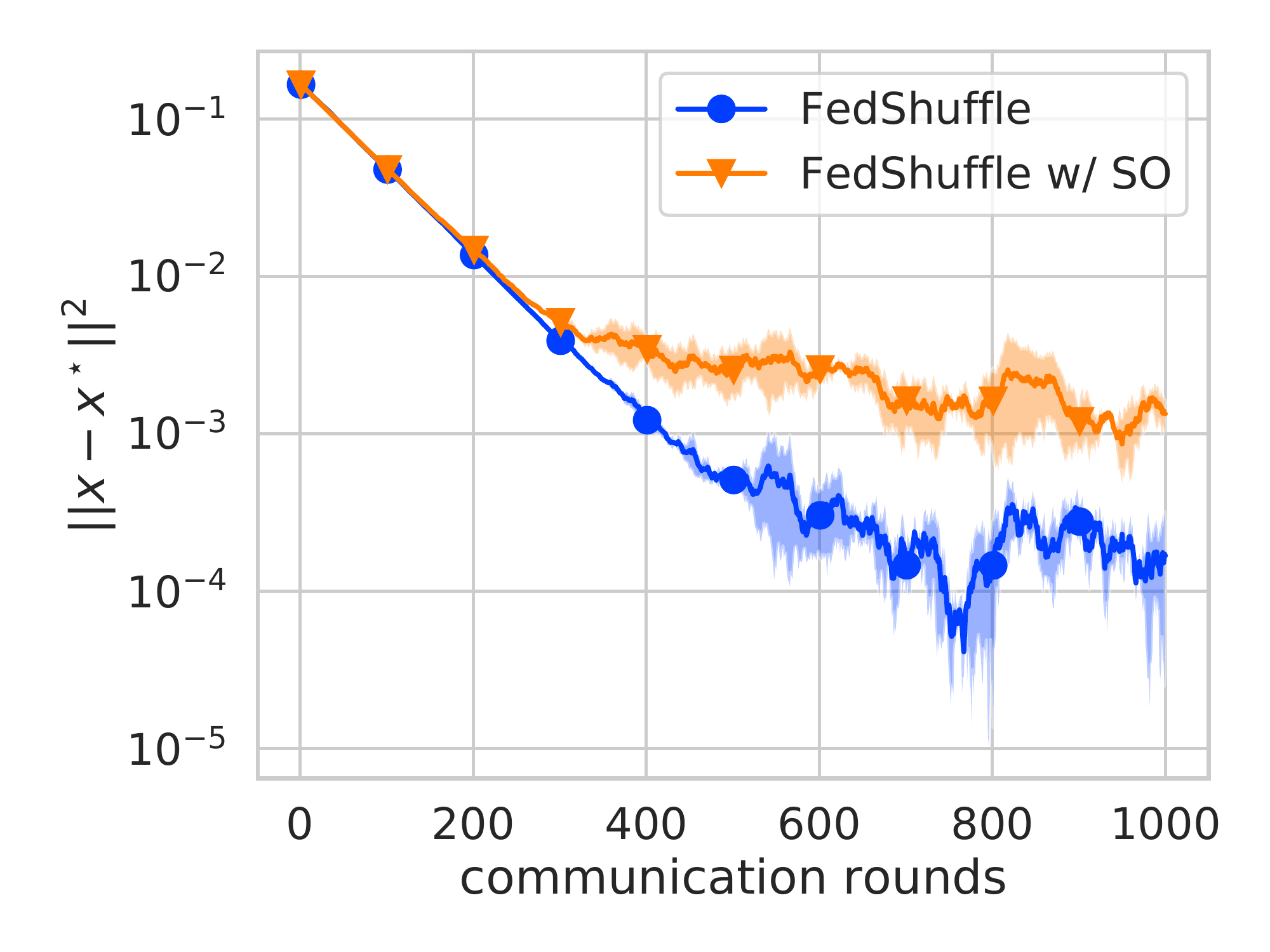}
    \includegraphics[width=0.4\textwidth]{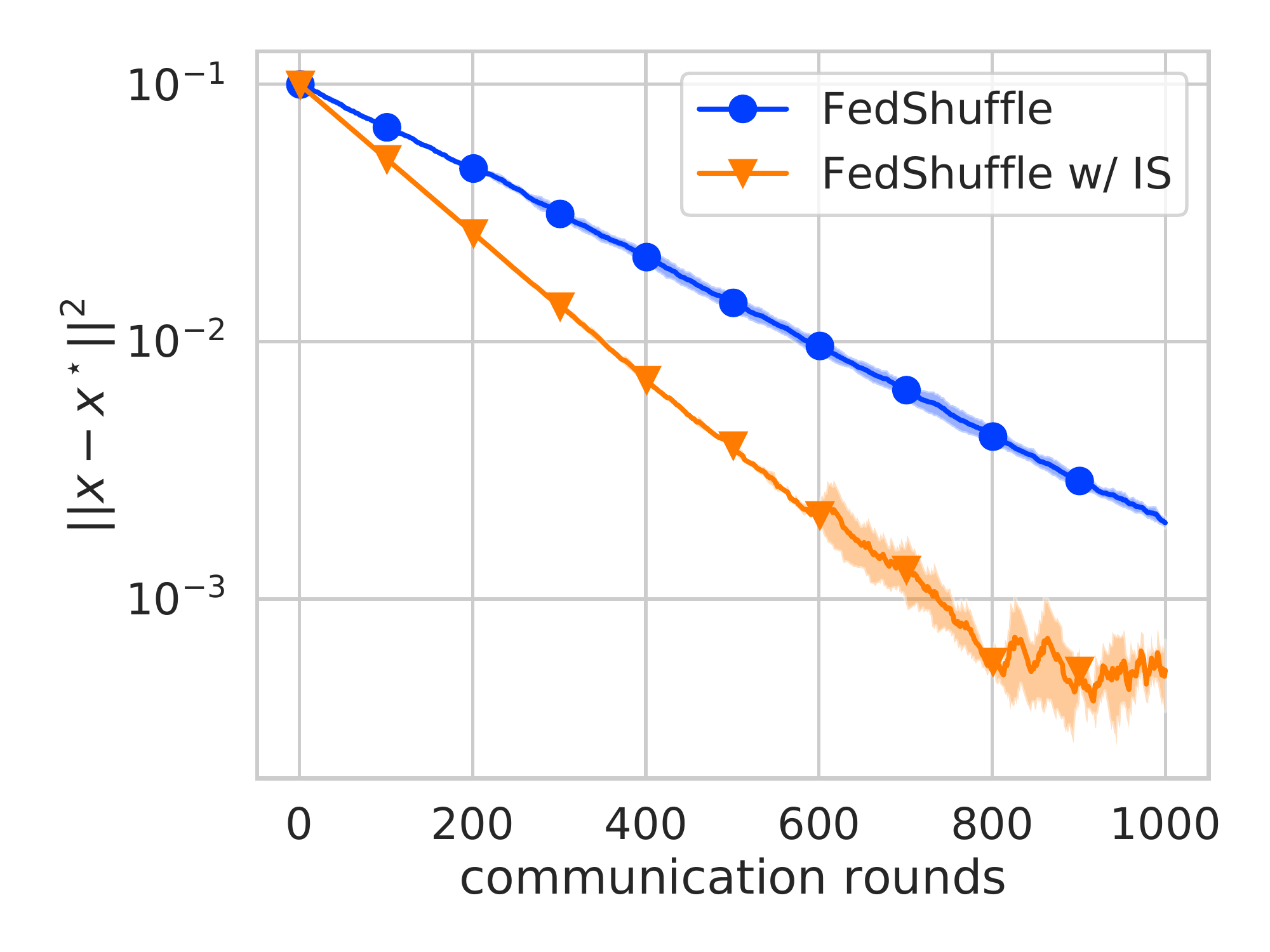}
    \caption{Quadratic objective as defined in \eqref{eq:quad_obj}. Each client runs one local epoch. Left: A comparison of \fedavg, \fedavg\ with reshuffling (\fedavgrr), \fednova\ and \fednova\ with reshuffling (\fednova\textsc{RR}) and \fedshuffle with full participation. Middle Left: the same baselines with the global momentum $0.9$. Middle Right: \fedshuffle\ w/ Sum One and plain \fedshuffle with partial participation (two clients sampled uniformly at random).
    Right: \fedshuffle\ with uniform sampling and importance sampling (IS) with partial participation (one client per round).
    }
    \label{fig:toy_experiments}
\end{figure}
For the experimental evaluation, we compare three methods --- \fedavg, \fednova, and our \fedshuffle --- with different extensions such as random reshuffling or momentum. We run two sets of experiments. In the first, we perform an ablation study on a simple distributed quadratic problem to verify the improvements predicted by our theory. In the second part, we compare all of the methods for training deep neural networks on the CIFAR100 and Shakespeare datasets. Details of the experimental setup can be found in Appendix~\ref{sec:setup}. As expected, our findings are that \fedshuffle consistently outperforms other baselines. Moreover,
global momentum 
leads to an improved performance of all methods as predicted by our theory.

\subsection{Results on quadratic functions}
In this section, we verify our theoretical findings on a convex quadratic objective; see \eqref{eq:quad_obj} in appendix for details. Figure~\ref{fig:toy_experiments} summarizes the results. The left-most plot showcases the comparison of \fedavg, \fedavg with reshuffling (\fedavgrr), \fednova as analyzed in \cite{fednova2020neurips} (with sampling), \fednova with reshuffling (\fednovarr) as we analyze it in Section~\ref{sec:fedgen}, and \fedshuffle. All clients participate in each round and run one local epoch with batch size $1$. 

As expected, \fedavgrr saturates at a higher loss, since it optimizes the wrong objective, see Section~\ref{sec:fedgen}. \fedmin fixs the objective inconsistency of the \fedavg, but it is still dominated by other baselines due to decreased amount of local work per client. \fednova provides a better performance since it does not contain any inconsistency by construction, but its performance is later dominated by noise coming from stochastic gradients. The same holds for \fedmean. As predicted by our theory, random reshuffling decreases the stochastic noise and \fednovarr improves upon the performance of \fednova. \fedshuffle dominates all the baselines since it does not have any objective inconsistency, it uses a superior method to remove inconsistencies compared to \fednova, and it also incorporates random reshuffling which itself provides some variance reduction.

\begin{figure}[t]
    \centering
    \subfigure[Shakespeare w/ LSTM]{
        \includegraphics[width=0.4\textwidth]{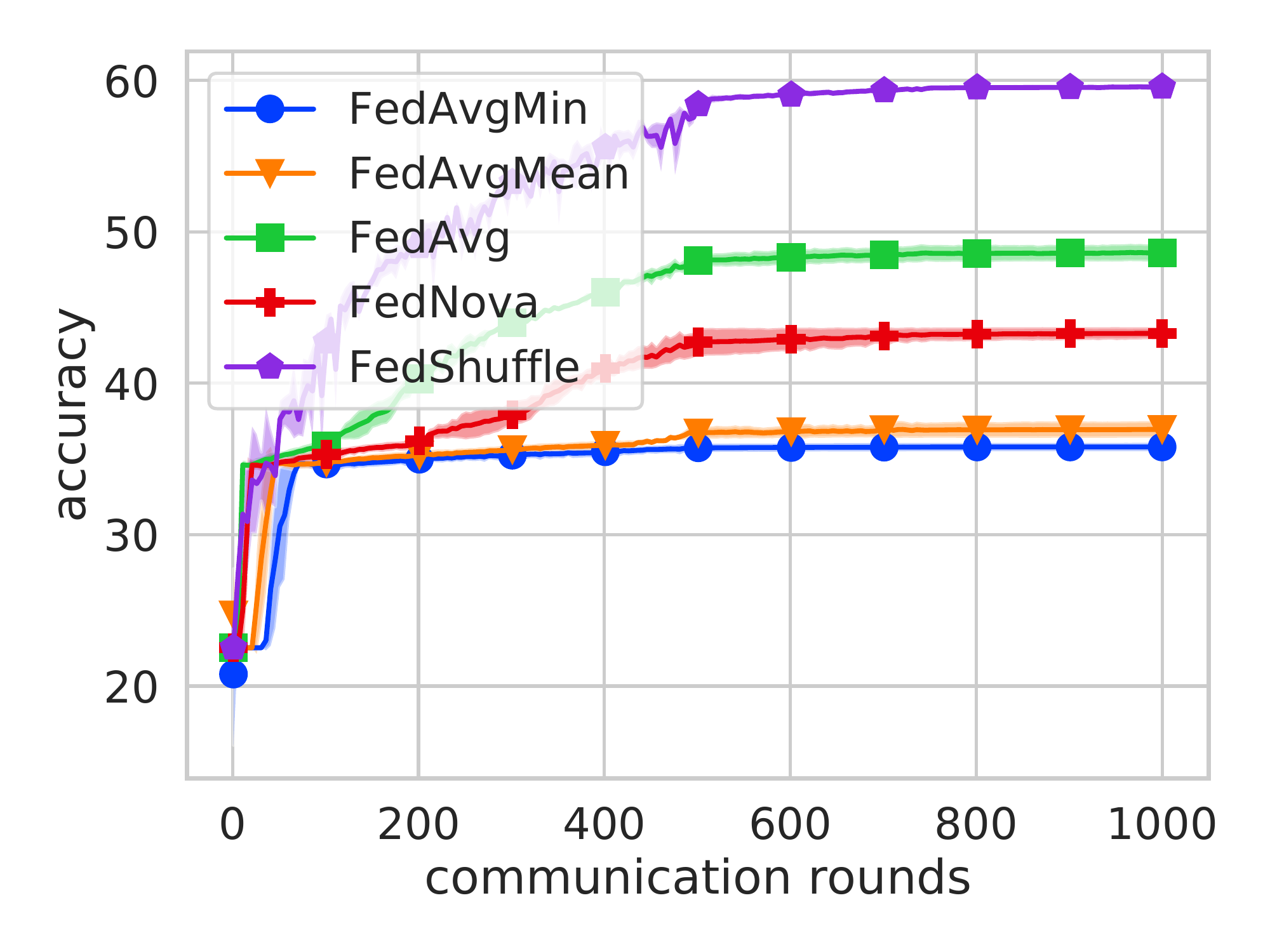}
        \includegraphics[width=0.4\textwidth]{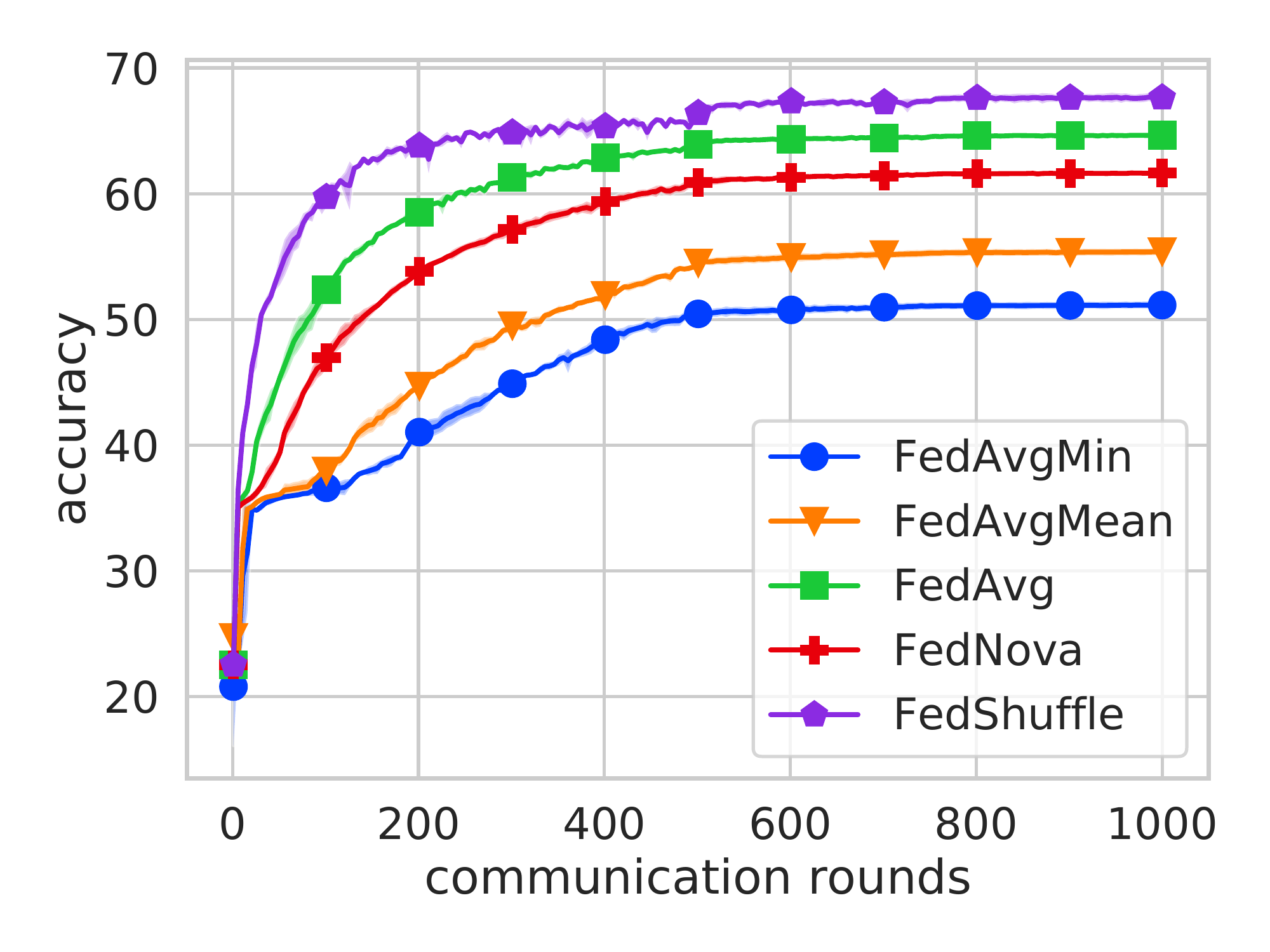}
        }
    \hfill
    \subfigure[CIFAR100 (TFF Split) w/ ResNet18]{
        \includegraphics[width=0.4\textwidth]{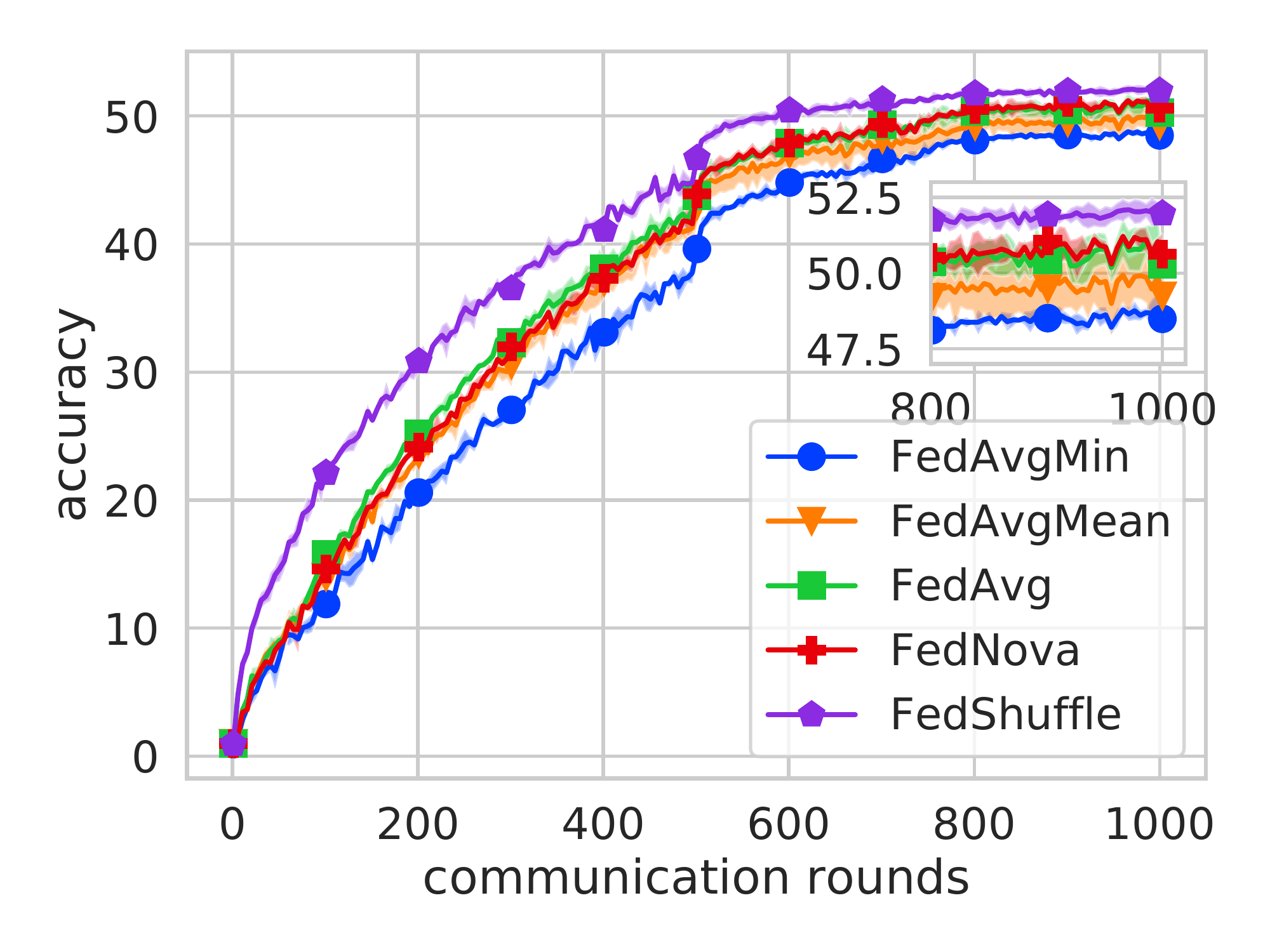} 
        \includegraphics[width=0.4\textwidth]{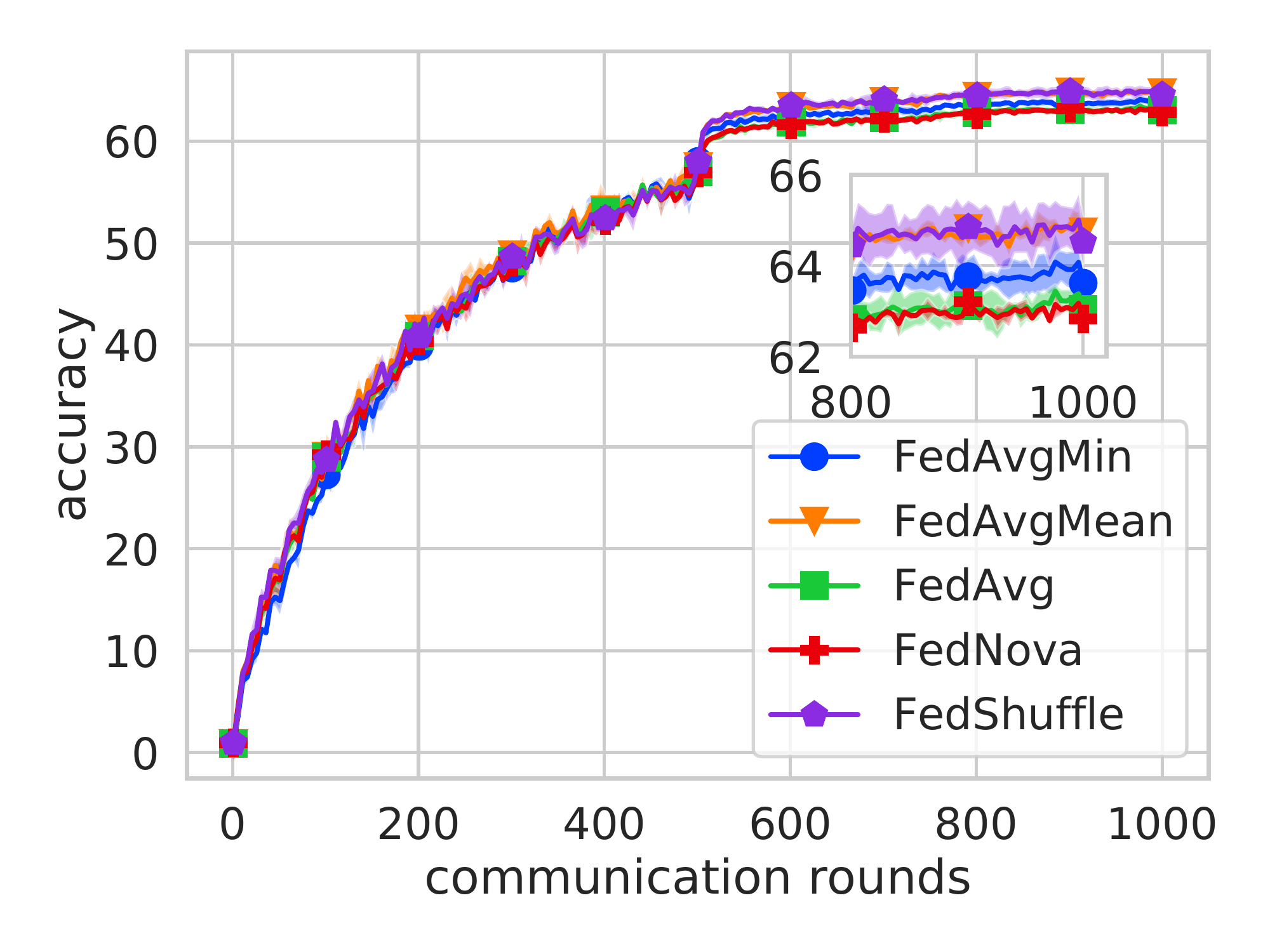}
        }
    \caption{Comparison of \fedmin, \fedavg, \fednova, \fedshuffle on real-world datasets.
    Partial participation: in each round 16 client is sampled uniformly at random.
    All methods use random reshuffling. For Shakespeare, number of local epochs is $2$ and for CIFAR100, it is 2 to 5 sampled uniformly at random at each communication round for each client. Left: Plain methods.
    Right: Global momentum $0.9$. 
    }
    \label{fig:neural_nets}
\end{figure}

In the second plot, we use the same baselines but include global momentum as defined in \eqref{eq:mvr_loc_update} and \eqref{eq:mvr_mom_update}. We can see that this technique helps \fedavg to reduce its objective inconsistency since the momentum in \eqref{eq:mvr_mom_update} is unbiased. For other methods, we can see that momentum has a beneficial variance reduction effect as expected, and we observe convergence to a solution with higher precision. Note that \fedshuffle still performs the best as predicted by our theory.

In the third plot, we analyse the difference between the default implementation of the aggregation step, that is denoted as ``sum one'' (\fedshuffle w/ SO) since the sum of weights during aggregation is normalized to be one, and our unbiased version, where the weights are scaled by the probability of sampling the given client. We sample two clients in each step uniformly at random and each client runs one local epoch. As was discussed in Section~\ref{sec:fedshuffle}, we observe that \fedshuffle w/ SO converges to a worse solution due to objective inconsistency resulting from the biased aggregation.  

Finally, we compare \fedshuffle with uniform and importance sampling, where we sample one client per round and the sampled client runs one local epoch. For importance sampling (IS), each client is sampled proportionally to its dataset size. To better showcase the effect of importance sampling, we use a slightly different objective, where $d=10$, the first client holds $8$ data points, and other two clients hold one.  As predicted by our theory (decrease of the $M$ term in Theorem~\ref{thm:fedavg-full}), importance sampling leads to a substantial improvement. 

\begin{table}[t]
\centering
\renewcommand{\arraystretch}{1.}
\begin{minipage}[b]{0.495\linewidth}
    \centering
    \caption{Shakespeare dataset.}
    \resizebox{\columnwidth}{!}{
    \begin{tabular}{|c|c|c|}
    \hline
    Accuracy   & Plain         & Momentum     \\ \hline 
    \fedmin     & $35.79 \pm 0.28$ & $51.13 \pm 0.27$ \\ \hline 
    \fedmean & $36.95 \pm 0.55$ & $55.38 \pm 0.27$ \\ \hline 
    \fedavg     & $48.61 \pm 0.56$ & $64.64 \pm 0.10$ \\ \hline 
    \fednova    & $43.29 \pm 0.40$ & $61.65 \pm 0.06$ \\ \hline 
    \fedshuffle & $\mathbf{59.57 \pm 0.14}$ & $\mathbf{67.63 \pm 0.35}$ \\ \hline 
    \end{tabular}
    }
    \label{tab:shakespeare}
\end{minipage}
\hfill
\begin{minipage}[b]{0.495\linewidth}
\centering
    \caption{CIFAR100 dataset.}
    \resizebox{\columnwidth}{!}{
    \begin{tabular}{|c|c|c|c|}
    \hline
    Accuracy   & Plain         & Momentum   \\ \hline
    \fedmin & $48.49 \pm 0.28$ & $63.62 \pm 0.20$ \\ \hline
    \fedmean & $49.24 \pm 0.92$ & $\mathbf{64.75 \pm 0.20}$ \\ \hline
    \fedavg   & $50.29 \pm 0.31$ & $63.04 \pm 0.51$ \\ \hline
    \fednova   & $50.63 \pm 0.66$ & $62.83 \pm 0.19$ \\ \hline
    \fedshuffle & $\mathbf{51.97 \pm 0.20}$ & $\mathbf{64.52 \pm 0.48}$ \\ \hline
    \end{tabular}
    }
    \label{tab:cifar100}
\end{minipage}
\end{table}

\subsection{Training deep neural networks}
In the next experiments, we evaluate the same methods on the CIFAR100~\cite{cifar} and Shakespeare~\cite{mcmahan17fedavg} datasets. The results already showcased theoretically (see Sections~\ref{sec:fedshuffle}
and Section~\ref{sec:fedgen}) and empirically in the previous experiments that reshuffling leads to a substantial improvement over random sampling with replacement.
Therefore, in these experiments we focus on other aspects of \fedshuffle and show its superiority over other methods that use random reshuffling. To do that, we only consider random reshuffling methods in this section; thus \fednova, \fedavg, \fedmin and \fedmean refer to \fedavgrr, \fednovarr, \fedmin and \fedmean with (partial) random reshuffling, respectively. For each task, we run $1000$ rounds and we sample $16$ clients in each round. For Shakespeare, each client runs two local epochs. For CIFAR100, all clients have the same number of data points. Thus, we follow \cite{fednova2020neurips} to create heterogeneity and test our \fedgen framework we assume each client runs 2 to 5 epochs uniformly at random. 
We investigate which method performs the best and, furthermore, we look at the effects of the momentum.
and importance sampling, where the description of the used importance sampling strategy can be found in \cite[Section 2.3]{horvath2018nonconvex}.
We report the test accuracies in Figure~\ref{fig:neural_nets} and Tables~\ref{tab:shakespeare} and \ref{tab:cifar100}. We observe that \fedshuffle outperforms all the baselines, for the Shakespeare dataset with a large margin. With respect to the global fixed momentum, we can see that this technique helps all the methods and substantially improves their performance.
We note that \fedmean\ and \fedmin\ perform exceptionally well for CIFAR100 (w/ momentum) but not for the Shakespeare dataset. We conjecture this is due to the significant heterogeneity of the Shakespeare dataset in terms of samples per client. In this setting, \fedmin\ does not utilize most of the data on clients with larger data-sets, while \fedmean\ over-uses the data on clients with smaller data-size. Note that for CIFAR100, this is not the case as we use an equal-sized split. For completeness, we include the train loss corresponding to Figure~\ref{fig:neural_nets} in Section~\ref{sec:setup} in the appendix. 

\chapter{Concluding remarks}

In this chapter, we first summarize the contribution of each chapter and then comment on the potential future work enabled by the results of this thesis.

\section{Summary}

In Chapter~\ref{chapter2:c_nat}, we have proposed two new  compression operators: natural compression $\NC$ and natural dithering $\ND$. Moreover, we have developed a general theory for \texttt{SGD} with arbitrary bi-directional compression, which allowed us to show that our new compression operators  bring substantial savings in the amount of communication per iteration with minimal or unnoticeable increase in the number of communications when compared to their  ``non-natural'' variants, e.g., $\ND$ vs. $\SDs{u}$, resulting in a faster overall running time. Moreover, our compression techniques are compatible with other compression techniques, and this combination leads to a theoretically superior method. Our experiments corroborate these theoretical predictions and we indeed observe  speedups on  a variety of tasks.

In Chapter~\ref{chapter3:diana_2}, we analyzed various distributed algorithms that support quantized communication between worker nodes. Our analysis is general; that is, not bound to a specific quantization scheme. This fact is especially interesting as we have shown that by choosing the quantization operator (respectively, the introduced noise) in the right way, we obtain communication-efficient schemes that converge as fast as their communication-intensive counterparts. We develop the first variance-reduced methods that support quantized communication and derive concise convergence rates for the strongly-convex, convex and non-convex settings.

In Chapter~\ref{chapter4:biased}, we first looked at three classes of biased compressors, out of which two are new. The theoretical analysis of these compressors, when applied to {\tt (S)GD}, is provided. The analysis shows that biased compressors can lead to linear convergence rates. Further, we analyze and give a rate of convergence for distributed compressed {\tt SGD} with error feedback. In addition, we provide an intuition about the superiority of greedy biased compressors. Lastly, several promising new greedy compressors are proposed.

In Chapter~\ref{chapter5:induced}, we argue that if compressed communication is required for distributed training due to communication overhead, it is better to use unbiased compressors. We show that this leads to strictly better convergence guarantees with fewer assumptions. In addition, we propose a new construction for transforming any compressor into an unbiased one using a compressed Error-Feedback-like approach. Besides  theoretical superiority, usage of unbiased compressors enjoys lower memory requirements. Our theoretical findings are corroborated with empirical evaluation.

In Chapter~\ref{chapter6:optimal_sampling}, we address the communication bottleneck by a novel importance sampling of clients for FL. This optimal client sampling can be computed by the closed-form formula that only requires the norm of the updates. Furthermore, the provided sampling is the first importance sampling that works with stateless clients and secure aggregation. The theoretical results are corroborated with experimental evaluation on artificial and real-world datasets showing superior results to existing baselines.

In Chapter~\ref{chapter7:fjord}, we have introduced \tool, a federated learning method for heterogeneous device training. To this direction, \tool builds on top of our Ordered Dropout technique as a means to extract submodels of smaller footprints from a main model in a way where training the part also participates in training the whole. We show that our Ordered Dropout is equivalent to SVD for linear mappings and demonstrate that \tool's performance in the local and federated setting exceeds that of competing techniques, while maintaining flexibility across different environment setups.

In Chapter~\ref{chapter8:fedshuffle}, we introduced and analyzed \fedshuffle, which incorporates the practice of running local epochs with reshuffling in common FL implementations while also accounting for data imbalance across clients, and correcting for the resulting objective inconsistency problem that arises in \fedavg-type methods. \fedshuffle involves adjusting local learning rates based on the amount of local data, in addition to modified aggregation weights compared to prior work like \fedavg or \fednova. Under an additional Hessian smoothness assumption, incorporating momentum variance reduction leads to order-optimal rates, in the sense of matching the lower bounds achieved by non-local-update methods. The theoretical contributions of this chapter are verified in controlled experiments using quadratic functions, and the superiority of \fedshuffle is also demonstrated in experiments training deep neural networks on standard FL benchmark problems.

\section{Future research work}

In this section, we outline a few promising directions for future work.
\begin{itemize}
    \item Firstly, although the compression operators introduced in Chapters~\ref{chapter2:c_nat}--\ref{chapter5:induced} provide provable theoretical advances in communication efficiency, there could be a discrepancy with the practical performance~\cite{dutta2019discrepancy, vogels2019powersgd} as they require expensive Gather communication primitive. Therefore, we are interested in devising new communication-efficient compressors compatible with the fast Reduce communication primitive. 
    \item Secondly, we would like to extend the optimal client sampling (introduced in Chapter~\ref{chapter6:optimal_sampling}) to account not only for the importance of the local updates but also for local client constraints, e.g., computational speed and network bandwidth. 
    \item Also, it is worth investigating  different techniques to prune deep neural networks for FL using ordered dropout (introduced in Chapter~\ref{chapter7:fjord}), e.g., depth-based pruning instead of width-based pruning.
    \item  Finally, stragglers (i.e., clients that can't compute their updates before timeout) hinder standard parallel FL training, and asynchronous aggregation is a popular practical technique to remedy this issue~\cite{huba2022papaya}. Therefore, generalizing the analysis of \fedshuffle (introduced in Chapter~\ref{chapter8:fedshuffle}) to cover asynchronous updates is an important future direction.
\end{itemize}





\begin{onehalfspacing}
\renewcommand*\bibname{\centerline{REFERENCES}} 
\addcontentsline{toc}{chapter}{References}
\newcommand{\BIBdecl}{\setlength{\itemsep}{0pt}}
\bibliographystyle{IEEEtran}
\bibliography{KAUST_Thesis/References}

\begin{thebibliography}{100}
\providecommand{\url}[1]{#1}
\csname url@samestyle\endcsname
\providecommand{\newblock}{\relax}
\providecommand{\bibinfo}[2]{#2}
\providecommand{\BIBentrySTDinterwordspacing}{\spaceskip=0pt\relax}
\providecommand{\BIBentryALTinterwordstretchfactor}{4}
\providecommand{\BIBentryALTinterwordspacing}{\spaceskip=\fontdimen2\font plus
\BIBentryALTinterwordstretchfactor\fontdimen3\font minus
  \fontdimen4\font\relax}
\providecommand{\BIBforeignlanguage}[2]{{%
\expandafter\ifx\csname l@#1\endcsname\relax
\typeout{** WARNING: IEEEtran.bst: No hyphenation pattern has been}%
\typeout{** loaded for the language `#1'. Using the pattern for}%
\typeout{** the default language instead.}%
\else
\language=\csname l@#1\endcsname
\fi
#2}}
\providecommand{\BIBdecl}{\relax}
\BIBdecl

\bibitem{fbdatacenter2018hpca}
K.~{Hazelwood} \emph{et~al.}, ``Applied machine learning at {F}acebook: {A}
  datacenter infrastructure perspective,'' in \emph{IEEE International
  Symposium on High Performance Computer Architecture (HPCA)}, 2018.

\bibitem{apple}
Apple, ``Learning with privacy at scale,'' in \emph{Differential Privacy Team
  Technical Report}, 2017.

\bibitem{gdpr}
\BIBentryALTinterwordspacing
{European Commission}, ``{GDPR}: 2018 reform of eu data protection rules.''
  [Online]. Available:
  \url{https://ec.europa.eu/commission/sites/beta-political/files/data-protection-factsheet-changes_en.pdf}
\BIBentrySTDinterwordspacing

\bibitem{qiu2021first}
X.~Qiu, T.~Parcollet, J.~Fernandez-Marques, P.~P.~B. de~Gusmao, D.~J. Beutel,
  T.~Topal, A.~Mathur, and N.~D. Lane, ``A first look into the carbon footprint
  of federated learning,'' \emph{arXiv preprint arXiv:2102.07627}, 2021.

\bibitem{deshpande2005model}
A.~Deshpande, C.~Guestrin, S.~R. Madden, J.~M. Hellerstein, and W.~Hong,
  ``Model-based approximate querying in sensor networks,'' \emph{The VLDB
  journal}, vol.~14, no.~4, pp. 417--443, 2005.

\bibitem{bonomi2012fog}
F.~Bonomi, R.~Milito, J.~Zhu, and S.~Addepalli, ``Fog computing and its role in
  the internet of things,'' in \emph{Proceedings of the first edition of the
  MCC workshop on Mobile cloud computing}, 2012, pp. 13--16.

\bibitem{konevcny2016afederated}
J.~Kone{\v{c}}n\'{y}, H.~B. McMahan, D.~Ramage, and P.~Richt{\'a}rik,
  ``Federated optimization: Distributed machine learning for on-device
  intelligence,'' \emph{arXiv preprint arXiv:1610.02527}, 2016.

\bibitem{FEDLEARN2016}
J.~Kone\v{c}n\'{y}, H.~B. McMahan, F.~X. Yu, P.~Richt\'{a}rik, A.~T. Suresh,
  and D.~Bacon, ``Federated learning: {S}trategies for improving communication
  efficiency,'' in \emph{NeurIPS Private Multi-Party Machine Learning
  Workshop}, 2016.

\bibitem{FEDOPT2016}
J.~Kone\v{c}n\'{y}, H.~B. McMahan, D.~Ramage, and P.~Richt\'{a}rik, ``Federated
  optimization: {D}istributed machine learning for on-device intelligence,''
  \emph{arXiv preprint arXiv:1610.02527}, 2016.

\bibitem{mcmahan17fedavg}
H.~B. McMahan, E.~Moore, D.~Ramage, S.~Hampson, and B.~A. y~Arcas,
  ``{Communication-efficient learning of deep networks from decentralized
  data},'' in \emph{International Conference on Artificial Intelligence and
  Statistics (AISTATS)}, 2017, pp. 1273--1282.

\bibitem{kairouz2019advances}
P.~Kairouz, H.~B. McMahan, B.~Avent, A.~Bellet, M.~Bennis, A.~N. Bhagoji,
  K.~Bonawitz, Z.~Charles, G.~Cormode, R.~Cummings \emph{et~al.}, ``Advances
  and open problems in federated learning,'' \emph{Foundations and
  Trends{\textregistered} in Machine Learning}, vol.~14, no. 1--2, pp. 1--210,
  2021.

\bibitem{hard18gboard}
A.~Hard, K.~Rao, R.~Mathews, F.~Beaufays, S.~Augenstein, H.~Eichner, C.~Kiddon,
  and D.~Ramage, ``Federated learning for mobile keyboard prediction,''
  \emph{arXiv preprint arXiv:1811.03604}, 2018.

\bibitem{gboard19emoji}
S.~Ramaswamy, R.~Mathews, K.~Rao, and F.~Beaufays, ``Federated learning for
  emoji prediction in a mobile keyboard,'' \emph{arXiv preprint
  arXiv:1906.04329}, 2019.

\bibitem{googleassistant2021}
Google, ``Your voice and audio data stays private while google assistant
  improves,'' \url{https://support.google.com/assistant/answer/10176224}, 2021.

\bibitem{apple19wwdc}
Apple, ``Designing for privacy (video and slide deck),'' Apple WWDC,
  \url{https://developer.apple.com/videos/play/wwdc2019/708}, 2019.

\bibitem{webank2020}
WeBank, ``Utilization of fate in anti money laundering through multiple
  banks,''
  \url{https://www.fedai.org/cases/utilization-of-fate-in-anti-money-laundering-through-multiple-banks/},
  2020.

\bibitem{intel2020}
Intel and Consilient, ``Intel and {C}onsilient join forces to fight financial
  fraud with ai,''
  \url{https://newsroom.intel.com/news/intel-consilient-join-forces-fight-financial-fraud-ai/},
  December 2020.

\bibitem{melloddy2020}
MELLODDY, ``Melloddy project meets its year one objective: Deployment of the
  world’s first secure platform for multi-task federated learning in drug
  discovery among 10 pharmaceutical companies,''
  \url{https://www.melloddy.eu/y1announcement}, September 2020.

\bibitem{owkin2020}
Owkin, ``Story of the 1st federated learning model at {O}wkin,''
  \url{https://owkin.com/federated-learning/federated-model/}, 2020.

\bibitem{thrun2010toward}
S.~Thrun, ``Toward robotic cars,'' \emph{Communications of the ACM}, vol.~53,
  no.~4, pp. 99--106, 2010.

\bibitem{samarakoon2018federated}
S.~Samarakoon, M.~Bennis, W.~Saad, and M.~Debbah, ``{Federated learning for
  ultra-reliable low-latency V2V communications},'' in \emph{2018 IEEE Global
  Communications Conference (GLOBECOM)}, 2018, pp. 1--7.

\bibitem{van2009multi}
C.~Van~Berkel, ``Multi-core for mobile phones,'' in \emph{Conference on Design,
  Automation and Test in Europe}, 2009.

\bibitem{huang2013depth}
J.~Huang, F.~Qian, Y.~Guo, Y.~Zhou, Q.~Xu, Z.~M. Mao, S.~Sen, and
  O.~Spatscheck, ``An in-depth study of lte: effect of network protocol and
  application behavior on performance,'' \emph{SIGCOMM Computer Communication
  Review}, vol.~43, pp. 363--374, 2013.

\bibitem{facebook2019hpca}
C.-J. {Wu} \emph{et~al.}, ``Machine learning at {F}acebook: {U}nderstanding
  inference at the edge,'' in \emph{IEEE International Symposium on High
  Performance Computer Architecture (HPCA)}, 2019.

\bibitem{ai_benchmark_2019}
A.~Ignatov, R.~Timofte, A.~Kulik, S.~Yang, K.~Wang, F.~Baum, M.~Wu, L.~Xu, and
  L.~Van~Gool, ``{AI} benchmark: {A}ll about deep learning on smartphones in
  2019,'' in \emph{International Conference on Computer Vision Workshops
  (ICCVW)}, 2019.

\bibitem{sysdesign_fl2019mlsys}
K.~Bonawitz \emph{et~al.}, ``Towards federated learning at scale: {S}ystem
  design,'' in \emph{Machine Learning and Systems (MLSys)}, 2019.

\bibitem{embench2019emdl}
M.~Almeida, S.~Laskaridis, I.~Leontiadis, S.~I. Venieris, and N.~D. Lane,
  ``{EmBench}: {Q}uantifying performance variations of deep neural networks
  across modern commodity devices,'' in \emph{The 3rd International Workshop on
  Deep Learning for Mobile Systems and Applications (EMDL)}, 2019.

\bibitem{comms_fl2020tnnls}
F.~{Sattler}, S.~{Wiedemann}, K.~R. {Müller}, and W.~{Samek}, ``Robust and
  communication-efficient federated learning from non-i.i.d. data,'' \emph{IEEE
  Transactions on Neural Networks and Learning Systems (TNNLS)}, vol.~31,
  no.~9, pp. 3400--3413, 2020.

\bibitem{smith2017federated}
V.~Smith, C.-K. Chiang, M.~Sanjabi, and A.~S. Talwalkar, ``Federated multi-task
  learning,'' in \emph{Advances in Neural Information Processing Systems
  (NeurIPS)}, 2017.

\bibitem{Hanzely2020}
F.~Hanzely and P.~Richt\'{a}rik, ``Federated learning of a mixture of global
  and local models,'' \emph{arXiv:2002.05516}, 2020.

\bibitem{hanzely2020lower}
F.~Hanzely, S.~Hanzely, S.~Horv{\'a}th, and P.~Richt{\'a}rik, ``Lower bounds
  and optimal algorithms for personalized federated learning,'' \emph{Advances
  in Neural Information Processing Systems (NeurIPS)}, vol.~33, pp. 2304--2315,
  2020.

\bibitem{gasanov2021flix}
E.~Gasanov, A.~Khaled, S.~Horv{\'a}th, and P.~Richt{\'a}rik, ``{FLIX}: {A}
  simple and communication-efficient alternative to local methods in federated
  learning,'' \emph{International Conference on Artificial Intelligence and
  Statistics (AISTATS)}, 2022.

\bibitem{PFL_universal2021}
F.~Hanzely, B.~Zhao, and M.~Kolar, ``Personalized federated learning: A unified
  framework and universal optimization techniques,'' \emph{arXiv preprint
  arXiv:2102.09743}, 2021.

\bibitem{duchi2014privacy}
J.~C. Duchi, M.~I. Jordan, and M.~J. Wainwright, ``Privacy aware learning,'' in
  \emph{Advances in Neural Information Processing Systems (NeurIPS)}, 2012.

\bibitem{dwork2014algorithmic}
C.~Dwork, A.~Roth \emph{et~al.}, ``The algorithmic foundations of differential
  privacy.'' \emph{Foundations and Trends in Theoretical Computer Science},
  vol.~9, no. 3-4, pp. 211--407, 2014.

\bibitem{carlini2018secret}
N.~Carlini, C.~Liu, J.~Kos, {\'U}.~Erlingsson, and D.~Song, ``The secret
  sharer: Measuring unintended neural network memorization \& extracting
  secrets,'' \emph{arXiv preprint arXiv:1802.08232}, 2018.

\bibitem{mcmahan18learning}
B.~McMahan, D.~Ramage, K.~Talwar, and L.~Zhang, ``Learning differentially
  private recurrent language models,'' in \emph{International Conference on
  Learning Representations (ICLR)}, 2018.

\bibitem{chaudhuri2011differentially}
K.~Chaudhuri, C.~Monteleoni, and A.~D. Sarwate, ``Differentially private
  empirical risk minimization.'' \emph{Journal of Machine Learning Research
  (JMLR)}, vol.~12, no.~3, 2011.

\bibitem{jain2014near}
P.~Jain and A.~G. Thakurta, ``({N}ear) dimension independent risk bounds for
  differentially private learning,'' in \emph{International Conference on
  Machine Learning (ICML)}, 2014, pp. 476--484.

\bibitem{geyer2017differentially}
R.~C. Geyer, T.~Klein, and M.~Nabi, ``{Differentially private federated
  learning: A client level perspective},'' \emph{arXiv preprint
  arXiv:1712.07557}, 2017.

\bibitem{du2001secure}
W.~Du and M.~J. Atallah, ``Secure multi-party computation problems and their
  applications: a review and open problems,'' in \emph{Workshop on New Security
  Paradigms}, 2001.

\bibitem{goryczka2015comprehensive}
S.~Goryczka and L.~Xiong, ``A comprehensive comparison of multiparty secure
  additions with differential privacy,'' \emph{IEEE Transactions on Dependable
  and Secure Computing}, vol.~14, pp. 463--477, 2015.

\bibitem{bonawitz2017practical}
K.~Bonawitz, V.~Ivanov, B.~Kreuter, A.~Marcedone, H.~B. McMahan, S.~Patel,
  D.~Ramage, A.~Segal, and K.~Seth, ``Practical secure aggregation for
  privacy-preserving machine learning,'' in \emph{Proceedings of the 2017 ACM
  SIGSAC Conference on Computer and Communications Security}, 2017, pp.
  1175--1191.

\bibitem{so2020turbo}
J.~So, B.~G{\"u}ler, and A.~S. Avestimehr, ``Turbo-aggregate: Breaking the
  quadratic aggregation barrier in secure federated learning,'' \emph{IEEE
  Journal on Selected Areas in Information Theory}, vol.~2, no.~1, pp.
  479--489, 2021.

\bibitem{hardy2017private}
S.~Hardy, W.~Henecka, H.~Ivey-Law, R.~Nock, G.~Patrini, G.~Smith, and
  B.~Thorne, ``Private federated learning on vertically partitioned data via
  entity resolution and additively homomorphic encryption,'' \emph{arXiv
  preprint arXiv:1711.10677}, 2017.

\bibitem{mo2020darknetz}
F.~Mo, A.~S. Shamsabadi, K.~Katevas, S.~Demetriou, I.~Leontiadis, A.~Cavallaro,
  and H.~Haddadi, ``{DarkneTZ}: {T}owards model privacy at the edge using
  trusted execution environments,'' in \emph{International Conference on Mobile
  Systems, Applications, and Services (MobiSys)}, 2020, pp. 161--174.

\bibitem{mo2021ppfl}
F.~Mo, H.~Haddadi, K.~Katevas, E.~Marin, D.~Perino, and N.~Kourtellis,
  ``{PPFL}: {P}rivacy-preserving federated learning with trusted execution
  environments,'' in \emph{International Conference on Mobile Systems,
  Applications, and Services (MobiSys)}, 2021, pp. 94--108.

\bibitem{nesterov2013introductory}
Y.~Nesterov, \emph{Introductory lectures on convex optimization: A basic
  course}.\hskip 1em plus 0.5em minus 0.4em\relax Springer Science \& Business
  Media, 2013, vol.~87.

\bibitem{proxskip}
K.~Mishchenko, G.~Malinovsky, S.~Stich, and P.~Richt{\'a}rik, ``{ProxSkip}:
  {Y}es! {L}ocal gradient steps provably lead to communication acceleration!
  {F}inally!'' \emph{arXiv preprint arXiv:2202.09357}, 2022.

\bibitem{li2020federated}
T.~Li, A.~K. Sahu, A.~Talwalkar, and V.~Smith, ``{Federated learning:
  Challenges, methods, and future directions},'' \emph{IEEE Signal Processing
  Magazine}, vol.~37, no.~3, pp. 50--60, 2020.

\bibitem{wang2021field}
J.~Wang, Z.~Charles, Z.~Xu, G.~Joshi, H.~B. McMahan, M.~Al-Shedivat, G.~Andrew,
  S.~Avestimehr, K.~Daly, D.~Data \emph{et~al.}, ``A field guide to federated
  optimization,'' \emph{arXiv preprint arXiv:2107.06917}, 2021.

\bibitem{ding2022federated}
J.~Ding, E.~Tramel, A.~K. Sahu, S.~Wu, S.~Avestimehr, and T.~Zhang, ``Federated
  learning challenges and opportunities: An outlook,'' \emph{arXiv preprint
  arXiv:2202.00807}, 2022.

\bibitem{Cnat}
S.~Horv\'{a}th, C.-Y. Ho, \v{L}udov\'{i}t Horv\'{a}th, A.~N. Sahu, M.~Canini,
  and P.~Richt\'{a}rik, ``Natural compression for distributed deep learning,''
  \emph{arXiv preprint arXiv:1905.10988}, 2019.

\bibitem{diana2}
S.~Horv\'{a}th, D.~Kovalev, K.~Mishchenko, S.~Stich, and P.~Richt\'{a}rik,
  ``Stochastic distributed learning with gradient quantization and variance
  reduction,'' \emph{arXiv preprint arXiv:1904.05115}, 2019.

\bibitem{beznosikov2020biased}
A.~Beznosikov, S.~Horv{\'a}th, P.~Richt{\'a}rik, and M.~Safaryan, ``On biased
  compression for distributed learning,'' \emph{arXiv preprint
  arXiv:2002.12410}, 2020.

\bibitem{horvath2021a}
S.~Horv{\'a}th and P.~Richt\'{a}rik, ``A better alternative to error feedback
  for communication-efficient distributed learning,'' in \emph{International
  Conference on Learning Representations (ICLR)}, 2021.

\bibitem{chen2020optimal}
W.~Chen, S.~Horv\'{a}th, and P.~Richt\'{a}rik, ``Optimal client sampling for
  federated learning,'' \emph{arXiv preprint arXiv:2010.13723}, 2020.

\bibitem{horvath2021fjord}
S.~Horv\'{a}th, S.~Laskaridis, M.~Almeida, I.~Leontiadis, S.~Venieris, and
  N.~Lane, ``{FjORD}: Fair and accurate federated learning under heterogeneous
  targets with ordered dropout,'' \emph{Advances in Neural Information
  Processing Systems (NeurIPS)}, vol.~34, 2021.

\bibitem{cutkosky2019momentum}
A.~Cutkosky and F.~Orabona, ``Momentum-based variance reduction in non-convex
  {SGD},'' in \emph{Advances in Neural Information Processing Systems
  (NeurIPS)}, 2019, pp. 15\,236--15\,245.

\bibitem{fednova2020neurips}
J.~Wang, Q.~Liu, H.~Liang, G.~Joshi, and H.~V. Poor, ``Tackling the objective
  inconsistency problem in heterogeneous federated optimization,''
  \emph{Advances in Neural Information Processing Systems (NeurIPS)}, 2020.

\bibitem{fedshuffle}
S.~Horv{\'a}th, M.~Sanjabi, L.~Xiao, P.~Richt{\'a}rik, and M.~Rabbat,
  ``Fedshuffle: Recipes for better use of local work in federated learning,''
  \emph{arXiv preprint arXiv:2204.13169}, 2022.

\bibitem{horvath2018nonconvex}
S.~Horv{\'a}th and P.~Richt\'{a}rik, ``{Nonconvex Variance Reduced Optimization
  with Arbitrary Sampling},'' in \emph{International Conference on Machine
  Learning (ICML)}, 2019, pp. 2781--2789.

\bibitem{Kovalev2019:svrg}
D.~Kovalev, S.~Horv\'{a}th, and P.~Richt\'{a}rik, ``Don’t jump through hoops
  and remove those loops: {SVRG} and {K}atyusha are better without the outer
  loop,'' in \emph{International Conference on Algorithmic Learning Theory
  (ALT)}, 2020.

\bibitem{horvath2020adaptivity}
S.~Horv{\'a}th, L.~Lei, P.~Richt{\'a}rik, and M.~I. Jordan, ``Adaptivity of
  stochastic gradient methods for nonconvex optimization,'' \emph{SIAM Journal
  on Mathematics of Data Science (SIMODS)}, vol.~4, no.~2, pp. 634--648, 2022.

\bibitem{horvath2021hyperparameter}
S.~Horv{\'a}th, A.~Klein, P.~Richt{\'a}rik, and C.~Archambeau, ``Hyperparameter
  transfer learning with adaptive complexity,'' in \emph{International
  Conference on Artificial Intelligence and Statistics (AISTATS)}, 2021, pp.
  1378--1386.

\bibitem{burlachenko2021fl_pytorch}
K.~Burlachenko, S.~Horv{\'a}th, and P.~Richt{\'a}rik, ``Fl\_pytorch:
  optimization research simulator for federated learning,'' in
  \emph{Proceedings of the 2nd ACM International Workshop on Distributed
  Machine Learning}, 2021, pp. 1--7.

\bibitem{pogran2021long}
E.~Pogran, A.~El-Razek, L.~Gargiulo, V.~Weihs, C.~Kaufmann, S.~Horv{\'a}th,
  A.~Geppert, M.~N{\"u}rnberg, E.~Wessely, P.~Smetana \emph{et~al.},
  ``Long-term outcome in patients with takotsubo syndrome,'' \emph{Wiener
  klinische Wochenschrift}, pp. 1--8, 2021.

\bibitem{cordonnier2018convex}
J.-B. Cordonnier, ``Convex optimization using sparsified stochastic gradient
  descent with memory,'' École Polytechnique Fédérale de Lausanne, Tech.
  Rep., 2018.

\bibitem{stich2018sparsified}
S.~U. Stich, J.-B. Cordonnier, and M.~Jaggi, ``Sparsified {SGD} with memory,''
  in \emph{Advances in Neural Information Processing Systems (NeurIPS)}, 2018,
  pp. 4447--4458.

\bibitem{koloskova2019decentralized}
A.~Koloskova, S.~Stich, and M.~Jaggi, ``Decentralized stochastic optimization
  and gossip algorithms with compressed communication,'' in \emph{International
  Conference on Machine Learning (ICML)}, 2019, pp. 3478--3487.

\bibitem{richtarik2016optimal}
P.~Richt{\'a}rik and M.~Tak{\'a}{\v{c}}, ``On optimal probabilities in
  stochastic coordinate descent methods,'' \emph{Optimization Letters},
  vol.~10, no.~6, pp. 1233--1243, 2016.

\bibitem{csiba2018importance}
D.~Csiba and P.~Richt{\'a}rik, ``Importance sampling for minibatches,''
  \emph{Journal of Machine Learning Research (JMLR)}, vol.~19, no.~1, pp.
  962--982, 2018.

\bibitem{hanzely2019accelerated}
F.~Hanzely and P.~Richt{\'a}rik, ``Accelerated coordinate descent with
  arbitrary sampling and best rates for minibatches,'' in \emph{International
  Conference on Artificial Intelligence and Statistics (AISTATS)}, 2019, pp.
  304--312.

\bibitem{reisizadeh2020fedpaq}
A.~Reisizadeh, A.~Mokhtari, H.~Hassani, A.~Jadbabaie, and R.~Pedarsani,
  ``Fedpaq: {A} communication-efficient federated learning method with periodic
  averaging and quantization,'' in \emph{International Conference on Artificial
  Intelligence and Statistics (AISTATS)}, 2020, pp. 2021--2031.

\bibitem{resnet}
K.~He, X.~Zhang, S.~Ren, and J.~Sun, ``Deep residual learning for image
  recognition,'' in \emph{Conference on Computer Vision and Pattern Recognition
  (CVPR)}, 2016, pp. 770--778.

\bibitem{bekkerman2011scaling}
R.~Bekkerman, M.~Bilenko, and J.~Langford, \emph{Scaling up machine learning:
  {P}arallel and distributed approaches}.\hskip 1em plus 0.5em minus
  0.4em\relax Cambridge University Press, 2011.

\bibitem{recht2011hogwild}
B.~Recht, C.~Re, S.~Wright, and F.~Niu, ``Hogwild: A lock-free approach to
  parallelizing stochastic gradient descent,'' in \emph{Advances in Neural
  Information Processing Systems (NeurIPS)}, 2011.

\bibitem{SGD}
H.~Robbins and S.~Monro, ``A stochastic approximation method,'' \emph{The
  Annals of Mathematical Statistics}, pp. 400--407, 1951.

\bibitem{Goyal2017:large}
P.~Goyal, P.~Doll{\'{a}}r, R.~B. Girshick, P.~Noordhuis, L.~Wesolowski,
  A.~Kyrola, A.~Tulloch, Y.~Jia, and K.~He, ``Accurate, large minibatch {SGD:}
  training {ImageNet} in 1 hour,'' \emph{CoRR}, vol. abs/1706.02677, 2017.

\bibitem{johnson2013accelerating}
R.~Johnson and T.~Zhang, ``Accelerating stochastic gradient descent using
  predictive variance reduction,'' in \emph{Advances in Neural Information
  Processing Systems (NeurIPS)}, 2013, pp. 315--323.

\bibitem{1bit}
F.~Seide, H.~Fu, J.~Droppo, G.~Li, and D.~Yu, ``1-bit stochastic gradient
  descent and application to data-parallel distributed training of speech
  {DNNs},'' in \emph{Interspeech 2014}, September 2014.

\bibitem{qsgd2017neurips}
D.~Alistarh, D.~Grubic, J.~Li, R.~Tomioka, and M.~Vojnovic, ``{QSGD}:
  {C}ommunication-efficient {SGD} via gradient quantization and encoding,'' in
  \emph{Advances in Neural Information Processing Systems (NeurIPS)}, 2017.

\bibitem{zhang2016zipml}
H.~Zhang, J.~Li, K.~Kara, D.~Alistarh, J.~Liu, and C.~Zhang, ``{ZipML}:
  Training linear models with end-to-end low precision, and a little bit of
  deep learning,'' in \emph{International Conference on Machine Learning
  (ICML)}, 2017.

\bibitem{deepgradcompress2018iclr}
Y.~Lin, S.~Han, H.~Mao, Y.~Wang, and B.~Dally, ``Deep gradient compression:
  {R}educing the communication bandwidth for distributed training,'' in
  \emph{International Conference on Learning Representations (ICLR)}, 2018.

\bibitem{lim20183lc}
H.~Lim, D.~G. Andersen, and M.~Kaminsky, ``{3LC: Lightweight and effective
  traffic compression for distributed machine learning},'' \emph{Machine
  Learning and Systems (MLSys)}, vol.~1, pp. 53--64, 2019.

\bibitem{Shamir2014:approxnewton}
O.~Shamir, N.~Srebro, and T.~Zhang, ``Communication-efficient distributed
  optimization using an approximate {N}ewton-type method,'' in
  \emph{International Conference on Machine Learning (ICML)}, vol.~32, 2014,
  pp. 1000--1008.

\bibitem{Hydra}
P.~Richt{\'a}rik and M.~Tak{\'a}{\v{c}}, ``Distributed coordinate descent
  method for learning with big data,'' \emph{Journal of Machine Learning
  Research (JMLR)}, vol.~17, no.~75, pp. 1--25, 2016.

\bibitem{Reddi:2016aide}
S.~J. Reddi, J.~Kone\v{c}n{\'y}, P.~Richt{\'a}rik, B.~P{\'o}czos, and A.~J.
  Smola, ``{AIDE}: Fast and communication efficient distributed optimization,''
  \emph{CoRR}, vol. abs/1608.06879, 2016.

\bibitem{Mann2009:parallelSGD}
R.~McDonald, M.~Mohri, N.~Silberman, D.~Walker, and G.~S. Mann, ``Efficient
  large-scale distributed training of conditional maximum entropy models,'' in
  \emph{Advances in Neural Information Processing Systems (NeurIPS)}, 2009, pp.
  1231--1239.

\bibitem{Zinkevich2010:parallelSGD}
M.~Zinkevich, M.~Weimer, L.~Li, and A.~J. Smola, ``Parallelized stochastic
  gradient descent,'' in \emph{Advances in Neural Information Processing
  Systems (NeurIPS)}, 2010, pp. 2595--2603.

\bibitem{local_SGD_stich_18}
S.~U. Stich, ``Local {SGD} converges fast and communicates little,'' in
  \emph{International Conference on Learning Representations (ICLR)}, 2019.

\bibitem{terngrad}
W.~Wen, C.~Xu, F.~Yan, C.~Wu, Y.~Wang, Y.~Chen, and H.~Li, ``{Terngrad: Ternary
  gradients to reduce communication in distributed deep learning},'' in
  \emph{Advances in Neural Information Processing Systems (NeurIPS)}, 2017, pp.
  1509--1519.

\bibitem{tonko}
J.~Wangni, J.~Wang, J.~Liu, and T.~Zhang, ``{Gradient sparsification for
  communication-efficient distributed optimization},'' in \emph{Advances in
  Neural Information Processing Systems (NeurIPS)}, 2018, pp. 1306--1316.

\bibitem{hubara2017quantized}
I.~Hubara, M.~Courbariaux, D.~Soudry, R.~El-Yaniv, and Y.~Bengio, ``Quantized
  neural networks: Training neural networks with low precision weights and
  activations,'' \emph{Journal of Machine Learning Research (JMLR)}, vol.~18,
  no.~1, pp. 6869--6898, 2017.

\bibitem{gupta2015deep}
S.~Gupta, A.~Agrawal, K.~Gopalakrishnan, and P.~Narayanan, ``Deep learning with
  limited numerical precision,'' in \emph{International Conference on Machine
  Learning (ICML)}, 2015, pp. 1737--1746.

\bibitem{Na2017:limitedprecision}
T.~Na, J.~H. Ko, J.~Kung, and S.~Mukhopadhyay, ``On-chip training of recurrent
  neural networks with limited numerical precision,'' in \emph{2017
  International Joint Conference on Neural Networks (IJCNN)}, May 2017, pp.
  3716--3723.

\bibitem{Suresh2017}
A.~T. Suresh, F.~X. Yu, S.~Kumar, and H.~B. McMahan, ``Distributed mean
  estimation with limited communication,'' in \emph{International Conference on
  Machine Learning (ICML)}, 2017.

\bibitem{RDME}
J.~Kone\v{c}n\'{y} and P.~Richt\'{a}rik, ``Randomized distributed mean
  estimation: accuracy vs communication,'' \emph{Frontiers in Applied
  Mathematics and Statistics}, vol.~4, no.~62, pp. 1--11, 2018.

\bibitem{alistarh2018sparse}
D.~Alistarh, T.~Hoefler, M.~Johansson, N.~Konstantinov, S.~Khirirat, and
  C.~Renggli, ``The convergence of sparsified gradient methods,'' in
  \emph{Advances in Neural Information Processing Systems (NeurIPS)}, vol.~31,
  2018, pp. 5973--5983.

\bibitem{Hydra2}
O.~Fercoq, Z.~Qu, P.~Richt\'{a}rik, and M.~Tak{\'a}{\v{c}}, ``Fast distributed
  coordinate descent for minimizing non-strongly convex losses,'' \emph{IEEE
  International Workshop on Machine Learning for Signal Processing}, 2014.

\bibitem{mishchenko2019distributed}
K.~Mishchenko, E.~Gorbunov, M.~Tak{\'a}{\v{c}}, and P.~Richt{\'a}rik,
  ``Distributed learning with compressed gradient differences,'' \emph{arXiv
  preprint arXiv:1901.09269}, 2019.

\bibitem{elia2001stabilization}
N.~Elia and S.~K. Mitter, ``Stabilization of linear systems with limited
  information,'' \emph{IEEE transactions on Automatic Control}, vol.~46, no.~9,
  pp. 1384--1400, 2001.

\bibitem{sun2011scalar}
J.~Z. Sun and V.~K. Goyal, ``Scalar quantization for relative error,'' in
  \emph{2011 Data Compression Conference}.\hskip 1em plus 0.5em minus
  0.4em\relax IEEE, 2011, pp. 293--302.

\bibitem{sun2012framework}
J.~Z. Sun, G.~I. Wang, V.~K. Goyal, and L.~R. Varshney, ``A framework for
  bayesian optimality of psychophysical laws,'' \emph{Journal of Mathematical
  Psychology}, vol.~56, no.~6, pp. 495--501, 2012.

\bibitem{switchML}
A.~Sapio, M.~Canini, C.~Ho, J.~Nelson, P.~Kalnis, C.~Kim, A.~Krishnamurthy,
  M.~Moshref, D.~R.~K. Ports, and P.~Richt\'{a}rik, ``Scaling distributed
  machine learning with in-network aggregation,'' in \emph{The 18th USENIX
  Symposium on Networked Systems Design and Implementation}, 2021.

\bibitem{khirirat2018distributed}
S.~Khirirat, H.~R. Feyzmahdavian, and M.~Johansson, ``{Distributed learning
  with compressed gradients},'' \emph{arXiv preprint arXiv:1806.06573}, 2018.

\bibitem{goodall1951television}
W.~Goodall, ``Television by pulse code modulation,'' \emph{Bell System
  Technical Journal}, vol.~30, no.~1, pp. 33--49, 1951.

\bibitem{roberts1962picture}
L.~Roberts, ``Picture coding using pseudo-random noise,'' \emph{IRE
  Transactions on Information Theory}, vol.~8, no.~2, pp. 145--154, 1962.

\bibitem{bubeck2015convex}
S.~Bubeck \emph{et~al.}, ``Convex optimization: Algorithms and complexity,''
  \emph{Foundations and Trends{\textregistered} in Machine Learning}, vol.~8,
  no. 3-4, pp. 231--357, 2015.

\bibitem{ghadimi2013stochastic}
S.~Ghadimi and G.~Lan, ``Stochastic first-and zeroth-order methods for
  nonconvex stochastic programming,'' \emph{SIAM Journal on Optimization},
  vol.~23, no.~4, pp. 2341--2368, 2013.

\bibitem{DoubleSqueeze2019}
H.~Tang, C.~Yu, X.~Lian, T.~Zhang, and J.~Liu, ``{D}ouble{S}queeze: {P}arallel
  stochastic gradient descent with double-pass error-compensated compression,''
  in \emph{International Conference on Machine Learning (ICML)}, 2019.

\bibitem{zheng2019communication}
S.~Zheng, Z.~Huang, and J.~Kwok, ``Communication-efficient distributed
  blockwise momentum {SGD} with error-feedback,'' in \emph{Advances in Neural
  Information Processing Systems (NeurIPS)}, 2019, pp. 11\,450--11\,460.

\bibitem{pytorch}
A.~Paszke \emph{et~al.}, ``Pytorch: An imperative style, high-performance deep
  learning library,'' in \emph{Advances in Neural Information Processing
  Systems (NeurIPS)}, 2019, pp. 8024--8035.

\bibitem{lecun1998gradient}
Y.~LeCun, L.~Bottou, Y.~Bengio, and P.~Haffner, ``Gradient-based learning
  applied to document recognition,'' \emph{Proceedings of the IEEE}, vol.~86,
  no.~11, pp. 2278--2324, 1998.

\bibitem{cifar}
A.~Krizhevsky, G.~Hinton \emph{et~al.}, ``Learning multiple layers of features
  from tiny images,'' 2009.

\bibitem{simonyan2014very}
K.~Simonyan and A.~Zisserman, ``Very deep convolutional networks for
  large-scale image recognition,'' \emph{International Conference on Learning
  Representations (ICLR)}, 2015.

\bibitem{you2017imagenet}
Y.~You, Z.~Zhang, C.-J. Hsieh, J.~Demmel, and K.~Keutzer, ``Imagenet training
  in minutes,'' in \emph{Proceedings of the 47th International Conference on
  Parallel Processing}, 2018, pp. 1--10.

\bibitem{strom2015scalable}
N.~Strom, ``{Scalable distributed DNN training using commodity GPU cloud
  computing},'' in \emph{Sixteenth Annual Conference of the International
  Speech Communication Association}, 2015.

\bibitem{grishchenko2018asynchronous}
D.~Grishchenko, F.~Iutzeler, J.~Malick, and M.-R. Amini, ``Distributed learning
  with sparse communications by identification,'' \emph{SIAM Journal on
  Mathematics of Data Science (SIMODS)}, vol.~3, no.~2, pp. 715--735, 2021.

\bibitem{mishchenko201999}
K.~Mishchenko, F.~Hanzely, and P.~Richt{\'a}rik, ``99\% of worker-master
  communication in distributed optimization is not needed,'' in
  \emph{Conference on Uncertainty in Artificial Intelligence}, 2020, pp.
  979--988.

\bibitem{LeRoux:2012sag}
N.~L. Roux, M.~Schmidt, and F.~R. Bach, ``A stochastic gradient method with an
  exponential convergence rate for finite training sets,'' in \emph{Advances in
  Neural Information Processing Systems (NeurIPS)}, 2012, pp. 2663--2671.

\bibitem{Defazio2014:saga}
A.~Defazio, F.~Bach, and S.~Lacoste-Julien, ``{SAGA}: {A} fast incremental
  gradient method with support for non-strongly convex composite objectives,''
  in \emph{Advances in Neural Information Processing Systems (NeurIPS)}, 2014,
  pp. 1646--1654.

\bibitem{Dryden2016:topk}
N.~Dryden, T.~Moon, S.~A. Jacobs, and B.~V. Essen, ``Communication quantization
  for data-parallel training of deep neural networks,'' in \emph{2016 2nd
  Workshop on Machine Learning in HPC Environments (MLHPC)}, Nov 2016, pp.
  1--8.

\bibitem{aji2017sparse}
A.~F. Aji and K.~Heafield, ``Sparse communication for distributed gradient
  descent,'' \emph{Proceedings of the 2017 Conference on Empirical Methods in
  Natural Language Processing}, 2017.

\bibitem{errorSGD}
J.~Wu, W.~Huang, J.~Huang, and T.~Zhang, ``Error compensated quantized {SGD}
  and its applications to large-scale distributed optimization,'' in
  \emph{International Conference on Machine Learning (ICML)}, 2018.

\bibitem{Dekel2012:minibatch}
O.~Dekel, R.~Gilad-Bachrach, O.~Shamir, and L.~Xiao, ``Optimal distributed
  online prediction using mini-batches,'' \emph{Journal of Machine Learning
  Research (JMLR)}, vol.~13, no.~1, pp. 165--202, Jan. 2012.

\bibitem{pegasos2}
M.~Tak{\'a}\v{c}, A.~Bijral, P.~Richt{\'a}rik, and N.~Srebro, ``Mini-batch
  primal and dual methods for {SVM}s,'' in \emph{International Conference on
  Machine Learning (ICML)}, 2013, pp. 537--552.

\bibitem{Nemirovski-Juditsky-Lan-Shapiro-2009}
A.~Nemirovski, A.~Juditsky, G.~Lan, and A.~Shapiro, ``Robust stochastic
  approximation approach to stochastic programming,'' \emph{SIAM Journal on
  Optimization}, vol.~19, no.~4, pp. 1574--1609, 2009.

\bibitem{ShalevShwartz:2013sdca}
S.~Shalev-Shwartz and T.~Zhang, ``Stochastic dual coordinate ascent methods for
  regularized loss,'' \emph{Journal of Machine Learning Research (JMLR)},
  vol.~14, no.~1, pp. 567--599, Feb. 2013.

\bibitem{qu2015quartz}
Z.~Qu, P.~Richt{\'a}rik, and T.~Zhang, ``{Quartz: Randomized dual coordinate
  ascent with arbitrary sampling},'' in \emph{Advances in Neural Information
  Processing Systems (NeurIPS)}, 2015.

\bibitem{SDCA-dual-free}
S.~Shalev-Shwartz, ``{SDCA} without duality, regularization, and individual
  convexity,'' in \emph{International Conference on Machine Learning (ICML)},
  vol.~48, 2016, pp. 747--754.

\bibitem{SDNA}
Z.~Qu, P.~Richt\'{a}rik, M.~Tak\'{a}\v{c}, and O.~Fercoq, ``{SDNA:}
  {S}tochastic dual {N}ewton ascent for empirical risk minimization,'' in
  \emph{International Conference on Machine Learning (ICML)}, 2016, pp.
  1823--1832.

\bibitem{dfSDCA}
D.~Csiba and P.~Richt\'{a}rik, ``Primal method for {ERM} with flexible
  mini-batching schemes and non-convex losses,'' \emph{arXiv:1506.02227}, 2015.

\bibitem{ADASDCA}
D.~Csiba, Z.~Qu, and P.~Richt\'{a}rik, ``Stochastic dual coordinate ascent with
  adaptive probabilities,'' in \emph{International Conference on Machine
  Learning (ICML)}, 2015, pp. 674--683.

\bibitem{nguyen2017sarah}
L.~M. Nguyen, J.~Liu, K.~Scheinberg, and M.~Tak{\'a}{\v{c}}, ``{SARAH}: A novel
  method for machine learning problems using stochastic recursive gradient,''
  in \emph{International Conference on Machine Learning (ICML)}, 2017, pp.
  2613--2621.

\bibitem{JacSketch}
R.~M. Gower, P.~Richt{\'a}rik, and F.~Bach, ``{Stochastic quasi-gradient
  methods: Variance reduction via Jacobian sketching},'' \emph{Mathematical
  Programming}, vol. 188, no.~1, pp. 135--192, 2021.

\bibitem{zhou2018direct}
K.~Zhou, Q.~Ding, F.~Shang, J.~Cheng, D.~Li, and Z.-Q. Luo, ``{Direct
  acceleration of SAGA using sampled negative momentum},'' in
  \emph{International Conference on Artificial Intelligence and Statistics
  (AISTATS)}, 2019, pp. 1602--1610.

\bibitem{Bottou2010:sgd}
L.~Bottou, ``Large-scale machine learning with stochastic gradient descent,''
  in \emph{Proceedings of COMPSTAT'2010}.\hskip 1em plus 0.5em minus
  0.4em\relax Heidelberg: Physica-Verlag HD, 2010, pp. 177--186.

\bibitem{kuenstner2017:svrg}
F.~K\"{u}nstner, ``Fully quantized distributed gradient descent,'' Semester
  Project, ({Adv}: {S. U. Stich, M. Jaggi}), EPFL, 2017.

\bibitem{Lei2017:singlepass}
L.~Lei and M.~Jordan, ``{Less than a single pass: Stochastically controlled
  stochastic gradient},'' in \emph{International Conference on Artificial
  Intelligence and Statistics (AISTATS)}, vol.~54, 2017, pp. 148--156.

\bibitem{Hannah2018:span}
R.~Hannah, Y.~Liu, D.~O'Connor, and W.~Yin, ``Breaking the span assumption
  yields fast finite-sum minimization,'' in \emph{Advances in Neural
  Information Processing Systems (NeurIPS)}, 2018, pp. 2318--2327.

\bibitem{Hofmann2015:saga}
T.~Hofmann, A.~Lucchi, S.~Lacoste-Julien, and B.~McWilliams, ``Variance reduced
  stochastic gradient descent with neighbors,'' in \emph{Advances in Neural
  Information Processing Systems (NeurIPS)}, 2015, pp. 2305--2313.

\bibitem{Raj2018:ksvrg}
A.~Raj and S.~U. Stich, ``{SVRG} meets {SAGA}: k-{SVRG}—a tale of limited
  memory,'' \emph{arXiv preprint arXiv:1805.00982}, 2018.

\bibitem{cocoa}
M.~Jaggi, V.~Smith, M.~Tak{\'a}\v{c}, J.~Terhorst, S.~Krishnan, T.~Hofmann, and
  M.~I. Jordan, ``Communication-efficient distributed dual coordinate ascent,''
  in \emph{Advances in Neural Information Processing Systems (NeurIPS)}, 2014.

\bibitem{COCOA+}
C.~Ma, V.~Smith, M.~Jaggi, M.~I. Jordan, P.~Richt\'{a}rik, and
  M.~Tak\'{a}\v{c}, ``Adding vs. averaging in distributed primal-dual
  optimization,'' in \emph{International Conference on Machine Learning
  (ICML)}, 2015.

\bibitem{COCOA+journal}
C.~Ma, J.~Kone\v{c}n\'{y}, M.~Jaggi, V.~Smith, M.~I. Jordan, P.~Richt\'{a}rik,
  and M.~Tak\'{a}\v{c}, ``Distributed optimization with arbitrary local
  solvers,'' \emph{Optimization Methods and Software}, vol.~32, no.~4, pp.
  813--848, 2017.

\bibitem{Li2019:decentralized}
X.~Li, W.~Yang, S.~Wang, and Z.~Zhang, ``Communication efficient decentralized
  training with multiple local updates,'' \emph{CoRR}, vol. abs/1910.09126,
  2019.

\bibitem{CC01a}
C.-C. Chang and C.-J. Lin, ``{LIBSVM}: {A} library for support vector
  machines,'' \emph{ACM Transactions on Intelligent Systems and Technology},
  vol.~2, pp. 27:1--27:27, 2011, software available at
  \url{http://www.csie.ntu.edu.tw/~cjlin/libsvm}.

\bibitem{dalcin2011parallel}
L.~D. Dalcin, R.~R. Paz, P.~A. Kler, and A.~Cosimo, ``Parallel distributed
  computing using {P}ython,'' \emph{Advances in Water Resources}, vol.~34,
  no.~9, pp. 1124--1139, 2011.

\bibitem{vaswani2018fast}
S.~Vaswani, F.~Bach, and M.~Schmidt, ``Fast and faster convergence of {SGD} for
  over-parameterized models (and an accelerated perceptron),'' in
  \emph{International Conference on Artificial Intelligence and Statistics
  (AISTATS)}, 2019.

\bibitem{Brown2020fewshot}
T.~Brown, B.~Mann, N.~Ryder, M.~Subbiah, J.~D. Kaplan, P.~Dhariwal,
  A.~Neelakantan, P.~Shyam, G.~Sastry, A.~Askell \emph{et~al.}, ``Language
  models are few-shot learners,'' \emph{Advances in Neural Information
  Processing Systems (NeurIPS)}, vol.~33, pp. 1877--1901, 2020.

\bibitem{beck2009fista}
A.~Beck and M.~Teboulle, ``A fast iterative shrinkage-thresholding algorithm
  for linear inverse problems,'' \emph{SIAM Journal on Imaging Sciences},
  vol.~2, no.~1, pp. 183--202, 2009.

\bibitem{allen2017katyusha}
Z.~Allen-Zhu, ``Katyusha: The first direct acceleration of stochastic gradient
  methods,'' in \emph{Proceedings of the 49th Annual ACM SIGACT Symposium on
  Theory of Computing}.\hskip 1em plus 0.5em minus 0.4em\relax ACM, 2017, pp.
  1200--1205.

\bibitem{bayoumi2020tighter}
A.~Khaled, K.~Mishchenko, and P.~Richt\'{a}rik, ``Tighter theory for local
  {SGD} on identical and heterogeneous data,'' in \emph{International
  Conference on Artificial Intelligence and Statistics (AISTATS)}, 2020, pp.
  4519--4529.

\bibitem{karimireddy2020scaffold}
S.~P. Karimireddy, S.~Kale, M.~Mohri, S.~Reddi, S.~Stich, and A.~T. Suresh,
  ``{SCAFFOLD: Stochastic controlled averaging for federated learning},'' in
  \emph{International Conference on Machine Learning (ICML)}, 2020, pp.
  5132--5143.

\bibitem{sign_descent_2019}
M.~Safaryan and P.~Richt\'{a}rik, ``Stochastic sign descent methods: {N}ew
  algorithms and better theory,'' in \emph{International Conference on Machine
  Learning (ICML)}, 2021.

\bibitem{stich2020error}
S.~U. Stich and S.~P. Karimireddy, ``The error-feedback framework: {B}etter
  rates for {SGD} with delayed gradients and compressed updates,''
  \emph{Journal of Machine Learning Research (JMLR)}, vol.~21, pp. 1--36, 2020.

\bibitem{vogels2019powersgd}
T.~Vogels, S.~P. Karimireddy, and M.~Jaggi, ``Power{SGD}: {P}ractical low-rank
  gradient compression for distributed optimization,'' in \emph{Advances in
  Neural Information Processing Systems (NeurIPS)}, 2019.

\bibitem{sun}
H.~Sun, Y.~Shao, J.~Jiang, B.~Cui, K.~Lei, Y.~Xu, and J.~Wang, ``Sparse
  gradient compression for distributed {SGD},'' in \emph{Database Systems for
  Advanced Applications}.\hskip 1em plus 0.5em minus 0.4em\relax Cham: Springer
  International Publishing, 2019, pp. 139--155.

\bibitem{zhao}
S.-Y. Zhao, Y.~Xie, H.~Gao, and W.-J. Li, ``Global momentum compression for
  sparse communication in distributed {SGD},'' \emph{arXiv preprint
  arXiv:1905.12948}, 2019.

\bibitem{karimireddy2019error}
S.~P. Karimireddy, Q.~Rebjock, S.~U. Stich, and M.~Jaggi, ``Error feedback
  fixes {S}ign{SGD} and other gradient compression schemes,'' in
  \emph{International Conference on Machine Learning (ICML)}, vol.~97, 2019,
  pp. 3252--3261.

\bibitem{biased2020}
A.~Beznosikov, S.~Horv\'{a}th, P.~Richt\'{a}rik, and M.~Safaryan, ``On biased
  compression for distributed learning,'' \emph{arXiv preprint
  arXiv:2002.12410}, 2020.

\bibitem{AjallStich2021biased}
A.~Ajalloeian and S.~U. Stich, ``On the convergence of {SGD} with biased
  gradients,'' \emph{arXiv preprint arXiv:2008.00051}, 2021.

\bibitem{Gorbunov2020EF-SGD}
E.~Gorbunov, D.~Kovalev, D.~Makarenko, and P.~Richt\'{a}rik, ``Linearly
  converging error compensated {SGD},'' in \emph{Advances in Neural Information
  Processing Systems (NeurIPS)}, 2020.

\bibitem{EF21}
P.~Richt\'{a}rik, I.~Sokolov, and I.~Fatkhullin, ``{EF21}: {A} new, simpler,
  theoretically better, and practically faster error feedback,'' in
  \emph{Advances in Neural Information Processing Systems (NeurIPS)}, 2021.

\bibitem{EF21-ext}
I.~Fatkhullin, I.~Sokolov, E.~Gorbunov, Z.~Li, and P.~Richt\'{a}rik, ``{EF21}
  with bells \& whistles: {P}ractical algorithmic extensions of modern error
  feedback,'' \emph{arXiv preprint arXiv:2110.03294}, 2021.

\bibitem{PCDM}
P.~Richt\'{a}rik and M.~Tak\'{a}\v{c}, ``Parallel coordinate descent methods
  for big data optimization,'' \emph{Mathematical Programming}, vol. 156, no.
  1-2, pp. 433--484, 2016.

\bibitem{arnold}
B.~C. Arnold, N.~Balakrishnan, and H.~N. Nagaraja, \emph{A First Course in
  order Statistics}.\hskip 1em plus 0.5em minus 0.4em\relax John Wiley and Sons
  Inc., 1992.

\bibitem{gorbunov2020unified}
E.~Gorbunov, F.~Hanzely, and P.~Richt{\'a}rik, ``{A unified theory of {SGD}:
  Variance reduction, sampling, quantization and coordinate descent},'' in
  \emph{International Conference on Artificial Intelligence and Statistics
  (AISTATS)}, 2020, pp. 680--690.

\bibitem{lin2018don}
T.~Lin, S.~U. Stich, and M.~Jaggi, ``Don't use large mini-batches, use local
  {SGD},'' in \emph{International Conference on Learning Representations
  (ICLR)}, 2018.

\bibitem{woodworth2020local}
B.~Woodworth, K.~K. Patel, S.~Stich, Z.~Dai, B.~Bullins, B.~Mcmahan, O.~Shamir,
  and N.~Srebro, ``Is local {SGD} better than minibatch {SGD}?'' in
  \emph{International Conference on Machine Learning (ICML)}, 2020, pp.
  10\,334--10\,343.

\bibitem{ramezani2019nuqsgd}
A.~Ramezani-Kebrya, F.~Faghri, I.~Markov, V.~Aksenov, D.~Alistarh, and D.~M.
  Roy, ``{NUQSGD}: {P}rovably communication-efficient data-parallel {SGD} via
  nonuniform quantization,'' \emph{Journal of Machine Learning Research
  (JMLR)}, vol.~22, no. 114, pp. 1--43, 2021.

\bibitem{konevcny2018randomized}
J.~Kone{\v{c}}n{\'y} and P.~Richt{\'a}rik, ``Randomized distributed mean
  estimation: Accuracy vs. communication,'' \emph{Frontiers in Applied
  Mathematics and Statistics}, vol.~4, p.~62, 2018.

\bibitem{basu2019qsparse}
D.~Basu, D.~Data, C.~Karakus, and S.~Diggavi, ``{Qsparse-local-SGD}:
  Distributed {SGD} with quantization, sparsification and local computations,''
  in \emph{Advances in Neural Information Processing Systems (NeurIPS)}, 2019,
  pp. 14\,695--14\,706.

\bibitem{li2020acceleration}
Z.~Li, D.~Kovalev, X.~Qian, and P.~Richt\'{a}rik, ``Acceleration for compressed
  gradient descent in distributed and federated optimization,'' in
  \emph{International Conference on Machine Learning (ICML)}, 2020, pp.
  5895--5904.

\bibitem{karimi2016linear}
H.~Karimi, J.~Nutini, and M.~Schmidt, ``Linear convergence of gradient and
  proximal-gradient methods under the {P}olyak-{Ł}ojasiewicz condition,'' in
  \emph{Joint European Conference on Machine Learning and Knowledge Discovery
  in Databases}.\hskip 1em plus 0.5em minus 0.4em\relax Springer, 2016, pp.
  795--811.

\bibitem{bottou2018optimization}
L.~Bottou, F.~E. Curtis, and J.~Nocedal, ``Optimization methods for large-scale
  machine learning,'' \emph{Siam Review}, vol.~60, no.~2, pp. 223--311, 2018.

\bibitem{necoara2019linear}
I.~Necoara, Y.~Nesterov, and F.~Glineur, ``Linear convergence of first order
  methods for non-strongly convex optimization,'' \emph{Mathematical
  Programming}, vol. 175, no. 1-2, pp. 69--107, 2019.

\bibitem{gower2019sgd}
R.~M. Gower, N.~Loizou, X.~Qian, A.~Sailanbayev, E.~Shulgin, and
  P.~Richt{\'a}rik, ``{SGD}: General analysis and improved rates,''
  \emph{International Conference on Machine Learning (ICML)}, 2019.

\bibitem{stich2019unified}
S.~U. Stich, ``Unified optimal analysis of the (stochastic) gradient method,''
  \emph{arXiv preprint arXiv:1907.04232}, 2019.

\bibitem{elibol2020variance}
M.~Elibol, L.~Lei, and M.~I. Jordan, ``Variance reduction with sparse
  gradients,'' in \emph{International Conference on Learning Representations
  (ICLR)}, 2020.

\bibitem{dutta2019discrepancy}
A.~Dutta, E.~H. Bergou, A.~M. Abdelmoniem, C.-Y. Ho, A.~N. Sahu, M.~Canini, and
  P.~Kalnis, ``On the discrepancy between the theoretical analysis and
  practical implementations of compressed communication for distributed deep
  learning,'' \emph{arXiv preprint arXiv:1911.08250}, 2019.

\bibitem{eichner2019semi}
H.~Eichner, T.~Koren, B.~McMahan, N.~Srebro, and K.~Talwar, ``Semi-cyclic
  stochastic gradient descent,'' in \emph{International Conference on Machine
  Learning (ICML)}, 2019, pp. 1764--1773.

\bibitem{kingma2014adam}
D.~P. Kingma and J.~Ba, ``Adam: A method for stochastic optimization,''
  \emph{arXiv preprint arXiv:1412.6980}, 2014.

\bibitem{reddi2019convergence}
S.~J. Reddi, S.~Kale, and S.~Kumar, ``On the convergence of {A}dam and
  beyond,'' \emph{International Conference on Learning Representations (ICLR)},
  2018.

\bibitem{yang2016unified}
T.~Yang, Q.~Lin, and Z.~Li, ``Unified convergence analysis of stochastic
  momentum methods for convex and non-convex optimization,'' \emph{arXiv
  preprint arXiv:1604.03257}, 2016.

\bibitem{gadat2018stochastic}
S.~Gadat, F.~Panloup, S.~Saadane \emph{et~al.}, ``Stochastic heavy ball,''
  \emph{Electronic Journal of Statistics}, vol.~12, no.~1, pp. 461--529, 2018.

\bibitem{loizou2017momentum}
N.~Loizou and P.~Richt\'{a}rik, ``Momentum and stochastic momentum for
  stochastic gradient, {N}ewton, proximal point and subspace descent methods,''
  \emph{Computational Optimization and Applications}, vol.~77, no.~3, p.
  653–710, dec 2020.

\bibitem{li2014scaling}
M.~Li, D.~G. Andersen, J.~W. Park, A.~J. Smola, A.~Ahmed, V.~Josifovski,
  J.~Long, E.~J. Shekita, and B.-Y. Su, ``Scaling distributed machine learning
  with the parameter server,'' in \emph{11th USENIX Symposium on Operating
  Systems Design and Implementation (OSDI 14)}, 2014, pp. 583--598.

\bibitem{wei2015managed}
J.~Wei, W.~Dai, A.~Qiao, Q.~Ho, H.~Cui, G.~R. Ganger, P.~B. Gibbons, G.~A.
  Gibson, and E.~P. Xing, ``Managed communication and consistency for fast
  data-parallel iterative analytics,'' in \emph{Proceedings of the Sixth ACM
  Symposium on Cloud Computing}, 2015, pp. 381--394.

\bibitem{hsieh2017gaia}
K.~Hsieh, A.~Harlap, N.~Vijaykumar, D.~Konomis, G.~R. Ganger, P.~B. Gibbons,
  and O.~Mutlu, ``Gaia: Geo-distributed machine learning approaching {LAN}
  speeds,'' in \emph{14th Symposium on Networked Systems Design and
  Implementation}, 2017, pp. 629--647.

\bibitem{bordes2005fast}
A.~Bordes, S.~Ertekin, J.~Weston, and L.~Bottou, ``Fast kernel classifiers with
  online and active learning,'' \emph{Journal of Machine Learning Research
  (JMLR)}, vol.~6, no. Sep, pp. 1579--1619, 2005.

\bibitem{nesterov2012efficiency}
Y.~Nesterov, ``Efficiency of coordinate descent methods on huge-scale
  optimization problems,'' \emph{SIAM Journal on Optimization}, vol.~22, no.~2,
  pp. 341--362, 2012.

\bibitem{richtarik2014iteration}
P.~Richt{\'a}rik and M.~Tak{\'a}{\v{c}}, ``Iteration complexity of randomized
  block-coordinate descent methods for minimizing a composite function,''
  \emph{Mathematical Programming}, vol. 144, no. 1-2, pp. 1--38, 2014.

\bibitem{shalev2014accelerated}
S.~Shalev-Shwartz and T.~Zhang, ``Accelerated proximal stochastic dual
  coordinate ascent for regularized loss minimization,'' in \emph{International
  Conference on Machine Learning (ICML)}, 2014, pp. 64--72.

\bibitem{lin2014accelerated}
Q.~Lin, Z.~Lu, and L.~Xiao, ``An accelerated proximal coordinate gradient
  method,'' in \emph{Advances in Neural Information Processing Systems
  (NeurIPS)}, 2014, pp. 3059--3067.

\bibitem{fercoq2015accelerated}
O.~Fercoq and P.~Richt{\'a}rik, ``Accelerated, parallel, and proximal
  coordinate descent,'' \emph{SIAM Journal on Optimization}, vol.~25, no.~4,
  pp. 1997--2023, 2015.

\bibitem{allen2016even}
Z.~Allen-Zhu, Z.~Qu, P.~Richt{\'a}rik, and Y.~Yuan, ``Even faster accelerated
  coordinate descent using non-uniform sampling,'' in \emph{International
  Conference on Machine Learning (ICML)}, 2016, pp. 1110--1119.

\bibitem{stich2017safe}
S.~U. Stich, A.~Raj, and M.~Jaggi, ``Safe adaptive importance sampling,'' in
  \emph{Advances in Neural Information Processing Systems (NeurIPS)}, 2017, pp.
  4381--4391.

\bibitem{needell2014stochastic}
D.~Needell, R.~Ward, and N.~Srebro, ``Stochastic gradient descent, weighted
  sampling, and the randomized {K}aczmarz algorithm,'' in \emph{Advances in
  Neural Information Processing Systems (NeurIPS)}, 2014.

\bibitem{zhao2015stochastic}
P.~Zhao and T.~Zhang, ``Stochastic optimization with importance sampling for
  regularized loss minimization,'' in \emph{International Conference on Machine
  Learning (ICML)}, 2015.

\bibitem{bengio2009curriculum}
Y.~Bengio, J.~Louradour, R.~Collobert, and J.~Weston, ``Curriculum learning,''
  in \emph{International Conference on Machine Learning (ICML)}, 2009, pp.
  41--48.

\bibitem{schroff2015facenet}
F.~Schroff, D.~Kalenichenko, and J.~Philbin, ``Facenet: A unified embedding for
  face recognition and clustering,'' in \emph{Conference on Computer Vision and
  Pattern Recognition (CVPR)}, 2015, pp. 815--823.

\bibitem{simo2015discriminative}
E.~Simo-Serra, E.~Trulls, L.~Ferraz, I.~Kokkinos, P.~Fua, and F.~Moreno-Noguer,
  ``Discriminative learning of deep convolutional feature point descriptors,''
  in \emph{Proceedings of the IEEE International Conference on Computer
  Vision}, 2015, pp. 118--126.

\bibitem{schaul2015prioritized}
T.~Schaul, J.~Quan, I.~Antonoglou, and D.~Silver, ``Prioritized experience
  replay,'' \emph{arXiv preprint arXiv:1511.05952}, 2015.

\bibitem{loshchilov2015online}
I.~Loshchilov and F.~Hutter, ``Online batch selection for faster training of
  neural networks,'' \emph{arXiv preprint arXiv:1511.06343}, 2015.

\bibitem{katharopoulos2018not}
A.~Katharopoulos and F.~Fleuret, ``Not all samples are created equal: Deep
  learning with importance sampling,'' in \emph{International Conference on
  Machine Learning (ICML)}, 2018, pp. 2525--2534.

\bibitem{li2021tilted}
T.~Li, A.~Beirami, M.~Sanjabi, and V.~Smith, ``Tilted empirical risk
  minimization,'' in \emph{International Conference on Learning Representations
  (ICLR)}, 2021.

\bibitem{leaf}
S.~Caldas, S.~M.~K. Duddu, P.~Wu, T.~Li, J.~Kone{\v{c}}n\'{y}, H.~B. McMahan,
  V.~Smith, and A.~Talwalkar, ``Leaf: {A} benchmark for federated settings,''
  \emph{arXiv preprint arXiv:1812.01097}, 2018.

\bibitem{liang2019think}
P.~P. Liang, T.~Liu, L.~Ziyin, N.~B. Allen, R.~P. Auerbach, D.~Brent,
  R.~Salakhutdinov, and L.-P. Morency, ``Think locally, act globally:
  {F}ederated learning with local and global representations,'' in
  \emph{NeurIPS 2019 Workshop on Federated Learning}, 2019.

\bibitem{fedprox2020mlsys}
T.~Li, A.~K. Sahu, M.~Zaheer, M.~Sanjabi, A.~Talwalkar, and V.~Smith,
  ``Federated optimization in heterogeneous networks,'' in \emph{Machine
  Learning and Systems (MLSys)}, 2020.

\bibitem{feature_leak2019sp}
L.~Melis, C.~Song, E.~De~Cristofaro, and V.~Shmatikov, ``Exploiting unintended
  feature leakage in collaborative learning,'' in \emph{IEEE Symposium on
  Security and Privacy (SP)}, 2019, pp. 691--706.

\bibitem{diff_priv_fl2020jiot}
R.~Hu, Y.~Guo, H.~Li, Q.~Pei, and Y.~Gong, ``Personalized federated learning
  with differential privacy,'' \emph{IEEE Internet of Things Journal (JIOT)},
  vol.~7, no.~10, pp. 9530--9539, 2020.

\bibitem{backdoor_fl2020aistats}
E.~Bagdasaryan, A.~Veit, Y.~Hua, D.~Estrin, and V.~Shmatikov, ``How to backdoor
  federated learning,'' in \emph{International Conference on Artificial
  Intelligence and Statistics (AISTATS)}, 2020, pp. 2938--2948.

\bibitem{caldas2018expanding}
S.~Caldas, J.~Kone{\v{c}}n\'{y}, H.~B. McMahan, and A.~Talwalkar, ``Expanding
  the reach of federated learning by reducing client resource requirements,''
  \emph{arXiv preprint arXiv:1812.07210}, 2018.

\bibitem{hapi2020iccad}
S.~Laskaridis, S.~I. Venieris, H.~Kim, and N.~D. Lane, ``{HAPI}:
  {H}ardware-aware progressive inference,'' in \emph{International Conference
  on Computer-Aided Design (ICCAD)}, 2020.

\bibitem{dropout2014jmlr}
N.~Srivastava, G.~Hinton, A.~Krizhevsky, I.~Sutskever, and R.~Salakhutdinov,
  ``Dropout: {A} simple way to prevent neural networks from overfitting,''
  \emph{Journal of Machine Learning Research (JMLR)}, vol.~15, no.~56, pp.
  1929--1958, 2014.

\bibitem{rippel2014learning}
O.~Rippel, M.~Gelbart, and R.~Adams, ``Learning ordered representations with
  nested dropout,'' in \emph{International Conference on Machine Learning
  (ICML)}, 2014, pp. 1746--1754.

\bibitem{pruning2015neurips}
S.~Han, J.~Pool, J.~Tran, and W.~Dally, ``Learning both weights and connections
  for efficient neural network,'' in \emph{Advances in Neural Information
  Processing Systems (NeurIPS)}, 2015, pp. 1135--1143.

\bibitem{dnn_surgery2016neurips}
Y.~Guo, A.~Yao, and Y.~Chen, ``Dynamic network surgery for efficient {DNN}s,''
  in \emph{Advances in Neural Information Processing Systems (NeuriPS)}, 2016,
  pp. 1387--1395.

\bibitem{pruning_filters2016iclr}
H.~Li, A.~Kadav, I.~Durdanovic, H.~Samet, and H.~P. Graf, ``Pruning filters for
  efficient {ConvNets},'' in \emph{International Conference on Learning
  Representations (ICLR)}, 2016.

\bibitem{snip2019iclr}
N.~Lee, T.~Ajanthan, and P.~Torr, ``{SNIP}: {S}ingle-shot network pruning based
  on connection sensitivity,'' in \emph{International Conference on Learning
  Representations (ICLR)}, 2019.

\bibitem{molchanov2019importance}
P.~Molchanov, A.~Mallya, S.~Tyree, I.~Frosio, and J.~Kautz, ``Importance
  estimation for neural network pruning,'' in \emph{Conference on Computer
  Vision and Pattern Recognition (CVPR)}, 2019, pp. 11\,264--11\,272.

\bibitem{privacy_learning2012neurips}
M.~J. Wainwright, M.~Jordan, and J.~C. Duchi, ``Privacy aware learning,'' in
  \emph{Advances in Neural Information Processing Systems (NeurIPS)}, 2012.

\bibitem{privacy_dl2015ccs}
R.~Shokri and V.~Shmatikov, ``Privacy-preserving deep learning,'' in
  \emph{Proceedings of the 22nd ACM SIGSAC Conference on Computer and
  Communications Security (CCS)}, 2015, pp. 1310--1321.

\bibitem{nestdnn2018mobicom}
B.~Fang, X.~Zeng, and M.~Zhang, ``{NestDNN}: {R}esource-aware multi-tenant
  on-device deep learning for continuous mobile vision,'' in \emph{Proceedings
  of the 24th Annual International Conference on Mobile Computing and
  Networking (MobiCom)}, 2018, pp. 115--127.

\bibitem{haq2019cvpr}
K.~Wang, Z.~Liu, Y.~Lin, J.~Lin, and S.~Han, ``{HAQ}: {H}ardware-aware
  automated quantization with mixed precision,'' in \emph{Conference on
  Computer Vision and Pattern Recognition (CVPR)}, 2019, pp. 8612--8620.

\bibitem{shrinkml2019interspeech}
{\L}.~Dudziak, M.~S. Abdelfattah, R.~Vipperla, S.~Laskaridis, and N.~D. Lane,
  ``{ShrinkML}: {E}nd-to-end {ASR} model compression using reinforcement
  learning,'' in \emph{INTERSPEECH}, 2019, pp. 2235--2239.

\bibitem{fedmd2019neuripsw}
D.~Li and J.~Wang, ``{FedMD}: {H}eterogenous federated learning via model
  distillation,'' in \emph{NeurIPS 2019 Workshop on Federated Learning for Data
  Privacy and Confidentiality}, 2019.

\bibitem{hsieh2020noniid}
K.~Hsieh, A.~Phanishayee, O.~Mutlu, and P.~B. Gibbons, ``The non-iid data
  quagmire of decentralized machine learning,'' in \emph{International
  Conference on Machine Learning (ICML)}, vol. 119, 2020, pp. 4387--4398.

\bibitem{personalised_fl2020neurips}
A.~Fallah, A.~Mokhtari, and A.~Ozdaglar, ``Personalized federated learning with
  theoretical guarantees: {A} model-agnostic meta-learning approach,''
  \emph{Advances in Neural Information Processing Systems (NeurIPS)}, 2020.

\bibitem{fair_fl2020iclr}
T.~Li, M.~Sanjabi, A.~Beirami, and V.~Smith, ``Fair resource allocation in
  federated learning,'' in \emph{International Conference on Learning
  Representations (ICLR)}, 2020.

\bibitem{clientsel_fl2019icc}
T.~Nishio and R.~Yonetani, ``Client selection for federated learning with
  heterogeneous resources in mobile edge,'' in \emph{IEEE International
  Conference on Communications (ICC)}, 2019.

\bibitem{cost_eff_fl2021infocom}
B.~Luo, X.~Li, S.~Wang, J.~Huang, and L.~Tassiulas, ``Cost-effective federated
  learning design,'' in \emph{INFOCOM}, 2021.

\bibitem{adaptivefl2019jsac}
S.~Wang, T.~Tuor, T.~Salonidis, K.~K. Leung, C.~Makaya, T.~He, and K.~Chan,
  ``Adaptive federated learning in resource constrained edge computing
  systems,'' \emph{IEEE Journal on Selected Areas in Communications (JSAC)},
  vol.~37, no.~6, 2019.

\bibitem{wang2018atomo}
H.~Wang, S.~Sievert, S.~Liu, Z.~Charles, D.~Papailiopoulos, and S.~Wright,
  ``{Atomo: Communication-efficient learning via atomic sparsification},''
  \emph{Advances in Neural Information Processing Systems (NeurIPS)}, vol.~31,
  pp. 9850--9861, 2018.

\bibitem{adapt_grad_sparse_fl2020icdcs}
P.~Han, S.~Wang, and K.~K. Leung, ``Adaptive gradient sparsification for
  efficient federated learning: {A}n online learning approach,'' in \emph{IEEE
  International Conference on Distributed Computing Systems (ICDCS)}, 2020.

\bibitem{muppet2020icml}
A.~Rajagopal, D.~Vink, S.~Venieris, and C.-S. Bouganis, ``Multi-precision
  policy enforced training ({M}u{PPET}) : {A} precision-switching strategy for
  quantised fixed-point training of {CNN}s,'' in \emph{International Conference
  on Machine Learning (ICML)}, 2020, pp. 7943--7952.

\bibitem{quant_grad_fl2020arxiv}
M.~M. Amiri, D.~Gunduz, S.~R. Kulkarni, and H.~V. Poor, ``Federated learning
  with quantized global model updates,'' \emph{arXiv preprint
  arXiv:2006.10672}, 2020.

\bibitem{prunefl2020neuripsw}
Y.~Jiang, S.~Wang, B.~J. Ko, W.-H. Lee, and L.~Tassiulas, ``Model pruning
  enables efficient federated learning on edge devices,'' in \emph{Workshop on
  Scalability, Privacy, and Security in Federated Learning (SpicyFL), NeurIPS},
  2020.

\bibitem{balancedsparse2019aaai}
Z.~Yao, S.~Cao, W.~Xiao, C.~Zhang, and L.~Nie, ``Balanced sparsity for
  efficient {DNN} inference on {GPU},'' in \emph{Conference on Artificial
  Intelligence (AAAI)}, vol.~33, 2019, pp. 5676--5683.

\bibitem{nas2017iclr}
B.~Zoph and Q.~Le, ``Neural architecture search with reinforcement learning,''
  in \emph{International Conference on Learning Representations (ICLR)}, 2017.

\bibitem{li2021fedbn}
X.~Li, M.~Jiang, X.~Zhang, M.~Kamp, and Q.~Dou, ``{FedBN}: {F}ederated learning
  on non-iid features via local batch normalization,'' in \emph{International
  Conference on Learning Representations (ICLR)}, 2021.

\bibitem{hinton2014distilling}
G.~Hinton, O.~Vinyals, and J.~Dean, ``Distilling the knowledge in a neural
  network,'' in \emph{NeurIPS Deep Learning Workshop}, 2014.

\bibitem{starfish2015mobisys}
R.~LiKamWa and L.~Zhong, ``{Starfish}: {E}fficient concurrency support for
  computer vision applications,'' in \emph{International Conference on Mobile
  Systems, Applications, and Services (MobiSys)}, 2015, pp. 213--226.

\bibitem{hochreiter1997long}
S.~Hochreiter and J.~Schmidhuber, ``Long short-term memory,'' \emph{Neural
  Computation}, vol.~9, no.~8, pp. 1735--1780, 1997.

\bibitem{beutel2020flower}
D.~J. Beutel, T.~Topal, A.~Mathur, X.~Qiu, T.~Parcollet, and N.~D. Lane,
  ``Flower: {A} friendly federated learning research framework,'' \emph{arXiv
  preprint arXiv:2007.14390}, 2020.

\bibitem{mitra2021fedlin}
A.~Mitra, R.~Jaafar, G.~Pappas, and H.~Hassani, ``Linear convergence in
  federated learning: {T}ackling client heterogeneity and sparse gradients,''
  \emph{Advances in Neural Information Processing Systems (NeurIPS)}, vol.~34,
  2021.

\bibitem{mishchenko2021proximal}
K.~Mishchenko, A.~Khaled, and P.~Richt{\'a}rik, ``Proximal and federated random
  reshuffling,'' \emph{arXiv preprint arXiv:2102.06704}, 2021.

\bibitem{yun2021minibatch}
C.~Yun, S.~Rajput, and S.~Sra, ``Minibatch vs local {SGD} with shuffling: Tight
  convergence bounds and beyond,'' in \emph{International Conference on
  Learning Representations (ICLR)}, 2022.

\bibitem{karimireddy2020mime}
S.~P. Karimireddy, M.~Jaggi, S.~Kale, M.~Mohri, S.~J. Reddi, S.~U. Stich, and
  A.~T. Suresh, ``Breaking the centralized barrier for cross-device federated
  learning,'' in \emph{Advances in Neural Information Processing Systems
  (NeurIPS)}, 2021.

\bibitem{wang2021cooperative}
J.~Wang and G.~Joshi, ``Cooperative {SGD}: {A} unified framework for the design
  and analysis of local-update {SGD} algorithms,'' \emph{Journal of Machine
  Learning Research (JMLR)}, vol.~22, no. 213, pp. 1--50, 2021.

\bibitem{zhou2017convergence}
F.~Zhou and G.~Cong, ``On the convergence properties of a $ k $-step averaging
  stochastic gradient descent algorithm for nonconvex optimization,'' in
  \emph{International Joint Conference on Artificial Intelligence}, 2018.

\bibitem{yu2019parallel}
H.~Yu, S.~Yang, and S.~Zhu, ``Parallel restarted {SGD} with faster convergence
  and less communication: Demystifying why model averaging works for deep
  learning,'' in \emph{Conference on Artificial Intelligence (AAAI)}, vol.~33,
  2019, pp. 5693--5700.

\bibitem{li2019convergence}
X.~Li, K.~Huang, W.~Yang, S.~Wang, and Z.~Zhang, ``On the convergence of
  {FedAvg} on non-iid data,'' in \emph{International Conference on Learning
  Representations (ICLR)}, 2020.

\bibitem{haddadpour2019convergence}
F.~Haddadpour and M.~Mahdavi, ``On the convergence of local descent methods in
  federated learning,'' \emph{arXiv preprint arXiv:1910.14425}, 2019.

\bibitem{haddadpour2019trading}
F.~Haddadpour, M.~M. Kamani, M.~Mahdavi, and V.~Cadambe, ``Trading redundancy
  for communication: Speeding up distributed {SGD} for non-convex
  optimization,'' in \emph{International Conference on Machine Learning
  (ICML)}, 2019, pp. 2545--2554.

\bibitem{haddadpour2019local}
------, ``{Local {SGD} with periodic averaging: {T}ighter analysis and adaptive
  synchronization},'' \emph{Advances in Neural Information Processing Systems
  (NeurIPS)}, vol.~32, 2019.

\bibitem{wang2019adaptive}
J.~Wang and G.~Joshi, ``Adaptive communication strategies to achieve the best
  error-runtime trade-off in local-update {SGD},'' \emph{Machine Learning and
  Systems (MLSys)}, vol.~1, pp. 212--229, 2019.

\bibitem{koloskova2020unified}
A.~Koloskova, N.~Loizou, S.~Boreiri, M.~Jaggi, and S.~U. Stich, ``A unified
  theory of decentralized {SGD} with changing topology and local updates,'' in
  \emph{International Conference on Machine Learning (ICML)}, 2020.

\bibitem{khaled2019first}
A.~Khaled, K.~Mishchenko, and P.~Richt{\'a}rik, ``First analysis of local {GD}
  on heterogeneous data,'' \emph{arXiv preprint arXiv:1909.04715}, 2019.

\bibitem{woodworth2018graph}
B.~Woodworth, J.~Wang, A.~Smith, B.~McMahan, and N.~Srebro, ``Graph oracle
  models, lower bounds, and gaps for parallel stochastic optimization,'' in
  \emph{Advances in Neural Information Processing Systems (NeurIPS)}, 2018.

\bibitem{xie2019local}
C.~Xie, O.~Koyejo, I.~Gupta, and H.~Lin, ``Local {A}da{A}lter:
  Communication-efficient stochastic gradient descent with adaptive learning
  rates,'' \emph{arXiv preprint arXiv:1911.09030}, 2019.

\bibitem{liang2019variance}
X.~Liang, S.~Shen, J.~Liu, Z.~Pan, E.~Chen, and Y.~Cheng, ``Variance reduced
  local {SGD} with lower communication complexity,'' \emph{arXiv preprint
  arXiv:1912.12844}, 2019.

\bibitem{gorbunov2021marina}
E.~Gorbunov, K.~P. Burlachenko, Z.~Li, and P.~Richt{\'a}rik, ``{MARINA: Faster
  non-convex distributed learning with compression},'' in \emph{International
  Conference on Machine Learning (ICML)}, 2021, pp. 3788--3798.

\bibitem{philippenko2021preserved}
C.~Philippenko and A.~Dieuleveut, ``Preserved central model for faster
  bidirectional compression in distributed settings,'' \emph{Advances in Neural
  Information Processing Systems (NeurIPS)}, vol.~34, 2021.

\bibitem{bernstein2018signsgd}
J.~Bernstein, Y.-X. Wang, K.~Azizzadenesheli, and A.~Anandkumar, ``{SignSGD:
  C}ompressed optimisation for non-convex problems,'' in \emph{International
  Conference on Machine Learning (ICML)}, 2018, pp. 560--569.

\bibitem{xu2021deepreduce}
H.~Xu, K.~Kostopoulou, A.~Dutta, X.~Li, A.~Ntoulas, and P.~Kalnis,
  ``Deepreduce: A sparse-tensor communication framework for federated deep
  learning,'' \emph{Advances in Neural Information Processing Systems
  (NeurIPS)}, vol.~34, 2021.

\bibitem{sahu2021rethinking}
A.~Sahu, A.~Dutta, A.~M~Abdelmoniem, T.~Banerjee, M.~Canini, and P.~Kalnis,
  ``Rethinking gradient sparsification as total error minimization,''
  \emph{Advances in Neural Information Processing Systems (NeurIPS)}, vol.~34,
  2021.

\bibitem{bottou2012stochastic}
L.~Bottou, ``Stochastic gradient descent tricks,'' \emph{Lecture Notes in
  Computer Science ({LNCS})}, 2012.

\bibitem{bottou2009curiously}
------, ``Curiously fast convergence of some stochastic gradient descent
  algorithms,'' in \emph{Proceedings of the Symposium on Learning and Data
  Science, Paris}, vol.~8, 2009, pp. 2624--2633.

\bibitem{gurbuzbalaban2021random}
M.~G{\"u}rb{\"u}zbalaban, A.~Ozdaglar, and P.~A. Parrilo, ``Why random
  reshuffling beats stochastic gradient descent,'' \emph{Mathematical
  Programming}, vol. 186, no.~1, pp. 49--84, 2021.

\bibitem{ahn2020tight}
K.~Ahn and S.~Sra, ``On tight convergence rates of without-replacement {SGD},''
  \emph{arXiv preprint arXiv:2004.08657}, 2020.

\bibitem{mishchenko2020random}
K.~Mishchenko, A.~Khaled Ragab~Bayoumi, and P.~Richt{\'a}rik, ``Random
  reshuffling: Simple analysis with vast improvements,'' \emph{Advances in
  Neural Information Processing Systems (NeurIPS)}, vol.~33, 2020.

\bibitem{lin2020ensemble}
T.~Lin, L.~Kong, S.~U. Stich, and M.~Jaggi, ``Ensemble distillation for robust
  model fusion in federated learning,'' \emph{Advances in Neural Information
  Processing Systems (NeurIPS)}, vol.~33, 2020.

\bibitem{he2020group}
C.~He, M.~Annavaram, and S.~Avestimehr, ``Group knowledge transfer: Federated
  learning of large cnns at the edge,'' \emph{Advances in Neural Information
  Processing Systems (NeurIPS)}, vol.~33, 2020.

\bibitem{diao2020heterofl}
E.~Diao, J.~Ding, and V.~Tarokh, ``Hetero{FL}: {C}omputation and communication
  efficient federated learning for heterogeneous clients,'' in
  \emph{International Conference on Learning Representations (ICLR)}, 2021.

\bibitem{Arjevani2015:lowerbound}
Y.~Arjevani and O.~Shamir, ``Communication complexity of distributed convex
  learning and optimization,'' in \emph{Advances in Neural Information
  Processing Systems (NeurIPS)}, 2015, pp. 1756--1764.

\bibitem{defazio2014saga}
A.~Defazio, F.~Bach, and S.~Lacoste-Julien, ``{SAGA}: A fast incremental
  gradient method with support for non-strongly convex composite objectives,''
  in \emph{Advances in Neural Information Processing Systems (NeurIPS)}, 2014,
  pp. 1646--1654.

\bibitem{tran2019hybrid}
Q.~Tran-Dinh, N.~H. Pham, D.~T. Phan, and L.~M. Nguyen, ``Hybrid stochastic
  gradient descent algorithms for stochastic nonconvex optimization,''
  \emph{arXiv preprint arXiv:1905.05920}, 2019.

\bibitem{huba2022papaya}
D.~Huba, J.~Nguyen, K.~Malik, R.~Zhu, M.~Rabbat, A.~Yousefpour, C.-J. Wu,
  H.~Zhan, P.~Ustinov, H.~Srinivas \emph{et~al.}, ``{Papaya: Practical,
  private, and scalable federated learning},'' \emph{Machine Learning and
  Systems (MLSys)}, vol.~4, 2022.

\bibitem{jiang}
P.~Jiang and G.~Agrawal, ``A linear speedup analysis of distributed deep
  learning with sparse and quantized communication,'' in \emph{Advances in
  Neural Information Processing Systems (NeurIPS)}, 2018, pp. 2525--2536.

\bibitem{lacoste2012simpler}
S.~Lacoste-Julien, M.~Schmidt, and F.~Bach, ``A simpler approach to obtaining
  an $\mathcal{O}(1/t)$ convergence rate for the projected stochastic
  subgradient method,'' \emph{arXiv preprint arXiv:1212.2002}, 2012.

\bibitem{grimmer2019convergence}
B.~Grimmer, ``Convergence rates for deterministic and stochastic subgradient
  methods without {L}ipschitz continuity,'' \emph{SIAM Journal on
  Optimization}, vol.~29, no.~2, pp. 1350--1365, 2019.

\bibitem{emnist}
G.~Cohen, S.~Afshar, J.~Tapson, and A.~Van~Schaik, ``{EMNIST}: {E}xtending
  {MNIST} to handwritten letters,'' in \emph{2017 International Joint
  Conference on Neural Networks (IJCNN)}.\hskip 1em plus 0.5em minus
  0.4em\relax IEEE, 2017, pp. 2921--2926.

\bibitem{tffdatasets}
\BIBentryALTinterwordspacing
TFF, ``Tensorflow federated datasets,'' 2021. [Online]. Available:
  \url{https://www.tensorflow.org/federated/ api_ docs/ python/ tff/
  simulation/ datasets}
\BIBentrySTDinterwordspacing

\bibitem{reddi2020adaptive}
S.~J. Reddi, Z.~Charles, M.~Zaheer, Z.~Garrett, K.~Rush, J.~Kone{\v{c}}n{\'y},
  S.~Kumar, and H.~B. McMahan, ``Adaptive federated optimization,'' in
  \emph{International Conference on Learning Representations (ICLR)}, 2021.

\bibitem{li2006pachinko}
W.~Li and A.~McCallum, ``Pachinko allocation: {DAG}-structured mixture models
  of topic correlations,'' in \emph{International Conference on Machine
  Learning (ICML)}, 2006, pp. 577--584.

\end{thebibliography}
\end{onehalfspacing}

\appendix

\newpage

\begingroup
\let\clearpage\relax
\begin{center}
\vspace*{20\baselineskip}
{ \textbf{{\Large APPENDICES}}} 
\addcontentsline{toc}{chapter}{Appendices} 
\end{center}
\endgroup

\refstepcounter{chapter}%
\chapter*{\thechapter \quad Appendix: Technicalities}
\label{appendix:technicalities}

\section{Convex smooth functions}

We first state some implications of Assumption~\ref{ass:smooth}. It directly leads to the following quadratic upper bound on function $h:\R^d\to\R$:
\begin{equation}
\label{eqn:quad-upper}
    h(y) \leq h(x) + \inp{\nabla h(x)}{y - x} + \frac{L}{2}\norm{y - x}^2, \quad \forall x,y \in \R^d\,.
\end{equation}

If convexity is assumed as well, then the following inequality holds
\begin{equation}
    \label{L-smooth3}
    \frac{1}{2L}\norm{\nabla h(x) - \nabla h(y)}^2 \leq h(x) - h(y) - \lin{ \nabla h(y), x-y }, \quad \forall x,y \in \R^d\,.
\end{equation}

In addition, if $h = \frac{1}{m} \sum_{i=1}^m h_i$ and $\cbr{h_i}$ are convex and $x^\star$ is an optimum of $h$, then
\begin{align}
\label{eqn:smoothness}
  \begin{split}
    \frac{1}{2Lm}\sum_{i=1}^m\norm{\nabla h_i(x) - \nabla
    h_i(x^\star)}^2 \leq h(x) - h(x^\star), \quad \forall x \in \R^d\,.
  \end{split}
  \end{align}

In the special case of $m=1$, we have
\begin{equation}
    \label{L-smooth4}
    \norm{\nabla h(x)}^2 \leq 2L(h(x) - h(x^*)), \quad \forall x \in \R^d\,.
\end{equation}

Further, if $h$ is twice-differentiable, then Assumption~\ref{eq:smooth} implies that $\norm{\nabla^2 h(x)} \leq L$ for all $x$
; see, e.g., 
\cite[Theorem~2.1.5]{nesterov2013introductory}, where $\norm{\cdot}$ refers here to the spectral norm. 

Next, we include lemma that is useful to provide bounds for any smooth and strongly-convex functions. It can be seen as a generalization of the standard strong convexity inequality \eqref{eq:def_strongly_convex}, but can bound gradients computed at slightly perturbed points using smoothness assumption. This lemma turns out to be especially useful for local methods and was also used in the analysis of \texttt{FedAVG} in \cite{karimireddy2020scaffold}.

 \begin{lemma}[Perturbed strong convexity]\label{lem:magic}
    The following holds for any $L$-smooth and $\mu$-strongly convex function $h$, and any $x, y, z$ in the domain of $h$:
    \begin{align}
     \label{lemma:FedAvgProof}   
        \inp{\nabla h(x)}{z - y} \geq h(z) - h(y) +\frac{\mu}{4}\norm{y - z}^2  - L \norm{z - x}^2, \quad \forall x,y,z \in \R^d\,.
     \end{align}
\end{lemma}
\begin{proof}
    Given any $x$, $y$, and $z$, one gets the following two inequalities using smoothness and strong convexity of the function $h$:
    \begin{align*}
\inp{\nabla h(x)}{z - x} &\geq h(z) - h(x) - \frac{L}{2}\norm{z - x}^2\\
\inp{\nabla h(x)}{x - y} &\geq h(x) - h(y) + \frac{\mu}{2}\norm{y - x}^2 \,.
    \end{align*}
    Further, applying the relaxed triangle inequality gives
    \[
\frac{\mu}{2}\norm{y - x}^2 \geq \frac{\mu}{4}\norm{y - z}^2 - \frac{\mu}{2}\norm{x - z}^2\,.
    \]
    Combining all the inequalities together we have
    \begin{align}
    \label{eq:perturbed}
\inp{\nabla h(x)}{z - y} \geq h(z) - h(y) +\frac{\mu}{4}\norm{y - z}^2  - \frac{L + \mu}{2}\norm{z - x}^2\,.
    \end{align}
    The lemma follows since it has to hold $L \geq \mu$.
\end{proof}

Note that in the case $z = x$, the Equation~\eqref{eq:perturbed} reduces to the standard strong convexity assumption~\eqref{eq:def_strongly_convex}. In the case of $z \neq x$, we evaluate the gradient at $x$, but we are interested in progress with respect to the perturbed point $z$. The term $L\norm{z-x}^2$ is the price we pay for the perturbation.  

We continue with some useful identities for smooth convex functions.  For a $L$-smooth and $\mu$-strongly convex function $h \colon \R^d \to \R$ we have the tight bound
\begin{align}
 \lin{\nabla h(x)-\nabla h(y),x-y} \geq \frac{\mu L}{\mu + L} \norm{x-y}^2 + \frac{1}{\mu+L} \norm{\nabla h(x) - \nabla h(y)}^2\,. \label{eq:coerc}
\end{align}

Each $\mu$-strongly convex function $h \colon \R^d \to \R$ implies the following upper bound
\begin{align}
\lin{\nabla h(x)-\nabla h(y),x-y} \geq \mu \norm{x-y}^2\,. \label{eq:coerc2}
\end{align}

The prox operator of a proper closed convex function $R:\R^d \rightarrow \R \cup \{\infty\}$ is non-expansive. That is, for any $\gamma > 0$,
\begin{align}
 \norm{\prox_{\gamma R}(x) - \prox_{\gamma R}(y)} &\leq \norm{x-y}\,.
\end{align}

Bregman divergence of a convex and $L$-smooth function is lower bounded by gradients difference,
\begin{align}
	B_h(x, y)
	\eqdef h(x) - h(y) - \lin{\nabla f(y), x - y}
	\ge \frac{1}{2L}\|\nabla h(x) - \nabla h(y)\|^2 \ge 0. \label{eq:bregman_grad_dif}
\end{align}

Finally, Definition~\ref{ass:1_optimal} can be relaxed by enforcing $y = \xs$, which leads to $\mu$-quasi convexity.

\begin{definition}[$\mu$-quasi convexity]
\label{ass:1_app}
$h$ is $\mu$-quasi convex if
\begin{equation}
\label{eq:quasi_convex}
h^\star \geq h(x) + \dotprod{\nabla h(x)}{ \xs - x} + \frac{\mu}{2}\norm{\xs - x}^2, \qquad \forall x \in \R^d,
\end{equation}
where $\xs$  the unique minimizer of $h$\footnote{This assumption can be relaxed, but we enforce it for ease of exposition.}.
\end{definition}

\section{Convergence derivations}

In this section, we cover some technical lemmas which are useful to unroll recursions and derive convergence rates.
computations later on. The provided Lemmas correspond to Lemma 1 and 2 in~\cite{karimireddy2020scaffold}.

\begin{lemma}[linear convergence rate]
\label{lem:constant}
For every non-negative sequence $\{d_{r-1}\}_{r \geq 1}$ and any
parameters $\mu > 0$, $\eta_{\max} \in (0, 1/\mu]$, $c \geq 0$,
$R \geq \frac{1}{2\eta_{\max} \mu}$, there exists a constant step size
$\eta \leq \eta_{\max}$ and weights $v_r\eqdef(1-\mu\eta)^{1-r}$ such that
for $V_R \eqdef \sum_{r=1}^{R+1} v_r$,
\begin{align*}
    \Psi_R \eqdef \frac{1}{V_R}\sum_{r=1}^{R+1} \left( \frac{v_r}{\eta} \left(1-\mu \eta \right) d_{r-1} - \frac{v_r}{\eta} d_{r} + c \eta v_r \right) = \tilde \cO \left( \mu d_0 \exp\rbr*{- \mu\eta_{\max}  R } + \frac{c}{\mu R}  \right)\,.
\end{align*}
\end{lemma}
\begin{proof}
  By substituting the value of $v_r$, we observe that we end up with a
  telescoping sum and estimate
    \begin{align*}
     \Psi_R = \frac{1}{\eta V_R} \sum_{r=1}^{R+1} \left(v_{r-1}d_{r-1} - v_{r} d_{r} \right) + \frac{c\eta}{V_R}\sum_{r=1}^{R+1} v_r \leq \frac{d_0}{\eta V_R} +  c \eta \,.
    \end{align*}
    When $R \geq \frac{1}{2\mu\eta}$, $(1 - \mu \eta)^R \leq \exp(-\mu\eta R) \leq \frac{2}{3}$. For such an $R$, we can lower bound $\eta V_R$ using
    \[
\eta V_R = \eta (1 - \mu \eta)^{-R} \sum_{r=0}^{R} (1 - \mu \eta)^r = \eta (1 - \mu \eta)^{-R} \frac{1 - (1 - \mu \eta)^R}{\mu \eta} \geq (1 - \mu \eta)^{-R} \frac{1}{3\mu}\,.
    \]
    This proves that for all $R \geq \frac{1}{2\mu \eta}$,
    \[
 \Psi_R \leq 3\mu d_0 (1 - \mu \eta)^{R}  +  c \eta \leq 3 \mu d_o \exp(-\mu \eta R) + c\eta\,.
    \]
    The lemma now follows by carefully tuning $\eta$. Consider the following two cases depending on the magnitude of $R$ and $\eta_{\max}$:
    \begin{itemize}
\item Suppose $\frac{1}{2\mu R} \leq \eta_{\max} \leq \frac{\log(\max(1,\mu^2 R d_0/ c))}{\mu R}$. Then we can choose $\eta = \eta_{\max}$,
\[
    \Psi_R \leq 3\mu d_0 \exp\left[-\mu \eta_{\max} R \right] + c \eta_{\max} \leq 3\mu d_0 \exp\left[-\mu \eta_{\max} R \right] + \tilde\cO\rbr[\bigg]{\frac{c}{\mu R}}\,.
\]
\item Instead if $\eta_{\max} > \frac{\log(\max(1,\mu^2 R d_0/ c))}{\mu R}$, we pick $\eta = \frac{\log(\max(1,\mu^2 R d_0/ c))}{\mu R}$ to claim that
\[
    \Psi_R \leq 3 \mu d_0 \exp\left[-\log(\max(1,\mu^2 R d_0/ c))\right] + \tilde\cO\rbr[\bigg]{\frac{c}{\mu R}} \leq \tilde\cO\rbr[\bigg]{\frac{c}{\mu R}} \,.
\]
    \end{itemize}
    \end{proof}

\begin{lemma}[sub-linear convergence rate]\label{lemma:general}
For every non-negative sequence  \\ $\{d_{r-1}\}_{r \geq 1}$ and any
parameters $\eta_{\max}  \geq 0$, $c \geq 0$,
$R \geq 0$, there exists a constant step size
$\eta \leq \eta_{\max}$ and weights $v_r  = 1$ such that,
\begin{align*}
    \Psi_R &\eqdef \frac{1}{R+1}\sum_{r=1}^{R+1} \left( \frac{d_{r-1}}{\eta}  - \frac{d_r}{\eta} + c_1 \eta +c_2 \eta^2\right) \\
    &\leq  \frac{d_0}{\eta_{\max}(R+1)} + \frac{2\sqrt{c_1 d_0}}{\sqrt{R +1}} + 2\rbr[\bigg]{\frac{d_0}{R+1}}^{\frac{2}{3}} c_2^{\frac{1}{3}}\,.
\end{align*}
\end{lemma}
\begin{proof} Unrolling the sum, we can simplify
\[
    \Psi_R \leq \frac{d_0}{\eta (R+1)} + c_1 \eta + c_2 \eta^2\,.
\]
Similar to the strongly convex case (Lemma~\ref{lem:constant}), we distinguish the following cases:
\begin{itemize}
    \item When $R+1 \leq \frac{d_0}{c_1 \eta_{\max}^2}$, and $R+1 \leq \frac{d_0}{c_2 \eta_{\max}^3}$ we pick $\eta = \eta_{\max}$ to claim
    \[
        \Psi_R \leq \frac{d_0}{\eta_{\max} (R+1)} + c_1 \eta_{\max} + c_2 \eta_{\max}^2 \leq \frac{d_0}{\eta_{\max} (R+1)} + \frac{\sqrt{c_1 d_0}}{\sqrt{R +1}}+ \rbr[\bigg]{\frac{d_0}{R+1}}^{\frac{2}{3}} c_2^{\frac{1}{3}} \,.
    \]
    \item In the other case, we have $\eta_{\max}^2 \geq \frac{d_0}{c_1(R+1)}$ or $\eta_{\max}^3 \geq \frac{d_0}{c_2(R+1)}$. We choose $\eta = \min\cbr[\bigg]{\sqrt{\frac{d_0}{c_1(R+1)}}, \sqrt[3]{\frac{d_0}{c_2(R+1)}}}$ to prove
    \[
        \Psi_R \leq \frac{d_0}{\eta (R+1)} + c \eta = \frac{2\sqrt{c_1 d_0}}{\sqrt{R +1}} + 2\sqrt[3]{\frac{d_0^2 c_2}{(R+1)^2}} \,.\vspace{-5mm}
    \]
\end{itemize}
\end{proof}

\section{Variance bounds}
\label{appendix:lemma proof}

In this section, we provide bounds that are useful to bound the variance of the gradient estimators. We first introduce standard variance decomposition and then state two lemmas. 

For random variable $X$ and any $y \in \R^d$, the variance can be decomposed as
\begin{align}
    \label{eq:var_decomposition}
    \E{\norm{X - \E{X}}^2} = \E{\norm{X - y}^2} - \norm{\E{X} - y}^2.
\end{align}

For any independent random variables $X_1,X_2, \cdots, X_n \in \R^d$ we have
\begin{align}
\E{\norm{\frac{1}{n}\sum_{i = 1}^n \left(X_i -\E{X_i}\right)}^2} = \frac{1}{n^2}\sum_{i = 1}^n\E{\norm{ X_i -\E{X_i}}^2}
\label{eq:indep}
\end{align}

The first lemma captures bounds for the variance of unbiased estimator with arbitrary proper sampling.

\begin{lemma} 
\label{LEM:UPPERV_app}
Let $\zeta_1,\zeta_2,\dots,\zeta_n$ be vectors in $\mathbb{R}^d$ and $w_1,w_2,\dots,w_n$ be non-negative real numbers such that $\sum_{i=1}^n w_i=1$. Define $\Tilde{\zeta} \coloneqq \sum_{i=1}^n w_i \zeta_i$. Let $S$ be a proper sampling. If $v\in\mathbb{R}^n$ is  such that
        \begin{equation}\label{eq:ESO_app}
            \mP -pp^\top \preceq {\rm \bf Diag}(p_1v_1,p_2v_2,\dots,p_nv_n),
        \end{equation}
        then
        \begin{equation}\label{eq:key_inequality_app}
            \E{ \norm{ \sum \limits_{i\in S} \frac{w_i\zeta_i}{p_i} - \Tilde{\zeta} }^2 } \leq \sum \limits_{i=1}^n w_i^2\frac{v_i}{p_i} \norm{\zeta_i}^2,
        \end{equation}
        where the expectation is taken over $S$. Whenever \eqref{eq:ESO_app} holds, it must be the case that $v_i \geq 1-p_i$. 
\end{lemma}

\begin{proof}
  Our proof technique can be seen as an extended version of our earlier work \cite{horvath2018nonconvex}.  Let $1_{i\in S} = 1$ if $i\in S$ and $1_{i\in S} = 0$ otherwise. Likewise, let $1_{i,j\in S} = 1$ if $i,j\in S$ and  $1_{i,j\in S} = 0$ otherwise. Note that $\E{1_{i\in S}} = p_i$ and $\E{1_{i,j\in S}}=p_{ij}$. Next, let us compute the mean of $ X \coloneqq \sum_{i\in S}\frac{w_i \zeta_i}{p_i}$:
    \begin{align*}
            \E{X} 
            = \E{ \sum_{i\in S}\frac{w_i \zeta_i}{p_i} } 
            = \E{ \sum_{i=1}^n \frac{w_i \zeta_i}{p_i} 1_{i\in S} } 
            = \sum_{i=1}^n \frac{w_i \zeta_i}{p_i} \E{1_{i\in S}} 
            = \sum_{i=1}^n w_i \zeta_i
            = \Tilde{\zeta}.
    \end{align*}
    Let $\boldsymbol{A}=[a_1,\dots,a_n]\in \mathbb{R}^{d\times n}$, where $a_i=\frac{w_i\zeta_i}{p_i}$, and let $e$ be the vector of all ones in $\mathbb{R}^n$. We now write the variance of $X$ in a form which will be convenient to establish a bound:
    \begin{equation}\label{eq:variance_of_X}
        \begin{split}
            \E{\norm{X - \E{X}}^2}
            &= \E{\norm{X}^2} - \norm{ \E{X} }^2 \\
            &=\E{\norm{\sum_{i\in S}\frac{w_i \zeta_i}{p_i}}^2} - \norm{\Tilde{\zeta}}^2 \\
            &=\E{ \sum_{i,j} \frac{w_i\zeta_i^\top}{p_i}\frac{w_j\zeta_j}{p_j}  1_{i,j\in S} } - \norm{\Tilde{\zeta}}^2 \\
            &= \sum_{i,j} p_{ij} \frac{w_i\zeta_i^\top}{p_i}\frac{w_j\zeta_j}{p_j} - \sum_{i,j} w_iw_j\zeta_i^\top\zeta_j \\
            &= \sum_{i,j} (p_{ij}-p_ip_j)a_i^\top a_j \\
            &= e^\top ((\boldsymbol{P}-pp^\top) \circ \boldsymbol{A}^\top\boldsymbol{A})e.
        \end{split}
    \end{equation}
    
    Since, by assumption, we have $\boldsymbol{P}-pp^\top \preceq \textbf{Diag}(p\circ v)$, we can further bound
    \begin{align*}
        e^\top ((\boldsymbol{P}-pp^\top) \circ \boldsymbol{A}^\top\boldsymbol{A})e \leq e^\top (\textbf{Diag}(p\circ v)\circ \boldsymbol{A}^\top\boldsymbol{A}) e = \sum_{i=1}^n p_iv_i\norm{a_i}^2.
    \end{align*}
    To obtain \eqref{eq:key_inequality_app}, it remains to combine this with \eqref{eq:variance_of_X}.
    The inequality $v_i\ \geq 1-p_i$ follows by comparing the diagonal elements of the two matrices in \eqref{eq:ESO_app}.
    Consider now the independent sampling. Clearly,
    \begin{align*}
        \boldsymbol{P} - pp^\top = 
        \begin{bmatrix}
            p_{1}(1-p_1) & 0& \dots  & 0 \\
            0 & p_{2}(1-p_2) & \dots  & 0 \\
            \vdots & \vdots & \ddots & \vdots \\
            0 & 0  & \dots  & p_{n}(1-p_n)
        \end{bmatrix} 
        =\textbf{Diag}(p_1 v_1,\dots, p_n v_n),
    \end{align*}
    which implies $v_i=  1-p_i$.
\end{proof}

The second lemma bounds the variance of the estimator obtained using the sampling without replacement. The provided lemma is adopted from \cite{mishchenko2020random}. 

\begin{lemma}\label{lem:sampling_wo_replacement}
	Let $\zeta_1,\dotsc, \zeta_n\in \R^d$ be fixed vectors, $\Tilde{\zeta} \eqdef \frac{1}{n}\sum_{i=1}^n \zeta_i$ be their average and $\sigma^2 \eqdef \frac{1}{n}\sum_{i=1}^n \norm{\zeta_i-\Tilde{\zeta}}^2$ be the population variance. Fix any $k\in\{1,\dotsc, n\}$, let $\zeta_{\pi_1}, \dotsc \zeta_{\pi_k}$ be sampled uniformly without replacement from $\{\zeta_1,\dotsc, \zeta_n\}$ and $\Tilde{\zeta}^k_\pi$ be their average. Then, the sample average and variance are given by
	\begin{align}
	\label{eq:sampling_wo_replacement}
		\E{\Tilde{\zeta}^k_\pi}=\Tilde{\zeta}, && \E{\norm{\Tilde{\zeta}^k_\pi - \Tilde{\zeta}}^2}= \frac{n-k}{k(n-1)}\sigma^2.
	\end{align}
\end{lemma}
\begin{proof}
	The first claim follows by linearity of expectation and uniformity of  sampling:
	\[
		\E{\Tilde{\zeta}^k_\pi} 
		= \frac{1}{k}\sum_{i=1}^k \E{\zeta_{\pi_i}}
		= \frac{1}{k}\sum_{i=1}^k \Tilde{\zeta}
		= \Tilde{\zeta}.
	\]
	To prove the second claim, let us first establish that the identity $\mathrm{cov}(\zeta_{\pi_i}, \zeta_{\pi_j})=-\frac{\sigma^2}{n-1}$ holds  for any $i\neq j$. Indeed,
	\begin{align*}
		\mathrm{cov}(\zeta_{\pi_i}, \zeta_{\pi_j})
		&= \E{ \inp{\zeta_{\pi_i} - \Tilde{\zeta}}{ \zeta_{\pi_j} - \Tilde{\zeta}}}
		= \frac{1}{n(n-1)}\sum_{l=1}^n\sum_{m=1,m\neq l}^n\inp{\zeta_l - \Tilde{\zeta}}{ \zeta_m - \Tilde{\zeta}} \\
		&= \frac{1}{n(n-1)}\sum_{l=1}^n\sum_{m=1}^n\inp{\zeta_l - \Tilde{\zeta}}{ \zeta_m - \Tilde{\zeta}} - \frac{1}{n(n-1)}\sum_{l=1}^n \norm{\zeta_l - \Tilde{\zeta}}^2 \\
		&= \frac{1}{n(n-1)}\sum_{l=1}^n \inp{\zeta_l - \Tilde{\zeta}}{ \sum_{m=1}^n(\zeta_m - \Tilde{\zeta})} - \frac{\sigma^2}{n-1} \\
		&=-\frac{\sigma^2}{n-1}.
	\end{align*}
This identity helps us to establish the formula for sample variance:
	\begin{align*}
		\E{\norm{\Tilde{\zeta}^k_\pi - \Tilde{\zeta}}^2}
		&= \frac{1}{k^2} \sum_{i=1}^k\sum_{j=1}^k \mathrm{cov}(\zeta_{\pi_i}, \zeta_{\pi_j}) \\
    &= \frac{1}{k^2}\E{\sum_{i=1}^k \norm{\zeta_{\pi_i} - \Tilde{\zeta}}^2} + \sum_{i=1}^k\sum_{j=1,j\neq i}^{k} \mathrm{cov}(\zeta_{\pi_i}, \zeta_{\pi_j})  \\
		&=\frac{1}{k^2}\left(k\sigma^2 - k(k-1)\frac{\sigma^2}{n-1}\right)
    = \frac{n-k}{k(n-1)}\sigma^2.
    \qedhere
	\end{align*}
\end{proof}

\section{Technical lemmas}

    Next, we state a relaxed triangle inequality for the squared
    $\ell_2$ norm.
\begin{lemma}[Young's inequality / Relaxed triangle inequality]\label{lem:norm-sum}
    Let $\{v_1,\dots,v_\tau\}$ be $\tau$ vectors in $\R^d$. Then the following are true
    \begin{align}
        \label{eq:triangle}
        \norm{v_i + v_j}^2 \leq (1 + a)\norm{v_i}^2 + \rbr*{1 + \frac{1}{a}}\norm{v_j}^2 \text{ for any } a >0
    \end{align}
    and
    \begin{align}
        \label{eq:sum_upper}
        \norm*{\sum_{i=1}^\tau v_i}^2 \leq \tau \sum_{i=1}^\tau\norm*{v_i}^2.
    \end{align}
\end{lemma}
\begin{proof}
The proof of the first statement for any $a > 0$ follows from the identity:
\[
    \norm*{v_i + v_j}^2 = (1 + a)\norm{v_i}^2 + \rbr*{1 + \frac{1}{a}}\norm*{v_j}^2 - \norm*{\sqrt{a} v_i + \frac{1}{\sqrt{a}}v_j}^2\,.
\]
For the second inequality, we use the convexity of
$x \rightarrow \norm{x}^2$ and Jensen's inequality
\[
     \norm*{\frac{1}{\tau}\sum_{i=1}^\tau v_i }^2 \leq \frac{1}{\tau}\sum_{i=1}^\tau\norm*{ v_i }^2\,.
\]
\end{proof}

\refstepcounter{chapter}%
\chapter*{\thechapter \quad Appendix: Natural compression for distributed deep learning}
\label{appendix:c_nat}

\section{Extra experiments}
\label{sec:extra_exp}

\begin{figure}[H]
\centering
\hfill
\subfigure[No Additional compression.]
{
\includegraphics[width=0.245\textwidth]{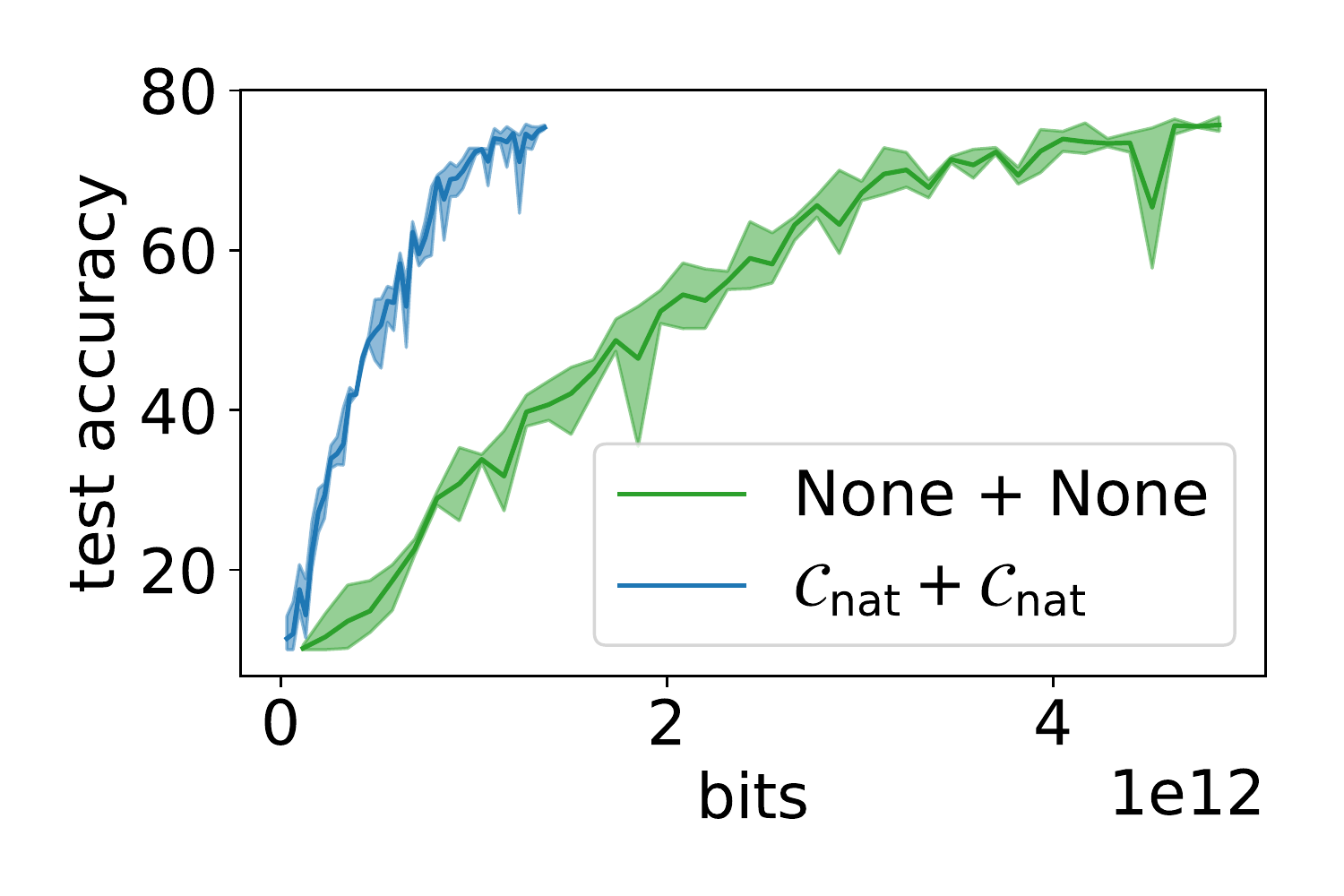}
\includegraphics[width=0.245\textwidth]{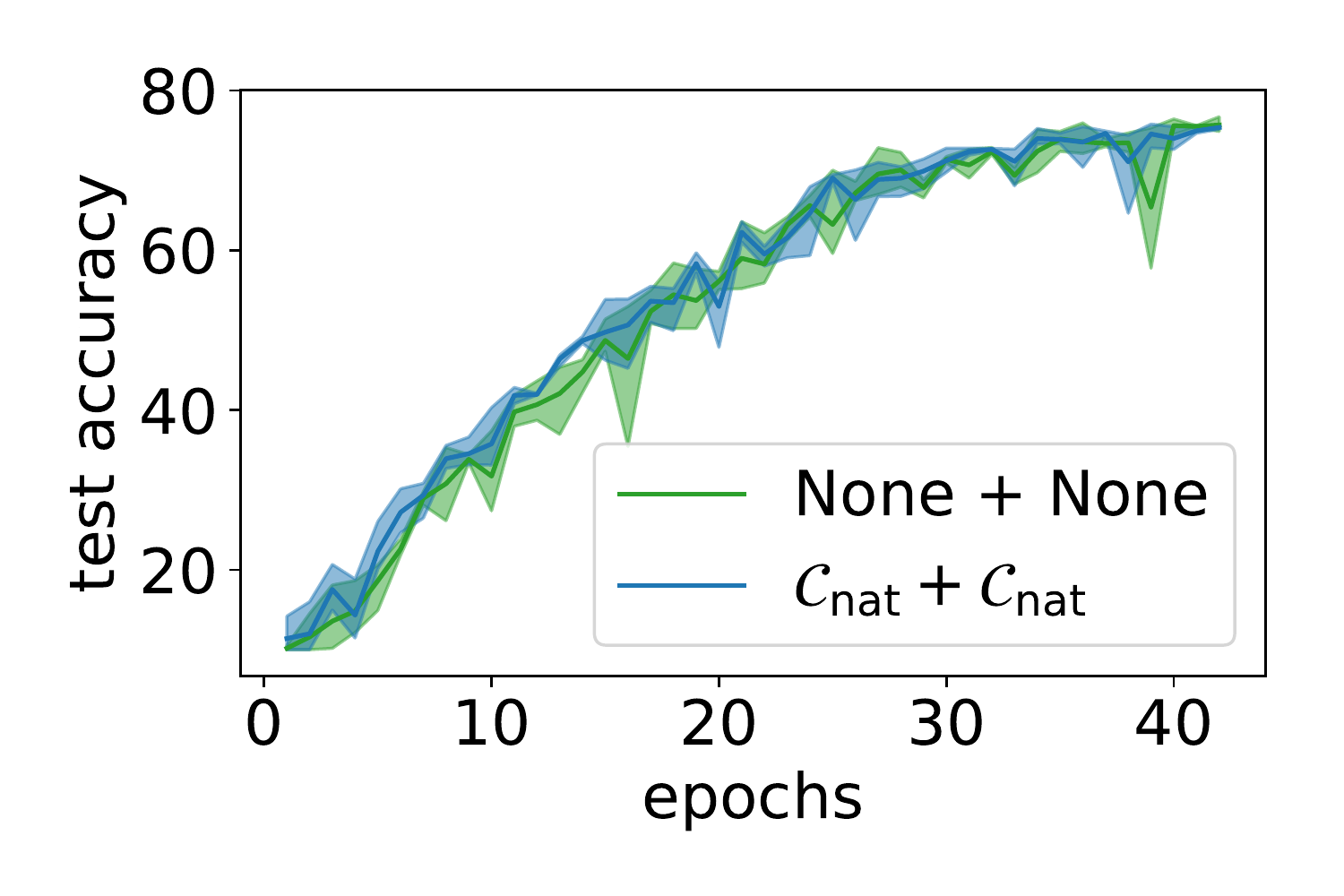}
\includegraphics[width=0.245\textwidth]{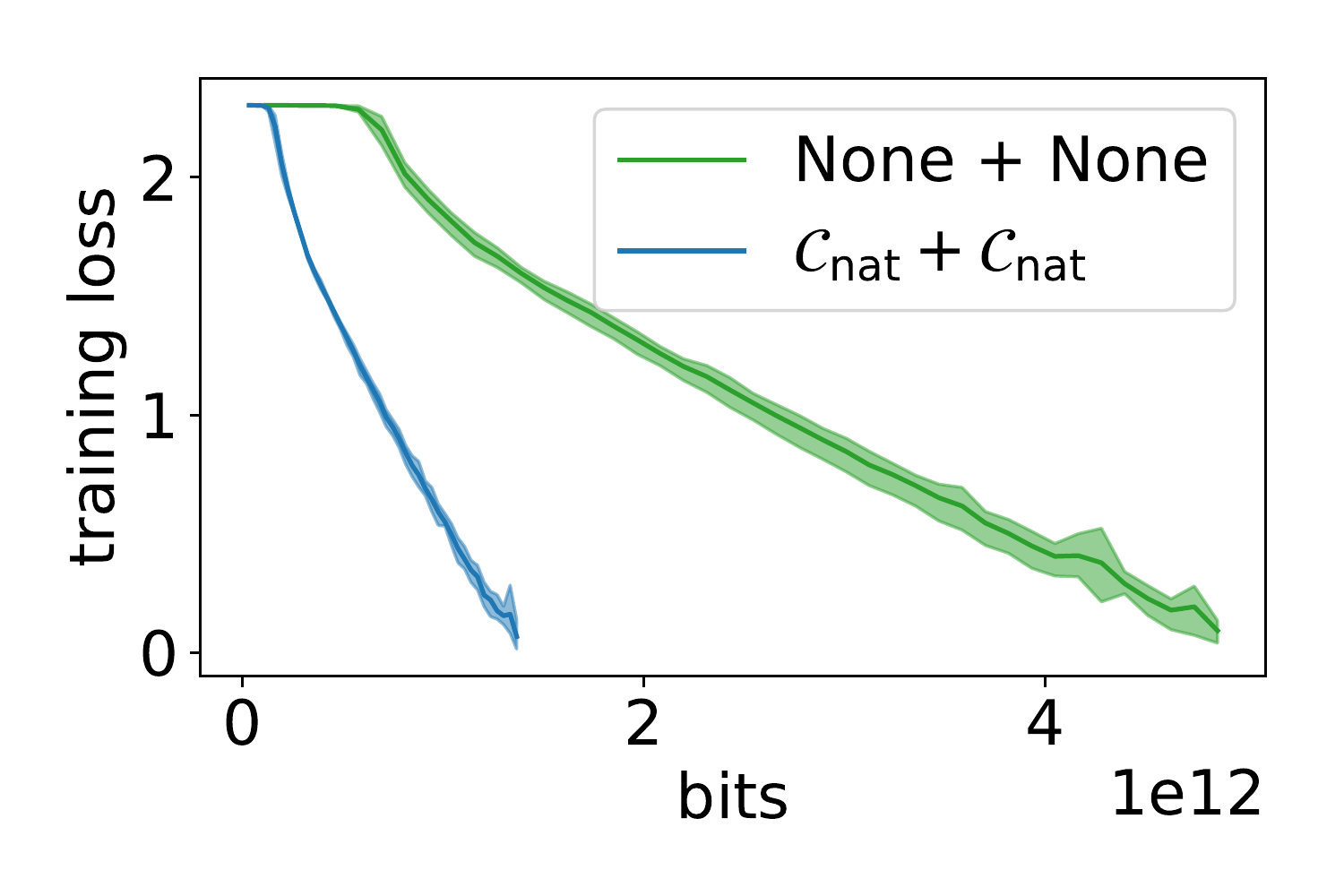}
\includegraphics[width=0.245\textwidth]{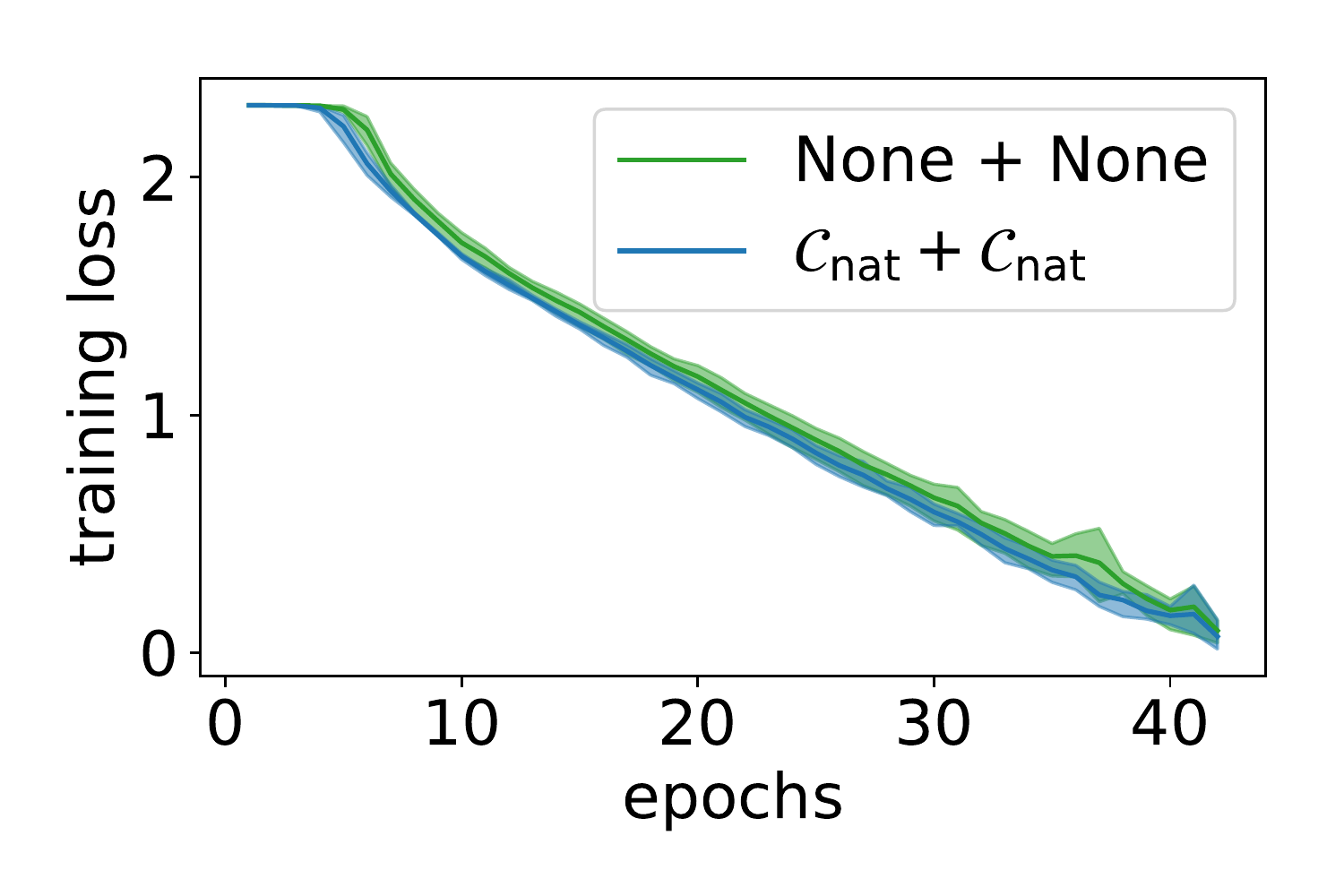}
} \\
\subfigure[Random sparsification, sparsity $10\%$.]
{
\includegraphics[width=0.245\textwidth]{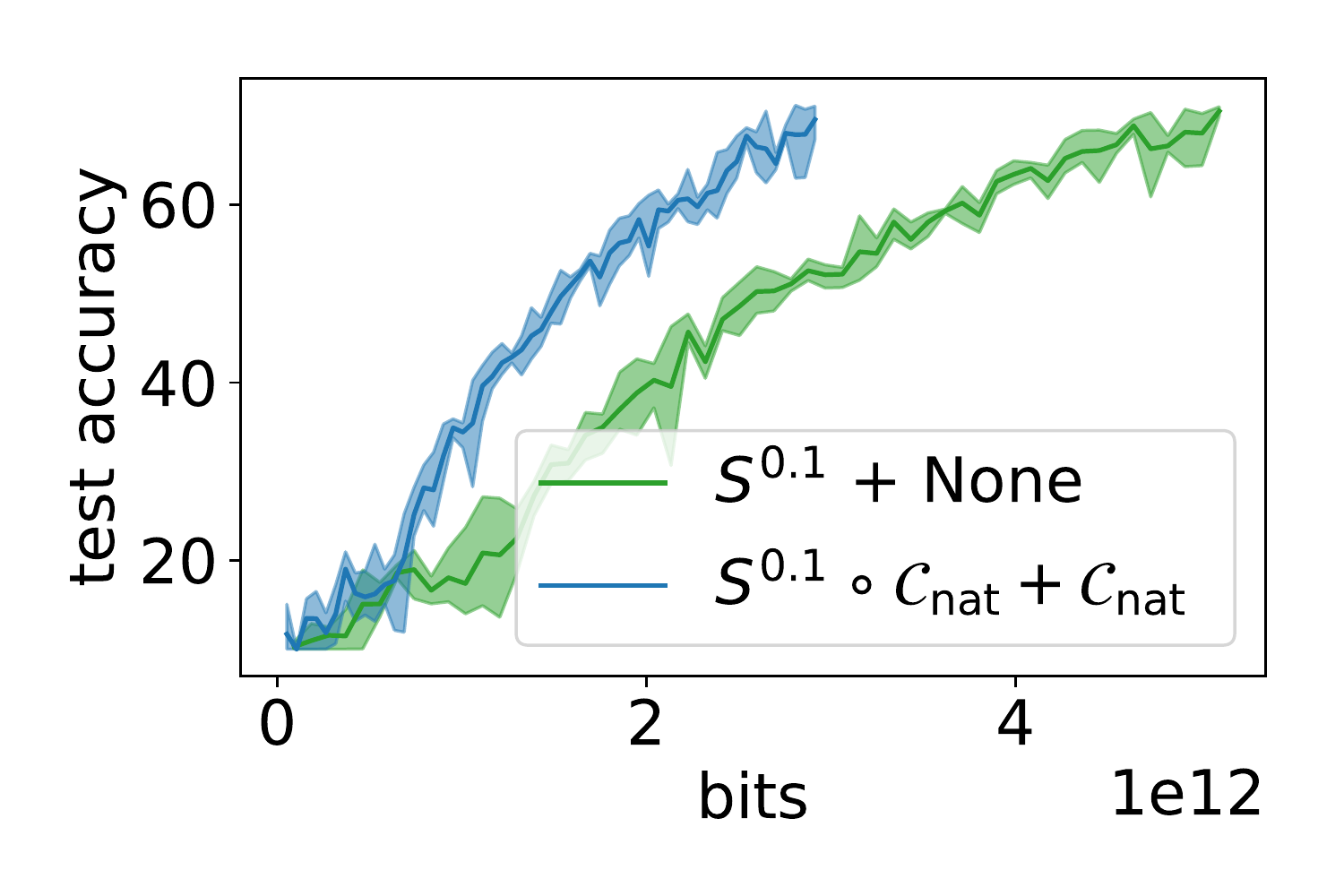}
\includegraphics[width=0.245\textwidth]{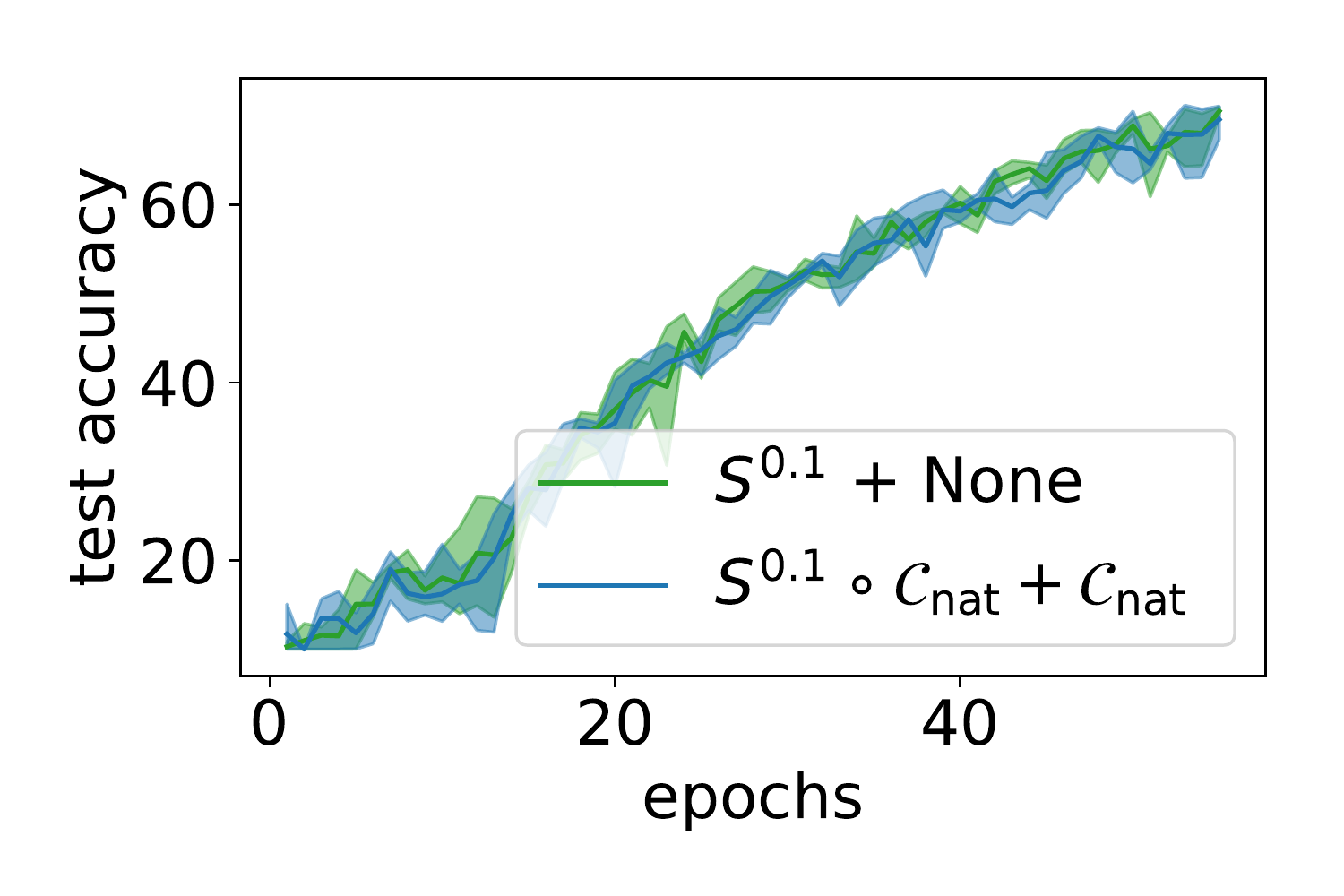}
\includegraphics[width=0.245\textwidth]{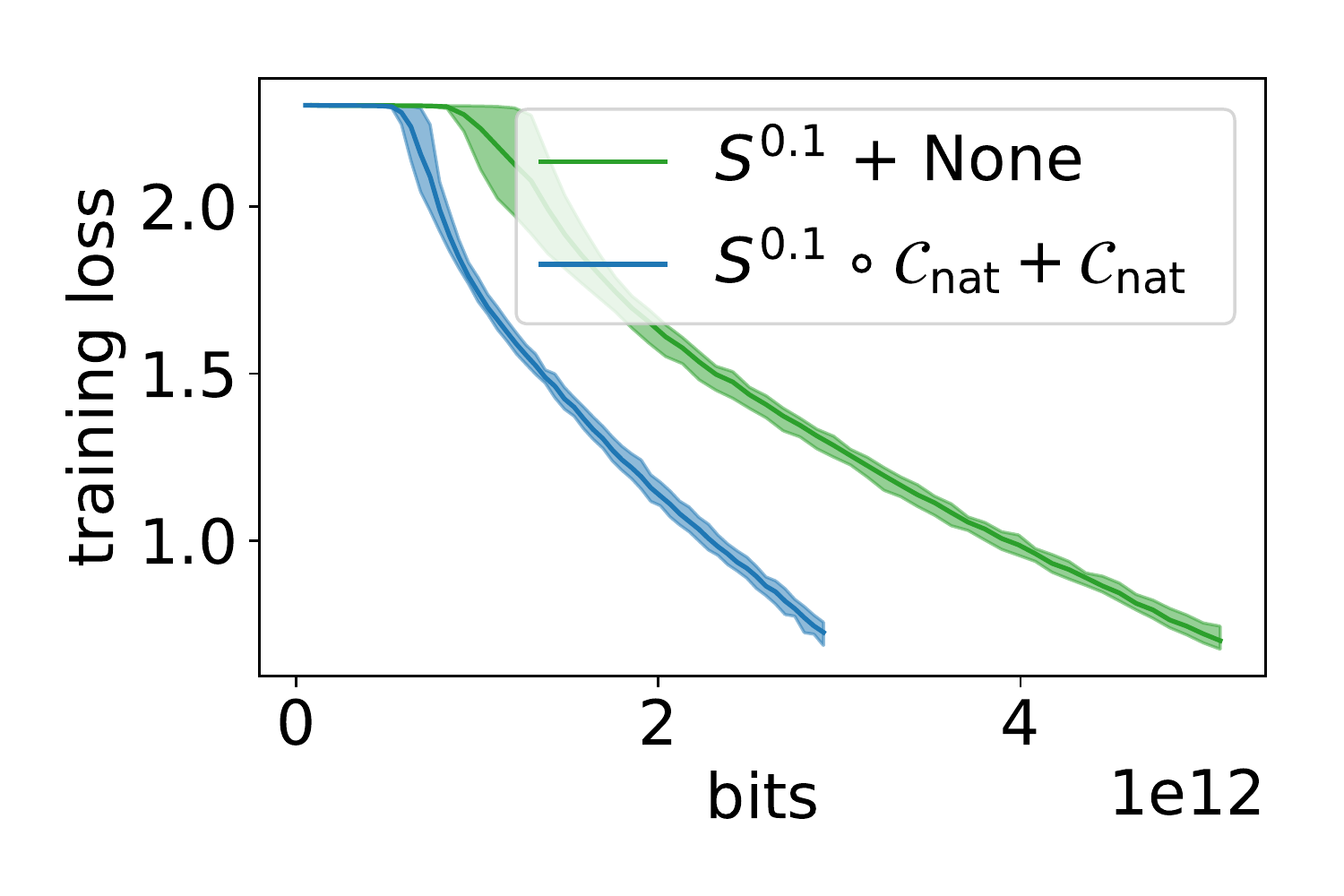}
\includegraphics[width=0.245\textwidth]{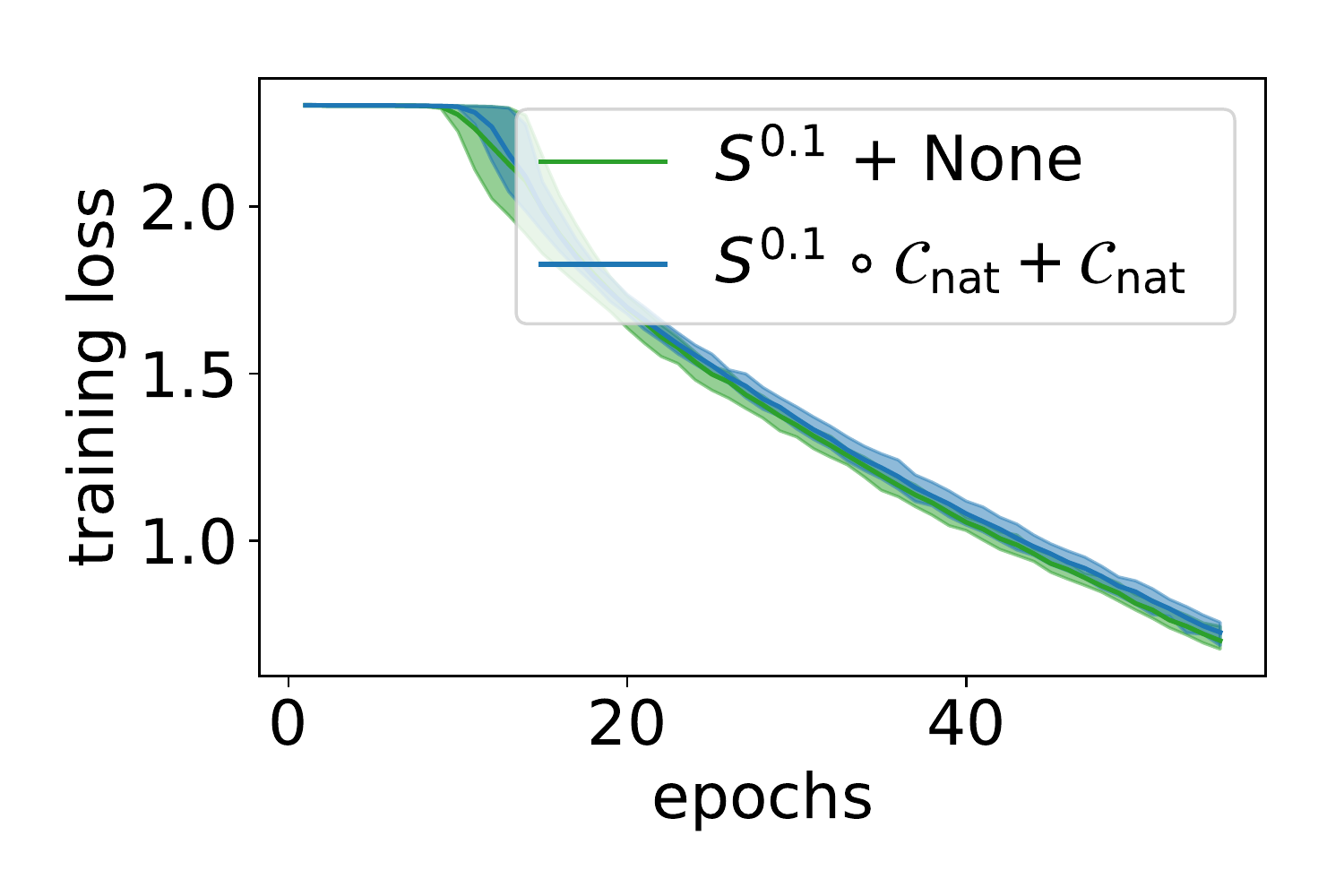} 
} \\
\subfigure[Random sparsification with non-uniform probabilities~\cite{tonko}, sparsity $10\%$.]
{
\includegraphics[width=0.245\textwidth]{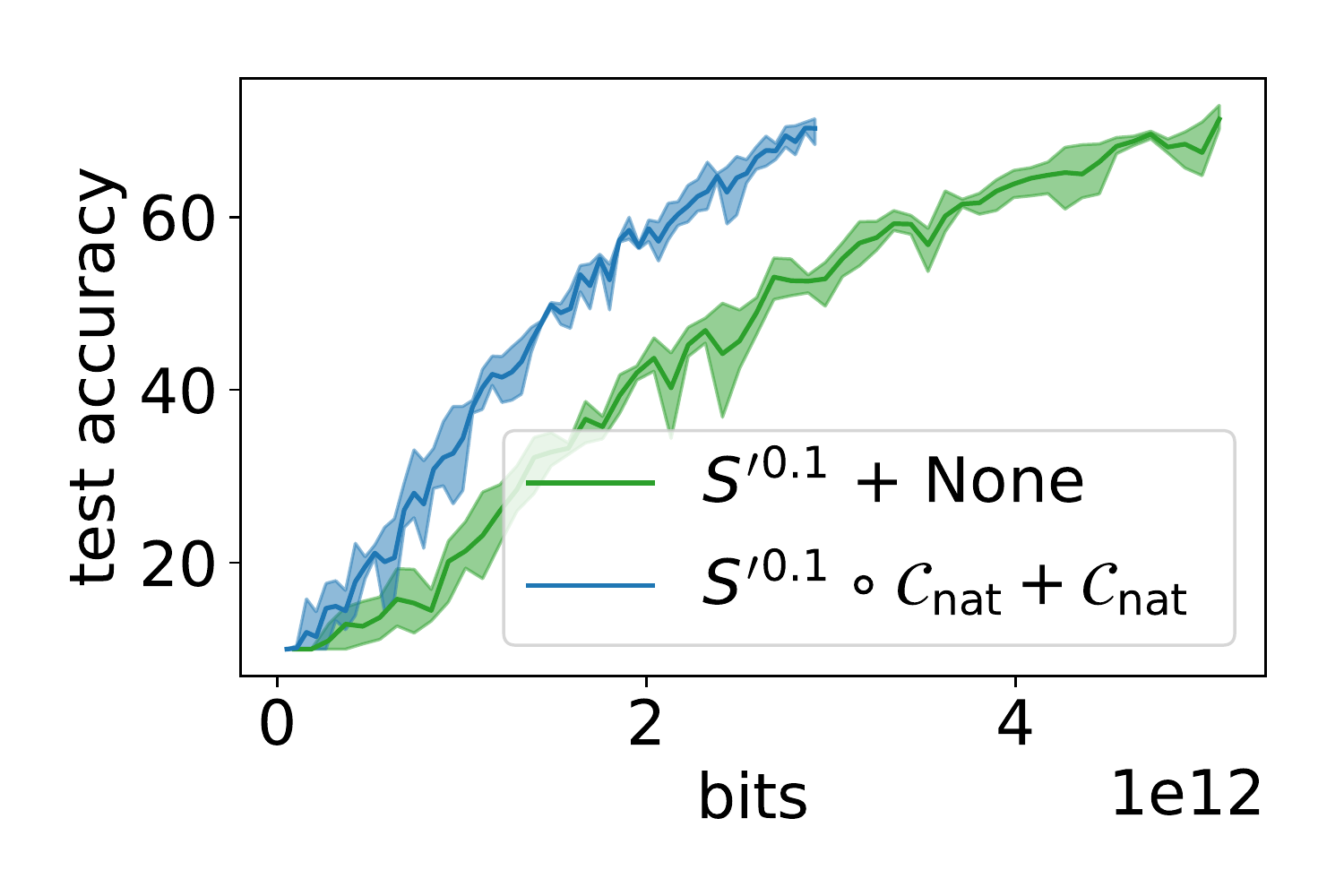}
\includegraphics[width=0.245\textwidth]{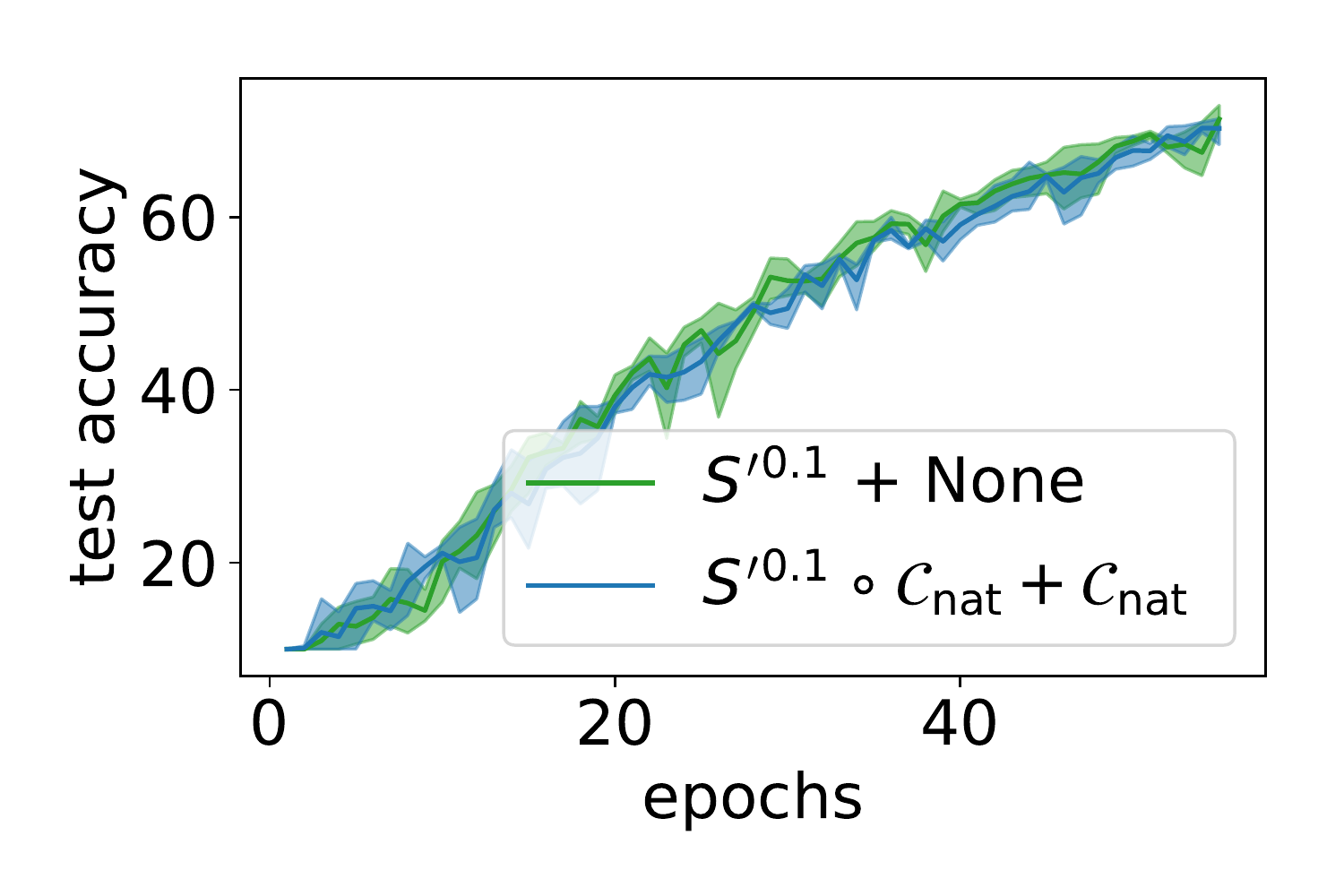}
\includegraphics[width=0.245\textwidth]{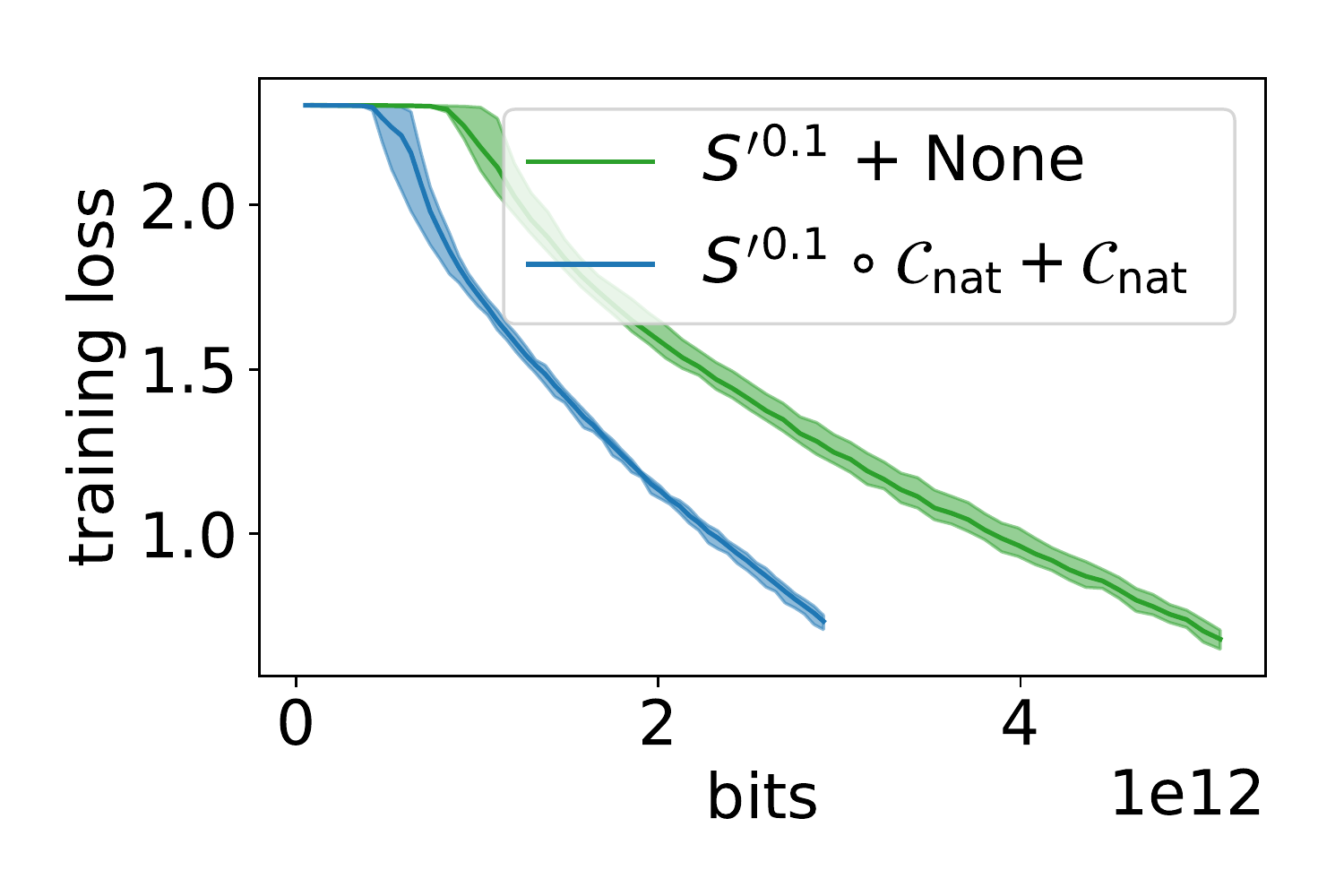}
\includegraphics[width=0.245\textwidth]{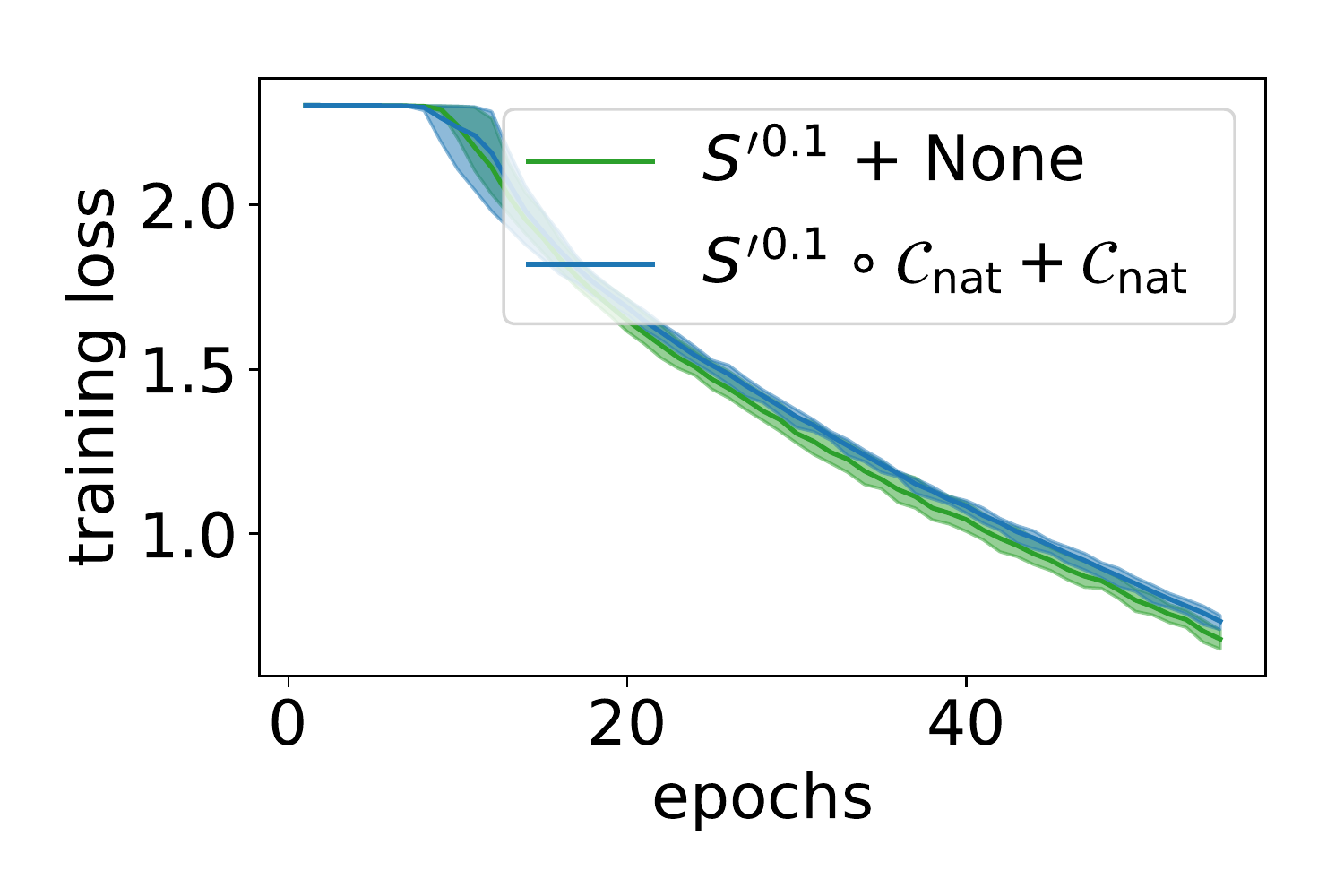} 
} \\
\subfigure[Random dithering, $s = 8$, $u = 2^7$, second norm.]
{
\includegraphics[width=0.245\textwidth]{plots/c_nat/cifar_VGG11_dithering_test_acc_0.08_4_60_bits_.pdf}
\includegraphics[width=0.245\textwidth]{plots/c_nat/cifar_VGG11_dithering_test_acc_0.08_4_60_epochs_.pdf}
\includegraphics[width=0.245\textwidth]{plots/c_nat/cifar_VGG11_dithering_train_loss_0.08_4_60_bits_.pdf}
\includegraphics[width=0.245\textwidth]{plots/c_nat/cifar_VGG11_dithering_train_loss_0.08_4_60_epochs_.pdf} 
} 
\caption{CIFAR10 with VGG11.}
\label{fig:bin_comparison_cifar}
\end{figure}

\begin{figure}[H]
\centering
\hfill
\subfigure[No Additional compression.]
{
\includegraphics[width=0.245\textwidth]{plots/c_nat/mnist_nothing_test_acc_0.1_4_20_bits_.pdf}
\includegraphics[width=0.245\textwidth]{plots/c_nat/mnist_nothing_test_acc_0.1_4_20_epochs_.pdf}
\includegraphics[width=0.245\textwidth]{plots/c_nat/mnist_nothing_train_loss_0.1_4_20_bits_.pdf}
\includegraphics[width=0.245\textwidth]{plots/c_nat/mnist_nothing_train_loss_0.1_4_20_epochs_.pdf}
} \\
\subfigure[Random sparsification, sparsity $10\%$.]
{
\includegraphics[width=0.245\textwidth]{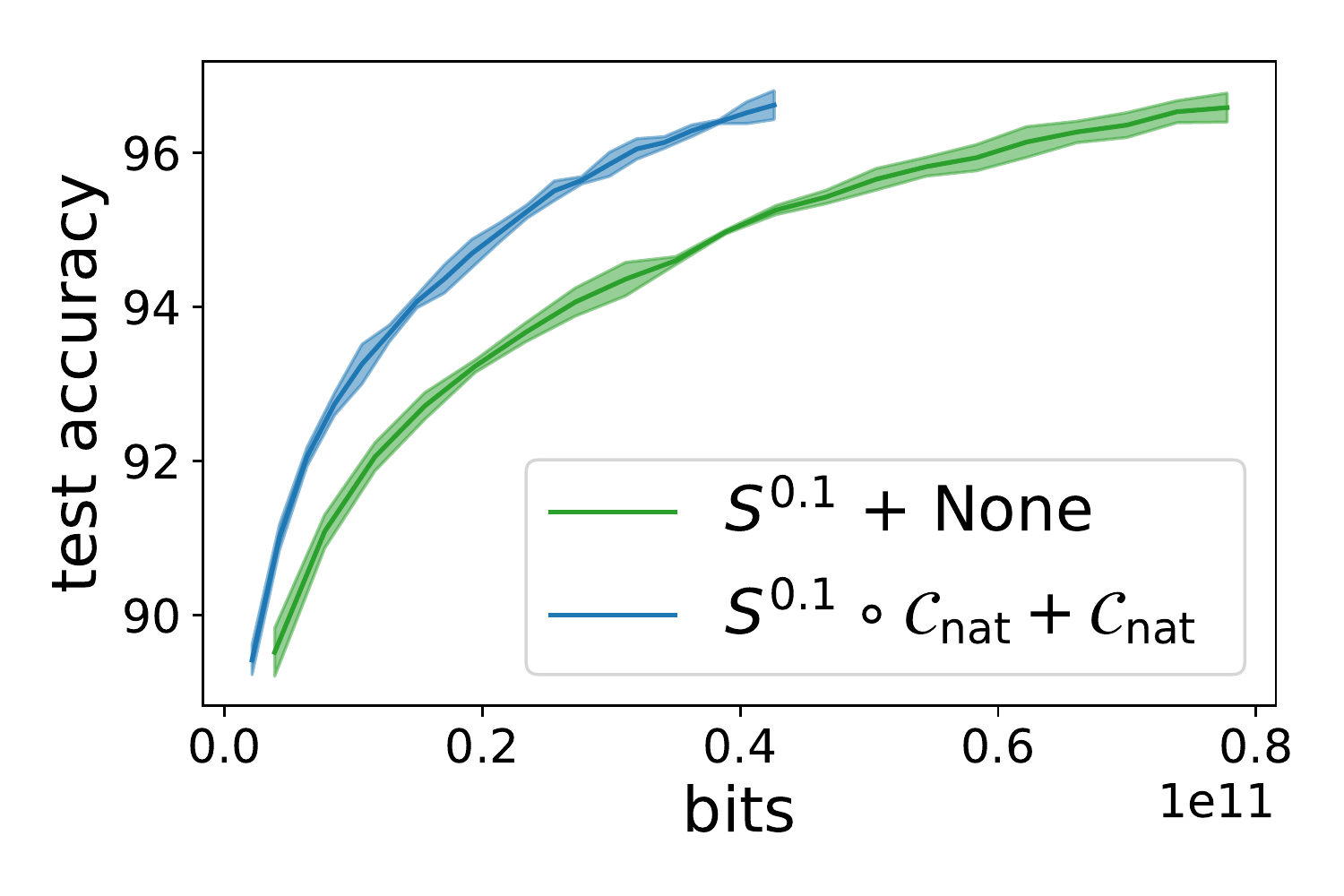}
\includegraphics[width=0.245\textwidth]{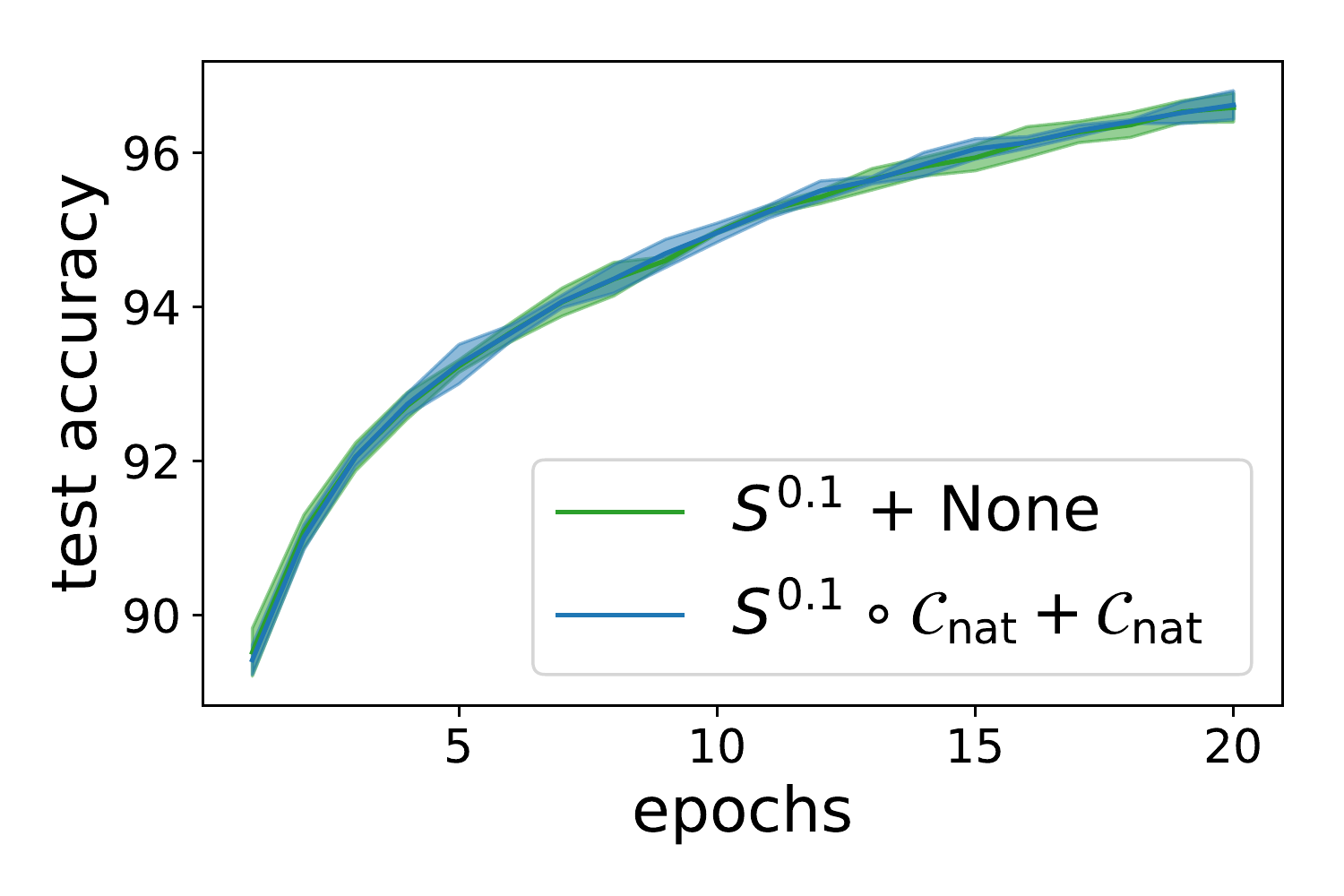}
\includegraphics[width=0.245\textwidth]{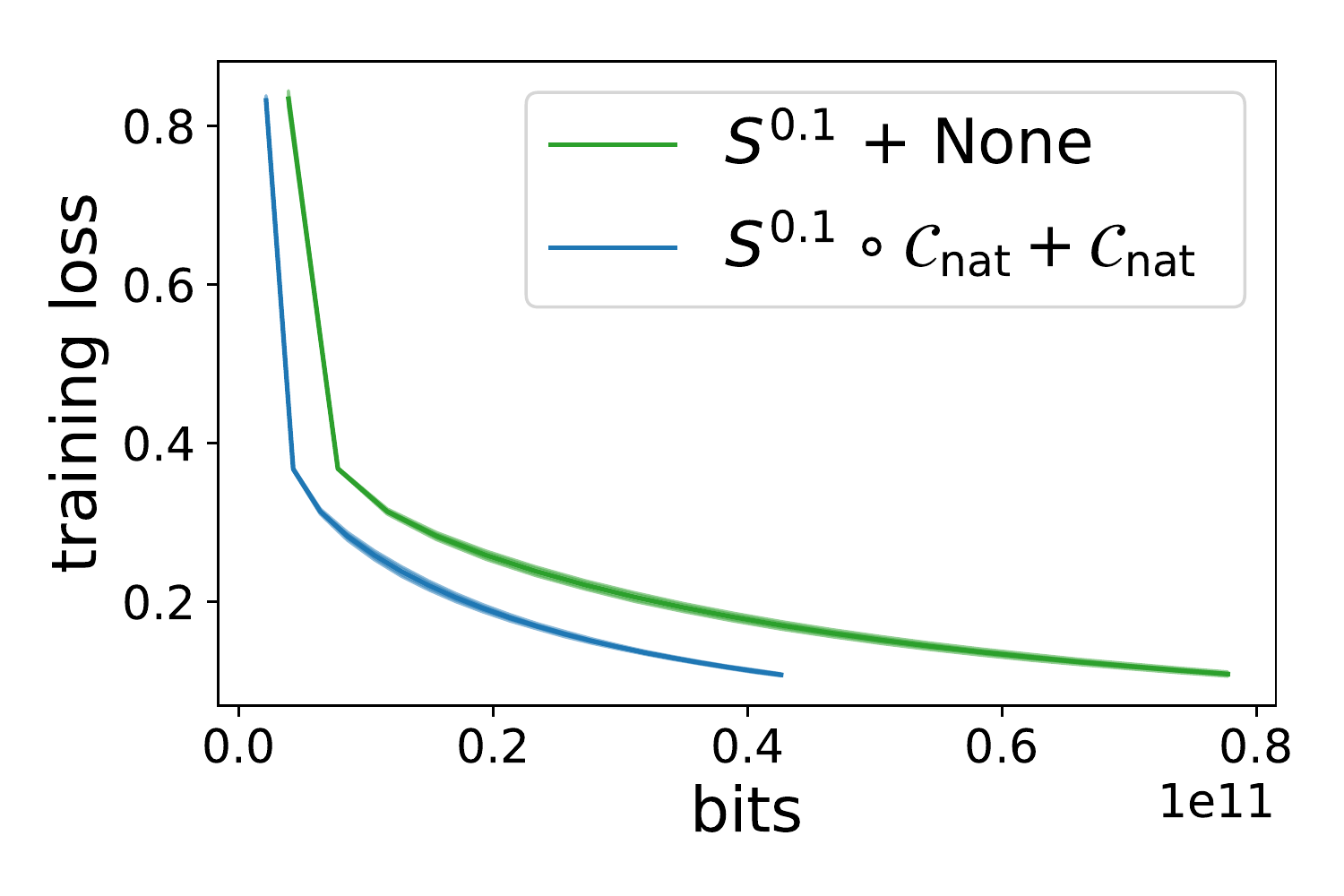}
\includegraphics[width=0.245\textwidth]{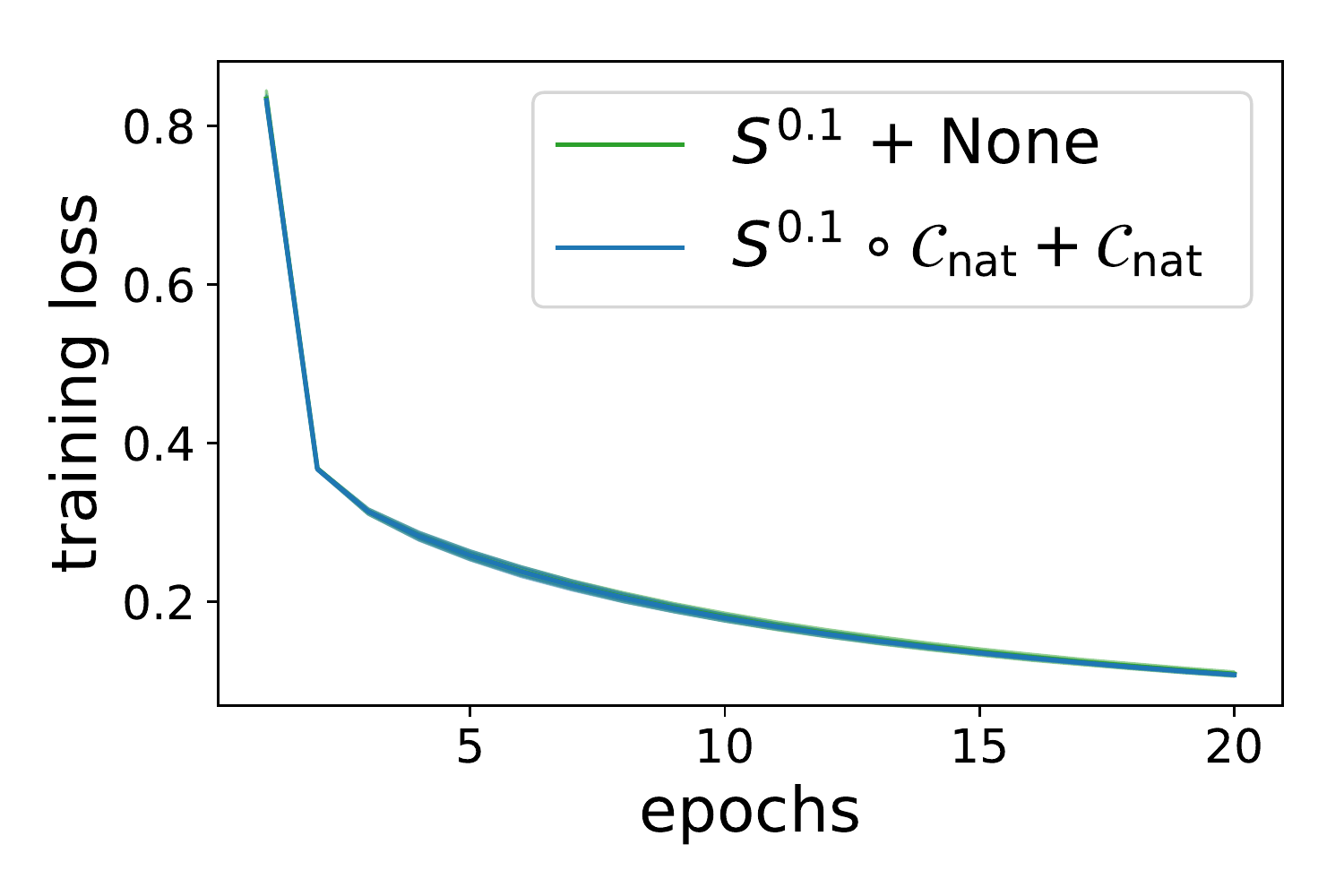} 
} \\
\subfigure[Random sparsification with non-uniform probabilities~\cite{tonko}, sparsity $10\%$.]
{
\includegraphics[width=0.245\textwidth]{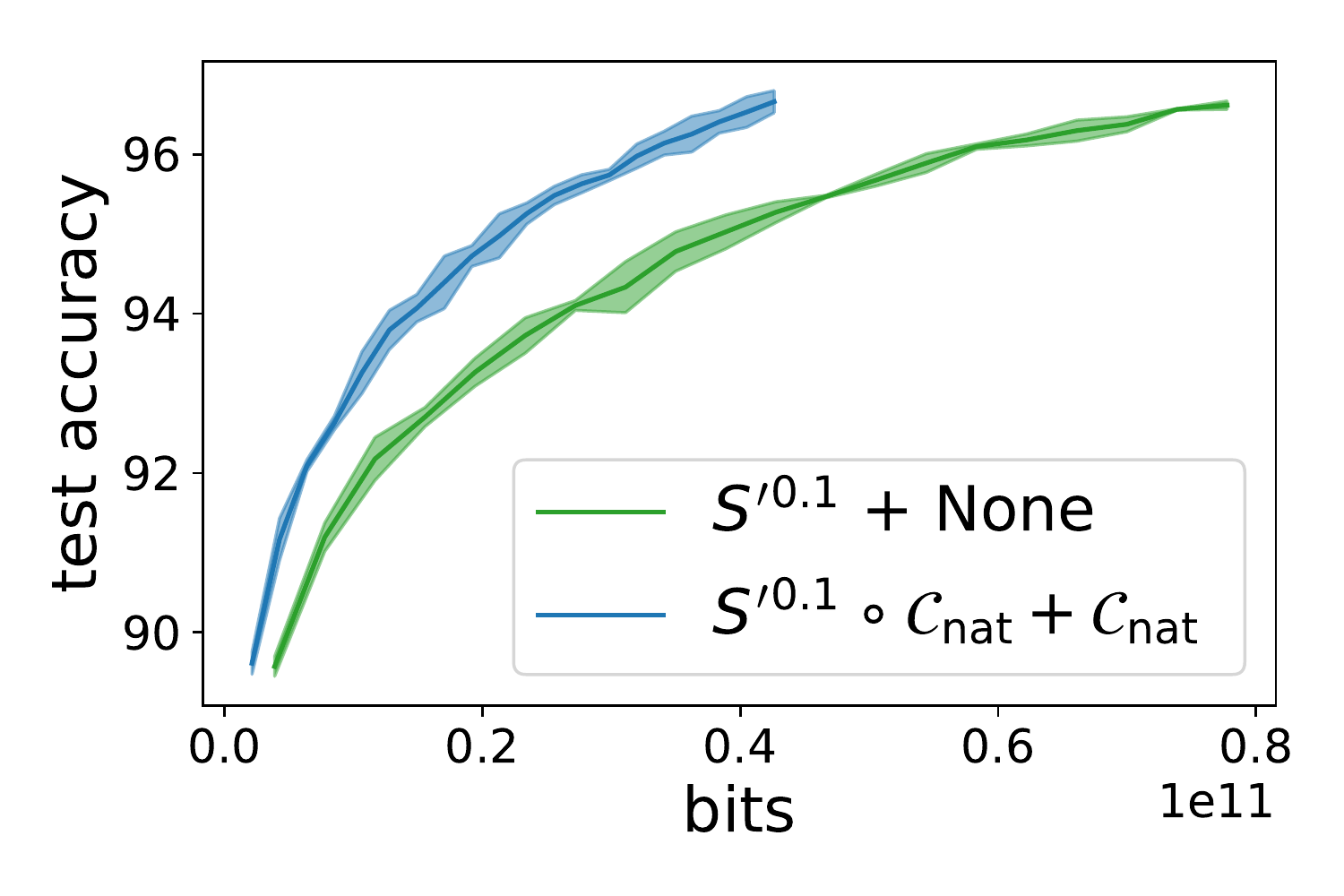}
\includegraphics[width=0.245\textwidth]{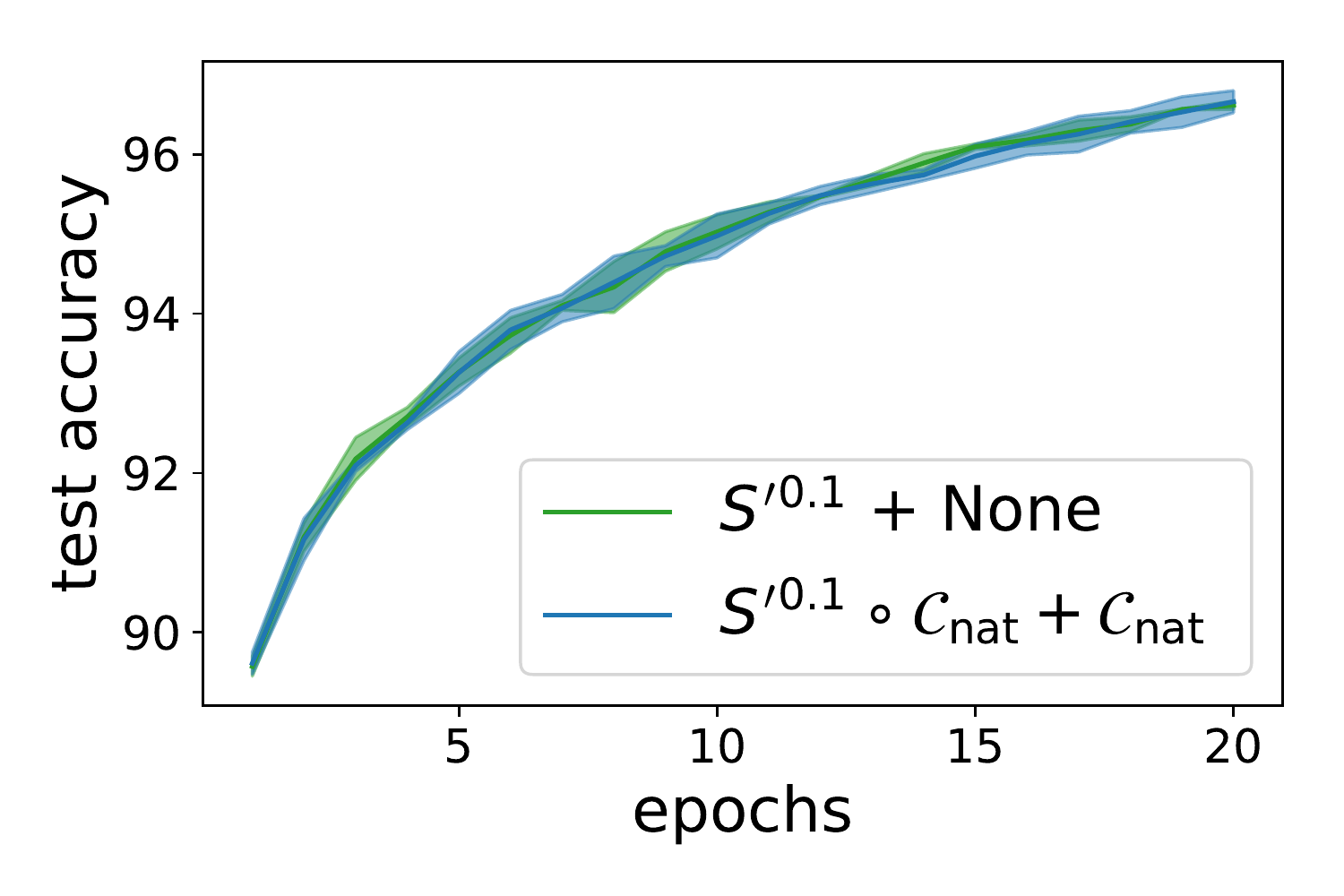}
\includegraphics[width=0.245\textwidth]{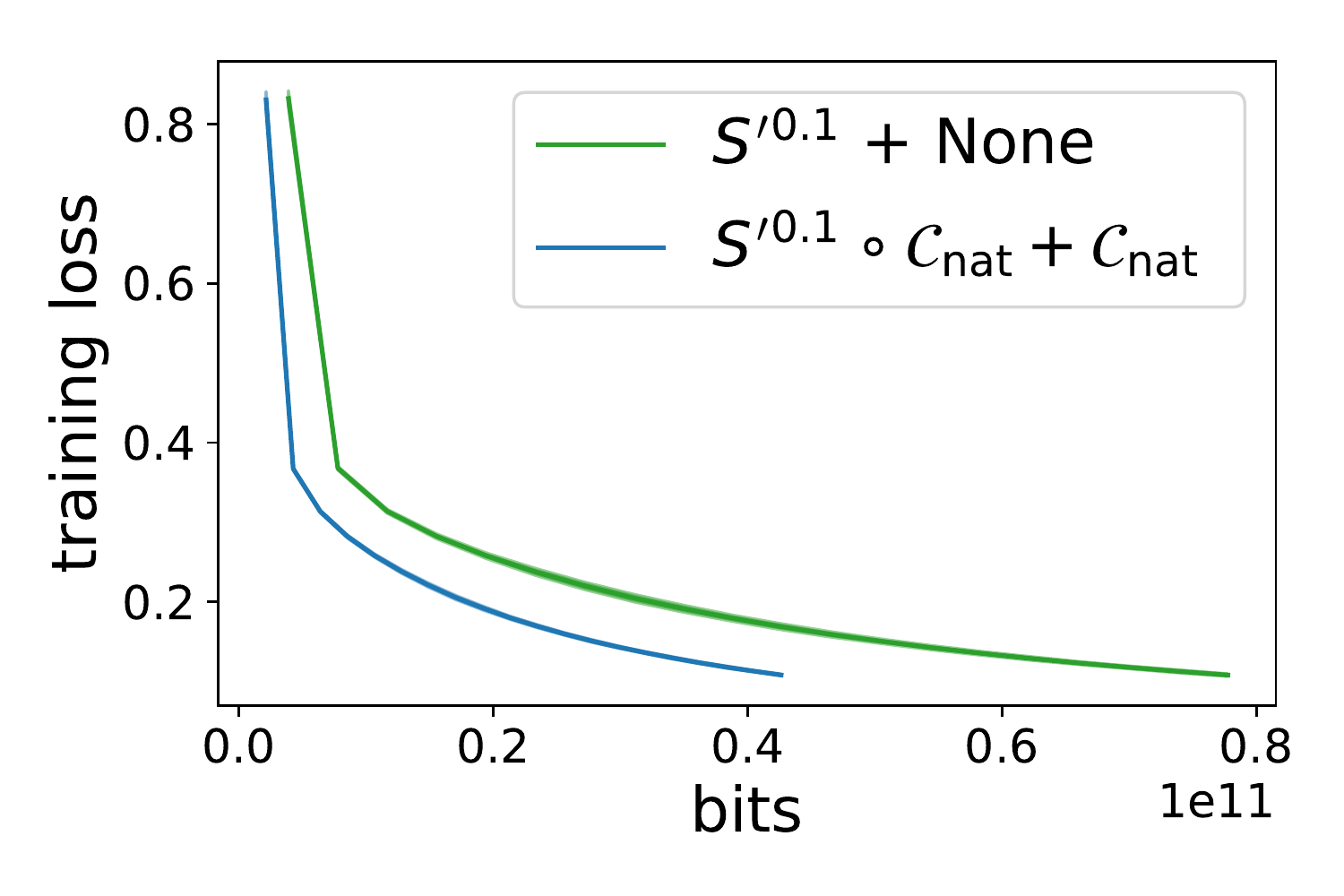}
\includegraphics[width=0.245\textwidth]{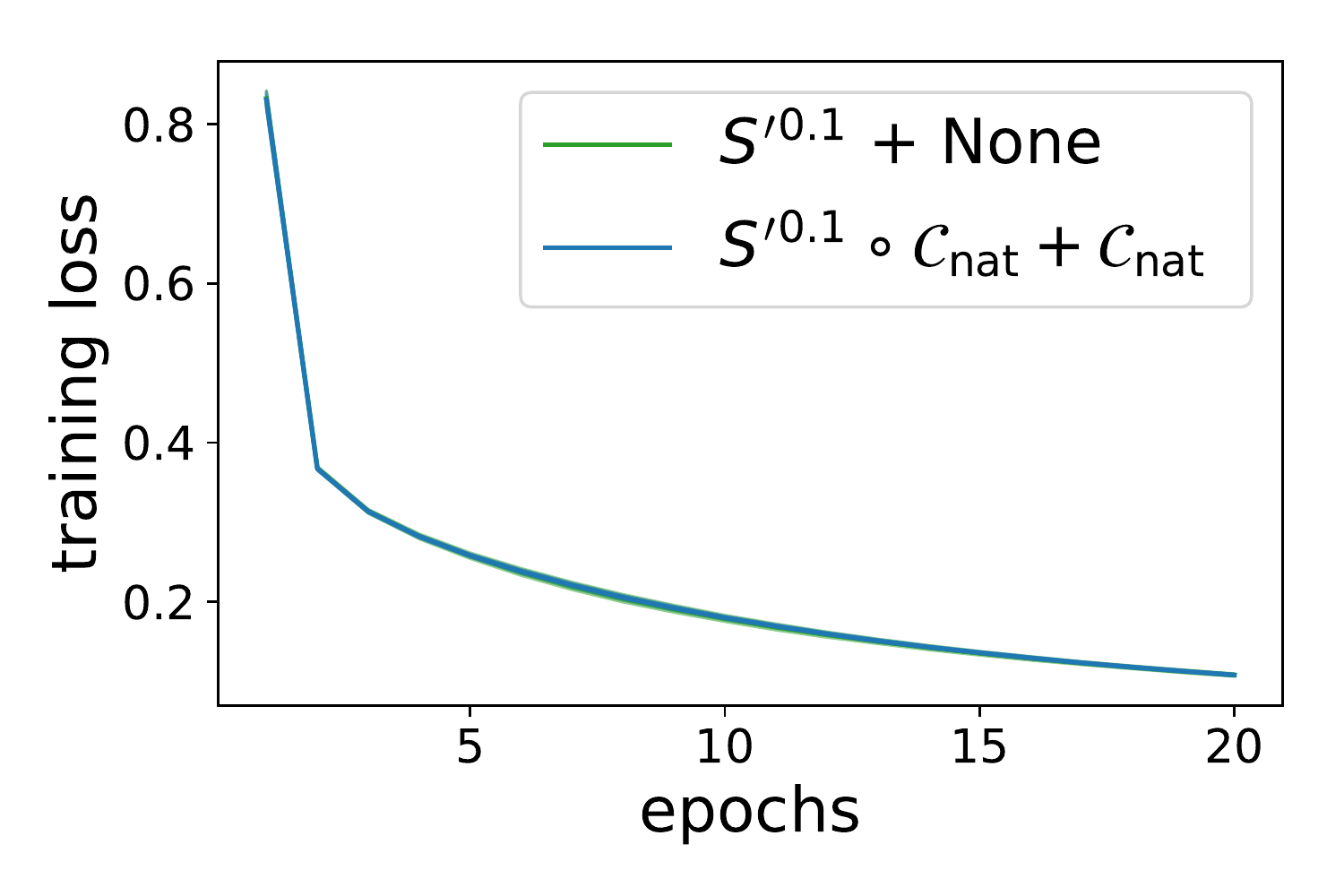} 
} \\
\subfigure[Random dithering, $s = 8$, $u = 2^7$, second norm.]
{
\includegraphics[width=0.245\textwidth]{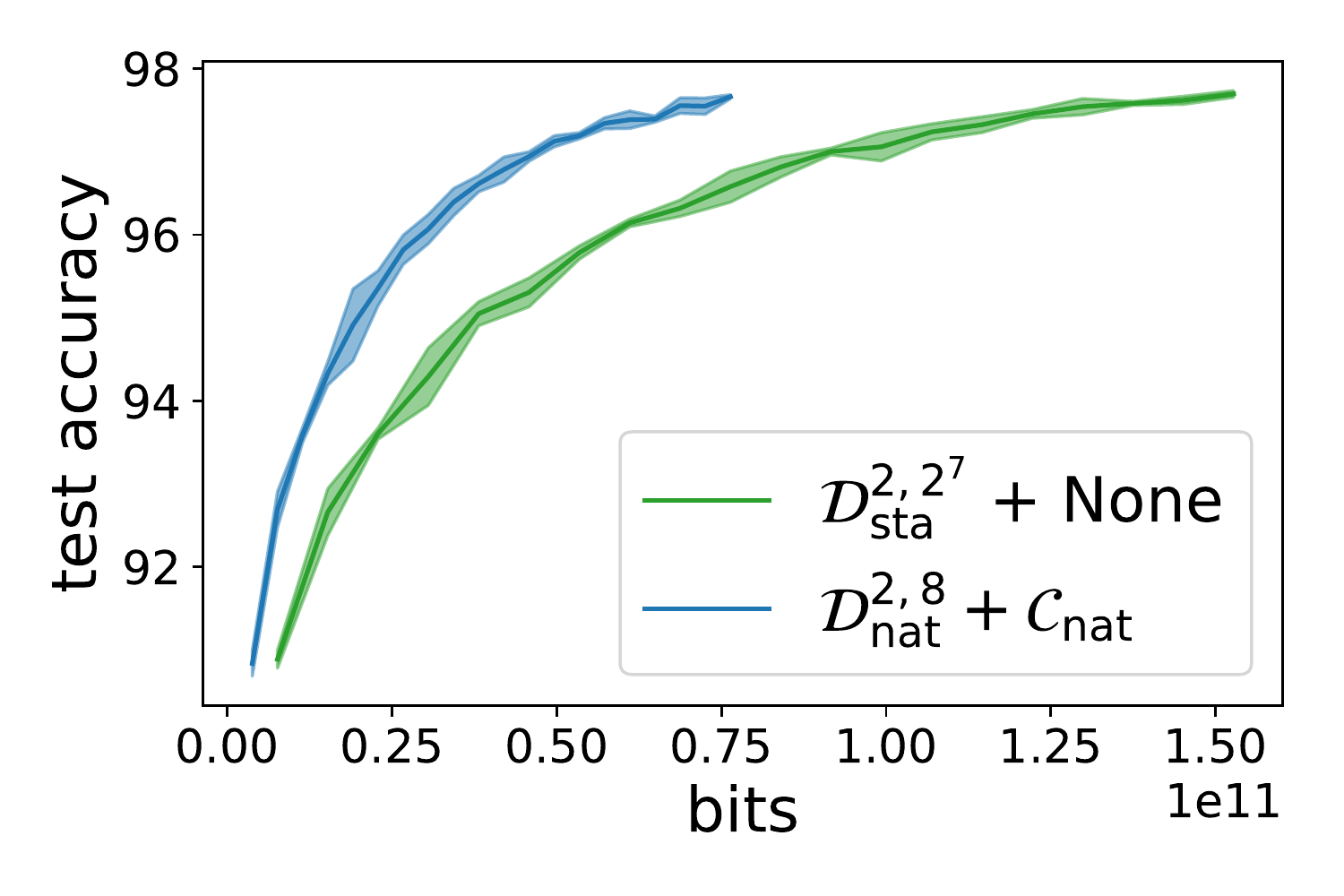}
\includegraphics[width=0.245\textwidth]{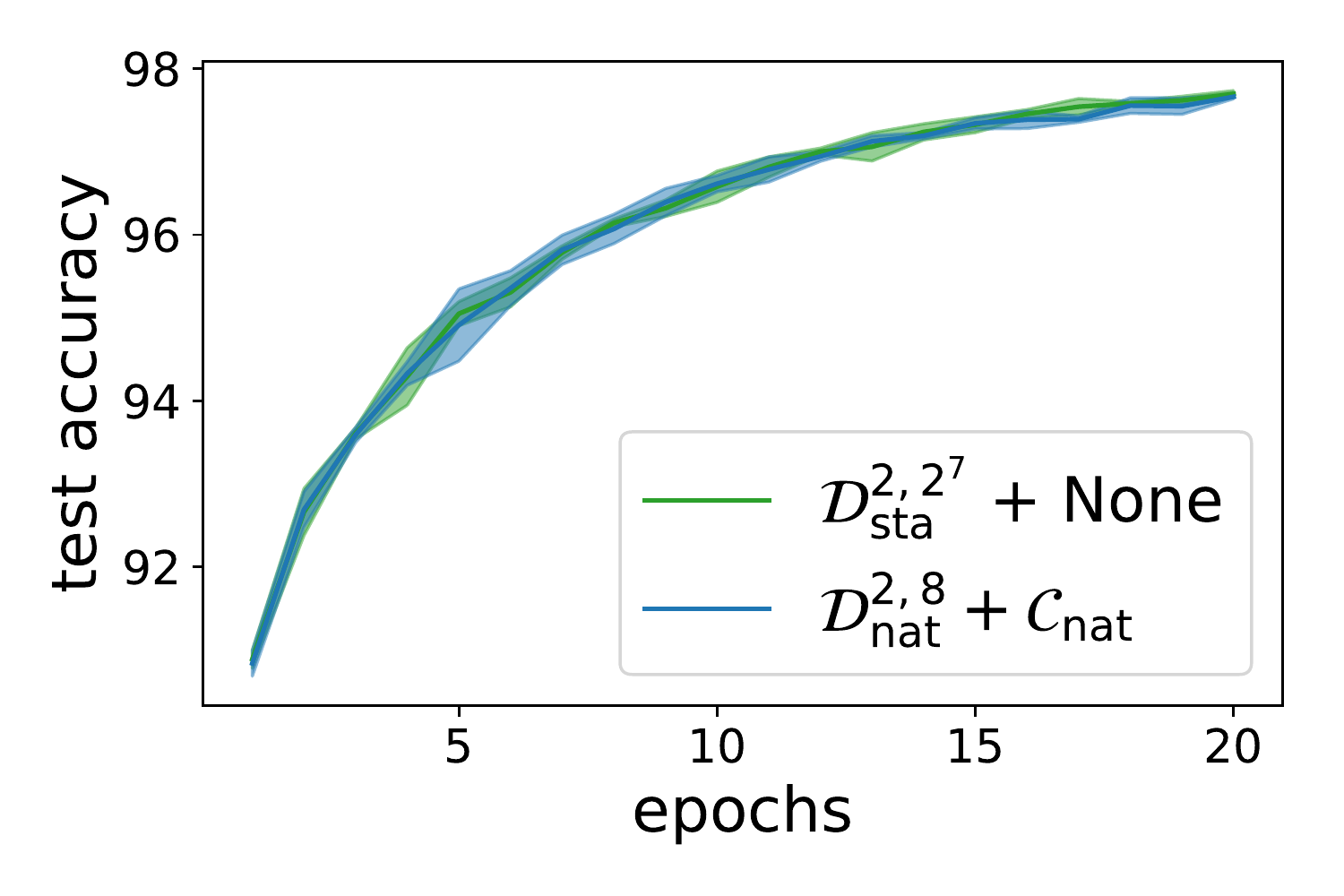}
\includegraphics[width=0.245\textwidth]{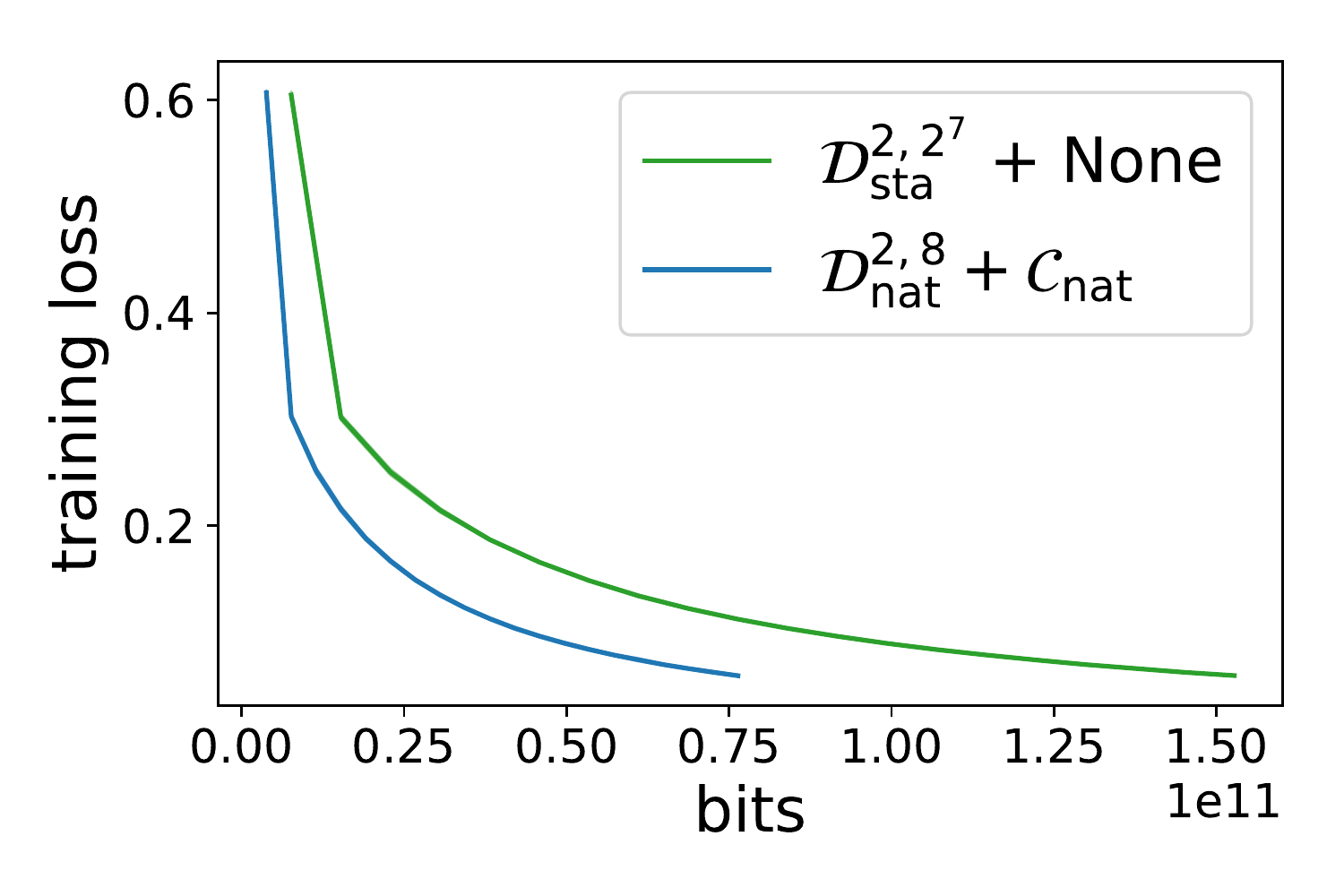}
\includegraphics[width=0.245\textwidth]{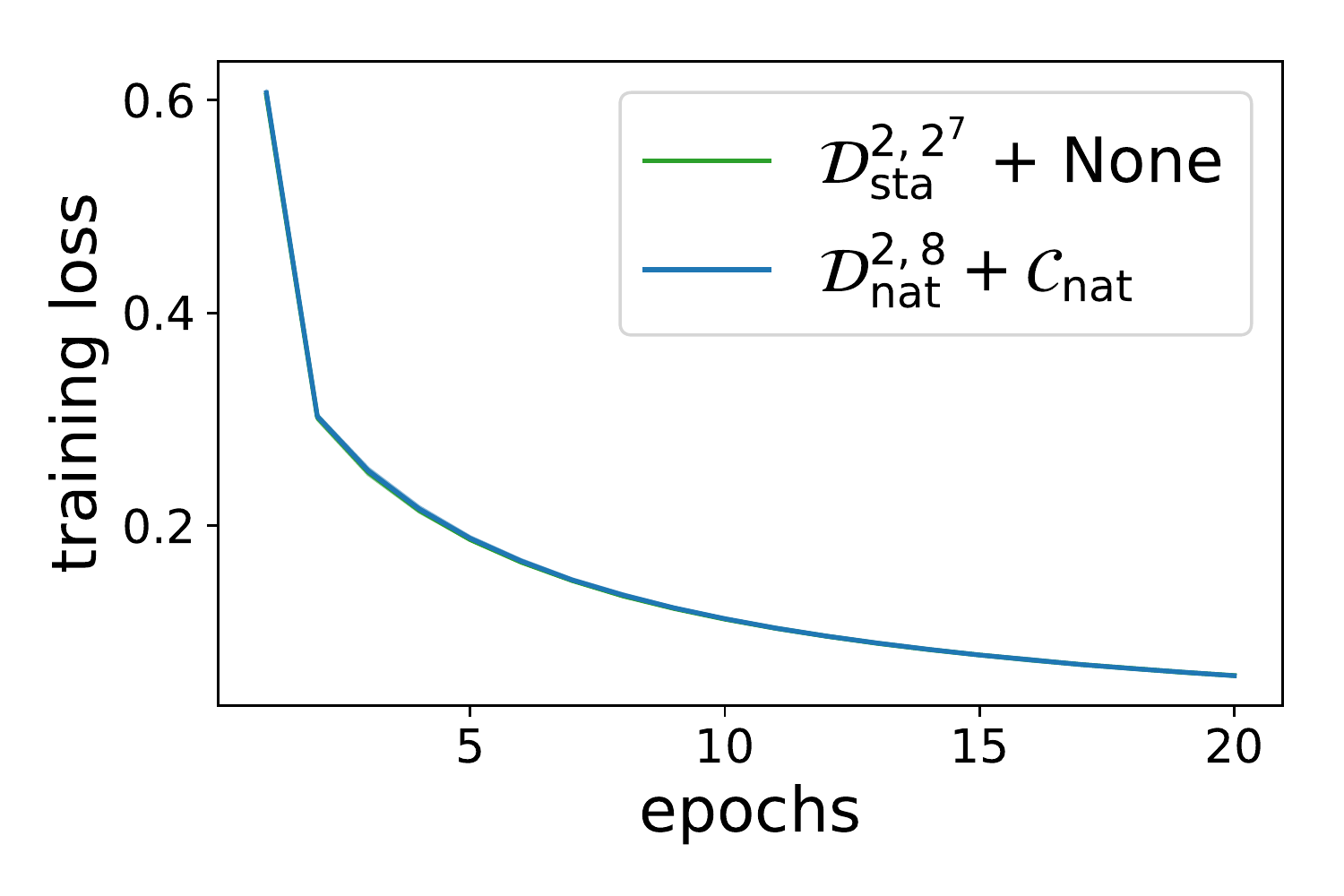} 
} 
\caption{MNIST with $2$ fully conected layers.}
\label{fig:bin_comparison_mnist}
\end{figure}

\newpage

\section{Details and proofs for Sections~\ref{sec:nat_compression} and \ref{sec:ND}}

\subsection{Proof of Theorem~\ref{lem:br_quant}} \label{sec:Proof_main_thm}

By linearity of expectation, the unbiasedness condition and the  second moment condition \eqref{eq:omega_quant} have the form
\begin{equation} \E{(\cC(x))_i}=x_i, \qquad \forall x\in \R^d, \quad \forall i\in [d] \label{def:Q-unbiased_proof}\end{equation}
and
\begin{equation}
 \sum_{i=1}^d \E{ (\cC(x))_i^2 } \leq (\omega+1) \sum_{i=1}^d x_i^2, \qquad \forall x \in \R^d. \label{def:Q-second_moment_proof}
\end{equation}

Recall that  $\NC(t)$ can be written in the form
 \begin{equation}
 \label{eq:ef_impl}
 \NC(t)= \signum (t) \cdot 2^{\floor{\log_2 \abs{t}}}(1+\lambda(t)).
 \end{equation}
 where the last step follows since $p(t)=\frac{2^{\ceil{\log_2 \abs{t}}}-\abs{t}}{2^{\floor{\log_2 \abs{t}}}}$.
Hence,
\begin{eqnarray*}
\E{\NC(t)} &\overset{\eqref{eq:ef_impl}}{=}& \E{ \signum (t) \cdot 2^{\floor{\log_2 \abs{t}}}(1+\lambda(t))}  =    \signum (t) \cdot 2^{\floor{\log_2 \abs{t}}} \left(1+\E{\lambda(t) } \right) \\
&=&    \signum (t) \cdot 2^{\floor{\log_2 \abs{t}}} \left(1+1 - p(t) \right) = t,
\end{eqnarray*}
This establishes unbiasedness \eqref{def:Q-unbiased_proof}.

In order to establish \eqref{def:Q-second_moment_proof}, it suffices to show that $\E{ (\NC(x))_i^2 } \leq (\omega+1) x_i^2$ for all $x_i\in \R$. Since by definition $(\NC(x))_i = \NC(x_i)$ for all $i\in [d]$, 
it suffices to show that
\begin{equation} \label{eq:to_prove}\E{ (\NC(t))^2 } \leq (\omega+1) t^2, \qquad \forall t\in \R. \end{equation}
If $t=0$ or $t=\signum (t) 2^\alpha$ with $\alpha$ being an integer, then $\NC(t)=t$, and \eqref{eq:to_prove} holds as an identity with $\omega=0$, and hence inequality \eqref{eq:to_prove}  holds for $\omega = \nicefrac{1}{8}$. Otherwise $t= \signum (t) 2^\alpha$ where $a \eqdef \lfloor \alpha \rfloor <\alpha < \lceil \alpha \rceil = a+1$.
With this notation, we can write
\begin{eqnarray*}
\E{(\NC(t))^2}  &=& 2^{2a}\frac{2^{a+1}-\abs{t}}{2^{a}} + 2^{2(a+1)}\frac{\abs{t}-2^a}{2^{a}} \quad = \quad  2^{a}(3\abs{t} - 2^{a+1}). \\ 
\end{eqnarray*}
So,
\begin{eqnarray*}
\frac{\E{(\NC(t))^2} }{t^2} &=& \frac{2^{a}(3\abs{t} - 2^{a+1})}{t^2} \quad \leq \quad \sup_{2^a<t < 2^{a+1}}\frac{2^{a}(3\abs{t} - 2^{a+1})}{t^2}\\
&= & \sup_{1< \theta < 2}\frac{2^{a}(3 \cdot 2^a \theta - 2^{a+1})}{(2^a \theta )^2} \quad =\quad  \sup_{1< \theta < 2}\frac{3 \theta - 2}{\theta^2}.
\end{eqnarray*}
The optimal solution of the last maximization problem is  $\theta = \frac{4}{3}$, with optimal objective value $\frac{9}{8}$. This implies that \eqref{eq:to_prove} holds with  $\omega = \frac{1}{8}$.

\subsection{Proof of Theorem~\ref{prop:introunding_negative}}

Let assume that there exists some $\omega < \infty$ for which $\Qint$ is the $\omega$ quantization.  Unbiased rounding to the nearest integer can be defined in the following way
\begin{align*}
\Qint(x_i) \eqdef \begin{cases} \floor{x_i} , & \quad \text{with probability} \quad p(x_i),\\
\ceil{x_i} , & \quad  \text{with probability} \quad 1-p(x_{i}),
\end{cases}
\end{align*}
where $p(x_i) = \ceil{x_i} - x_i$. Let's take $1$-D example, where $x \in (0,1)$, then 
\begin{align*}
\E{\Qint(x^2)} = (1 - x) 0^2 + x 1^2 = x \leq \omega x^2,
\end{align*} 
which implies $\omega \geq 1/x$, thus taking $x \rightarrow 0^+$, one obtains $\omega \rightarrow \infty$, which contradicts the existence of finite $\omega$.

\subsection{Proof of Theorem~\ref{thm:composition}}

The main building block of the proof is the tower property of mathematical expectation. The tower property says: If $X$ and $Y$ are random variables, then
$
\E{X} =\E{ \E{X \;|\; Y}}.
$
Applying it to the composite compression operator $\cC_{1}\circ \cC_2$, we get
\begin{align*}
\E{\left(\cC_1\circ \cC_2\right)(x)} = \E{\E{\cC_1(\cC_2(x)) \;|\; \cC_2(x)}} \overset{\eqref{eq:omega_quant}}{=} \E{\cC_2(x)} \overset{\eqref{eq:omega_quant}}{=} x\;.
\end{align*} 

 For the second moment, we have 
 \begin{align*}
 \E{\norm*{\left(\cC_1\circ \cC_2\right)(x)}^2} &=  \E{\E{\norm*{\cC_1(\cC_2(x))}^2\;|\; \cC_2(x)}}\\
 &\overset{\eqref{eq:omega_quant}}{\leq} (\omega_2+1) \E{\norm*{\cC_1(x)}^2} \\
 & \overset{\eqref{eq:omega_quant}}{\leq} (\omega_1+1)(\omega_2+1) \norm*{x}^2,
 \end{align*}
 which concludes the proof.

\subsection{Proof of Theorem~\ref{thm:natural_dithering}}

Unbiasedness of $\ND$ is a direct consequence of unbiasedness of $\GD$.
For the second part, we first establish a  bound on the second moment of $\xi$:
\begin{align}
\E{\xi\left(\frac{x_i}{\norm*{x}_p}\right)^2} &\leq  \mathds{1}\left(\frac{\abs{x_i}}{\norm*{x}_p} \geq 2^{1-s}\right) \frac{9}{8} \frac{\abs{x_i}^2}{\norm*{x}^2_p}  + \mathds{1}\left(\frac{\abs{x_i}}{\norm*{x}_p} < 2^{1-s}\right) \frac{\abs{x_i}}{\norm*{x}_p}2^{1-s} \notag \\
&\leq \frac{9}{8} \frac{\abs{x_i}^2}{\norm*{x}^2_p}  + \mathds{1}\left(\frac{\abs{x_i}}{\norm*{x}_p} < 2^{1-s}\right) \frac{\abs{x_i}}{\norm*{x}_p}2^{1-s}  \; .
\label{eq:exp_xi}
\end{align}

Using this bound, we have
\begin{align*}
\E{\norm*{\ND(x)}^2} &= \E{\norm*{x}^2_p}\sum_{i = 1}^d \E{\xi\left(\frac{x_i}{\norm*{x}_p}\right)^2} \\
&\overset{\eqref{eq:exp_xi}}{\leq}  \norm*{x}^2_p \left( \frac{9\norm*{x}^2}{8\norm*{x}^2_p} +\sum_{i = 1}^d  \mathds{1}\left(\frac{\abs{x_i}}{\norm*{x}_p} < 2^{1-s}\right) \frac{\abs{x_i}}{\norm*{x}_p}2^{1-s}  \right) \\
&\leq   \frac{9}{8}\norm*{x}^2 + \min\left\{2^{1-s}\norm*{x}_p\norm*{x}_1,  2^{2-2s}d  \norm*{x}^2_p \right\}  \\
&\leq    \frac{9}{8}\norm*{x}^2 + \min\left\{d^{1/2}2^{1-s}\norm*{x}_p\norm*{x},  2^{2-2s}d \norm*{x}^2_p\right\}  \\
&\leq    \left( \frac{9}{8} + d^{1/\min\{p,2\}}2^{1-s}\min\left\{1, d^{1/\min\{p,2\}}2^{1-s}\right\} \right) \norm*{x}^2\; ,
\end{align*}
where the second inequality follows from $\sum\min\{a_i, b_i\}  \leq \min\{\sum a_i,\sum b_i\}$ and the last two inequalities follow from the following consequence of H\"{o}lder's inequality $\norm*{x}_p \leq d^{1/p-1/2}\norm*{x} $ for $1 \leq p < 2$ and from the fact that $\norm*{x}_p \leq \norm*{x}$ for $p \geq 2$. This concludes the proof.

\subsection{Proof of Theorem~\ref{thm:expon_better}}

The  main building block of the proof is useful connection between $\ND$ and $\SDs{2^{s-1}}$, which can be formally written as
\begin{equation}
\ND(x) \overset{D}{=} \norm*{x}_p \cdot \signum(x) \cdot \NC(\xi(x)) \; ,
 \label{eq:brd_2}
\end{equation}
where  $(\xi(x))_i = \xi(\nicefrac{x_i}{\norm*{x}_p})$ with levels $0,\nicefrac{1}{2^{s-1}},\nicefrac{2}{2^{s-1}}, \cdots, 1$. Graphical visualization can be found in Figure~\ref{fig:brrd_with_rd}.

Equipped with this, we can proceed with 

\begin{eqnarray*}
\E{\norm*{\xi(\nicefrac{x_i}{\norm*{x}_p})}^2} &\overset{\eqref{eq:brd_2}}{=}& \E{\norm*{ \norm*{x}_p \cdot \signum(x) \cdot \NC(\xi(x))}^2} \\
&=&\E{ \norm*{x}_p^2} \cdot \E{\norm*{\NC(\xi(x))}^2} \\
&\overset{\text{Thm.~\ref{lem:br_quant}}}{\leq}& \frac{9}{8} \E{\norm*{\norm*{x}_p\signum(x)\xi(x)}^2} \\
&=& \frac{9}{8} \E{\norm*{\SDs{2^{s-1}}}^2(x)} \\
&\leq&  \frac{9}{8}(\omega+1),
\end{eqnarray*}
which concludes the proof.

\subsection{Natural compression and dithering allow for fast aggregation}

Besides communication savings, our new compression operators $\NC$ (natural compression) and $\ND$ (natural dithering) bring another advantage, which is {\em ease of aggregation}. Firstly, our updates allow in-network aggregation on a primitive switch, which can speed up training by up to $300\%$~(Sapio et al., 2021~\cite{switchML}) itself. Moreover, our updates are so simple that if one uses integer format on the master side for update aggregation, then our updates have just one non-zero bit, which leads to additional speed up. For this reason, one needs to operate with at least $64$ bits during the aggregation step, which is the reason why we also do $\NC$ compression on the master side; and hence we need to transmit just exponent to workers. Moreover, the translation from floats to integers and back is computation-free due to structure of our updates. Lastly, for $\ND$ compression we obtain additional speed up with respect to standard randomized dithering $\SD$ as our levels are computationally less expensive due to their natural compatibility with floating points. In addition, for  effective communication one needs to communicate signs, norm and levels as a tuple for both $\ND$ and $\SD$, which needs to be then multiplied back on the master side. For $\ND$, this is   just the summation of exponents rather than actual multiplication as is the case for $\SD$.


\section{Details and proofs for Section~\ref{sec:SGD}}
\label{sec:gen_sgd}

\subsection{Assumptions and definitions}

Formal definitions of some concepts used in Section~\ref{sec:SGD} follows:
\begin{definition}
\label{def:stoch_grad}
Let  be fixed function.
A \emph{stochastic gradient} of $f: \R^d \to \R$ is a random mapping $g(x)$ such that $\E{g(x)} = \nabla f(x),\; \forall x \in \R^d$.
\end{definition}

In order to obtain the rate, we introduce additional assumptions on $g_i(x)$ and $\nabla f_i(x)$.

\begin{assumption}[Bounded Variance]
\label{as:bounded_variance}
We say the stochastic gradient has variance at most $\sigma_i^2$ if $\E{\norm*{g_i(x) - \nabla f_i(x) }^2} \leq \sigma_i^2$ for all $x \in \R^d$. Moreover, let $\sigma^2 = \frac{1}{n} \sum_{i=1}^n \sigma_i^2$.
\end{assumption}

\begin{assumption}[Similarity]
\label{as:sim}
We say the variance of gradient among nodes is  at most $\zeta_i^2$ if $\norm*{\nabla f_i(x) - \nabla f(x) }^2 \leq \zeta_i^2$ for all $x \in \R^d$. Moreover, let $\zeta^2 = \frac{1}{n} \sum_{i=1}^n \zeta_i^2$.
\end{assumption}

Moreover, we assume that $f$ is $L$-smooth (gradient is $L$-Lipschitz). These are classical assumptions for non-convex \texttt{SGD} \cite{ghadimi2013stochastic,jiang,mishchenko2019distributed} and comparing to some previous works \cite{qsgd2017neurips}, our analysis does not require bounded iterates and bounded the second moment of the stochastic gradient. Assumption~\ref{as:sim} is automatically satisfied with $\zeta^2 = 0$ if every worker has access to the whole dataset. If one does not like Assumption~\ref{as:sim} one can use the \texttt{DIANA} algorithm \cite{diana2} as a base algorithm instead of \texttt{SGD}, then there is no need for this assumption. For simplicity, we decide to pursue just \texttt{SGD} analysis and we keep Assumption~\ref{as:sim}.

\subsection{Three lemmas needed for the proof of Theorem~\ref{thm:arbSGD_short}}

Before we proceed with the theoretical guarantees for Algorithm~\ref{alg:arbSGD} in smooth non-convex setting, we first state three lemmas which are used to bound the variance of $g^k$ as a stochastic estimator of the true gradient $\nabla f(x^k)$. In this sense compression at the master-node has the effect of injecting additional variance into the gradient estimator. Unlike in \texttt{SGD}, where stochasticity is used to speed up computation, here we use it to reduce communication.

\begin{lemma}[Tower property + Compression]
\label{lem:tow_prop}
If $\cC \in \U(\omega)$ and $z$ is a  random vector independent of $\cC$, then
\begin{equation}
\label{eq:tow_prop}
\E{\norm*{\cC(z) - z}^2} \leq\omega\E{\norm*{z}^2}; \qquad  \E{\norm*{\cC(z)}^2} \leq (\omega + 1) \E{\norm*{z}^2}\;.
\end{equation}
\end{lemma}

\begin{proof}
Recall from the discussion following Definition~\ref{def:omegaquant} that the variance of a compression operator $\cC\in \U(\omega)$ can be bounded as
\[
\E{\norm*{\cC(x) - x}^2} \leq \omega \norm*{x}^2, \qquad \forall x\in \R^d.
\]
Using this with $z=x$, this can be written in the form
\begin{align}
\E{\norm*{\cC(z) - z}^2 \;|\; z} \leq \omega \norm*{z}^2, \qquad \forall x\in \R^d\; , \label{def:omega_proof_ug}
\end{align}
which we can use in our argument:
\begin{eqnarray*}
\E{\norm*{\cC(z) - z}^2} &=& \E{\E{\norm*{\cC(z) - z}^2 \;|\; z}} \\
&\overset{\eqref{def:omega_proof_ug}}{\leq}&  \E{ \omega\norm*{z}^2 }  \\
&=&  \omega  \E{\norm*{z}^2}.
\end{eqnarray*}
The second inequality can be proved exactly same way.
\end{proof}

\begin{lemma}[Local compression variance]
\label{lem:local_var}
Suppose $x$ is fixed, $\cC \in \U(\omega)$, and $g_i(x)$ is an unbiased estimator of $\nabla f_i(x)$. Then
\begin{equation}
\label{eq:local_var}
\E{\norm*{\cC(g_i(x)) - \nabla f_i(x)}^2} \leq (\omega + 1)\sigma_i^2 + \omega\norm*{\nabla f_i(x)}^2.
\end{equation}
\end{lemma}

\begin{proof}
\begin{eqnarray*}
&\E{\norm*{\cC(g_i(x)) - \nabla f_i(x)}^2} \\
&\overset{\text{Def.}~\ref{def:stoch_grad} + \eqref{eq:omega_quant}}{=}& \E{\norm*{\cC(g_i(x)) - g_i(x)}^2} + \E{\norm*{g_i(x) - \nabla f_i(x)}^2} \\
&\overset{\eqref{eq:tow_prop}}{\leq}& \omega\E{\norm*{g_i(x)}^2} +  \E{\norm*{g_i(x) - \nabla f_i(x)}^2} \\
&\overset{\text{Def.}~\ref{def:stoch_grad} + \eqref{eq:omega_quant}}{=}& (\omega + 1) \E{\norm*{g_i(x) - \nabla f_i(x)}^2} + \omega \norm*{\nabla f_i(x)}^2 \\
&\overset{\text{Assum.}~\ref{as:bounded_variance}}{\leq}&  (\omega + 1)\sigma_i^2 + \omega\norm*{\nabla f_i(x)}^2.
\end{eqnarray*}
\end{proof}

\begin{lemma}[Global compression variance]
\label{lem:global_var}
Suppose $x$ is fixed, $\cC_{W_i} \in \U(\omega_{W_i})$ for all $i$, $\cC_M \in\U(\omega_M)$, and $g_i(x)$ is an unbiased estimator of $\nabla f_i(x)$ for all $i$. Then
\begin{equation}
\label{eq:global_var}
\E{\norm*{\frac{1}{n}\cC_M\left(\sum_{i=1}^n \cC_{W_i}(g_i(x))\right)}^2} \leq 
\alpha + \beta  \norm*{\nabla f(x)}^2,
\end{equation}
where $\omega_W = \max_{i \in [n]}\omega_{W_i}$ and 
\begin{align}
\label{eq:alpha_beta}
\alpha = \frac{(\omega_M+1)(\omega_W+1)}{n}\sigma^2 + \frac{(\omega_M+1)\omega_W}{n} \zeta^2\,, &   &
\beta = 1+ \omega_M + \frac{(\omega_M+1)\omega_W}{n}  \,.
\end{align}
\end{lemma}

\begin{proof}
For added clarity, let us denote 
\begin{align*}
\Delta = \sum_{i=1}^n \cC_{W_i}(g_i(x)).
\end{align*}
Using this notation, the proof proceeds as follows:
\begin{eqnarray*}
&\E{\norm*{\frac{1}{n}\cC_M\left(\Delta \right)}^2}& \\
&\overset{\text{Def.}~\ref{def:stoch_grad} + \eqref{eq:omega_quant}}{=}& \E{\norm*{\frac{1}{n}\cC_M\left(\Delta \right) - \nabla f(x)}^2} + \norm*{\nabla f(x)}^2 \\
&\overset{\text{Def.}~\ref{def:stoch_grad} + \eqref{eq:omega_quant}}{=}&  \frac{1}{n^2}\E{\norm*{ \cC_M \left(\Delta \right) - \Delta}^2} + \E{\norm*{\frac{1}{n}\Delta - \nabla f(x)}^2} + \norm*{\nabla f(x)}^2 \\
&\overset{\eqref{eq:tow_prop}}{\leq}& \frac{\omega_M}{n^2}\E{\norm*{\Delta}^2} +  \E{\norm*{\frac{1}{n} \Delta - \nabla f(x)}^2} + \norm*{\nabla f(x)}^2 \\
&\overset{\text{Def.}~\ref{def:stoch_grad} + \eqref{eq:omega_quant}}{=}& (\omega_M+1)\E{\norm*{\frac{1}{n} \Delta - \nabla f(x)}^2} + (\omega_M + 1)\norm*{\nabla f(x)}^2 \\
&=& \frac{\omega_M + 1}{n^2} \sum_{i=1}^n\E{\norm*{\cC_{W_i}(g_i(x)) - \nabla f_i(x)}^2} + (\omega_M + 1) \norm*{\nabla f(x)}^2 \\
&\overset{\eqref{eq:local_var}}{\leq}& \frac{(\omega_M+1)(\omega_W+1)}{n}\sigma^2 + \frac{(\omega_M+1)\omega_W}{n} \frac{1}{n}\sum_{i=1}^n \norm*{\nabla f_i(x)}^2 \\
&& \quad + (\omega_M + 1) \norm*{\nabla f(x)}^2 \\
&=& \frac{(\omega_M+1)(\omega_W+1)}{n}\sigma^2 + \frac{(\omega_M+1)\omega_W}{n} \frac{1}{n}\sum_{i=1}^n \norm*{\nabla f_i(x) - \nabla f(x)}^2\\
&& \quad + \left( 1+ \omega_M + \frac{(\omega_M+1)\omega_W}{n} \right)  \norm*{\nabla f(x)}^2  \\
&\overset{\text{Assum.}~\ref{as:sim}}{\leq}& \frac{(\omega_M+1)(\omega_W+1)}{n}\sigma^2 + \frac{(\omega_M+1)\omega_W}{n} \zeta^2 \\
&& \quad + \left( 1+ \omega_M + \frac{(\omega_M+1)\omega_W}{n} \right)  \norm*{\nabla f(x)}^2  .
\end{eqnarray*}
\end{proof}

\subsection{Proof of Theorem~\ref{thm:arbSGD_short}}

Using $L$-smoothness of $f$ and then applying Lemma~\ref{lem:global_var}, we get
\begin{eqnarray*}
&\E{f(x^{k+1})}& \\
&\leq & \E{f(x^k)} + \E{\dotprod{\nabla f(x^k)}{x^{k+1} - x^k}} + \frac{L}{2}\E{\norm*{x^{k+1} - x^k}^2} \\
&\leq & \E{f(x^k)} - \eta_k\E{\norm*{\nabla f(x^k)}^2} + \frac{L}{2}\eta_k^2\E{ \norm*{\frac{g^k}{n}}^2}  \\
&\overset{\eqref{eq:global_var}}{\leq}& \E{f(x^k)} - \left(\eta_k - \frac{L}{2}\beta\eta_k^2\right)\E{\norm*{\nabla f(x^k)}^2} + \frac{L}{2}\alpha\eta_k^2\, .
\end{eqnarray*}
Summing these inequalities for $k=0, ..., T-1$, we obtain
\begin{align*}
\sum_{k=0}^{T-1}\left(\eta_k - \frac{L}{2}\beta\eta_k^2\right)\E{\norm*{\nabla f(x^k)}^2} \leq f(x^0) - f(x^{\star}) +  \frac{TL\alpha\eta_k^2}{2} \,.
\end{align*}
Taking $\eta_k = \eta$ and assuming \begin{equation}\label{eq:bn989gd8f} \eta < \frac{2}{L \beta },\end{equation} one obtains
\begin{align*}
\E{\norm*{\nabla f(x^a)}^2}  \leq \frac{1}{T}\sum_{k=0}^{T-1} \E{\norm*{\nabla f(x^k)}^2} \leq \frac{2(f(x^0) - f(x^{\star}))}{T\eta \left(2 - L \beta\eta\right)} +  \frac{L\alpha\eta}{2 - L\beta\eta} \eqdef \delta(\eta,T)\,.
\end{align*}
It is easy to check that if we choose $\eta = \frac{\varepsilon}{L(\alpha + \varepsilon \beta )}$ (which satisfies \eqref{eq:bn989gd8f} for every $\varepsilon>0$), then for any $T \geq \frac{2L(f(x^0) - f(x^{\star})) (\alpha + \epsilon \beta)}{\epsilon^2}$ we have 
$\delta(\eta,T) \leq \varepsilon$, concluding the proof.


\subsection{A different stepsize rule for Theorem~\ref{thm:arbSGD_short}}

Looking at Theorem~\ref{thm:arbSGD_short}, one can see that setting step size \[\eta_k = \eta = \sqrt{\frac{2(f(x^0)-f(x^{\star}))}{LT\alpha}} \]  with \[T \geq \frac{L\beta^2(f(x^0)-f(x^{\star}))}{\alpha}\] (number of iterations), we have iteration complexity \[\cO\left( \sqrt{\frac{(\omega_W+1)(\omega_M+1)}{Tn}}\right),\] which will be essentially same as doing no compression on master and using $\cC_W\circ \cC_M$ or $\cC_W \circ \cC_M$ on the workers' side. Our rate generalizes to the rate of \cite{ghadimi2013stochastic} without compression and dependency on the compression operator is better comparing to the linear one in \cite{jiang}\footnote{\cite{jiang}  allows compression on the worker side only.}. Moreover, our rate enjoys linear speed-up in the number of workers $n$, the same as \cite{ghadimi2013stochastic}. In addition, if one introduces mini-batching on each worker of size $b$ and assuming each worker has access to the whole data, then $\sigma^2 \rightarrow \sigma^2/b$ and $\zeta^2 \rightarrow 0$, which implies \[\cO\left( \sqrt{\frac{(\omega_W+1)(\omega_M+1)}{Tn}}\right) \rightarrow \cO\left( \sqrt{\frac{(\omega_W+1)(\omega_M+1)}{Tbn}}\right),\] and hence one can also obtain linear speed-up in terms of mini-batch size, which matches with \cite{jiang}.

\subsection{\texttt{SGD} with bidirectional compression:  four models}
\label{sec:dif_regimes}

It is possible to consider several different regimes for our distributed optimization/training setup, depending on factors such as:
\begin{itemize}
\item The relative speed of communication (per bit) from workers to the master and from the master to the workers,
\item The intelligence of the master, i.e., its ability or the lack thereof of the master to perform aggregation of real  numbers (e.g., a switch can only perform integer aggregation),
\item Variability of various resources (speed, memory, etc) among the workers.
\end{itemize}

For simplicity, we will consider four situations/regimes only, summarized in Table~\ref{tbl:4regimes}.

\begin{table}[t]
\resizebox{\columnwidth}{!}{
\begin{tabular}{|c|cc|}
\hline
 &  \begin{tabular}{c} Master can aggregate \\ real numbers \\ (e.g., a workstation) \end{tabular} & \begin{tabular}{c}Master can aggregate \\ integers only \\ (e.g., SwitchML~\cite{switchML})\end{tabular}\\
 \hline
Same communication speed both ways & MODEL 1 & MODEL 3  \\
Master communicates infinitely fast &  MODEL 2 & MODEL 4 \\
\hline
\end{tabular}
}
\caption{Four theoretical models.}
\label{tbl:4regimes}
\end{table}

\textbf{Direct consequences of Theorem~\ref{thm:arbSGD_short}}:
 Notice that \eqref{eq:main_thm_SGD} posits a $\cO(\nicefrac{1}{T})$ convergence of the gradient norm to the value $\tfrac{\alpha L \eta }{2-\beta L \eta}$, which depends linearly on $\alpha$. In view of \eqref{eq:alpha_beta_out}, the more compression we perform, the larger this value. More interestingly, assume now that the same compression operator is used at each worker: $\cC_W =\cC_{W_i}$. Let $\cC_W \in \U(\omega_W)$ and  $\cC_M \in \U(\omega_M)$ be the compression on master side. Then,  $T(\omega_M, \omega_W) \eqdef 2L(f(x^0)-f(x^{\star}))\varepsilon^{-2}(\alpha + \varepsilon \beta)$  is its iteration complexity.  In the special case of equal data on all nodes, i.e., $\zeta=0$, we get $\alpha =  \nicefrac{(\omega_M+1)(\omega_W+1) \sigma^2}{n}$ and $\beta = (\omega_M+1)\left(1+\nicefrac{\omega_W}{n}\right)$. If no compression is used, then $\omega_W = \omega_M = 0$ and $\alpha +\varepsilon \beta = \nicefrac{\sigma^2}{n} +\varepsilon$. So, the {\em relative slowdown} of Algorithm~\ref{alg:arbSGD} used {\em with} compression compared to Algorithm~\ref{alg:arbSGD} used {\em without} compression is given by
\begin{equation}  \tfrac{T(\omega_M, \omega_W)}{T(0,0)} 
= \frac{\left(\nicefrac{(\omega_W+1)\sigma^2}{n} +(1+\nicefrac{\omega_W}{n})\varepsilon\right)}{\nicefrac{\sigma^2}{n} +\varepsilon} (\omega_M+1) \in \left( \omega_M+1, (\omega_M+1)(\omega_W+1) \right].\end{equation}
The upper bound is achieved for $n=1$ (or for any $n$ and $\varepsilon\to 0$), and the lower bound is achieved in the limit as $n\to \infty$. So, {\em the slowdown caused by compression on worker side decreases with $n$.} More importantly, {\em the savings in communication due to compression can outweigh the iteration slowdown, which leads to an overall speedup!}

\subsubsection{Model 1}
First, we start with the comparison, where we assume that transmitting one bit from worker to node takes the same amount of time as from master to worker. 
\begin{table}[H]
\begin{center}
\footnotesize
\begin{tabular}{|c|c|c|c|}
\hline
 \begin{tabular}{c} Compression\\
 $\cC \in \U(\omega)$ \end{tabular}  &   \begin{tabular}{c} No.\ iterations  \\
 $T(\omega) = \cO((\omega+1)^{1+\theta})$ \end{tabular} &  \begin{tabular}{c} Bits per iteration \\  $W_i \mapsto M + M \mapsto W_i$
 \end{tabular}  &  \begin{tabular}{c} Speedup \\
 $\frac{T(0)B(0)}{T(\omega)B(\omega)}$ \end{tabular} \\
\hline 
None & $1$ & $2 \cdot 32d $  & $1$   \\
 {\color{blue}$\NC$}  &  {\color{blue}$(\frac{9}{8})^{1+\theta}$ }&   {\color{blue}$2\cdot 9d$} &   {\color{blue}$2.81\times $--$3.16\times$} \\
$S^q$  & $(\frac{d}{q})^{1+\theta}$ &  $2\cdot(33+\log_2d)q$ &   $0.06\times$--$0.60\times$ \\
 {\color{blue}$S^q \circ \NC$ } &  {\color{blue}$(\frac{9d}{8q})^{1+\theta}$} &   {\color{blue}$2\cdot (10+\log_2d)q$} &   {\color{blue}$0.09\times$--$0.98\times$} \\
$\SDs{2^{s-1}}$ &  $ \left( 1  + \sqrt{d} 2^{1-s} \kappa\right)^{1+\theta}$ &  $2\cdot (32+ d(s+2))$ & $1.67\times$--$1.78\times$ \\
 {\color{blue}$\ND$} &  {\color{blue} $ \left( \frac{81}{64}   + \frac{9}{8} \sqrt{d} 2^{1-s} \kappa\right)^{1+\theta}$ }&   {\color{blue}$2\cdot(8+ d(\log_2 s+2))$ }& {\color{blue}$3.19\times$--$4.10\times$} \\
\hline
\end{tabular}
\end{center} 
\caption{Our compression techniques can speed up the overall runtime  (number of iterations $T(\omega)$ times the bits sent per iteration) of distributed {\tt SGD}. We assume {\em binary$32$} floating point representation, bi-directional compression using $\cC$, and the same speed of communication from worker to master ($W_i \mapsto M$) and back ($M \mapsto W_i$). The relative number of iterations (communications) sufficient to guarantee $\varepsilon$ optimality is
$T'(\omega) \eqdef (\omega+1)^\theta$, where  $\theta\in (1,2]$ (see Theorem~\ref{thm:arbSGD_short}). Note that big $n$ regime leads to better iteration bound $T(\omega)$ since for big $n$ we have $\theta\approx 1$, while for small $n$ we have $\theta\approx 2$. For dithering, $\kappa =  \min \{1, \sqrt{d }   2^{1-s}\}$. The $2.81\times $ speedup for $\NC$ is obtained for $\theta=1$, and the $3.16\times$ speedup for $\theta=0$. The speedup figures were calculated for $d=10^6$, $p=2$ (dithering),optimal choice of $s$ (dithering), and $q = 0.1d$ (sparsification).}
\label{tab:algo-comparison}
\end{table}
\subsubsection{Model 2}

For the second model, we assume that the master communicates much faster than workers thus communication from workers is the bottleneck and we don't need to compress updates after aggregation, thus $\cC_M$ is identity operator with $\omega_M = 0$. This is the case  we mention in the main paper. For completeness, we provide the same table here.

\begin{table}[!h]
\begin{center}
\footnotesize
\begin{tabular}{|c|c|c|c|c|}
\hline
Approach &$\cC_{W_i}$  &  No. iterations  &  Bits per $1$ iter.  & Speedup  \\
 &   & $T'(\omega_W) = \cO((\omega_W+1)^\theta)$  &  $W_i \mapsto M$ & Factor  \\
\hline 
Baseline & identity  & $1$  & $32d$  & 1   \\
\bf {\color{blue} New} & {\color{blue}$\NC$ } & {\color{blue}$(\nicefrac{9}{8})^{\theta}$} &  {\color{blue}$9d$} & {\color{blue}$3.2 \times$--$3.6\times$ } \\
\hline
Sparsification & $\cS^q$ & $(\nicefrac{d}{q})^{\theta}$ & $(33 +\log_2d)q$ & $0.6\times$--$6.0\times$   \\
\bf {\color{blue} New} & ${\color{blue}\NC} \circ \cS^q$  &  {\color{blue}$(\nicefrac{9d}{8q})^{\theta}$} & {\color{blue}$(10+\log_2d)q$} & {\color{blue}$1.0\times$--$10.7\times$ } \\
\hline
Dithering & $\SDs{2^{s-1}}$  & $(1 +\kappa d^{\nicefrac{1}{r}}2^{1-s})^{\theta}$  & $31+d(2+s)$ & $1.8\times$--$15.9\times$  \\
\bf {\color{blue} New} & {\color{blue}$\ND$}  & {\color{blue}$(\nicefrac{9}{8} +  \kappa d^{\frac{1}{r}}2^{1-s})^\theta$} &  {\color{blue}$31+d(2+\log_2s  )$} & {\color{blue}$4.1\times$--$16.0\times$} \\
\hline
\end{tabular}
\end{center} 
\caption{
The overall speedup of distributed {\tt SGD} with compression on nodes via $\cC_{W_i}$ over a Baseline variant without compression. Speed is measured by multiplying the \# communication rounds (i.e., iterations  $T(\omega_W)$) by the bits sent from worker to master ($W_i \mapsto M$) per 1 iteration. We neglect $M \mapsto W_i$ communication as in practice this is much faster. We assume {\em binary$32$} representation. The relative \# iterations sufficient to guarantee $\varepsilon$ optimality is $T'(\omega_W) \eqdef(\omega_W+1)^\theta$, where  $\theta\in (0,1]$ (see Theorem~\ref{thm:arbSGD_short}). Note that in the big $n$ regime the iteration bound $T(\omega_W)$ is better due to $\theta\approx 0$ (however, this is not very practical as $n$ is usually small), while for small $n$ we have $\theta\approx 1$. For dithering, $r=\min\{p,2\}$,  $\kappa =  \min \{1, \sqrt{d}   2^{1-s}\}$. The lower bound for the Speedup Factor is obtained for $\theta=1$, and the upper bound for $\theta=0$. The Speedup Factor $\left(\frac{T(\omega_W)\cdot\text{\# Bits}}{T(0)\cdot 32d}\right)$ figures were calculated for $d=10^6$, $q=0.1d$, $p=2$ and optimal choice of $s$ with respect to speedup.}
\label{tab:algo-comparison_2_copy}
\end{table}

\subsubsection{Model 3}
Similarly to previous sections, we also do the comparison for methods that might be used for In-Network Aggregation. Note that for INA, it is useful to do compression also from master back to workers as the master works just with integers, hence in order to be compatible with floats, it needs to use bigger integers format. Moreover,  $\NC$ compression guarantees free translation to floats. For the third model,  we assume we have the same assumptions on communication as for Model 1.  As a baseline, we take \texttt{SGD} with $\NC$ as this is the most simple analyzable method, which supports INA.

\begin{table}[!h]
\begin{center}
\footnotesize
\begin{tabular}{|c|c|c|c|c|c|}
\hline
Approach &$\cC$  &  Slowdown  &    Bits per iter. & Speedup  \\
 &  & (iters / baseline)  &  $W_i \mapsto M + M \mapsto W_i$ & factor \\
\hline 
\bf {\color{blue} Baseline} & {\color{blue}$\NC$ }  &  {\color{blue}$1$} &  {\color{blue}$2\cdot 9d$}  & {\color{blue}$1$ } \\
\hline
\bf {\color{blue} Sparsification} & $\cS^q \circ {\color{blue}\NC} $ &  {\color{blue}$(\nicefrac{d}{q})^{1+\theta}$} & {\color{blue}$2\cdot(10+\log_2 d)q$} &   {\color{blue}$0.03\times$--$0.30\times$} \\
\hline
\bf {\color{blue} Dithering} & {\color{blue}$\ND$} & {\color{blue}$(\nicefrac{9}{8} +\kappa d^{\frac{1}{r}}2^{1-s})^{1+\theta}$} &  {\color{blue}$2\cdot(8+d(2 +\log_2 s))$} &   {\color{blue}$1.14\times$--$1.30\times$} \\
\hline
\end{tabular}
\end{center} 
\caption{Overall speedup (number of iterations $T$ times the bits sent per iteration ($W_i \mapsto M + M \mapsto W_i$) of distributed {\tt SGD}. We assume {\em binary$32$} floating point representation, bi-directional compression using the same compression $\cC$. The relative number of iterations (communications) sufficient to guarantee $\varepsilon$ optimality is
displayed in the third column, where  $\theta\in (0,1]$ (see Theorem~\ref{thm:arbSGD_short}). Note that big $n$ regime leads to smaller slowdown since for big $n$ we have $\theta\approx 0$, while for small $n$ we have $\theta\approx 1$. For dithering, we chose $p=2$ and $\kappa =  \min \{1, \sqrt{d }   2^{1-s}\}$. The speedup factor figures were calculated for $d=10^6$, $p=2$ (dithering),optimal choice of $s$ (dithering), and $q = 0.1d$ (sparsification).
}
\label{tab:algo-comparison_3}
\end{table}

\subsubsection{Model 4}

Here, we do the same comparison as for Model 3. In contrast, for communication we use the same assumptions as for Model 2.

\begin{table}[!h]
\begin{center}
\footnotesize
\begin{tabular}{|c|c|c|c|c|c|c|}
\hline
Approach &$\cC_{W_i}$ & $\cC_M$ &  Slowdown  &  $W_i \mapsto M$  commun. & Speedup  \\
 &  &  & (iters / baseline)  &  (bits / iteration) & factor \\
\hline 
\bf {\color{blue} Baseline} & {\color{blue}$\NC$ } & {\color{blue}$\NC$ }  &  {\color{blue}$1$} &  {\color{blue}$9d$}  & {\color{blue}$1$ } \\
\hline
\bf {\color{blue} Sparsification} & $\cS^q \circ {\color{blue}\NC} $  & {\color{blue}$\NC$ } &  {\color{blue}$(\nicefrac{d}{q})^\theta$} & {\color{blue}$(10+\log_2 d)q$} &   {\color{blue}$0.30\times$--$3.00\times$} \\
\hline
\bf {\color{blue} Dithering} & {\color{blue}$\ND$}  & {\color{blue}$\NC$ } & {\color{blue}$(\nicefrac{9}{8} +\kappa d^{\frac{1}{r}}2^{1-s})^\theta$} &  {\color{blue}$(8+d(2 +\log_2 s ))$} &   {\color{blue}$1.3\times$--$4.5\times$} \\
\hline
\end{tabular}
\end{center} 
\caption{Overall speedup (number of iterations $T$ times the bits sent per iteration ($W_i \mapsto M$) of distributed {\tt SGD}. We assume {\em binary$32$} floating point representation, bi-directional compression using $\cC_{W_i}$, $\cC_M$. The relative number of iterations (communications) sufficient to guarantee $\varepsilon$ optimality is
displayed in the third column, where  $\theta\in (0,1]$ (see Theorem~\ref{thm:arbSGD_short}). Note that big $n$ regime leads to smaller slowdown since for big $n$ we have $\theta\approx 0$, while for small $n$ we have $\theta\approx 1$. For dithering, we chose $p=2$ and $\kappa =  \min \{1, \sqrt{d }   2^{1-s}\}$. The speedup factor figures were calculated for $d=10^6$, $p=2$ (dithering),optimal choice of $s$ (dithering), and $q = 0.1d$ (sparsification).
}
\label{tab:algo-comparison_4}
\end{table}

\subsubsection{Communication strategies used in Tables~\ref{tab:algo-comparison_2}, \ref{tab:algo-comparison}, \ref{tab:algo-comparison_3}, \ref{tab:algo-comparison_4}}

{\bf No Compression or $\NC$.}  Each worker has to communicate a (possibly dense) $d$ dimensional vector of scalars, each represented by $32$ or $9$ bits, respectively. 

{\bf Sparsification $\cS^q$ with or without $\NC$.} Each worker has to communicate a sparse vector of $q$ entries with full $32$  or limited $9$ bit precision. We assume that $q$ is small, hence one would prefer to transmit positions of non-zeros, which takes $q(\log_2 (d) + 1)$ additional bits for each worker.

{\bf Dithering ($\SD$ or $\ND$).}  Each worker has to communicate $31$($8$ -- $\ND$) bits (sign is always positive, so does not need to be communicated) for the norm, and $\log_2(s)+1$ bits for every coordinate for level encoding (assuming uniform encoding) and $1$ bit for the sign.

\subsection{Sparsification - formal definition}

Here we give a  formal definition of  the sparsification operator $\cS^q$ used in Tables~\ref{tab:algo-comparison_2}, \ref{tab:algo-comparison},\ref{tab:algo-comparison_3},\ref{tab:algo-comparison_4}.

\begin{definition}[Random sparsification]
Let $1\leq q \leq d$ be an integer, and let $\circ$ denote the Hadamard (element-wise) product. The random sparsification operator $\cS^q: \R^d\to \R^d$ is defined as follows: \[\cS^q(x) = \frac{d}{q} \cdot \xi \circ x,\] where $\xi \in \R^d$ is a random vector chosen uniformly from the collection of all binary vectors $y \in  \{0,1\}^d$ with exactly  $q$ nonzero entries (i.e., $ \norm*{y}_0 = q\}$).
\end{definition}

The next result describes the variance of $\cS^q$:
\begin{theorem}
$\cS^q \in \U(\nicefrac{d}{q}-1)$.
\label{thm:sparsification_omega}
\end{theorem}

Notice that in the special case $q=d$, $\cS^q$ reduces to the identity operator (i.e., no compression is applied), and the above theorem yields a tight variance  estimate:  $\nicefrac{d}{d}-1=0$.

\begin{proof}
See e.g.~\cite{stich2018sparsified}(Lemma A.1).
\end{proof}

Let us now compute the variance of the composition  $\NC \circ \cS^q$. Since $\NC\in \U(\nicefrac{1}{8})$ (Theorem~\ref{lem:br_quant}) and $\cS^q \in \U(\nicefrac{d}{q}-1)$,  in the view of the composition result (Theorem~\ref{thm:composition}) we have 
\begin{equation}\label{eq:nb98fg98gfds}
 \cC_{W} = \NC \circ \cS^q \in \U(\omega_{W}),\text{ where } \omega_{W} = \frac{1}{8} \left(\frac{d}{q}-1\right) + \frac{1}{8} + \frac{d}{q} -1 = \frac{9d}{8q} -1.
 \end{equation}

\section{Limitations and extensions}
\label{sec:lim_ext}
 Quantization techniques can be divided into two categories: biased~\cite{aji2017sparse,stich2018sparsified} and unbiased~\cite{qsgd2017neurips,terngrad,tonko}. While the focus of this paper was on unbiased quantizations, it is possible to combine our natural quantization mechanisms in conjunction with   biased techniques, such as the TopK sparsifier proposed in~\cite{Dryden2016:topk,aji2017sparse} and recently  analyzed in~\cite{stich2018sparsified}, and still obtain convergence guarantees.



\refstepcounter{chapter}%
\chapter*{\thechapter \quad Appendix: Stochastic distributed learning with gradient quantization  and double variance reduction}
\label{appendix:diana_2}

\section{Extra experiments}
\begin{figure}[H]
	\makebox[\textwidth][c]{
		\subfigure[\texttt{$\ell_{\infty}$, $\lambda_2 = 2\cdot 10^{-1}$}]
		{\includegraphics[scale=0.2]{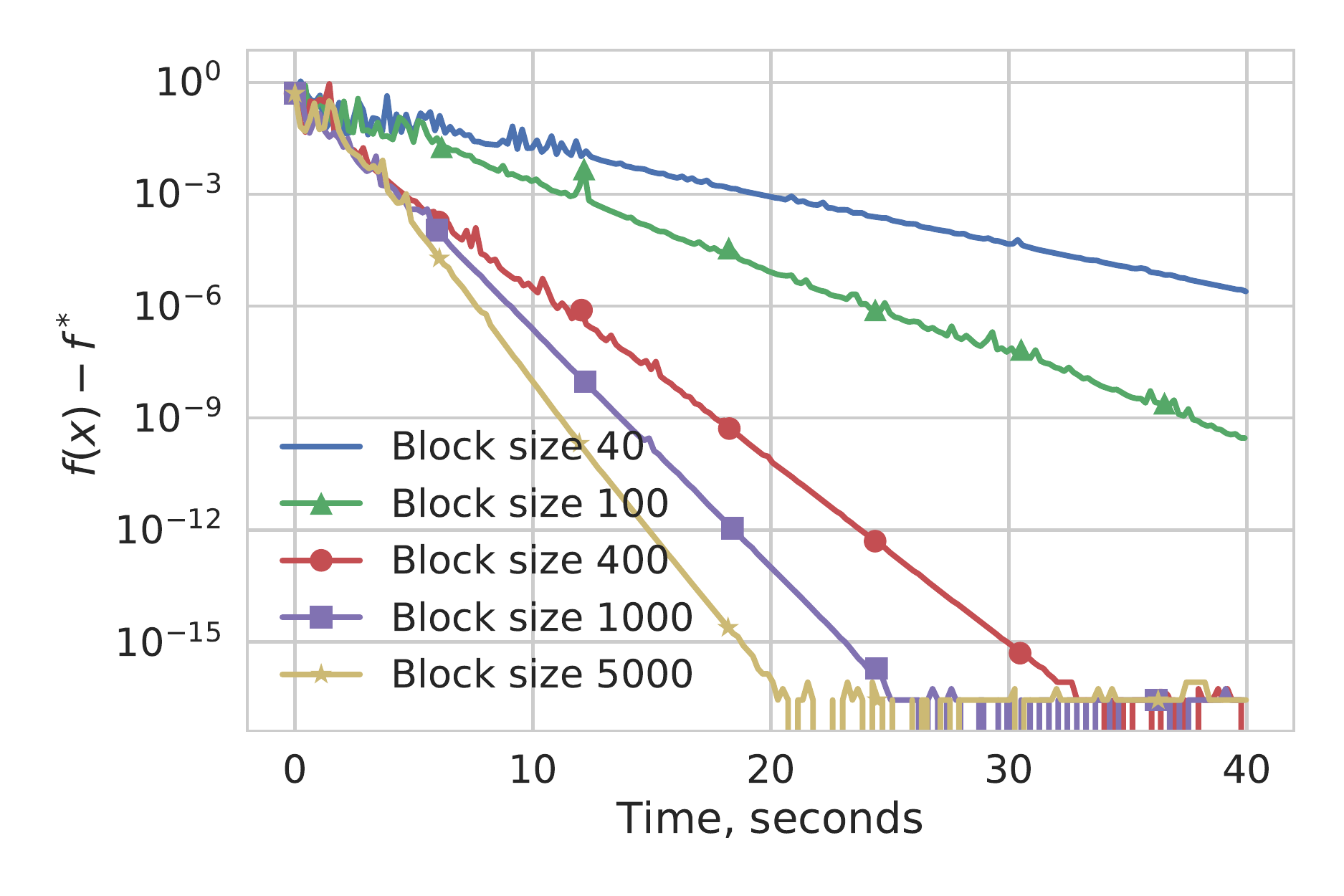}}
		\subfigure[\texttt{$\ell_{\infty}$, $\lambda_2 = 2\cdot 10^{-2}$}]
		{\includegraphics[scale=0.2]{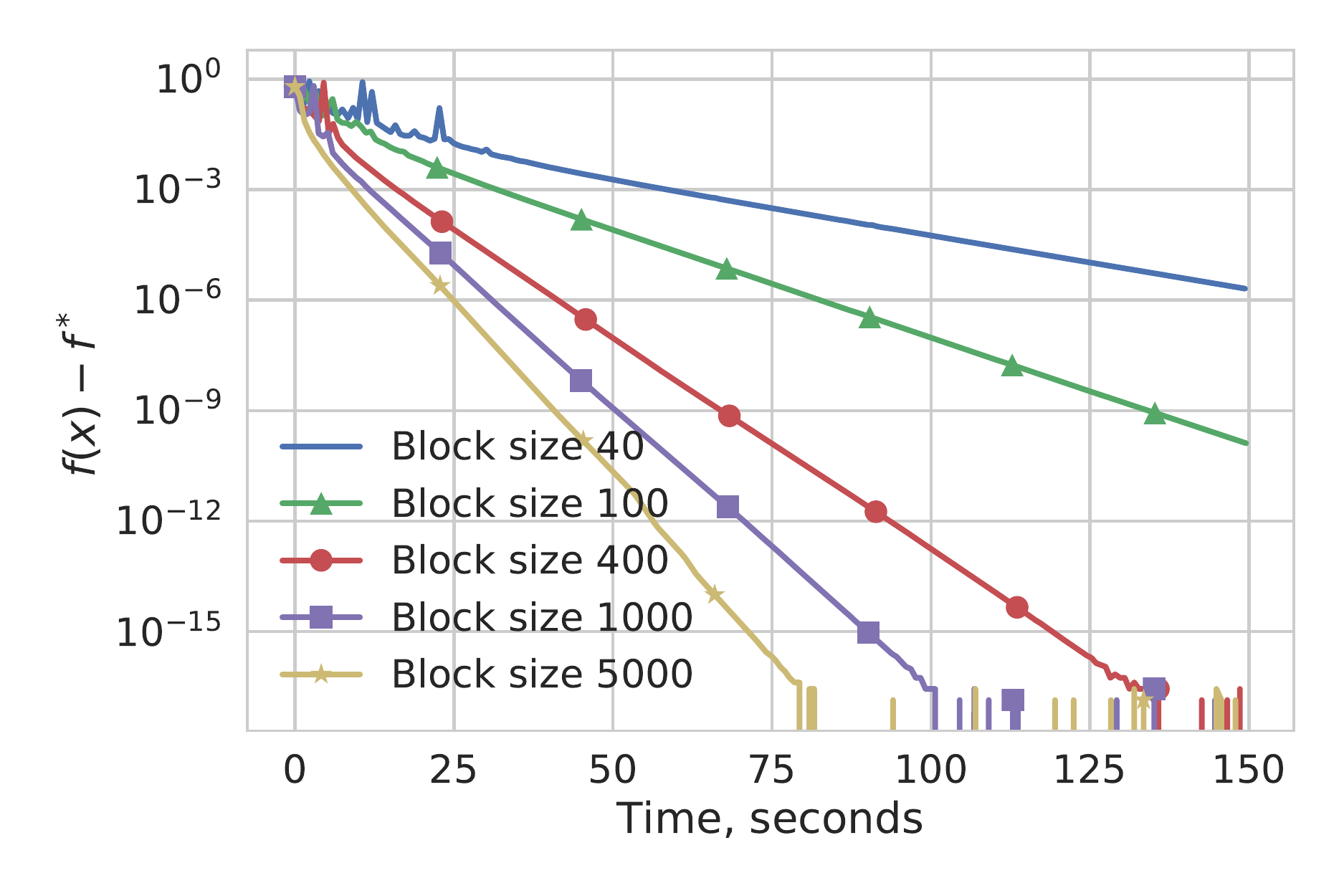}}
		\subfigure[\texttt{$\ell_{\infty}$, $\lambda_2 = 2\cdot 10^{-1}$}]
		{\includegraphics[scale=0.2]{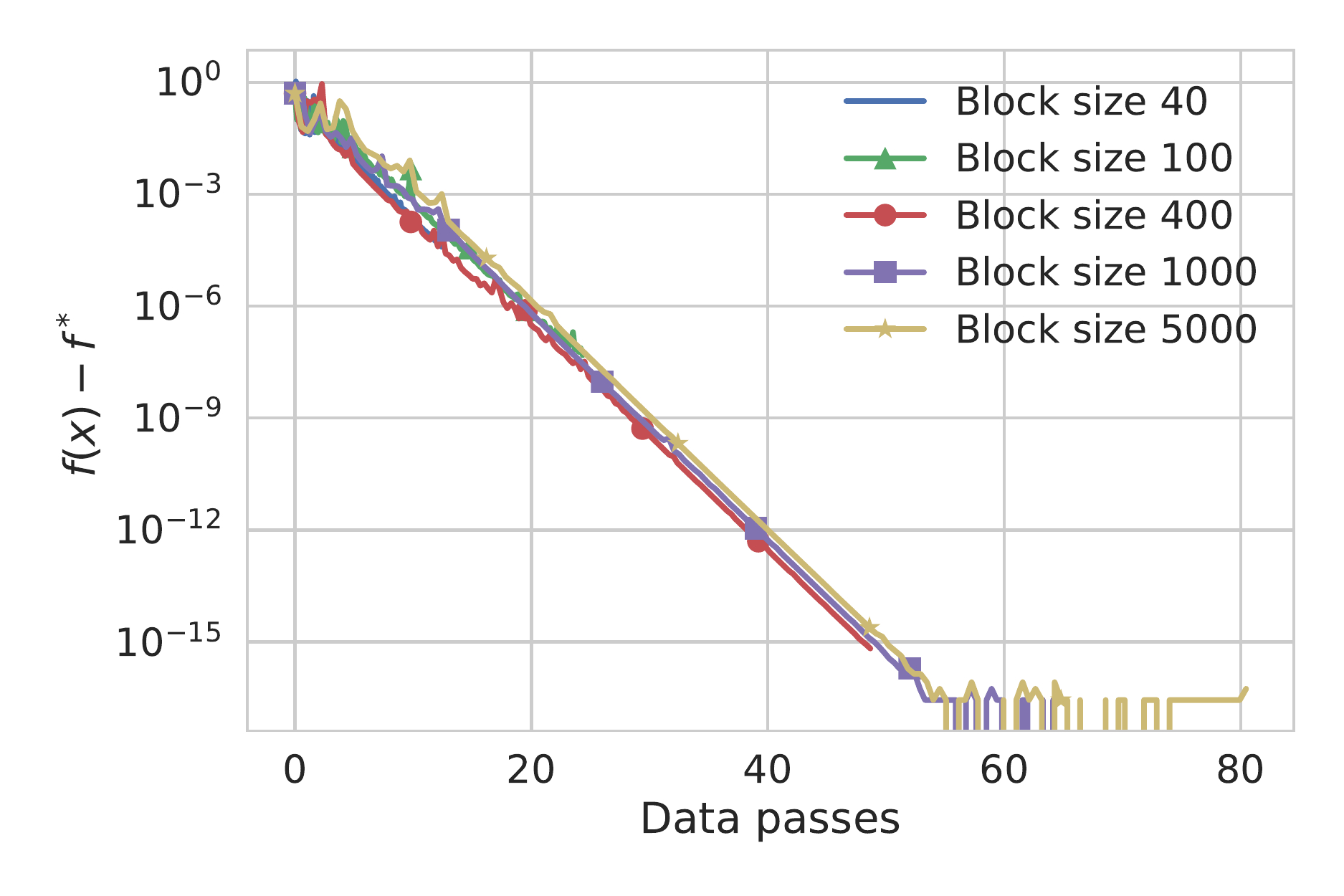}}
		\subfigure[\texttt{$\ell_{\infty}$, $\lambda_2 = 2\cdot 10^{-2}$}]
		{\includegraphics[scale=0.2]{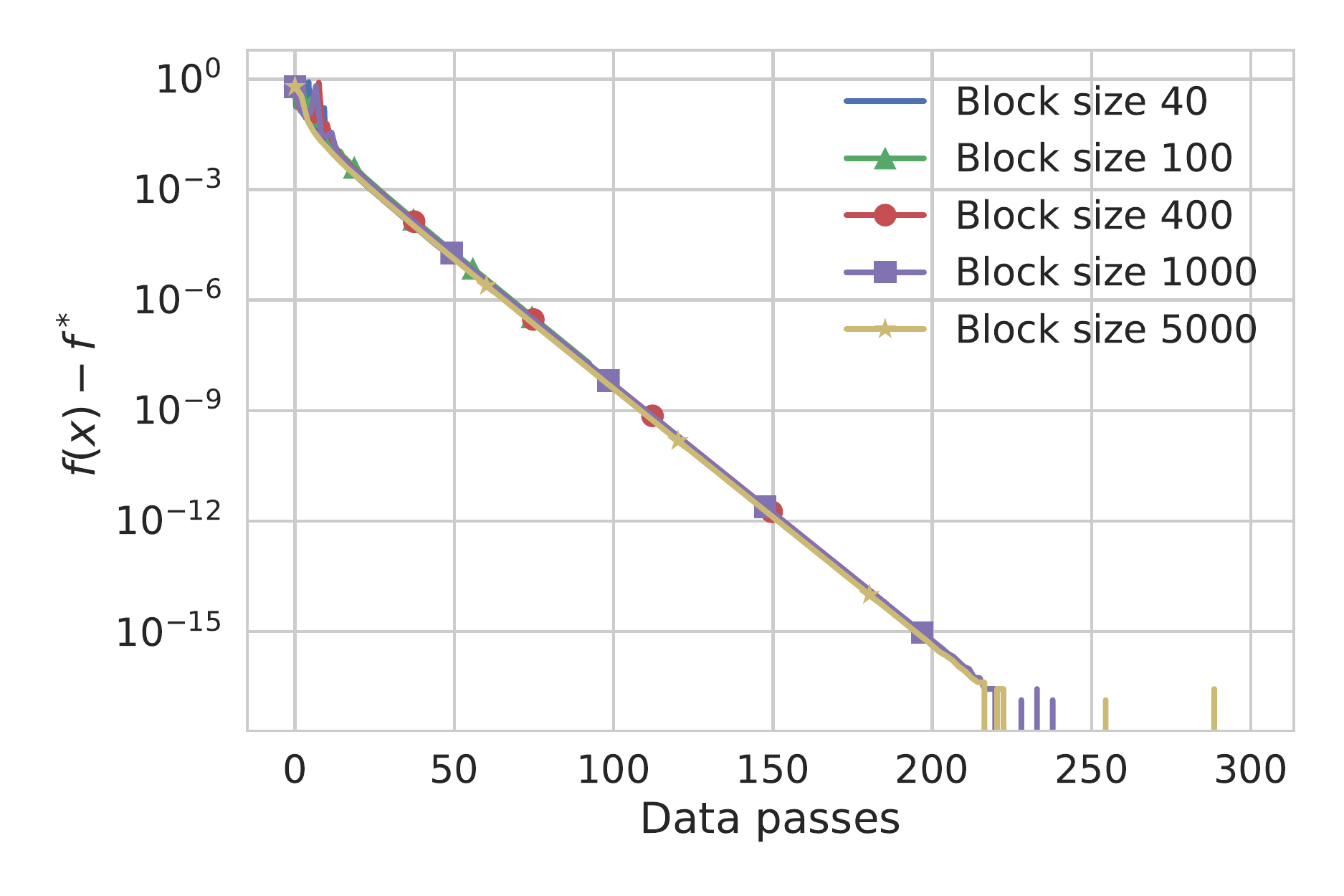}}
	}
	\makebox[\textwidth][c]{
		\subfigure[\texttt{$\ell_{2}$, $\lambda_2 = 2\cdot 10^{-1}$}]
		{\includegraphics[scale=0.2]{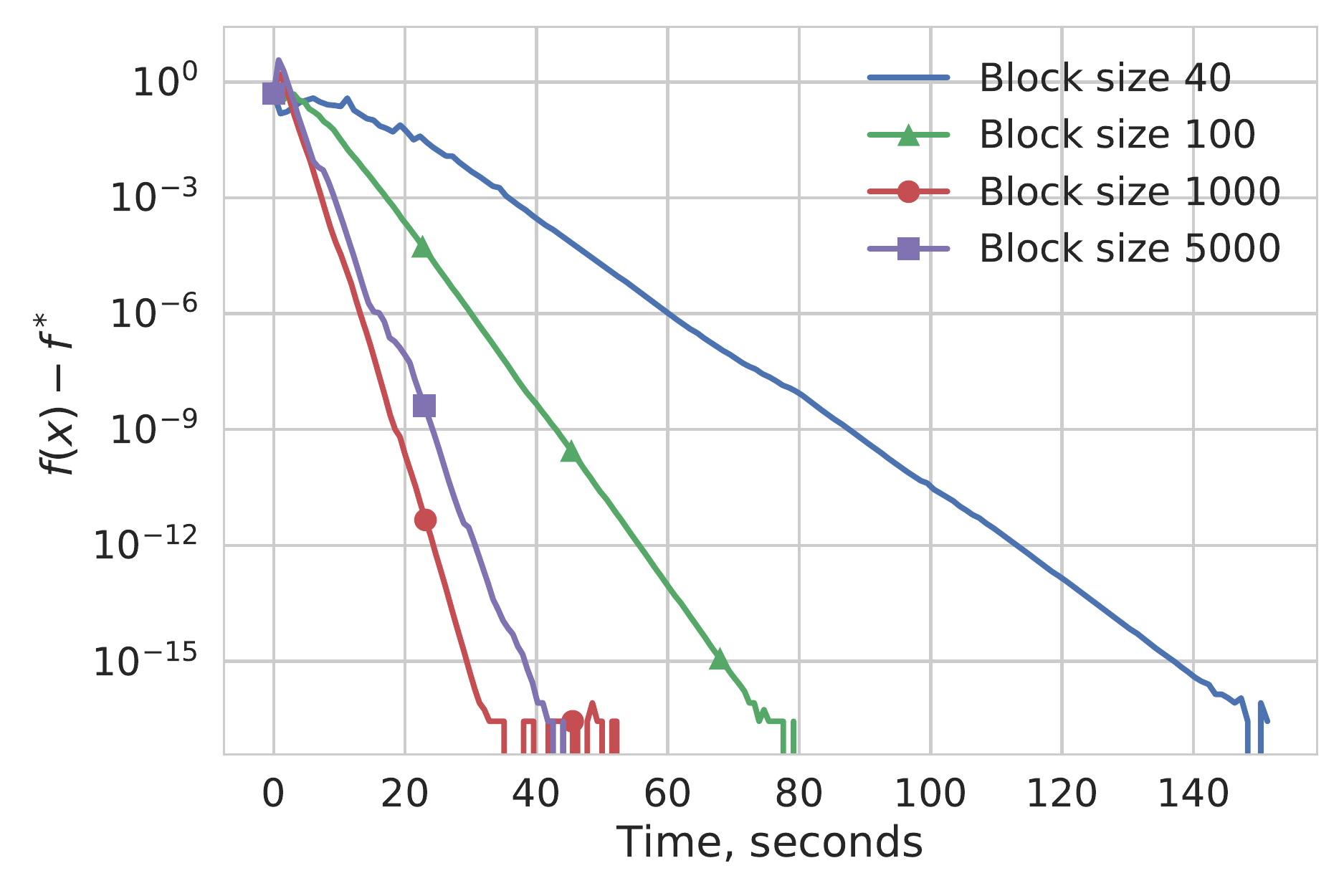}}
		\subfigure[\texttt{$\ell_{2}$, $\lambda_2 = 2\cdot 10^{-2}$}]
		{\includegraphics[scale=0.2]{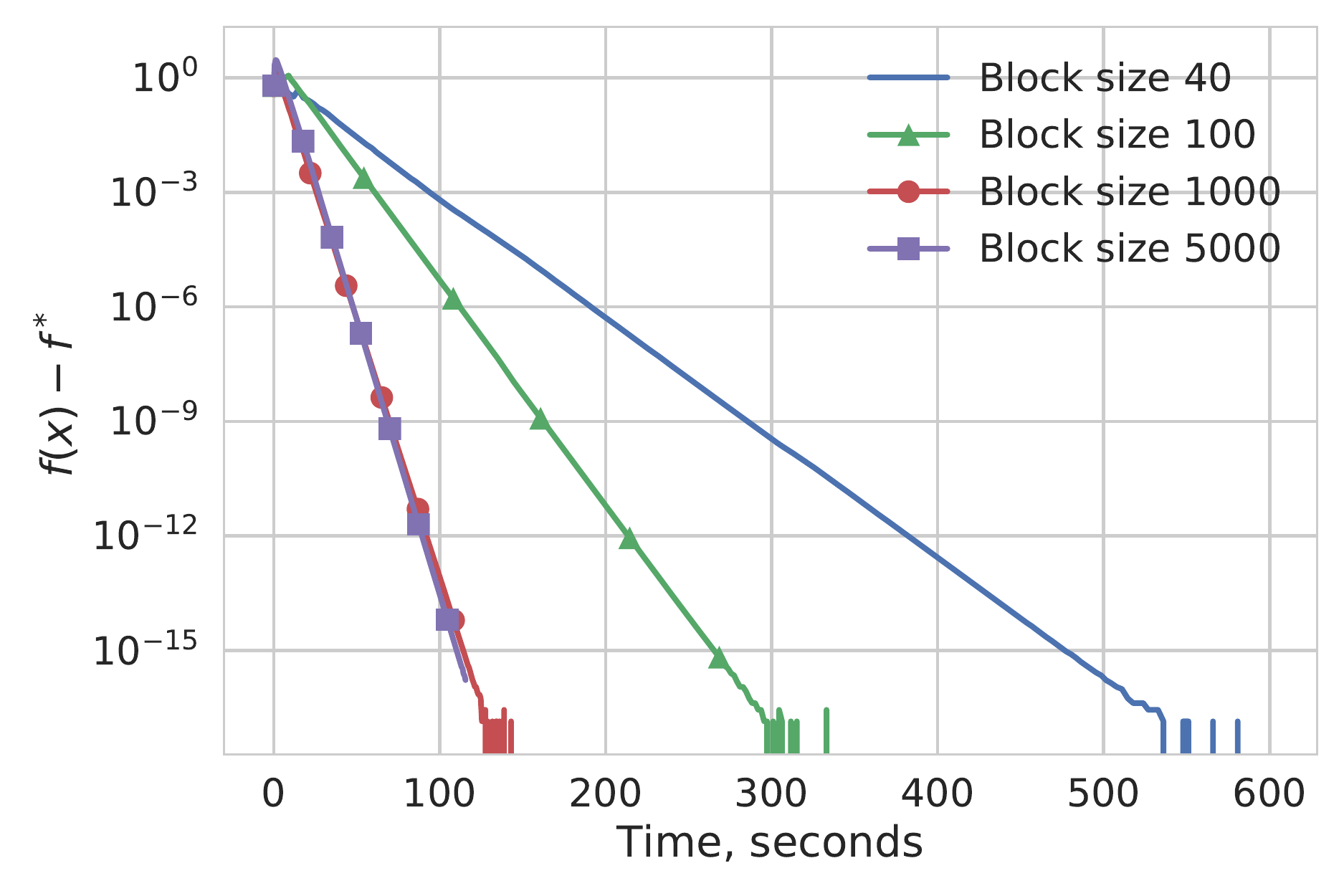}}
		\subfigure[\texttt{$\ell_{2}$, $\lambda_2 = 2\cdot 10^{-1}$}]
		{\includegraphics[scale=0.2]{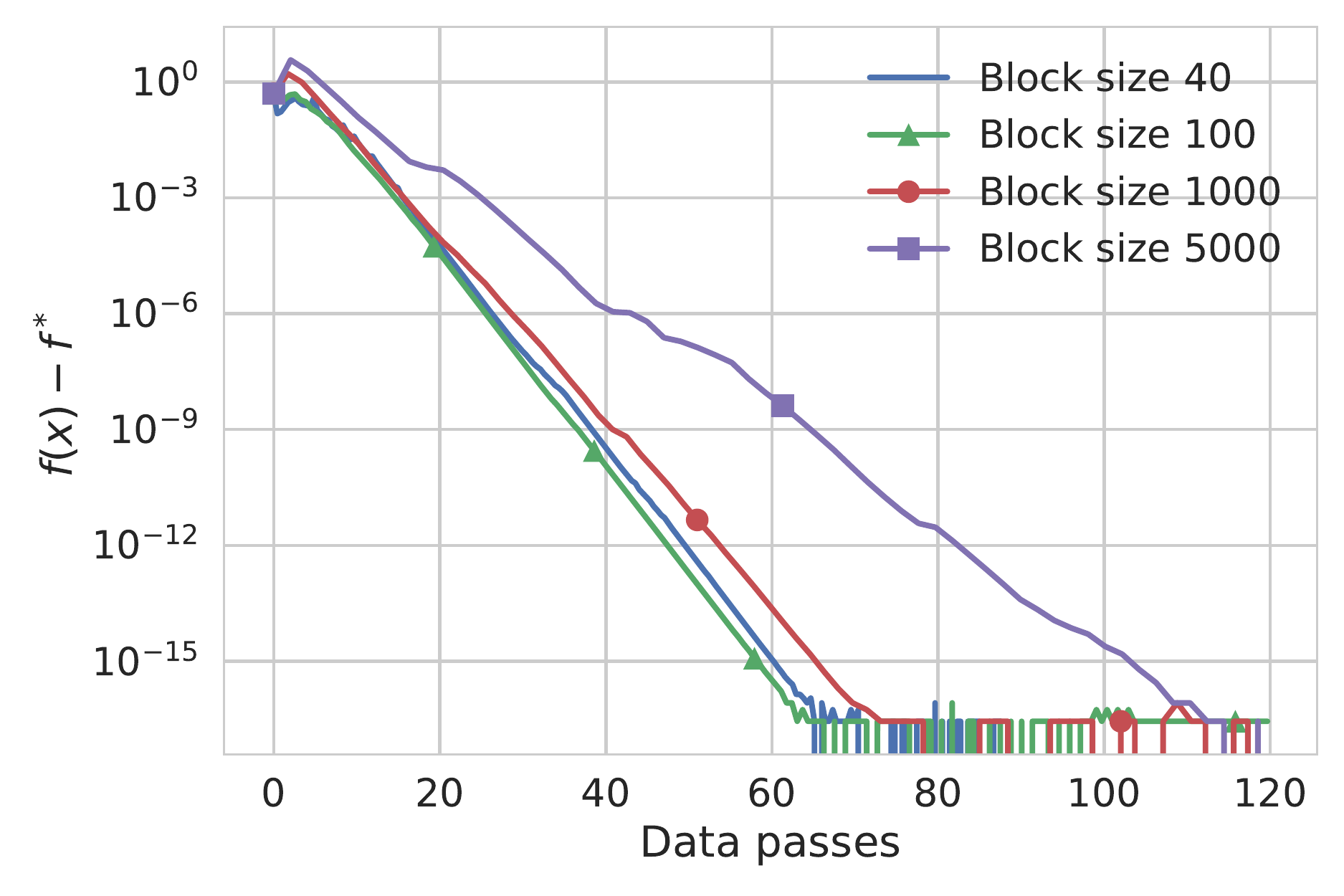}}
		\subfigure[\texttt{$\ell_{2}$, $\lambda_2 = 2\cdot 10^{-2}$}]
		{\includegraphics[scale=0.2]{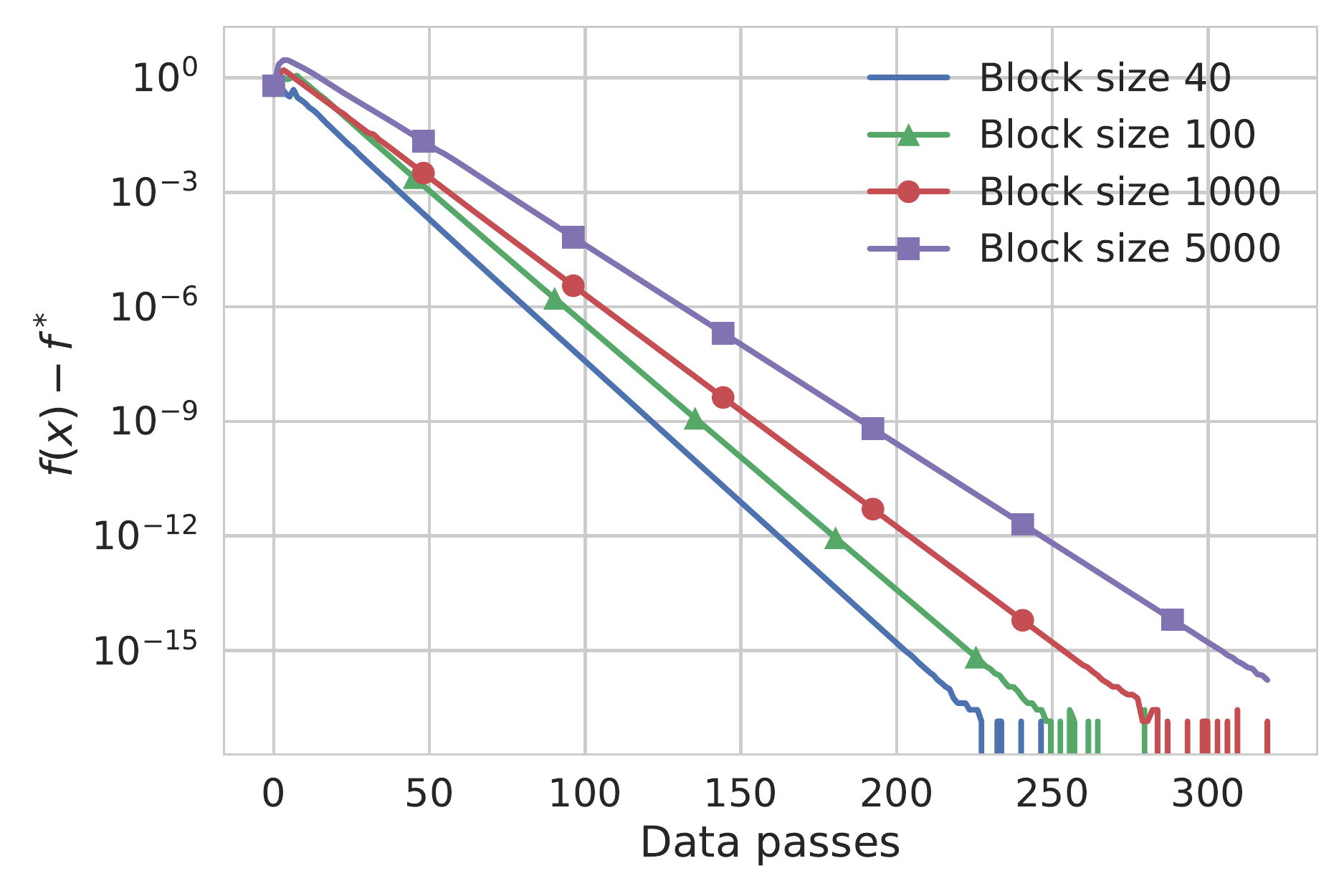}}
	}
	\caption{Experiments with \texttt{DIANA}-\texttt{SVRG} and different block sizes applied to the Gisette dataset ($d=5000$) with $n=20$ workers. In the first two columns we show convergence over time and in the last two we show convergence over epochs. We used 1-bit random dithering with $\ell_{\infty}$ (first row) and $\ell_2$ norm and found out that even quantization with full vector quantization often does not slow down iteration complexity, but helps significantly with communication time. At the same time, $\ell_2$ dithering is noisier and yields a more significant impact of the block sizes on the iteration complexity. For each line in the plots, an optimal stepsize was found and we found out that larger block sizes require slightly smaller steps in case of $\ell_2$ random dithering. In all cases, we chose $\alpha$ to be $\frac{1}{2\omega}$, where $\omega$ was computed using the block size. Best seen in color by zooming on a computer screen.}
\end{figure}
\begin{figure}[H]
\begin{center}
\makebox[\textwidth][c]{
	\subfigure[\texttt{SVRG}]
	{\includegraphics[scale=0.32]{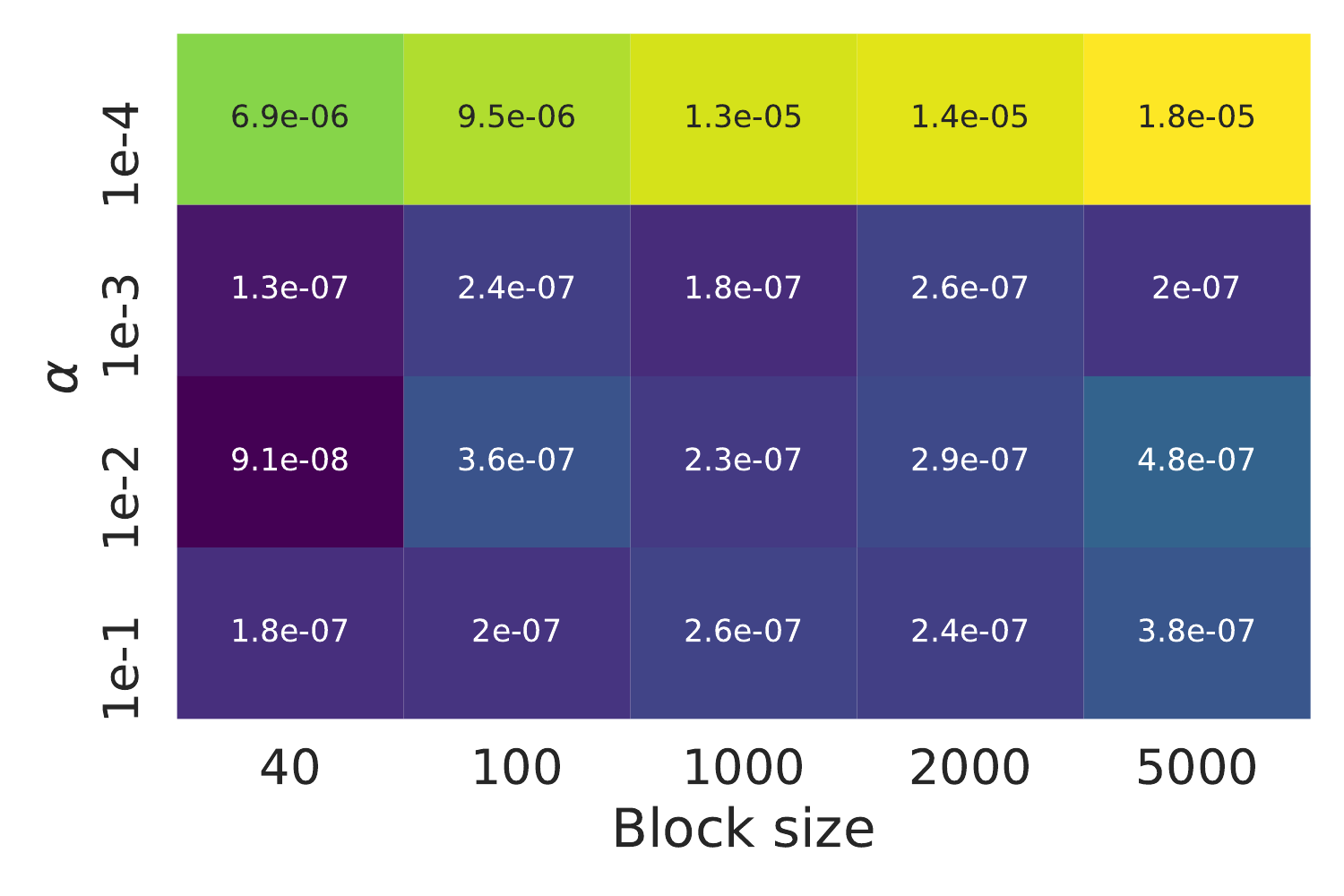}}
	\subfigure[\texttt{SAGA}]
	{\includegraphics[scale=0.32]{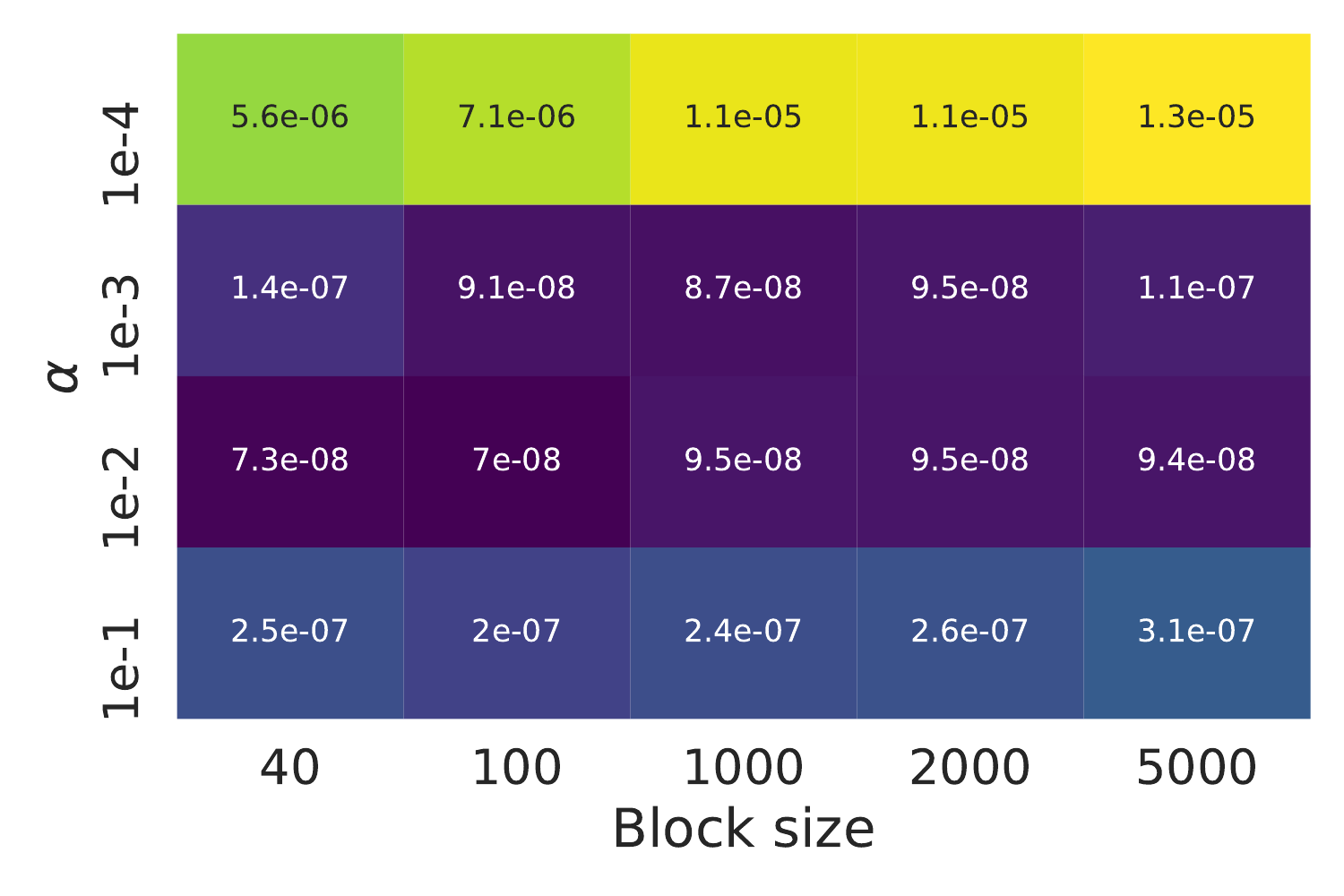}}
	\subfigure[\texttt{L-SVRG}]
	{\includegraphics[scale=0.32]{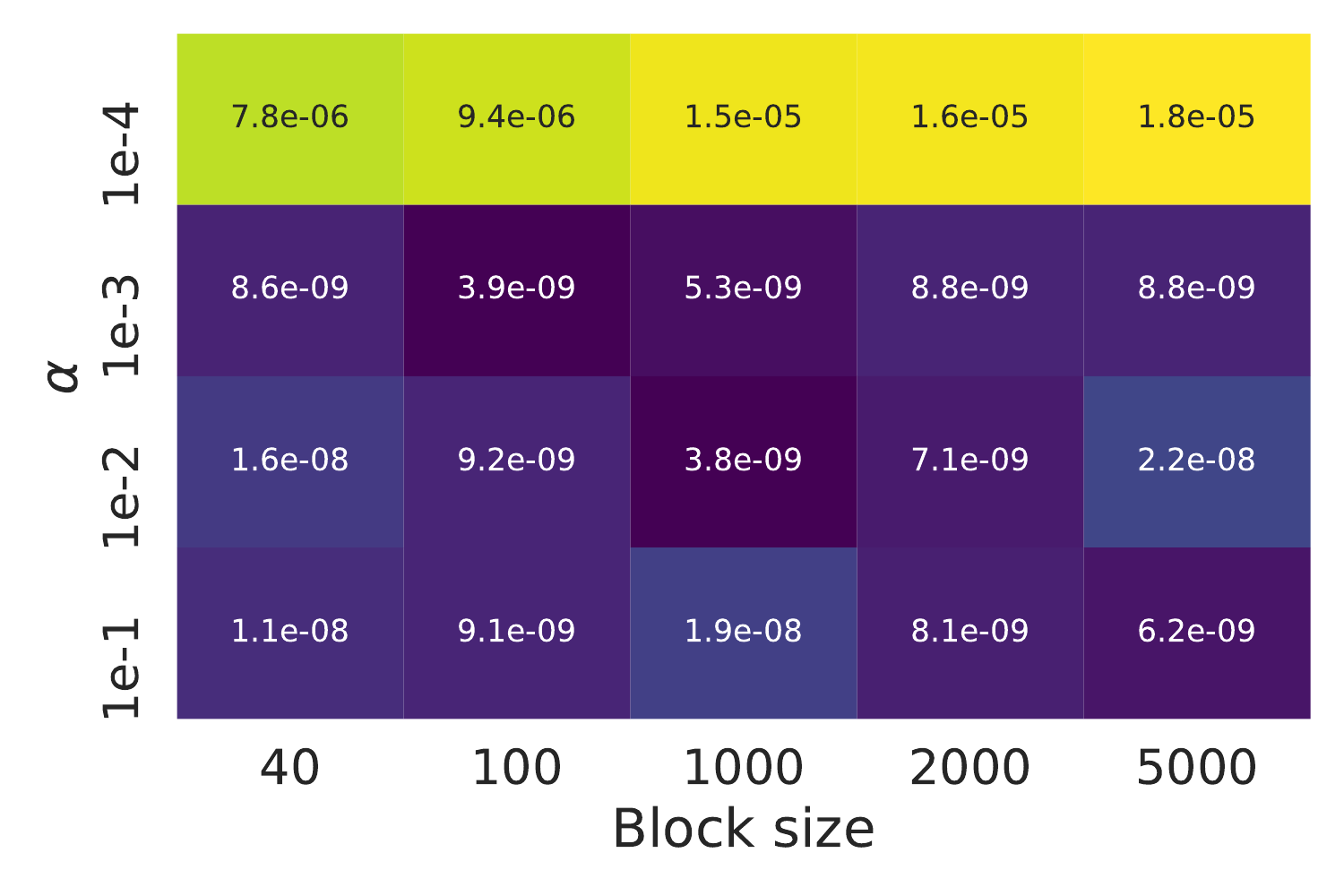}}
	}
	\caption{\label{fig:blocks_and_alphas}Functional gap after 20 epochs over the Gisette dataset ($d=5000$) with \texttt{DIANA}-VR and different combinations of block sizes and $\alpha$ with $n=20$ workers. The same (optimal) stepsize is used in all cases with $\ell_{\infty}$ random  dithering. For each cell, we average the results over 3 runs with the same parameters as the final accuracy is random. The best iteration performance is achieved using small blocks as expected, and the value of $\alpha$ does not matter much unless chosen too far from $\frac{1}{\omega + 1}$. The values of $\frac{1}{\omega + 1}$ for the columns from the left to the right approximately are $\{0.14,  0.09, 0.03, 0.02, 0.01\}$. Best seen in color by zooming on a computer screen.}
\end{center}
\end{figure}

\section{Details}
Throughout the whole appendix we use conditional expectation $\E{\cX| x^k, h_i^k}$ for \texttt{DIANA} and $\E{\cX|  x^k, h_i^k, w_{ij}^k}$ for VR-\texttt{DIANA} and $\E{\cX|  x^k, h_i^k, z^s}$ for \texttt{SVRG}-\texttt{DIANA}, but for simplicity, we will denote these expectations as $\E{\cX}$. If $\E{\cX}$ refers to unconditional expectation, it is directly mentioned.

\section{\texttt{DIANA} with arbitrary quantization}
\label{sec:generalDIANA}

In this section, we present the generalized analysis of the \texttt{DIANA} algorithm~\cite{mishchenko2019distributed} by allowing for {\em arbitrary $\omega$-quantization operators} (random dithering quantization operators \eqref{eq:Q}  with $p\geq 1$ and $s=1$ only were considered in \cite{mishchenko2019distributed}).

\begin{algorithm}[h]
\begin{algorithmic}[1]
  \STATE {\bfseries Input:} learning rates $\alpha > 0$, and $\gamma > 0$,
  initial vectors $x^0$, $h_1^0, \dots, h_n^0$ and $h^0 = \frac{1}{n} \sum_{i=1}^n h_i^0$
  \FOR{$k=0,1,\dots$}
  \STATE broadcast $x^k$ to all workers\; 
  \STATE $\triangleright$\ \textit{Worker side}
  \FOR{ $i=1,\dots,n$ }
   \STATE sample $g_i^k$ such that $\E{g_i^k \mid x^k} = \nabla f_i(x^k)$\;   
   \STATE $\Delta_i^k = g_i^k - h_i^k$\;
   \STATE $\hat{\Delta}_i^k = \cC(\Delta_i^k)$\;
   \STATE $h_i^{k+1} = h_i^k + \alpha \hat{\Delta}_i^k$\;
   \STATE $\hat{g}_i^k = h_i^k + \hat{\Delta}_i^k$\;
  \ENDFOR
   \STATE $q^k = \tfrac{1}{n} \sum_{i=1}^n \hat{\Delta}_i^k$ 
   \STATE $\hat{g}^k = h^k + q^k$\;
   \STATE $x^{k+1} = \prox_{\gamma R}(x^k - \gamma \hat{g}^k)$\;
   \STATE $h^{k+1}= h^k + \alpha q^k \left( = \frac1n \sum_{i=1}^n h_i^{k+1}\right)$
 \ENDFOR
\end{algorithmic}  
\caption{\texttt{DIANA} with arbitrary unbiased quantization}
\label{alg:improvedDIANA}
\end{algorithm}

\subsection{Convergence of \texttt{DIANA} with arbitrary quantization}
 We make the following assumptions:
\begin{assumption}
\label{assumption:diana}
In problem~\eqref{eq:probR_diana} we assume each $f_i \colon \R^d \to \R$ to be $\mu$-strongly convex and $L$-smooth. We assume that each $g_i^k$ in~Algorithm~\ref{alg:improvedDIANA} has bounded variance
\begin{align}
 \E{\norm*{g_i^k - \nabla f_i(x^k)}^2 } &\leq \sigma_i^2\,, & &\forall k \geq 0, \quad i=1,\dots,n \label{def:sigmai}
\end{align}
for constants $\sigma_i \leq \infty$, $\sigma^2 \eqdef \frac{1}{n} \sum_{i=1}^n \sigma_i^2$.\end{assumption}

We show linear convergence of Algorithm~\ref{alg:improvedDIANA} with arbitrary $\omega$-quantization schemes:
\begin{theorem}
\label{thm:improvedDIANA}
Consider Algorithm~\ref{alg:improvedDIANA} with $\omega$-quantization $\cC$ and stepsize $\alpha \leq \frac{1}{\omega+1}$.
Define the Lyapunov function 
\vspace{-0.5ex}
\begin{align*}
 \Psi^k \eqdef \norm*{x^k - x^\star}^2 + \frac{c \gamma^2}{n} \sum_{i=1}^n \norm*{h_i^k - \nabla f_i(x^\star)}^2
\end{align*}
for $c \geq \frac{4 \omega}{\alpha n}$
and assume $\gamma \leq \frac{2}{(\mu + L)(1+\frac{2\omega}{n} + c\alpha)}$ and $\gamma \leq \frac{\alpha}{2\mu}$. Then under Assumption~\ref{assumption:diana}:
\begin{align}
 \E{\Psi^k} \leq  (1- \gamma \mu)^k \Psi^0 + \frac{2}{\mu(\mu+L)} \sigma^2 \,. \label{eq:h98gf98f}
\end{align}
\end{theorem}
\begin{corollary}
Let $c = \frac{4 \omega}{\alpha n}$, $\alpha = \frac{1}{\omega+1}$ and $\gamma = \min\bigl\{ \frac{2}{(\mu + L)(1+\frac{6\omega}{n})}, \frac{1}{2\mu(\omega +1)} \bigr\}$. Furthermore, define $\kappa = \frac{L+\mu}{2\mu}$. Then the conditions of Theorem~\ref{thm:improvedDIANA} are satisfied and the leading term in the iteration complexity bound is 
\[
\frac{1}{\gamma \mu} = \kappa + \kappa \frac{2\omega}{n} + 2(\omega+1)\,.
\]
\end{corollary}
\begin{remark}
For the special case of quantization~\eqref{eq:Q} in arbitrary $p$-norms and $s=1$, this result recovers~\cite{mishchenko2019distributed} up to small differences in the constants. 

\end{remark}

\subsection{Proof of Theorem~\ref{thm:improvedDIANA}}

\begin{lemma}
\label{lemma:boundg}
For all iterations $k \geq 0$ of Algorithm~\ref{alg:improvedDIANA} it holds
\begin{align}
 \EE{\cC}{\hat{g}^k} &= g^k \eqdef \frac{1}{n}\sum_{i=1}^{n}g_i^k\,, &
 \EE{\cC}{\norm*{\hat{g}^k-g^k}^2} &\leq \frac{\omega}{n^2} \sum_{i=1}^{n}\norm*{\Delta_i^k}^2\,, & 
 \E{g^k} &= \nabla f(x^k)\,.
 \label{eq:explain0}
\end{align}
Furthermore, for $h^\star = \nabla f(x^\star)$, $h_i^\star \eqdef \nabla f_i(x^\star)$
\begin{align}
 \E{\norm*{\hat{g}^k-h^\star}^2} &\leq \left(1 + \frac{2\omega}{n} \right)  \frac{1}{n} \sum_{i=1}^{n} \E{\norm*{\nabla f_i(x^k) - h_i^\star}^2} \notag \\
 &\quad + \left(1 + \omega\right)\frac{\sigma^2}{n}  + \frac{2\omega}{n^2} \sum_{i=1}^n \E{\norm*{h_i^\star - h_i^k}^2}\,. 
\label{eq:explain2}
\end{align}
\end{lemma}
\begin{proof}
The first equation in~\eqref{eq:explain0} follows from the unbiasedness of the quantization operator.
By the contraction property~\eqref{def:omega} we have
\begin{align*}
 \EE{\cC}{\norm*{\hat{g}_i^k - g_i^k}^2} \leq \omega \norm*{\Delta_i^k}^2\,,
\end{align*}
for every $i=1,\dots, n$ and the second relation in~\eqref{eq:explain0} follows from independence of $\hat{g}_1^k,\dots,\hat{g}_n^k$.
The last equality in~\eqref{eq:explain0} follows from the assumption that each $g_i^k$ is an unbiased estimate of $\nabla f_i(x^k)$.

By applying two times the identity 
\[\E{\norm*{X-y}^2} = \E{\norm*{X-\E{X}}^2} + \E{\norm*{\E{X}-y}^2}\] for random variable $X$ and $y \in \R^d$, 
we get
\begin{align*}
\E{\norm*{\hat{g}^k-h^\star}^2} 
&\stackrel{\eqref{eq:var_decomposition}}{=} \E{\norm*{\hat{g}^k - g^k}^2} + \E{\norm*{g^k - h^\star}^2} \\
&\stackrel{\eqref{eq:var_decomposition}}{=} \E{\norm*{\hat{g}^k - g^k}^2} + \E{\norm*{g^k - \nabla f(x^k)}^2} + \E{\norm*{\nabla f(x^k) - h^\star}^2}
\end{align*}
and thus
\begin{align}
 \E{\norm*{\hat{g}^k-h^\star}^2} &\stackrel{\eqref{def:sigmai}+\eqref{eq:explain0}}{\leq} \frac{\omega}{n^2} \sum_{i=1}^{n} \E{ \norm*{\Delta_i^k}^2} + \frac{\sigma^2}{n} +  \E{\norm*{\nabla f(x^k) - h^\star}^2}\,. \label{eq:5435}
\end{align}
Note that
\begin{align*}
\norm*{\nabla f(x^k) - h^\star}^2 \leq \frac{1}{n} \sum_{i=1}^n \norm*{\nabla f_i(x^k) - h_i^\star}^2\,,
\end{align*}
by Jensen's inequality. Further,
\begin{align*}
\E{ \norm*{\Delta_i^k}^2 } &= \E{ \norm*{g_i^k - h_i^k }^2} \stackrel{\eqref{eq:var_decomposition}}{=} \E{ \norm*{\nabla f_i(x^k) - h_i^k}^2 } + \E{\norm*{\nabla f_i(x^k) - g_i^k}^2} \\
 &\stackrel{\eqref{def:sigmai}}{\leq} \E{\norm*{\nabla f_i(x^k) - h_i^k}^2} + \sigma_i^2 \\
 &\stackrel{\eqref{eq:sum_upper}}{\leq} 2 \E{\norm*{\nabla f_i(x^k) - h_i^\star}^2} +2 \E{\norm*{h_i^\star - h_i^k}^2} + \sigma_i^2.
\end{align*} 
By summing up these bounds and plugging the result into~\eqref{eq:5435}, equation~\eqref{eq:explain2} follows.
\end{proof}

\begin{lemma}
Let $\alpha (\omega + 1) \leq 1$. For $i=1,\dots,n$, we can upper bound the second moment of $h_i^{k+1}$ as
\begin{align}
 \EE{\cC}{\norm*{h_{i}^{k+1} - h_i^\star}^2} \leq (1-\alpha) \norm*{h_{i}^{k} - h_i^\star}^2 +\alpha\norm*{\nabla f_i(x^k)- h_i^\star}^2 + \alpha \sigma^2_i \,.
 \label{eq:hdecrement}
\end{align}
\end{lemma}
\begin{proof}
Since $h_i^{k+1} = h_i^k + \alpha \hat{\Delta}_i^k$ we can decompose
\begin{eqnarray*}
 \EE{\cC}{\norm*{h_{i}^{k+1} - h_i^\star}^2} &=& 
 \EE{\cC}{\norm*{\alpha \hat{\Delta}_i^k + (h_i^k - h_i^\star)}^2} \\
 &=& \norm*{h_{i}^{k} - h_i^\star}^2 + 2 \EE{\cC}{\lin{\alpha \hat{\Delta}_i^k,h_i^k - h_i^\star}} + \EE{\cC}{\norm*{\alpha \hat{\Delta}_i^k}^2} \\
%
&\stackrel{\eqref{eq:omega_quant}}{\leq}& \norm*{h_{i}^{k} - h_i^\star}^2  +  2\lin{\alpha \Delta_i^k,h_i^k - h_i^\star} + \alpha^2 (\omega+1) \norm*{\Delta_i^k}^2.
\end{eqnarray*}
Let plug in the bound $(\omega+1) \alpha \leq 1$ and continue the derivation:
\begin{align*}
 \EE{\cC}{\norm*{h_{i}^{k+1} - h_i^\star}^2} &\leq \norm*{h_{i}^{k} - h_i^\star}^2  +  2\lin{\alpha \Delta_i^k,h_i^k - h_i^\star} + \alpha \norm*{\Delta_i^k}^2 \\
 &=\norm*{h_{i}^{k} - h_i^\star}^2 + \alpha \lin{g_i^k - h_i^k,g_i^k + h_i^k - 2 h_i^\star} \\
 &=\norm*{h_{i}^{k} - h_i^\star}^2 + \alpha\norm*{g_i^k- h_i^\star}^2 - \alpha\norm*{h_i^k - h_i^\star}^2 \\
&\leq (1-\alpha) \norm*{h_{i}^{k} - h_i^\star}^2 +  \alpha\norm*{g_i^k- h_i^\star}^2.
\end{align*}
The second term can be further upper-bounded by $\norm*{\nabla f_i(x^k)- h_i^\star}^2 + \sigma^2_i$, where we use \eqref{eq:var_decomposition}, which concludes the proof.
\end{proof}

\begin{proof}[Proof of Theorem~\ref{thm:improvedDIANA}]
If $x^\star$ is a solution of~\eqref{eq:probR_diana}, then $x^\star = \prox_{\gamma R}(x^\star - \gamma h^\star)$ (for $\gamma > 0$. Using this identity together with the non-expansiveness of the prox operator we can bound the first term of the Lyapunov function:
\begin{align*}
 \E{\norm*{x^{k+1}- x^\star}^2} 
 &= \E{\norm*{\prox_{\gamma R}(x^k - \gamma \hat{g}^k) - \prox_{\gamma R}(x^\star - \gamma h^\star)}^2} \\
 &\leq \E{\norm*{x^k - \gamma \hat{g}^k - (x^\star - \gamma h^\star))}^2} \\
 &=\E{\norm*{x^k - x^\star}^2} - 2 \gamma \E{\lin{\hat{g}^k - h^\star, x^k - x^\star}} + \gamma^2 \E{\norm*{\hat{g}^k - h^\star}^2} \\
 &= \E{\norm*{x^k - x^\star}^2} - 2 \gamma \lin{\nabla f(x^k) - h^\star, x^k - x^\star} + \gamma^2 \E{\norm*{\hat{g}^k - h^\star}^2}. 
\end{align*}
Next, we use strong convexity of each component $f_i$:
\begin{eqnarray*}
 &&\E{ \lin{\nabla f(x^k) - h^\star, x^k - x^\star}} \\
 &=&
 \frac{1}{n} \sum_{i=1}^n  \E{ \lin{\nabla f_i(x^k) - h_i^\star, x^k - x^\star}} \\
 &\stackrel{\eqref{eq:coerc}}{\geq}& \frac{1}{n}\sum_{i=1}^n \left( \frac{\mu L}{\mu +L} \E{\norm*{x^k - x^\star}^2} + \frac{1}{\mu + L}  \E{\norm*{\nabla f_i(x^k) - h_i^\star}^2} \right) \\
 &=& \frac{\mu L}{\mu +L} \E{\norm*{x^k - x^\star}^2} + \frac{1}{\mu + L} \frac{1}{n} \sum_{i=1}^n \E{\norm*{\nabla f_i(x^k) - h_i^\star}^2} \,.
\end{eqnarray*}
Hence,
\begin{align*}
\E{\norm*{x^{k+1}- x^\star}^2}  &\leq \left(1- \frac{2 \gamma \mu L}{\mu +L}  \right)  \E{\norm*{x^k - x^\star}^2} - \frac{2\gamma}{\mu + L} \frac{1}{n} \sum_{i=1}^n \E{\norm*{\nabla f_i(x^k) - h_i^\star}^2}  \\
&\quad + \gamma^2 \E{\norm*{\hat{g}^k - h^\star}^2} 
\end{align*}
and by Lemma~\ref{lemma:boundg}:
\begin{align}
\begin{split}
\E{\norm*{x^{k+1}- x^\star}^2}  &\stackrel{\eqref{eq:explain2}}{\leq} \left(1- \frac{2 \gamma \mu L}{\mu +L}  \right)  \E{\norm*{x^k - x^\star}^2} \\
&\qquad + \left(\gamma^2 \left(1+ \frac{2 \omega}{n}\right) -  \frac{2\gamma}{\mu + L} \right) \frac{1}{n} \sum_{i=1}^n \E{\norm*{\nabla f_i(x^k) - h_i^\star}^2}  \\ &\qquad +\gamma^2 \left(1+\omega\right) \frac{\sigma^2}{n} + \left( \gamma^2 \frac{2\omega}{n}\right) \frac{1}{n} \sum_{i=1}^n \E{\norm*{h_i^k - h_i^\star}^2}\,.
\end{split}
\label{eq:6943}
\end{align}
Now let us consider the Lyapunov function:
\begin{align}
\begin{split}
 \E{\Psi^{k+1}} & \stackrel{\eqref{eq:hdecrement}+\eqref{eq:6943}}{\leq} \left(1- \frac{2 \gamma \mu L}{\mu +L}  \right)  \E{\norm*{x^k - x^\star}^2} + \gamma^2\left(1+ \omega\right) \frac{\sigma^2}{n} \\
 &\qquad +   \left(\gamma^2 \left(1+ \frac{2\omega}{n} +c\alpha\right) -  \frac{2\gamma}{\mu + L} \right) \frac{1}{n} \sum_{i=1}^n \E{\norm*{\nabla f_i(x^k) - h_i^\star}^2} \\
 &\qquad + \gamma^2  \left(\frac{2\omega}{n} + (1-\alpha) c \right) \frac{1}{n} \sum_{i=1}^n \norm*{h_i^k - h_i^\star}^2 + \gamma^2c\alpha \sigma^2 \,.
 \end{split}
 \label{eq:2945}
\end{align}
In view of the assumption on $\gamma$ we have 
$\gamma^2 \left(1+ \frac{2\omega}{n} +c\alpha\right) -  \frac{2\gamma}{\mu + L} \leq 0$.
Since each $f_i$ is $\mu$-strongly convex, we have
$ \mu \norm*{x^k - x^\star}^2 \stackrel{\eqref{eq:coerc2}}{\leq} \lin{\nabla f_i(x^k)-h_i^\star,x^k-x^\star}$
and thus $\mu^2 \norm*{x^k - x^\star}^2 \leq \norm*{\nabla f_i(x^k)-h_i^\star}^2$ with Cauchy-Schwarz. 
Using these observations we can absorb the third term in~\eqref{eq:2945} in the first one:
\begin{align*}
\begin{split}
 \E{\Psi^{k+1}} & \leq \left(1- 2\gamma \mu + \mu^2 \gamma^2 \left(1+ \frac{2\omega}{n} +c\alpha\right) \right)  \E{\norm*{x^k - x^\star}^2} + \gamma^2 \left(c\alpha+ \frac{\omega+1}{n}\right) \sigma^2 \\
 &\qquad +  \gamma^2\left(\frac{2\omega}{n} + (1-\alpha)c\right) \frac{1}{n} \sum_{i=1}^n \norm*{h_i^k - h_i^\star}^2\,.
 \end{split}
 \end{align*}
 By the first assumption on $\gamma$ it follows $\left(1- 2\gamma \mu + \mu^2 \gamma^2 \left(1+ \frac{2\omega}{n} +c\alpha\right) \right) \leq( 1-\gamma \mu)$. 
By the assumption on $c$ we have $\left(\frac{2\omega}{n} + (1-\alpha)c\right) \leq \left(1- \frac{\alpha}{2}\right)c$. 
 An the second assumption on $\gamma$ implies $\left(1-\frac{\alpha}{2}\right)\leq(1-\gamma\mu)$. Thus
\begin{align*}
\E{\Psi^{k+1}} \leq (1-\gamma\mu) \Psi^k + \gamma^2  \left(c\alpha+ \frac{\omega+1}{n}\right) \sigma^2\,.
\end{align*}
Unrolling the recurrence and the estimate $\sum_{\ell=0}^{k-1} (1-\gamma \mu)^k \leq \frac{1}{\mu \gamma}$ for all $k \geq 1$ leads to
\begin{align*}
 \E{\Psi^k} \leq (1- \gamma \mu)^k \Psi^0 + \frac{\gamma}{\mu}  \left(c\alpha+ \frac{\omega+1}{n}\right) \sigma^2 \leq  (1- \gamma \mu)^k \Psi^0 + \frac{2}{\mu(\mu+L)} \sigma^2\,,
\end{align*}
by the first assumption on $\gamma$.
\end{proof}

\section{Variance reduced \texttt{DIANA}: \texttt{L-SVRG} method and \texttt{SAGA} detailed theorems and proofs}

We are going to prove a convergence of a Lyapunov function that will include convergence of $x^k$ to $x^\star$, of $h_i^k$ to $\nabla f_i(x^\star)$ for each $i$ and also that of individual gradients, $\nabla f_{ij}(w_{ij}^k)$ to $\nabla f_{ij}(x^\star)$. For each term we have simple recursions that are formulated in the lemmas below.
\label{sec:vrdiana}

\begin{lemma}
For all iterates $k \leq 0$ of Algorithm~\ref{alg:VR-DIANA}, it holds that $g_i^k$ is an unbiased estimate of the local gradient $\nabla f_i(x^k)$
	\begin{equation*}
		\E{g_i^k} = \nabla f_i(x^k)
	\end{equation*}
	and $g^k$ is that of the full gradient $\nabla f(x^k)$:
	\begin{equation*}
	\E{g^k} = \nabla f(x^k).
	\end{equation*}
\end{lemma}
\begin{proof}
	It is a straightforward consequence of how we define sampling:
	\begin{align*}
		\E{g_i^k}
		=
		\E{\nabla f_{ij_i^k}(x^k) - \nabla f_{ij_i^k}(w_{ik_i^k}^k) + \mu_i^k}
		=
		\nabla f_i(x^k) - \mu_i^k + \mu_i^k
		=
		\nabla f_i(x^k).
	\end{align*}
	Similarly,
	\begin{align*}
		\E{g^k}
		&=
		\frac{1}{n}
		\sum\limits_{i=1}^{n}
			\E{\cC(g_i^k - h_i^k) + h_i^k} \\
		&=
		\frac{1}{n}
		\sum\limits_{i=1}^{n}
			\E{g_i^k - h_i^k + h_i^k} \\
		&=
		\frac{1}{n}
		\sum\limits_{i=1}^{n}
			\nabla f_i(x^k) \\
		&=
		\nabla f(x^k).
	\end{align*}
\end{proof}

\subsection{Strongly convex case}
\begin{lemma}
Let $f$ be $\mu$-strongly convex. Then, we can upper bound the second moment of $x^k $ in the following way
	\begin{equation}
	\begin{split}
	    \E{\norm*{x^{k+1} - x^\star}^2} &\leq (1 - \mu \gamma)\norm*{x^k - x^\star}^2 - 2\gamma B_f(x^k, x^\star)  \\
	    &\quad + \gamma^2 \E{\norm*{g^k - \nabla f(x^\star)}^2}.
	\end{split}
    \label{lem:up_x_k_VR}	
	\end{equation}
\end{lemma}
\begin{proof}
	\begin{eqnarray*}
	&&\E{\norm*{x^{k+1} - x^\star}^2}  \\
	&=& \E{\norm*{\prox_{\gamma R}(x^k - \gamma g^k) - \prox_{\gamma R}(x^\star - \gamma \nabla f(x^\star)}^2}\\
	&\le& \E{\norm*{x^k - \gamma g^k - (x^\star - \gamma \nabla f(x^\star))}^2}\\
	&=&
	\E{
		\norm*{x^{k} - x^\star}^2 + 2\gamma\dotprod{g^k - \nabla f(x^\star)}{x^\star - x^k} + \gamma^2 \norm*{g^k- \nabla f(x^\star)}^2
	}\\
	&=& \norm*{x^{k} - x^\star}^2 + 2\gamma\dotprod{\nabla f(x^k) - \nabla f(x^\star)}{x^\star - x^k} + \gamma^2 \E{\norm*{g^k - \nabla f(x^\star)}^2}\\
	&\overset{\eqref{eq:def_strongly_convex}}{\leq}&
	\norm*{x^{k} - x^\star}^2
	-
	2\gamma
	\left(B_f(x^k, x^\star) + \frac{\mu}{2}\norm*{x^k - x^\star}^2
	\right)
	+
	\gamma^2 \E{\norm*{g^k - \nabla f(x^\star)}^2}\\
	&=&
	(1 - \mu \gamma)\norm*{x^k - x^\star}^2 - 2\gamma B_f(x^k, x^\star) + \gamma^2 \E{\norm*{g^k - \nabla f(x^\star)}^2},
	\end{eqnarray*}
	where the first equation follows from the definition of $x^{k+1}$ in Algorithm~\ref{alg:VR-DIANA}.
\end{proof}

\begin{lemma} Let $\alpha(\omega+1) \leq 1$. We can upper bound $H^{k+1}$ in the following way 
	\begin{equation}
		\E{H^{k+1}}
		\leq
		(1-\alpha)H^k + \frac{2\alpha}{m} D^k + 8\alpha Ln  B_f(x^k, x^\star),
		\label{lem:up_H_VR}
	\end{equation}
	where 
	\begin{equation}
		H^{k}
		\eqdef
		\sum\limits_{i=1}^{n}
			\norm*{h_i^{k} - \nabla f_i(x^\star)}^2
		\label{lem:def_H_VR}
	\end{equation}
	and
	\begin{equation}
		D^{k}
		\eqdef
		\sum\limits_{i=1}^{n}
			\sum\limits_{j=1}^{m}
				\norm*{\nabla f_{ij}(w_{ij}^{k}) - \nabla f_{ij}(x^\star)}^2.
		\label{lem:def_D_VR}
	\end{equation}
\end{lemma}

\begin{proof}
	To get the recurrence for $H^{k+1}$, we start with simply using the definitions of $h_i^{k+1}$ and obtain
	\begin{eqnarray*}
		&&\E{H^{k+1}} \\
		&=&
		\E{\sum\limits_{i=1}^{n}
			\norm*{h_i^{k+1} - \nabla f_i(x^\star)}^2} \\
		&=&
		\sum\limits_{i=1}^{n}
			\norm*{h_i^k - \nabla f_i(x^\star)}^2 +
		\sum\limits_{i=1}^{n}
			\E{
				2\dotprod{\alpha \cC(g_i^k - h_i^k)}{h_i^k - \nabla f_i(x^\star)}
				+
				\alpha^2\norm*{\cC(g_i^k - h_i^k)}^2
			}\\
		&\leq&
		H^k
		+
		\sum\limits_{i=1}^{n}\E{
			2\alpha \dotprod{g_i^k - h_i^k}{h_i^k - \nabla f_i(x^\star)}
			+
			\alpha \left(\alpha \cdot (\omega+1)\right) \norm*{g_i^k - h_i^k}^2}\\
		&\overset{\text{Alg.}~\ref{alg:VR-DIANA}}{\le}&
		H^k
		+
		\E{\sum\limits_{i=1}^{n}
			\alpha \dotprod{g_i^k - h_i^k}{g_i^k + h_i^k - 2\nabla f_i(x^\star)}}\\
		&=&
		H^k
		+
		\E{\sum\limits_{i=1}^{n}
			\alpha 
			\left(
				\norm*{g_i^k - \nabla f_i(x^\star)}^2
				-
				\norm*{h_i^k - \nabla f_i(x^\star)}^2
			\right)},
		\end{eqnarray*}
		where the second equality uses definition of $h_i^{k+1}$ in Algorithm~\ref{alg:VR-DIANA}  and the first inequality follows from $\alpha(\omega+1) \leq 1$.
		
		Now, let us bound the second term in the produced bound:
		\begin{eqnarray*}
		&&\E{\sum\limits_{i=1}^{n}
			\alpha 
			\left(
				\norm*{g_i^k - \nabla f_i(x^\star)}^2
			\right)} \\
		&\overset{\eqref{eq:sum_upper}}{\leq}&
		\sum\limits_{i=1}^{n}
		\left(
			2\alpha \E{		
				\norm*{g_i^k - \nabla f_i(x^k)}^2
			} 
			+ 2\alpha\norm*{\nabla f_i(x^k) - \nabla f_i(x^\star)}^2
			\right)\\
		&\overset{\text{Alg.}~\ref{alg:VR-DIANA}}{=}&
		\sum\limits_{i=1}^{n}\E{
			2\alpha 
				\norm*{\nabla f_{ij_i^k}(x^k) - \nabla f_{ij_i^k}(w_{ij_i^k}^k) - \E{\nabla f_{ij_i^k}			(x^k) - \nabla f_{ij_i^k}(w_{ij_i^k}^k)}}^2
			} \\
			&&\quad +
			2\alpha
			\sum\limits_{i=1}^{n}
			\norm*{\nabla f_i(x^k) - \nabla f_i(x^\star)}^2
			\\
		&\overset{\eqref{eq:var_decomposition}}{\leq}&
		\sum\limits_{i=1}^{n}\left(
		\E{
			2\alpha 
				\norm*{\nabla f_{ij_i^k}(x^k) - \nabla f_{ij_i^k}(w_{ij_i^k}^k) }^2
			}
			+
			2\alpha\norm*{\nabla f_i(x^k) - \nabla f_i(x^\star)}^2
			\right)\\
		&\overset{\eqref{eq:sum_upper}}{\leq}&
		\frac{2\alpha}{m}
		\sum\limits_{i=1}^{n}
			\sum\limits_{j=1}^{m}
			\left( \norm*{\nabla f_{ij}(x^k) - \nabla f_{ij}(x^\star)}^2 +
				\norm*{\nabla f_{ij}(w_{ij}^k) - \nabla f_{ij}(x^\star)}^2\right)\\
		&&\quad +
			2\alpha
			\sum\limits_{i=1}^{n}
			\norm*{\nabla f_i(x^k) - \nabla f_i(x^\star)}^2
			\\
		&\overset{\eqref{eq:bregman_grad_dif}}{\leq}&
		\frac{2\alpha}{m}
		\sum\limits_{i=1}^{n}
			\sum\limits_{j=1}^{m}
				\norm*{\nabla f_{ij}(w_{ij}^k) - \nabla f_{ij}(x^\star)}^2
		+
		8\alpha Ln B_f(x^k, x^\star) \\
		&=&\frac{2\alpha}{m} D^k + 8\alpha Ln B_f(x^k, x^\star).
	\end{eqnarray*}
\end{proof}
Now we switch our attention to the convergence of local gradients.
\begin{lemma} If $f_{ij}$ is convex and $L$-smooth for all $i$ and $j$, we can upper bound $D^{k+1}$ in the following way
	\begin{equation}
		\E{D^{k+1}}
		\leq
		\left(1 - \frac{1}{m}\right) D^k 
		+
		2LnB_f(x^k, x^\star).
		\label{lem:up_D_VR}
	\end{equation}
\end{lemma}
\begin{proof}
	\begin{eqnarray*}
		\E{D^{k+1}}
		&=&
		\sum\limits_{i=1}^{n}
			\sum\limits_{j=1}^{m}
				\E{\norm*{\nabla f_{ij}(w_{ij}^{k+1}) - \nabla f_{ij}(x^\star)}^2}\\
		&=&
		\sum\limits_{i=1}^{n}
			\sum\limits_{j=1}^{m}
				\left[
					\left(
						1 - \frac{1}{m}
					\right)
					\norm*{\nabla f_{ij}(w_{ij}^k) - \nabla f_{ij}(x^\star)}^2
					+
					\frac{1}{m}
					\norm*{\nabla f_{ij}(x^k) - \nabla f_{ij}(x^\star)}^2			
				\right]\\
		&\overset{\eqref{eq:bregman_grad_dif}}{\leq}&
		\left(1 - \frac{1}{m}\right)D^k
		+
		2LnB_f(x^k, x^\star),
	\end{eqnarray*}
	where the second equality uses definition of $w_{ij}^{k+1}$ in Algorithm~\ref{alg:VR-DIANA}.
\end{proof}

\begin{lemma}
Let $f_{ij}$ be convex and $L$-smooth for all $i,j$. Then, we can upper bound the second moment of $g^k$ in the following way
	\begin{equation}
		\E{\norm*{g^k - \nabla f(x^\star)}^2}
		\leq
		2L\left(1 + \frac{4\omega+2}{n}\right)B_f(x^k,  x^\star)
		+
		\frac{2\omega}{mn^2}D^k
		+
		\frac{2(\omega+1)}{n^2} H^k.
		\label{lem:up_g_VR}
	\end{equation}
\end{lemma}
\begin{proof}
	Let us decompose the second moment of $g^k$ with respect to randomness of $\cC$,
	\begin{align*}
		\EE{\cC}{\norm*{g^k- \nabla f(x^\star)}^2} &\overset{\eqref{eq:var_decomposition}}{=}
		\underbrace{
			\norm*{\EE{\cC}{g^k - \nabla f(x^\star)}}^2
		}_{T_1}
		+
		\underbrace{
			\EE{\cC}{\norm*{g^k - \EE{\cC}{g^k}}^2}
		}_{T_2}.
	\end{align*}
	We proceed with upper bounding terms $T_1$ and $T_2$ separately. For $T_1$ we can use the definition of $g^k$ in order to obtain 
	\begin{align*}
		T_1
		&= 
		\Bigl\|
			\frac{1}{n}
			\sum\limits_{i=1}^{n}
				\EE{\cC}{\cC(g_i^k - h_i^k) + h_i^k - \nabla f(x^\star)}
		\Bigr\|^2
		=
		\Bigl\|
			\frac{1}{n}
			\sum\limits_{i=1}^{n}
				g_i^k- \nabla f(x^\star)
		\Bigr\|^2
	\end{align*}
	and 
	\begin{eqnarray*}
		T_2
		&=&
		\EE{\cC}{
			\Bigl\|
				\frac{1}{n}
				\sum\limits_{i=1}^{n}
					\cC(g_i^k - h_i^k) - (g_i^k - h_i^k)
			\Bigr\|^2
		}\\
		&\overset{\eqref{eq:indep}}{=}&
		\frac{1}{n^2}
		\sum\limits_{i=1}^{n}
			\EE{\cC}{
				\norm*{\cC(g_i^k - h_i^k) - (g_i^k - h_i^k)}^2
			} \\ 
		&\overset{\eqref{def:omega}}{\leq}&
		\frac{\omega}{n^2}
		\sum\limits_{i=1}^{n}
			\norm*{g_i^k - h_i^k}^2.
	\end{eqnarray*}
	Let us calculate full expectations conditioned on previous iteration:
	\begin{eqnarray*}
		\E{T_2}
		&=&
		\frac{\omega}{n^2}
		\sum\limits_{i=1}^{n}
			\E{\norm*{g_i^k - h_i^k}^2} \\
		&=&
		\frac{\omega}{n^2}
		\sum\limits_{i=1}^{n}
		    \left(
			\norm*{\E{g_i^k - h_i^k}}^2
			+
			\E{\norm*{g_i^k - h_i^k - \E{g_i^k - h_i^k}}^2}
			\right)\\
		&=&
		\frac{\omega}{n^2}
		\sum\limits_{i=1}^{n}
		    \left(
			\norm*{\nabla f_i(x^k) - h_i^k}^2
			+
			\E{\norm*{g_i^k - \nabla f_i(x^k)}^2}
			\right)\\
		&=&
		\frac{\omega}{n^2}
		\sum\limits_{i=1}^{n}
			\norm*{\nabla f_i(x^k) - h_i^k}^2 \\
			&& \quad + 
	    \frac{\omega}{n^2}
	    \sum\limits_{i=1}^{n}
			\E{\norm*{\nabla f_{ij_i^k}(x^k) - \nabla f_{ij_i^k}(w_{ij_i^k}^k) - \E{\nabla f_{ij_i^k}(x^k) - \nabla f_{ij_i^k}(w_{ij_i^k}^k)}}^2}
		\\
		&\overset{\eqref{eq:var_decomposition}}{\leq}&
		\frac{\omega}{n^2}
		\sum\limits_{i=1}^{n}
		     \left(
			 \norm*{\nabla f_i(x^k) - h_i^k}^2
			 +
			 \E{\norm*{\nabla f_{ij_i^k}(x^k) - \nabla f_{ij_i^k}(w_{ij_i^k}^k)}^2}
			 \right)\\
		&\overset{\eqref{eq:sum_upper}}{\leq}&
		\frac{2\omega}{n^2}
		\sum\limits_{i=1}^{n}
		    \left(
			\norm*{h_i^k - \nabla f_i(x^\star)}^2
			+
			\norm*{\nabla f_i(x^k) - \nabla f_i(x^\star)}^2
			\right)\\
		&&\quad+
		\frac{2\omega}{mn^2}
		\sum\limits_{i=1}^{n}
			\sum\limits_{j=1}^{m}
			    \left(
				\norm*{\nabla f_{ij}(w_{ij}^k) - \nabla f_{ij}(x^\star)}^2
				+
				\norm*{\nabla f_{ij}(x^k) - \nabla f_{ij}(x^\star)}^2
				\right)\\
		&\overset{\eqref{eq:bregman_grad_dif}}{\leq}&
		\frac{2\omega}{n^2} H^k
		+
		\frac{2\omega}{mn^2} D^k
		+
		\frac{8L\omega}{n} B_f(x^k, x^\star).
	\end{eqnarray*}
	The other term follows in a similar way:
	\begin{eqnarray*}
		\E{T_1}
		&=&
		\E{\norm*{
			\frac{1}{n}
			\sum\limits_{i=1}^{n}
				g_i^k - \nabla f(x^\star)
		}^2} \\
		&=&
		\norm*{
			\frac{1}{n}
			\sum\limits_{i=1}^{n}
			\E{g_i^k} - \nabla f(x^\star)
		}^2
		+
		\E{\norm*{
				\frac{1}{n}
				\sum\limits_{i=1}^{n}
				\left( g_i^k - \E{g_i^k}\right)
			}^2}\\
		&=&
		\norm*{\nabla f(x^k) - \nabla f(x^\star)}^2 
		\frac{1}{n^2}
		\sum\limits_{i=1}^{n}
			\E{\norm*{g_i^k - \nabla f_i(x^k)}^2}\\
		&\overset{\eqref{eq:bregman_grad_dif}}{\leq}&
		2LB_f(x^k, x^\star)\\
		&&\quad +
		\frac{1}{n^2}
		\sum\limits_{i=1}^{n}
			\E{\norm*{\nabla f_{ij_i^k}(x^k) - \nabla f_{ij_i^k}(w_{ij_i^k}^k) - \E{\nabla f_{ij_i^k}(x^k) - \nabla f_{ij_i^k}(w_{ij_i^k}^k)}}^2}\\
		&\overset{\eqref{eq:var_decomposition}}{\leq}&
		2LB_f(x^k, x^\star)
		+ 
		\frac{1}{n^2}
		\sum\limits_{i=1}^{n}
			\E{\norm*{\nabla f_{ij_i^k}(x^k) - \nabla f_{ij_i^k}(w_{ij_i^k}^k)}^2}\\
		&\overset{\text{Alg.}~\ref{alg:VR-DIANA}}{=}&
		2LB_f(x^k, x^\star)
		+
		\frac{1}{mn^2}
		\sum\limits_{i=1}^{n}
			\sum\limits_{j=1}^{m}
				\norm*{\nabla f_{ij}(x^k) - \nabla f_{ij}(w_{ij}^k)}^2\\
		&\overset{\eqref{eq:sum_upper}}{\leq}&
		2LB_f(x^k, x^\star) \\
		&& \quad + 
		\frac{2}{mn^2}
		\sum\limits_{i=1}^{n}
			\sum\limits_{j=1}^{m}
			    \left(
				\norm*{\nabla f_{ij}(w_{ij}^k) - \nabla f_{ij}(x^\star)}^2
				+
				\norm*{\nabla f_{ij}(x^k) - \nabla f_{ij}(x^\star)}^2
				\right)\\
		&\overset{\eqref{eq:bregman_grad_dif}}{\leq}&
		\left(2L + \frac{4L}{n}\right)B_f(x^k, x^\star)
		+
		\frac{2}{mn^2} D^k.
	\end{eqnarray*}
	Now, summing $\E{T_1}$ and $\E{T_2}$ we get
	\begin{align*}
		\E{\norm*{g^k - \nabla f(x^\star)}^2}
		&=
		\E{T_1 + T_2}
		\leq
		\left(2L + \frac{4L}{n}\right)
		B_f(x^k, x^\star)
		+
		\frac{2}{mn^2} D^k\\
		&\quad +
		\frac{2\omega}{n^2} H^k
		+
		\frac{2\omega}{mn^2} D^k
		+
		\frac{8L\omega}{n} B_f(x^k, x^\star)\\
		&\leq
		\left(
			2L + \frac{4L}{n} + \frac{8L\omega}{n}
		\right) B_f(x^k, x^\star)
		+
		\frac{2\omega}{n^2} H^k
		+
		\frac{2(\omega+1)}{mn^2} D^k,
	\end{align*}
	which concludes the proof.
\end{proof}

\begin{proof}[Proof of Theorem~\ref{thm:VR-DIANA}]
Combining all lemmas together we may finalize proof. By the definition of Lyapunov function we have
	\begin{eqnarray}
		\E{\psi^{k+1}}
		&=&
		\E{\norm*{x^{k+1} - x^\star}^2 + b \gamma^2 H^{k+1} + c \gamma^2 D^{k+1}} \notag \\
		&\overset{\eqref{lem:up_x_k_VR}}{\leq}&
		(1 - \mu \gamma)\norm*{x^k - x^\star}^2 - 2\gamma B_f(x^k, x^\star) + \gamma^2 \E{\norm*{g^k - \nabla f(x^\star)}^2} \notag \\
		&&\quad +
		\E{
			b \gamma^2 H^{k+1} + c \gamma^2 D^{k+1}
		}\notag \\
		&\overset{\eqref{lem:up_H_VR}-\eqref{lem:up_g_VR}}{\leq}&
		(1 - \mu \gamma)\norm*{x^k - x^\star}^2 - 2\gamma B_f(x^k, x^\star)\notag \\
		&&\quad +
		\gamma^2
		\left(
			2L\left(1 + \frac{4\omega+2}{n}\right)B_f(x^k, x^\star)
			+
			\frac{2(\omega+1)}{mn^2}D^k
			+
			\frac{2\omega}{n^2} H^k
		\right)\notag \\
		&&\quad +
		b \gamma^2
		\left(
			(1 - \alpha)H^k + \frac{2\alpha}{m} D^k + 8\alpha Ln B_f(x^k, x^\star)
		\right) \notag \\
		&&\quad
		+
		c \gamma^2
		\left(
			\left(1 - \frac{1}{m}\right)D^k 
			+
			2LnB_f(x^k, x^\star)
		\right)\notag \\
		&=&
		(1 - \mu \gamma)\norm*{x^k - x^\star}^2
		+
		\left(
			1 - \alpha + \frac{2 \omega}{b n^2}
		\right)
		b \gamma^2 H^k \notag \\
		&&\quad
		+
		\left(
			1 - \frac{1}{m}
			+
			\frac{2b \alpha}{c m}
			+
			\frac{2(\omega+1)}{mn^2 c}
		\right)
		c \gamma^2 D^k\notag \\
		&& \quad -
		\left(
		 	2\gamma - 
		 	2L\gamma^2
		 	\left[
		 		1 + \frac{4\omega+2}{n} + c n + 4b\alpha n	 		
		 	\right]
		\right)
		B_f(x^k, x^\star). \label{conv_VR}
	\end{eqnarray}
	Now, choosing
	$b = \frac{4 (\omega+1)}{\alpha n^2}$
	and
	$c = \frac{16 (\omega+1)}{ n^2}$
	, we get
	\begin{align*}
	\E{\psi^{k+1}}
	&\leq
	(1 - \mu \gamma)\norm*{x^k - x^\star}^2
	+
	b \gamma^2 
	\left(
		1 - \frac{\alpha}{2}
	\right)H^k
	+
	c \gamma^2
	\left(
		1 - \frac{3}{8m}
	\right) D^k\\
	&\quad -
	\left(
		2\gamma - 
		2L\gamma^2
		\left[
			1 + \frac{36(\omega+1)}{n}	 		
		\right]
	\right)
	B_f(x^k, x^\star).
	\end{align*}
	Setting $\gamma = \frac{1}{L\left(1 + 36(\omega+1)/n\right)}$ gives
	\begin{align*}
		E{\psi^{k+1}}
		&\leq
		\left(
			1  - \frac{\mu}{L\left(1 + 36(\omega+1)/n\right)}
		\right)
		\norm*{x^k - x^\star}^2
		+
		b \gamma^2 
		\left(
			1 - \frac{\alpha}{2}
		\right)H^k \\
		&\quad
		+
		c \gamma^2 
		\left(
			1 - \frac{3}{8m}
		\right)D^k,
	\end{align*}
	which concludes the proof.
\end{proof}

\subsection{Convex case}

\begin{proof}[Proof of Theorem~\ref{thm:VR-DIANAweak}]
Using~\eqref{conv_VR} with $\mu = 0$ assuming that $\alpha \geq \frac{2w}{b n^2}$ and
 $1 \geq \frac{2b \alpha}{c} + \frac{2w}{c n^2}$, we obtain
\begin{align*}
&\E{\norm*{x^{k+1} - x^\star}^2 + b \gamma^2 H^{k+1} + c \gamma^2 D^{k+1}}  \notag\\
		&\quad \leq	
		\norm*{x^{k} - x^\star}^2 + b \gamma^2 H^{k} + c \gamma^2 D^{k}
		-
	\left(
		2\gamma - 
		2L\gamma^2
		\left[
			1 + \frac{28\omega}{n}	 		
		\right]
	\right)B_f(x^k, x^\star),
\end{align*}
which implies
\begin{align*}
-
	\left(
		2\gamma - 
		2L\gamma^2
		\left[
			1 + \frac{128\omega}{n}	 		
		\right]
	\right)B_f(x^k, x^\star) \leq\psi^k - \E{\psi^{k+1}},
\end{align*}
which after removing conditional expectation can be summed over all iterations $1,2, \ldots, k$ and one has
\begin{align*}
\E{B_f(x^a, x^\star)} \leq \frac{\psi_0}{2k\left(
		\gamma - 
		L\gamma^2
		\left[
			1 + \frac{36(\omega+1)}{n}	 		
		\right]
	\right)},
\end{align*}
where index $a$ is uniformly at random picked from ${1,2,\cdots, k}$ ,which concludes the proof.
\end{proof}

\subsection{Non-convex case}

\begin{theorem}\label{thm:VR-DIANA-non-convex}
    Consider Algorithm~\ref{alg:VR-DIANA} with $\omega$-quantization $\cC$,
	and stepsize $\alpha \leq \frac{1}{\omega + 1}$.
Choose any $p>0$ (which will appear in sequence $c^k$ below) to consider the following Lyapunov function
\begin{equation*}
	R^k =f(x^k) + c^k W^k + d^k F^k,
\end{equation*}
where
\begin{align*}
	W^k &= \frac{1}{mn}\sum\limits_{i=1}^{n}\sum\limits_{j=1}^{m} \norm*{x^k - w_{ij}^k}^2,\,\quad
	F^k =
	 \frac{1}{n}\sum\limits_{i=1}^{n} 
		\norm*{\nabla f_{i}(x^k) - h_i^k}^2\,,
\end{align*}
and
\begin{align*}
	c^k &=
	 c^{k+1}\left( 1 - \frac{1}{m} + \gamma p + \frac{\omega+1}{n} L^2 \gamma^2 \right) + d^{k+1} \left(\alpha L^2 +  \left(1+\frac{2}{\alpha}\right)\frac{\omega+1}{n} L^4 \gamma^2\right) \\
		&\quad +  \frac{\omega+1}{n} \frac{\gamma^2L^3}{2} \,,\\
	d^k &=
	   d^{k+1}\left( 1 - \frac{\alpha}{2} +  \left(1+\frac{2}{\alpha}\right)\frac{\omega}{n}L^2 \gamma^2\right) + c^{k+1}\frac{\omega}{n} \gamma^2 + \frac{\omega}{n} \frac{\gamma^2L}{2}\,.
\end{align*}
under Assumption~\ref{assumption:VRdiana}
	\begin{equation*}
		\E{R^{k+1}} \leq
		R^k - \Gamma^k \norm*{\nabla f(x^k)}^2\, ,
	\end{equation*}
	where 
	\begin{equation*}
		 \Gamma^k = 
		\gamma - \frac{\gamma^2L}{2}  - c^{k+1} \left(\gamma^2 + \frac{\gamma}{p}\right) - d^{k+1} \left(1+\frac{2}{\alpha}\right) L^2 \gamma^2\, .
	\end{equation*}
Taking $x^a \sim_{u.a.r.} \{x^0, x^1,\dots, x^{k-1}\}$ of Algorithm~\ref{alg:VR-DIANA} one obtains
	\begin{equation}
	\label{up:delta}
		\E{\norm*{\nabla f(x^a)}^2} \leq
		\frac{R^0 - R^k}{k\Delta},
	\end{equation}
where $\Delta = \min_{t \in [k]} \Gamma^t > 0$ .
\end{theorem}

\begin{lemma}  We can upper bound $W^{k+1}$ in the following way 
	\begin{equation}
		\E{W^{k+1}}
		\leq
		\E{\norm*{x^{k+1} - x^k}^2} + \left(1 - \frac{1}{m} + \gamma p\right)W^k +\frac{\gamma}{p}\norm*{\nabla f(x^k)}^2.
		\label{lem:up_W_VR_nc}
	\end{equation}
	
\end{lemma}
\begin{proof}
	\begin{align*}
		\E{W^{k+1}}
		&=
		\E{\frac{1}{mn} \sum\limits_{i=1}^{n}\sum\limits_{j=1}^{m} \norm*{x^{k+1} - w_{ij}^{k+1}}^2}\\
		&=
		\frac{1}{mn} \sum\limits_{i=1}^{n}\sum\limits_{j=1}^{m}\left( \frac{1}{m}\E{\norm*{x^{k+1} - x^k}^2} + \frac{m-1}{m}\E{\norm*{x^{k+1} - w_{ij}^{k}}^2} \right)\\
		&=
		\frac{1}{m}\E{\norm*{x^{k+1} - x^k}^2} + \frac{1}{mn} \sum\limits_{i=1}^{n}\sum\limits_{j=1}^{m} \frac{m-1}{m}\left(\E{\norm*{x^{k+1} -x^k + x^k - w_{ij}^{k}}^2} \right)
		\\
		&=
		\frac{1}{m}\E{\norm*{x^{k+1} - x^k}^2} + \frac{1}{mn} \sum\limits_{i=1}^{n}\sum\limits_{j=1}^{m} \frac{m-1}{m}\Big(\E{\norm*{x^{k+1} -x^k}^2} + \norm*{x^k - w_{ij}^{k}}^2  \\
		& \quad + 2\dotprod{\E{x^{k+1} -x^k}}{x^k - w_{ij}^{k}}\Big)
		\\
		&\leq
		\E{\norm*{x^{k+1} - x^k}^2} + \frac{1}{mn} \sum\limits_{i=1}^{n}\sum\limits_{j=1}^{m} \frac{m-1}{m} \norm*{x^k - w_{ij}^{k}}^2  \\
		&\quad + 2\gamma\left(\frac{1}{2p}\norm*{\nabla f(x^k)}^2 + \frac{p}{2}\norm*{x^k - w_{ij}^{k}}^2\right)
		\\
		&=
		\E{\norm*{x^{k+1} - x^k}^2} + \left(1 - \frac{1}{m} + \gamma p\right)W^k +\frac{\gamma}{p}\norm*{\nabla f(x^k)}^2,
	\end{align*}
	where the second equality uses the update of $w_{ij}^{k+1}$ in Algorithm~\ref{alg:VR-DIANA} and the inequality uses Cauchy-Schwarz and Young inequalities.
\end{proof}

\begin{lemma} We can upper bound quantity 
$\frac{1}{n}
\sum\limits_{i=1}^{n}
\E{\norm*{g_i^k - h_i^k}^2}$
in the following way
	\begin{equation}
		\frac{1}{n}
		\sum\limits_{i=1}^{n}
		\E{\norm*{g_i^k - h_i^k}^2}
		\leq 
		F^k + L^2 W^k.
		\label{lem:up_g-h_VR_nc}
	\end{equation}
\end{lemma}

\begin{proof}
\begin{eqnarray*}
&&\frac{1}{n}
		\sum\limits_{i=1}^{n}
		\E{\norm*{g_i^k - h_i^k}^2} \\
		&=&
		\frac{1}{n}
		\sum\limits_{i=1}^{n}
		    \left(
			\norm*{\E{g_i^k - h_i^k}}^2
			+
			\E{\norm*{g_i^k - h_i^k - \E{g_i^k - h_i^k}}^2}
			\right)\\
		&=&
		\frac{1}{n}
		\sum\limits_{i=1}^{n}
		    \left(
			\norm*{\nabla f_i(x^k) - h_i^k}^2
			+
			\E{\norm*{g_i^k - \nabla f_i(x^k)}^2}
			\right)\\
		&=&
		F^k +
		\frac{1}{n}
		\sum\limits_{i=1}^{n}
			\E{\norm*{\nabla f_{ij_i^k}(x^k) - \nabla f_{ij_i^k}(w_{ij_i^k}^k) - \E{\nabla f_{ij_i^k}(x^k) - \nabla f_{ij_i^k}(w_{ij_i^k}^k)}}^2}
		\\
		&\overset{\eqref{eq:var_decomposition}}{\leq}&
		F^k +
		\frac{1}{n}
		\sum\limits_{i=1}^{n}
			 \E{\norm*{\nabla f_{ij_i^k}(x^k) - \nabla f_{ij_i^k}(w_{ij_i^k}^k)}^2}
			 \\
		&=&
		F^k +
		\frac{1}{n}
		\sum\limits_{i=1}^{n}
		\frac{1}{m}
		\sum\limits_{j=1}^{m}
			 \E{\norm*{\nabla f_{ij}(x^k) - \nabla f_{ij}(w_{ij}^k)}^2}
			 \\
		&\overset{\eqref{eqn:quad-upper}}{\leq}&
		F^k + L^2 W^k.
	\end{eqnarray*}
\end{proof}
Equipped with this lemma, we are ready to prove a recurrence inequality for $F^k$:
\begin{lemma} Let $\alpha (\omega + 1)\leq 1$. We can upper bound $F^{k+1}$ in the following way
	\begin{equation}
		\E{F^{k+1}}
		\leq
		\left(1+\frac{2}{\alpha}\right)L^2\E{\norm*{x^{k+1} - x^k}^2} + 
	\left(1-\frac{\alpha}{2}\right) F^k + \alpha L^2 W^k.
		\label{lem:up_F_VR_nc}
	\end{equation}
\end{lemma}
\begin{proof}
	\begin{eqnarray*}
		&&\E{F^{k+1}} \\
		&=&
		\frac{1}{n}
		\sum\limits_{i=1}^{n}
				\E{\norm*{\nabla f_{i}(x^{k+1}) - h_i^{k+1}}^2}\\
		&=&
		\frac{1}{n}
		\sum\limits_{i=1}^{n}
				\E{\norm*{\nabla f_{i}(x^{k+1}) - \nabla f_{i}(x^k) + \nabla f_{i}(x^k) - h_i^k - \alpha \cC\left( g_i^k - h_i^k\right)}^2}\\
		& = & 
		\frac{1}{n}
		\sum\limits_{i=1}^{n}
		\left(
				 \E{\norm*{\nabla f_{i}(x^{k+1}) - \nabla f_{i}(x^k)}^2} + \E{\norm*{\nabla f_{i}(x^k) - h_i^k - \alpha \cC\left( g_i^k - h_i^k\right)}^2}
		\right)\\
		&& \quad +  (1-\alpha)\frac{1}{n}
		\sum\limits_{i=1}^{n}
		\dotprod{\nabla f_{i}(x^{k+1}) - \nabla f_{i}(x^k)}{ \nabla f_{i}(x^k) - h_i^k } \\
		&\overset{\eqref{eqn:quad-upper}}{\leq}&
		\frac{1}{n}
		\sum\limits_{i=1}^{n}
		\left(
				 \left(1 + \frac{1-\alpha}{\tau}\right)L^2\E{\norm*{x^{k+1} - x^k}^2} + 
		 (1+(1-\alpha)\tau)\norm*{\nabla f_{i}(x^k) - h_i^k}^2\right) \\
		 && \quad +  \alpha^2\sum\limits_{i=1}^{n}\E{\norm*{ \cC\left( g_i^k - h_i^k\right)}^2} - 2 \frac{\alpha}{n} \sum\limits_{i=1}^{n}\dotprod{\nabla f_{i}(x^k) - h_i^k}{\E{\cC\left( g_i^k - h_i^k\right)}}
		\\
	&\overset{\eqref{def:omega}}{\leq}&
	 \left(1 + \frac{1-\alpha}{\tau}\right)L^2\E{\norm*{x^{k+1} - x^k}^2} + 
	\frac{1}{n}
		\sum\limits_{i=1}^{n}
		\left( (1+ (1-\alpha)\tau)
		\norm*{\nabla f_{i}(x^k) - h_i^k}^2\right)  \\
		&& \quad  + \alpha^2\sum\limits_{i=1}^{n}(\omega+1)\E{\norm*{g_i^k - h_i^k}^2}  - 2 \frac{\alpha}{n}\sum\limits_{i=1}^{n}\dotprod{\nabla f_{i}(x^k) - h_i^k}{\nabla f_{i}(x^k) - h_i^k}\\
	&\leq &
	 \left(1 + \frac{1-\alpha}{\tau}\right)L^2\E{\norm*{x^{k+1} - x^k}^2} + 
	(1+(1-\alpha)\tau-2\alpha) F^k  \\
		&&\quad 
	 + \alpha\frac{1}{n}
		\sum\limits_{i=1}^{n}\E{\norm*{g_i^k - h_i^k}^2} 
		\\	
	&\overset{\eqref{lem:up_g-h_VR_nc}}{\leq}&
	 \left(1 + \frac{1-\alpha}{\tau}\right)L^2\E{\norm*{x^{k+1} - x^k}^2} + 
	(1+\tau -\alpha) F^k + \alpha L^2 W^k,
	\end{eqnarray*}
	where the second equality uses definition of $h_i^{k+1}$ in Algorithm~\ref{alg:VR-DIANA} and the first inequality follows from Cauchy inequality and holds for any $\tau > 0$. 
	
	Taking $\tau = \alpha/2$, we obtain desired inequality.
\end{proof}

\begin{lemma}
We can upper bound the second moment of the $g^k$ in the following way
	\begin{equation}
		\E{\norm*{g^k}^2}
		\leq
		\frac{\omega}{n}F^k + \frac{\omega+1}{n} L^2 W^k + \norm*{\nabla f(x^k)}^2
		.
		\label{lem:up_g_VR_nc}
	\end{equation}
\end{lemma}
\begin{proof}
	\begin{align*}
		\EE{\cC}{\norm*{g^k}^2} &\overset{\eqref{eq:var_decomposition}}{=}
		\underbrace{
			\norm*{\EE{\cC}{g^k}}^2
		}_{T_1}
		+
		\underbrace{
			\EE{\cC}{\norm*{g^k - \EE{\cC}{g^k}}^2}
		}_{T_2}.
	\end{align*}
	We can use the definition of $g^k$ in order to obtain 
	\begin{align*}
		T_1
		&= 
		\norm*{
			\frac{1}{n}
			\sum\limits_{i=1}^{n}
				\EE{\cC}{\cC(g_i^k - h_i^k) + h_i^k}
		}^2
		=
		\norm*{
			\frac{1}{n}
			\sum\limits_{i=1}^{n}
				g_i^k
		}^2
	\end{align*}
	and 
	\begin{eqnarray*}
		T_2
		&=&
		\EE{\cC}{
			\norm*{
				\frac{1}{n}
				\sum\limits_{i=1}^{n}
					\cC(g_i^k - h_i^k) - (g_i^k - h_i^k)
			}^2
		}\\
		&\overset{\eqref{eq:indep}}{=}&
		\frac{1}{n^2}
		\sum\limits_{i=1}^{n}
			\EE{\cC}{
				\norm*{\cC(g_i^k - h_i^k) - (g_i^k - h_i^k)}^2
			} \\ 
		&\overset{\eqref{def:omega}}{\leq}&
		\frac{\omega}{n^2}
		\sum\limits_{i=1}^{n}
			\norm*{g_i^k - h_i^k}^2.
	\end{eqnarray*}
	Now we calculate full expectations conditioned on previous iteration:
	\begin{eqnarray*}
		\E{T_2}
		&=&
		\frac{\omega}{n^2}
		\sum\limits_{i=1}^{n}
			\E{\norm*{g_i^k - h_i^k}^2} \\
		&\overset{\eqref{lem:up_g-h_VR_nc}}{\leq}&
		\frac{\omega}{n}F^k + \frac{\omega}{n}L^2 W^k.
	\end{eqnarray*}
	As for $T_1$, we have
	\begin{eqnarray*}
		\E{T_1} 
		&=&
		\E{\norm*{
			\frac{1}{n}
			\sum\limits_{i=1}^{n}
				g_i^k
		}^2}
		=
		\norm*{
			\frac{1}{n}
			\sum\limits_{i=1}^{n}
			\E{g_i^k}
		}^2
		+
		\E{\norm*{
				\frac{1}{n}
				\sum\limits_{i=1}^{n}
				g_i^k - \E{g_i^k}
			}^2}\\
		&=&
		\norm*{\nabla f(x^k)}^2
		+
		\frac{1}{n^2}
		\sum\limits_{i=1}^{n}
			\E{\norm*{g_i^k - \nabla f_i(x^k)}^2}\\
		&=&
		\norm*{\nabla f(x^k)}^2  \\
		&&\quad 
		+
		\frac{1}{n^2}
		\sum\limits_{i=1}^{n}
			\E{\norm*{\nabla f_{ij_i^k}(x^k) - \nabla f_{ij_i^k}(w_{ij_i^k}^k) - \E{\nabla f_{ij_i^k}(x^k) - \nabla f_{ij_i^k}(w_{ij_i^k}^k)}}^2}\\
		&\overset{\eqref{eq:var_decomposition}}{\leq}&
		\norm*{\nabla f(x^k)}^2
		+
		\frac{1}{n^2}
		\sum\limits_{i=1}^{n}
			\E{\norm*{\nabla f_{ij_i^k}(x^k) - \nabla f_{ij_i^k}(w_{ij_i^k}^k)}^2}\\
		&\overset{\text{Alg.}~\ref{alg:VR-DIANA}}{=}&
		\norm*{\nabla f(x^k)}^2
		+
		\frac{1}{mn^2}
		\sum\limits_{i=1}^{n}
			\sum\limits_{j=1}^{m}
				\norm*{\nabla f_{ij}(x^k) - \nabla f_{ij}(w_{ij}^k)}^2\\
		&\overset{\eqref{eqn:quad-upper}}{\leq}&
		\norm*{\nabla f(x^k)}^2
		+
		\frac{1}{n}L^2 W^k.
	\end{eqnarray*}
	Now, summing $\E{T_1}$ and $\E{T_2}$ we get
	\begin{align*}
		\E{\norm*{g^k}^2}
		&=
		\E{T_1 + T_2}
		\leq
		\frac{\omega}{n}F^k + \frac{\omega+1}{n} L^2 W^k + \norm*{\nabla f(x^k)}^2,
	\end{align*}
	which concludes the proof.
\end{proof}

\begin{proof}[Proof of Theorem~\ref{thm:VR-DIANA-non-convex}]
Using $L$-smoothness one gets
	\begin{align}
		\E{f(x^{k+1})} &\leq f(x^k) + \dotprod{\nabla f(x^k)}{\E{x^{k+1}-x^k}} + \frac{L}{2}\E{\norm*{x^{k+1} - x^k}^2} \notag \\
		& = f(x^k) - \gamma\norm*{\nabla f(x^k)}^2 + \frac{L\gamma^2}{2}\E{\norm*{g^k}^2},
		\label{lem:nc_smooth}
	\end{align}
where we use the definition of $x^{k+1}$ in Algorithm~\ref{alg:VR-DIANA}.

By combining definition of $\E{R^{k+1}}$ with \eqref{lem:up_W_VR_nc}, \eqref{lem:up_F_VR_nc} and \eqref{lem:nc_smooth} one obtains
\begin{eqnarray*}
&&\E{R^{k+1}} \\
&\leq&  f(x^k) + \dotprod{\nabla f(x^k)}{\E{x^{k+1}-x^k}} + \frac{L}{2}\E{\norm*{x^{k+1} - x^k}^2} \\
&&\quad + c^{k+1}\left( \E{\norm*{x^{k+1} - x^k}^2} + \left(1 - \frac{1}{m} + \gamma p\right)W^k +\frac{\gamma}{p}\norm*{\nabla f(x^k)}^2\right) \\
&&\quad + d^{k+1}\left(  \left(1+\frac{2}{\alpha}\right)L^2\E{\norm*{x^{k+1} - x^k}^2} + \left(1-\frac{\alpha}{2}\right) F^k +\alpha L^2 W^k\right) \\
&=&
f(x^k) - \gamma\norm*{\nabla f(x^k)}^2 + \left(\frac{\gamma^2L}{2} + c^{k+1} \gamma^2 + d^{k+1}  \left(1+\frac{2}{\alpha}\right)L^2 \gamma^2\right)\E{\norm*{g^k}^2} \\
&&\quad + c^{k+1}\left( \left(1 - \frac{1}{m} + \gamma p\right)W^k +\frac{\gamma}{p}\norm*{\nabla f(x^k)}^2\right) \\
&&\quad + d^{k+1}\left(\left(1-\frac{\alpha}{2}\right)F^k + \left(1+\frac{2}{\alpha}\right)\alpha L^2 W^k\right) \\
&\overset{\eqref{lem:up_g_VR_nc}}{\leq}&
f(x^k) - \left(\gamma - \frac{\gamma^2L}{2}  - c^{k+1} \gamma^2 - d^{k+1} \left(1+\frac{2}{\alpha}\right) L^2 \gamma^2 - c^{k+1}\frac{\gamma}{p}\right)\norm*{\nabla f(x^k)}^2\\ 
&&\quad + \left( c^{k+1}\left( 1 - \frac{1}{m} + \gamma p\right) + d^{k+1}\left(1+\frac{2}{\alpha}\right) \alpha L^2  \right) W^k \\ 
&& \quad + \left( \frac{\omega+1}{n}  L^2\left(\frac{\gamma^2L}{2}  + c^{k+1} \gamma^2 + d^{k+1}  \left(1+\frac{2}{\alpha}\right)L^2 \gamma^2\right) \right) W^k\\
&&\quad + \left( d^{k+1}\left(1-\frac{\alpha}{2}\right) + \frac{\omega}{n} \left(\frac{\gamma^2L}{2}  + c^{k+1} \gamma^2 + d^{k+1}  \left(1+\frac{2}{\alpha}\right)L^2 \gamma^2\right) \right) F^k\\
&=&
f(x^k) - \left(\gamma - \frac{\gamma^2L}{2}  - c^{k+1} \left(\gamma^2 + \frac{\gamma}{p}\right) - d^{k+1} \left(1+\frac{2}{\alpha}\right) L^2 \gamma^2\right)\norm*{\nabla f(x^k)}^2\\ 
&&\quad + c^{k+1}\left( 1 - \frac{1}{m} + \gamma p + \frac{\omega+1}{n} L^2 \gamma^2 \right)  W^k  \\
&&\quad  + \left(d^{k+1} \left(\alpha L^2 +  \left(1+\frac{2}{\alpha}\right)\frac{\omega+1}{n} L^4 \gamma^2\right)  +  \frac{\omega+1}{n} \frac{\gamma^2L^3}{2} \right)W^k\\
&&\quad + \left( d^{k+1}\left( 1 - \frac{\alpha}{2} +  \left(1+\frac{2}{\alpha}\right)\frac{\omega}{n}L^2 \gamma^2\right) + c^{k+1}\frac{\omega}{n} \gamma^2 + \frac{\omega}{n} \frac{\gamma^2L}{2} \right) F^k \\
&=& R^k - \Gamma^k \norm*{\nabla f(x^k)}^2.
\end{eqnarray*}
Applying the full expectation and telescoping the equation, one gets desired inequality.
\end{proof}
We can proceed to the proof of Theorem~\ref{thm:VR-DIANAnc}.
\begin{proof}[Proof of Theorem~\ref{thm:VR-DIANAnc}]
Recursion for $c^t, d^t$ can be written in a form
\begin{align*}
y^t = Ay^{t+1} + b,
\end{align*}
where 
\begin{align*}
A &= \begin{bmatrix}
 1 - \frac{1}{m} + \gamma p + \frac{\omega+1}{n} L^2 \gamma^2 &\alpha L^2 +  \left(1+\frac{2}{\alpha}\right)\frac{\omega+1}{n} L^4 \gamma^2\\
\frac{\omega}{n} \gamma^2  & 1 - \frac{\alpha}{2} +  \left(1+\frac{2}{\alpha}\right)\frac{\omega}{n}L^2 \gamma^2
\end{bmatrix} , \\
y^t &=  \begin{bmatrix}
 c^t \\
 d^t
\end{bmatrix} , \\
b &=  \begin{bmatrix}
 \frac{\omega+1}{n} \frac{\gamma^2L^3}{2} \\
 \frac{\omega}{n} \frac{\gamma^2L}{2}
\end{bmatrix} .
\end{align*}
Recall that we once used Young's inequality with an arbitrary parameter $p$, so we can now specify it. Choosing $p =  \frac{L\left(1+ \frac{\omega}{n} \right)^{1/2}}{ (m^{2/3} + \omega + 1)^{1/2}}$, $\gamma = \frac{1}{10L\left( 1 + \frac{\omega}{n} \right)^{1/2} (m^{2/3} + \omega + 1)}$, and $\alpha = \frac{1}{\omega+1}$, where  $c^T = d^T = 0$ we can upper bound each element of matrix $A$ and construct its upper bound $\hat{A}$ and a corresponding vector $\hat b$, where 
\begin{align*}
\hat{A} &= \begin{bmatrix}
 1 - \frac{89}{100m}  & L^2 \frac{103}{100(\omega+1)} \\
 \frac{1}{100L^2m} & 1 -  \frac{47}{100(\omega+1)}
 \end{bmatrix},\qquad
 \hat{b} = \left(1+\frac{\omega}{n}\right)\frac{\gamma^2 L}{2}\begin{bmatrix}
 L^2 \\
 1
 \end{bmatrix}.
\end{align*}
Due to the structure of $\hat{A}$ and $\hat{b}$ we can work with matrices 
\begin{align*}
\tilde{A} &= \begin{bmatrix}
 1 - \frac{89}{100m}  & \frac{103}{100(\omega+1)} \\
 \frac{1}{100m} & 1 -  \frac{47}{100(\omega+1)}
 \end{bmatrix},\qquad
 \tilde{b} = \left(1+\frac{\omega}{n}\right)\frac{\gamma^2 L}{2}\begin{bmatrix}
 1 \\
 1
 \end{bmatrix},
\end{align*}
where it holds $\tilde{A}^k \tilde{b} = (y_1, y_2)^\top \implies \hat{A}^k \hat{b} = (L^2y_1, y_2)^\top$, thus we can work with $\tilde{A}$ which is independent of $L$.
In the sense of Lemma~\ref{lem:sequence}, we have that eigenvalues of $\tilde{A}$ are less than $1-\frac{1}{3\max\{m, \omega + 1\}}, 1-\frac{1}{6\min\{m, \omega + 1\}}$, respectively, and $\abs{t} \leq \frac{2}{\min\{m, \omega + 1\}}$, thus
\begin{align*}
c^k &\leq 90\max\{m, \omega+1\}\left(1+\frac{\omega}{n}\right) \frac{\gamma^2L^3}{2} \\
&=  90\max\{m, \omega+1\}\left(1+\frac{\omega}{n}\right) \frac{L^3}{200\left(1 + \frac{\omega}{n}\right) L^2 (m^{2/3} + \omega + 1)^2} \\
&\leq \frac{L}{2(m^{2/3} + \omega + 1)^{1/2}}.
\end{align*}
By the same reasoning
\begin{align*}
d^k \leq \frac{1}{2L(m^{2/3} + \omega + 1)^{1/2}}\, .
\end{align*}
This implies 
\begin{align*}
\Gamma^k & = \gamma - \frac{\gamma^2L}{2}  - c^{k+1} \left(\gamma^2 + \frac{\gamma}{p}\right) - d^{k+1}  \left(1 + \frac{2}{\alpha} \right) L^2 \gamma^2 \\
& \geq \frac{1}{40L\left(1 +  \frac{\omega}{n} \right)^{1/2}(m^{2/3} + \omega + 1)},
\end{align*}
which guarantees $\Delta \geq \frac{1}{40L\left(1 +  \frac{\omega}{n} \right)^{1/2}(m^{2/3} + \omega + 1)}$. Plugging the lower bound on $\Delta$ into \eqref{up:delta} one completes the proof.
\end{proof}

\section{Variance reduced \texttt{DIANA}: \texttt{SVRG} detailed theorems and proofs}
\label{sec:SVRG_DIANA}

\begin{lemma}
For all iterates $k \geq 0$ of Algorithm~\ref{alg:SVRG_DIANA}, we have
	\begin{equation*}
		\E{g^k} = \nabla f(x^k).
	\end{equation*}
\end{lemma}
\begin{proof}
	\begin{align*}
		\E{g^k}
		&=
		\frac{1}{n}\sum\limits_{i=1}^n \E{\cC\left( g_i^k - h_{i}^k\right) + h_i^k}
		=
		\frac{1}{n}\sum\limits_{i=1}^n \E{ g_i^k- h_{i}^k + h_i^k}\\
		&=
		\frac{1}{n}\sum\limits_{i=1}^n \E{ g_i^k} = \frac{1}{n}\sum\limits_{i=1}^n \nabla f_i(x^k) = \nabla f(x^k),
	\end{align*}
	where the first inequality follows from definition of $g^k$ in Algorithm~\ref{alg:SVRG_DIANA}.
\end{proof}

\subsection{Strongly convex case}

To prove the convergence of Algorithm~\ref{alg:SVRG_DIANA} we consider Lyapunov function of the following form:
\begin{equation}
\label{def:SVRG_psi^k}
	\psi^s = \left( f(z^s) - f^\star\right) + \bar{b} \gamma^2 \bar{H}^s,
\end{equation}
where
\begin{equation}
\label{def:SVRG_H^k}
	H^k \eqdef \sum\limits_{i=1}^{n} \norm*{h_{i}^k - \nabla f_{i}(x^\star)}^2
\end{equation}
and $\bar{H}^s = H^{ls}$.

The following theorem establishes linear convergence rate of Algorithm~\ref{alg:SVRG_DIANA}.

\begin{theorem}\label{thm:SVRG-DIANA}
	Let Assumptions~\ref{assumption:VRdiana}, \ref{assumption:VRdianasc} hold and $R\equiv 0$. Then for $\alpha \leq \frac{1}{\omega + 1}$, the following inequality holds:
	\begin{equation}
	\E{\psi^{s+1}} \leq \psi^s \max
	\left\{
	\frac{(1-\theta)^l}{1-(1-\theta)^l}\frac{2\theta + (1-(1-\theta)^l)c\mu}{\mu (2\gamma-c)}
	,
	\left(1
	-
	\theta\right)^l
	\right\},
	\end{equation}
	where $c = \frac{6L\omega}{n} \gamma^2
		 + 
		\left(2L+\frac{4L}{n}\right)\gamma^2
		 +
		 4b \gamma^2 L\alpha n$, $\theta = \min\{\mu\gamma, \alpha
		-
		\frac{3\omega}{n^2b}\}$, $p_r = \frac{(1-\theta)^{l-1-r}}{\sum_{t=0}^{l-1}(1-\theta)^{l-1-t}}$ for $r = 0,1,\dots, l-1$, and $b = \bar{b}l(2\gamma-c)$.
\end{theorem}

\begin{corollary}
	Taking $\alpha = \frac{1}{\omega + 1}$, $b = 6\frac{\omega}{n^2 \alpha}$, $\gamma = \frac{1}{10L\left( 2+ \frac{4}{n} + 30\frac{\omega}{n} \right)}$, and $l = \frac{2}{\theta}$
 \texttt{SVRG}-\texttt{DIANA} needs
	$
		O\left(
			\left( \kappa + \kappa\frac{\omega}{n} + \omega + m\right) \log \frac{1}{\epsilon}
		\right)
	$
	iterations
	to achieve precision $\E{\psi^s} \leq \varepsilon \psi^0$.
\end{corollary}

\begin{lemma}
We can upper bound the second moment of the $g^k$ in the following way
	\begin{equation}
		\E{\norm*{g^k}^2} \leq \left(2L + \frac{4L}{n} + \frac{6L\omega}{n} \right) (f(x^k) - f^\star  + f(z^s) - f^\star )
		+
		\frac{3\omega}{n^2} H^k,
		\label{lem:def_g_SVRG}
	\end{equation}
\end{lemma}
\begin{proof}
	\begin{align*}
	&\E{\norm*{g^k}^2}
	=
	\E{\EE{\cC}{\norm*{g^k}^2}}
	\overset{\eqref{eq:var_decomposition}}{=}
	\E{
		\norm*{\EE{\cC}{g^k}}^2
		+
		\EE{\cC}{\norm*{g^k - \EE{\cC}{g^k}}^2}
	}
	\\
	&=
	\E{
		\underbrace{
			\norm*{
				\frac{1}{n}\sum\limits_{i=1}^{n} g_i^k 
			}^2
		}_{T_1}
		+
		\underbrace{
			\EE{\cC}{
				\norm*{
					\frac{1}{n}\sum\limits_{i=1}^{n}
					\cC\left(g_i^k - h_{i}^k\right)
					-
					\left(g_i^k - h_{i}^k\right)
				}^2
			}	
		}_{T_2}
	},
	\end{align*}
where the third inequality uses the definition of $g^k$ in Algorithm~\ref{alg:SVRG_DIANA}.  We can further bound $\E{T_1}$ and $\E{T_2 }$.
\begin{eqnarray*}
	\E{T_1}
		&=&
		\E{\norm*{
			\frac{1}{n}
			\sum\limits_{i=1}^{n}
				g_i^k
		}^2}
		=
		\norm*{
			\frac{1}{n}
			\sum\limits_{i=1}^{n}
			\E{g_i^k}
		}^2
		+
		\E{\norm*{
				\frac{1}{n}
				\sum\limits_{i=1}^{n}
				\left( g_i^k - \E{g_i^k}\right)
			}^2}\\
		&=&
		\norm*{\nabla f(x^k)}^2
		+
		\frac{1}{n^2}
		\sum\limits_{i=1}^{n}
			\E{\norm*{g_i^k - \nabla f_i(x^k)}^2}\\
		&\overset{\eqref{eqn:quad-upper}}{\leq}&
		2L(f(x^k) - f^\star) \\
		&&\quad 
		+ 
		\frac{1}{n^2}
		\sum\limits_{i=1}^{n}
			\E{\norm*{\nabla f_{ij_i^k}(x^k) - \nabla f_{ij_i^k}(z^s) - \E{\nabla f_{ij_i^k}(x^k) - \nabla f_{ij_i^k}(z^s)}}^2}\\
		&\overset{\eqref{eq:var_decomposition}}{\leq}&
		2L(f(x^k) - f^\star)
		+
		\frac{1}{n^2}
		\sum\limits_{i=1}^{n}
			\E{\norm*{\nabla f_{ij_i^k}(x^k) - \nabla f_{ij_i^k}(z^s)}^2}\\
		&\overset{\text{Alg.}~\ref{alg:SVRG_DIANA}}{=}&
		2L(f(x^k) - f^\star)
		+
		\frac{1}{mn^2}
		\sum\limits_{i=1}^{n}
			\sum\limits_{j=1}^{m}
				\norm*{\nabla f_{ij}(x^k) - \nabla f_{ij}(z^s)}^2\\
		&\overset{\eqref{eq:sum_upper}}{\leq}&
		2L(f(x^k) - f^\star) \\
		&&\quad 
		+
		\frac{2}{mn^2}
		\sum\limits_{i=1}^{n}
			\sum\limits_{j=1}^{m}
			    \left(
				\norm*{\nabla f_{ij}(x^k) - \nabla f_{ij}(x^\star)}^2
				+
				\norm*{\nabla f_{ij}(z^s) - \nabla f_{ij}(x^\star)}^2
				\right)\\
		&\overset{\eqref{eqn:quad-upper}}{\leq}&
		\left(2L + \frac{4L}{n}\right)
		(f(x^k) - f^\star  + f(z^s) - f^\star).
\end{eqnarray*}
\begin{eqnarray*}
		\E{T_2}
		&\overset{\eqref{eq:indep}}{=}&
		\frac{1}{n^2}
		\E{
			\sum\limits_{i=1}^{n}
			\EE{\cC}{
				\norm*{
					\cC\left( g_i^k - h_{i}^k\right)
					-
					\left(g_i^k - h_{i}^k\right)
				}^2
			}
		}\\
		&
		\overset{\eqref{eq:var_decomposition}}{\leq}&
		\frac{\omega}{n^2}
		\sum\limits_{i=1}^{n} 
		\E{\norm*{ g_i^k - h_{i}^k}^2}\\
		&
		\overset{\eqref{eq:sum_upper}+\text{Alg.}~\ref{alg:SVRG_DIANA}}{\leq}&
		\frac{3\omega}{n^2}
		\sum\limits_{i=1}^{n} 
		\E{
			\norm*{\nabla f_{ij_i^k}(x^k) - \nabla f_{ij_i^k}(x^\star)}^2
			+ 
			\norm*{h_{i}^k - \nabla f_{i}(x^\star)}^2
		} \\
		&& \quad +  \frac{3\omega}{n^2}
		\sum\limits_{i=1}^{n} 
		\E{
			\norm*{\nabla f_{ij_i^k}(z^s) - \nabla f_{ij_i^k}(x^\star) - (\nabla f_{i}(z^s) - \nabla f_{i}(x^\star))}^2
		}\\
		&
		\overset{\eqref{eq:var_decomposition}}{\leq}&
		\frac{3\omega}{n^2}
		\sum\limits_{i=1}^{n} 
		\E{
			\norm*{\nabla f_{ij_i^k}(x^k) - \nabla f_{ij_i^k}(x^\star)}^2
		} \\
		&& \quad +  \frac{3\omega}{n^2}
		\sum\limits_{i=1}^{n} 
		\E{
			\norm*{\nabla f_{ij_i^k}(z^s) - \nabla f_{ij_i^k}(x^\star)}^2
			 +			
			\norm*{h_{i}^k - \nabla f_{i}(x^\star)}^2
		}\\
		&\overset{\eqref{eqn:quad-upper}}{\leq}&
		\frac{6L\omega}{n} (f(x^k) - f^\star  + f(z^s) - f^\star )
		+
		\frac{3\omega}{n^2} H^k.
	\end{eqnarray*}
	
	Summing up $\E{T_1}$ and $\E{T_2}$ we conclude the proof:
	\begin{align*}
		\E{\norm*{g^k}^2} = \E{T_1 + T_2} \leq
		\left(2L + \frac{4L}{n} + \frac{6L\omega}{n} \right) (f(x^k) - f^\star  + f(z^s) - f^\star )
		+
		\frac{3\omega}{n^2} H^k.
	\end{align*}

\end{proof}

\begin{lemma}
 Let $\alpha(\omega+1) \leq 1$. We can upper bound $H^{k+1}$ in the following way 
	\begin{equation}
		\E{H^{k+1}} \leq H^k\left(1 - \alpha\right)
		+
		4L\alpha n (f(x^k) - f^\star + f(z^s) - f^\star).
		\label{lem:def_H_SVRG}
	\end{equation}
\end{lemma}
\begin{proof}
	\begin{align*}
		\E{H^{k+1}}&=
		\sum\limits_{i=1}^n \E{\norm*{h_{i}^{k+1} - h_{i}^{k} + h_{i}^{k} - \nabla f_{i}(x^\star)}^2}
		\\
		&=
		\sum\limits_{i=1}^n
		\E{
			\norm*{h_{i}^k - \nabla f_{i}(x^\star)}^2
			+
			2\dotprod{h_{i}^{k+1} - h_{i}^k}{h_{i}^k - \nabla f_{ij}(x^\star)}
			+
			\norm*{h_{i}^{k+1} - h_{i}^k}^2
		}\\
		&=
		H^k + 
		\E{\sum\limits_{i=1}^n
		\left(
			2\dotprod{\EE{\cC}{h_{i}^{k+1} - h_{i}^k}}{h_{i}^k - \nabla f_{i}(x^\star)}
			+
			\EE{\cC}{\norm*{h_{i}^{k+1} - h_{i}^k}^2}
		\right)}.
	\end{align*}
	No we calculate expectations:
	\begin{align*}
		\EE{\cC}{h_{i}^{k+1} - h_{i}^k}
		= \alpha
		\EE{\cC}{\hat{\Delta}_i^k}
		=
		\alpha
		(g_i^k - h_{i}^k),
	\end{align*}
	\begin{align*}
		\EE{\cC}{\norm*{h_{i}^{k+1} - h_{i}^k}^2}
		&=
		\alpha^2
		\EE{\cC}{\norm*{\hat{\Delta}_i^k}^2}
		=
		\alpha^2
		(\omega + 1) \norm*{g_i^k - h_{i}^k}^2\\
		&\leq
		\alpha \norm*{g_i^k - h_{i}^k}^2.
	\end{align*}
	Finally we obtain
	\begin{eqnarray*}
		\E{H^{k+1}}
		&=&
		H^k
		+
		\E{\sum\limits_{i=1}^n
		\left(
			2\alpha\dotprod{g_i^k - h_{i}^k}{h_{i}^k - \nabla f_{i}(x^\star)}
			+
			\alpha\norm*{g_i^k - h_{i}^k}^2
		\right)}\\
		&=&
		H^k
		+
		\frac{\alpha}{m}
		\E{\sum\limits_{i=1}^n\sum\limits_{j=1}^m
		\dotprod{g_i^k - h_{i}^k}{g_i^k + h_{i}^k - 2\nabla f_{i}(x^\star)}}\\
		&=&
		H^k
		+
		\E{\frac{\alpha}{m}
		\sum\limits_{i=1}^n\sum\limits_{j=1}^m
		\left(
			\norm*{g_i^k - \nabla f_{i}(x^\star) }^2
			-
			\norm*{h_{i}^k - \nabla f_{i}(x^\star) }^2
		\right)}\\
		&\overset{\eqref{eqn:quad-upper}}{\leq}&
		H^k\left(1 - \alpha\right)
		+
		4L\alpha n (f(x^k) - f^\star + f(z^s) - f^\star),
	\end{eqnarray*}
	which concludes the proof.
\end{proof}

\begin{proof}[Proof of Theorem~\ref{thm:SVRG-DIANA}]
	\begin{eqnarray}
		&&\E{\norm*{x^{k+1} - x^\star}^2 + b \gamma^2 H^{k+1}} \notag \\
		&= &
		\norm*{x^{k} - x^\star}^2 
		+ 
		2\gamma\dotprod{x^{k} - x^\star}{\E{g^k}} 
		+
		\gamma^2 \E{\norm*{g^k}^2}
		+
		b \gamma^2\E{ H^{k+1}} \notag\\
		&=&
		\norm*{x^{k} - x^\star}^2 
		+ 
		2\gamma \dotprod{x^{k} - x^\star}{\nabla f(x^k)}
		+
		\gamma^2 \E{\norm*{g^k}^2}
		+
		b \gamma^2\E{ H^{k+1}}\notag \\	
		&\overset{\eqref{lem:def_g_SVRG}+\eqref{lem:def_H_SVRG}+\eqref{eq:def_strongly_convex}}{\leq}&
		(1-\mu\gamma)\norm*{x^{k} - x^\star}^2 \notag \\
		&& \quad
		+ 
		\left(
		L\frac{6\omega}{n} \gamma^2
		 + 
		 \left(2L + \frac{4L}{n}\right)\gamma^2
		 +
		 4b \gamma^2 L\alpha n
		 - 
		 2\gamma
		\right)
		(f(x^k)- f^\star) \notag\\
		&&\quad +
		\left(
		L\frac{6\omega}{n} \gamma^2
		 + 
		 \left(2L + \frac{4L}{n}\right)\gamma^2
		 +
		 4b \gamma^2 L\alpha n
		\right)
		(f(z^s)- f^\star) \notag\\
		&&\quad +
		b \gamma^2 H^k
		\left(
		1
		-
		\alpha
		+
		\frac{3\omega}{n^2b}
		\right).
		\label{dasdmasndlka}
	\end{eqnarray}
	Let $c = L\frac{6\omega}{n} \gamma^2
		 + 
		  \left(2L + \frac{4L}{n}\right)\gamma^2
		 +
		 4b \gamma^2 L\alpha n$, $\theta = \min\{\mu\gamma, \alpha
		-
		\frac{3\omega}{n^2b}\}$, $p_r = \frac{(1-\theta)^{l-1-r}}{\sum_{t=0}^{l-1}(1-\theta)^{l-1-t}}$ for $r = 0,1,\dots, l-1$, and assume that $b$ is picked such that  $\alpha
		-
		\frac{3\omega}{n^2b} > 0$ . We can apply previous inequality recursively for $ k = (s+1)l, (s+1)l-1,\dots, sl+1$, which implies
		\begin{align}
		&\E{\norm*{x^{(s+1)l} - x^\star}^2 + b \gamma^2 \bar{H}^{s+1} + \frac{1-(1-\theta)^l}{\theta}(2\gamma - c) (f(z^{s+1})- f^\star)} \notag\\
		&\quad \leq	
		(1-\theta)^l\norm*{z^s - x^\star}^2 
		+ 
		\frac{1-(1-\theta)^l}{\theta}c
		(f(z^s)- f^\star) 
		+
		b \gamma^2 \bar{H}^s
		\left(
		1
		-
		\theta
		\right)^l \label{conv_SVRG}\\
		&\quad \leq	
		\frac{2}{\mu}(1-\theta)^l(f(z^s)- f^\star) 
		+ 
		\frac{1-(1-\theta)^l}{\theta}c
		(f(z^s)- f^\star) 
		+
		b \gamma^2 \bar{H}^s
		\left(
		1
		-
		\theta
		\right)^l	. \notag		
	\end{align}
	Choosing $\bar{b} = \frac{b}{l(2\gamma-c)}$ we got
	\begin{align*}
		\E{\psi^{k+1}}
		 \leq	
		\frac{(1-\theta)^l}{1-(1-\theta)^l}\frac{2\theta + (1-(1-\theta)^l)c\mu}{\mu (2\gamma-c)}(f(z^s)- f^\star) 
		+ 
		\bar{b} \gamma^2 \bar{H}^s
		\left(
		1
		-
		\theta
		\right)^l,
	\end{align*}
	which concludes the proof.
\end{proof}
\subsection{Convex case}

Let us look at the convergence under weak convexity assumption, thus $\mu = 0$.

\begin{theorem}
\label{thm:conv_SVRG-DIANA}
Let $p_r = 1/l$ for $r = 0,1,\dots, l-1$ and $\alpha \leq \frac{1}{\omega + 1}$.
Under Assumptions~\ref{assumption:VRdiana}, \ref{assumption:VRdianac}and $R\equiv 0$, output   $x^a \sim_{u.a.r.} \{x^0 ,x^1,\dots, x^{k-1}\}$ of Algorithm~\ref{alg:SVRG_DIANA} satisfies
\begin{equation}
\E{f(x^a) - f^\star} \leq \frac{\norm*{x_0 - x^\star}^2 + lc(f(x_0)-f^\star)) + b \gamma^2H^0}{2k(\gamma-c)}, 
\end{equation} 
where $c = L\gamma^2\left(\frac{6\omega}{n}
		 + 
		 2 + \frac{1}{n}
		 +
		 4b\alpha n\right)$
		 and $k$ is number of iterations, which is multiple of $l$, the inner loop size.
\end{theorem}

\begin{corollary}
    Let $\gamma = \frac{1}{L\sqrt{m}\left(2+\frac{4}{n} + 18\frac{\omega}{n}\right)}$, $b = \frac{3\omega(\omega +1)}{ n^2}$, $l = m$ and $\alpha = \frac{1}{\omega + 1}$.
	To achieve precision $\E{f(x^a) - f^\star} \leq\varepsilon$ \texttt{SVRG}-\texttt{DIANA} needs
	$
	\cO\left(
	\frac{\left(1+\frac{\omega}{n}\right)\sqrt{m} + \frac{\omega}{\sqrt{m}}}{\epsilon}
	\right)
	$
	iterations.
\end{corollary}

\begin{proof}[Proof of Theorem~\ref{thm:conv_SVRG-DIANA}]
Using \eqref{dasdmasndlka} assuming that $b$ is picked such that  $\alpha
		-
		\frac{3\omega}{n^2b} \geq 0$, we obtain
\begin{align}
&\E{\norm*{x^{k+1} - x^\star}^2 + b \gamma^2 H^{k+1}}  \notag\\
		&\quad \leq	
		\norm*{x^{k} - x^\star}^2 
		+ 
		\left(
		c
		 - 
		 2\gamma
		\right)
		(f(x^k)- f^\star) 
		+
		c
		(f(z^s)- f^\star) 
		 +
		b \gamma^2 H^k.	
\end{align}
Taking $P^s = \E{\norm*{x^{sl} - x^\star}^2 
		+ 
		lc
		(f(z^s)- f^\star)
		+
		b \gamma^2 \bar{H}^{s}}$, full expectation, and using 
		\begin{align*}
		\E{l(f(z^{s+1})- f^\star)} = \sum_{j = sl}^{(s+1)l-1}(f(x^j)- f^\star),
		\end{align*}
		  we get
\begin{align*}
 (2\gamma - 2c)\sum_{j = sl}^{(s+1)l-1}(f(x^j)- f^\star) \leq P^s - P^{s+1},
\end{align*}
which can be summed over all epochs and one obtains
\begin{align*}
\E{(f(x^a)- f^\star)} \leq \frac{P^0}{2k(\gamma - c)},
\end{align*}
which concludes the proof.
\end{proof}

\subsection{Non-convex case}

\begin{theorem}\label{thm:SVRG-DIANA-non-convex}
    Consider Algorithm~\ref{alg:SVRG_DIANA} with $\omega$-quantization $\cC$,
	and stepsize $\alpha \leq \frac{1}{\omega + 1}$.
We consider the following Lyapunov function
\begin{equation*}
	R^k =f(x^k) + c^k Z^k + d^k F^k,
\end{equation*}
where
\begin{equation*}
	Z^k = \norm*{x^k - z^s}^2\,,
\end{equation*}
\begin{equation*}
	F^k =
	 \frac{1}{n}\sum\limits_{i=1}^{n} 
		\norm*{\nabla f_{i}(x^k) - h_i^k}^2\,,
\end{equation*}
\begin{equation*}
	c^k =
	 c^{k+1}\left( 1 + \gamma p + \frac{\omega+1}{n} L^2 \gamma^2 \right) + d^{k+1} \left(\alpha L^2 +  \left(1+\frac{2}{\alpha}\right)\frac{\omega+1}{n} L^4 \gamma^2\right)+  \frac{\omega+1}{n} \frac{\gamma^2L^3}{2} \,,
\end{equation*}
and
\begin{equation*}
	d^k =
	   d^{k+1}\left( 1 - \frac{\alpha}{2} +  \left(1+\frac{2}{\alpha}\right)\frac{\omega}{n}L^2 \gamma^2\right) + c^{k+1}\frac{\omega}{n} \gamma^2 + \frac{\omega}{n} \frac{\gamma^2L}{2}\,.
\end{equation*}
Then	
under Assumption~\ref{assumption:VRdiana} and $R\equiv 0$
	\begin{equation*}
		\E{R^{k+1}} \leq
		R^k - \Gamma^k \norm*{\nabla f(x^k)}^2\, ,
	\end{equation*}
	where 
	\begin{equation*}
		 \Gamma^k = 
		\gamma - \frac{\gamma^2L}{2}  - c^{k+1} \left(\gamma^2 + \frac{\gamma}{p}\right) - d^{k+1} \left(1+\frac{2}{\alpha}\right) L^2 \gamma^2\, .
	\end{equation*}
Taking $x^a \sim_{u.a.r.} \{x^0,\dots, x^{l-1}\}$ of Algorithm~\ref{alg:SVRG_DIANA} one obtains
	\begin{equation}
	\label{up:delta_SVRG}
		\E{\norm*{\nabla f(x^a)}^2} \leq
		\frac{R^0 - R^l}{k\Delta},
	\end{equation}
where $\Delta = \min_{t \in [k]} \Gamma^t > 0$ .
\end{theorem}

\begin{theorem}\label{thm:SVRG-DIANAnc}
Let Assumption~\ref{assumption:VRdiana} hold and $R\equiv 0$.. Moreover, let   \[\gamma = \frac{1}{10L\left( 1 + \frac{\omega}{n} \right)^{1/2} (m^{2/3} + \omega + 1)}\], $l = m$, $p_{l-1} = 1$, $p_r = 0$ for $r = 0,1,\hdots, l-2$, and $\alpha = \frac{1}{\omega+1}$, then  a randomly chosen iterate $x^a \sim_{u.a.r.} \{x^0,x^1,\dots, x^{k-1}\}$ of Algorithm~\ref{alg:SVRG_DIANA} satisfies
\begin{align*}
\E{\norm*{\nabla f(x^a)}^2} \leq 
\frac{40(f(x^0) - f^\star)L\left(1 +  \frac{\omega}{n} \right)^{1/2}(m^{2/3} + \omega + 1)}{k}, 
\end{align*}
where $k$ denotes the number of iterations, which is multiple of $m$.
\end{theorem}

\begin{corollary}
	To achieve precision $\E{\norm*{\nabla f(x^a)}^2} \leq\varepsilon$ \texttt{SVRG}-\texttt{DIANA} needs
	$$
	\cO\left(
	\left(1 + \frac{\omega}{n} \right)^{1/2} \frac{ m^{2/3} + \omega}{\varepsilon}
	\right)
	$$
	iterations.
\end{corollary}

\begin{lemma}  We can upper bound $Z^{k+1}$ in the following way 
	\begin{equation}
		\E{Z^{k+1}}
		\leq
		\E{\norm*{x^{k+1} - x^k}^2} + \left(1 + \gamma p\right)Z^k +\frac{\gamma}{p}\norm*{\nabla f(x^k)}^2.
		\label{lem:up_W_VR_nc_SVRG}
	\end{equation}
	
\end{lemma}
\begin{proof}
	\begin{align*}
		\E{Z^{k+1}}
		&=
		\E{\norm*{x^{k+1} - z^s}^2} = \E{\norm*{x^{k+1} - x^k + x^k - z^s}^2}\\
		&=
		 \E{\norm*{x^{k+1} - x^k}^2} + \norm*{x^{k} - z^s}^2
		 + 2\dotprod{\E{x^{k+1} -x^k}}{x^k - w_{ij}^{k}}
		\\
		&\leq
		\E{\norm*{x^{k+1} - x^k}^2} + \norm*{x^k - z^s}^2 + 2\gamma\left(\frac{1}{2p}\norm*{\nabla f(x^k)}^2 + \frac{p}{2}\norm*{x^k - z^s}^2\right)
		\\
		&=
		\E{\norm*{x^{k+1} - x^k}^2} + \left(1 + \gamma p\right)Z^k +\frac{\gamma}{p}\norm*{\nabla f(x^k)}^2,
	\end{align*}
	where the inequality uses Cauchy-Schwarz and Young inequalities with $p>0$.
\end{proof}

\begin{lemma} We can upper bound quantity 
$\frac{1}{n}
\sum\limits_{i=1}^{n}
\E{\norm*{g_i^k - h_i^k}^2}$
in the following way
	\begin{equation}
		\frac{1}{n}
		\sum\limits_{i=1}^{n}
		\E{\norm*{g_i^k - h_i^k}^2}
		\leq 
		F^k + L^2 Z^k.
		\label{lem:up_g-h_VR_nc_SVRG}
	\end{equation}
\end{lemma}

\begin{proof}
\begin{eqnarray*}
&&\frac{1}{n}
		\sum\limits_{i=1}^{n}
		\E{\norm*{g_i^k - h_i^k}^2} \\
		&=&
		\frac{1}{n}
		\sum\limits_{i=1}^{n}
		    \left(
			\norm*{\E{g_i^k - h_i^k}}^2
			+
			\E{\norm*{g_i^k - h_i^k - \E{g_i^k - h_i^k}}^2}
			\right)\\
		&=&
		\frac{1}{n}
		\sum\limits_{i=1}^{n}
		    \left(
			\norm*{\nabla f_i(x^k) - h_i^k}^2
			+
			\E{\norm*{g_i^k - \nabla f_i(x^k)}^2}
			\right)\\
		&=&
		F^k +
		\frac{1}{n}
		\sum\limits_{i=1}^{n}
			\E{\norm*{\nabla f_{ij_i^k}(x^k) - \nabla f_{ij_i^k}(z^s) - \E{\nabla f_{ij_i^k}(x^k) - \nabla f_{ij_i^k}(z^s)}}^2}
		\\
		&\overset{\eqref{eq:var_decomposition}}{\leq}&
		F^k +
		\frac{1}{n}
		\sum\limits_{i=1}^{n}
			 \E{\norm*{\nabla f_{ij_i^k}(x^k) - \nabla f_{ij_i^k}(z^s)}^2}
			 \\
		&=&
		F^k +
		\frac{1}{n}
		\sum\limits_{i=1}^{n}
		\frac{1}{m}
		\sum\limits_{i=1}^{m}
			 \E{\norm*{\nabla f_{ij}(x^k) - \nabla f_{ij}(z^s)}^2}
			 \\
		&\overset{\eqref{eqn:quad-upper}}{\leq}&
		F^k + L^2 Z^k.
	\end{eqnarray*}
\end{proof}
Equipped with this lemma, we are ready to prove a recurrence inequality for $F^k$:
\begin{lemma} Let $\alpha (\omega + 1)\leq 1$. We can upper bound $F^{k+1}$ in the following way
	\begin{equation}
		\E{F^{k+1}}
		\leq
		\left(1+\frac{2}{\alpha}\right)L^2\E{\norm*{x^{k+1} - x^k}^2} + 
	\left(1-\frac{\alpha}{2}\right) F^k + \alpha L^2 Z^k
		\label{lem:up_F_VR_nc_SVRG}
	\end{equation}
\end{lemma}
\begin{proof}
	\begin{eqnarray*}
		&&\E{F^{k+1}} \\
		&=&
		\frac{1}{n}
		\sum\limits_{i=1}^{n}
				\E{\norm*{\nabla f_{i}(x^{k+1}) - h_i^{k+1}}^2}\\
		&=&
		\frac{1}{n}
		\sum\limits_{i=1}^{n}
				\E{\norm*{\nabla f_{i}(x^{k+1}) - \nabla f_{i}(x^k) + \nabla f_{i}(x^k) - h_i^k - \alpha \cC\left( g_i^k - h_i^k\right)}^2}\\
		& = &
		\frac{1}{n}
		\sum\limits_{i=1}^{n}
		\left(
				 \E{\norm*{\nabla f_{i}(x^{k+1}) - \nabla f_{i}(x^k)}^2} + \E{\norm*{\nabla f_{i}(x^k) - h_i^k - \alpha \cC\left( g_i^k - h_i^k\right)}^2}
		\right)\\
		&& \quad +  (1-\alpha)\frac{1}{n}
		\sum\limits_{i=1}^{n}
		\dotprod{\nabla f_{i}(x^{k+1}) - \nabla f_{i}(x^k)}{ \nabla f_{i}(x^k) - h_i^k } \\
		&\overset{\eqref{eqn:quad-upper}}{\leq}&
		\frac{1}{n}
		\sum\limits_{i=1}^{n}
		\left(
				 \left(1 + \frac{1-\alpha}{\tau}\right)L^2\E{\norm*{x^{k+1} - x^k}^2} + 
		 (1+(1-\alpha)\tau)\norm*{\nabla f_{i}(x^k) - h_i^k}^2 \right) \\
		&&\quad +  \alpha^2\sum\limits_{i=1}^{n}\E{\norm*{ \cC\left( g_i^k - h_i^k\right)}^2}  - 2 \frac{\alpha}{n} \sum\limits_{i=1}^{n}\dotprod{\nabla f_{i}(x^k) - h_i^k}{\E{\cC\left( g_i^k - h_i^k\right)}}
		\\
	&\overset{\eqref{def:omega}}{\leq}&
	 \left(1 + \frac{1-\alpha}{\tau}\right)L^2\E{\norm*{x^{k+1} - x^k}^2} + 
	\frac{1}{n}
		\sum\limits_{i=1}^{n}
		\left( (1+ (1-\alpha)\tau)
		\norm*{\nabla f_{i}(x^k) - h_i^k}^2 \right) \\
		&&\quad + \alpha^2\sum\limits_{i=1}^{n}
		(\omega+1)\E{\norm*{g_i^k - h_i^k}^2} - 2 \frac{\alpha}{n}\sum\limits_{i=1}^{n}\dotprod{\nabla f_{i}(x^k) - h_i^k}{\nabla f_{i}(x^k) - h_i^k}\\
	&\leq&
	 \left(1 + \frac{1-\alpha}{\tau}\right)L^2\E{\norm*{x^{k+1} - x^k}^2} + 
	(1+(1-\alpha)\tau-2\alpha) F^k \\
	&	&\quad  + \alpha\frac{1}{n}
		\sum\limits_{i=1}^{n}\E{\norm*{g_i^k - h_i^k}^2} 
		\\	
	&\overset{\eqref{lem:up_g-h_VR_nc_SVRG}}{\leq}	&
	 \left(1 + \frac{1-\alpha}{\tau}\right)L^2\E{\norm*{x^{k+1} - x^k}^2} + 
	(1+\tau-\alpha) F^k + \alpha L^2 Z^k,
	\end{eqnarray*}
	where the second equality uses definition of $h_i^{k+1}$ in Algorithm~\ref{alg:SVRG_DIANA} and the first inequality follows from Cauchy  inequality and holds for any $\tau > 0$. 
	
	Taking $\tau = \alpha/2$, we obtain desired inequality.
\end{proof}

\begin{lemma}
We can upper bound the second moment of the $g^k$ in the following way
	\begin{equation}
		\E{\norm*{g^k}^2}
		\leq
		\frac{\omega}{n}F^k + \frac{\omega+1}{n} L^2 Z^k + \norm*{\nabla f(x^k)}^2
		.
		\label{lem:up_g_VR_nc_SVRG}
	\end{equation}
\end{lemma}
\begin{proof}
	\begin{align*}
		\EE{\cC}{\norm*{g^k}^2} &\overset{\eqref{eq:var_decomposition}}{=}
		\underbrace{
			\norm*{\EE{\cC}{g^k}}^2
		}_{T_1}
		+
		\underbrace{
			\EE{\cC}{\norm*{g^k - \EE{\cC}{g^k}}^2}
		}_{T_2}.
	\end{align*}
	We can use the definition of $g^k$ in order to obtain 
	\begin{align*}
		T_1
		&= 
		\norm*{
			\frac{1}{n}
			\sum\limits_{i=1}^{n}
				\EE{\cC}{\cC(g_i^k - h_i^k) + h_i^k}
		}^2
		=
		\norm*{
			\frac{1}{n}
			\sum\limits_{i=1}^{n}
				g_i^k
		}^2
	\end{align*}
	and 
	\begin{eqnarray*}
		T_2
		&=&
		\EE{\cC}{
			\norm*{
				\frac{1}{n}
				\sum\limits_{i=1}^{n}
					\cC(g_i^k - h_i^k) - (g_i^k - h_i^k)
			}^2
		}\\
		&\overset{\eqref{eq:sum_upper}}{=}&
		\frac{1}{n^2}
		\sum\limits_{i=1}^{n}
			\EE{\cC}{
				\norm*{\cC(g_i^k - h_i^k) - (g_i^k - h_i^k)}^2
			} \\ 
		&\overset{\eqref{eq:omega_quant}}{\leq}&
		\frac{\omega}{n^2}
		\sum\limits_{i=1}^{n}
			\norm*{g_i^k - h_i^k}^2.
	\end{eqnarray*}
	Now we calculate full expectations conditioned on previous iteration:
	\begin{eqnarray*}
		\E{T_2}
		=
		\frac{\omega}{n^2}
		\sum\limits_{i=1}^{n}
			\E{\norm*{g_i^k - h_i^k}^2} 
		\overset{\eqref{lem:up_g-h_VR_nc_SVRG}}{\leq}
		\frac{\omega}{n}F^k + \frac{\omega}{n}L^2 Z^k.
	\end{eqnarray*}
	As for $T_1$, we have
	\begin{eqnarray*}
		&&\E{T_1} \\
		&=&
		\E{\norm*{
			\frac{1}{n}
			\sum\limits_{i=1}^{n}
				g_i^k
		}^2}
		=
		\norm*{
			\frac{1}{n}
			\sum\limits_{i=1}^{n}
			\E{g_i^k}
		}^2
		+
		\E{\norm*{
				\frac{1}{n}
				\sum\limits_{i=1}^{n}
				g_i^k - \E{g_i^k}
			}^2}\\
		&=&
		\norm*{\nabla f(x^k)}^2
		+
		\frac{1}{n^2}
		\sum\limits_{i=1}^{n}
			\E{\norm*{g_i^k - \nabla f_i(x^k)}^2}\\
		&=&
		\norm*{\nabla f(x^k)}^2
		+
		\frac{1}{n^2}
		\sum\limits_{i=1}^{n}
			\E{\norm*{\nabla f_{ij_i^k}(x^k) - \nabla f_{ij_i^k}(z^s) - \E{\nabla f_{ij_i^k}(x^k) - \nabla f_{ij_i^k}(z^s)}}^2}\\
		&\overset{\eqref{eq:var_decomposition}}{\leq}&
		\norm*{\nabla f(x^k)}^2
		+
		\frac{1}{n^2}
		\sum\limits_{i=1}^{n}
			\E{\norm*{\nabla f_{ij_i^k}(x^k) - \nabla f_{ij_i^k}(w_{ij_i^k}^k)}^2}\\
		&\overset{\text{Alg.}~\ref{alg:SVRG_DIANA}}{=}&
		\norm*{\nabla f(x^k)}^2
		+
		\frac{1}{mn^2}
		\sum\limits_{i=1}^{n}
			\sum\limits_{j=1}^{m}
				\norm*{\nabla f_{ij}(x^k) - \nabla f_{ij}(z^s)}^2\\
		&\overset{\eqref{eqn:quad-upper}}{\leq}&
		\norm*{\nabla f(x^k)}^2
		+
		\frac{1}{n}L^2 Z^k.
	\end{eqnarray*}
	Now, summing $\E{T_1}$ and $\E{T_2}$ we get
	\begin{align*}
		\E{\norm*{g^k}^2}
		&=
		\E{T_1 + T_2}
		\leq
		\frac{\omega}{n}F^k + \frac{\omega+1}{n} L^2 Z^k + \norm*{\nabla f(x^k)}^2,
	\end{align*}
	which concludes the proof.
\end{proof}

\begin{proof}[Proof of Theorem~\ref{thm:SVRG-DIANA-non-convex}]
Using $L$-smoothness one gets
	\begin{align}
		\E{f(x^{k+1})} &\leq f(x^k) + \dotprod{\nabla f(x^k)}{\E{x^{k+1}-x^k}} + \frac{L}{2}\E{\norm*{x^{k+1} - x^k}^2} \notag \\
		& = f(x^k) - \gamma\norm*{\nabla f(x^k)}^2 + \frac{L\gamma^2}{2}\E{\norm*{g^k}^2},
		\label{lem:nc_smooth_SVRG}
	\end{align}
where we use the definition of $x^{k+1}$ in Algorithm~\ref{alg:SVRG_DIANA}.

By combining definition of $\E{R^{k+1}}$ with \eqref{lem:up_W_VR_nc_SVRG},\eqref{lem:up_F_VR_nc_SVRG},\eqref{lem:nc_smooth_SVRG} one obtains
\begin{eqnarray*}
\E{R^{k+1}} &\leq&  f(x^k) + \dotprod{\nabla f(x^k)}{\E{x^{k+1}-x^k}} + \frac{L}{2}\E{\norm*{x^{k+1} - x^k}^2} \\
&&\quad + c^{k+1}\left( \E{\norm*{x^{k+1} - x^k}^2} + \left(1 + \gamma p\right)Z^k +\frac{\gamma}{p}\norm*{\nabla f(x^k)}^2\right) \\
&&\quad + d^{k+1}\left(  \left(1+\frac{2}{\alpha}\right)L^2\E{\norm*{x^{k+1} - x^k}^2} + \left(1-\frac{\alpha}{2}\right) F^k +\alpha L^2 Z^k\right) \\
&=&
f(x^k) - \gamma\norm*{\nabla f(x^k)}^2 + \left(\frac{\gamma^2L}{2} + c^{k+1} \gamma^2 + d^{k+1}  \left(1+\frac{2}{\alpha}\right)L^2 \gamma^2\right)\E{\norm*{g^k}^2} \\
&&\quad + c^{k+1}\left( \left(1 - \frac{1}{m} + \gamma p\right)Z^k +\frac{\gamma}{p}\norm*{\nabla f(x^k)}^2\right) \\
&&\quad + d^{k+1}\left(\left(1-\frac{\alpha}{2}\right)F^k + \left(1+\frac{2}{\alpha}\right)\alpha L^2 Z^k\right) \\
&\overset{\eqref{lem:up_g_VR_nc_SVRG}}{\leq}&
f(x^k) - \left(\gamma - \frac{\gamma^2L}{2}  - c^{k+1} \left(\gamma^2 + \frac{\gamma}{p}\right) - d^{k+1} \left(1+\frac{2}{\alpha}\right) L^2 \gamma^2\right)\norm*{\nabla f(x^k)}^2\\ 
&&\quad + c^{k+1}\left( 1 - \frac{1}{m} + \gamma p + \frac{\omega+1}{n} L^2 \gamma^2 \right) Z^k \\
&& \quad+ d^{k+1} \left(\alpha L^2 +  \left(1+\frac{2}{\alpha}\right)\frac{\omega+1}{n} L^4 \gamma^2\right)+  \frac{\omega+1}{n} \frac{\gamma^2L^3}{2} Z^k\\
&&\quad + \left( d^{k+1}\left( 1 - \frac{\alpha}{2} +  \left(1+\frac{2}{\alpha}\right)\frac{\omega}{n}L^2 \gamma^2\right) + c^{k+1}\frac{\omega}{n} \gamma^2 + \frac{\omega}{n} \frac{\gamma^2L}{2} \right) F^k \\
&=& R^k - \Gamma^k \norm*{\nabla f(x^k)}^2.
\end{eqnarray*}
Applying the full expectation and telescoping the equation, one gets desired inequality.
\end{proof}
We can proceed to the proof of Theorem~\ref{thm:SVRG-DIANAnc}.
\begin{proof}[Proof of Theorem~\ref{thm:SVRG-DIANAnc}]
Recursion for $c^t, d^t$ can be written in a form
\begin{align*}
y^t = Ay^{t+1} + b,
\end{align*}
where 
\begin{align*}
A &= \begin{bmatrix}
 1 + \gamma p + \frac{\omega+1}{n} L^2 \gamma^2 &\alpha L^2 +  \left(1+\frac{2}{\alpha}\right)\frac{\omega+1}{n} L^4 \gamma^2\\
\frac{\omega}{n} \gamma^2  & 1 - \frac{\alpha}{2} +  \left(1+\frac{2}{\alpha}\right)\frac{\omega}{n}L^2 \gamma^2
\end{bmatrix} , \\
y^t &=  \begin{bmatrix}
 c^t \\
 d^t
\end{bmatrix} , \\
b &=  \begin{bmatrix}
 \frac{\omega+1}{n} \frac{\gamma^2L^3}{2} \\
 \frac{\omega}{n} \frac{\gamma^2L}{2}
\end{bmatrix} .
\end{align*}

Choosing $\gamma = \frac{1}{10L\left( 1 + \frac{\omega}{n} \right)^{1/2} (m^{2/3} + \omega + 1)}$, $p =  \frac{L\left(1+ \frac{\omega}{n} \right)^{1/2}}{ (m^{2/3} + \omega + 1)^{1/2}}$, $l = m$, and $\alpha = \frac{1}{\omega+1}$, where  $c^l = d^l = 0$ we can upper bound each element of matrix $A$ and construct its upper bound $\hat{A}$, where 
\begin{align*}
\hat{A} &= \begin{bmatrix}
 1 + \frac{11}{100m}  & L^2 \frac{103}{100(\omega+1)} \\
 \frac{1}{100L^2m} & 1 -  \frac{47}{100(\omega+1)}
 \end{bmatrix},\qquad
 \hat{b} = \left(1+\frac{\omega}{n}\right)\frac{\gamma^2 L}{2}\begin{bmatrix}
 L^2 \\
 1
 \end{bmatrix}.
\end{align*}
The same as for Proof of Theorem~\ref{thm:VR-DIANAnc}, due to structure of $\hat{A}$ and $\hat{b}$ we can work with matrices 
\begin{align*}
\tilde{A} &= \begin{bmatrix}
 1 + \frac{11}{100m}  & \frac{103}{100(\omega+1)} \\
 \frac{1}{100m} & 1 -  \frac{47}{100(\omega+1)}
 \end{bmatrix},\qquad
 \tilde{b} = \left(1+\frac{\omega}{n}\right)\frac{\gamma^2 L}{2}\begin{bmatrix}
 1 \\
 1
 \end{bmatrix},
\end{align*}
where it holds $\tilde{A}^k \tilde{b} = (y_1, y_2)^\top \implies \hat{A}^k \hat{b} = (L^2y_1, y_2)^\top$, thus we can work with $\tilde{A}$ which is independent of $L$.
In the sense of Lemma~\ref{lem:sequence}, we have that eigenvalues of $\tilde{A}$ are less than $1-\frac{1}{3(\omega + 1)}, 1+\frac{1}{m}$, respectively, and $\abs{t} \leq \frac{2}{\min\{m, \omega + 1\}}$, thus for $k \leq l-1$
\begin{align*}
c^k &\leq 20\max\{m, \omega+1\}\left(\left(1+\frac{1}{m}\right)^m -1 \right)\left(1+\frac{\omega}{n}\right) \frac{\gamma^2L^3}{2} \\
&=  20(e - 1)\max\{m, \omega+1\}\left(1+\frac{\omega}{n}\right) \frac{L^3}{200\left(1 + \frac{\omega}{n}\right) L^2 (m^{2/3} + \omega + 1)^2} \\
&\leq \frac{L}{2(m^{2/3} + \omega + 1)^{1/2}}.
\end{align*}
By the same reasoning
\begin{align*}
d^k \leq \frac{1}{2L(m^{2/3} + \omega + 1)^{1/2}}\, .
\end{align*}
This implies 
\begin{align*}
\Gamma^k & = \gamma - \frac{\gamma^2L}{2}  - c^{k+1} \left(\gamma^2 + \frac{\gamma}{p}\right) - d^{k+1}  \left(1 + \frac{2}{\alpha} \right) L^2 \gamma^2 \\
& \geq \frac{1}{40L\left(1 +  \frac{\omega}{n} \right)^{1/2}(m^{2/3} + \omega + 1)},
\end{align*}
which guarantees $\Delta \geq \frac{1}{40L\left(1 +  \frac{\omega}{n} \right)^{1/2}(m^{2/3} + \omega + 1)}$ for $k = 0, 1, \hdots, l-1$. Using iterates $k = cl, cl+1, \hdots, (c+1)l-1$, where $c$ is any positive integer, one can obtain the same bound of $\Gamma^k$ for arbitrary $k$.  Plugging this uniform lower bound on $\Delta$ into \eqref{up:delta_SVRG} for all iterates and using the fact that $p_{l-1} = 1$ and all other $p_r$'s are zeros, one obtains
\begin{align*}
\E{\norm*{\nabla f(x^a)}^2} \leq 
\frac{40(f(x^0) - f^\star)L\left(1 +  \frac{\omega}{n} \right)^{1/2}(m^{2/3}+ \omega + 1)}{m}.
\end{align*}
where $x^a \sim_{u.a.r.} \{x^0, x^1, \dots, x^{k-1}\}$, which concludes the proof.
\end{proof}

\section{Technical Lemma}

\begin{lemma}
\label{lem:sequence}
Let $A$ be a $2 \times 2$ matrix of which all entries are non-negative and $\{ y^1, y^2, \hdots, y^T\}$ be a sequence of vectors for which $y^k = Ay^{k+1} + b$ and $y^T = (0, 0)$, where $b$ is a vector with non-negative entries, and $\hat{A} = A+B$, $\hat{b} = b+ y$, where $B$ and $y$ have all entries non-negative. Then for the sequence $\hat{y}^k = \hat{A}\hat{y}^{k+1} + \hat{b}$ it always holds that 
$\hat{y}^k \geq y^k$ (coordinate-wise) for $k = 0,1, \dots, T$. Moreover, let $\hat{A}$ has positive real eigenvalues $\lambda_1 \geq \lambda_2 > 0$ different from $1$, thus there exists a real Schur decomposition of matrix $\hat{A}  = UTU^{\top}$, where 
\begin{align*}
T  = \begin{bmatrix}
 \lambda_1  & t\\
0 & \lambda_2
 \end{bmatrix}
\end{align*} and $t \in \R$, and $U$ is real unitary matrix, then for every element of $\hat{y}^k$ it holds
\begin{equation*}
\hat{y}^k_j \leq \left(\frac{(1 - \lambda_1^T)(1 - \lambda_2^T)}{(1 - \lambda_1)(1 - \lambda_2)}\abs{t} + \frac{(1 - \lambda_1^T)}{(1 - \lambda_1)} + \frac{(1 - \lambda_2^T)}{(1 - \lambda_2)}\right)(b_1 + b_2)
\end{equation*}
for $k = 0,1,2, \hdots, T-1$.
\end{lemma}

\begin{proof}
From $y^k = Ay^{k+1} + b$ and $y^T = (0, 0)$ one can obtain
\begin{align*}
y^{T-k} &= A^{k-1}b + A^{k-2}b + \dots + b \\
\hat{y}^{T-k} &= (A+B)^{k-1}(b+y) + (A+B)^{k-2}(b+y) + \dots + (b+y) .
\end{align*}
From these equalities it is trivial to see that $\hat{y}^k \leq y^k$, because $\hat{y}^k$ contains at least all the elements of $y^k$ and every element is non-negative.

For the second part of the claim, we have for every element of  $\hat{y}^{T-k}$
\begin{align*}
\hat{y}^{T-k}_j &\leq \norm*{\hat{y}^{T-k}} = \norm*{\hat{A}^{k-1}\hat{b} + \hat{A}^{k-2}\hat{b} + \dots + \hat{b}} \\
&\leq \norm*{\hat{A}^{k-1}\hat{b}} + \norm*{\hat{A}^{k-2}\hat{b}} + \dots + \norm*{\hat{b}} \\
& =  \norm*{UT^{k-1}U^{\top}\hat{b}} + \norm*{UT^{k-2}U^{\top}\hat{b}} + \dots + \norm*{\hat{b}} \\
& \leq \left(\norm*{T^{k-1}} + \norm*{T^{k-2}} + \dots + \norm*{I}\right)\norm*{\hat{b}} \\
& \leq  \left(\frac{(1 - \lambda_1^T)(1 - \lambda_2^T)}{(1 - \lambda_1)(1 - \lambda_2)}\abs{t} + \frac{(1 - \lambda_1^T)}{(1 - \lambda_1)} + \frac{(1 - \lambda_2^T)}{(1 - \lambda_2)}\right)(\hat{b}_1 + \hat{b}),
\end{align*}
where the last inequality follows from fact that for any matrix $C$, $\norm*{C} \leq \norm*{C}_F \leq \abs{C_{11}} + \abs{C_{12}} + \abs{C_{21}} + \abs{C_{22}}$ and 
\begin{align*}
T^k  = \begin{bmatrix}
 \lambda_1^k  & t \sum_{i=0}^k \lambda_1^i  \lambda_2^{k-i}\\
0 & \lambda_2^k
 \end{bmatrix},
\end{align*}
which concludes the proof.
\end{proof}

\refstepcounter{chapter}%
\chapter*{\thechapter \quad Appendix: On biased compression for distributed
learning}
\label{appendix:biased}

\section{Proofs for Section~\ref{sec:examples}}

\subsection{Proof of Lemma \ref{lem-ex:ur-sparse}: Unbiased random sparsification}

From the definition of $k$-nice sampling we have $p_i \eqdef \Prob(i\in S) = \tfrac{k}{d}$. Hence
\begin{align*}
\Exp{\cC(x)} &= \frac{d}{k} \Exp{\sum_{i\in S} x_i e_i} = \frac{d}{k}\sum_{i=1}^d p_i x_i e_i = \sum_{i=1}^d x_i e_i = x, \\
\Exp{\norm*{\cC(x)}^2} &= \frac{d^2}{k^2} \Exp{\sum_{i\in S} x_i^2} = \frac{d^2}{k^2} \sum_{i=1}^d p_i x_i^2 = \frac{d}{k} \sum_{i=1}^d x_i^2 = \frac{d}{k} \norm*{x}^2,
\end{align*}
which implies $\cC\in\U(\tfrac{d}{k})$.

\subsection{Proof of Lemma \ref{lem-ex:br-sparse}: Biased random sparsification}

Let $S\subseteq [d]$ be a proper sampling with probability vector $p=(p_1,\dots,p_d)$, where  $p_i \eqdef \Prob(i\in S)>0$ for all $i$. Then
$$\Exp{\cC(x)} = \Diag{p}x = \sum_{i=1}^d p_i x_i e_i \quad\text{and}\quad \Exp{\norm*{\cC(x)}^2} = \sum_{i=1}^d p_i x_i^2.$$
Letting $q \eqdef \min_i p_i$, we get
\[ q \norm*{x}^2 \leq   \sum \limits_{i=1}^d p_i x_i^2 = \Exp{\norm*{\cC(x)}^2}
 =  \left \langle \Exp{\cC(x)}, x \right \rangle . \]
So, $\cC\in \mathbb{B}^1(q, 1)$ and $\cC\in \mathbb{B}^2(q, 1)$. For the third class, note that
$$\Exp{\norm*{\cC(x) - x}^2} = \sum \limits_{i=1}^d (1-p_i) x_i^2 \leq (1-q)\norm*{x}^2.$$
Hence, $\cC\in \mathbb{B}^3(\tfrac{1}{q})$.

\subsection{Proof of Lemma \ref{lem-ex:ar-sparse}: Adaptive random sparsification}

From the definition of the compression operator, we have
\begin{align*}
\Exp{\norm*{\cC(x)}^2} &= \Exp{x_i^2} =  \sum \limits_{i=1}^d \frac{|x_i|}{\onenorm{x}} x_i^2 = \frac{\threenorm{x}^3}{\onenorm{x}},\\
\Exp{\lin{\cC(x),x}} &=\Exp{x_i^2} =  \frac{\threenorm{x}^3}{\onenorm{x}},
\end{align*}
whence $\beta=1$. Furthermore, by Chebychev's sum inequality, we have
\begin{equation*}
\tfrac{1}{d^2}\onenorm{x} \norm*{x}^2 = \left(\sum \limits_{i=1}^d \tfrac{1}{d} |x_i|\right) \left(\sum \limits_{i=1}^d \tfrac{1}{d}x_i^2\right) \leq \sum \limits_{i=1}^d \tfrac{1}{d} |x_i| x_i^2 = \tfrac{1}{d} \threenorm{x}^3,
\end{equation*}
which implies that $\alpha=\frac{1}{d},\,\delta = d$. So, $\cC \in \mathbb{B}^1(\frac{1}{d},1)$, $\cC \in \mathbb{B}^2(\frac{1}{d},1)$, and $\cC \in \mathbb{B}^3(d)$.

\subsection{Proof of Lemma \ref{lem-ex:top-sparse}: Top-$k$ sparsification}

Clearly, $\norm*{\cC(x)}^2  = \sum_{i=d-k+1}^d x_{(i)}^2$ and $\norm*{\cC(x)-x}^2  = \sum_{i=1}^{d-k} x_{(i)}^2$. Hence
$$\frac{k}{d} \norm*{x}^2 \leq \norm*{\cC(x)}^2 = \lin{\cC(x),x} \leq \norm*{x}^2,\quad \norm*{\cC(x)-x}^2 \le \left(1-\frac{k}{d}\right)\norm*{x}^2.$$
So, $\cC\in \mathbb{B}^1(\frac{k}{d},1)$, $\cC \in \mathbb{B}^2(\frac{k}{d},1)$, and $\cC \in \mathbb{B}^3(\frac{d}{k})$.

\subsection{Proof of Lemma \ref{lem-ex:gu-rounding}: General unbiased rounding}

The unbiasedness follows immediately from the definition (\ref{ex:gu-rounding})
\begin{equation}\label{apx:umbiasedness-rounding}
\begin{split}
\Exp{\cC(x)} &= \sum_{i=1}^d \Exp{\cC(x)_i}e_i = \sum_{i=1}^d \sign(x_i)\left( a_k \frac{a_{k+1}-|x_i|}{a_{k+1}-a_k} + a_{k+1} \frac{|x_i| - a_k}{a_{k+1}-a_k} \right)e_i \\
&= \sum_{i=1}^d x_i e_i = x.
\end{split}
\end{equation}

Since the rounding compression operator $\cC$ applies to each coordinate independently, without loss of generality we can consider the compression of scalar values $x=t>0$ and show that $\Exp{\cC(t)^2} \le \zeta \cdot t^2$. From the definition we compute the second moment as follows
\begin{equation}\label{apx:eq-it-10}
\begin{split}
\Exp{\mathcal{C}(t)^2} &= a_k^2 \frac{a_{k+1}-t}{a_{k+1}-a_k} + a_{k+1}^2 \frac{t - a_k}{a_{k+1}-a_k} = (a_k+a_{k+1})t - a_k a_{k+1} \\
&= t^2 + (t-a_k)(a_{k+1}-t),
\end{split}
\end{equation}
from which
\begin{equation}\label{apx:eq-it-11}
\frac{\Exp{ \mathcal{C}(t)^2}}{t^2} = 1 + \left(1 - \frac{a_k}{t}\right) \left(\frac{a_{k+1}}{t} - 1\right), \quad a_k\le t\le a_{k+1}.
\end{equation}
Checking the optimality condition, one can show that the maximum is achieved at
$$
t_* = \frac{2 a_k a_{k+1}}{a_k+a_{k+1}} = \frac{2}{\frac{1}{a_k} + \frac{1}{a_{k+1}}},
$$
which being the harmonic mean of $a_k$ and $a_{k+1}$, is in the range $[a_k, a_{k+1}]$. Plugging it to the expression for variance we get
$$
\frac{\Exp{ \mathcal{C}(t_*)^2}}{t_*^2} = 1 + \frac{1}{4}\left(1-\frac{a_k}{a_{k+1}}\right) \left(\frac{a_{k+1}}{a_k} - 1\right) = \frac{1}{4}\left(\frac{a_k}{a_{k+1}} + \frac{a_{k+1}}{a_k} + 2\right).
$$

Thus, the parameter $\zeta$ for general unbiased rounding would be
$$
\zeta = \sup_{t>0} \frac{\Exp{ \mathcal{C}(t)^2}}{t^2} = \sup_{k\in\Z}\sup_{a_k\le t\le a_{k+1}} \frac{\Exp{ \mathcal{C}(t)^2}}{t^2} = \frac{1}{4} \sup_{k\in\Z}\left(\frac{a_k}{a_{k+1}} + \frac{a_{k+1}}{a_k} + 2\right) \ge 1.
$$

\subsection{Proof of Lemma \ref{lem-ex:gb-rounding}: General biased rounding}

From the definition (\ref{ex:gb-rounding}) of compression operator $\cC$ we derive the following inequalities
\begin{align*}
\inf_{k\in\Z}\left(\frac{2a_k}{a_k+a_{k+1}}\right)^2\|x\|^2 &\le \|\cC(x)\|^2, \\
\|\cC(x)\|^2 &\le \sup_{k\in\Z}\frac{2a_{k+1}}{a_k+a_{k+1}}\lin{\cC(x),x},\\
\inf_{k\in\Z}\frac{2a_k}{a_k+a_{k+1}} \|x\|^2 &\le \lin{\cC(x),x},
\end{align*}
which imply that $\cC\in\B^1(\alpha, \beta)$ and  $\cC \in \B^2(\gamma, \beta)$, with
$$\beta = \sup_{k\in\Z}\frac{2a_{k+1}}{a_k+a_{k+1}},\qquad \gamma = \inf_{k\in\Z}\frac{2a_k}{a_k+a_{k+1}},\qquad \alpha = \gamma^2.$$

For the third class $\B^3(\delta)$, we need to upper bound the ratio $\norm*{\cC(x)-x}^2/\norm*{x}^2$. Again, as $\cC$ applies to each coordinate independently, without loss of generality we consider the case when $x=t>0$ is a scalar.
From  definition (\ref{ex:gb-rounding}), we get
\begin{equation}\label{eq:biased-rounding-nvar}
\frac{(\cC(t)-t)^2}{t^2} = \min\left[\left(1-\frac{a_k}{t}\right)^2, \left(1-\frac{a_{k+1}}{t}\right)^2\right], \qquad a_k \le t \le a_{k+1}.
\end{equation}

It can be easily checked that $\left(1-\frac{a_k}{t}\right)^2$ is an increasing function and $\left(1-\frac{a_{k+1}}{t}\right)^2$ is a decreasing function of $t\in[a_k, a_{k+1}]$. Thus, the maximum is achieved when they are equal. In contrast to unbiased general rounding, it happens at the middle of the interval,
$$
t_* = \frac{a_k+a_{k+1}}{2} \in [a_k, a_{k+1}].
$$
Plugging $t_*$  into (\ref{eq:biased-rounding-nvar}), we get 
$$
\frac{(\cC(t_*)-t_*)^2}{t_*^2} = \left(\frac{a_{k+1}-a_k}{a_{k+1}+a_k}\right)^2.
$$

Given this, the parameter $\delta$ can be computed from
$$
1-\frac{1}{\delta} = \sup_{k\in\Z}\sup_{a_k\le t\le a_{k+1}} \frac{(\cC(t)-t)^2}{t^2} = \sup_{k\in\Z} \left(\frac{a_{k+1}-a_k}{a_{k+1}+a_k}\right)^2,
$$
which gives
$$
\delta = \sup_{k\in\Z}\frac{\left(a_k + a_{k+1}\right)^2}{4a_k a_{k+1}} \ge 1,
$$
and $\cC \in \mathbb{B}^3(\delta)$.

\subsection{Proof of Lemma \ref{lem-ex:ge-dithering}: General exponential dithering}

The proof goes with the same steps as in Theorem 4 of \cite{Cnat}. To show the unbiasedness of $\cC$, first we show the unbiasedness of $\xi(t)$ for $t\in[0,1]$ in the same way as (\ref{apx:umbiasedness-rounding}) was done. Then we note that
$$
\Exp{\cC(x)} = \sign(x) \times \|x\|_p \times \Exp{\xi\left(\frac{|x|}{\|x\|_p}\right)} = \sign(x) \times \|x\|_p \times \left(\frac{|x|}{\|x\|_p}\right) = x.
$$

To compute the parameter $\zeta$, we first estimate the second moment of $\xi$ as follows:
\begin{align*}
&\leq \mathbbm{1}\left(\frac{|x_i|}{\|x\|_p} \ge b^{1-s}\right) \cdot \frac{1}{4}\left(b+\frac{1}{b}+2\right) \frac{x_i^2}{\|x\|^2_p} + \mathbbm{1}\left(\frac{|x_i|}{\|x\|_p} < b^{1-s}\right) \cdot \frac{|x_i|}{\|x\|_p} b^{1-s} \\
&\leq \frac{1}{4}\left(b+\frac{1}{b}+2\right) \frac{x_i^2}{\|x\|^2_p} + \mathbbm{1}\left(\frac{|x_i|}{\|x\|_p} < b^{1-s}\right) \cdot \frac{|x_i|}{\|x\|_p} b^{1-s}\,.
\end{align*}

Then we use this bound to estimate the second moment of compressor $\cC$:
\begin{eqnarray*}
\Exp{\norm*{\cC(x)}^2}
&=& \|x\|_p^2 \sum_{i=1}^d \Exp{\xi\left(\frac{|x_i|}{\|x\|_p}\right)^2} \\
&\le & \|x\|_p^2 \sum_{i=1}^d \left( \frac{1}{4}\left(b+\frac{1}{b}+2\right) \frac{x_i^2}{\|x\|^2_p} + \mathbbm{1}\left(\frac{|x_i|}{\|x\|_p} < b^{1-s}\right) \cdot \frac{|x_i|}{\|x\|_p} b^{1-s} \right) \\
&=& \frac{1}{4}\left(b+\frac{1}{b}+2\right) \norm*{x}^2 + \sum_{i=1}^d \mathbbm{1}\left(\frac{|x_i|}{\|x\|_p} < b^{1-s}\right) \cdot |x_i| \|x\|_p b^{1-s} \\
&\le & \frac{1}{4}\left(b+\frac{1}{b}+2\right) \norm*{x}^2 + \min\left(\|x\|_1\|x\|_p b^{1-s},  d \|x\|_p^2 b^{2-2s}\right) \\
&\le &\frac{1}{4}\left(b+\frac{1}{b}+2\right) \norm*{x}^2 + \min\left(d^{\nicefrac{1}{2}}\|x\|_2\|x\|_p b^{1-s},  d \|x\|_p^2 b^{2-2s}\right) \\
&\le &\left[ \frac{1}{4}\left(b+\frac{1}{b}+2\right) + d^{\nicefrac{1}{r}} b^{1-s} \min\left(1,  d^{\nicefrac{1}{r}} b^{1-s}\right) \right] \norm*{x}^2 \\
&= &\zeta_b \norm*{x}^2,
\end{eqnarray*}
where $r = \min(p,2)$ and  H\"{o}lder's inequality is used to bound $\|x\|_p \le d^{\nicefrac{1}{p}-\nicefrac{1}{2}}\norm*{x}$ in case of $0\le p\le 2$ and $\|x\|_p \le \norm*{x}$ in the case $p\ge 2$.

\subsection{Proof of Lemma \ref{lem-ex:top-gendith}: Top-$k$ combined with exponential dithering}

From the unbiasedness of general dithering operator $\cC_{dith}$ we have
$$
\Exp{\cC(x)} = \Exp{\cC_{dith}(\cC_{top}(x))} = \cC_{top}(x),
$$
from which we conclude $\lin{\Exp{\cC(x)}, x} = \lin{\cC_{top}(x), x} = \norm*{\cC_{top}(x)}^2$. Next, using Lemma \ref{lem-ex:ge-dithering} on exponential dithering we get
\begin{equation*}
\Exp{\norm*{\cC(x)}^2} \leq \zeta_b \cdot \norm*{\cC_{top}(x)}^2 = \zeta_b \cdot \lin{\Exp{\cC(x)}, x},
\end{equation*}
which implies $\beta=\zeta_b$. Using Lemma \ref{lem-ex:top-sparse} we show $\gamma=\tfrac{k}{d}$ as $\lin{\Exp{\cC(x)}, x}  = \norm*{\cC_{top}(x)}^2 \ge \tfrac{k}{d}\norm*{x}^2$. Utilizing the derivations (\ref{apx:eq-it-10}) and (\ref{apx:eq-it-11}) it can be shown that $\Exp{\norm*{\cC_{dith}(x)}^2} \ge \norm*{x}^2$ and therefore
$$
\Exp{\norm*{\cC(x)}^2} \ge \norm*{\cC_{top}(x)}^2 \ge \tfrac{k}{d}\norm*{x}^2.
$$

Hence, $\alpha=\tfrac{k}{d}$. To compute the parameter $\delta$ we use Theorem \ref{thm:compression_properties}, which yields $\delta=\tfrac{\beta}{\gamma}=\tfrac{d}{k}\zeta_b$.

\section{Proofs for Section~\ref{sec:analysis_of_biased_GD}}

We now perform analysis of {\tt CGD} for compression operators in $\mathbb{B}^1$, $\mathbb{B}^2$ and $\mathbb{B}^3$, establishing Theorems~\ref{thm:main-I}, \ref{thm:main-II}
 and \ref{thm:main-III}, respectively.

\subsection{Analysis for $\cC\in \mathbb{B}^1(\alpha,\beta)$}

\begin{lemma} \label{lem:1} Assume $f$ is $L$-smooth. Let $\cC\in \mathbb{B}^1(\alpha,\beta)$. Then as long as $0\leq \stepsize \leq \frac{2}{\beta L}$, for each $x\in \R^d$ we have
\[\Exp{f\left(x-\stepsize \cC(\nabla f(x))\right)} \leq f(x) -  \alpha \stepsize \left(1 - \frac{\stepsize \beta  L}{2} \right)  \norm*{\nabla f(x)}^2.\]
\end{lemma}
\begin{proof}
Letting $g=\nabla f(x)$, we have\footnote{Alternatively, we can write
\begin{eqnarray*}
\Exp{f\left(x-\stepsize \cC(g)\right)} &\leq & 
 f(x) 
 - \stepsize \langle \Exp{\cC(g)}, g \rangle
  + \frac{\stepsize^2 L}{2} \Exp{\norm*{\cC(g)}^2}\\
&\overset{\eqref{eq:alpha-beta}}{\leq }& f(x) - \frac{\stepsize}{\beta} \Exp{\norm*{\cC(g)}^2}  
+ \frac{\stepsize^2 L}{2} \Exp{\norm*{\cC(g)}^2}\\
&=& f(x) - \frac{\stepsize}{\beta} \left( 1-  \frac{\stepsize \beta  L}{2} \right)\Exp{\norm*{\cC(g)}^2}\\
&\overset{\eqref{eq:alpha-beta}}{\leq }& f(x) -  \frac{\alpha}{ \beta} \stepsize \left(1 - \frac{\stepsize \beta  L}{2} \right)  \norm*{g}^2.
\end{eqnarray*}
Both approaches lead to the same bound.}
\begin{eqnarray*}
\Exp{f\left(x-\stepsize \cC(g)\right)} &\leq & \Exp{ f(x) + \left \langle g, -\stepsize \cC(g) \right \rangle + \frac{L}{2}\norm*{-\stepsize \cC(g) }^2} \\
&=& f(x) - \stepsize \langle \Exp{\cC(g)}, g \rangle + \frac{\stepsize^2 L}{2} \Exp{\norm*{\cC(g)}^2}\\
&\overset{\eqref{eq:alpha-beta}}{\leq }& f(x) - \stepsize \langle \Exp{\cC(g)}, g \rangle + \frac{ \stepsize^2 \beta L}{2} \left\langle \Exp{\cC(g)}, g \right\rangle \\
&=& f(x) - \stepsize \left(1 - \frac{\stepsize \beta  L}{2} \right) \langle \Exp{\cC(g)}, g \rangle  \\
&\overset{\eqref{eq:alpha-beta}}{\leq }& f(x) -  \frac{\alpha}{ \beta} \stepsize \left(1 - \frac{\stepsize \beta  L}{2} \right)  \norm*{g}^2.
\end{eqnarray*}
\end{proof}

\paragraph{Proof of Theorem~\ref{thm:main-I}}

\begin{proof}
Since $f$ is $\mu$-strongly convex, $\norm*{\nabla f(x^k)}^2 \geq 2 \mu (f(x^k)-f(x^\star))$. Combining this with Lemma~\ref{lem:1} applied to $x=x^k$ and $g=\nabla f(x^k)$, we get
\begin{align*}
&\Exp{f\left(x^k-\stepsize \cC(\nabla f(x^k))\right)}  - f(x^\star) \\ &\quad \leq f(x^k) - f(x^\star)-   \frac{\alpha}{\beta} \stepsize \mu \left(2 - \stepsize \beta  L \right)  (f(x^k)-f(x^\star))\\
&\quad= \left(1- \frac{\alpha}{\beta} \stepsize \mu \left(2 - \stepsize \beta  L \right)\right) (f(x^k)-f(x^\star)).
\end{align*}
\end{proof}

\subsection{Analysis for $\cC\in \mathbb{B}^2(\gamma,\beta)$}

\begin{lemma} \label{lem:1-II} Assume $f$ is $L$-smooth. Let $\cC\in \mathbb{B}^2(\gamma,\beta)$. Then as long as $0\leq \stepsize \leq \frac{2}{\beta L}$, for each $x\in \R^d$ we have
\[\Exp{f\left(x-\stepsize \cC(\nabla f(x))\right)} \leq f(x) -  \gamma \stepsize \left(1 - \frac{\stepsize \beta L}{2} \right)  \norm*{\nabla f(x)}^2.\]
\end{lemma}
\begin{proof}
Letting $g=\nabla f(x)$, we have
\begin{eqnarray*}
\Exp{f\left(x - \stepsize \cC(g)\right)} &\leq & \Exp{ f(x) + \dotprod{g}{ -\stepsize \cC(g) } + \frac{L}{2}\norm*{-\stepsize \cC(g) }^2} \\
&=& f(x) - \stepsize \dotprod{ \Exp{\cC(g)}}{ g } + \frac{\stepsize^2 L}{2} \Exp{\norm*{\cC(g)}^2}\\
&\overset{\eqref{eq:alpha-betaII}}{\leq }& f(x) - \stepsize \left(1 - \frac{\stepsize \beta   L}{2} \right) \dotprod{ \Exp{\cC(g)}}{ g} \\
&\overset{\eqref{eq:alpha-betaII}}{\leq }& f(x) -  \gamma \stepsize \left(1 - \frac{\stepsize \beta L}{2} \right)  \norm*{g}^2.
\end{eqnarray*}
\end{proof}

\paragraph{Proof of Theorem~\ref{thm:main-II}}
\begin{proof}
Since $f$ is $\mu$-strongly convex, $\norm*{\nabla f(x^k)}^2 \geq 2 \mu (f(x^k)-f(x^\star))$. Combining this with Lemma~\ref{lem:1-II} applied to $x=x^k$ and $g=\nabla f(x^k)$, we get
\begin{eqnarray*}
\Exp{f\left(x^k-\stepsize \cC(\nabla f(x^k))\right)}  - f(x^\star) &\leq &f(x^k) - f(x^\star)-  \mu \gamma \stepsize (2 - \stepsize \beta  L)  (f(x^k)-f(x^\star))\\
&=& \left(1- \mu \gamma \stepsize (2 - \stepsize \beta  L)\right) (f(x^k)-f(x^\star)).
\end{eqnarray*}

\end{proof}

\subsection{Analysis for $ \cC\in \mathbb{B}^3(\delta)$}

\begin{lemma} \label{lem:1-III} Assume $f$ is $L$-smooth. Let $ \cC\in \mathbb{B}^3(\delta)$. Then as long as $0\leq \stepsize \leq \frac{1}{L}$, for each $x\in \R^d$ we have
\[\Exp{f\left(x-\stepsize \cC(\nabla f(x))\right)} \leq f(x) - \frac{\stepsize}{2 \delta}\norm*{\nabla f(x)}^2.\]
\end{lemma}
\begin{proof}
Letting $g=\nabla f(x)$, note that for any stepsize $\stepsize\in \R$ we have
\begin{eqnarray}
\Exp{f\left(x-\stepsize \cC(g)\right)} &\leq & \Exp{ f(x) + \left \langle g, -\stepsize \cC(g) \right \rangle + \frac{L}{2}\norm*{-\stepsize \cC(g) }^2} \notag \\
&=& f(x) - \stepsize \langle \Exp{\cC(g)}, g \rangle + \frac{\stepsize^2 L}{2} \Exp{\norm*{\cC(g)}^2}.\label{eq:niugf7gh--bygTTr}
\end{eqnarray}
Since $\cC \in \mathbb{B}^3(\delta)$, we have 
$ \Exp{\norm*{ \cC(g)-g}^2} \leq   \left(1- \frac{1}{\delta} \right) \norm*{g}^2$. Expanding the square, we get
\[ \norm*{g}^2 - 2 \Exp{ \langle  \cC(g), g \rangle } +  \Exp{\norm*{\cC(g)}^2} \leq   \left(1- \frac{1}{\delta} \right)\norm*{g}^2 .\]
Subtracting $\norm*{g}^2$ from both sides, and multiplying both sides by $\frac{\stepsize}{2}$ (now we assume that $\stepsize>0$), we get 
\[- \stepsize \langle \Exp{\cC(g)}, g \rangle + \frac{\stepsize}{2} \Exp{\norm*{\cC(g)}^2} \leq  -\frac{\stepsize}{2\delta} \norm*{g}^2  .\]
Assuming that $\stepsize L \leq 1$, we can combine this with \eqref{eq:niugf7gh--bygTTr} and the lemma is proved.

\end{proof}

\paragraph{Proof of Theorem~\ref{thm:main-III} }
\begin{proof}
Since $f$ is $\mu$-strongly convex, $\norm*{\nabla f(x^k)}^2 \geq 2 \mu (f(x^k)-f(x^\star))$. Combining this with Lemma~\ref{lem:1-III} applied to $x=x^k$ and $g=\nabla f(x^k)$, we get
\begin{eqnarray*}
\Exp{f\left(x^k-\stepsize \cC( \nabla f(x^k))\right)}  - f(x^\star) &\leq &f(x^k) - f(x^\star)-   \frac{\stepsize \mu}{\delta}  (f(x^k)-f(x^\star)) \\
&=& \left(1-\frac{\stepsize \mu}{\delta} \right) (f(x^k)-f(x^\star)).
\end{eqnarray*}

\end{proof}

\section{Proofs for Section \ref{sec:stat}} \label{apx:stat}

\begin{proof}[Proof of Lemma \ref{lem:stat-top-random}]
{\bf (a)} 
As it was already mentioned, we have the following expressions for $\omega_{rnd}^k$ and $\omega_{top}^k$:
$$
\omega^k_{rnd}(x) = \left(1-\frac{k}{d}\right)\sum_{i=1}^{d} x_i^2, \quad \omega^k_{top}(x) = \sum_{i=1}^{d-k} x_{(i)}^2.
$$

The expected variance $\Exp{\omega_{rnd}^k}$ for Rand-$k$ is easy to compute as all coordinates are independent and uniformly distributed on $[0,1]$:
\begin{equation}\label{apx:id-15}
\Exp{x_i^2} \equiv \int_{[0,1]^d} x_i^2\,dx = \int_0^1 x_i^2\,d x_i = \frac{1}{3},
\end{equation}
which implies
\begin{equation}\label{apx:id-12}
\Exp{\omega_{rnd}^k(x)} = \left(1-\frac{k}{d}\right)\sum_{i=1}^{d} \Exp{x_i^2} = \left(1-\frac{k}{d}\right)\frac{d}{3} = \frac{d-k}{3}.
\end{equation}

In order to compute the expected variance $\Exp{\omega_{top}^k}$ for Top-$k$, we use the following formula from order statistics\footnote{see \url{https://en.wikipedia.org/wiki/Order_statistic}, \url{https://www.sciencedirect.com/science/article/pii/S0167715212001940}} (see e.g. \cite{arnold})
\begin{equation}\label{apx:id-14}
\Exp{x^2_{(i)}} \equiv \int_{[0,1]^d} x_{(i)}^2\,dx = \frac{\Gamma(i+2)\Gamma(d+1)}{\Gamma(i)\Gamma(d+3)} = \frac{i(i+1)}{(d+1)(d+2)},
\end{equation}
from which we derive
\begin{align}\label{apx:id-13}
\begin{split}
\Exp{\omega_{top}^k}
&= \sum_{i=1}^{d-k} \Exp{x^2_{(i)}} = \frac{1}{(d+1)(d+2)} \sum_{i=1}^{d-k} i(i+1) \\
&= \frac{1}{(d+1)(d+2)} \cdot \frac{(d-k)(d-k+1)(d-k+2)}{3} \\
&= \frac{d-k}{3}\left(1-\frac{k}{d+1}\right)\left(1-\frac{k}{d+2}\right).
\end{split}
\end{align}

Combining (\ref{apx:id-12}) and (\ref{apx:id-13}) completes the first relation. Thus, on average (w.r.t. uniform distribution) Top-$k$ has roughly $\left(1-\nicefrac{k}{d}\right)^2$ times less variance than Rand-$k$.

For the second relation, we use (\ref{apx:id-15}) and (\ref{apx:id-14}) for $i=d$ and get
\begin{equation*}
\frac{\Exp{s^1_{top}(x)}}{\Exp{s^1_{rnd}(x)}} = \frac{\Exp{x^2_{(d)}}}{\Exp{x_d^2}} = \frac{\tfrac{d(d+1)}{(d+1)(d+2)}}{\tfrac{1}{3}} = \frac{3 d}{d+2}.
\end{equation*}
Clearly, one can extend this for any $k\in[d]$.

{\bf (b)} 
Recall that for the standard exponential distribution (with $\lambda=1$) probability density function (PDF) is given as follows:
\begin{equation*}\label{apx:st-exp-dist}
\phi(t) = e^{-t}, \quad t\in[0,\infty).
\end{equation*}
Both mean and variance can be shown to be equal to $1$. The expected saving $\Exp{s^1_{rnd}}$ can be computed directly:
\begin{equation*}\label{rexp}
\Exp{s^1_{rnd}(x)} = \Exp{x_d^2} = \Var{[x_d]} + \Exp{x_d}^2  = 2.
\end{equation*}

To compute the expected saving $\Exp{s^1_{top}(x)} = \Exp{x^2_{(d)}}$ we prove the following lemma:

\begin{lemma}
Let $x_1,\,x_2,\,\dots,\,x_d$ be an i.i.d.\ sample from the standard exponential distribution and
$$y_i \eqdef (d-i+1)(x_{(i)} - x_{(i-1)}), \quad  1\le i\le d,$$
where $x_{(0)} \eqdef 0$. Then $y_1,\,y_2,\,\dots,\,y_d$ is an i.i.d.\ sample from the standard exponential distribution.
\end{lemma}
\begin{proof}
    The joint density function of $x_{(1)}, \ldots, x_{(d)}$ is given by (see \cite{arnold})
    \begin{equation*}
        \phi_{x_{(1)}, \dots, x_{(d)}} (u_1, \dots, u_d) = d! \prod\limits_{i = 1}^d \phi(u_i) = d! \exp{\left(-\sum\limits_{i=1}^d u_i\right)}, \quad 0 \leq u_1 \le \ldots \le u_d < \infty.
    \end{equation*}
    Next we express variables $x_{(i)}$ using new variables $y_i$
    \begin{equation*}
        x_{(1)} = \frac{y_1}{d},\; x_{(2)} = \frac{y_1}{d} + \frac{y_2}{d-1},\; \ldots,\; x_{(d)}=\frac{y_1}{d} + \frac{y_2}{d-1} + \ldots + y_d,
    \end{equation*}
    with the transformation matrix
    \begin{equation*}
    A = 
    \begin{pmatrix}
    \frac{1}{d} & 0 & \ldots & 0\\
    \frac{1}{d} & \frac{1}{d-1} & \ldots & 0\\
    \vdots & \vdots & \ddots & \vdots\\
    \frac{1}{d} & \frac{1}{d-1} & \ldots & 1
    \end{pmatrix}
    \end{equation*}
    Then the joint density $\psi_{y_1, \ldots, y_d} (u) = \psi_{y_1, \ldots, y_d} (u_1,\dots,u_d)$ of new variables $y_1, \ldots, y_d$ is given as follows
    \begin{equation*}
        \psi_{y_1, \ldots, y_d} (u) = \frac{\phi_{x_{(1)}, \ldots, x_{(d)}} (Au)}{\abs{\textrm{det}\,A^{-1}}} = \abs{\textrm{det}\,A}\cdot \phi_{x_{(1)}, \ldots, x_{(d)}} (Au)
    \end{equation*}
    Notice that $\sum\limits_{i=1}^d u_i = \sum\limits_{i=1}^d (Au)_i$ and $|\textrm{det}\, A| = \nicefrac{1}{d!}$. Hence
    \begin{equation*}
        \psi_{y_1, \ldots, y_d} (u) = \exp{\left(-\sum\limits_{i=1}^d u_i\right)}, \quad 0 \leq u_1 \le \ldots \le u_d \le \infty,
    \end{equation*}
    which means that variables $y_1, \ldots y_d$ are independent and have standard exponential distribution.
\end{proof}

Using this lemma we can compute the mean and the second moment of $x_{(d)} = \sum_{i=1}^d \frac{y_i}{d-i+1}$ as follows
\begin{align*}
    \Exp{x_{(d)}} &= \sum_{i=1}^d \Exp{\frac{y_i}{d-i+1}} = \sum_{i=1}^d \frac{\Exp{y_i}}{d-i+1} = \sum\limits_{i=1}^d \frac{1}{i}, \\
    \Var{[x_{(d)}]} &= \sum_{i=1}^d \Var\left[{\frac{y_i}{d-i+1}}\right] = \sum_{i=1}^d \frac{\Var{[y_i]}}{(d-i+1)^2} = \sum\limits_{i=1}^d \frac{1}{i^2},
\end{align*}
from which we conclude the lemma as
\begin{equation*}
    \Exp{s^1_{top}(x)} = \Exp{x^2_{(d)}} = \Var{[x_{(d)}]} + \Exp{x_{(d)}}^2 = \sum\limits_{i=1}^d \frac{1}{i^2} + \left( \sum\limits_{i=1}^d \frac{1}{i}\right)^2 \approx \cO(\log^2 d).
\end{equation*}

\end{proof}



\subsection{Proof of Theorem \ref{thm:sparsified} (main)}

In this section, we include our analysis for the Distributed \texttt{SGD} with biased compression. Our analysis is closely related to the analysis of \cite{stich2020error}.

We start with the definition of some auxiliary objects:
\begin{definition}
    The sequence $\{a^k\}_{k\geq0}$ of positive values is $\tau$-slow decreasing for parameter $\tau$:
    \begin{eqnarray}
    \label{decrease}
    a^{k+1} \leq a^{k}, \quad a^{k+1}\left(1 + \frac{1}{2\tau} \right) \geq a^k, \quad \forall k \geq 0
    \end{eqnarray}
    The sequence $\{a^k\}_{k\geq0}$ of positive values is $\tau$-slow increasing for parameter $\tau$:
    \begin{eqnarray}
    \label{increase}
    a^{k+1} \geq a^{k}, \quad a^{k+1} \leq a^k\left(1 + \frac{1}{2\tau} \right), \quad \forall k \geq 0
    \end{eqnarray}
\end{definition}
And let:
\begin{equation}
    \label{tx}
    \tilde x^k = x^k - \frac{1}{n} \sum\limits_{i=1}^n e_i^k \,, \quad \forall k \geq 0
\end{equation}
\begin{equation}
    \label{gt}
 g^k = \frac{1}{n} \sum\limits_{i=1}^n g_i^k.
\end{equation}
It is easy to see:
\begin{eqnarray}
 \label{eq:tilde}
 \tilde x^{k+1} &=& x^{k+1} - \frac{1}{n} \sum\limits_{i=1}^n e_i^{k+1} \nonumber\\
 &\stackrel{\eqref{error}, \eqref{step}}{=} &
 \left(x^k - \frac{1}{n} \sum\limits_{i=1}^n \tilde g_i^k \right) - \left(\frac{1}{n} \sum\limits_{i=1}^n [e_i^k + \stepsize^k g_i^k- \tilde g_i^k ]\right) \nonumber\\
 &=& \tilde x^k - \frac{\stepsize^k}{n} \sum\limits_{i=1}^n  g_i^k.
\end{eqnarray}
\begin{lemma}
\label{lemma:main}
If $\stepsize^k \leq \frac{1}{4L\left(1 + \nicefrac{2B}{n}\right)}$, $\forall k \geq 0$, then for $\{\tilde x^k\}_{k \geq 0}$ defined as in~\eqref{tx},
\begin{eqnarray}
\label{eq:main}
  \Exp{\norm*{\tilde x^{k+1} - x^*}^2}  &\leq&
 \left(1-\frac{\mu \stepsize^k}{2}\right) \Exp{\norm*{\tilde x^{k} - x^*}^2} 
 - \frac{\stepsize^k}{2} \Exp{f(x^k)- f^*} \nonumber\\
  & & ~ + ~3 L \stepsize^k   \Exp{\norm*{x^k - \tilde x^k}^2} + (\stepsize^k)^2\frac{C+2BD}{n}.
\end{eqnarray}

\end{lemma}
\begin{proof}
We consider the following equalities, using the relationship between  $\tilde x_{k+1}$ and $\tilde x_k$:
\begin{eqnarray*}
  \norm*{\tilde x^{k+1} - x^*}^2 
  &\stackrel{\eqref{gt},\eqref{eq:tilde}}{=}& \norm*{\tilde x^{k} - x^*}^2 - 2\stepsize^k \lin{g^k, \tilde x^k-x^*}+ (\stepsize^k)^2 \norm*{g^k}^2 \\
  &=& \norm*{\tilde x^{k} - x^*}^2 - 2\stepsize^k \lin{g^k, x^k-x^*}+ (\stepsize^k)^2 \norm*{g^k}^2  + 2\stepsize^k\lin{g^k, x^k - \tilde x^k}.
\end{eqnarray*}
Taking the conditional expectation conditioned on previous iterates, we get
\begin{eqnarray*}
  && \Exp{\norm*{\tilde x^{k+1} - x^*}^2} \\
  &=& \norm*{\tilde x^{k} - x^*}^2 - 2\stepsize^k \lin{\Exp{g^k}, x^k-x^*} + (\stepsize^k)^2 \cdot \Exp{\norm*{g^k}^2} + 2\stepsize^k\lin{ \Exp{g^k}, x^k - \tilde x^k}  \nonumber\\
  &\stackrel{\eqref{stgr1_main},\eqref{gt}}{=}& \norm*{\tilde x^{k} - x^*}^2 - 2\stepsize^k \lin{ \Exp{g^k}, x^k-x^*}  \nonumber\\
  && \quad  + (\stepsize^k)^2 \cdot \Exp{\norm*{\nabla f(x^k) + \frac{1}{n} \sum\limits_{i=1}^n \xi_i^k }^2} + 2\stepsize^k \lin{\Exp{g^k}, x^k - \tilde x^k}  \nonumber\\
  &=& \norm*{\tilde x^{k} - x^*}^2 - 2\stepsize^k \lin{ \Exp{g^k}, x^k-x^*}  + 2\stepsize^k \lin{\Exp{g^k}, x^k - \tilde x^k} \nonumber\\
  &&\quad + (\stepsize^k)^2 \cdot \Exp{\norm*{\nabla f(x^k)}^2 + 2 \lin{\nabla f(x^k), \frac{1}{n} \sum\limits_{i=1}^n \xi_i^k} + \norm*{\frac{1}{n} \sum\limits_{i=1}^n \xi_i^k }^2}.
\end{eqnarray*}
Given the unbiased stochastic gradient ($\Exp{\xi_i^k} = 0$):
\begin{eqnarray*}  
  \Exp{\norm*{\tilde x^{k+1} - x^*}^2} &\stackrel{}{=}& \norm*{\tilde x^{k} - x^*}^2 - 2\stepsize^k \lin{\nabla f(x^k), x^k-x^*} + 2\stepsize^k \lin{\nabla f(x^k), x^k - \tilde x^k}\nonumber\\
  &&+ (\stepsize^k)^2 \norm*{\nabla f(x^k)}^2 + (\stepsize^k)^2 \cdot \Exp{\norm*{\frac{1}{n} \sum\limits_{i=1}^n \xi_i^k }^2}. 
\end{eqnarray*}

Using that $\xi_i^k$ mutually independent and $\Exp{\xi_i^k} = 0$ we have:
\begin{eqnarray} 
&&\Exp{\norm*{\tilde x^{k+1} - x^*}^2} \\
  &\stackrel{\eqref{eq:sum_upper}}{\leq}&  \norm*{\tilde x^{k} - x^*}^2 - 2\stepsize^k \lin{\nabla f(x^k), x^k-x^*} \nonumber \\
  &&\quad+ (\stepsize^k)^2 \cdot \norm*{\nabla f(x^k)}^2 + (\stepsize^k)^2 \cdot \frac{1}{n^2}\sum\limits_{i=1}^n\Exp{\norm*{\xi_i^k}^2} +2\stepsize^k \lin{\nabla f(x^k), x^k - \tilde x^k} \nonumber \\
  &\stackrel{\eqref{stgr3_main}}{\leq}&  \norm*{\tilde x^{k} - x^*}^2 - 2\stepsize^k \lin{ \nabla f(x^k), x^k-x^*}  \nonumber\\
  &&\quad+ (\stepsize^k)^2 \cdot \norm*{\nabla f(x^k)}^2 + \frac{(\stepsize^k)^2}{n^2}\sum\limits_{i=1}^n\left[B\norm*{\nabla f_i(x^k)}^2\right]  + \frac{(\stepsize^k)^2}{n} C \nonumber\\
  &&\quad+ 2\stepsize^k \lin{\nabla f(x^k), x^k - \tilde x^k} \nonumber\\
  &\stackrel{\eqref{L-smooth4}}{\leq}&  \norm*{\tilde x^{k} - x^*}^2 - 2\stepsize^k \lin{ \nabla f(x^k), x^k-x^*}  \nonumber\\
  &&\quad+ (\stepsize^k)^2 \cdot 2L(f(x^k) - f(x^*)) + \frac{(\stepsize^k)^2}{n^2}\sum\limits_{i=1}^n\left[B\norm*{\nabla f_i(x^k)}^2\right]  + \frac{(\stepsize^k)^2}{n} C \nonumber\\
  &&\quad+ 2\stepsize^k \lin{\nabla f(x^k), x^k - \tilde x^k}.
  \label{long1}
\end{eqnarray}

All $f_i$ are $L$-smooth and $ \mu $-strongly convex, thus $f$ is $L$-smooth and $\mu$-strongly convex. We can rewrite $\frac{1}{n}\sum\limits_{i=1}^n\norm*{\nabla f_i(x^k)}^2$:
\begin{eqnarray*}
    \frac{1}{n}\sum\limits_{i=1}^n\norm*{\nabla f_i(x^k)}^2
    &=& \frac{1}{n}\sum\limits_{i=1}^n\norm*{\nabla f_i(x^k) - \nabla f_i(x_*) + \nabla f_i(x^*)}^2  \nonumber\\ 
    &\stackrel{\eqref{eq:sum_upper}}{\leq}&  \frac{2}{n}\sum\limits_{i=1}^n \left(\norm*{\nabla f_i(x^k) - \nabla f_i(x^*)}^2 + \norm*{\nabla f_i(x^*)}^2\right) \nonumber\\ 
    &\stackrel{\eqref{L-smooth3}}{\leq}&  \frac{2}{n}\sum\limits_{i=1}^n \left[ 2L \left(f_i(x^k) - f_i(x^*) -\langle\nabla f_i(x^*), x^k - x^*\rangle \right) + \norm*{\nabla f_i(x^*)}^2\right] .
\end{eqnarray*}
Using definition of $D = \frac{1}{n} \sum_{i=1}^n \norm*{\nabla f_i(x^*)}^2$:
\begin{eqnarray}
    \frac{1}{n}\sum\limits_{i=1}^n\norm*{\nabla f_i(x^k)}^2 &\stackrel{}{\leq}&   4L \left(f(x^k) - \nabla f(x^*)\right) + 2D.
    \label{sum_sqr_grad}
\end{eqnarray}

Substituting \eqref{sum_sqr_grad} to \eqref{long1}:
\begin{align}
  \Exp{\norm*{\tilde x^{k+1} - x^*}^2}
  &=  \norm*{\tilde x^{k} - x^*}^2 - 2\stepsize^k \lin{ \nabla f(x^k), x^k-x^*} \nonumber \\
  & \quad + (\stepsize^k)^2\cdot 2L\left(1 + \frac{2B}{n}\right)(f(x^k) - f(x^*)) \nonumber\\
  & \quad   + (\stepsize^k)^2\frac{C+2BD}{n}  + 2\stepsize^k \lin{\nabla f(x^k), x^k - \tilde x^k}.
  \label{long2}
\end{align}
By \eqref{eq:def_strongly_convex} we have for $f$:
\begin{eqnarray}
  -2\lin{\nabla f(x^k), x^k-x^*} \leq - \mu \norm*{x^k- x^*}^2  - 2(f(x^k)-f^*).
  \label{temp1_lem1}
\end{eqnarray}
Using \eqref{eq:triangle} with $\xi=\nicefrac{1}{2L}$ and $L$-smothness of $f$ \eqref{L-smooth4}:
\begin{align}
2\lin{ \nabla f(x^k), \tilde x^k - x^k}  &\leq \frac{1}{2L} \norm*{\nabla f(x^k)}^2 + 2L\norm*{x^k -\tilde x^k}^2 \notag \\
&\leq f(x^k)-f^* + 2L \norm*{x^k -\tilde x^k}^2 .
\label{temp2_lem1}
\end{align}
By \eqref{eq:sum_upper} for $\norm*{\tilde x^k - x^*}^2$, we get:
\begin{eqnarray}
 -\norm*{x^k - x^*}^2 \leq - \frac{1}{2} \norm*{\tilde x^k - x^*}^2 + \norm*{x^k - \tilde x^k}^2.
 \label{temp3_lem1}
\end{eqnarray}
Plugging \eqref{temp1_lem1}, \eqref{temp2_lem1}, \eqref{temp3_lem1} into \eqref{long2}:
\begin{eqnarray*}
 \norm*{\tilde x^{k+1} - x^*}^2
  &\leq& \left(1-\frac{\mu \stepsize^k}{2}\right) \norm*{\tilde x^{k} - x^*}^2 - \stepsize^k \left[1 - \stepsize^k\cdot 2L\left(1 + \frac{2B}{n}\right)\right] (f(x^k)- f^*)  \\
  && + \stepsize^k (2L+\mu) \norm*{x^k - \tilde x^k}^2 + (\stepsize^k)^2\frac{C+2BD}{n}.
\end{eqnarray*}
The lemma follows by the choice $\stepsize^k \leq \frac{1}{4L\left(1 + \nicefrac{2B}{n}\right)}$ and $L \geq \mu$.
\end{proof}
\begin{lemma} \label{lemma:sparse_final}
$\stepsize^k \leq \frac{1}{14 (2\delta + B) L}$, $\forall k \geq 0$ and $\{(\stepsize^k)^2\}_{k \geq 0}$ -- $2\delta$-slow decreasing. Then
\begin{eqnarray}
 \Exp{\norm*{\frac{1}{n}\sum\limits_{i=1}^n e_i^{k+1}}^2} &\leq& \frac{(1-1/\delta)}{49L (2\delta + B)} \sum_{j=0}^k\left[ \left(1-\frac{1}{4\delta}\right)^{k-j} (f(x^j) - f(x^*))\right] \notag \\
 &&\qquad + \stepsize^k \frac{2(\delta-1)}{7L}\left(2D + \frac{C}{2\delta + B} \right) \,. \label{eq:sparse_bound1}
\end{eqnarray}
Furthermore, for any $4\delta$-slow increasing non-negative sequence $\{w^k\}_{k \geq 0}$ it holds:
\begin{align}
    3L \cdot \sum\limits_{k=0}^K w^k \cdot  \Exp{\norm*{\frac{1}{n}\sum\limits_{i=1}^n e_i^{k}}^2} 
    &\leq \frac{1}{4} \sum\limits_{k=0}^K w^k (\Exp{f(x^k)} - f(x_*)) \notag \\
    & \quad +  \left(3\delta D + \frac{3C}{4}\right) \sum\limits_{k=0}^K w^k\stepsize^k. \label{eq:sparse_bound3}
\end{align}
\end{lemma} 
\begin{proof}

We prove the first part of the statement:
\begin{eqnarray*}
\Exp{\norm*{\frac{1}{n}\sum\limits_{i=1}^n e_i^{k+1}}^2}  &\stackrel{\eqref{eq:sum_upper}}{\leq}& \frac{1}{n} \Exp{\sum\limits_{i=1}^n \norm*{e_i^{k+1}}^2} \\ &\stackrel{\eqref{error}}{=}&\frac{1}{n}\Exp{\sum\limits_{i=1}^n \norm*{e_i^k + \stepsize^k g^k_i - \tilde g_i^k }^2}\nonumber \\ &\stackrel{\eqref{comp_grad}}{=}& \frac{1}{n} \sum\limits_{i=1}^n \Exp{\norm*{e_i^k + \stepsize^k g_i^k - \cC(e_i^k + \stepsize^k g^k_i)}^2} \nonumber \\
&\stackrel{\eqref{eq:quant}}{\leq}& \frac{1-1/\delta}{n} \sum\limits_{i=1}^n \Expg{\norm*{e_i^k + \stepsize^k g_i^k}^2} \\ 
&\stackrel{\eqref{stgr1_main}}{=}& \frac{1-1/\delta}{n} \sum\limits_{i=1}^n \Expg{\norm*{e_i^k + \stepsize^k\nabla f_i(x^k) + \stepsize^k \xi_i^k }^2}.
\end{eqnarray*}
Here we have taken into account that the operator of full expectation is a combination of operators of expectation by the randomness of the operator and the randomness of the stochastic gradient, i.e. $\Exp{\cdot} = \ExpC{\Expg{\cdot}}$.
Given the unbiased stochastic gradient ($\Exp{\xi_i^k} = 0$):
\begin{eqnarray*}
\Exp{\norm*{\frac{1}{n}\sum\limits_{i=1}^n e_i^{k+1}}^2} &\stackrel{}{\leq}&\frac{1-1/\delta}{n} \sum\limits_{i=1}^n \left[\norm*{e_i^k + \stepsize^k\nabla f_i(x^k)}^2 + \Expg{ \norm*{\stepsize^k \xi_i^k}^2}\right] \\
&\stackrel{\eqref{stgr3_main}}{\leq}&\frac{1-1/\delta}{n} \sum\limits_{i=1}^n \left[\norm*{e_i^k + \stepsize^k\nabla f_i(x^k)}^2 +(\stepsize^k)^2 \left(B \norm*{\nabla f_i(x^k)}^2 + C \right)\right]. \nonumber\\
\end{eqnarray*}
Using \eqref{eq:triangle} with some $\xi$:
\begin{eqnarray*}
&& \Exp{\norm*{\frac{1}{n}\sum\limits_{i=1}^n e_i^{k+1}}^2} \\
&\leq& \frac{1}{n} \Exp{\sum\limits_{i=1}^n \norm*{e_i^{k+1}}^2} \\
&\leq& \frac{1- \frac{1}{\delta}  }{n} \sum\limits_{i=1}^n  \left[(1+\xi) \norm*{e_i^k}^2 
+ (\stepsize^k)^2 \left(1+ \frac{1}{\xi} \right) \norm*{\nabla f_i(x^k)}^2
+(\stepsize^k)^2 B \norm*{\nabla f_i(x^k)}^2 + (\stepsize^k)^2 C \right]\\
&\stackrel{\eqref{sum_sqr_grad}}{\leq}& \left(1- \frac{1}{\delta} \right)  \left[(1+\xi) \left(\frac{1}{n} \sum\limits_{i=1}^n  \norm*{e_i^k}^2\right)\right] \\
&&\quad + \left(1- \frac{1}{\delta} \right)  \left[(\stepsize^k)^2 \left( 1+ \frac{1}{\xi} + B \right) \left( 4L(f(x^k) - f(x^*)) + 2D \right) + (\stepsize^k)^2 C\right].
\end{eqnarray*}
Using the recurrence for $\frac{1}{n} \sum\limits_{i=1}^n \norm*{e_i^k}^2$ , and let $\xi = \frac{1}{2(\delta - 1)}$, then  $(1+1/\xi) \leq 2\delta$, and $(1-1/\delta)(1+\xi) = (1-\nicefrac{1}{2\delta})$  we have
\begin{eqnarray*}
&&\Exp{\norm*{\frac{1}{n}\sum\limits_{i=1}^n e_i^{k+1}}^2} \\
&\leq& \frac{1}{n} \Exp{\sum\limits_{i=1}^n \norm*{e_i^{k+1}}^2} \\
&\leq&  \left(1- \frac{1}{\delta} \right) \sum_{j=0}^k (\stepsize^j)^2 \left[ \left(1- \frac{1}{\delta}\right)(1+\xi)  \right]^{k-j}   \left( 1+ \frac{1}{\xi} + B \right)  \left( 4L(\Exp{f(x^j)} - f(x^*)) + 2D \right) \\
&&\quad + \left(1- \frac{1}{\delta} \right) \sum_{j=0}^k (\stepsize^j)^2 \left[ \left(1- \frac{1}{\delta}\right)(1+\xi)  \right]^{k-j} C \\
 &\leq & \left(1- \frac{1}{\delta} \right)  \sum_{j=0}^k  (\stepsize^j)^2 \left(1-\frac{1}{2\delta}\right)^{k-j}  \left( \left( 2\delta + B \right) \left( 4L(\Exp{f(x^j)} - f(x^*)) + 2D \right) + C \right)\,.
\end{eqnarray*} 
For $2\delta$-slow decreasing  $\{(\stepsize^k)^2\}_{k \geq 0}$ by definition \eqref{decrease}  we get that $(\stepsize^{j})^2\leq (\stepsize^k)^2 \left(1+\frac{1}{4\delta} \right)^{k-j}$. Due to the fact that $(1-\nicefrac{1}{2\delta})(1+\nicefrac{1}{4\delta})\leq (1-\nicefrac{1}{4\delta})$, we have:
\begin{eqnarray*}
&&\Exp{\norm*{\frac{1}{n}\sum\limits_{i=1}^n e_i^{k+1}}^2} \\
&\leq& \frac{1}{n} \Exp{\sum\limits_{i=1}^n \norm*{e_i^{k+1}}^2} \\
&\leq &  \left(1- \frac{1}{\delta} \right)  \sum_{j=0}^k (\stepsize^k)^2\left(1+\frac{1}{4\delta}\right)^{k-j} \left(1-\frac{1}{2\delta}\right)^{k-j}  \left( 2\delta + B \right) \left( 4L(\Exp{f(x^j)} - f(x^*)) + 2D \right) \\
&&\quad + \left(1- \frac{1}{\delta} \right)  \sum_{j=0}^k (\stepsize^k)^2\left(1+\frac{1}{4\delta}\right)^{k-j} \left(1-\frac{1}{2\delta}\right)^{k-j} C \\
&\leq& (\stepsize^k)^2 \left(1- \frac{1}{\delta} \right)   \left( 2\delta + B \right) \sum_{j=0}^k\left[ \left(1-\frac{1}{4\delta}\right)^{k-j}4L\left( \Exp{f(x^j)} - f(x^*)\right) \right]\\
&+& (\stepsize^k)^2 \left(1- \frac{1}{\delta} \right)  4\delta [C + 2D(2\delta + B)] \,.
\end{eqnarray*}
As the last step, we use formula for geometric progression in the following way: $$\sum\limits_{j=0}^k \left(1 - \frac{1}{4\delta}\right)^{k-j} = \sum\limits_{j=0}^k \left(1 - \frac{1}{4\delta}\right)^{j} \leq \sum\limits_{j=0}^{\infty} \left(1 - \frac{1}{4\delta}\right)^{j} = 4\delta.$$

By observing that the choice of the stepsize $\stepsize^k \leq \frac{1}{14 (2\delta + B) L}$:
\begin{eqnarray*}
\Exp{\norm*{\frac{1}{n}\sum\limits_{i=1}^n e_i^{k+1}}^2} &\leq& \frac{1}{n} \Exp{\sum\limits_{i=1}^n \norm*{e_i^{k+1}}^2} \\
&\leq& \frac{(1-1/\delta)}{49L (2\delta + B)} \sum_{j=0}^k\left[ \left(1-\frac{1}{4\delta}\right)^{k-j}(\Exp{f(x^j)} - f(x^*)) \right] \\
&&\qquad + \stepsize^k \frac{2(\delta-1)}{7 L} \left(2D + \frac{C}{2\delta + B}\right) \,,
\end{eqnarray*}
which concludes the proof of \eqref{eq:sparse_bound1}. For the second part, we use the previous results. Summing over all $k$:
\begin{align*}
 &\sum_{k=0}^K w^k \cdot \Exp{\norm*{\frac{1}{n}\sum\limits_{i=1}^n e_i^{k}}^2} \\ &\stackrel{\eqref{eq:sparse_bound1}}{\leq}  \frac{(1-1/\delta)}{49L (2\delta + B)} \sum_{k=0}^K w^k \sum_{j=0}^{k-1} \left(1-\frac{1}{4\delta}\right)^{k-j-1} \left(\Exp{f(x^j)} - f(x^*)\right) \\
 &\quad + \frac{2(\delta-1)}{7 L} \left(2D + \frac{C}{2\delta + B}\right) \sum_{k=0}^K w^k \stepsize^{k-1}.
\end{align*}
For $2\delta$-slow decreasing $\{(\stepsize^k)^2\}_{k \geq 0}$, it holds  $(\stepsize^{k-1})^2\leq (\stepsize^k)^2 \bigl(1+\frac{1}{4\delta} \bigr)$ which follows from \eqref{decrease} and $\stepsize^{k-1}\leq \stepsize^k \bigl(1+\frac{1}{4\delta} \bigr)$ and for $4\delta$-slow increasing $\{w^k\}_{k \geq 0}$ by \eqref{increase} we have $w^k \leq w^{k-j} \bigl(1+\frac{1}{8\delta}\bigr)^j$. Then
\begin{eqnarray*}
 &&\sum_{k=0}^K w^k \cdot \Exp{\norm*{\frac{1}{n}\sum\limits_{i=1}^n e_i^{k}}^2} \\ &\stackrel{\eqref{eq:sparse_bound1}}{\leq}&  \frac{(1-1/\delta)}{49L (2\delta + B)} \sum_{k=0}^K w^k  \sum_{j=0}^{k-1} \left(1-\frac{1}{4\delta}\right)^{k-j-1} \left(\Exp{f(x^j)} - f(x^*)\right) \\
 &&\quad + \frac{2(\delta-1)}{7 L} \left(2D + \frac{C}{2\delta + B} \right)\left(1+\frac{1}{4\delta}\right) \sum_{k=0}^K w^k \stepsize^{k} \\
 &\leq&  \frac{(1-1/\delta)}{49L (2\delta + B)} \sum_{k=0}^K  \sum_{j=0}^{k-1} w^{j} \left(1+\frac{1}{8\delta}\right)^{k-j} \left(1-\frac{1}{4\delta}\right)^{k-j} \left(\Exp{f(x^j)} - f(x^*)\right) \\
 &&\quad +  \frac{\delta-1}{2 L} \left(2D + \frac{C}{2\delta + B}\right) \sum\limits_{k=0}^K w^k \stepsize^k\\
 &\leq& \frac{(1-1/\delta)}{49L (2\delta + B)}  \sum_{k=0}^K\sum_{j=0}^{k-1} w_{j} \left(1-\frac{1}{8\delta}\right)^{k-j}  \left(\Exp{f(x^j)} - f(x^*)\right) \\
 &&\quad +  \frac{\delta-1}{2 L} \left(2D + \frac{C}{2\delta + B}\right) \sum\limits_{k=0}^K w^k \stepsize^k \\
 &\leq&  \frac{(1-1/\delta)}{49L (2\delta + B)} \sum_{k=0}^K w^k  \left(\Exp{f(x^k)} - f(x^*)\right) \sum_{j=0}^{\infty} \left(1-\frac{1}{8\delta}\right)^{j}  \\
 &&\quad +  \frac{\delta-1}{2 L} \left(2D + \frac{C}{2\delta + B}\right) \sum\limits_{k=0}^K w^k \stepsize^k\,.
\end{eqnarray*}
Observing $\sum_{j=0}^\infty (1-\nicefrac{1}{8\delta})^j \leq 8\delta$ and using $\nicefrac{\delta - 1}{2\delta + B} \leq \nicefrac{1}{2}$ concludes the proof.
\end{proof}

The proof of the main theorem follows

\begin{proof}[Proof of the Theorem \ref{thm:sparsified}]
It is easy to see that $\nicefrac{1}{14(2\delta + B)L} \leq \nicefrac{1}{4L(1 + 2B/n)}$. This means that the Lemma~\ref{lemma:main} is satisfied. With the notation $r^k \eqdef \Exp{\norm*{\tilde x^{k+1}-x^\star}^2}$ and $s^k \eqdef \Exp{f(x^k)}-f^\star$ we have for any $w^k > 0$:
\begin{eqnarray*}
 \frac{w^k}{2} s^k \stackrel{\eqref{eq:main}}{\leq} \frac{w^k}{\stepsize^k} \left(1-\frac{\mu \stepsize^k}{2}\right) r^k - \frac{w^k}{\stepsize^k} r^{k+1} + \stepsize^k w^k \frac{C+2BD}{n} + 3 w^k L \cdot \Exp{\norm*{\frac{1}{n}\sum\limits_{i=1}^n e_i^k}^2}\,.
\end{eqnarray*}
Substituting \eqref{eq:sparse_bound3} and summing over $k$ we have:
\begin{eqnarray*}
 \frac{1}{2} \sum_{k=0}^K w^k s^k \leq \sum_{k=0}^K \left( \frac{w^k}{\stepsize^k} \left(1-\frac{\mu \stepsize^k}{2}\right) r^k - \frac{w^{k}}{\stepsize^k} r^{k+1} + \stepsize^k w^k \tilde C \right) + \frac{1}{4}\sum_{k=0} ^K w^k s^k  \,.
\end{eqnarray*}
where $\tilde C = C\left(1 + \frac{1}{n}\right) + D \left(\frac{2B}{n} + 3\delta\right)$ .

This can be rewritten as
\begin{eqnarray*}
 \frac{1}{W^K} \sum_{k=0}^K w^k s^k \leq  \frac{4}{W^K} \sum_{k=0}^K \left( \frac{w^k}{\stepsize^k} \left(1-\frac{\mu \stepsize^k}{2}\right) r^k - \frac{w^{k}}{\stepsize^k} r^{k+1} +  \stepsize^k w^k \tilde C \right).
\end{eqnarray*}
First, when the stepsizes $\stepsize^k = \frac{4}{\mu (\kappa + k)}$, it is easy to see that $\stepsize^k \leq \frac{1}{14(2\delta + B) L}$:
$$\stepsize^k \leq \stepsize^0 = \frac{4}{\mu \kappa} \leq  \frac{4}{\mu} \cdot \frac{\mu}{56(2\delta +B)L} =\frac{1}{14(2\delta + B) L}.$$
Not difficult to check that $\{(\stepsize^k)^2\}_{k\geq 0}$ is $2\delta$ slow decreasing:
\begin{align*}
\frac{(\stepsize^{k+1})^2}{(\stepsize^k)^2}&=\left(\frac{\kappa + k +1}{\kappa + k}\right)^2 \leq \left(1 + \frac{1}{\kappa + k}\right)^2 \leq \left(1 + \frac{1}{\kappa} \right)^2 \\
&= \left(1 + \frac{\mu}{56 (2\delta + B)L}\right)^2 \leq 1 + \frac{1}{4\delta}.
\end{align*}
Furthermore, the weights $\{w^k = \kappa + k\}_{k \geq 0}$ are $4\delta$-slow increasing:
\begin{eqnarray*}
\frac{w^{k+1}}{w^k}=\frac{\kappa + k +1}{\kappa + k} = 1 + \frac{1}{\kappa + k} \leq 1 + \frac{1}{\kappa} = 1 + \frac{\mu}{56 (2\delta + B)L} \leq 1 + \frac{1}{8\delta}.
\end{eqnarray*}

The conditions for Lemma~\ref{lem:constant} are satisfied, and we obtain the desired statement. In case $\mu = 0$, we invoke Lemma~\ref{lemma:general}.
\end{proof}



\refstepcounter{chapter}%
\chapter*{\thechapter \quad Appendix: A better alternative to error feedback for communication-efficient distributed learning}
\label{appendix:induced}

\begin{figure}[t]
\center
\includegraphics[width=0.36\textwidth]{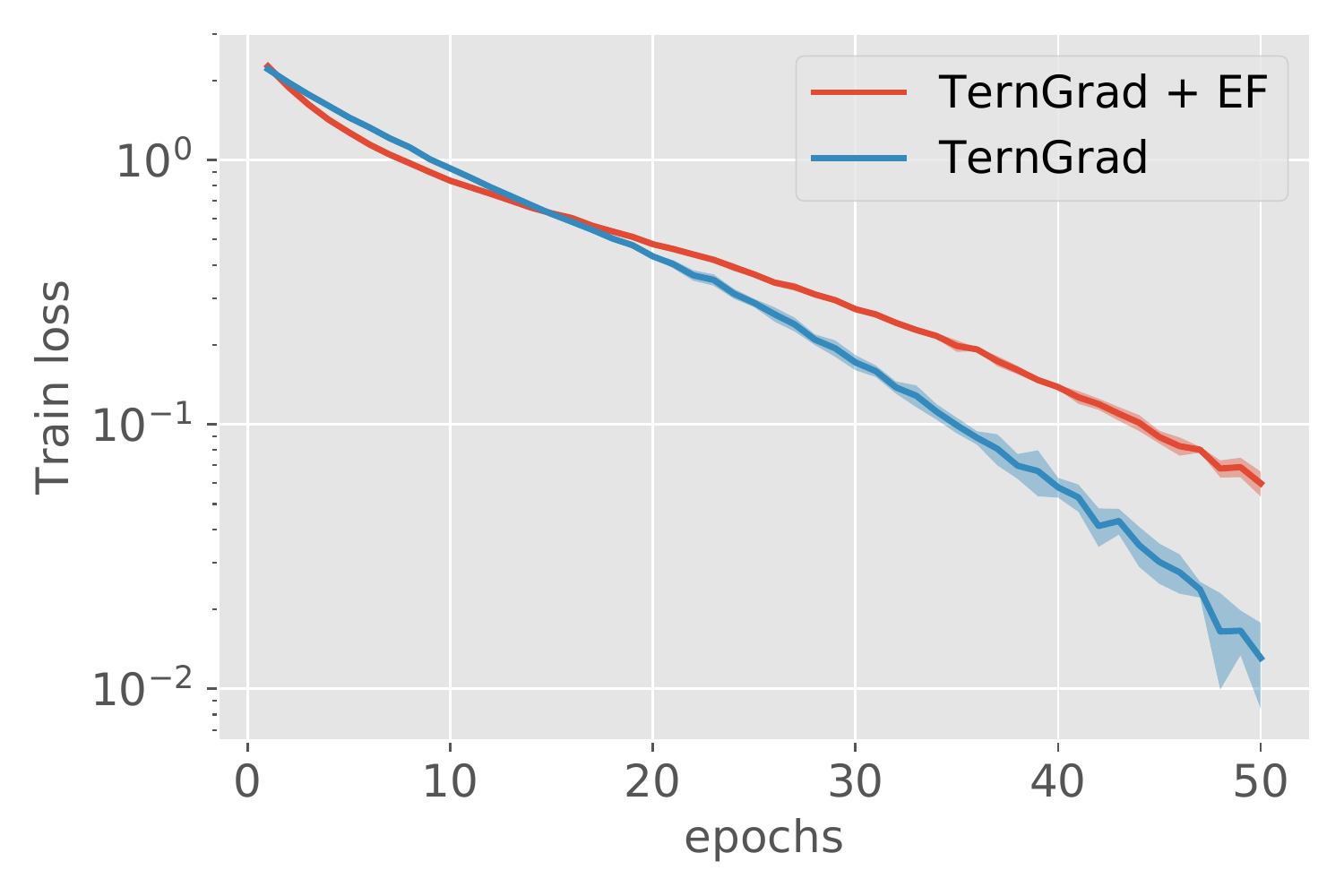}
\includegraphics[width=0.36\textwidth]{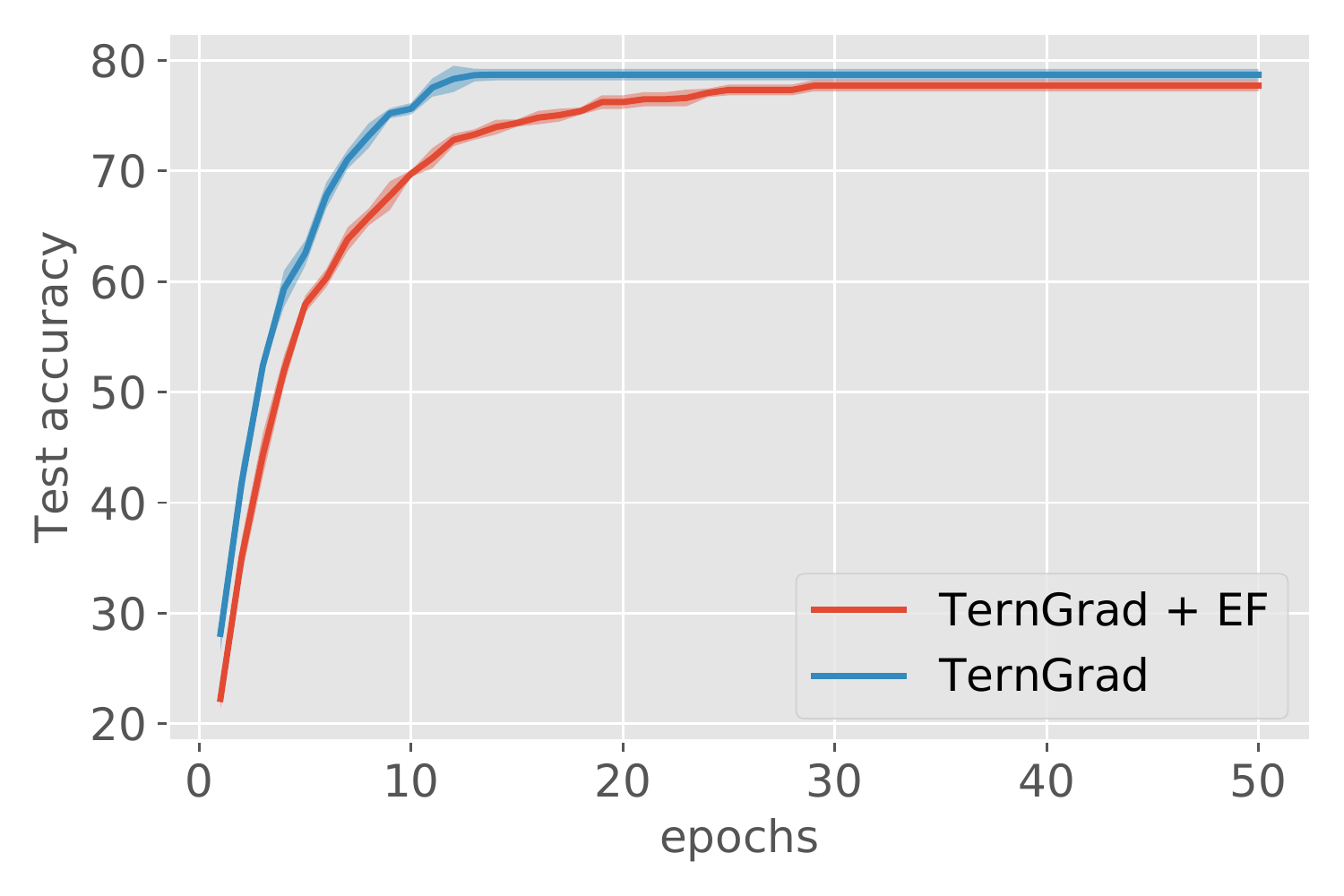} \\
\includegraphics[width=0.36\textwidth]{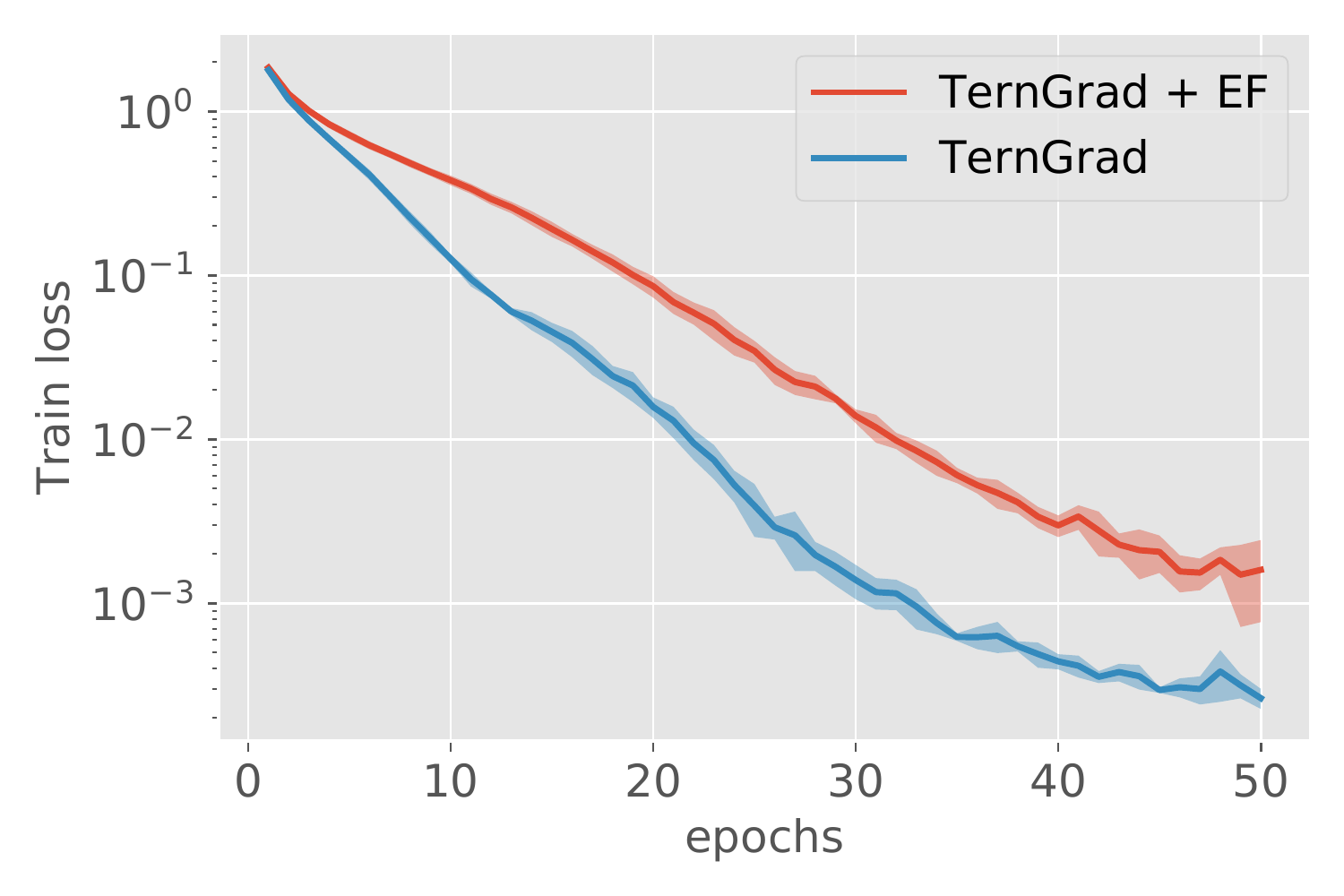}
\includegraphics[width=0.36\textwidth]{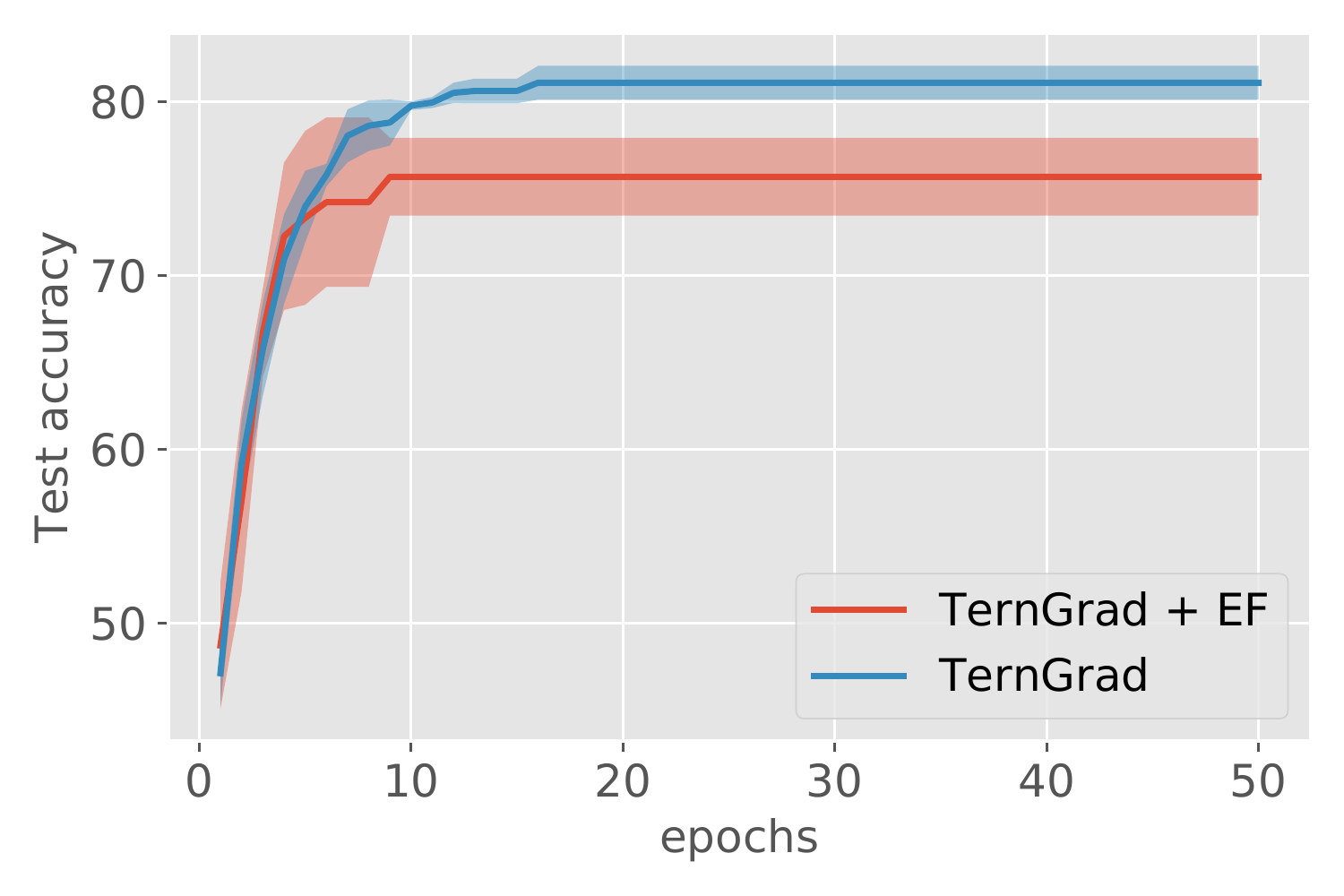}
\caption{Algorithm~\ref{alg:UC_SGD} vs. Algorithm~\ref{alg:EF_SGD} on CIFAR10 with ResNet18 (bottom), VGG11 (top) and TernGrad as a compression.}
\label{fig:exp_app}
\end{figure}

\section{Experimental details}
 To be fair, we always compare methods with the same communication complexity per iteration. We report the number of epochs (passes over the dataset) with respect to training loss and testing accuracy. The test accuracy is obtained by evaluating the best model in terms of validation accuracy. A validation accuracy is computed based on $10$ \%  randomly selected training data. We tune the step-size using based on the training loss. For every experiment, we randomly distributed the training dataset among $8$ workers; each worker computes its local gradient-based on its own dataset. We used a local batch size of $32$. All the provided figures display the mean performance with one standard error over $5$ independent runs. For a fair comparison, we use the same random seed for the compared methods. Our experimental results are based on a Python implementation of all the methods running in PyTorch. All reported quantities are independent of the system architecture and network bandwidth.

\textbf{Dataset and Models.} We do an evaluation on CIFAR10  dataset. We consider VGG11~\cite{simonyan2014very} and ResNet18~\cite{resnet} models and step-sizes $0.1, 0.05$ and $0.01$.

\subsection{Extra experiments}

\textbf{Momentum.}
In this extra experiment, we look at the effect of momentum on Algorithm~\ref{alg:UC_SGD} and~\ref{alg:EF_SGD}. We set momentum to $0.9$. Similarly to Figure~\ref{fig:exp1}, we work with the unbiased compressor, concretely \texttt{TernGrad}~\cite{terngrad} (coincides with \texttt{QSGD}~\cite{qsgd2017neurips} and natural dithering~\cite{Cnat} with the infinity norm and one level), to see the effect of adding Error Feedback. We can see from Figure~\ref{fig:exp_app} that adding Error Feedback can hurt the performance, which agrees with our theoretical findings. 

\section{Example 1, \cite{beznosikov2020biased}}
\label{app:counter-example}
In this section, we present example considered in~\cite{beznosikov2020biased}, which was used as a counterexample to show that some form of error correction is needed in order for biased compressors to work/provably converge.  In addition, we run experiments on their construction and show that while Error Feedback fixes divergence, it is still significantly dominated by unbiased non-uniform sparsification as can be seen in Figure~\ref{fig:exp4}. The construction follows.

Consider $n=d=3$ and define the following smooth and strongly convex quadratic functions
$$
f_1(x) = \dotprod{a}{x}^2 + \frac{1}{4}\norm*{x}^2, \qquad f_2(x) = \dotprod{b}{x}^2 + \frac{1}{4}\norm*{x}^2, \qquad f_3(x) = \dotprod{c}{x}^2 + \frac{1}{4}\norm*{x}^2,
$$
where $a=(-3,2,2), b=(2,-3,2), c=(2,2,-3)$. Then, with the initial point $x^0=(t,t,t),\;t>0$
$$
\nabla f_1(x^0) =\frac{t}{2} (-11, 9, 9), \qquad \nabla f_2(x^0) =\frac{t}{2} (9,-11, 9), \qquad \nabla f_3(x^0) = \frac{t}{2} (9,9,-11).
$$
Using the Top-$1$ compressor, we get $$\cC(\nabla f_1(x^0)) = \tfrac{t}{2}(-11, 0, 0) , \quad \cC(\nabla f_2(x^0)) = \tfrac{t}{2}(0,-11, 0), \quad \cC(\nabla f_3(x^0)) = \tfrac{t}{2}(0,0,-11).$$ The next iterate of DCGD is
$$
x^1 = x^0 -  \frac{\eta}{3} \sum_{i=1}^3 \cC(\nabla f_i(x^0)) =  \left(1+\frac{11 \eta}{6}\right)x^0.
$$
Repeated application gives $x^k = \left(1+\frac{11 \eta}{6}\right)^k x^0$, which diverges exponentially fast to $+\infty$ since $\eta>0$. 

As a initial point, we use $(1, 1, 1)^\top$ in our experiments and we choose step size $\frac{1}{L}$, where $L$ is smoothness parameter of $f = \frac{1}{3}(f_1 + f_2 + f_3)$. Note that zero vector is the unique minimizer of $f$.

\section{Proofs}

\subsection{Proof of Lemma~\ref{lem:subset}}
We follow \eqref{eq:omega_quant}, which holds for $\cC \in \U(\delta - 1)$.
\begin{eqnarray*}
\E{\norm*{\frac{1}{\delta}\cC^k(x) - x}^2} & = & \frac{1}{\delta^2}\E{\norm*{\cC^k(x)}^2} - 2 \frac{1}{\delta} \dotprod{\E{\cC^k(x)},x} + \norm*{x}^2 \\
& \leq & \lp \frac{1}{\delta} - \frac{2}{\delta}  + 1 \rp \norm*{x}^2 \\
& = & \lp 1 - \frac{1}{\delta}  \rp \norm*{x}^2,
\end{eqnarray*}
which concludes the proof.

\subsection{Proof of Theorem~\ref{thm:u_n}}
We use the update of Algorithm~\ref{alg:UC_SGD} to bound the following quantity
\begin{eqnarray*}
&&\E{\norm*{x^{k+1} - \xs}^2 |x^k} \\
&=& \norm*{x^k - \xs}^2 - \frac{\eta^k}{n} \sum_{i=1}^n\E{\dotprod{\cC^k(g_i^k)}{ x^k - \xs }|x^k} + \\
 && \quad \lp\frac{\eta^k}{n}\rp^2 \E{\norm*{\sum_{i=1}^n\cC^k(g_i^k)}^2 |x^k} \\
&\overset{\eqref{eq:omega_quant}}{\leq}& \norm*{x^k - \xs}^2 - \eta^k \dotprod{\nabla f(x^k)}{ x^k - \xs } +  \\
&&\quad \frac{(\eta^k)^2 }{n^2}  \E{ \sum_{i=1}^n\norm*{\cC^k(g_i^k) -  g_i^k}^2 + \norm*{ \sum_{i=1}^n g_i^k }^2|x^k} \\
&\overset{\eqref{eq:omega_quant} }{\leq}& \norm*{x^k - \xs}^2 - \eta^k \dotprod{\nabla f(x^k)}{ x^k - \xs } +  \\
&&\quad \frac{(\eta^k)^2 }{n^2}  \E{ (\delta - 1)\sum_{i=1}^n\norm*{\ g_i^k}^2 + \norm*{ \sum_{i=1}^n g_i^k }^2|x^k} \\
&\overset{\eqref{eq:L_smooth_f_i} + \eqref{eq:L_smooth_f}}{\leq}& \norm*{x^k - \xs}^2 - \eta^k \dotprod{\nabla f(x^k)}{ x^k - \xs } + \\
&&\quad 2L (\eta^k)^2  \lp \delta_n(f(x^k) - \fs) +  (\delta_n - 1)\frac{1}{n}\sum_{i=1}^n (f_i(\xs) - f _i^\star) \rp +    (\eta^k)^2 \frac{\delta\sigma^2}{n}\\
&\overset{\eqref{eq:quasi_convex}}{\leq}& (1 - \mu \eta^k) \norm*{x^k - \xs}^2 - 2\eta^k \lp 1 - \eta^k \delta_n L   \rp (f(x^k) - \fs) + \\
&&\quad  (\eta^k)^2 \lp(\delta_n - 1) D +  \frac{\delta\sigma^2}{n}\rp.
\end{eqnarray*}
Taking full expectation and $\eta^k \leq \frac{1}{2\delta_n L}$, we obtain
\begin{align*}
\E{\norm*{x^{k+1} - \xs}^2} &\leq (1 - \mu \eta^k) \E{\norm*{x^k - \xs}^2} - \eta^k  \E{f(x^k) - \fs} \\
&\quad +  (\eta^k)^2 \lp(\delta_n-1) D +  \frac{\delta\sigma^2}{n}\rp.
\end{align*}

The rest of the analysis is closely related to the one of \cite{stich2019unified}. We would like to point out that similar results to~\cite{stich2019unified} were also present in~\cite{lacoste2012simpler, stich2018sparsified, grimmer2019convergence}.

We first rewrite the previous inequality to the form
\begin{eqnarray}
\label{eq:sequence}
r^{k+1} \leq (1 - a \eta^k) r^k -  \eta^k  s^k +  (\eta^k)^2 c,
\end{eqnarray}
where $r^k = \E{\norm*{x^k - \xs}^2}$, $s^k = \E{f(x^k) - \fs}$, $a = \mu$, $c = (\delta_n -1) D +  \frac{\delta\sigma^2}{n}$.

We proceed with lemmas that establish a convergence guarantee for every recursion of type \eqref{eq:sequence}.

\begin{lemma}
\label{lem:first}
Let  $\{r^k\}_{k\geq 0}$, $\{s^k\}_{k\geq 0}$ be as in~\eqref{eq:sequence} for $a > 0$ and for constant stepsizes $\eta^k \equiv \eta \eqdef \frac{1}{d}$, $\forall k \geq 0$. Then it holds for all $T \geq 0$:
\begin{align*}
 r^{T} \leq r^0 \exp\left[- \frac{a T}{d} \right] + \frac{c}{ad}\,.
\end{align*}
\end{lemma}
\begin{proof}
This follows by relaxing~\eqref{eq:sequence} using $\E{f(x^k) - \fs} \geq 0 $,and unrolling the recursion
\begin{align*}
 r^{T} &\leq (1-a\eta)r^{T-1} + c \gamma^2 \leq (1-a\eta)^T r^0 +  c \eta^2 \sum_{k=0}^{T-1} (1-a\eta)^k \leq (1-a\eta)^T r^0 + \frac{c  \eta}{a}\,.
\end{align*}
\end{proof}

\begin{lemma}
\label{lem:lacoste}
Let  $\{r^k\}_{k\geq 0}$, $\{s^k\}_{k\geq 0}$ as in~\eqref{eq:sequence} for $a > 0$ and for decreasing stepsizes $\eta^k \eqdef \frac{2}{ a(\kappa + k)}$, $\forall k \geq 0$, with parameter $\kappa \eqdef \frac{2d}{a}$, and weights $w^k \eqdef (\kappa + k)$. Then
\begin{align*}
 \frac{1}{W^T} \sum_{k=0}^{T}s^k w^k + a r^{T+1} \leq \frac{2 a \kappa^2 r_0}{T^2} +  \frac{2c}{aT} \,,
\end{align*}
where $W^T \eqdef \sum_{k=0}^T w^k$.
\end{lemma}
\begin{proof}
We start by re-arranging \eqref{eq:sequence} and multiplying both sides with $w^k$
\begin{align*}
 s^k w^k &\leq \frac{w^k(1-a\eta^k)r^k}{\eta^k} - \frac{w^k r^{k+1}}{\eta^k} + c \eta^k w^k \\
&= a (\kappa + k) (\kappa + k-2) r^k -  a (\kappa + k)^2 r^{k+1} + \frac{c}{a} \\
&\leq  a(\kappa + k-1)^2 r^k -  a (\kappa + k)^2 r^{k+1} + \frac{c}{a} \,,
\end{align*}
where the equality follows from the definition of $\eta^k$ and $w^k$ and the inequality from $(\kappa + k)(\kappa + k-2) = (\kappa + k-1)^2 - 1 \leq  (\kappa + k-1)^2$. Again we have a telescoping sum:
\begin{align*}
 \frac{1}{W^T} \sum_{k=0}^T s^k w^k + \frac{ a(\kappa + T)^2 r^{T+1}}{W^T} \leq \frac{ a \kappa^2 r^0}{W^T} + \frac{c (T+1)}{a W^T}\,,
\end{align*}
with
\begin{itemize}
\item $W^T = \sum_{k=0}^T w^k = \sum_{k=0}^T (\kappa + k) = \frac{(2\kappa+T) (T+1)}{2} \geq \frac{T(T+1)}{2} \geq \frac{T^2}{2}$,
\item and $W^T =  \frac{(2\kappa+T) (T+1)}{2} \leq \frac{2(\kappa + T)(1+T)}{2} \leq (\kappa + T)^2$ for $\kappa = \frac{2d}{a} \geq 1$.
\end{itemize}
By applying these two estimates we conclude the proof.
\end{proof}

The convergence can be obtained as the combination of these two lemmas.

\begin{lemma}
\label{lem:sequence_induced_induced}
Let  $\{r^k\}_{k\geq 0}$, $\{s^k\}_{k\geq 0}$ as in~\eqref{eq:sequence}, $a > 0$. Then there exists stepsizes $\eta^k \leq \frac{1}{ d}$ and weighs $w^k \geq 0$, $W^T \eqdef \sum_{k=0}^T w^k$, such that
\begin{align*}
 \frac{1}{W^T} \sum_{k=0}^{T}s^k w^k + a r^{T+1} \leq   32 d r_0 \exp \left[-\frac{a T}{2 d} \right] + \frac{36c}{a T}\,.
\end{align*}
\end{lemma}

\begin{proof}[Proof of Lemma~\ref{lem:sequence_induced_induced}]
For integer $T \geq 0$, we choose stepsizes and weights as follows
\begin{align*}
&\text{if $T \leq \frac{d}{a}$}\,, & \eta^k &= \frac{1}{ d}\,, & w^k &= (1-a\eta^k)^{-(k+1)} = \left(1-\frac{a}{d}\right)^{-(k+1)}, \\
&\text{if $T > \frac{d}{a}$ and $k < t_0$}, & \eta^k &= \frac{1}{ d}\,, & w^k &= 0\,, \\
& \text{ if $T > \frac{d}{a}$ and $k \geq t_0$}, & \eta^k &= \frac{2}{ a(\kappa + k-t_0)}\,,   & w^k &= (\kappa + k-t_0)^2\,,
\end{align*}
for $\kappa = \frac{2d}{a}$ and $t_0 =  \bigl\lceil \frac{T}{2} \bigr\rceil$. We will now show that these choices imply the claimed result.

\noindent We start with the case $T \leq \frac{d}{a}$. For this case, the choice $\eta = \frac{1}{ d}$ gives
\begin{align*}
\frac{1}{W^T} \sum_{k=0}^{T}s^k w^k + a r^{T+1} &\leq (1-a\eta)^{(T+1)} \frac{r_0}{\eta} + c \eta \\
&\leq \frac{r_0}{\eta} \exp \left[-a\eta (T+1)\right] +c \eta \\
&\leq  d r_0 \exp \left[-\frac{a T}{ d} \right] + \frac{c}{aT}\,.
\end{align*}
\noindent If $T > \frac{d}{a}$, then we obtain from Lemma~\ref{lem:first} that
\begin{align*}
 r^{t_0} \leq  r^0 \exp \left[-\frac{a T}{2 d} \right] + \frac{c}{ad}  \,.
\end{align*}
From Lemma~\ref{lem:lacoste} we have for the second half of the iterates:
\begin{align*}
\frac{1}{W^T} \sum_{k=0}^{T}s^k w^k + a r^{T+1} &= \frac{1}{W^T} \sum_{k=t_0}^T s^k w^k + a r^{T+1} \leq  \frac{8  a \kappa^2 r^{t_0}}{T^2} + \frac{4c}{aT} \,.
\end{align*}
Now we observe that the restart condition $r^{t_0}$ satisfies:
\begin{align*}
 \frac{ a \kappa^2 r^{t_0} }{T^2} = \frac{ a \kappa^2 r^0 \exp \left( - \frac{aT}{2d} \right)}{T^2} + \frac{\kappa^2 c}{d T^2} \leq 4  a r^0 \exp \left[ - \frac{aT}{2 d} \right] + \frac{4c}{aT}\,,
\end{align*}
because $T > \frac{d}{a}$. These conclude the proof.

\end{proof}

Having these general convergence lemmas for the recursion of the form \eqref{eq:sequence}, the proof of the theorem follows directly from Lemmas~\ref{lem:first} and \ref{lem:sequence_induced_induced} with $a=\mu$, $c=\sigma^2$, $d = 2\delta_n L$ . It is easy to check that condition $\eta^k \leq \frac{1}{ d} = \frac{1}{2 \delta_n L}$ is satisfied.

\subsection{Proof of Theorem~\ref{thm:biased_to_unbiased}}

We have to show that our new compression is unbiased and has bounded variance. We start with the first property with $\lambda = 1$.
\begin{align*}
\E{\cC_1(x) + \cC_2(x - \cC_1(x))} &= \EE{\cC_1}{\EE{\cC_2}{\cC_1(x) + \cC_2(x - \cC_1(x))| \cC_1(x)} } \\
 &=  \EE{\cC_1}{\cC_1(x)  + x - \cC_1(x)} = x,
\end{align*}
where the first equality follows from tower property and the second from unbiasedness of $\cC_2$. For the second property, we also use tower property
\begin{align*}
\E{\norm*{\cC_1(x) -x + \cC_2(x - \cC_1(x))}^2} &= \EE{\cC_1}{\EE{\cC_2}{\norm*{\cC_1(x) - x + \cC_2(x - \cC_1(x))}^2| \cC_1(x)} } \\
 & \leq (\delta_2 - 1) \EE{\cC_1}{\norm*{\cC_1(x)  - x}^2}\\
 &  \leq  (\delta_2 - 1)\lp 1 - \frac{1}{\delta_1} \rp\norm*{x}^2,
\end{align*}
where the first and second inequalities follow directly from \eqref{eq:omega_quant} and \eqref{eq:quant}.

\subsection{Proof of Theorem~\ref{thm:u_n_p}}
Similarly to the proof of Theorem~\ref{thm:u_n}, we use the update of Algorithm~\ref{alg:UC_SGD} to bound the following quantity
\begin{eqnarray*}
&&\E{\norm*{x^{k+1} - \xs}^2 |x^k}\\
&=& \norm*{x^k - \xs}^2 - \eta^k \sum_{i=1}^n\E{\dotprod{\sum_{i \in S^k} \frac{1}{np_i}\cC^k(g_i^k)}{ x^k - \xs }|x^k} + \\
&& \quad  \E{\norm*{\sum_{i \in S^k}\frac{\eta^k}{np_i}\cC^k(g_i^k)}^2|x^k} \\
&\overset{\eqref{eq:omega_quant}}{\leq}& \norm*{x^k - \xs}^2 - \eta^k \dotprod{\nabla f(x^k)}{ x^k - \xs } +  \\
&& (\eta^k)^2\lp\E{\norm*{\sum_{i \in S^k}\frac{1}{np_i}\cC^k(g_i^k) - \frac{1}{n}\sum_{i=1}^n \cC^k(g_i^k)}^2|x^k} +  \E{\norm*{\frac{1}{n}\sum_{i=1}^n \cC^k(g^k)}^2|x^k}\rp\\
&\overset{\eqref{eq:omega_quant} + \eqref{eq:key_inequality_optimal}}{\leq}& \norm*{x^k - \xs}^2 - \eta^k \dotprod{\nabla f(x^k)}{ x^k - \xs } + \\
&&\quad \frac{(\eta^k)^2}{n^2}\E{\sum_{i=1}^n\lp\frac{\delta v_i}{p_i} + \delta - 1\rp\norm*{g_i^k} + \norm*{ \sum_{i=1}^n g_i^k }^2|x^k} \\
&\overset{\eqref{eq:quasi_convex} + \eqref{eq:L_smooth_f_i}+ \eqref{eq:L_smooth_f}}{\leq}& (1 - \mu \eta^k) \norm*{x^k - \xs}^2 - 2\eta^k \lp 1 - \eta^k \delta_\Sam L   \rp (f(x^k) - \fs) + \\
&&\quad  (\eta^k)^2 \lp(\delta_\Sam - 1) D +  \lp 1 +a_\Sam\rp \frac{\delta\sigma^2}{n}\rp.
\end{eqnarray*}
Taking full expectation and $\eta^k \leq \frac{1}{2\delta_\Sam L}$, we obtain
\begin{align*}
&\E{\norm*{x^{k+1} - \xs}^2} \\ & \quad \leq (1 - \mu \eta^k) \E{\norm*{x^k - \xs}^2} - \eta^k  \E{f(x^k) - \fs}
  +  (\eta^k)^2 \lp(\delta_\Sam - 1) D +  \lp 1 +a_\Sam\rp \frac{\delta\sigma^2}{n}\rp.
\end{align*}
The rest of the analysis is identical to the proof of Theorem~\ref{thm:u_n} with only difference $c = (\delta_\Sam - 1) D +  \lp 1 + a_\Sam \rp \frac{\delta\sigma^2}{n}$.


\refstepcounter{chapter}%
\chapter*{\thechapter \quad Appendix: Optimal client sampling for federated learning}
\label{appendix:optimal_sampling}

    \section{The improvement factor for optimal client sampling} \label{appendix:improvement factor}
        By Lemma \ref{LEM:UPPERV}, the independent sampling (which operates by independently flipping a coin and with probability $p_i$ includes element $i$ into $S$) is optimal. In addition, for independent sampling, \eqref{eq:key_inequality_optimal} holds as equality. Thus, letting $\Tilde{U}_i^k=w_i\U$, we have
        \begin{equation}\label{app:varianceSGD}
            \Tilde{\alpha}_{S^k} \eqdef\E{ \norm*{ \sum_{i\in S^k}\frac{w_i}{p_i^k}\U - \sum_{i=1}^n w_i \U }^2 }
            =\E{ \sum_{i=1}^n \frac{1-p_i^k}{p_i^k} \norm*{ \Tilde{U}_i^k }^2 }.
        \end{equation}
        The optimal probabilities are obtained by optimizing \eqref{app:varianceSGD} subject to the constraints $0 \leq p_i^k \leq 1$ and $m \geq b^k = \sum_{i=1}^n p_i^k$ using KKT conditions. Using an similar argument in~\cite{horvath2018nonconvex} (Lemma 2) gives the following solution
        \begin{align}
        \label{unique_sol_main1}
            p_i^k = 
            \begin{cases}
                (m +  l - n)\frac{\norm*{\Tilde{U}_i^k}}{\sum_{j=1}^l \norm*{\Tilde{U}_{(j)}^k}}, &\quad\text{if } i \notin A^k,  \\
                1, &\quad\text{if } i \in A^k,
            \end{cases}
        \end{align} 
        where $\norm*{\Tilde{U}_{(j)}^k}$ is the $j$-th largest value among the values $\norm*{\Tilde{U}_1^k}, \norm*{\Tilde{U}_2^k}, \dots, \norm*{\Tilde{U}_n^k}$, $l$ is the largest integer  for which $0<m+l-n \leq \frac{\sum_{i =1}^l \norm*{\Tilde{U}_{(i)}^k}}{\norm*{\Tilde{U}_{(l)}^k}}$ (note that this inequality at least holds for $l= n-m+1$), and $A^k$ contains indices $i$ such that $\norm*{\Tilde{U}_i^k} \geq \norm*{\Tilde{U}_{(l+1)}^k}$.
        
        Plugging the optimal probabilities obtained in \eqref{unique_sol_main1} into \eqref{app:varianceSGD} gives
        \begin{align*}
                \Tilde{\alpha}_{S^k}^\star
                =\E{ \sum_{i=1}^n \frac{1}{p_i^k} \norm*{ \Tilde{U}_i^k }^2 - \sum_{i=1}^n \norm*{ \Tilde{U}_i^k }^2 }
                =\E{ \frac{1}{m-(n-l)} \left( \sum_{i=1}^l \norm*{ \Tilde{U}_{(i)}^k } \right)^2 - \sum_{i=1}^l \norm*{ \Tilde{U}_{(i)}^k }^2 }.
        \end{align*}
        With $m\norm*{\Tilde{U}_{(n)}^k}\leq \sum_{i=1}^n \norm*{\Tilde{U}_i^k}$, we have
        \begin{align*}
            \E{ \frac{1}{m} \left( \sum_{i=1}^n \norm*{ \Tilde{U}_i^k } \right)^2 - \sum_{i=1}^n \norm*{ \Tilde{U}_i^k }^2 }
            &=\E{ \frac{1}{m} \left( \sum_{i=1}^n \norm*{ \Tilde{U}_i^k } \right)^2 \left( 1 - m \frac{\sum_{i=1}^n \norm*{ \Tilde{U}_i^k }^2}{\left( \sum_{i=1}^n \norm*{ \Tilde{U}_i^k } \right)^2} \right) }\\
            &\leq \frac{n-m}{nm}\E{ \left( \sum_{i=1}^n \norm*{ \Tilde{U}_i^k } \right)^2 }.
        \end{align*}
        For independent uniform sampling $U^k\sim\mathbb{U}$ ($p_i^U=\frac{m}{n}$ for all $i$), we have
        \begin{align*}
            \Tilde{\alpha}_{U^k} 
            \eqdef\E{ \norm*{ \sum_{i\in U^k}\frac{w_i}{p_i^U}\U - \sum_{i=1}^n w_i \U }^2 } 
            =\E{ \sum_{i=1}^n \frac{1-\frac{m}{n}}{\frac{m}{n}} \norm*{ \Tilde{U}_i^k }^2 }
            = \frac{n-m}{m}\E{ \sum_{i=1}^n  \norm*{ \Tilde{U}_i^k }^2 }.
        \end{align*}
        Putting them together gives the improvement factor:
        \begin{align*}
            \alpha^k \eqdef
            \frac{\Tilde{\alpha}_{S^k}^\star}{\Tilde{\alpha}_{U^k}}
            =\frac{
               \E{ \norm*{ \sum_{i\in S^k}\frac{w_i}{p_i^k}\U - \sum_{i=1}^n w_i \U }^2 }
            }
            {
               \E{ \norm*{ \sum_{i\in U^k}\frac{w_i}{p_i^U}\U - \sum_{i=1}^n w_i \U }^2 }
            }
            \leq \frac{
               \E{ \left( \sum_{i=1}^n \norm*{ \Tilde{U}_i^k } \right)^2 }
            }
            {
                n\E{\sum_{i=1}^n  \norm*{ \Tilde{U}_i^k }^2 }
            }
            \leq 1,
        \end{align*}
        The upper bound is attained when all $\norm*{\Tilde{U}_i^k}$ are identical. Note that the lower bound $0$ can also be attained in the case where the number of non-zero updates is at most $m$. These considerations are discussed in the main paper.

    \section{{\tt DSGD} with optimal client sampling}\label{appendix:proof_dsgd}

    \subsection{Proof of Theorem \ref{THM:DSGD_MAIN}}\label{appendix:proof_dsgd_convex}
        
        \begin{proof}
        $L$-smoothness of $f_i$ and the assumption on the gradient  imply that the inequality $$\E{\norm*{g_i^k}^2}\leq 2L(1+M)(f_i(x^k)-f_i(x^\star)+R_i) + \sigma^2$$ holds for all $ k\geq0$.         We first take expectations over $x^{k+1}$ conditioned on $x^k$ and over the sampling $S^k$:
        \begin{align*}
               \E{\norm*{r^{k+1}}^2}
                &= \norm*{r^k}^2 - 2\eta^k\E{ \left< \sum_{i\in S^k}\frac{w_i}{p_i^k}g_i^k,r^k \right> } + (\eta^k)^2\E{ \norm*{ \sum_{i\in S^k}\frac{w_i}{p_i^k}g_i^k }^2 } \\
                &= \norm*{r^k}^2 - 2\eta^k \left< \nabla f(x^k),r^k \right>  \\
                &\quad + (\eta^k)^2 \left(\E{ \norm*{ \sum_{i\in S^k}\frac{w_i}{p_i^k}g_i^k - \sum_{i=1}^n w_i g_i^k }^2} + \E{ \norm*{ \sum_{i=1}^n w_i g_i^k }^2 } \right) \\
                &\leq (1-\mu\eta^k)\norm*{r^k}^2 - 2\eta^k \left( f(x^k)-f^\star \right) \\
                &\quad + (\eta^k)^2 \left(\E{ \norm*{ \sum_{i\in S^k}\frac{w_i}{p_i^k}g_i^k - \sum_{i=1}^n w_i g_i^k }^2} +\E{ \norm*{ \sum_{i=1}^n w_i g_i^k }^2 } \right),
        \end{align*}
        where
        \begin{align*}
               &\E{ \norm*{ \sum_{i\in S^k}\frac{w_i}{p_i^k}g_i^k - \sum_{i=1}^n w_i g_i^k }^2} \\
                &= \alpha^k \frac{n-m}{m}\E{ \sum_{i=1}^n w_i^2 \norm*{ g_i^k }^2 } \\
                &= \alpha^k \frac{n-m}{m}\E{ \sum_{i=1}^n w_i^2 \left(\norm*{ g_i^k - \nabla f_i(x^k) }^2 + \norm*{\nabla f_i(x^k)}^2\right)} \\
                &= \alpha^k \frac{n-m}{m}\E{ \sum_{i=1}^n w_i^2 \left(\norm*{ \xi_i^k }^2 + \norm*{\nabla f_i(x^k)}^2\right)} \\
                &\leq \alpha^k \frac{n-m}{m} \sum_{i=1}^n w_i^2 \left( 2L(1+M)(f_i(x^k)-f_i(x^\star)+R_i)+\sigma^2 \right)\\
                &\leq \alpha^k \frac{n-m}{m}  \left( 2WL(1+M)(f(x^k)-f^\star) + \sum_{i=1}^n w_i^2(2L(1+M)R_i+\sigma^2)  \right),
        \end{align*}
        and 
        \begin{align*}
           \E{ \norm*{ \sum_{i=1}^n w_i g_i^k }^2 } &= \E{ \norm*{ \sum_{i=1}^n w_i g_i^k - \nabla f(x^k) }^2 } + \norm*{\nabla f(x^k)}^2  \\
           &= \sum_{i=1}^n \E{ \norm*{ w_i g_i^k - w_i \nabla f_i(x^k) }^2 } + \norm*{\nabla f(x^k)}^2  \\
           &= \sum_{i=1}^n w_i^2 \E{ \norm*{ \xi_i^k }^2 } + \norm*{\nabla f(x^k)}^2  \\
            &\leq \sum_{i=1}^n w_i^2 (2LM(f_i(x^k) - f_i^\star) + \sigma^2) + 2L(f(x^k) - f^\star) \\
            &=  2L\left(1+WM\right)(f(x^k)-f^\star) + \sum_{i=1}^n w_i^2 (2LMR_i+\sigma^2).
        \end{align*}
        Therefore, we obtain
        \begin{align*}
               &\E{\norm*{r^{k+1}}^2} \\
                &\leq (1-\mu\eta^k)\norm*{r^k}^2 - 2\eta^k \left( f(x^k)-f^\star \right) \\
                &\quad+ (\eta^k)^2 \left(2L\left(1+WM\right)(f(x^k)-f^\star) + \sum_{i=1}^n w_i^2 (2LMR_i+\sigma^2)\right) \\
                &\quad + (\eta^k)^2 \alpha^k \frac{n-m}{m}  \left( 2WL(1+M)(f(x^k)-f^\star) + \sum_{i=1}^n w_i^2(2L(1+M)R_i+\sigma^2)  \right)\\
                &\leq (1-\mu\eta^k)\norm*{r^k}^2 - 2\eta^k \left( 1-\eta^k\frac{(\alpha^k(n-m)+m)(1 + WM)L}{m} \right) \left( f(x^k)-f^\star \right) \\
                &\quad+ (\eta^k)^2 \frac{\alpha^k(n-m)+m}{m} \left( \sum_{i=1}^n w_i^2 (2L(1+M)R_i + \sigma^2) \right)  - (\eta^k)^2 2L\sum_{i=1}^n w_i^2R_i.
        \end{align*}
        Now choose any $0<\eta^k\leq\frac{m}{(\alpha^k(n-m)+m)(1 + WM)L}$ and define
        \begin{align*}
            \beta_1 &\eqdef \sum_{i=1}^n w_i^2 (2L(1+M)R_i + \sigma^2),\quad \beta_2 \eqdef 2L\sum_{i=1}^n w_i^2R_i, \\
            \gamma^k &\eqdef \frac{m}{\alpha^k(n-m)+m} \in \left[\frac{m}{n},  1 \right].
        \end{align*}
      Taking full expectation yields the desired result:
        \begin{align*}
          \E{\norm*{r^{k+1}}^2}
            \leq (1-\mu\eta^k)\E{\norm*{r^k}^2} +(\eta^k)^2 \left( \frac{\beta_1}{\gamma^k} - \beta_2 \right).
        \end{align*}
        \end{proof}

\subsection{Proof of Theorem \ref{THM:DSGD_NONCONVEX}}\label{appendix:proof_dsgd_nonconvex}
    \begin{proof}
    Using equation \eqref{eq:SGD_step}, we have
    \begin{align*}
        f(x^{k+1})&=f(x^k - \eta^k \mG^k)\\
        &=f(x^k)- \eta^k \left<\mG^k,\nabla f(x^k)\right> + \frac{(\eta^k)^2}{2} \left<\mG^k,\nabla^2 f(z^k) \mG^k\right>, \quad\text{for some $z^k\in\mathbb{R}^d$}.
    \end{align*}
    Since all $f_i$'s are $L$-smooth, $f$ is also $L$-smooth. Therefore, we have $-L\boldsymbol{I}\preceq \nabla^2 f(x) \preceq L\boldsymbol{I}$ for all $x\in\mathbb{R}^d$. Combining this with the fact that $\mG^k$ is an unbiased estimator of $\nabla f(x^k)$, we have
    \begin{align}\label{eq:nonconve:middle}
        \E{f(x^{k+1})}&\leq
        f(x^{k}) -\eta^k \norm*{\nabla f(x^k)}^2+\frac{(\eta^k)^2 L}{2}\E{\norm*{\mG^k}^2},
    \end{align}
    where the expectations are conditioned on $x^k$. In Appendix \ref{appendix:proof_dsgd_convex}, we already obtained the upper bound for the last term in equation \eqref{eq:nonconve:middle}:
    \begin{align*}
        \E{\norm*{\mG^k}^2}
        &\leq
        \left( (1+M)\alpha^k\frac{n-m}{m}+M \right)\sum_{i=1}^n w_i^2\norm*{\nabla f_i(x^k)}^2 \\
                &\quad +\left( \alpha^k\frac{n-m}{m}+1 \right)\sum_{i=1}^nw_i^2\sigma^2+\norm*{\nabla f(x^k)}^2\\
        &=\left( \frac{1+M}{\gamma^k}-1 \right)\sum_{i=1}^n w_i^2\norm*{\nabla f_i(x^k)}^2+\frac{1}{\gamma^k}\sum_{i=1}^nw_i^2\sigma^2+\norm*{\nabla f(x^k)}^2.
    \end{align*}
    By Assumption \ref{ass:local grad similarity}, we further bound
    \begin{align*}
        \sum_{i=1}^n w_i^2\norm*{\nabla f_i(x^k)}^2
        &\leq W \sum_{i=1}^n w_i\norm*{\nabla f_i(x^k)}^2\\
        &\leq W\left( \sum_{i=1}^n w_i\norm*{\nabla f_i(x^k)-\nabla f(x^k)}^2 + \norm*{\nabla f(x^k)}^2 \right)\\
        &\leq W\rho + \norm*{\nabla f(x^k)}^2.
    \end{align*}
    Combining the inequalities above and taking full expectation yields equation \eqref{eq:DSGD_nonconvex}.
    \end{proof}

    \section{{\tt FedAvg} with optimal client sampling}\label{appendix:proof_fedavg}

    \subsection{Proof of Theorem \ref{THM:FEDAVG_MAIN}}\label{appdix:proof fedavg convex}
        \begin{proof}
        The master update during round $k$ can be written as (superscript $k$ is dropped from here onward)
        \begin{align*}
            \eta_g\Delta x=\frac{\eta}{R}\sum_{i\in S,r}\frac{w_i}{p_i}g_i(y_{i,r-1})
            \quad\text{and}\quad
            \E{\eta_g\Delta x} = \frac{\eta}{R}\sum_{i,r}w_i\E{\nabla f_i(y_{i,r-1})}.
        \end{align*}
        
        Summations are always over $i\in[n]$ and $r\in[R]$ unless  stated otherwise. Taking expectations over $x$ conditioned on the results prior to round $k$ and over the sampling $S$ gives
        \begin{align*}
               \E{\norm*{x-\eta_g\Delta x-x^\star}^2}
                &= \norm*{x-x^\star}^2\underbrace{-\frac{2\eta}{R}\sum_{i,r}\left<w_i\nabla f_i(y_{i,r-1}),x-x^\star\right>}_{\mathcal{A}_1} \\
                &\quad + \underbrace{\frac{\eta^2}{R^2}\E{\norm*{ \sum_{i\in S,r}\frac{w_i}{p_i}g_i(y_{i,r-1})  }^2}}_{\mathcal{A}_2}.
        \end{align*}
        
        Applying Lemma \ref{lemma:FedAvgProof} with $h=w_if_i$, $x=y_{i,r-1}$, $y=x^\star$ and $z=x$ gives
        \begin{align*}
                \mathcal{A}_1
                &\leq -\frac{2\eta}{R}\sum_{i,r}\left( w_if_i(x) - w_if_i(x^\star) + w_i\frac{\mu}{4} \norm*{ x-x^\star }^2 - w_iL \norm*{x-y_{i,r-1}}^2 \right)\\
                &\leq -2\eta \left( f(x)-f^\star + \frac{\mu}{4} \norm*{ x-x^\star }^2 \right) + 2L\eta\mathcal{E},
        \end{align*}
        where $\mathcal{E}$ is the drift caused by the local updates on the clients:
        \begin{equation}\label{eq:drift}
            \mathcal{E} \eqdef \frac{1}{R}\sum_{i,r}w_i \E{\norm*{x-y_{i,r-1}}^2}.
        \end{equation}
        
        Bounding $\mathcal{A}_2$, we obtain 
        \begin{align*}
               \frac{1}{\eta^2} \mathcal{A}_2
                &=\E{\norm*{ \sum_{i\in S}\frac{w_i}{p_i}\frac{1}{R}\sum_r g_i(y_{i,r-1}) - \sum_{i}w_i\frac{1}{R}\sum_r g_i(y_{i,r-1}) }^2}\\
                &\quad +\E{\norm*{ \sum_{i}w_i\frac{1}{R}\sum_r g_i(y_{i,r-1}) }^2}\\
                &\leq  \alpha \frac{n-m}{m}\sum_i w_i^2 \E{ \norm*{ \frac{1}{R}\sum_r g_i(y_{i,r-1}) }^2 } +\E{\norm*{ \sum_{i}w_i\frac{1}{R}\sum_rg_i(y_{i,r-1}) }^2}\\
                &=   \alpha \frac{n-m}{m}\sum_i w_i^2 \left(\E{ \norm*{ \frac{1}{R}\sum_r \xi_{i, r-1} }^2 }  + \E{ \norm*{ \frac{1}{R}\sum_r  \nabla f_i(y_{i,r-1}) }^2 } \right) \\
                &\quad + \E{\norm*{ \sum_{i}w_i \frac{1}{R}\sum_r\xi_{i, r-1} }^2} + \E{\norm*{ \sum_{i}w_i \frac{1}{R}\sum_r\nabla f_i(y_{i, r-1}) }^2}.
        \end{align*}
        
       Using independence, zero mean and bounded second moment of the random variables $\xi_{i, r}$, we obtain
        \begin{align*}
          &\frac{1}{\eta^2} \mathcal{A}_2 \\
          &\leq   \alpha \frac{n-m}{m}\sum_{i} w_i^2 \left(\frac{1}{R^2}\sum_r \E{ \norm*{  \xi_{i, r-1} }^2 }  + \E{ \norm*{ \frac{1}{R}\sum_{r}  \nabla f_i(y_{i,r-1}) }^2 } \right) \\
                &\quad + \sum_{i}w_i^2 \frac{1}{R^2}\sum_r \E{\norm*{  \xi_{i, r-1} }^2} + \E{\norm*{ \sum_{i}w_i \frac{1}{R}\sum_r \nabla f_i(y_{i, r-1}) }^2}\\
                &\leq  \alpha \frac{n-m}{m}\sum_{i} w_i^2  \left(\left(\frac{M}{R^2}+ \frac{1}{R}\right) \sum_r \E{\norm*{\nabla f_i(y_{i, r-1}) }^2 } +\frac{\sigma^2}{R} \right) \\
                &\quad + \sum_{i}w_i^2 \left(\frac{M}{R^2} \sum_r \E{\norm*{  \nabla f_i(y_{i, r-1}) }^2} + \frac{\sigma^2}{R} \right) + \E{\norm*{ \sum_{i}w_i  \frac{1}{R}\sum_r \nabla f_i(y_{i, r-1}) }^2}\\
                &= \left(\frac{M}{R} + \left(\frac{M}{R}+1\right)\alpha \frac{n-m}{m} \right)  \sum_{i}w_i^2 \frac{1}{R} \sum_r \E{\norm*{  \nabla f_i(y_{i, r-1}) - \nabla f_i(x) + \nabla f_i(x) }^2}  \\
                &\quad +  \E{\norm*{ \sum_{i}w_i \frac{1}{R} \sum_r  (\nabla f_i(y_{i, r-1}) - \nabla f_i(x)) + \nabla f(x)  }^2} + \frac{\sigma^2}{R\gamma} \sum_{i}w_i^2 \\
                &\leq \left(\frac{M}{R} + \left(\frac{M}{R}+1\right) \alpha \frac{n-m}{m} \right) \frac{2}{R} \sum_r \E{\norm*{  \nabla f_i(y_{i, r-1}) - \nabla f_i(x)}^2} \sum_{i}w_i^2 \\
                & \quad +  \left(\frac{M}{R} + \left(\frac{M}{R}+1\right) \alpha \frac{n-m}{m} \right) 2\E{\norm*{\nabla f_i(x) }^2} \sum_{i}w_i^2 \\
                &\quad +  2\E{\norm*{ \sum_{i}w_i \frac{1}{R} \sum_r  (\nabla f_i(y_{i, r-1}) - \nabla f_i(x))}^2} + 2\E{\norm*{\nabla f(x) }^2} + \frac{\sigma^2}{R\gamma} \sum_{i}w_i^2.
          \end{align*}
          
        Combining the smoothness of $f_i$'s, the definition of $\mathcal{E}$, and Jensen's inequality with definition $\gamma\eqdef \frac{m}{\alpha(n-m)+m}$, we obtain  
        \begin{align*}
            &\frac{1}{\eta^2} \mathcal{A}_2 \\
                &\leq  2\left(\frac{M}{R} + \left(\frac{M}{R}+1 \right) \alpha \frac{n-m}{m} \right)\left(WL^2\mathcal{E} + 2WL(f(x) - f^\star) + 2L\sum_i w_i^2 R_i \right)\\
                &\quad +  2L^2\mathcal{E}  + 4L(f(x) - f(x^\star)) + \frac{\sigma^2}{R\gamma} \sum_{i} w_i^2 \\
                 &=  2L^2\left((1 - W) + \frac{W}{\gamma}\left(\frac{M}{R}+1 \right) \right) \mathcal{E} + 4L \left( \frac{1}{\gamma}\left(\frac{M}{R}+1  \right) - 1 \right)  \sum_i w_i^2 R_i \\
                &\quad+4L\left((1 - W)+ \frac{W}{\gamma}\left(\frac{M}{R}+1\right) \right)(f(x) - f^\star) + \frac{\sigma^2}{R\gamma} \sum_{i} w_i^2.
        \end{align*}
       
        Putting these bounds on $\mathcal{A}_1$ and $\mathcal{A}_2$ together and using the fact that $1-W\leq \nicefrac{1}{\gamma}$ yields
        \begin{align*}
               &\E{\norm*{x-\eta_g\Delta x-x^\star}^2} \\
                &\leq
                \left( 1 - \frac{\mu\eta}{2} \right) \norm*{x-x^\star}^2
                - 2\eta\left( 1-2L\frac{\eta}{\gamma} \left( W\left( \frac{M}{R}+1 \right) + 1 \right)\right) (f(x)-f^\star)\\
			&\quad + \eta^2 \left( \frac{\sigma^2}{R\gamma} \sum_{i}w_i^2+ 4L \left( \frac{1}{\gamma}\left(\frac{M}{R}+1  \right) - 1 \right)  \sum_i w_i^2 R_i \right) \\
                &\quad +\left( 1 +  \eta L \left((1 - W) + \frac{W}{\gamma}\left(\frac{M}{R}+1 \right) \right) \right) 2L\eta \mathcal{E}.
        \end{align*}
        
        Let $\eta\leq\frac{\gamma}{8(1 + W(1 + \nicefrac{M}{R}))L}$, then
        \begin{align*}
            \frac{3}{4} \leq   1-2L\frac{\eta}{\gamma} \left( W\left( \frac{M}{R}+1 \right) + 1 \right), 
        \end{align*}
        which in turn yields
        \begin{align}
        \label{eq:fedavg_rec}
               \E{\norm*{x-\eta_g\Delta x-x^\star}^2}
                &\leq \left( 1 - \frac{\mu\eta}{2} \right) \norm*{x-x^\star}^2
                - \frac{3\eta}{2}(f(x)-f^\star)  \notag \\
                &\quad+ \eta^2 \left( \frac{\sigma^2}{R\gamma} \sum_{i}w_i^2+ 4L \left( \frac{1}{\gamma}\left(\frac{M}{R}+1  \right) - 1 \right)  \sum_i w_i^2 R_i \right)  \notag \\
                &\quad +\left( 1 + \eta L \left((1 - W) + \frac{W}{\gamma}\left(\frac{M}{R}+1 \right) \right) \right) 2L\eta \mathcal{E}. 
        \end{align}    
        
        Next, we need to bound the drift $\mathcal{E}$. For $R\geq 2$, we have
        \begin{align*}
                \E{\norm*{y_{i,r}-x}^2}
                &= \E{\norm*{y_{i,r-1}-x-\eta_l g_i(y_{i,r-1})}^2} \\
                &\leq \E{\norm*{y_{i,r-1}-x-\eta_l \nabla f_i(y_{i,r-1})}^2} + \eta_l^2 (M\norm*{\nabla f_i(y_{i,r-1})}^2 + \sigma^2)\\
                &\leq \left( 1+\frac{1}{R-1} \right) \E{\norm*{y_{i,r-1}-x}^2} + (R + M)\eta_l^2\norm*{\nabla f_i(y_{i,r-1})}^2+ \eta_l^2\sigma^2\\
                &=\left( 1+\frac{1}{R-1} \right) \E{\norm*{y_{i,r-1}-x}^2}  \\ 
                & \quad + \left(1 + \frac{M}{R} \right)\frac{\eta^2}{R\eta_g^2}\norm*{\nabla f_i(y_{i,r-1})}^2+ \frac{\eta^2\sigma^2}{R^2\eta_g^2}\\
                &\leq \left( 1+\frac{1}{R-1} \right) \E{\norm*{y_{i,r-1}-x}^2} \\ 
                & \quad + \left(1 + \frac{M}{R} \right)  \frac{2\eta^2}{R\eta_g^2}\norm*{\nabla f_i(y_{i,r-1})-\nabla f_i(x) }^2 \\
                &\quad + \left(1 + \frac{M}{R} \right)\frac{2\eta^2}{R\eta_g^2}\norm*{\nabla f_i(x)}^2+ \frac{\eta^2\sigma^2}{R^2\eta_g^2}\\
                &\leq \left( 1+\frac{1}{R-1} +  \left(1 + \frac{M}{R} \right)\frac{2\eta^2L^2}{R\eta_g^2}\right) \E{\norm*{y_{i,r-1}-x}^2} \\ 
                & \quad + \left(1 + \frac{M}{R} \right)\frac{2\eta^2}{R\eta_g^2}\norm*{\nabla f_i(x)}^2+ \frac{\eta^2\sigma^2}{R^2\eta_g^2}.
        \end{align*}
        
        If we further restrict $\eta\leq\frac{1}{8L(2 + \nicefrac{M}{R})}$, then for any $\eta_g \geq 1$, we have
        \begin{align*}
            \left(1 + \frac{M}{R} \right)\frac{2\eta^2L^2}{R\eta_g^2} \leq \frac{2L^2}{R\eta_g^2}\frac{1}{64L^2} \leq \frac{1}{32R} \leq \frac{1}{32(R-1)},
        \end{align*}
   and    therefore,
        \begin{align*}
                &\E{\norm*{y_{i,r}-x}^2} \\
                &\leq \left( 1+\frac{33}{32(R-1)} \right) \E{\norm*{y_{i,r-1}-x}^2} +   \left(1 + \frac{M}{R} \right)\frac{2\eta^2}{R\eta_g^2}\norm*{\nabla f_i(x)}^2 + \frac{\eta^2\sigma^2}{R^2\eta_g^2} \\
                &\leq \sum_{\tau=0}^{r-1} \left( 1+\frac{33}{32(R-1)} \right)^\tau \left(  \left(1 + \frac{M}{R} \right) \frac{2\eta^2}{R\eta_g^2}\norm*{\nabla f_i(x)}^2+ \frac{\eta^2\sigma^2}{R^2\eta_g^2} \right)\\
                &\leq 8R\left(  \left(1 + \frac{M}{R} \right) \frac{2\eta^2}{R\eta_g^2}\norm*{\nabla f_i(x)}^2+ \frac{\eta^2\sigma^2}{R^2\eta_g^2} \right)\\
                &= 16 \left(1 + \frac{M}{R} \right)\eta^2\norm*{\nabla f_i(x)}^2+ \frac{8\eta^2\sigma^2}{R\eta_g^2}.
        \end{align*}
        
        Hence, the drift is bounded by
        \begin{align*}
                \mathcal{E}
                &\leq 16 \left(1 + \frac{M}{R} \right)\eta^2 \sum_i w_i \norm*{\nabla f_i(x)}^2+ \frac{8\eta^2\sigma^2}{R\eta_g^2} \\
                &\leq 32 \left(1 + \frac{M}{R} \right)\eta^2L \sum_i w_i (f_i(x)-f_i^\star)+ \frac{8\eta^2\sigma^2}{R\eta_g^2} \\
                &= 32 \left(1 + \frac{M}{R} \right)\eta^2L (f(x)-f^\star)+ 32 \left(1 + \frac{M}{R} \right)\eta^2L \sum_i w_i R_i + \frac{8\eta^2\sigma^2}{R\eta_g^2}\\
                &\leq 4\eta (f(x)-f^\star)+ 32 \left(1 + \frac{M}{R} \right)\eta^2L \sum_i w_i R_i + \frac{8\eta^2\sigma^2}{R\eta_g^2}.
        \end{align*}
        
        Due to the upper bound on the step size $\eta\leq\frac{1}{8L(2 + \nicefrac{M}{R})}$, we have the inequalities
        \begin{align} \label{eq:fedavg_simplify}
                1 +  \eta L \left((1 - W) + \frac{W}{\gamma}\left(\frac{M}{R}+1 \right) \right) \leq \frac{9}{8} \quad \mbox{and} \quad 8\eta L \leq 1.
        \end{align}
        
        Plugging these to \eqref{eq:fedavg_rec}, we obtain
        \begin{align*}
               \E{\norm*{x-\eta_g\Delta x-x^\star}^2}
                &\leq \left( 1 - \frac{\mu\eta}{2} \right) \norm*{x-x^\star}^2
                 -\frac{3}{8}\eta(f(x)-f^\star)\\
                &\quad + \eta^2 \left(\frac{\sigma^2}{\gamma R} \left(\frac{\gamma}{\eta_g^2} + \sum_{i}w_i^2\right)+ 4L \left( \frac{M}{R}+1  - \gamma\right)  \sum_i w_i^2 R_i \right) \\
                &\quad + \eta^3 72L^2 \left(1 + \frac{M}{R} \right) \sum_i w_i R_i. 
        \end{align*}    
        
        Rearranging the terms in the last inequality, taking full expectation  and including superscripts lead to
          \begin{align*}
                \frac{3}{8} \E{(f(x^k)-f^\star)}  &\leq \frac{1}{\eta^k}\left( 1 - \frac{\mu\eta^k}{2} \right) \E{\norm*{x^k-x^\star}^2} - \frac{1}{\eta^k}\E{\norm*{x^{k+1}-x^\star}^2} \\
                &\quad +  \eta^k \left( \frac{\sigma^2}{\gamma^k R} \left(\frac{\gamma^k}{\eta_g^2} + \sum_{i}w_i^2\right)+ 4L \left( \frac{M}{R}+1   - \gamma^k \right)  \sum_i w_i^2 R_i \right) \\
                &\quad +  (\eta^k)^2 72 L^2\left(1 + \frac{M}{R} \right) \sum_i w_i R_i.
        \end{align*}
        Plugging the assumption $\eta_g^k \geq \sqrt{\frac{\gamma^k}{\sum_{i}w_i^2}}$ into the RHS of the above inequality completes the proof.
    \end{proof}
    
    \subsection{Proof of Theorem \ref{THM:FEDAVG_MAIN_nonconvex}}
    \begin{proof}
        We drop superscript $k$ and write the master update during round $k$ as:
        \begin{align*}
            \eta_g\Delta x=\frac{\eta}{R}\sum_{i\in S,r}\frac{w_i}{p_i}g_i(y_{i,r-1})
            \eqdef \eta\Tilde{\Delta}.
        \end{align*}
        Summations are always over $i\in[n]$ and $r\in[R]$ unless stated otherwise. Taking expectations conditioned on $x$ and using a similar argument as in the proof in Appendix \ref{appendix:proof_dsgd_nonconvex}, we have
        \begin{align*}
            \E{f(x-\eta_g\Delta x)}
            &\leq f(x)-\eta \left<\nabla f(x),\E{\Tilde{\Delta}}\right>+\frac{\eta^2L}{2}\E{\norm*{\Tilde{\Delta}}^2}\\
            &= f(x)-\eta \norm*{\nabla f(x)}^2+\eta\left<\nabla f(x),\nabla f(x)-\E{\Tilde{\Delta}}\right>+\frac{\eta^2L}{2}\E{\norm*{\Tilde{\Delta}}^2}\\
            &\leq f(x)- \frac{\eta}{2} \norm*{\nabla f(x)}^2+\frac{\eta}{2}\E{\norm*{\nabla f(x)-\EE{S}{\Tilde{\Delta}}}^2}+\frac{\eta^2L}{2}\E{\norm*{\Tilde{\Delta}}^2},
        \end{align*}
        where the last inequality follows since $\left<a,b\right>\leq\frac{1}{2}\norm*{a}^2+\frac{1}{2}\norm*{b}^2,~\forall a,b\in\mathbb{R}^d$. Since $f_i$'s are $L$-smooth, by the (relaxed) triangular inequality, we have
        \begin{align*}
            \frac{\eta}{2}\E{\norm*{\nabla f(x)-\E{\Tilde{\Delta}}}^2}
            &= \frac{\eta}{2}\E{\norm*{\frac{1}{R} \sum_{i,r}w_i\left( \nabla f_i(x) - \nabla f_i(y_{i,r-1}) \right)}^2}\\
            &\leq \frac{\eta L^2}{2R}\sum_{i,r}w_i\E{\norm*{x-y_{i,r-1}}^2}
            = \frac{\eta L^2}{2}\mathcal{E},
        \end{align*}
        where $\mathcal{E}$ is the drift caused by the local updates on the clients as defined in \eqref{eq:drift}. 
        
        In Appendix \ref{appdix:proof fedavg convex}, we already obtained the upper bound for $\frac{1}{\eta^2}\mathcal{A}_2=\E{\norm*{\Tilde{\Delta}}^2}$:
        \begin{align*}
            \E{\norm*{\Tilde{\Delta}}^2}
            &\leq \frac{\sigma^2}{R\gamma} \sum_{i}w_i^2 +2L^2\mathcal{E} + 2\norm*{\nabla f(x) }^2 \\ 
                & \quad + 2W\left(\frac{M}{R} + \left(\frac{M}{R}+1\right) \alpha \frac{n-m}{m} \right)  \left(L^2\mathcal{E}  + \sum_{i}w_i\norm*{\nabla f_i(x) }^2\right).
        \end{align*}
        Together with Assumption \ref{ass:local grad similarity} that
        \begin{align*}
            \sum_{i}w_i\norm*{\nabla f_i(x) }^2 - \norm*{\nabla f(x)}^2
            \leq \sum_{i}w_i\norm*{\nabla f_i(x) - \nabla f(x)}^2 \leq \rho,
        \end{align*}
        we have
        \begin{align*}
            \E{\norm*{\Tilde{\Delta}}^2}
            &\leq \frac{\sigma^2}{R\gamma} \sum_{i}w_i^2 + \frac{2W}{\gamma}\left(\frac{M}{R}+1 -\gamma \right)  \left(L^2\mathcal{E}  + \norm*{\nabla f(x) }^2+\rho\right) \\
            &\quad +2L^2\mathcal{E} + 2\norm*{\nabla f(x) }^2.
        \end{align*}
        Combining the above inequalities gives
        \begin{align*}
            &\E{f(x-\eta_g\Delta x)} \\
            &\leq f(x) + \eta^2\frac{\sigma^2L}{2R\gamma}\sum_i w_i^2+ \eta L^2\left( \eta L\left( (1- W) + \frac{W}{\gamma} \left( 1+\frac{M}{R} \right) \right)+\frac{1}{2} \right)\mathcal{E}\\
            &\quad+ \eta \left( \eta L\left( (1- W) + \frac{W}{\gamma} \left( 1+\frac{M}{R} \right) \right)-\frac{1}{2} \right)\norm*{\nabla f(x)}^2\\
            &\quad+ \eta \left( \eta L\left( (1- W) + \frac{W}{\gamma} \left( 1+\frac{M}{R} \right) \right)-\eta L \right)\rho.
        \end{align*}
        Now, applying inequality \eqref{eq:fedavg_simplify} gives
        \begin{align*}
            \E{f(x-\eta_g\Delta x)}
            &\leq f(x) + \frac{\eta^2\sigma^2L}{2R\gamma}\sum_i w_i^2 +\frac{5\eta L^2}{8}\mathcal{E} -\frac{3\eta}{8}\norm*{\nabla f(x)}^2+\frac{\eta}{8}(1-8\eta L)\rho.
        \end{align*}
        In Appendix \ref{appdix:proof fedavg convex}, we also obtained the upper bound for the drift $\mathcal{E}$:
        \begin{align*}
            \mathcal{E}
                &\leq 16 \left(1 + \frac{M}{R} \right)\eta^2 \sum_i w_i \norm*{\nabla f_i(x)}^2+ \frac{8\eta^2\sigma^2}{R\eta_g^2}\\
                &\leq 16 \left(1 + \frac{M}{R} \right)\eta^2 (\norm*{\nabla f(x)}^2+\rho)+ \frac{8\eta^2\sigma^2}{R\eta_g^2}.
        \end{align*}
        Since $8\eta L\leq 8\eta L(1+\nicefrac{M}{R})\leq 1$, we have
        \begin{align*}
            \frac{5\eta L^2}{8}\mathcal{E}
            &\leq 10\eta^3 L^2 \left(1 + \frac{M}{R} \right) (\norm*{\nabla f(x)}^2+\rho)+  \frac{5\eta^3L^2\sigma^2}{R\eta_g^2}\\
            &\leq \frac{5\eta^2 L}{4}(\norm*{\nabla f(x)}^2+\rho)+  \frac{5\eta^2L\sigma^2}{8R\eta_g^2}.
        \end{align*}
        This further simplifies the iterate to
        \begin{align*}
            \E{f(x-\eta_g\Delta x)}
            &\leq f(x) - \frac{3}{8}\eta\left( 1-\frac{10}{3}\eta L \right) \norm*{\nabla f(x)}^2 \\ 
                & \quad  +\frac{1}{8}\eta\left(1+2\eta L\right)\rho+\frac{\eta^2\sigma^2 L}{2R\gamma}\left( \frac{5\gamma}{4\eta_g^2}+\sum_i w_i^2 \right).
        \end{align*}
        Applying the assumption that $\eta_g\geq\sqrt{\frac{5\gamma}{4\sum_i w_i^2}}$ and taking full expectations completes the proof:
        \begin{align*}
            \E{f(x-\eta_g\Delta x)}
            &\leq \E{f(x)} - \frac{3}{8}\eta\left( 1-\frac{10}{3}\eta L \right) \E{\norm*{\nabla f(x)}^2}+\eta\frac{\rho}{8}\\ 
                & \quad +\eta^2\left(\frac{\rho}{4}+\frac{\sigma^2 }{R\gamma}\sum_{i=1}^n w_i^2\right)L.
        \end{align*}
    \end{proof}
    
    \clearpage
    
    \begin{figure}[!t]
            \centering
            \begin{subfigure}
                \centering
                \includegraphics[width=0.45\textwidth]{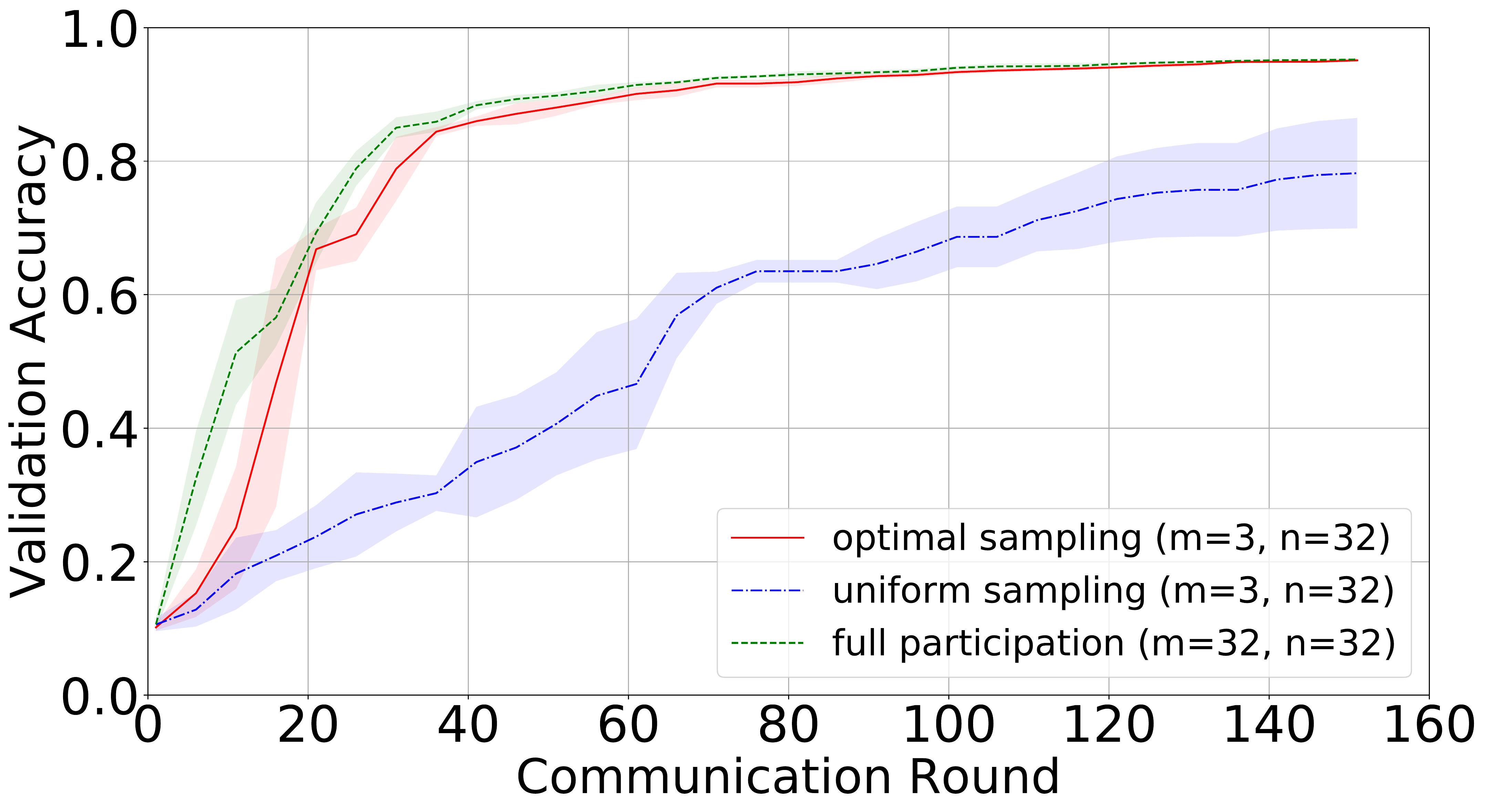}
            \end{subfigure}
            \begin{subfigure}
                \centering
                \includegraphics[width=0.45\textwidth]{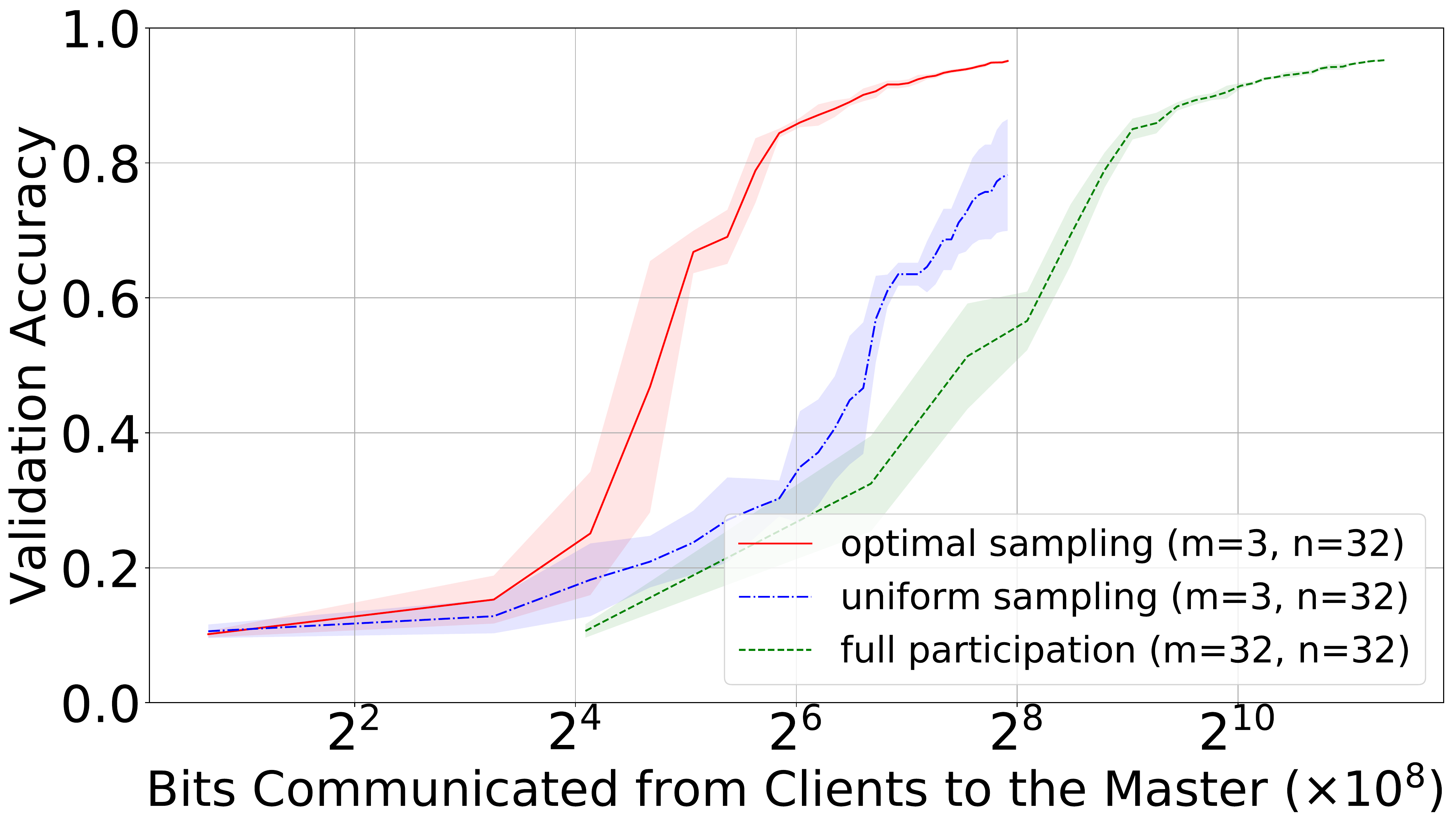}
            \end{subfigure}
            \caption{(FEMNIST Dataset 1) current best validation accuracy as a function of the number of communication rounds and the number of bits communicated from clients to the master.}
            \label{fig:cookup1_best}
        \end{figure}
    
               \begin{figure}[!t]
            \centering
            \begin{subfigure}
                \centering
                \includegraphics[width=0.45\textwidth]{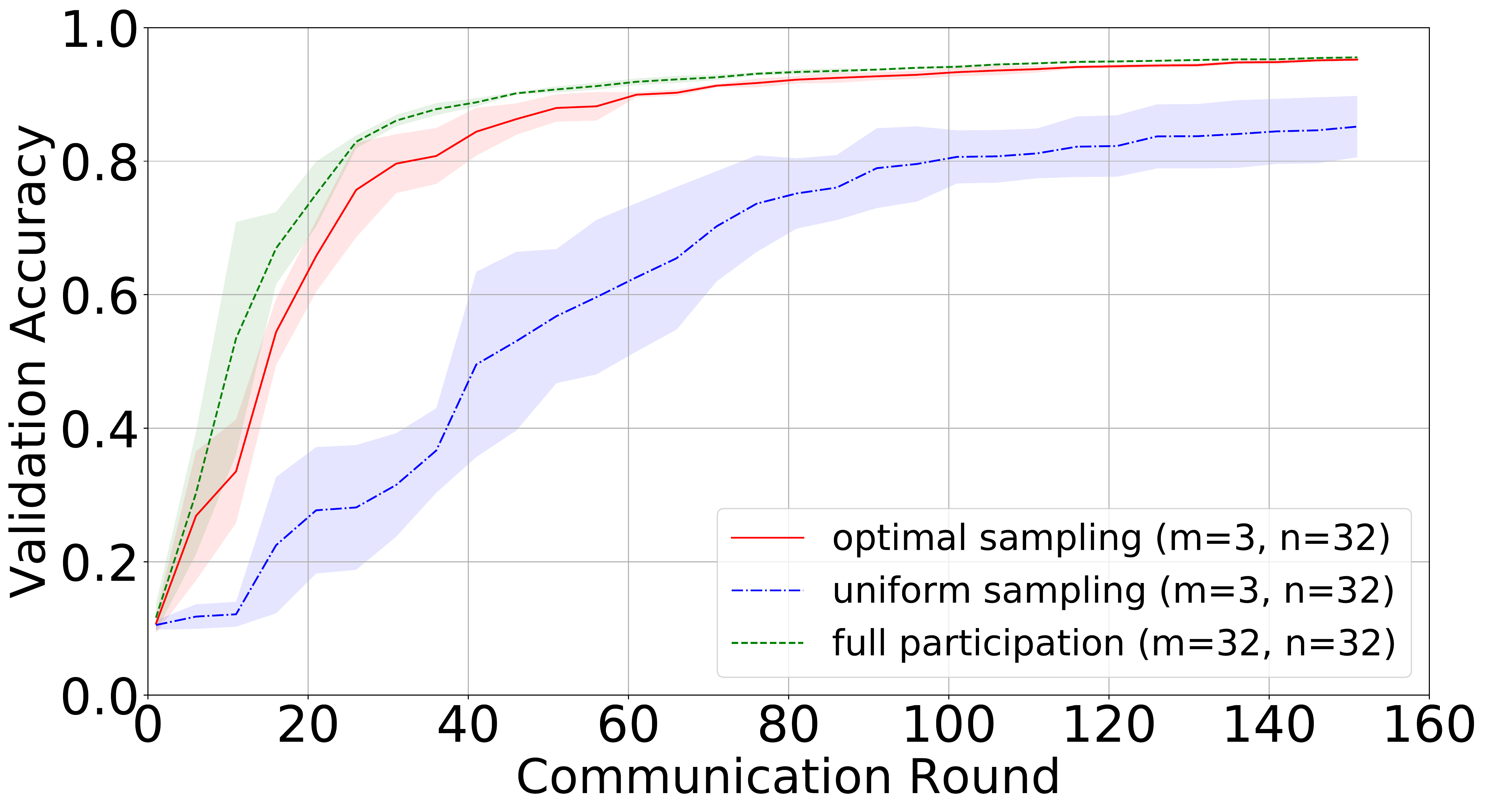}
            \end{subfigure}
            \begin{subfigure}
                \centering
                \includegraphics[width=0.45\textwidth]{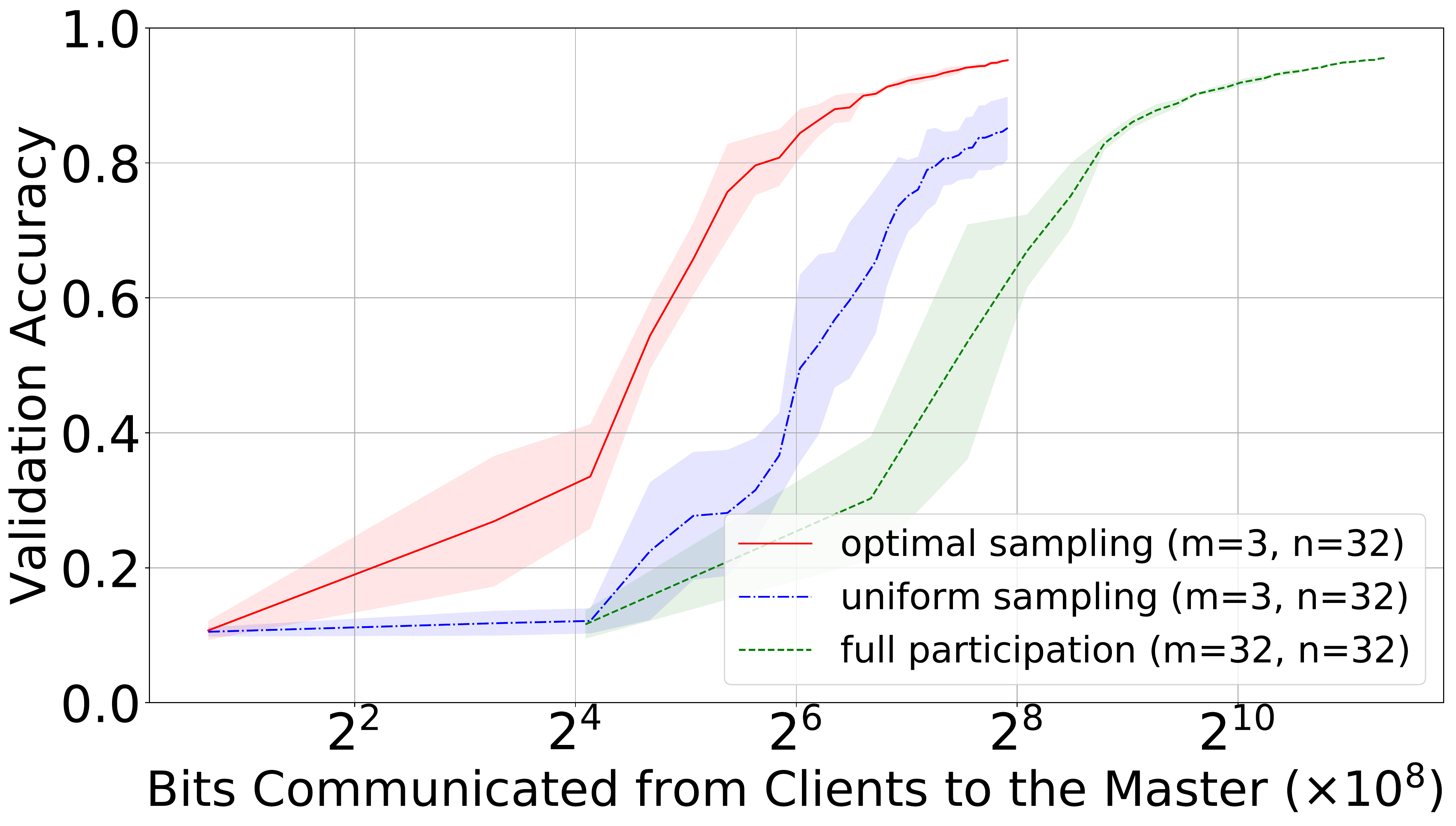}
            \end{subfigure}
            \caption{(FEMNIST Dataset 2) current best validation accuracy as a function of the number of communication rounds and the number of bits communicated from clients to the master.}
            \label{fig:cookup2_best}
        \end{figure}
        
        \begin{figure}[!t]
            \centering
            \begin{subfigure}
                \centering
                \includegraphics[width=0.45\textwidth]{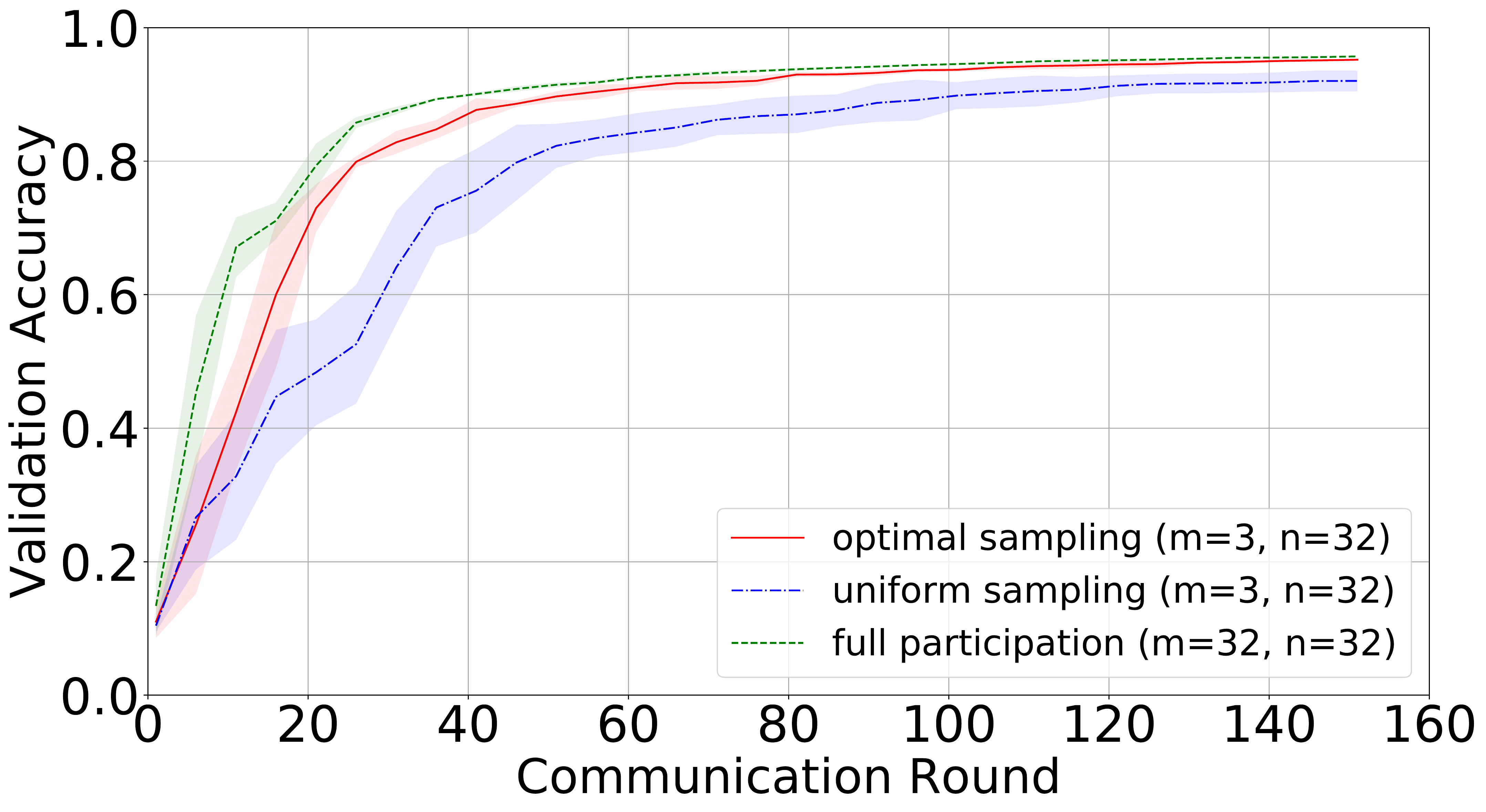}
            \end{subfigure}
            \begin{subfigure}
                \centering
                \includegraphics[width=0.45\textwidth]{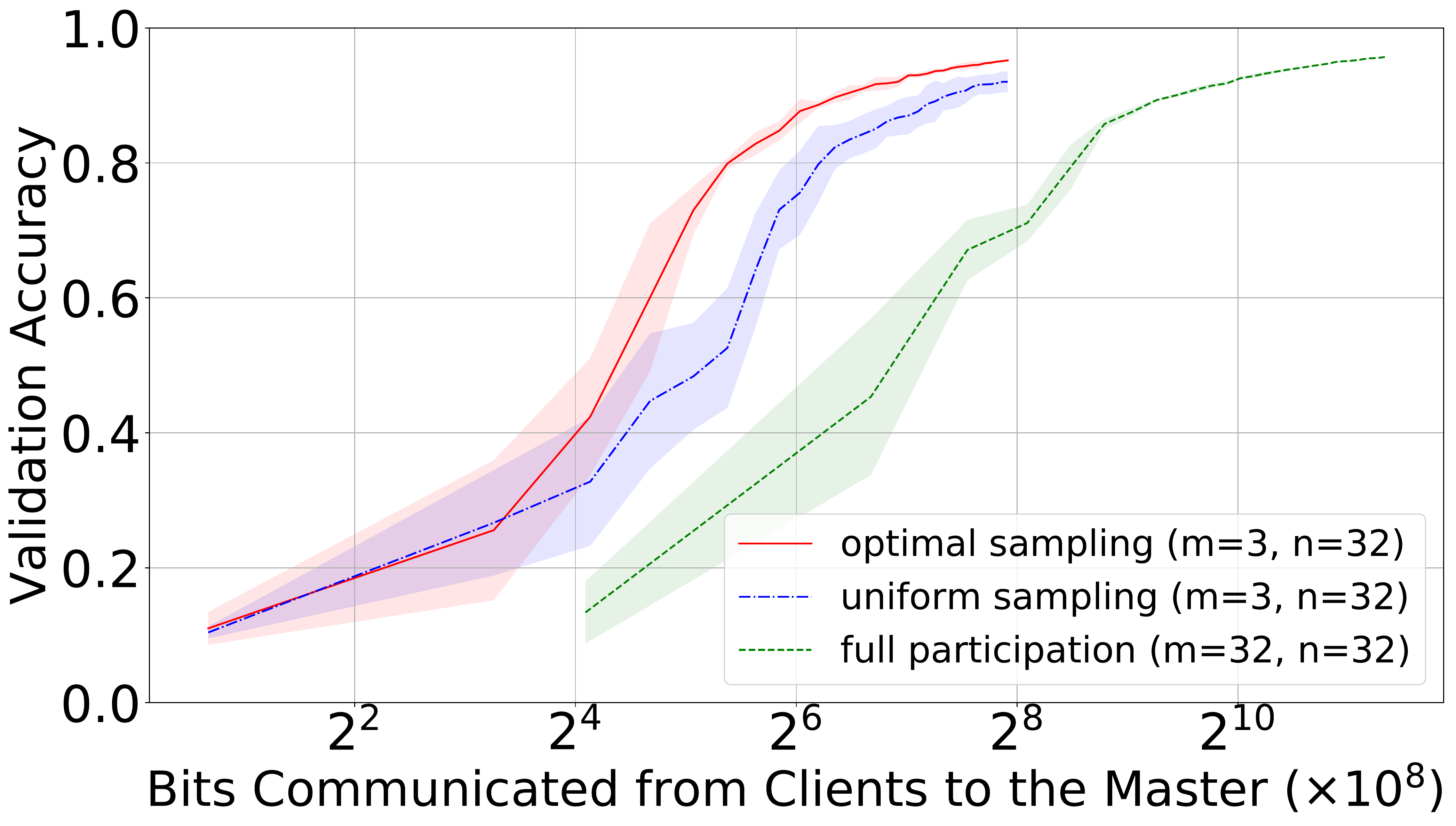}
            \end{subfigure}
            \caption{(FEMNIST Dataset 3) current best validation accuracy as a function of the number of communication rounds and the number of bits communicated from clients to the master.}
            \label{fig:cookup3_best}
        \end{figure}
        
        \begin{figure}[!t]
            \centering
            \begin{subfigure}
                \centering
                \includegraphics[width=0.45\textwidth]{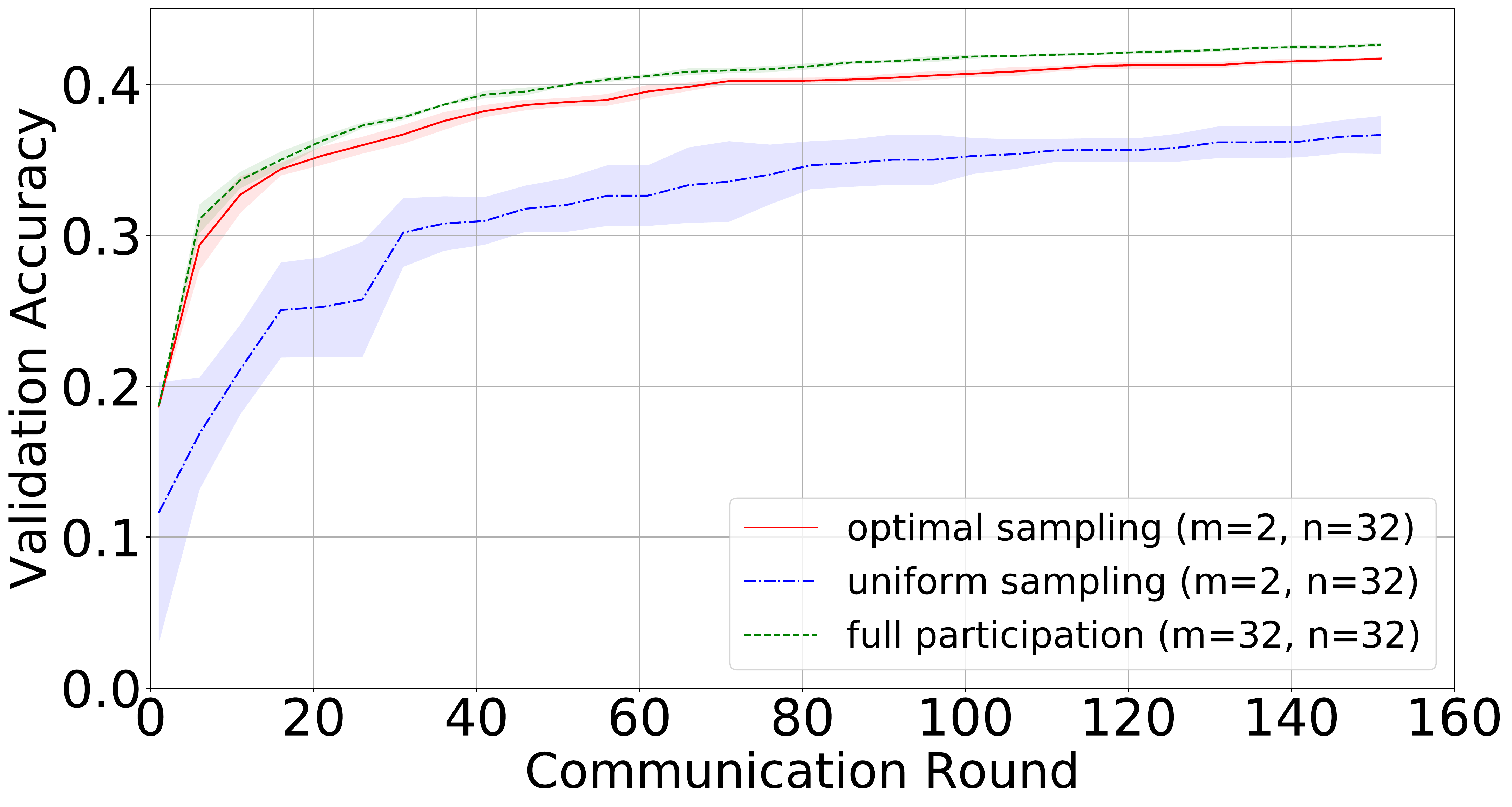}
            \end{subfigure}
            \begin{subfigure}
                \centering
                \includegraphics[width=0.45\textwidth]{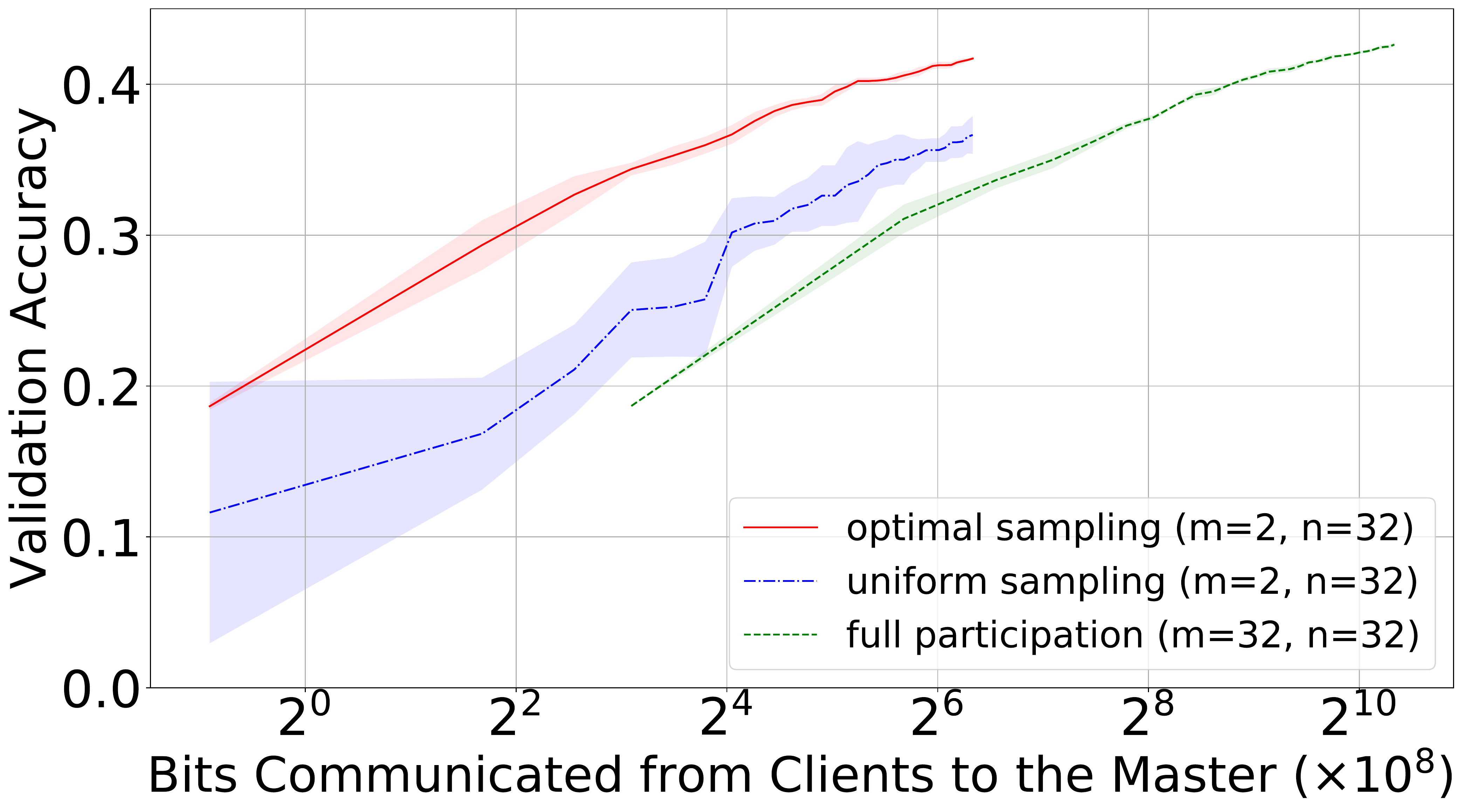}
            \end{subfigure}
            \caption{(Shakespeare Dataset) current best validation accuracy as a function of the number of communication rounds and the number of bits communicated from clients to the master.}
            \label{fig:ss_best}
        \end{figure}
    
    \section{Experimental details}
    
    \subsection{Federated EMNIST dataset}\label{appendix:EMNIST experiment}
        We detail the hyper-parameters used in the experiments on the FEMNIST datasets. For each experiment, we run $151$ communication rounds, reporting (local) training loss every round and validation accuracy every $5$ rounds. In each round, $n=32$ clients are sampled from the client pool, each of which then performs \texttt{SGD} for $1$ epoch on its local training images with batch size $20$. For partial participation, the expected number of clients allowed to communicate their updates back to the master is set to $m=3$. We use vanilla \texttt{SGD} and constant step sizes for all experiments, where we set $\eta_g=1$ and tune $\eta_l$ from the set of value $\{2^{-1},2^{-2},2^{-3},2^{-4},2^{-5}\}$. If the optimal step size hits a boundary value, then we try one more step size by extending that boundary and repeat this until the optimal step size is not a boundary value. For full participation and optimal sampling, it turns out that $\eta_l=2^{-3}$ is the optimal local step size for all three datasets. For uniform sampling, the optimal is $\eta_l=2^{-5}$ for Dataset 1 and $\eta_l=2^{-4}$ for Datasets 2 and 3. For the extra communications in Algorithm \ref{alg:AOCS}, we set $j_{max}=4$.
        
        We also present some additional figures of the experiment results. Figures \ref{fig:cookup1_best}, \ref{fig:cookup2_best} and \ref{fig:cookup3_best} show the current best validation accuracy as a function of the number of communication rounds and the number of bits communicated from clients to the master on Datasets 1, 2 and 3, respectively.
        
    \subsection{Shakespeare dataset}\label{appendix:shakespeare experiment}
        We detail the hyper-parameters used in the experiments on the Shakespeare dataset. For each experiment, we run $151$ communication rounds, reporting (local) training loss every round and validation accuracy every $5$ rounds. In each round, $n=32$ clients are sampled from the client pool, each of which then performs \texttt{SGD} for $1$ epoch on its local training data with batch size $8$ (each batch contains $8$ example sequences of length $5$). For partial participation, the expected number of clients allowed to communicate their updates back to the master is set to $m=2$. We use vanilla \texttt{SGD} and constant step sizes for all experiments, where we set $\eta_g=1$ and tune $\eta_l$ from the set of value $\{2^{-1},2^{-2},2^{-3},2^{-4},2^{-5}\}$. If the optimal step size hits a boundary value, then we try one more step size by extending that boundary and repeat this until the optimal step size is not a boundary value. For full participation and optimal sampling, it turns out that $\eta_l=2^{-2}$ is the optimal local step size. For uniform sampling, the optimal is $\eta_l=2^{-3}$. For the extra communications in Algorithm \ref{alg:AOCS}, we set $j_{max}=4$.
        
        We also present an additional figure of the experiment result. Figure \ref{fig:ss_best} shows the current best validation accuracy as a function of the number of communication rounds and the number of bits communicated from clients to the master.

\refstepcounter{chapter}%
\chapter*{\thechapter \quad Appendix: Fair and accurate federated learning
under heterogeneous targets with ordered dropout}
\label{appendix:fjord}

\section{Proof of Theorem~\ref{thm:od_is_svd}}
\begin{proof}
Let $A =\Hat{U}\Sigma\Hat{V}^\top = \sum_{i=1}^k \sigma_i \hat{u}_i \hat{v}_i^\top$ denote SVD decomposition of $A$. This decomposition is unique as $A$ is full rank with distinct singular values. We also denote $A_b = \sum_{i=1}^b \sigma_i \hat{u}_i \hat{v}_i^\top$


Assuming a linkage of input data $x$ and response $y$ through a linear mapping $Ax=y$, 
we obtain the following
\begin{align*}
    \min_{U,V} \EE{x \sim \mathcal{X},\; p \sim \Dp}{\|\F_p(x) - y \|^2} = \min_{U,V} \EE{x \sim \mathcal{X},\; p \sim \Dp}{\|\F_p(x) - Ax \|^2}.
\end{align*}

Let us denote $u_i$ to be $i$-th row of matrix U and $v_i^\top$ to be $i$-th row of V.  Due to $\mathcal{X}$ being uniform on unit ball and structure of the neural network, we can further simplify the objective to
\begin{align*}
    \min_{U,V} \EE{p \sim \Dp}{\left\| \sum_{i=1}^{\ceil{ p\cdot k}}u_i v_i^\top - A \right\|_F^2},
\end{align*}
 where $F$ denotes Frobenius norm. Since for each $b$ there exists nonzero probability $q_b$ such that $b = \ceil{ p\cdot k} $, we can explicitly compute expectation, which leads to 
 \begin{align*}
    \min_{U,V} \sum_{b=1}^k q_b \left\| \sum_{i=1}^{b}u_i v_i^\top - A \right\|_F^2.
\end{align*}
 
 Realising that $\sum_{i=1}^{b}u_i v_i^\top$ has rank at most $b$, we can use Eckart–Young theorem which implies that
 \begin{align*}
    \min_{U,V} \sum_{b=1}^k q_b \left\| \sum_{i=1}^{b}u_i v_i^\top - A \right\|_F^2 \geq \sum_{b=1}^k q_b \left\| A_b - A \right\|_F^2.
\end{align*}
 
 Equality is obtained if and only if $ A_b =  \sum_{i=1}^{b}u_i v_i^\top$ for all $i \in \{1, 2, \dots, k\}$. This can be achieved, e.g.,~$v_i = \hat{v}_i$ and $u_i = \sigma_i \hat{u}_i$ for all $i \in \{1, 2, \dots, k\}$.
\end{proof}

\section{Experimental details}
\subsection{Datasets and models}
Below, we provide detailed description of the datasets and models used in this paper. We use vision datasets {EMNIST}~\cite{emnist} and its federated equivalent FEMNIST and {CIFAR10}~\cite{cifar}, as well as the language modelling dataset {Shakespeare}~\cite{mcmahan17fedavg}. In the centralised training scenarios, we use the union of dataset partitions for training and validation, while in the federeated deployment, we adopt either a random partitioning in IID datasets or the pre-partitioned scheme available in TensorFlow Federated (TFF)~\cite{tffdatasets}. 
Detailed description of the datasets is provided below.

\textbf{CIFAR10.} The CIFAR10 dataset is a computer vision dataset consisting of $32 \times 32 \times 3$ images with $10$ possible labels. For federated version of {CIFAR10}, we randomly partition dataset among $100$ clients, each client having $500$ data-points. 
We train a ResNet18~\cite{resnet} on this dataset, where for Ordered Dropout, we train independent batch normalization layers for every $p \in \Dp$ as Ordered Dropout affects distribution of layers' outputs. We perform standard data augmentation and preprocessing, \textit{i.e.}\; a random crop to shape $(24, 24, 3)$ followed by a random horizontal flip and then we normalize the pixel values according to their mean and standard deviation. 

\textbf{(F)EMNIST.} {EMNIST} consists of $28 \times 28$ gray-scale images of both numbers and upper and lower-case English characters, with $62$ possible labels in total. The digits are partitioned according to their author, resulting in a naturally heterogeneous federated dataset. {EMNIST} is collection of all the data-points. We do not use any preprocessing on the images. We train a Convolutional Neural Network (CNN), which contains two
convolutional layers, each with $5\times5$ kernels with $10$ and $20$ filters, respectively. Each convolutional layer is followed by a $2 \times 2$ max pooling layer. Finally, the model has a
dense output layer followed by a softmax activation. FEMNIST refers to the federated variant of the dataset, which has been partitioned based on the writer of the digit/character \cite{leaf}.

\textbf{Shakespeare.} {Shakespeare} dataset is also derived from the benchmark designed by \cite{leaf}. The dataset corpus is the collected works of William Shakespeare, and the clients correspond to roles in Shakespeare’s plays with at least two lines of dialogue. Non-federated dataset is constructed as a collection of all the clients' data-points in the same way as for {FEMNIST}. For the preprocessing step, we apply the same technique as TFF dataloader, where we split each client’s lines into sequences of $80$ characters, padding if necessary. We use a vocabulary size of $90$ entities -- $86$ characters contained in Shakespeare’s work, beginning and end of line tokens, padding tokens, and out-of-vocabulary tokens. We perform next-character prediction on the clients’ dialogue using an Recurrent Neural Network (RNN). The RNN takes as input a sequence of $80$ characters, embeds it into a learned $8$-dimensional space, and passes the embedding through two LSTM~\cite{hochreiter1997long} layers, each with $128$ units. Finally, we use a softmax output layer with $90$ units. 
 For this dataset, we don't apply Ordered Dropout to the embedding layer, but only to the subsequent LSTM layers, due to its insignificant impact on the size of the model. 

\subsection{Implementation details}
\label{sec:implementation_details}


\tool was built on top of \ttt{PyTorch}~\cite{pytorch} and  \ttt{Flower}~\cite{beutel2020flower}, an open-source federated learning framework which we extended to support Ordered, Federated, and  Adaptive Dropout and Knowledge Distillation. Our OD aggregation was implemented in the form of a \ttt{Flower} strategy that considers each client maximum width $\pim$. Server and clients run in a multiprocess setup, communicating over gRPC\footnote{\url{https://www.grpc.io/}} channels and can be distributed across multiple devices. To scale to hundreds of clients per cloud node, we optimised \ttt{Flower} so that clients only allocate GPU resources when actively participating in a federated client round. This is accomplished by separating the forward/backward propagation of clients into a separate spawned process which frees its resources when finished. Timeouts are also introduced in order to limit the effect of stragglers or failed client processes to the entire training round.

\subsection{Hyperparameters}

In this section we lay out the hyperparameters used for each $\langle \text{model, dataset, deployment} \rangle$ tuple. 

\subsubsection{Non-federated experiments}

For centralised training experiments, we employ \ttt{SGD} with momentum $0.9$ as an optimiser. We also note that the training epochs of this setting are significantly fewer that the equivalent federated training rounds, as each iteration is a full pass over the dataset, compared to an iteration over the sampled clients.

\begin{itemize}
    \item \textbf{ResNet18.} We use batch size $128$, step size of $0.1$ and train for $300$ epochs. We decrease the step size by a factor of $10$ at epochs $150$ and $225$.
    \item \textbf{CNN.} We use batch size $128$ and train for $20$ epochs. We keep the step size constant at $0.1$.
    \item \textbf{RNN.} We use batch size $32$ and train for $50$ epochs. We keep the step size constant at $0.1$.
\end{itemize}

\subsubsection{Federated experiments}

For each $\langle \text{model, dataset} \rangle$ federated deployment, we start the communication round by randomly sampling $10$ clients to model client availability and for each available client we run one local epoch. We decrease the client step size by $10$ at $50\%$ and $75\%$ of total rounds. We run 500 global rounds of training across experiments and use \ttt{SGD} without momentum.

\begin{itemize}
    \item \textbf{ResNet18.} We use local batch size $32$ and step size of $0.1$. 
    \item \textbf{CNN.} We use local batch size $16$ and step size of $0.1$.
    \item \textbf{RNN.} We use local batch size $4$ and step size of $1.0$.
\end{itemize}

\section{Additional experiments}

\subsubsection*{Federated Dropout vs. eFD}

\begin{figure*}[h]
    
    \centering
    \begin{tabular}{ccc}
        \subfigure[ResNet18 - CIFAR10]{
            \includegraphics[width=0.3\textwidth]{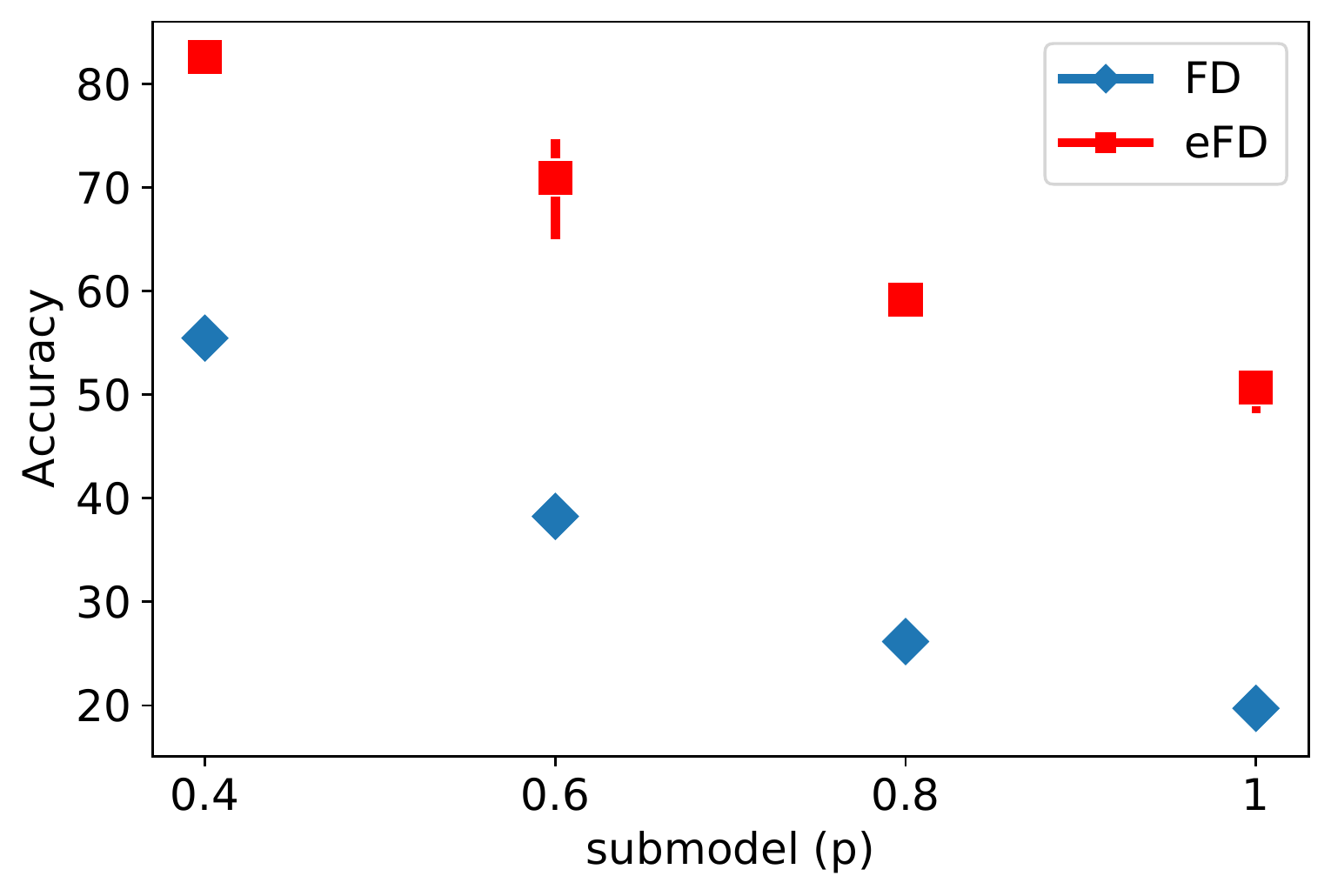}
            }
        \hfill
        \subfigure[CNN - FEMNIST]{
            \includegraphics[width=0.3\textwidth]{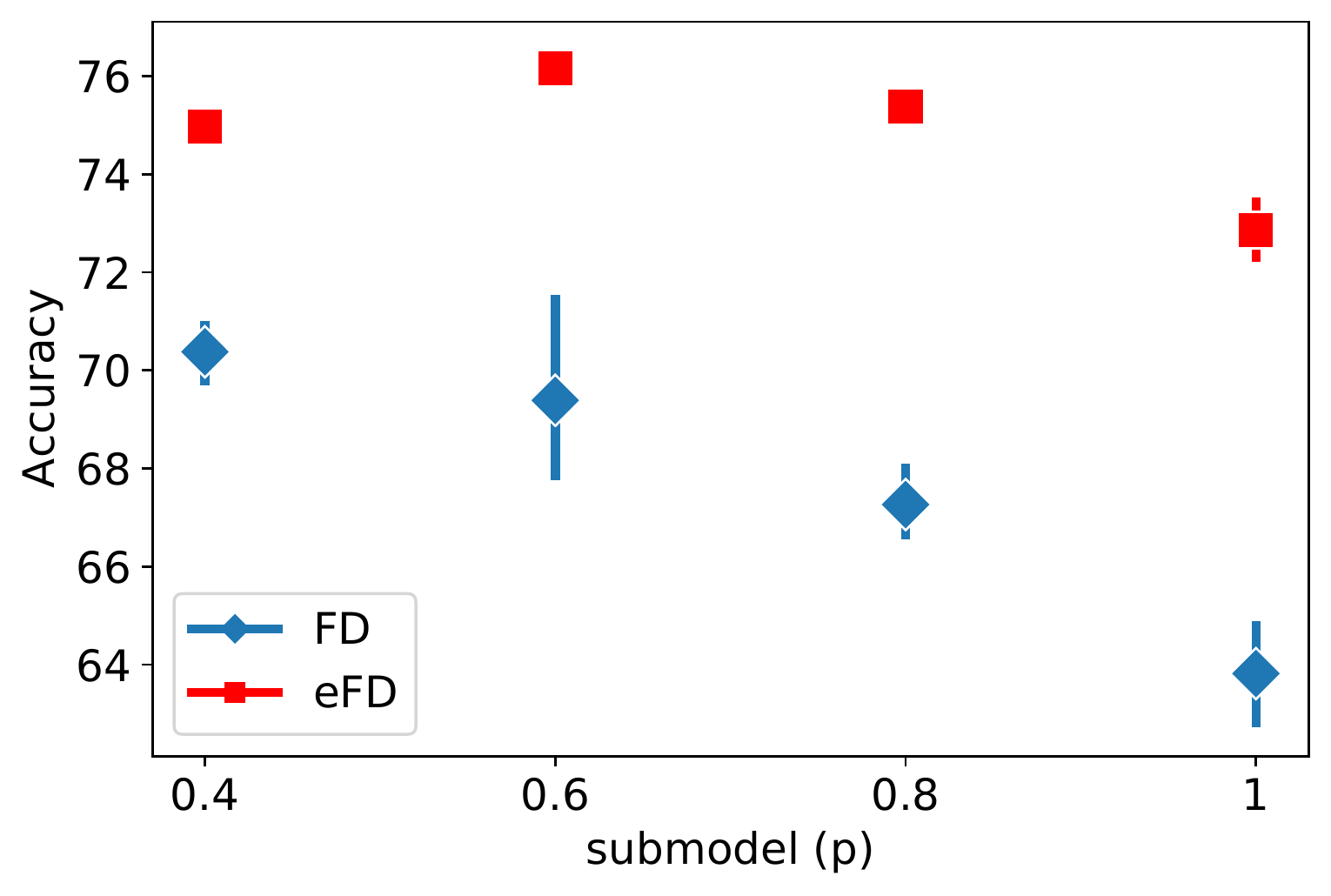}
            }
        \hfill
        \subfigure[RNN - Shakespeare]{
            \includegraphics[width=0.3\textwidth]{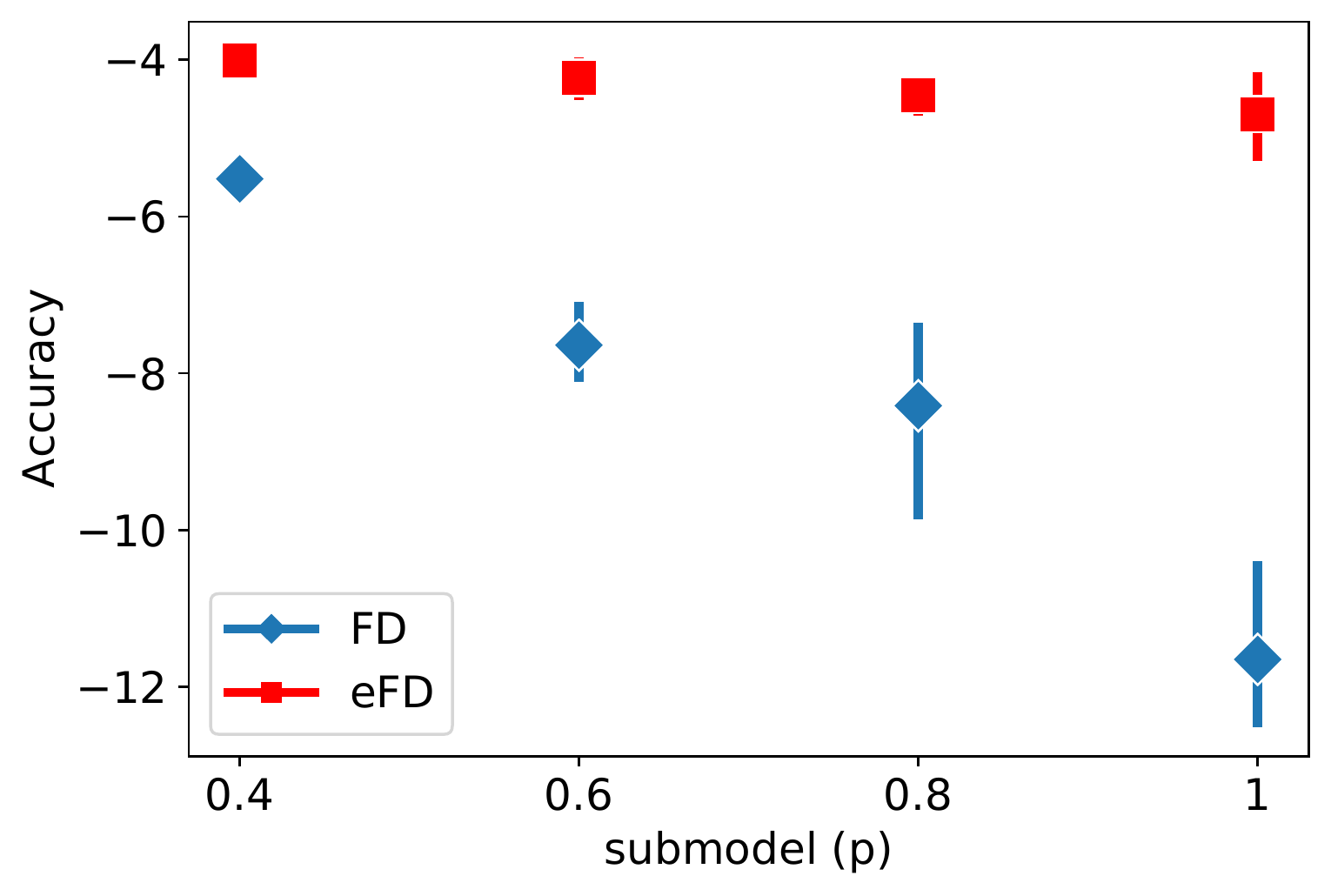}
            }
    \end{tabular}
  
    \caption{Federated Dropout (FD) vs Extended Federated Dropout (eFD). Performance vs dropout rate \textit{p} across different networks and datasets.}
    \label{fig:fd-perf-comparison}

\end{figure*}

In this section we provide evidence of eFD's accuracy dominance over FD. We inherit the setup of the experiment in \S~\ref{sec:perf_eval} to be able to compare results and extrapolate  across similar conditions.
From Fig.~\ref{fig:fd-perf-comparison}, it is clear that eFD's performance dominates the baseline FD by 
27.13-33 percentage points (pp) (30.94pp avg.) on CIFAR10, 4.59-9.04pp (7.13pp avg.) on FEMNIST and 1.51-6.96 points (p) (3.96p avg.) on Shakespeare. The superior performance of eFD, as a technique, can be attributed to the fact that it allows for an adaptive dropout rate based on the device capabilities. As such, instead of imposing a uniformly high dropout rate to accommodate the low-end of the device spectrum, more capable devices are able to update larger portion of the network, thus utilising its capacity more intelligently.

However, it should be also noted that despite FD's accuracy drop, on average it is expected to have a lower computation/upstream network bandwidth/energy impact on devices of the higher end of the spectrum, as they run the largest dropout rate possible to accommodate the computational need of their lower-end counterparts. This behaviour, however, can also be interpreted as wasted computation potential on the higher end -- especially under unconstrained environments (\textit{i.e.}~device charging overnight) -- at the expense of global model accuracy.


\refstepcounter{chapter}%
\chapter*{\thechapter \quad Appendix: Recipes for better use of local work in federated learning}
\label{appendix:fedshuffle}

\begin{figure}[t]
	\centering
    \begin{algorithm}[H]
        \begin{algorithmic}[1]
        \STATE {\bfseries Input:}  initial global model $x^0$, global and local step sizes $\eta_g^r$, $\eta_l^r$, proper distribution $\cS$ 
        \FOR{each round $r=0,\dots,R-1$}
            \STATE master broadcasts $x^{r}$ to all clients $i\in \mathcal{S}^r \sim S$
            \FOR{each client $i\in \cS^r$ (in parallel)}
                \STATE initialize local model $y_{i, 0, 0}^r\leftarrow x^{r}$
                \FOR{$e=1,\dots, E$}
                    \STATE Sample permutation $\{\Pi^r_{i, e, 0}, \hdots, \Pi^k_{i, e, |\cD_i|-1}\}$ of $\{1, \hdots, |\cD_i| \}$
                    \FOR{$j=1,\dots, |\cD_i|$}
                        \STATE update $y_{i, e, j}^r = y_{i, e, j-1}^r - \nicefrac{\eta_l^r}{|\cD_i|} \nabla f_{i\Pi^r_{i, e, j-1}}(y_{i, e, j-1}^r)$
                    \ENDFOR
                    \STATE $y_{i, e+1, 0}^r = y_{i, e, |\cD_i|}^r$
                \ENDFOR
                \STATE send $\Delta_i^r = y_{i,E, |\cD_i|}^r - x^r$ to master
            \ENDFOR
           \STATE  master computes $\Delta^r = \sum_{i\in \cS^r} \frac{w_i}{p_i}\Delta_i^r$
            \STATE master updates global model $x^{r+1} = x^{r} - \eta_g^r \Delta^r$
        \ENDFOR
        \end{algorithmic}
        \caption{\fedshuffle}
        \label{alg:FedShuffle}
    \end{algorithm}
\end{figure}

\section{Algorithm~\ref{alg:FedShuffle}: Convergence analysis (proof of Theorem~\ref{thm:fedavg-full})}

The style of our proof technique is related to the analysis of \fedavg\ of \cite{karimireddy2020scaffold}.  We start with proof for convex functions. By $\EE{r}{\cdot}$, we denote the expectation conditioned on the all history prior to communication round $r$. 

\begin{lemma}\textbf{\em (one round progress)}\label{lem:fedavg-progress}
Suppose Assumptions \ref{ass:strong-convexity} -- \ref{ass:bounded_var} hold. For any constant step sizes $\eta_l^r \eqdef \eta_l$ and $\eta_l^r \eqdef \eta_l$  satisfying
$\eta_l \leq \frac{1}{(1 + MB^2)4LE\eta_g}$ and effective step size
$\tilde\eta \eqdef E \eta_g \eta_l$, the updates of \fedavg\ satisfy
\[
    \E{\norm{x^{r} - x^\star}^2 }\leq \rbr*{1 - \frac{\mu
      \tilde\eta}{2}} \E {\norm{x^{r-1} - x^\star}^2} - \tilde\eta \EE{r-1}{f(x^{r-1}) - f^\star}  \\
      +3L \tilde\eta\xi^r + 2\tilde\eta^2 MG^2\,,
\]
where $\xi_{r}$ is the drift caused by the local updates on the
clients defined to be
\[
    \xi^{r} \eqdef \frac{1}{|\cD|E}\sum_{e=1}^{E}\sum_{i=1}^n \sum_{j=1}^{|\cD_i|}
    \EE{r-1}{\norm{y_{i,e, j-1}^r - x^{r-1}}^2}\,
\]
and $M \eqdef \max_{i \in [n]} \cbr*{\frac{v_i}{p_i} w_i}$.
\end{lemma}
\begin{proof}
For a better readability of the proofs in one round progress, we drop the superscript that represents the current completed communication round $r-1$.

By the definition in Algorithm~\ref{alg:FedShuffle}, the update $\Delta$ can be written as 
\[
    \Delta = - \eta_g \sum_{i \in \cS} \frac{w_i}{p_i} \Delta_i = -\frac{\tilde\eta}{E|\cD|} \sum_{i \in \cS} \sum_{e=1}^E \sum_{j=1}^{|\cD_i|} \frac{1}{p_i}\nabla f_{i\Pi_{i,e,j-1}}(y_{i,e,j-1})\,.
\]
We adopt the convention that summation $\sum_{i \in \cM ,e,j}$ ($\cM$ is either $[n]$ or $\cS$) refers to the summations $\sum_{i \in \cM} \sum_{e=1}^E \sum_{j=1}^{|\cD_i|}$ unless otherwise stated. Furthermore, we denote $g_{i,e,j} \eqdef \nabla f_{i\Pi_{i,e,j-1}}(y_{i,e,j-1})$. Using above, we proceed as
\begin{align*}
    \EE{r-1}{\norm{x + \Delta - x^\star}^2} &= \norm{x - x^\star}^2 \underbrace{-2\EE{r-1}{\frac{\tilde\eta}{E|\cD|}\sum_{i \in \cS ,e,j}\frac{1}{p_i}\inp{g_{i,e,j}}{x - x^\star}}}_{\cA_1}  \\
    & \quad + \underbrace{\tilde\eta^2\EE{r-1}{\norm*{\frac{1}{E|\cD|}\sum_{i \in \cS ,e,j}\frac{1}{p_i} g_{i,e,j}}^2}}_{\cA_2}\,.
\end{align*}
To bound the term $\cA_1$, we apply Lemma~\ref{lem:magic} to each term of the summation with $h = f_{ij}$, $x = y_{i,e,j-1}$, $y = x^\star$, and $z = x$. Therefore, 
\begin{align*}
    \cA_1 &= -\EE{r-1}{\frac{2\tilde\eta}{E|\cD|}\sum_{i \in \cS ,e,j} \frac{1}{p_i} \inp{g_{i,e,j}}{x - x^\star}} \\
    &\leq \EE{r-1}{\frac{2\tilde\eta}{E|\cD|}\sum_{i \in \cS ,e,j} \frac{1}{p_i} \rbr*{f_{i\Pi_{i,e,j-1}}(x^\star) - f_{i\Pi_{i,e,j-1}}(x) + L \norm{y_{i,e,j-1} - x}^2}}\\
    & \quad 2\tilde\eta\frac{\mu}{4}\norm{x - x^\star}^2\\
    &= -2\tilde\eta\rbr*{f(x) - f^\star + \frac{\mu}{4}\norm{x - x^\star}^2} + 2L \tilde\eta\xi\,.
\end{align*}
    For the second term $\cA_2$, we have
    \begin{align*}
        \cA_2 &= \tilde\eta^2\EE{r-1}{\norm*{\frac{1}{E|\cD|}\sum_{i \in \cS , e, j} \frac{1}{p_i} g_{i, e, j}}^2}\\
        &\overset{\eqref{eq:var_decomposition}}{\leq} \frac{\tilde\eta^2}{E^2|\cD|^2}\EE{r-1}{\norm*{\sum_{i \in \cS ,e,j} \frac{1}{p_i} g_{i,e,j} - \sum_{i \in [n] ,e,j} g_{i,e,j}}^2 + \norm*{\sum_{i \in [n] ,e,j} g_{i,e,j}}^2} \\
        &\overset{\eqref{eq:key_inequality_app}}{\leq} \frac{\tilde\eta^2}{E^2|\cD|^2}\EE{r-1}{\sum_{i \in [n]}\frac{v_i}{p_i}\norm*{ \sum_{e, j} g_{i,e,j}}^2 + \norm*{\sum_{i \in [n] ,e,j} g_{i,e,j}}^2} \\
        &\overset{\eqref{eq:sum_upper}}{\leq} \frac{2\tilde\eta^2}{E^2|\cD|^2}\EE{r-1}{\sum_{i \in [n]}\frac{v_i}{p_i}\norm*{ \sum_{e, j} g_{i,e,j} - \nabla f_{i\Pi_{i,e,j-1}}(x) }^2} + 2\tilde\eta^2\sum_{i \in [n]}\frac{v_i}{p_i} w_i^2\norm*{\nabla f_{i}(x)}^2 \\
        &\quad + \frac{2\tilde\eta^2}{E^2|\cD|^2}\EE{r-1}{\norm*{\sum_{i \in [n] ,e,j} g_{i,e,j} - \nabla f_{i\Pi_{i,e,j-1}}(x)}^2 } + 2\tilde\eta^2 \norm*{\nabla f(x)}^2\\
        &\overset{\eqref{eq:sum_upper}}{\leq} \frac{2\tilde\eta^2}{E|\cD|}\EE{r-1}{\sum_{i \in [n], e, i}\frac{v_i}{p_i}w_i\norm*{ g_{i,e,j} - \nabla f_{i\Pi_{i,e,j-1}}(x) }^2} + 2\tilde\eta^2\sum_{i \in [n]}\frac{v_i}{p_i} w_i^2\norm*{\nabla f_{i}(x)}^2 \\
        &\quad + \frac{2\tilde\eta^2}{E|\cD|}\sum_{i \in [n] ,e,j}\EE{r-1}{\norm*{g_{i,e,j} - \nabla f_{i\Pi_{i,e,j-1}}(x)}^2 } + 2\tilde\eta^2 \norm*{\nabla f(x)}^2\\
        &\overset{\eqref{eq:smooth}}{\leq} 2 \max_{i \in [n]} \cbr*{\frac{v_i}{p_i} w_i}\tilde\eta^2 L^2 \xi + 2\tilde\eta^2\max_{i \in [n]} \cbr*{\frac{v_i}{p_i} w_i} \sum_{i \in [n]}w_i\norm*{\nabla f_{i}(x)}^2 \\
        &\quad + 2\tilde\eta^2L^2\xi + 2\tilde\eta^2 \norm*{\nabla f(x)}^2\\
        &\overset{\eqref{eq:heterogeneity}}{\leq} 2 \rbr*{1 + \max_{i \in [n]} \cbr*{\frac{v_i}{p_i} w_i}}\tilde\eta^2 L^2 \xi + 2\tilde\eta^2\max_{i \in [n]} \cbr*{\frac{v_i}{p_i} w_i} G^2 \\
        &\quad + 2\tilde\eta^2 \rbr*{1 + \max_{i \in [n]} \cbr*{\frac{v_i}{p_i} w_i}B^2}  \norm*{\nabla f(x)}^2\\
        &\overset{\eqref{eqn:smoothness}}{\leq} 2 \rbr*{1 + \max_{i \in [n]} \cbr*{\frac{v_i}{p_i} w_i}}\tilde\eta^2 L^2 \xi + 2\tilde\eta^2\max_{i \in [n]} \cbr*{\frac{v_i}{p_i} w_i} G^2 \\
        &\quad + 4L\tilde\eta^2 \rbr*{1 + \max_{i \in [n]} \cbr*{\frac{v_i}{p_i} w_i}B^2}  (f(x) - f^\star).
    \end{align*}
    Recall $M \eqdef \max_{i \in [n]} \cbr*{\frac{v_i}{p_i} w_i}$, therefore by  plugging back the bounds on $\cA_1$ and $\cA_2$,
    \begin{multline*}
        \EE{r-1}{\norm{x + \Delta - x^\star}^2} \leq (1 - \frac{\mu\tilde\eta}{2})\norm{x - x^\star}^2 - (2\tilde\eta - 4L\tilde\eta^2(MB^2 + 1))(f(x) - f^\star)  \\
      +(1 + (1 + M)\tilde\eta L)2L \tilde\eta\xi + 2\tilde\eta^2 MG^2\,.
    \end{multline*}
    The lemma now follows by observing that $4L\tilde\eta(MB^2 + 1) \leq 1$ and that $B \geq 1$.
\end{proof}

\begin{lemma}\textbf{\em (bounded drift)}\label{lem:fedavg-error-bound}
Suppose Assumptions \ref{ass:smoothness} -- \ref{ass:bounded_var} hold. Then the updates of \fedavg\ with reshuffling
for any step size satisfying \[ \eta_l \leq \frac{1}{(1 + (P+1) B + M B^2))4L E\eta_g}\]
have bounded drift:
\[
    3L\tilde\eta \xi_{r} \leq  \frac{9}{10} \tilde\eta(f(x^r) - f^\star) + 9\frac{\tilde\eta^3}{\eta_g^2 E^2} \rbr*{\rbr*{E^2 + P^2}G^2 + \sigma^2}\,, 
\]
where $P^2 \eqdef \max_{i \in [n]} \frac{P_i^2}{3|\cD_i|}$, $\sigma^2 \eqdef \frac{1}{3|\cD|}\sum_{i \in [n]} \sigma_i^2$ and $M \eqdef \max_{i \in [n]} \cbr*{\frac{v_i}{p_i} w_i}$.
\end{lemma}
\begin{proof}
We adopt the same convention as for the previous proof, i.e., dropping superscripts, simplifying sum notation and having $g_{i,e,j} \eqdef \nabla f_{i\Pi_{i,e,j-1}}(y_{i,e,j-1})$. Therefore,
\begin{align*}
\xi &=  \frac{1}{|\cD|E} \sum_{i \in [n], i, j} \EE{r-1}{\norm{y_{i,e, j-1}^r - x^{r-1}}^2}\\
    &= \frac{\eta_l^2}{|\cD|E} \sum_{i \in [n], e, j} \frac{1}{|\cD_i|^2} \EE{r-1}{\norm*{\sum_{l=0}^{e-1} \sum_{c=1}^{\cD_i} g_{i, l, c} + \sum_{c=1}^j g_{i, e, c}}^2}\\
    &\overset{\eqref{eq:triangle}}{\leq} \frac{2\eta_l^2}{|\cD|E} \sum_{i \in [n], e, j} \frac{1}{|\cD_i|^2} \\
    & \quad  \EE{r-1}{\norm*{\sum_{l=0}^{e-1} \sum_{c=1}^{\cD_i} g_{i, l, c} - (e-1) |\cD_i| \nabla f_i (x) + \sum_{c=1}^j g_{i, e, c} - \nabla f_{i\Pi_{i,e,c-1}}(x)}^2} \\
    &\quad + \frac{2\eta_l^2}{|\cD|E} \sum_{i \in [n], e, j} \frac{1}{|\cD_i|^2} \EE{r-1}{\norm*{(e-1) |\cD_i| \nabla f_i (x) + \sum_{c=1}^j \nabla f_{i\Pi_{i,e,c-1}}(x)}^2}.
\end{align*}
Next,  we upper bound each term separately. For the first term, we have
\begin{align*}
    & \frac{2\eta_l^2}{|\cD|E} \sum_{i \in [n], e, j} \frac{1}{|\cD_i|^2}  \\
    & \quad  \EE{r-1}{\norm*{\sum_{l=0}^{e-1} \sum_{c=1}^{\cD_i} g_{i, l, c} - (e-1) |\cD_i| \nabla f_i (x) + \sum_{c=1}^j g_{i, e, c} - \nabla f_{i\Pi_{i,e,c-1}}(x)}^2} \\
     &\overset{\eqref{eq:sum_upper}}{\leq} \frac{2\eta_l^2}{|\cD|E} \sum_{i \in [n], e, j} \frac{(e-1) |\cD_i| + j}{|\cD_i|^2} \\
    & \quad  \EE{r-1}{\rbr*{\sum_{l=0}^{e-1} \sum_{c=1}^{\cD_i}\norm*{ g_{i, l, c} - \nabla f_{i\Pi_{i,e,j-1}}(x)}^2 + \sum_{c=1}^j \norm*{g_{i, e, c} - \nabla f_{i\Pi_{i,e,c-1}}(x)}^2}} \\
     &\overset{\eqref{eq:smooth}}{\leq} 2\eta_l^2 L^2 E^2 \xi. \\
\end{align*}
For the second term,
\begin{align*}
    & \frac{2\eta_l^2}{|\cD|E} \sum_{i \in [n], e, j} \frac{1}{|\cD_i|^2} \EE{r-1}{\norm*{(e-1) |\cD_i| \nabla f_i (x) + \sum_{c=1}^j \nabla f_{i\Pi_{i,e,c-1}}(x)}^2} \\
    &\overset{\eqref{eq:var_decomposition}}{\leq} \frac{2\eta_l^2}{|\cD|E} \sum_{i \in [n], e, j} \frac{1}{|\cD_i|^2}  \\
    & \quad \EE{r-1}{\norm*{((e-1) |\cD_i| + j) \nabla f_i (x)}^2 + \norm*{\sum_{c=1}^j \nabla f_{i\Pi_{i,e,c-1}}(x) - j\nabla f_i (x)}^2}\\
    &\overset{\eqref{eq:sampling_wo_replacement}}{\leq} \frac{2\eta_l^2}{|\cD|E} \sum_{i \in [n], e, j} \frac{1}{|\cD_i|^2}  \\
    & \quad \rbr*{((e-1) |\cD_i| + j)^2\norm*{\nabla f_i (x)}^2 + \frac{j(|\cD_i| - j)}{(|\cD_i| - 1)} \frac{1}{|\cD_i|}\sum_{c=1}^{\cD_i} \norm*{\nabla f_{ic}(x) - \nabla f_i (x)}^2}\\
    &\overset{\eqref{eqn:bounded_var}}{\leq} \frac{2\eta_l^2}{|\cD|E} \sum_{i \in [n], e, j} \frac{1}{|\cD_i|^2} \\
    & \quad  \rbr*{((e-1) |\cD_i| + j)^2\norm*{\nabla f_i (x)}^2 + \frac{j(|\cD_i| - j)}{(|\cD_i| - 1)}(\sigma^2 + P^2 \norm*{\nabla f_i (x)}^2)}\\
    &\leq \frac{2\eta_l^2}{|\cD|E} \sum_{i \in [n]} \frac{|\cD_i|^3 E^3}{|\cD_i|^2} \norm*{\nabla f_i (x)}^2 + \frac{E(|\cD_i + 1|)|\cD_i|}{6|\cD_i|^2}(\sigma_i^2 + P_i^2 \norm*{\nabla f_i (x)}^2)\\
    &\leq 2\eta_l^2 E^2\sum_{i \in [n]} \rbr*{1 + \frac{P_i^2}{3|\cD_i|E^2}} w_i\norm*{\nabla f_i (x)}^2 + \frac{\sigma^2_i}{3|\cD|E^2}\\
     &\overset{\eqref{eq:heterogeneity-convex}}{\leq} 4LB^2 \eta_l^2 E^2\rbr*{1 + P^2} (f(x) - f^\star) + 2\eta_l^2 \rbr*{\rbr*{E^2 + P^2}G^2 + \sigma^2}.
\end{align*}
Combining the upper bounds
\begin{align*}
\xi &\leq  2\eta_l^2 L^2 E^2 \xi + 4LB^2 \eta_l^2 E^2\rbr*{1 + P^2} (f(x) - f^\star) + 2\eta_l^2 \rbr*{\rbr*{E^2 + P^2}G^2 + \sigma^2}.
\end{align*}
Since $2\eta_l^2 L^2 E^2 \leq \frac{1}{8}$ and $4LB^2 \eta_l^2 E^2\rbr*{1 + P^2} \leq \frac{1}{4L}$, therefore
\begin{align*}
3L\tilde\eta \xi &\leq  \frac{9}{10} \tilde\eta(f(x) - f^\star) + 9\tilde\eta \eta_l^2 \rbr*{\rbr*{E^2 + P^2}G^2 + \sigma^2},
\end{align*}
which concludes the proof.
\end{proof}

Adding the statements of the above Lemmas \ref{lem:fedavg-progress} and \ref{lem:fedavg-error-bound}, we get
\begin{align*}
    \E{\norm{x^{r} - x^\star}^2 }&\leq (1 - \frac{\mu\tilde\eta}{2}) \E {\norm{x^{r-1} - x^\star}^2} - \frac{\tilde\eta}{10} \EE{r-1}{f(x^{r-1}) - f^\star}  \\
    & \quad + 2\tilde\eta^2 MG^2 + \frac{9\tilde\eta^3}{\eta_g^2 E^2} \rbr*{\rbr*{E^2 + P^2}G^2 + \sigma^2} \\
      &= (1 - \frac{\mu\tilde\eta}{2}) \E {\norm{x^{r-1} - x^\star}^2} - \frac{\tilde\eta}{10} \EE{r-1}{f(x^{r-1}) - f^\star}  \\
    & \quad + 2\tilde\eta^2 \rbr*{MG^2 + \frac{9\tilde\eta}{2\eta_g^2 E^2} \rbr*{\rbr*{E^2 + P^2}G^2 + \sigma^2}}\,,
\end{align*}
Moving the $(f(x^{r-1}) - f(x^\star))$ term and dividing both sides by $\frac{\tilde\eta}{10}$, we get the following bound for any $\tilde\eta \leq \frac{1}{(1 + (P+1) B + M B^2))4L }$
\begin{align*}
    \EE{r-1}{f(x^{r-1}) - f^\star} &\leq  \frac{10}{\tilde\eta}(1 - \frac{\mu\tilde\eta}{2}) \E {\norm{x^{r-1} - x^\star}^2}  \\
    & \quad + 20\tilde\eta \rbr*{MG^2 + \frac{9\tilde\eta}{2\eta_g^2 E^2} \rbr*{\rbr*{E^2 + P^2}G^2 + \sigma^2}}\,.
\end{align*}
If $\mu = 0$ (weakly-convex), we can directly apply Lemma~\ref{lemma:general}. Otherwise, by averaging using weights $v_r = (1  - \frac{\mu \tilde\eta}{2})^{1-r}$ and using the same weights to pick output $\bar{x}^R$, we can simplify the above recursive bound (see proof of Lem.~\ref{lem:constant}) to prove that for any $ \tilde\eta$ satisfying $\frac{1}{\mu R} \leq \tilde\eta \leq \frac{1}{(1 + (P+1) B + M B^2))4L }$
\begin{align*}
    \E{f(\bar{x}^{R})] - f(x^\star)} &\leq \underbrace{10 \norm{x^0 - x^\star}^2}_{\eqdef d}\mu\exp(-\frac{\tilde\eta}{2}\mu R) + \tilde\eta \rbr[\big]{\underbrace{4MG^2}_{\eqdef c_1}}  \\
    & \quad +  \tilde\eta^2 \underbrace{\rbr*{\frac{18}{\eta_g^2 E^2} \rbr*{\rbr*{E^2 + P^2}G^2 + \sigma^2}}}_{\eqdef c_2}
\end{align*}
Now, the choice of $\tilde \eta =\min\cbr*{ \frac{\log(\max(1,\mu^2 R d/c_1))}{\mu R}, \frac{1}{(1 + (P+1) B + M B^2))4L }}$ yields the desired rate. 

For the non-convex case, one first exploits the smoothness assumption (Assumption~\ref{ass:smoothness}) (extra smoothness term $L$ in the first term in the convergence guarantee) and the rest of the proof follows essentially in the same steps as the provided analysis. The only difference is that distance to an optimal solution is replaced by functional difference, i.e., $\norm{x^0 - x^\star}^2 \rightarrow f(x^0) - f^\star$. The final convergence bound also relies on Lemma~\ref{lemma:general}. For completeness, we provide the proof below.

We adapt the same notation simplification as for the prior cases, see proof of Lemma~\ref{lem:fedavg-progress}. Since $\cbr{f_{ij}}$ are $L$-smooth then $f$ is also $L$ smooth. Therefore,

\begin{eqnarray*}
    &&\EE{r-1}{f(x + \Delta)} \\
    &\leq& f(x) + \EE{r-1}{\inp*{\nabla f(x)}{\Delta}} + \frac{L}{2}\EE{r-1}{\norm*{\Delta}^2} \\
    &=& f(x) - \EE{r-1}{\inp*{\nabla f(x)}{\frac{\tilde \eta}{E|\cD|}\sum_{i \in [n] ,e,j} g_{i,e,j}}} + \frac{L\tilde\eta^2}{2}\EE{r-1}{\norm*{\frac{1}{E|\cD|}\sum_{i \in \cS ,e,j}\frac{1}{p_i} g_{i,e,j}}^2} \\
    &\overset{\eqref{eq:triangle}}{\leq}& f(x) - \frac{\tilde \eta}{2} \norm*{\nabla f(x)}^2 + \frac{\tilde \eta}{2} \EE{r-1}{\norm*{\frac{1}{E|\cD|}\sum_{i \in [n] ,e,j} g_{i,e,j} - \nabla f_{i\Pi_{i,e,j-1}}(x) }^2}\\
    &&\quad  + \frac{L\tilde\eta^2}{2}\EE{r-1}{\norm*{\frac{1}{E|\cD|}\sum_{i \in \cS ,e,j}\frac{1}{p_i} g_{i,e,j}}^2} \\
    &\overset{\eqref{eq:smooth} + \eqref{eq:sum_upper}}{\leq}& f(x) - \frac{\tilde \eta}{2} \norm*{\nabla f(x)}^2 + \frac{\tilde \eta L^2}{2} \xi + \frac{L\tilde\eta^2}{2}\EE{r-1}{\norm*{\frac{1}{E|\cD|}\sum_{i \in \cS ,e,j}\frac{1}{p_i} g_{i,e,j}}^2}\;.
\end{eqnarray*}
We upper-bound the last term using the bound of $\cA_2$ in the proof of Lemma~\ref{lem:fedavg-progress} (note that this proof does not rely on the convexity assumption). Thus, we have
\begin{align*}
    \EE{r-1}{f(x + \Delta)}  &\leq f(x) - \frac{\tilde \eta}{2} \norm*{\nabla f(x)}^2 + \frac{\tilde \eta L^2}{2} \xi + \rbr*{1 + M}\tilde\eta^2 L^3 \xi  \\
    & \quad + \tilde\eta^2 M G^2 L + \tilde\eta^2 \rbr*{1 + MB^2}L  \norm*{\nabla f(x)}^2\;.
\end{align*}
The bound on the step size $\tilde \eta \leq \frac{1}{(1 + (P+1) B + M B^2))4L}$ implies
\begin{align*}
    \EE{r-1}{f(x + \Delta)}  &\leq f(x) - \frac{\tilde \eta}{4} \norm*{\nabla f(x)}^2 + \frac{3\tilde \eta L^2}{4} \xi + \tilde\eta^2 M G^2 L\;.
\end{align*}
Next, we reuse the partial result of Lemma~\ref{lem:fedavg-error-bound} that does not require convexity, i.e., we replace $f(x) - f^\star$ with $\frac{1}{2L}\norm{\nabla f(x)}^2$. Therefore,
\begin{align*}
    \EE{r-1}{f(x + \Delta)}  &\leq f(x) - \frac{\tilde \eta}{8} \norm*{\nabla f(x)}^2 + \frac{9\tilde\eta^3 L}{4\eta_g^2 E^2} \rbr*{\rbr*{E^2 + P^2}G^2 + \sigma^2} + \tilde\eta^2 M G^2 L\;.
\end{align*}
By adding $f^\star$ to both sides, reordering, dividing by $\tilde \eta$ and taking full expectation, we obtain
\begin{align*}
    \E{\norm*{\nabla f(x^r)}^2}  &\leq \frac{8}{\tilde \eta}\rbr*{\E{f(x^r) - f^\star} - \E{f(x^{r+1}) - f^\star}}  \\
    & \quad + \tilde\eta \underbrace{8 M G^2 L}_{c_1} + \tilde\eta^2 \underbrace{\frac{18 L}{\eta_g^2 E^2} \rbr*{\rbr*{E^2 + P^2}G^2 + \sigma^2}}_{c_2}\;.
\end{align*}
Applying Lemma~\ref{lemma:general} concludes the proof.

\newpage

\section{\fedshufflemvr: Convergence analysis (proof of Theorem~\ref{thm:mvr_is_better})}\label{sec:improvement-analysis}

The style of our proof technique is related to the convergence analysis of \\ \textsc{MimeLiteMVR} \cite{karimireddy2020mime} and non-convex Random Reshuffling~\cite{mishchenko2020random}. By $\EE{r}{\cdot}$, we denote the expectation conditioned on the all history prior to communication round $r$. For the sake of notation, we drop the superscript that represents the communication round $r$ in the proofs. Furthermore, we use superscripts $^+$ and $^-$ to denote $^{r+1}$ and $^{r-1}$, respectively.

Firstly, we analyse a single local epoch. For the ease of presentation, we denote $y_{i,e,0} \eqdef y_{i, e}$ and 
\begin{equation}
 \label{eq:g_ie}
 g_{i,e} \eqdef \frac{1}{|\cD_i|}\sum_{j=1}^{|\cD_i|} \nabla f_{i\Pi_{i, e, j-1}}(y_{i, e, j-1})
\end{equation}
and 
\begin{equation}
 \label{eq:d_ie}
 d_{i,e} \eqdef a\nabla f_i(y_{i, e}) + (1-a)m +  (1-a)\rbr{ \nabla f_i(y_{i, e}) - \nabla f_i(x)}.
\end{equation}
In the first lemma, we investigate how $ g_{i,e}$ update differs from the full local gradient update. 

\begin{lemma}
  \label{lem:rr_var_nc}
  Suppose Assumptions~\ref{ass:smoothness}-\ref{ass:bounded_var} hold with $P_i = 0$ for all $i \in [n]$ and $B = 1$ and only one client is sampled. Then, for $\eta_l \leq \frac{1}{2L}$
  \begin{equation}
     \label{eq:rr_var_nc} 
     \EE{e}{\norm*{g_{i,e} - \nabla f_i(y_{i, e})}^2} \leq 2\eta_l^2 L^2 \frac{\sigma_i^2}{|\cD_i|} + \eta_l^2 L^2 \norm*{d_{i,e}}^2.
  \end{equation}
\end{lemma}

\begin{proof}
  By definition,
  \begin{align*}
      \EE{e}{\norm*{g_{i,e} - \nabla f_i(y_{i, e})}^2} &=  \EE{e}{\norm*{\frac{1}{|\cD_i|}\sum_{j=1}^{|\cD_i|} \nabla f_{i\Pi_{i, e, j-1}}(y_{i, e, j-1}) - \nabla f_{i\Pi_{i, e, j-1}}(y_{i, e})}^2} \\
      &\overset{\eqref{eqn:smoothness}, \eqref{eq:sum_upper}}{\leq} \frac{L^2}{|\cD_i|}\EE{e}{\sum_{j=1}^{|\cD_i|} \norm*{ y_{i, e, j-1} - y_{i, e}}^2}.
  \end{align*}
  Let us now analyze $\sum_{j=1}^{|\cD_i|} \norm*{ y_{i, e, j-1} - y_{i, e}}^2$,
  \begin{eqnarray*}
      \sum_{j=1}^{|\cD_i|} \norm*{ y_{i, e, j-1} - y_{i, e}}^2 &=& \frac{\eta_l^2}{|\cD_i|^2} \sum_{j=1}^{|\cD_i|} \norm*{j(1-a) \rbr{m - \nabla f_i(x)} -  \sum_{k=0}^{j-1}\nabla f_{i\Pi_{i, e, k}}(y_{i, e, k})}^2 \\
      &\overset{\eqref{eq:sum_upper}}{\leq}& \frac{2\eta_l^2}{|\cD_i|^2} \sum_{j=1}^{|\cD_i|} \norm*{ \sum_{k=0}^{j-1}\nabla f_{i\Pi_{i, e, k}}(y_{i, e, k}) - \nabla f_{i\Pi_{i, e, k}}(y_{i, e})}^2 \\ 
      &&\quad + \frac{2\eta_l^2}{|\cD_i|^2}\sum_{j=1}^{|\cD_i|} \norm*{j (1-a) \rbr{m - \nabla f_i(x)} + \sum_{k=0}^{j-1}\nabla f_{i\Pi_{i, e, k}}(y_{i, e})}^2\\
      &\overset{\eqref{eqn:smoothness}, \eqref{eq:sum_upper}}{\leq}& \frac{2\eta_l^2L^2}{|\cD_i|^2} \sum_{j=1}^{|\cD_i|}j\sum_{k=0}^{j-1} \norm*{ y_{i, e, k} - y_{i, e}}^2  \\
      &&\quad + \frac{2\eta_l^2}{|\cD_i|^2}\sum_{j=1}^{|\cD_i|} \norm*{j (1-a) \rbr{m - \nabla f_i(x)} + \sum_{k=0}^{j-1}\nabla f_{i\Pi_{i, e, k}}(y_{i, e})}^2 \\
      &\leq& \eta_l^2 L^2\sum_{j=1}^{|\cD_i|} \norm*{ y_{i, e, j-1} - y_{i, e}}^2 \\
      &&\quad + \frac{2\eta_l^2}{|\cD_i|^2}\sum_{j=1}^{|\cD_i|} \norm*{j (1-a) \rbr{m - \nabla f_i(x)} + \sum_{k=0}^{j-1}\nabla f_{i\Pi_{i, e, k}}(y_{i, e})}^2.
  \end{eqnarray*}
  By $\eta_l \leq \frac{1}{2L}$, we have 
  \begin{eqnarray*}
      &&\EE{e}{\sum_{j=1}^{|\cD_i|} \norm*{ y_{i, e, j-1} - y_{i, e}}^2} \\
     &\leq& \frac{8\eta_l^2}{3|\cD_i|^2}\sum_{j=1}^{|\cD_i|} j^2\EE{e}{\norm*{(1-a) \rbr{m - \nabla f_i(x)} + \frac{1}{j}\sum_{k=0}^{j-1}\nabla f_{i\Pi_{i, e, k}}(y_{i, e})}^2}\\
      &\overset{\eqref{eq:sampling_wo_replacement}, \eqref{eqn:bounded_var}}{\leq}& \frac{8\eta_l^2}{3|\cD_i|^2}\sum_{j=1}^{|\cD_i|} \rbr*{\frac{(|\cD_i| - j)j}{(|\cD_i|-1)} \sigma_i^2 + j^2\norm*{d_{i,e}}^2} \\
      &\leq& 2\eta_l^2\sigma_i^2 + \eta_l^2 |\cD_i|\norm*{d_{i,e}}^2.
  \end{eqnarray*}
  Combining this bound with the previous results concludes the proof.
\end{proof}

Next, we examine the variance of our update in each local epoch $d_{i, e}$.

\begin{lemma}\label{lem:shufflemvr-update-bound}
  For the client update \eqref{eq:mvr_loc_update}, given Assumption~\ref{ass:hessian_sim} and one client is sampled, the following holds for any $a \in [0,1]$  where $h \eqdef m - \nabla f(x)$
  
\begin{equation}
\label{eq:shufflemvr-update-bound}
    \norm{d_{i,e} - \nabla f(y_{i,e})}^2 \leq 3\norm*{h}^2 + 3\delta^2\norm{y_{i,e} - x}^2 + 3a^2\norm*{\nabla f_i(x) - \nabla f(x)}^2\,.
\end{equation}

\end{lemma}
\begin{proof}
  Starting from the client update \eqref{eq:mvr_loc_update}, we can rewrite it as
  \begin{align*}
    d_{i,e} - \nabla f(y_{i, e}) &= (1-a)h \\
    &\quad + \rbr*{\nabla f_i(y_{i,e}) - \nabla f_i(x) - \nabla f(y_{i,e}) + \nabla f(x)} \\
    &\quad + a\rbr*{\nabla f_i(x) - \nabla f(x)}\,.
  \end{align*}
  We can use the relaxed triangle inequality Lemma~\ref{lem:norm-sum} to claim
  \begin{eqnarray*}
    \norm{d_{i,e} - \nabla f(y_{i, e})} &=& (1-a)^2\norm{h}^2 \\
    &&\quad + 3\norm*{\nabla f_i(y_{i,e}) - \nabla f_i(x) - \nabla f(y_{i,e}) + \nabla f(x)}^2 \\
    &&\quad + 3a^2\norm*{\nabla f_i(x) - \nabla f(x)}^2\,.\\
    &\overset{\eqref{eqn:hessian_sim}}{\leq}& 3 (1-a)^2 \norm*{h}^2 + 3\delta^2\norm{y_{i,e} - x}^2 + 3a^2\norm*{\nabla f_i(x) - \nabla f(x)}^2 \,.
  \end{eqnarray*}
 It remains to use that $(1-a)^2 \leq 1$ since $a \in [0,1]$.
\end{proof}

In the following lemma, we introduce a bound that controls the distance moved by a client in each step during the client update.
\begin{lemma}\label{lem:fedshufflemvr-distance-bound}
  For the client update  updates \eqref{eq:mvr_loc_update} with $\eta_l \leq \min\cbr*{ \frac{1}{4 E \delta}, \frac{1}{2L}}$ and given Assumptions~\ref{ass:smoothness}, \ref{ass:bounded_var} and \ref{ass:hessian_sim} with $P_i = 0$ for all $i \in [n]$ and one client is sampled, the following holds
\begin{equation}
\label{eq:fedshufflemvr-distance-bound}
\begin{split}
         \E{\sum_{e=0}^{E-1} \norm*{y_{i,e} - x}^2} &\leq 16E^2\eta_l^2 \sum_{e=0}^{E-1}\E{\norm*{\nabla f(y_{i,e})}^2}  \\ 
         &\quad +  16E^3\eta_l^2\rbr*{3\norm*{h}^2  + 3 a^2\norm*{\nabla f_i(x) - \nabla f(x)}^2  + \eta_l^2 L^2 \frac{\sigma_i^2}{|\cD_i|}}\,. 
\end{split}
\end{equation}
\end{lemma}
\begin{proof}
  Starting from the \fedmvr\ update \eqref{eq:mvr_loc_update},
  \begin{eqnarray*}
    &&\sum_{e=0}^{E-1} \E{\norm*{y_{i,e} - x}^2}  \\
    &=& \eta_l^2 \sum_{e=0}^{E-1} \E{\norm*{\sum_{l=0}^{e-1}\frac{1}{|\cD_i|}\sum_{j=1}^{|\cD_i|}d_{i,e,j-1}}^2} \\
    &=& \eta_l^2 \sum_{e=0}^{E-1} \E{\norm*{\sum_{l=0}^{e-1}\rbr*{d_{i,e} - \nabla f(y_{i,e}) + \nabla f(y_{i,e}) - (g_{i,e} - \nabla f_i(y_{i, e}))}}^2} \\
    &\overset{\eqref{eq:sum_upper}}{\leq}& 3\eta_l^2\sum_{e=0}^{E-1} \sum_{l=0}^{e-1}e\rbr*{\E{\norm*{d_{i,e} - \nabla f(y_{i,e})}^2} + \E{\norm*{\nabla f(y_{i,e})}^2} + \E{\norm*{ g_{i,e} - \nabla f_i(y_{i, e})}^2}} \\
    &\leq& 2E^2\eta_l^2\sum_{e=0}^{E-1}\rbr*{\E{\norm*{d_{i,e} - \nabla f(y_{i,e})}^2} + \E{\norm*{\nabla f(y_{i,e})}^2} + \E{\norm*{ g_{i,e} - \nabla f_i(y_{i, e})}^2}} \\
    &\overset{\eqref{eq:rr_var_nc}}{\leq}& 2E^2\eta_l^2\sum_{e=0}^{E-1}\rbr*{\E{\norm*{d_{i,e} - \nabla f(y_{i,e})}^2} + \E{\norm*{\nabla f(y_{i,e})}^2} + 2\eta_l^2 L^2 \frac{\sigma_i^2}{|\cD_i|}}  \\
    && \quad  + 2E^2 L^2\eta_l^4\sum_{e=0}^{E-1} \E{\norm*{d_{i,e}}^2} \\
    &\overset{\eqref{eq:sum_upper}}{\leq}& 4E^2\eta_l^2\sum_{e=0}^{E-1}\rbr*{\E{\norm*{d_{i,e} - \nabla f(y_{i,e})}^2} + \E{\norm*{\nabla f(y_{i,e})}^2} + \eta_l^2 L^2 \frac{\sigma_i^2}{|\cD_i|}},
  \end{eqnarray*}
  where the last inequality also uses the step size bound $\eta_l \leq \frac{1}{2L}$.
  By exploiting \eqref{eq:shufflemvr-update-bound}, we can further bound
  \begin{align*}
   \E{\sum_{e=0}^{E-1} \norm*{y_{i,e} - x}^2 } &\leq 12E^2\eta_l^2\delta^2 \sum_{e=0}^{E-1} \E{\norm{y_{i,e} - x}^2} + 4E^2\eta_l^2 \sum_{e=0}^{E-1}\E{\norm*{\nabla f(y_{i,e})}^2} \\
   &\quad +  12E^3\eta_l^2\norm*{h}^2  + 12E^3\eta_l^2 a^2\norm*{\nabla f_i(x) - \nabla f(x)}^2  + 4E^3\eta_l^4 L^2 \frac{\sigma_i^2}{|\cD_i|}\,.
  \end{align*}
  $\eta_l \leq \frac{1}{4E\delta}$ implies that $12E^2\eta_l^2\delta^2 \leq \frac{3}{4}$, therefore
  \begin{align*}
   \E{\sum_{e=0}^{E-1} \norm*{y_{i,e} - x}^2} &\leq 16E^2\eta_l^2 \sum_{e=0}^{E-1}\E{\norm*{\nabla f(y_{i,e})}^2}  \\ 
         &\quad +  16E^3\eta_l^2\rbr*{3\norm*{h}^2  + 3a^2\norm*{\nabla f_i(x) - \nabla f(x)}^2  + \eta_l^2 L^2 \frac{\sigma_i^2}{|\cD_i|}}\,. 
  \end{align*}
\end{proof}

We compute the error of the server momentum $m$ defined as $h \eqdef m - \nabla f(x)$. Its expected norm can be bounded as follows.
\begin{lemma}\label{lem:shufflemvr-momentum-bound}
  For the momentum update \eqref{eq:mvr_mom_update}, given that Assumptions~\ref{ass:smoothness}-\ref{ass:hessian_sim} with $P_i = 0$ for all $i \in [n]$ and $B = 0$ hold and one client is sampled, the following holds for any $\eta_l \leq \cbr*{\frac{1}{4L} \frac{1}{40 \delta E}}$ and $1 \geq a \geq 1152 E^2 \delta^2 \eta_l^2$,
  \begin{equation}
  \label{eq:shufflemvr-momentum-bound}
  \begin{split}
      \E{\norm{h^+}^2} &\leq (1-\frac{23a}{24}) \norm{h}^2 + 3a^2 \E{\norm*{ \nabla f_i(x) - \nabla f(x)}^2}\\
      &\quad +  16\eta_l^4E^2\delta^2 L^2 \frac{\sigma_i^2}{|\cD_i|} + 8\eta_l^2\delta^2 E \sum_{e=0}^{E-1} \E{\norm*{\nabla f(y_{i,e})}^2} \,.
  \end{split}
  \end{equation}
\end{lemma}
\begin{proof}
  Starting from the momentum update \eqref{eq:mvr_mom_update},
  \begin{align*}
    h^+ &= (1-a)h  \\
    &\quad + (1-a)\rbr*{ \nabla f_i(x^+) - \nabla f_j(x)) - \nabla f(x^+) + \nabla f(x)} \\
  &\quad + a\rbr*{ \nabla f_i(x^+) - \nabla f(x)}\,.
  \end{align*}
  Now, the first term does not have any information from current round and hence is statistically independent of the rest of the terms. Further, the rest of the terms have mean $0$. Hence, we can separate out the zero mean noise terms from the $h$ and then the relaxed triangle inequality Lemma~\ref{lem:norm-sum} to claim
  \begin{eqnarray*}
    \E{\norm{h^+}^2} &=& (1-a)^2\E{\norm{h}^2}  \\
    &&\quad + 2(1-a)^2\E{\norm*{ \nabla f_i(^+x) - \nabla f_i(x)) - \nabla f(x^+) + \nabla f(x)}^2} \\
  &&\quad + 2a^2\E{\norm*{ \nabla f_i(x) - \nabla f(x)}^2} \\
    &\overset{\eqref{eqn:hessian_sim}}{\leq}& (1-a)^2 \E{\norm{h}^2}  + 2(1-a)^2\delta^2\E{\norm{x^+ - x}^2} \\
    &&\quad + 2a^2\E{\norm*{ \nabla f_i(x) - \nabla f(x)}^2} \,.
  \end{eqnarray*}
 Finally, note that $(1-a)^2 \leq (1-a)\leq 1$ for $a \in [0,1]$. Therefore,
  \begin{align*}
    \E{\norm{h^+}^2} &\leq (1-a) \E{\norm{h}^2}  + 2\delta^2\E{\norm{x^+ - x}^2} + 2a^2 \E{\norm*{ \nabla f_i(x) - \nabla f(x)}^2}\,.
  \end{align*}
   We can continue upper bounding $\E{\norm{x^+ - x}^2}$.
    \begin{eqnarray*}
    \E{\norm{x^+ - x}^2} &=& \eta_l^2\E{\norm*{\sum_{e=0}^{E-1} \frac{1}{|\cD_i|}\sum_{j-1}^{|\cD_i|} d_{i,e, j-1}}^2}\\
    &=& \eta_l^2\E{\norm*{\sum_{e=0}^{E-1} (d_{i,e} - \nabla f(y_{i,e})) - (g_{i,e} - \nabla f_i(y_{i, e}))}^2} \\
    &\overset{\eqref{eq:sum_upper}}{\leq}& 2\eta_l^2E \sum_{e=0}^{E-1} \E{\norm*{ d_{i,e} - \nabla f(y_{i,e})}^2 + \norm*{g_{i,e} - \nabla f_i(y_{i, e})}^2} \\
    &\overset{\eqref{eq:rr_var_nc}}{\leq}& 2\eta_l^2E \sum_{e=0}^{E-1} \E{\norm*{ d_{i,e} - \nabla f(y_{i,e})}^2 + 2\eta_l^2 L^2 \frac{\sigma_i^2}{|\cD_i|} + \eta_l^2 L^2 \norm*{d_{i,e}}^2} \\
    &\overset{\eqref{eq:sum_upper}}{\leq}& 4\eta_l^4E^2 L^2 \frac{\sigma_i^2}{|\cD_i|} + 4\eta_l^2E \sum_{e=0}^{E-1} \E{\norm*{ d_{i,e} - \nabla f(y_{i,e})}^2}  \\
    &&\quad  + 4\eta_l^4 L^2 E  \sum_{e=0}^{E-1} \E{\norm*{\nabla f(y_{i,e})}^2} \\
    &\overset{\eqref{eq:shufflemvr-update-bound}}{\leq}& 4\eta_l^4E^2 L^2 \frac{\sigma_i^2}{|\cD_i|} + 12\eta_l^2E^2 \norm*{h}^2 + 12\eta_l^2E^2a^2 \norm*{\nabla f_i(x) - \nabla f(x)}^2 \\
    &&\quad  + 12\eta_l^2E \delta^2  \sum_{e=0}^{E-1} \E{\norm{y_{i,e} - x}^2}  + 4\eta_l^4 L^2 E  \sum_{e=0}^{E-1} \E{\norm*{\nabla f(y_{i,e})}^2} \\
    &\overset{\eqref{eq:fedshufflemvr-distance-bound}}{\leq}& 4\eta_l^4E^2 L^2 \rbr*{1 + 48\eta_l^2E^2\delta^2} \frac{\sigma_i^2}{|\cD_i|} \\
    &&\quad + 12\eta_l^2E^2 \rbr*{1 + 48\eta_l^2E^2\delta^2} \norm*{h}^2 \\
    &&\quad  + 12\eta_l^2E^2a^2 \rbr*{1 + 48\eta_l^2E^2\delta^2} \norm*{\nabla f_i(x) - \nabla f(x)}^2 \\
    &&\quad + 4\eta_l^4 E \rbr*{L^2 + 48\delta^2E^2}  \sum_{e=0}^{E-1} \E{\norm*{\nabla f(y_{i,e})}^2} \\
    &\overset{\eqref{eq:fedshufflemvr-distance-bound}}{\leq}& 8\eta_l^4E^2 L^2 \frac{\sigma_i^2}{|\cD_i|} + 24\eta_l^2E^2 \norm*{h}^2 + 24\eta_l^2E^2a^2\norm*{\nabla f_i(x) - \nabla f(x)}^2 \\
    &&\quad + 4\eta_l^2 E \sum_{e=0}^{E-1} \E{\norm*{\nabla f(y_{i,e})}^2}\,,
  \end{eqnarray*}
  where the last inequality uses the upper bound on the step size $\eta_l$. Plugging this back to previous bound yields
  \begin{align*}
      \E{\norm{h^+}^2} &\leq \rbr*{1- a + 48\eta_l^2 E^2 \delta^2} \E{\norm{h}^2} + 3a^2 \E{\norm*{ \nabla f_i(x) - \nabla f(x)}^2}\\
      &\quad +  16\eta_l^4E^2\delta^2 L^2 \frac{\sigma_i^2}{|\cD_i|} + 8\eta_l^2\delta^2 E \sum_{e=0}^{E-1} \E{\norm*{\nabla f(y_{i,e})}^2} \,.
  \end{align*}
 The last step used the bound on the momentum parameter that $1 \geq a \geq 1152\eta_l^2\delta^2E^2$. Note that $\eta_l \leq \frac{1}{40\delta E}$ ensures that this set is non-empty.
\end{proof}

Now we can compute the overall progress.
\begin{lemma}\label{lem:shufflemvr-step-progress}
  For any client update step with step size $\eta_l \leq \min\cbr*{\frac{1}{4L}, \frac{1}{40 \delta E}}$, $a \geq 1152\eta_l^2\delta^2E^2$ and given that Assumptions~\ref{ass:smoothness}-\ref{ass:hessian_sim} hold with $P_i = 0$ for all $i \in [n]$ and $B = 0$ and one client is sampled with probabilities $\cbr*{w_i}$, we have
  \begin{equation}
  \label{eq:shufflemvr-step-progress}
   \frac{1}{RE}\sum_{r=0}^{R-1}\sum_{e=0}^{E-1}\E{\norm*{\nabla f(y^r_{i^r,e})}^2}  \leq \frac{5\Psi^0 }{\eta_l R E} +  25\eta_l^2 L^2 \sigma^2 + 255 a G^2\,,
  \end{equation}
  where $\Psi^r \eqdef (f(x^r) - f^\star) + \frac{288\eta_l}{23a}E\norm{m^r - \nabla f(x^r)}^2$ and $\sigma^2 = \frac{1}{|\cD|}\sum_{i=1}^n \sigma_i^2$.
\end{lemma} 
\begin{proof}
  The assumption that $f$ is $L$-smooth implies a quadratic upper bound \eqref{eqn:quad-upper}.
  \begin{eqnarray*}
    &&\EE{e}{f(y_{i,e+1})} - f(y_{i,e}) \\
    &\leq& -\eta_l\EE{e}{\inp*{\nabla f(y_{i,e})}{  \frac{1}{|\cD_i|}\sum_{j=1}^{|\cD_i|} d_{i,e,j-1}}} + \frac{L\eta_l^2}{2}\EE{e}{\norm*{  \frac{1}{|\cD_i|}\sum_{j=1}^{|\cD_i|} d_{i,e,j-1}}^2}\\
    &=&-\frac{\eta_l}{2}\norm*{\nabla f(y_{i,e})}^2 -\frac{\eta_l}{2} \rbr*{1 - \eta_l L}\EE{e}{\norm*{  \frac{1}{|\cD_i|}\sum_{j=1}^{|\cD_i|} d_{i,e,j-1}}^2}  \\
    &&\quad + \frac{\eta_l}{2}\EE{e}{\norm*{  \frac{1}{|\cD_i|}\sum_{j=1}^{|\cD_i|} d_{i,e,j-1} - \nabla f(y_{i,e})}^2} \\
    &\leq& -\frac{\eta_l}{2}\norm*{\nabla f(y_{i,e})}^2  + \EE{e}{\frac{\eta_l}{2}\norm*{(d_{i,e} - \nabla f(y_{i,e})) - (g_{i,e} - \nabla f_i(y_{i, e}))}^2} \\
    &\overset{\eqref{eq:sum_upper}}{\leq}& -\frac{\eta_l}{2}\norm*{\nabla f(y_{i,e})}^2  + \eta_l\norm*{d_{i,e} - \nabla f(y_{i,e})}^2  +  \eta_l\EE{e}{\norm*{g_{i,e} - \nabla f_i(y_{i, e})}^2}\\
    &\overset{\eqref{eq:rr_var_nc}}{\leq}& -\frac{\eta_l}{2}\norm*{\nabla f(y_{i,e})}^2  + \eta_l\norm*{d_{i,e} - \nabla f(y_{i,e})}^2  +  2\eta_l^3 L^2 \frac{\sigma_i^2}{|\cD_i|} + \eta_l^3 L^2 \norm*{d_{i,e}}^2\\
    &\overset{\eqref{eq:sum_upper}}{\leq}& -\frac{\eta_l}{2}\rbr*{1 - 4\eta_l^2L^2}\norm*{\nabla f(y_{i,e})}^2  + 2\eta_l\norm*{d_{i,e} - \nabla f(y_{i,e})}^2  +  2\eta_l^3 L^2 \frac{\sigma_i^2}{|\cD_i|}\\
    &\overset{\eqref{eq:shufflemvr-update-bound}}{\leq}& -\frac{\eta_l}{2}\rbr*{1 - 4\eta_l^2L^2}\norm*{\nabla f(y_{i,e})}^2 +  2\eta_l^3 L^2 \frac{\sigma_i^2}{|\cD_i|}\\
    &&\quad + 6\eta_l\rbr*{\norm*{h}^2 + \delta^2\norm{y_{i,e} - x}^2 + a^2\norm*{\nabla f_i(x) - \nabla f(x)}^2}\,.
  \end{eqnarray*}

  The first equality used the fact that for any $a, b$: $2 \inp{a}{b} = \norm{a-b}^2 - \norm{a}^2 - \norm{b}^2$. The second term in the first equality can be removed since $\eta_l \leq \frac{1}{L}$. 
  
Applying this inequality recursively with full expectation over the current communication round 
\begin{eqnarray*}
    &&\E{f(x^+)} - f(x)  \\
    &\leq& -\frac{\eta_l}{2}\rbr*{1 - 4\eta_l^2L^2}\sum_{e=0}^{E-1}\norm*{\nabla f(y_{i,e})}^2 +  2\eta_l^3 L^2 E\frac{\sigma_i^2}{|\cD_i|} + 6 \eta_l\delta^2\sum_{e=0}^{E-1} \norm{y_{i,e} - x}^2\\
    &&\quad + 6E\eta_l\rbr*{\norm*{h}^2 + a^2\norm*{\nabla f_i(x) - \nabla f(x)}^2} \\
    &\overset{\eqref{eq:fedshufflemvr-distance-bound}}{\leq}&-\frac{\eta_l}{2}\rbr*{1 - 4\eta_l^2L^2 - 192\eta_l^2\delta^2E^2}\sum_{e=0}^{E-1}\norm*{\nabla f(y_{i,e})}^2 +  2\eta_l^3 L^2 E \rbr*{1 + 48\eta_l^2E^2\delta^2}\frac{\sigma_i^2}{|\cD_i|} \\
    &&\quad + 6E\eta_l\rbr*{1 + 192\eta_l^2E^2\delta^2}\rbr*{\norm*{h}^2 + a^2\norm*{\nabla f_i(x) - \nabla f(x)}^2} \\
    &\leq& -\frac{\eta_l}{4}\sum_{e=0}^{E-1}\norm*{\nabla f(y_{i,e})}^2 +  4\eta_l^3 L^2 E \frac{\sigma_i^2}{|\cD_i|} + 12E\eta_l\rbr*{\norm*{h}^2 + a^2\norm*{\nabla f_i(x) - \nabla f(x)}^2}\,.
 \end{eqnarray*}
  Adding an extra term  $12E\frac{24\eta_l}{23a}\norm{h^+}^2$ term results to
  \begin{eqnarray*}
    &&\E{f(x^+) + 12E\frac{24\eta_l}{23a}\norm{h^+}^2} - f(x) \\ &\overset{\eqref{eq:shufflemvr-momentum-bound}}{\leq}& -\frac{\eta_l}{4}\rbr*{1 - 104\frac{\eta_l^2E^2\delta^2}{a}}\sum_{e=0}^{E-1}\norm*{\nabla f(y_{i,e})}^2 +  \rbr*{4 + 208 \frac{\eta_l^2E^2\delta^2}{a} } E \eta_l^3 L^2 \frac{\sigma_i^2}{|\cD_i|} \\
    &&\quad + 12E\frac{24\eta_l}{23a}\norm{h}^2 + \rbr*{12\eta_l a^2 + 39\eta_l a}E\norm*{\nabla f_i(x) - \nabla f(x)}^2 \\
    &\leq& -\frac{\eta_l}{5}\sum_{e=0}^{E-1}\norm*{\nabla f(y_{i,e})}^2 +  5E \eta_l^3 L^2 \frac{\sigma_i^2}{|\cD_i|} \\
    &&\quad + 12E\frac{24\eta_l}{23a}\norm{h}^2 + 51\eta_l a E\norm*{\nabla f_i(x) - \nabla f(x)}^2\,.
   \end{eqnarray*}
  
  Taking the full expectation including the client selection step yields
 \begin{align*}
    \frac{5}{\eta_l E}\E{\Psi^+ - \Psi} &\overset{\eqref{eq:heterogeneity}}{\leq} -\frac{1}{E}\sum_{e=0}^{E-1}\E{\norm*{\nabla f(y_{i,e})}^2} +  25\eta_l^2 L^2 \sigma^2 + 255 a G^2\,.
\end{align*}
Applying this inequality recursively with the fact $\Psi \geq 0$ leads to the desired result.
\end{proof}

Equipped with Lemma~\ref{lem:shufflemvr-step-progress}, we have
  \begin{align*}
    \frac{1}{RE}\sum_{r=0}^{R-1}\sum_{e=0}^{E-1}\E{\norm*{\nabla f(y^r_{i^r,e})}^2}  \leq \frac{5\Psi^0 }{\eta_l R E} +  25\eta_l^2 L^2 \sigma^2 + 255 a G^2\,,
  \end{align*}
  where $\Psi^0 \eqdef (f(x^0) - f^\star) + \frac{288\eta_l}{23a}E\norm{m^0 - \nabla f(x^0)}^2$.

Further,  we need to control $\norm*{m^0 - \nabla f(x^0)}^2$. \cite{cutkosky2019momentum} show that by using time-varying step sizes, it is possible to directly control this error. Alternatively, \cite{tran2019hybrid} use a large initial accumulation for the momentum term. For the sake of simplicity, we will follow the latter approach simarly to \cite{karimireddy2020mime}. Extension to the time-varying step size case is straightforward.
Suppose that we run the algorithm for $2R$ communication rounds wherein for the first $R$ rounds, we simply compute 
$
    m^0 = \frac{1}{R}\sum_{t=1}^{R}\nabla f_{i^r}(x^0)\,.
$
With this, we have 
$
\norm{m^0 - \nabla f(x^0)}^2 \leq \frac{G^2}{R}\,.
$
Thus, we have for the first round $r=0$
\[ \Psi^0 = (f(x^0) - f^\star) + \frac{288\eta_l}{23a}E\norm{m^0 - \nabla f(x^0)}^2 \leq (f(x^0) - f^\star) + \frac{288\eta_l EG^2}{23aR} \,.
\]
Together, this gives
\begin{align*}
    \frac{1}{RE}\sum_{r=0}^{R-1}\sum_{e=0}^{E-1}\E{\norm*{\nabla f(y^r_{i^r,e})}^2}  &\leq \frac{5(f(x^{0}) - f^\star)}{\eta_l E R}  + \frac{63G^2}{aR^2}   +  25\eta_l^2 L^2 \sigma^2 + 293760  G^2\,.
  \end{align*}
  The above equation holds for any choice of $\eta \leq \min\rbr*{\frac{1}{4L}, \frac{1}{40 \delta K}}$ and momentum parameter $a \geq 1152 E^2 \delta^2 \eta_l^2$. Set the momentum parameter as 
  \[
    a = \max\rbr*{1152 E^2 \delta^2 \eta_l^2 , \frac{1}{T}}
\]
    With this choice, we can simplify the rate of convergence as 
\[
    \frac{5(f(x^{0}) - f^\star)}{\eta_l E R}  + \frac{318G^2}{R} +   \eta_l^2 \rbr*{25L^2 \sigma^2 + 293760 E^2 \delta^2 G^2}\,.
\]
      Now let us pick 
\[
    \eta_l = \min\cbr*{ \frac{1}{4L}, \frac{1}{40 \delta E}, \rbr*{\frac{f(x^{0}) - f^\star}{E^3R(58752\delta^2 G^2 + 5\sigma^2 \frac{L^2}{E^2})}}^{1/3}}\,.
\]
For this combination of step size $\eta_l$ and $a$, the rate simplifies to
\[
    \frac{318 G^2}{R} + 390 \rbr*{\frac{(f(x^{0}) - f^\star) \rbr*{\delta^2 G^2 + \sigma^2 \frac{L^2}{E^2}}}{R^2}}^{1/3} +  \frac{20 (L + 10 \delta E) (f(x^{0}) - f^\star)}{ER}\,.
\]
  Since $E \geq \frac{L}{\delta}$ then $\frac{L^2}{E^2} \leq \delta^2$ and $L \leq E\delta$, which concludes the proof.

 \subsection{Discussion on theoretical assumptions}
 
 Note that in Theorem~\ref{thm:mvr_is_better}, we assume that $P_i = 0$ for all $i \in [n]$ and $B = 1$. These are not necessary assumptions, but we rather introduce these more restrictive requirements to match the assumptions that were used to obtain existing lower bounds. 
Furthermore, note that we require only one client to be sampled in each round.
Unfortunately, our results cannot be simply extended to the case with multiple clients sampled in each round. The problem is the aggregation step (line 15 in Algorithm~\ref{alg:FedShuffle_simple}) which involves averaging. For our analysis and all the previous works that show the improvement over non-local methods in the non-convex setting to work, one would need to assume that the average model performs not worse than the average output of the sampled client. Such a property was empirically observed in~\cite{mcmahan17fedavg}; thus, is reasonable to assume. Another possibility is to replace the averaging by randomly selecting model from one of the sampled clients. We note that while the second option might be preferable in theory as it does not require an extra assumption, it might affect the client privacy and therefore not applicable in practice. 


\newpage

\section{\fedgen: General shuffling method}
\label{sec:fedgen_appendix}
\subsection{Proof of Theorem~\ref{thm:fedshuffle_gem-full}}

 By $\EE{r}{\cdot}$, we denote the expectation conditioned on the all history prior to communication round $r$.

\begin{lemma}\textbf{\em (one round progress)}\label{lem:fedavg-progress_gen}
Suppose Assumptions \ref{ass:strong-convexity} -- \ref{ass:bounded_var} hold. For any constant step sizes $\eta_l^r \eqdef \eta_l$ and $\eta_l^r \eqdef \eta_l$  satisfying
$\eta_l \leq \frac{1}{(1 + M_2 + M_1 B^2)4L\eta_g}$ and effective step size
$\tilde\eta \eqdef W \eta_g \eta_l$, the updates of \fedavg\ satisfy
\begin{align*}
    \E{\norm{x^{r} - \hx}^2 }&\leq \rbr*{1 - \frac{\mu
      \tilde\eta}{2}} \E {\norm{x^{r-1} - \hx}^2} - \tilde\eta \EE{r-1}{\hf(x^{r-1}) - \hfs}  \\
      & \quad + 3L \tilde\eta\xi^r + 2\tilde\eta^2 M_1G^2\,,
\end{align*}
where $\xi_{r}$ is the drift caused by the local updates on the
clients defined to be
\[
    \xi^{r} \eqdef \frac{1}{W}\sum_{i=1}^n \sum_{e=1}^{E_i} \sum_{j=1}^{|\cD_i|} \frac{\tw_i}{q_i b_i}
    \EE{r-1}{\norm{y_{i,e, j-1}^r - x^{r-1}}^2}\,
\]
$M_1 \eqdef \max_{i \in [n]} \cbr*{h_i v_i \hw_i}$ and $M_2 \eqdef \rbr*{\sum_{i \in [n]} E_i |\cD_i|} \max_{i \in [n]} \cbr*{\frac{\tw_i}{W q_i b_i}}$.
\end{lemma}
\begin{proof}
For a better readability of the proofs in one round progress, we drop the superscript that represents the current completed communication round $r-1$.

By the definition in Algorithm~\ref{alg:FedGen}, the update $\Delta$ can be written as 
\[
    \Delta = - \eta_g \sum_{i \in \cS} \frac{w_i}{q_i} \Delta_i = -\frac{\tilde\eta}{W} \sum_{i \in \cS} \sum_{e=1}^{E_i} \sum_{j=1}^{|\cD_i|} \frac{\tw_i}{q^\cS_i b_i}\nabla f_{i\Pi_{i,e,j-1}}(y_{i,e,j-1})\,.
\]
We adopt the convention that summation $\sum_{i \in \cM ,e,j}$ ($\cM$ is either $[n]$ or $\cS$) refers to the summations $\sum_{i \in \cM} \sum_{e=1}^{E_i} \sum_{j=1}^{|\cD_i|}$ unless otherwise stated. Furthermore, we denote $g_{i,e,j} \eqdef \nabla f_{i\Pi_{i,e,j-1}}(y_{i,e,j-1})$. Using above, we proceed as
\begin{align*}
    \EE{r-1}{\norm{x + \Delta - \hx}^2} &= \norm{x - \hx}^2 \underbrace{-2\EE{r-1}{\frac{\tilde\eta}{W}\sum_{i \in \cS ,e,j}\frac{\tw_i}{q^\cS_i b_i}\inp{g_{i,e,j}}{x - \hx}}}_{\cA_1} \\
    &\quad  + \underbrace{\tilde\eta^2\EE{r-1}{\norm*{\frac{1}{W}\sum_{i \in \cS ,e,j}\frac{\tw_i}{q^\cS_i b_i} g_{i,e,j}}^2}}_{\cA_2}\,.
\end{align*}
To bound the term $\cA_1$, we apply Lemma~\ref{lem:magic} to each term of the summation with $h = f_{ij}$, $x = y_{i,e,j-1}$, $y = \hx$, and $z = x$. Therefore, 
\begin{eqnarray*}
    \cA_1 &=& -\EE{r-1}{2\frac{\tilde\eta}{W} \sum_{i \in \cS ,e,j} \frac{\tw_i}{q^\cS_i b_i} \inp{g_{i,e,j}}{x - \hx}} \\
    &\leq& \EE{r-1}{2\frac{\tilde\eta}{W}\sum_{i \in \cS ,e,j} \frac{\tw_i}{q^\cS_i b_i} \rbr*{f_{i\Pi_{i,e,j-1}}(\hx) - f_{i\Pi_{i,e,j-1}}(x) + L \norm{y_{i,e,j-1} - x}^2 }}  \\
    &&\quad - 2\tilde\eta\frac{\mu}{4}\norm{x - \hx}^2\\
    &=& -2\tilde\eta\rbr*{\hf(x) - \hf^\star + \frac{\mu}{4}\norm{x - \hx}^2} + 2L \tilde\eta\xi\,.
\end{eqnarray*}
    For the second term $\cA_2$, we have
    \begin{eqnarray*}
        \cA_2 &=& \tilde\eta^2\EE{r-1}{\norm*{\frac{1}{W}\sum_{i \in \cS , e, j} \frac{\tw_i}{q^\cS_i b_i} g_{i, e, j}}^2}\\
        &\overset{\eqref{eq:var_decomposition}}{\leq}& \tilde\eta^2 \EE{r-1}{\norm*{\frac{1}{W}\sum_{i \in \cS ,e,j} \frac{w_i}{q^\cS_i b_i} g_{i,e,j} - \frac{1}{W}\sum_{i \in [n] ,e,j} \frac{\tw_i}{q_i b_i} g_{i,e,j}}^2 } \\
        && \quad + \tilde\eta^2 \EE{r-1}{\norm*{\frac{1}{W}\sum_{i \in [n] ,e,j} \frac{\tw_i}{q_i b_i}  g_{i,e,j}}^2}\\
        &\overset{\eqref{eq:var_arbitrary_gen}}{\leq}& \frac{\tilde\eta^2}{W^2} \EE{r-1}{\sum_{i \in [n]} \frac{h_i v_i\tw_i^2}{q_i^2 b_i^2} \norm*{ \sum_{e, j} g_{i,e,j}}^2 + \norm*{\sum_{i \in [n] ,e,j} \frac{\tw_i }{q_i b_i}  g_{i,e,j}}^2} \\
        &\overset{\eqref{eq:sum_upper}}{\leq}& \frac{2\tilde\eta^2}{W^2}\EE{r-1}{\sum_{i \in [n]}\frac{h_i v_i\tw_i^2}{q_i^2 b_i^2}\norm*{ \sum_{e, j} g_{i,e,j} - \nabla f_{i\Pi_{i,e,j-1}}(x) }^2} \\
        &&\quad + 2\frac{\tilde\eta^2}{W^2}\sum_{i \in [n]}\frac{h_iv_i\tw_i^2 E_i^2 |\cD_i|^2}{q_i^2 b_i^2}\norm*{\nabla f_{i}(x)}^2 \\
        &&\quad + \frac{2\tilde\eta^2}{W^2}\EE{r-1}{\norm*{\sum_{i \in [n] ,e,j}\frac{\tw_i }{q_i b_i} \cbr{g_{i,e,j} - \nabla f_{i\Pi_{i,e,j-1}}(x)}}^2 } + 2\tilde\eta^2 \norm*{\nabla \hf(x)}^2.
    \end{eqnarray*}
    Applying \eqref{eq:sum_upper} leads to 
    \begin{eqnarray*}
        \cA_2 &\leq& \frac{2\tilde\eta^2}{W^2}\EE{r-1}{\sum_{i \in [n] ,e,j}\frac{h_i v_i\tw_i^2 E_i |\cD_i|}{q_i^2 b_i^2}\norm*{ g_{i,e,j} - \nabla f_{i\Pi_{i,e,j-1}}(x) }^2}\\
        &&\quad + 2\frac{\tilde\eta^2}{W^2}\sum_{i \in [n]}\frac{h_iv_i\tw_i^2 E_i^2 |\cD_i|^2}{q_i^2 b_i^2}\norm*{\nabla f_{i}(x)}^2 + 2\tilde\eta^2 \norm*{\nabla \hf(x)}^2 \\
        &&\quad + \frac{2\tilde\eta^2}{W^2}\sum_{i \in [n] ,e,j}\frac{\tw_i^2 \rbr*{\sum_{i \in [n]} E_i|\cD_i|}}{q_i^2 b_i^2}\EE{r-1}{\norm*{g_{i,e,j} - \nabla f_{i\Pi_{i,e,j-1}}(x)}^2 } \\
        &\overset{\eqref{eq:smooth}}{\leq}& 2 M_1\tilde\eta^2 L^2 \xi + 2\tilde\eta^2 M_1 \sum_{i \in [n]}\hw_i\norm*{\nabla f_{i}(x)}^2  + 2 \tilde\eta^2 L^2 \xi M_2  + 2\tilde\eta^2 \norm*{\nabla \hf(x)}^2\\
        &\overset{\eqref{eq:heterogeneity}}{\leq}& 2 \rbr*{M_2  + M_1}\tilde\eta^2 L^2 \xi + 2\tilde\eta^2 M_1 G^2 + 2\tilde\eta^2 \rbr*{1 + M_1B^2}  \norm*{\nabla \hf(x)}^2\\
        &\overset{\eqref{eqn:smoothness}}{\leq}& 2 \rbr*{M_2 + M_1}\tilde\eta^2 L^2 \xi + 2\tilde\eta^2 M_1 G^2  + 4L\tilde\eta^2 \rbr*{1 + M_1 B^2}  (\hf(x) - \hf^\star).
    \end{eqnarray*}
  By plugging back the bounds on $\cA_1$ and $\cA_2$,
    \begin{multline*}
        \EE{r-1}{\norm{x + \Delta - \hx}^2} \leq (1 - \frac{\mu\tilde\eta}{2})\norm{x - \hx}^2 - (2\tilde\eta - 4L\tilde\eta^2(M_1 B^2 + 1))(\hf(x) - \hf^\star)  \\
      +(1 + (M_1 + M_2)\tilde\eta L)2L \tilde\eta\xi + 2\tilde\eta^2 M_1 G^2\,.
    \end{multline*}
    The lemma now follows by observing that $4L\tilde\eta(M_1 B^2 + M_2 + 1) \leq 1$ and that $B \geq 1$.
\end{proof}

\begin{lemma}\textbf{\em (bounded drift)}\label{lem:fedavg-error-bound_gen}
Suppose Assumptions \ref{ass:smoothness} -- \ref{ass:bounded_var} hold. Then the updates of \fedavg\ with reshuffling
for any step size satisfying 
\[
\eta_l \leq \frac{1}{(1  + M_2 + (P+1) B + M_1 B^2))4L\eta_g}
\]
and $b_i \geq E_i|\cD_i|$ for all $i \in [n]$ have bounded drift:
\[
    3L\tilde\eta \xi_{r} \leq  \frac{9}{10} \tilde\eta(\hf(x^r) - \hf^\star) + 9\frac{\tilde\eta^3}{\eta_g^2} (\rbr*{1 + P^2}G^2 + \sigma^2)\,, 
\]
where $P^2 \eqdef \max_{i \in [n]}\frac{P_i^2}{3|\cD_i|E_i^2}$, $\sigma^2 \eqdef \sum_{i \in [n]} \hw_i \frac{\sigma^2_i}{3|\cD_i|E_i^2}$ and $M_1 \eqdef \max_{i \in [n]} \cbr*{h_i v_i \hw_i}$.
\end{lemma}
\begin{proof}
We adopt the same convention as for the previous proof, i.e., dropping superscripts, simplifying sum notation and having $g_{i,e,j} \eqdef \nabla f_{i\Pi_{i,e,j-1}}(y_{i,e,j-1})$. Therefore,
\begin{eqnarray*}
\xi
&=&  \frac{1}{W} \sum_{i \in [n], i, j} \frac{\tw_i}{q_i b_i} \EE{r-1}{\norm{y_{i,e, j-1}^r - x^{r-1}}^2}\\
    &=&\frac{\eta_l^2}{W} \sum_{i \in [n], e, j} \frac{\tw_i}{q_i b_i^2} \EE{r-1}{\norm*{\sum_{l=0}^{e-1} \sum_{c=1}^{\cD_i} g_{i, l, c} + \sum_{c=1}^j g_{i, e, c}}^2}\\
    &\overset{\eqref{eq:triangle}}{\leq}& \frac{2\eta_l^2}{W} \sum_{i \in [n], e, j} \frac{\tw_i}{q_i b_i^3} \\
    &&\quad \EE{r-1}{\norm*{\sum_{l=0}^{e-1} \sum_{c=1}^{\cD_i} g_{i, l, c} - (e-1) |\cD_i| \nabla f_i (x) + \sum_{c=1}^j g_{i, e, c} - \nabla f_{i\Pi_{i,e,c-1}}(x)}^2} \\
    &&\quad + \frac{2\eta_l^2}{W} \sum_{i \in [n], e, j} \frac{\tw_i}{q_i b_i^3} \EE{r-1}{\norm*{(e-1) |\cD_i| \nabla f_i (x) + \sum_{c=1}^j \nabla f_{i\Pi_{i,e,c-1}}(x)}^2}
\end{eqnarray*}
Next,  we upper bound each term separately. For the first term, we have
\begin{align*}
    &  \frac{2\eta_l^2}{W} \sum_{i \in [n], e, j} \frac{\tw_i}{q_i b_i^3} \EE{r-1}{\norm*{\sum_{l=0}^{e-1} \sum_{c=1}^{\cD_i} g_{i, l, c} - (e-1) |\cD_i| \nabla f_i (x) + \sum_{c=1}^j g_{i, e, c} - \nabla f_{i\Pi_{i,e,c-1}}(x)}^2} \\
     &\overset{\eqref{eq:sum_upper}}{\leq} \frac{2\eta_l^2}{W} \sum_{i \in [n], e, j} \frac{\tw_i ((e-1) |\cD_i| + j)}{q_i b_i^3} \\
     &\quad \EE{r-1}{\rbr*{\sum_{l=0}^{e-1} \sum_{c=1}^{\cD_i}\norm*{ g_{i, l, c} - \nabla f_{i\Pi_{i,e,j-1}}(x)}^2 + \sum_{c=1}^j \norm*{g_{i, e, c} - \nabla f_{i\Pi_{i,e,c-1}}(x)}^2}} \\
     &\overset{\eqref{eq:smooth}}{\leq} 2\eta_l^2 L^2 \xi \\
\end{align*}
For the second term,
\begin{eqnarray*}
    && \frac{2\eta_l^2}{W} \sum_{i \in [n], e, j} \frac{\tw_i}{q_i b_i^3} \EE{r-1}{\norm*{(e-1) |\cD_i| \nabla f_i (x) + \sum_{c=1}^j \nabla f_{i\Pi_{i,e,c-1}}(x)}^2} \\
    &\overset{\eqref{eq:var_decomposition}}{\leq}& \frac{2\eta_l^2}{W} \sum_{i \in [n], e, j} \frac{\tw_i}{q_i b_i^3} \\
    && \quad \EE{r-1}{\norm*{((e-1) |\cD_i| + j) \nabla f_i (x)}^2 + \norm*{\sum_{c=1}^j \nabla f_{i\Pi_{i,e,c-1}}(x) - j\nabla f_i (x)}^2}\\
    &\overset{\eqref{eq:sampling_wo_replacement}}{\leq}& \frac{2\eta_l^2}{W} \sum_{i \in [n], e, j} \frac{\tw_i}{q_i b_i^3}  \\
    && \quad \rbr*{((e-1) |\cD_i| + j)^2\norm*{\nabla f_i (x)}^2 + \frac{j(|\cD_i| - j)}{(|\cD_i| - 1)} \frac{1}{|\cD_i|}\sum_{c=1}^{\cD_i} \norm*{\nabla f_{ic}(x) - \nabla f_i (x)}^2}\\
    &\overset{\eqref{eqn:bounded_var}}{\leq}& \frac{2\eta_l^2}{W} \sum_{i \in [n], e, j} \frac{\tw_i}{q_i b_i^3}  \\
    && \quad \rbr*{((e-1) |\cD_i| + j)^2\norm*{\nabla f_i (x)}^2 + \frac{j(|\cD_i| - j)}{(|\cD_i| - 1)}(\sigma^2 + P^2 \norm*{\nabla f_i (x)}^2)}\\
    &\leq& 2\eta_l^2 \sum_{i \in [n]} \frac{|\cD_i|^2 E_i^2}{b_i^2} \hw_i \norm*{\nabla f_i (x)}^2 + \frac{|\cD_i + 1|)}{6 b_i^2} \hw_i(\sigma_i^2 + P_i^2 \norm*{\nabla f_i (x)}^2)\\
    &\leq& 2\eta_l^2 \sum_{i \in [n]} \rbr*{1 + \frac{P_i^2}{3|\cD_i|E_i^2}} \hw_i \norm*{\nabla f_i (x)}^2 + \hw_i \frac{\sigma^2_i}{3|\cD_i|E_i^2}\\
     &\overset{\eqref{eq:heterogeneity-convex}}{\leq}& 4LB^2 \eta_l^2\rbr*{1 + P^2} (\hf(x) - \hf^\star) + 2\eta_l^2 \rbr*{\rbr*{1 + P^2}G^2 + \sigma^2}.
\end{eqnarray*}
Combining the upper bounds
\begin{align*}
\xi &\leq  2\eta_l^2 L^2 \xi +  4LB^2 \eta_l^2\rbr*{1 + P} (\hf(x) - \hf^\star) + 2\eta_l^2 \rbr*{\rbr*{1 + P^2}G^2 + \sigma^2}.
\end{align*}
Since $2\eta_l^2 L^2 \leq \frac{1}{8}$ and $4LB^2 \eta_l^2\rbr*{1 + P^2} \leq \frac{1}{4L}$, therefore
\begin{align*}
3L\tilde\eta \xi &\leq  \frac{9}{10} \tilde\eta(\hf(x) - \hf^\star) + 9\tilde\eta \eta_l^2 (\rbr*{1 + P^2}G^2 + \sigma^2),
\end{align*}
which concludes the proof.
\end{proof}

Adding the statements of the above Lemmas \ref{lem:fedavg-progress_gen} and \ref{lem:fedavg-error-bound_gen}, we get
\begin{align*}
    \E{\norm{x^{r} - \hx}^2 }&\leq (1 - \frac{\mu\tilde\eta}{2}) \E {\norm{x^{r-1} - \hx}^2} - \frac{\tilde\eta}{10} \EE{r-1}{\hf(x^{r-1}) - \hf^\star} \\
    &\quad + 2\tilde\eta^2 M_1G^2 + \frac{9\tilde\eta^3}{\eta_g^2} (\rbr*{1 + P^2}G^2 + \sigma^2) \\
      &= (1 - \frac{\mu\tilde\eta}{2}) \E {\norm{x^{r-1} - \hx}^2} - \frac{\tilde\eta}{10} \EE{r-1}{f(x^{r-1}) - \hf^\star} \\
    &\quad + 2\tilde\eta^2 \rbr*{M_1G^2 + \frac{9\tilde\eta}{2\eta_g^2} (\rbr*{1 + P^2}G^2 + \sigma^2)}\,,
\end{align*}
Moving the $(f(x^{r-1}) - f(x^\star))$ term and dividing both sides by $\frac{\tilde\eta}{10}$, we get the following bound for any $\tilde\eta \leq \frac{1}{(1 + M_2 + (P+1) B + M_1 B^2))4L }$
\begin{align*}
    \EE{r-1}{f(x^{r-1}) - f^\star} &\leq  \frac{10}{\tilde\eta}(1 - \frac{\mu\tilde\eta}{2}) \E {\norm{x^{r-1} - x^\star}^2} \\
    &\quad + 20\tilde\eta \rbr*{M_1G^2 + \frac{9\tilde\eta}{2\eta_g^2} (\rbr*{1 + P^2}G^2 + \sigma^2)}\,.
\end{align*}
If $\mu = 0$ (weakly-convex), we can directly apply Lemma~\ref{lemma:general}. Otherwise, by averaging using weights $v_r = (1  - \frac{\mu \tilde\eta}{2})^{1-r}$ and using the same weights to pick output $\bar{x}^R$, we can simplify the above recursive bound (see proof of Lemma~\ref{lem:constant}) to prove that for any $ \tilde\eta$ satisfying $\frac{1}{\mu R} \leq \tilde\eta \leq \frac{1}{(1 + M_2 + (P+1) B + M_1 B^2))4L }$
\begin{align*}
    \E{f(\bar{x}^{R})] - f(x^\star)} &\leq \underbrace{10 \norm{x^0 - x^\star}^2}_{\eqdef d}\mu\exp(-\frac{\tilde\eta}{2}\mu R) \\
    &\quad + \tilde\eta \rbr[\big]{\underbrace{4M_1G^2}_{\eqdef c_1}} +  \tilde\eta^2 \cbr*{\underbrace{\frac{18}{\eta_g^2} (\rbr*{1 + P^2}G^2 + \sigma^2)}_{\eqdef c_2}}
\end{align*}
Now, the choice of $\tilde \eta =\min\cbr*{ \frac{\log(\max(1,\mu^2 R d/c_1))}{\mu R}, \frac{1}{(1 + M_2 + (P+1) B + M_1 B^2))4L }}$ yields the desired rate. 

For the non-convex case, one first exploits the smoothness assumption (Assumption~\ref{ass:smoothness}) (extra smoothness term $L$ in the first term in the convergence guarantee) and the rest of the proof follows essentially in the same steps as the provided analysis. The only difference is that distance to an optimal solution is replaced by functional difference, i.e., $\norm{x^0 - \hx}^2 \rightarrow \hf(x^0) - \hf^\star$. The final convergence bound also relies on Lemma~\ref{lemma:general}. For completeness, we provide the proof below.

We adapt the same notation simplification as for the prior cases, see proof of Lemma~\ref{lem:fedavg-progress_gen}. Since $\cbr{f_{ij}}$ are $L$-smooth then $f$ is also $L$ smooth. Therefore,

\begin{eqnarray*}
    \EE{r-1}{\hf(x + \Delta)}  &\leq& \hf(x) + \EE{r-1}{\inp*{\nabla \hf(x)}{\Delta}} + \frac{L}{2}\EE{r-1}{\norm*{\Delta}^2} \\
    &=& \hf(x) - \EE{r-1}{\inp*{\nabla \hf(x)}{\frac{\tilde \eta}{W}\sum_{i \in [n] ,e,j}\frac{\tw_i}{q_i b_i} g_{i,e,j}}} \\ 
    &&\quad + \frac{L\tilde\eta^2}{2}\EE{r-1}{\norm*{\frac{1}{W}\sum_{i \in \cS ,e,j}\frac{\tw_i}{q^\cS_i b_i} g_{i,e,j}}^2} \\
    &\overset{\eqref{eq:triangle}}{\leq}& \hf(x) - \frac{\tilde \eta}{2} \norm*{\nabla \hf(x)}^2 + \frac{L\tilde\eta^2}{2}\EE{r-1}{\norm*{\frac{1}{W}\sum_{i \in \cS ,e,j}\frac{\tw_i}{q^\cS_i b_i} g_{i,e,j}}^2} \\
    &&\quad  + \frac{\tilde \eta}{2} \EE{r-1}{\norm*{\frac{1}{W}\sum_{i \in [n] ,e,j}\frac{\tw_i}{q_i b_i} \rbr*{g_{i,e,j} - \nabla f_{i\Pi_{i,e,j-1}}(x)} }^2} \\
    &\overset{\eqref{eq:smooth} + \eqref{eq:sum_upper}}{\leq}& \hf(x) - \frac{\tilde \eta}{2} \norm*{\nabla \hf(x)}^2 + \frac{\tilde \eta L^2}{2} M_2\xi  \\ 
    && \quad + \frac{L\tilde\eta^2}{2}\EE{r-1}{\norm*{\frac{1}{W}\sum_{i \in \cS ,e,j}\frac{\tw_i}{q^\cS_i b_i} g_{i,e,j}}^2}\;.
\end{eqnarray*}
We upper-bound the last term using the bound of $\cA_2$ in the proof of Lemma~\ref{lem:fedavg-progress} (note that this proof does not rely on the convexity assumption). Thus, we have
\begin{align*}
    \EE{r-1}{\hf(x + \Delta)}  &\leq \hf(x) - \frac{\tilde \eta}{2} \norm*{\nabla \hf(x)}^2 + \frac{\tilde \eta L^2}{2} M_2\xi + \rbr*{M_1 + M_2}\tilde\eta^2 L^3 \xi  \\ 
    & \quad + \tilde\eta^2 M_1 G^2 L + \tilde\eta^2 \rbr*{1 + M_1B^2}L  \norm*{\nabla \hf(x)}^2\;.
\end{align*}
The bound on the step size $\tilde \eta \leq \frac{1}{(1 + M_2 + (P+1) B + M_1 B^2))4L}$ implies
\begin{align*}
    \EE{r-1}{f(x + \Delta)}  &\leq f(x) - \frac{\tilde \eta}{4} \norm*{\nabla f(x)}^2 + \frac{3\tilde \eta L^2}{4} \xi + \tilde\eta^2 M_1 G^2 L\;.
\end{align*}
Next, we reuse the partial result of Lemma~\ref{lem:fedavg-error-bound} that does not require convexity, i.e., we replace $\hf(x) - \hf^\star$ with $\frac{1}{2L}\norm{\nabla \hf(x)}^2$. Therefore,
\begin{align*}
    \EE{r-1}{\hf(x + \Delta)}  &\leq \hf(x) - \frac{\tilde \eta}{8} \norm*{\nabla \hf(x)}^2 + \frac{9\tilde\eta^3 L}{4\eta_g^2} (\rbr*{1 + P^2}G^2 + \sigma^2) + \tilde\eta^2 M_1 G^2 L\;.
\end{align*}
By adding $\hf^\star$ to both sides, reordering, dividing by $\tilde \eta$ and taking full expectation, we obtain
\begin{align*}
    \E{\norm*{\nabla \hf(x^r)}^2}  &\leq \frac{8}{\tilde \eta}\rbr*{\E{\hf(x^r) - \hf^\star} - \E{\hf(x^{r+1}) - \hf^\star}}  \\ 
    & \quad + \tilde\eta \underbrace{8 M_1 G^2 L}_{c_1} + \tilde\eta^2 \underbrace{\frac{18 L}{\eta_g^2} (\rbr*{1 + P^2}G^2 + \sigma^2)}_{c_2}\;.
\end{align*}
Applying Lemma~\ref{lemma:general} concludes the proof.

\subsection{General variance bound}
\label{sec:gen_bound}

In this section, we provide a more general version of Lemma~\ref{LEM:UPPERV}. We recall the definition of the indicator function used in the proof of the aforementioned lemma, where $1_{i\in \cS} = 1$ if $i\in \cS$ and $1_{i\in \cS} = 0$ otherwise and, likewise, $1_{i,j\in \cS} = 1$ if $i,j\in \cS$ and  $1_{i,j\in \cS} = 0$ otherwise. Further, let $\mH$ be $n \times n$ matrix, where $\mH_{i,j} = \EE{\cS}{\frac{q_i q_j}{q^\cS_i q^\cS_j} 1_{i,j\in \cS}}$ and $h$ is its diagonal with $h_i =  \EE{\cS}{\frac{q_i^2}{\rbr{q^\cS_i}^2} 1_{i\in \cS}}$. We are ready to proceed with the lemma.

\begin{lemma}
\label{LEM:UPPERV_fedshuffle_gen}
Let $\zeta_1,\zeta_2,\dots,\zeta_n$ be vectors in $\mathbb{R}^d$ and $w_1,w_2,\dots,w_n$ be non-negative real numbers such that $\sum_{i=1}^n w_i=1$. Define $\Tilde{\zeta} \eqdef \sum_{i=1}^n \frac{w_i}{q_i} \zeta_i$. Let $\cS$ be a proper sampling. If $s\in\mathbb{R}^n$ is  such that
        \begin{equation}\label{eq:ESO_fedshuffle_gen}
            \mH - ee^\top \preceq {\rm \bf Diag}(h_1 v_1, h_2 v_2,\dots, h_n v_n),
        \end{equation}
        then
        \begin{equation}\label{eq:var_arbitrary_gen}
            \E{ \norm*{ \sum \limits_{i\in \cS} \frac{w_i\zeta_i}{g^\cS_i} - \Tilde{\zeta} }^2 } \leq \sum \limits_{i=1}^n h_iv_i \norm*{\frac{w_i\zeta_i}{q_i}}^2,
        \end{equation}
        where the expectation is taken over $\cS$  and $e$ is the vector of all ones in $\mathbb{R}^n$. 
\end{lemma}

\begin{proof}
 Let us first compute the mean of $ X \eqdef \sum_{i\in \cS}\frac{w_i \zeta_i}{q^\cS_i}$:
    \begin{align*}
            \E{X} 
            = \E{ \sum_{i\in \cS}\frac{w_i \zeta_i}{q^\cS_i} } 
            = \E{ \sum_{i=1}^n \frac{w_i \zeta_i}{q^\cS_i} 1_{i\in \cS} } 
            = \sum_{i=1}^n w_i \zeta_i \E{\frac{1}{q^\cS_i} 1_{i\in \cS}}  
            = \sum_{i=1}^n \frac{w_i}{q_i} \zeta_i
            = \Tilde{\zeta}.
    \end{align*}
    Let $\boldsymbol{A}=[a_1,\dots,a_n]\in \mathbb{R}^{d\times n}$, where $a_i=\frac{w_i\zeta_i}{q_i}$. We now write the variance of $X$ in a form which will be convenient to establish a bound:
    \begin{equation}\label{eq:variance_of_X_fedshuffle_fedshuffle_gen}
        \begin{split}
            \E{\norm{X - \E{X}}^2}
            &= \E{\norm{X}^2} - \norm{ \E{X} }^2 \\
            &=\E{\norm{\sum_{i\in \cS}\frac{w_i \zeta_i}{q^\cS_i}}^2} - \norm{\Tilde{\zeta}}^2 \\
            &=\E{ \sum_{i,j} \frac{w_i\zeta_i^\top}{q^\cS_i}\frac{w_j\zeta_j}{q^\cS_j}  1_{i,j\in \cS} } - \norm{\Tilde{\zeta}}^2 \\
            &= \sum_{i,j} \mH_{ij} \frac{w_i\zeta_i^\top}{q_i}\frac{w_j\zeta_j}{q_j} - \sum_{i,j} \frac{w_i\zeta_i^\top}{q_i}\frac{w_j\zeta_j}{q_j} \\
            &= \sum_{i,j} (\mH_{ij}-1)a_i^\top a_j \\
            &= e^\top ((\mH -ee^\top) \circ \boldsymbol{A}^\top\boldsymbol{A})e.
        \end{split}
    \end{equation}
    
    Since, by assumption, we have $\mH -ee^\top \preceq \textbf{Diag}(h\circ s)$, we can further bound
    \begin{align*}
        e^\top ((\mH -ee^\top) \circ \boldsymbol{A}^\top\boldsymbol{A})e \leq e^\top (\textbf{Diag}(h\circ s)\circ \boldsymbol{A}^\top\boldsymbol{A}) e = \sum_{i=1}^n h_iv_i\norm{a_i}^2.
    \end{align*}
\end{proof}

\clearpage
\section{Experimental setup and extra experiments}
\label{sec:setup}

\begin{table}[t]
\renewcommand{\arraystretch}{1.}
\centering
\caption{Baselines.}
\resizebox{\columnwidth}{!}{
\begin{tabular}{|c|c|c|c|}
\hline
 & \multirow{2}{*}{Local Steps/ Epochs} & W/ or W/O Replacement  & \multirow{2}{*}{Pseudocode} \\ 
 &  & Gradient Sampling$^a$&  \\ \hline 
\fedavgrr          & Epochs                       & W/O                                                                        &     Algorithm~\ref{alg:FedAvg}       \\ \hline
\multirow{2}{*}{\fedmin}           & \multirow{2}{*}{Steps}                        & \multirow{2}{*}{W/}                                                                       &    Algorithm 1 \cite{reddi2020adaptive}   \\
 & & & with $K = K_{\text{min}}^b$\\ \hline
\multirow{2}{*}{\fedmean}          & \multirow{2}{*}{Steps}                        & \multirow{2}{*}{W/}                                                                       &      Algorithm 1  \cite{reddi2020adaptive}  \\
& & &  with $K = K_{\text{mean}}^c$ \\ \hline
\fednova          & Epochs                       & W/                                                                       &    \cite{fednova2020neurips}         \\ \hline
\multirow{2}{*}{\fednovarr}       & \multirow{2}{*}{Epochs}                       & \multirow{2}{*}{W/O}  &    Algorithm~\ref{alg:FedGen} with \\
& & & \fednova aggregation       \\ \hline
\fedshuffle       & Epochs                       & W/O                                                                       &           Algorithm~\ref{alg:FedShuffle_simple}  \\ \hline
\end{tabular} 
}\\
\raggedright
{\footnotesize $^a$ For real-world datasets, as it is commonly implemented in practice, all methods use without replacement sampling, i.e., random reshuffling.} \\
{\footnotesize $^b$ $K = K_{\text{min}}$ corresponds to the minimal number of steps across clients for \fedavgrr.}
\\
{\footnotesize $^c$ $K = K_{\text{mean}}$ corresponds to the average number of steps across clients for \fedavgrr, i.e.,  \fedmean and \fedavgrr perform the same total computation.} 
\end{table}

As mentioned in the previous section, we compare three algorithms -- \fedavg, \fednova, and our \fedshuffle. For \fedavg, we include an extra baseline that fixes objective inconsistency by running the same number of local steps per each client. To avoid extra computational burden on clients, i.e., to avoid stragglers, we select the number of local steps to be the minimal number of steps across clients. We refer to this baseline as \fedmin. Another \fedavg-type baseline that we compare to is \fedmean. Similarly to \fedmin, \fedmean\ runs the same number of local steps per each device with a difference that the number of steps is selected such that the total computation performed by \fedmean\ is equal to \fedavg\ with local epochs. We include this theoretical baseline to see the effect of heterogeneity in local steps and whether it is beneficial to run full epochs. We consider three extensions: (i) random reshuffling for \fedavg, and \fednova, since these methods were not originally analysed using biased random reshuffling but rather with unbiased stochastic gradients obtained by sampling with replacement, (ii) global momentum.
Our implementation of \fedavg and \fednova follows \cite{reddi2020adaptive} and \cite{fednova2020neurips}, respectively. In all of the experiments, the reported performance is average with one standard deviation over 3 runs with fixed seeds across methods for a fair comparison.

\begin{figure}[t]
    \centering
    \subfigure[Shakespeare w/ LSTM]{
        \includegraphics[width=0.35\textwidth]{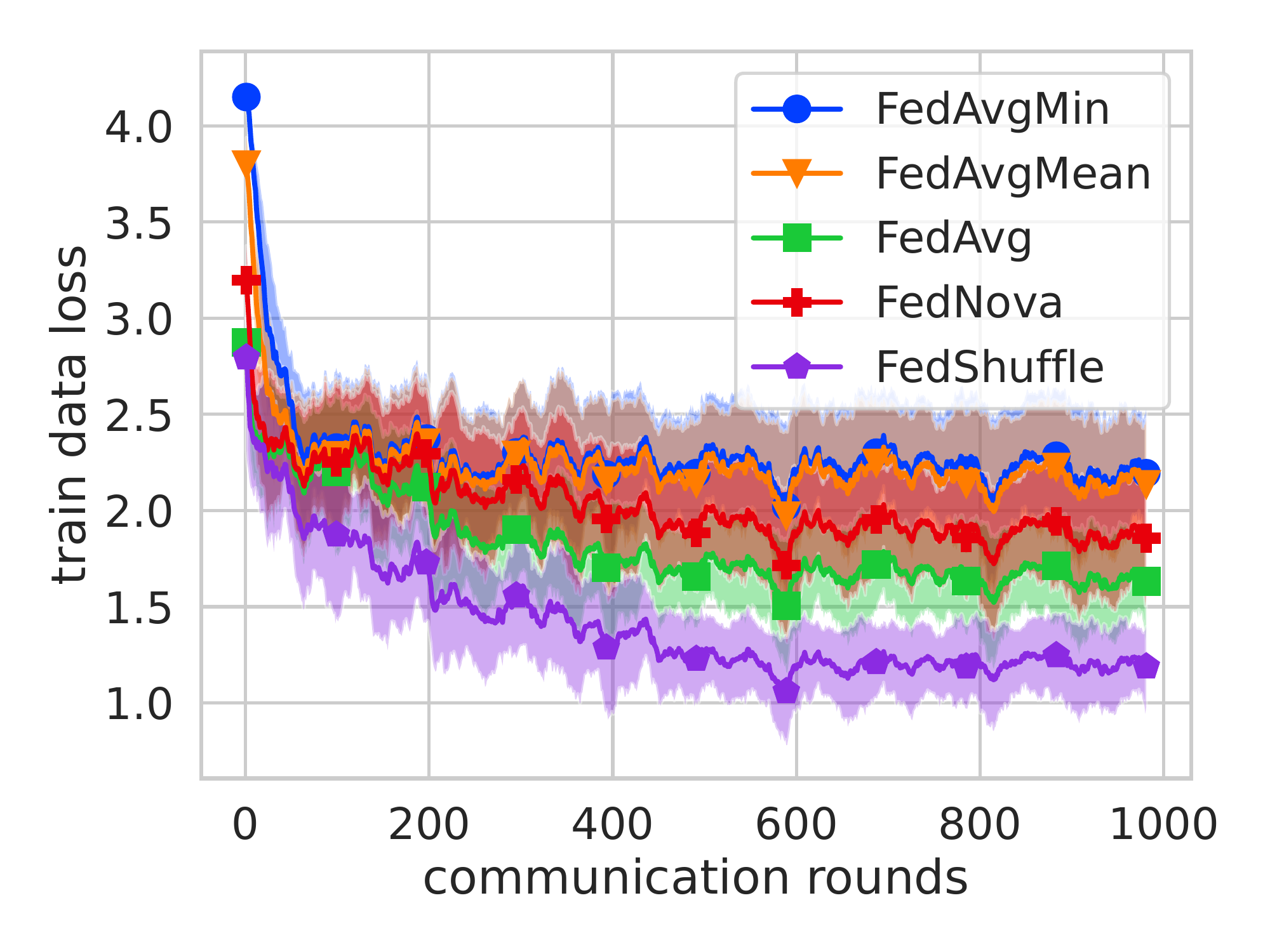}
        \includegraphics[width=0.35\textwidth]{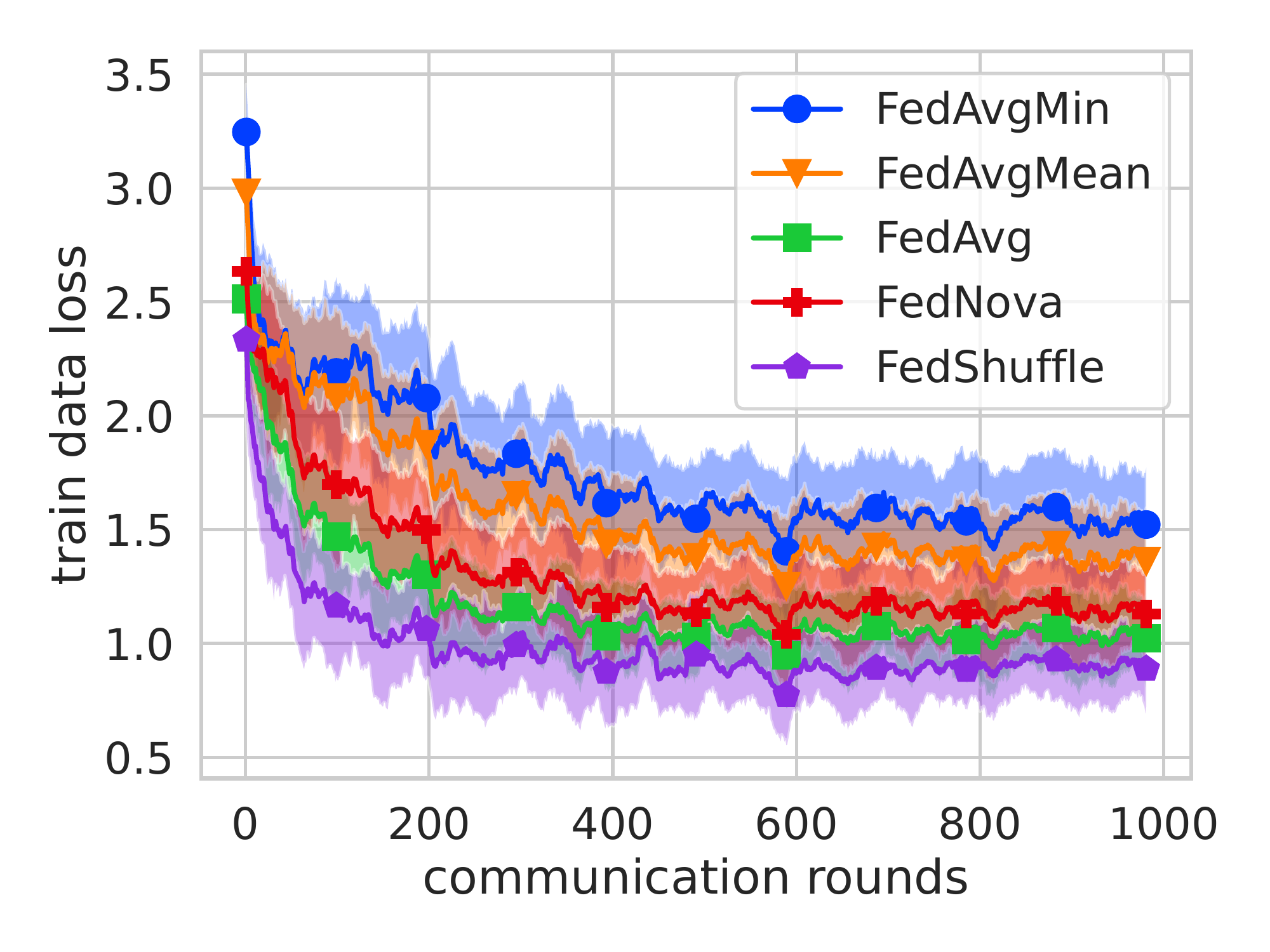}
        }
    \hfill
    \subfigure[CIFAR100 (TFF Split) w/ ResNet18]{
        \includegraphics[width=0.35\textwidth]{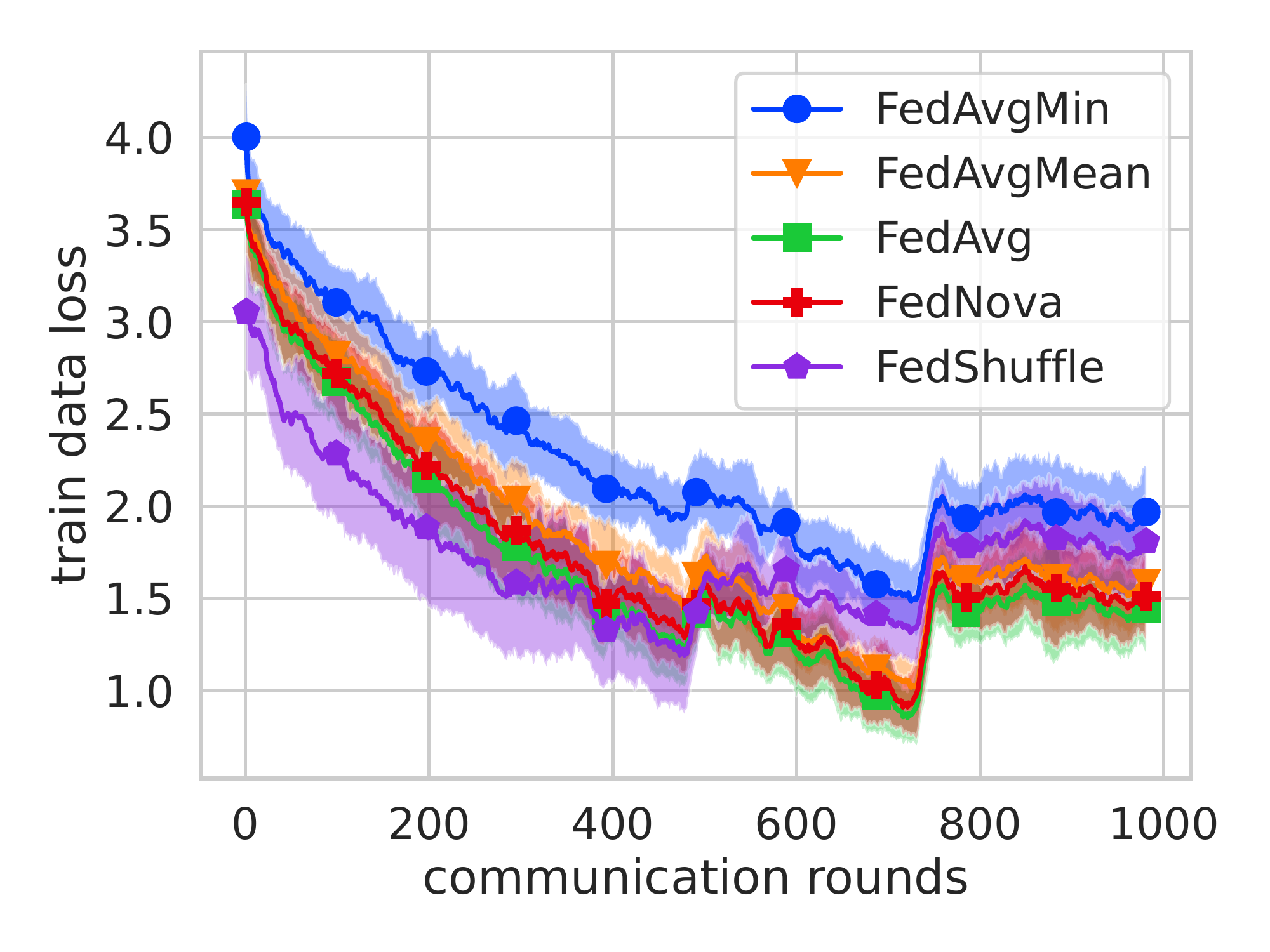} 
        \includegraphics[width=0.35\textwidth]{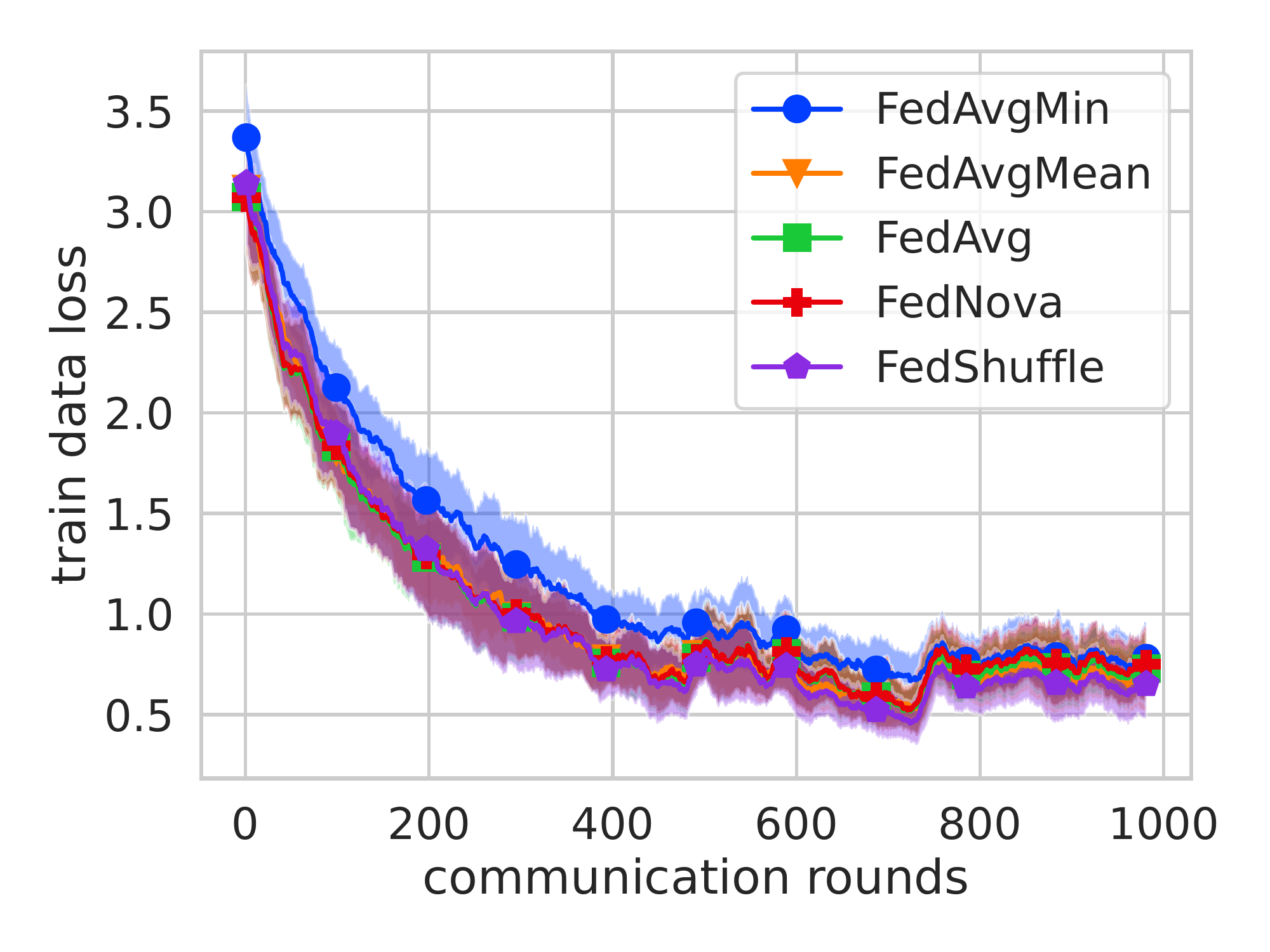}
        }
    \caption{Comparison of the moving average with slide 20 of the train data loss on \fedmin, \fedavg, \fednova, \fedshuffle on real-world datasets. 
    Partial participation: in each round 16 client is sampled uniformly at random.
    All methods use random reshuffling. For Shakespeare, number of local epochs is $2$ and for CIFAR100, it is 2 to 5 sampled uniformly at random at each communication round for each client. Left: Plain methods.
    Right: Global momentum $0.9$. \textit{Note that the displayed loss is only for the train data, and it does not include the l2 penalty; therefore, its informative value is limited}.
    }
    \label{fig:neural_nets_data_loss}
\end{figure}

\textbf{Hyperparameters.} The server learning rate for all methods is kept at its default value of $1$. For all momentum methods, we pick momentum $0.9$. The rest of the parameters for the simple quadratic objective are selected based on the theoretical values for all the methods. For the other experiments (Shakespeare and CIFAR100), the local learning rate is tuned by searching over a grid $\cbr{0.1, 0.01, 0.001}$ and we use a standard learning schedule, where the step size is decreased by $10$ at $50\%$ and $75\%$ of all communication rounds. We note that \fedshuffle uses different local step sizes for each client $\eta_{l,i}^r = \nicefrac{\eta_l^r}{|\cD_i|}$.
For a fair comparison, we take the local learning rate for \fedshuffle to be the local learning rate for the client with the largest number of local steps as this follows directly from the theory. The global momentum as defined in \eqref{eq:mvr_loc_update} and \eqref{eq:mvr_mom_update} is only used in this form for the toy experiment. For the real-word experiments, we ignore the last term and we approximate $\nabla f_{i}(x^r) \approx \frac{1}{E|\cD_i|\eta_{l,i}^r} \Delta_i^r$ to avoid extra communication and computation. Note that for both \fednova\ and \fedshuffle, the estimator of $\nabla f(x^r)$ is proportional to $\Delta^r$ and, therefore, the server can directly update the global momentum using only this update. For \fedavg , this is not the case and it is closely related to the objective inconsistency issue. Despite this fact, we provide \fedavg\ with the proper momentum estimator that improves its performance as it reduces the objective inconsistency. 
Lastly, for neural network experiments we allow each method to benefit from the weight decay $\cbr{0, 1e-4}$ and the gradient clipping $\cbr{5, \infty}$, i.e., we do a grid search over all the combinations of step sizes, weight decays and gradient clippings. For the toy experiment, the batch size is $1$ and for neural nets, we select batch size to be $32$.

\textbf{Datasets and models.} We run experiments on three datasets with three corresponding models. First, we consider a simple quadratic objective 
\begin{equation}
    \label{eq:quad_obj}
    \min_{x \in \R^6} \frac{1}{12} \sum_{i=1}^6 \norm{x - e_i}^2,
\end{equation}
where $e_i$ is the canonical basis vector with $1$ at $i$-th coordinate and $0$ elsewhere. We partition data into 3 clients, where the first client owns the first data point, the second is assigned $e_2$ and $e_3$ and the rest belongs to the third client. 

For the second task, we evaluate a next character prediction model on the Shakespeare dataset~\cite{leaf}. This dataset consists of 715 users (characters of Shakespeare plays), where each example corresponds to a contiguous set of lines spoken by the character in a given play. For the model, an input character is first transformed into an 8-dimensional vector using an embedding layer, and this is followed by an LSTM~\cite{hochreiter1997long} with two hidden layers and 512-dimensional hidden state.

Last, we consider an image classification task on the CIFAR100 dataset~\cite{cifar} with a ResNet18 model~\cite{resnet}, where we replace batch normalization layers with group normalization following ~\cite{hsieh2020noniid}. The data is partitioned across 500 clients, where each client owns 100 data points using a hierarchical Latent Dirichlet Allocation (LDA) process~\cite{li2006pachinko}.  We use equivalent splits to those provided in  TensorFlow Federated datasets~\cite{tffdatasets}.

Our implementation is in PyTorch~\cite{pytorch}.

For the first experiment, we only consider performance on the train set (i.e., training loss), while for the real-world datasets we report performance on the corresponding validation datasets (validation accuracy).


\refstepcounter{chapter}%
\chapter*{\thechapter \quad Papers Submitted and Under Preparation (In Order of Preparation)}

\begin{enumerate}
    \item \cite{horvath2018nonconvex} Samuel Horv\'{a}th, and Peter Richt\'{a}rik, {``Nonconvex variance reduced optimization with arbitrary sampling"}, \emph{International Conference on Machine Learning}, 2019. \textbf{The Best Poster Prize}: Data Science Summer School (DS$^3$), École Polytechnique, Paris, 2018.
    \item \cite{Kovalev2019:svrg} Dmitry Kovalev, Samuel Horv\'{a}th, and Peter Richt\'{a}rik,, {``Don’t jump through hoops and remove those loops: SVRG and Katyusha are
    better without the outer loop"}, \emph{International Conference on Algorithmic Learning Theory } (ALT 2020), 2019.
    \item \cite{diana2}  Samuel Horv\'{a}th, Dmitry Kovalev, Konstantin Mishchenko, Sebastian Stich, and Peter Richt\'{a}rik, {``Stochastic distributed learning with gradient quantization and variance reduction"}, arXiv preprint arXiv:1904.05115, 2019. \textbf{The Best Poster Prize}: Control, Information and Optimization Summer School, Voronovo, Russia, presented by D. Kovalev.
    \item \cite{Cnat} Samuel Horv\'{a}th, Chen-Yu Ho, Ľudov\'{i}t Horv\'{a}th, Atal Narayan Sahu, Marco Canini, and Peter Richt\'{a}rik, {``Natural compression for distributed deep learning"}, arXiv preprint arXiv:1905.10988, 2019. Workshop on AI Systems at Symposium on Operating Systems Principles (SOSP 2019).
    \item \cite{beznosikov2020biased} Aleksandr Beznosikov, Samuel Horv\'{a}th, Peter Richt\'{a}rik, and Mher Safaryan, {``On biased compression for distributed learning"}, arXiv preprint arXiv:2002.12410, 2020. \textbf{Oral Presentation}: \emph{Workshop on Scalability, Privacy, and Security in Federated Learning} (NeurIPS 2020).
    \item \cite{horvath2020adaptivity} Samuel Horv\'{a}th, Lihua Lei, Peter Richt\'{a}rik, and Michael I. Jordan, {``Adaptivity of stochastic gradient methods for nonconvex optimization"}, \emph{SIAM Journal on Mathematics of Data Science (SIMODS)}, 2022. \textbf{Spotlight}: \emph{Optimization for Machine Learning Workshop} (NeurIPS 2020.
    \item \cite{horvath2021a} Samuel Horv\'{a}th, and Peter Richt\'{a}rik, {``A better alternative to error feedback for communication-efficient distributed learning"}, in \emph{International Conference on Learning Representations} (ICLR 2021), 2020. \textbf{The Best Paper Award}: \emph{Workshop on Scalability, Privacy, and Security in Federated Learning} (NeurIPS 2020).
    \item \cite{hanzely2020lower} Filip Hanzely, Slavom\'{i}r Hanzely, Samuel Horv\'{a}th, and Peter Richt\'{a}rik, {``Lower bounds and optimal algorithms for personalized federated learning"}, in \emph{ Conference on Neural Information Processing Systems} (NeurIPS 2020), 2020.
    \item \cite{chen2020optimal} Wenlin Chen, Samuel Horv\'{a}th, and Peter Richt\'{a}rik, {``Optimal client sampling for federated learning"}, arXiv preprint arXiv:2010.13723, 2020. \emph{Privacy Preserving Machine Learning Workshop} (NeurIPS 2020).
    \item \cite{horvath2021hyperparameter} Samuel Horv\'{a}th, Aaron Klein, Peter Richt\'{a}rik, and Cedric Archambeau, {``Hyperparameter transfer learning with adaptive complexity"}, in \emph{International Conference on Artificial Intelligence and Statistics} (AISTATS 2021), 2021. 
    \item \cite{horvath2021fjord} Samuel Horv\'{a}th, Stefanos Laskaridis, Mario Almeida, Ilias Leontiadis, Stylianos I. Venieris, Nicholas D. Lane, {``FjORD: Fair and accurate federated learning under heterogeneous targets with ordered dropout"}, in \emph{Conference on Neural Information Processing Systems} \textbf{Spotlight} (NeurIPS 2021), 2021. \emph{International Workshop on Federated Learning for User Privacy and Data Confidentiality} (ICML 2021).
    \item \cite{gasanov2021flix} Elnur Gasanov, Ahmed Khaled, Samuel Horv\'{a}th and Peter Richt\'{a}rik, {``FLIX: A simple and communication-efficient alternative to local methods in federated learning''}, in \emph{International Conference on Artificial Intelligence and Statistics} (AISTATS 2022), 2021. \emph{International Workshop on Federated Learning for User Privacy and Data Confidentiality} (ICML 2021).
    \item \cite{wang2021field} Wang et al. (50+ authors including Samuel Horv\'{a}th), {``A field guide to federated optimization"}, arXiv preprint arXiv:2107.06917, 2021.
    \item \cite{pogran2021long} Edita Pogran, Ahmed Abd El-Razek, Laura Gargiulo, Valerie Weihs, Christoph Kaufmann, Samuel Horv\'{a}th, Alexander Geppert, Michael Nürnberg, Emil Wessely, Peter Smetana, and Kurt Huber, {``Long-term outcome in patients with takotsubo syndrome"}, in \emph{Wiener klinische Wochenschrift}, 2021.
    \item \cite{burlachenko2021fl_pytorch} Konstantin Burlachenko, Samuel Horv\'{a}th, and Peter Richt\'{a}rik, {``FL\_PyTorch: Optimization research simulator for federated learning"}, in \emph{ACM International Workshop on Distributed Machine Learning}, 2021.
    \item \cite{fedshuffle} Samuel Horv\'{a}th, Maziar Sanjabi, Lin Xiao, Peter Richt\'{a}rik, and Michael Rabbat {``FedShufffle: Recipes for better use of local work in federated learning"}, arXiv preprint arXiv:2204.13169, 2022.
\end{enumerate}


\end{document}